\documentclass[openany]{now} 
\pdfoutput=1
\usepackage{enumerate}
\usepackage{graphicx} \usepackage{bmpsize} 

\graphicspath{{./images/}}
\usepackage{mystyle}
\usepackage{amsmath,amssymb}
\usepackage{notations}

\floatstyle{ruled}
\newfloat{boxfloat}{h}{boxfloat}[chapter]{
\title{Computational Optimal Transport}
\author{%
Gabriel Peyr\'e\\CNRS and DMA, ENS
\and 
Marco Cuturi \\Google and CREST, ENSAE}}

\begin{document}

\maketitle

\begin{verbatim}@article{COTFNT,
year = {2019},
volume = {11},
journal = {Foundations and Trends in Machine Learning},
title = {Computational Optimal Transport},
number = {5-6},
pages = {355--607}
author = {Gabriel Peyr\'e and Marco Cuturi}
}\end{verbatim}

\tableofcontents

\mainmatter


\begin{abstract}
Optimal transport (OT) theory can be informally described using the words of the French mathematician Gaspard Monge (1746--1818): A worker with a shovel in hand has to move a large pile of sand lying on a construction site. The goal of the worker is to erect with all that sand a target pile with a prescribed shape (for example, that of a giant sand castle). Naturally, the worker wishes to minimize her total effort, quantified for instance as the total distance or time spent carrying shovelfuls of sand.
Mathematicians interested in OT cast that problem as that of comparing two probability distributions---two different piles of sand of the same volume. They consider all of the many possible ways to morph, \emph{transport} or reshape the first pile into the second, and associate a ``global'' cost to every such transport, using the ``local'' consideration of how much it costs to move a grain of sand from one place to another. Mathematicians are interested in the properties of that least costly transport, as well as in its efficient computation.  
That smallest cost not only defines a distance between distributions, but it also entails a rich geometric structure on the space of probability distributions. That structure is canonical in the sense that it borrows key geometric properties of the underlying ``ground'' space on which these distributions are defined. For instance, when the underlying space is Euclidean, key concepts such as interpolation, barycenters, convexity or gradients of functions extend naturally to the space of distributions endowed with an OT geometry.



OT has been (re)discovered in many settings and under different forms, giving it a rich history. While Monge's seminal work was motivated by an engineering problem, Tolstoi in the 1920s and Hitchcock, Kantorovich and Koopmans in the 1940s established its significance to logistics and economics. Dantzig solved it numerically in 1949 within the framework of linear programming, giving OT a firm footing in optimization. OT was later revisited by analysts in the 1990s, notably Brenier, while also gaining fame in computer vision under the name of earth mover's distances.
Recent years have witnessed yet another revolution in the spread of OT, thanks to the emergence of approximate solvers that can scale to large problem dimensions. As a consequence, OT is being increasingly used to unlock various problems in imaging sciences (such as color or texture processing), graphics (for shape manipulation) or machine learning (for regression, classification and generative modeling).

This paper reviews OT with a bias toward numerical methods, and covers the theoretical properties of OT that can guide the design of new algorithms.
We focus in particular on the recent wave of efficient algorithms that have helped OT find relevance in data sciences. We give a prominent place to the many generalizations of OT that have been proposed in but a few years, and connect them with related approaches originating from statistical inference, kernel methods and information theory. 
All of the figures can be reproduced using code made available in a companion website\footnote{\url{https://optimaltransport.github.io/}}. This website hosts the book project Computational Optimal Transport. You will also find slides and computational resources.
\end{abstract}
%

\chapter{Introduction}
\label{c-intro} 

%
%
%
%
%

The shortest path principle guides most decisions in life and sciences: When a commodity, a person or a single bit of information is available at a given point and needs to be sent at a target point, one should favor using the least possible effort. This is typically reached by moving an item along a straight line when in the plane or along geodesic curves in more involved metric spaces. The theory of optimal transport generalizes that intuition in the case where, instead of moving only one item at a time, one is concerned with the problem of moving simultaneously several items (or a continuous distribution thereof) from one configuration onto another. As schoolteachers might attest, planning the transportation of a group of individuals, with the constraint that they reach a given target configuration upon arrival, is substantially more involved than carrying it out for a single individual. Indeed, thinking in terms of groups or distributions requires a more advanced mathematical formalism which was first hinted at in the seminal work of~\citet{Monge1781}. Yet, no matter how complicated that formalism might look at first sight, that problem has deep and concrete connections with our daily life. Transportation, be it of people, commodities or information, very rarely involves moving only one item. All major economic problems, in logistics, production planning or network routing, involve moving distributions, and that thread appears in all of the seminal references on optimal transport. Indeed~\citet{tolstoi1930methods},~\citet{Hitchcock41} and~\citet{Kantorovich42} were all guided by practical concerns. It was only a few years later, mostly after the 1980s, that mathematicians discovered, thanks to the works of~\citet{Brenier91} and others, that this theory provided a fertile ground for research, with deep connections to convexity, partial differential equations and  statistics. At the turn of the millenium, researchers in computer, imaging and more generally data sciences understood that optimal transport theory provided very powerful tools to study distributions in a different and more abstract context, that of comparing distributions readily available to them under the form of bags-of-features or descriptors.

Several reference books have been written on optimal transport, including the two recent monographs by~\citeauthor{Villani03} (\citeyear{Villani03,Villani09}), those by~\citeauthor{rachev1998mass} (\citeyear{rachev1998mass,rachev1998mass2}) and more recently that by~\citet{SantambrogioBook}. As exemplified by these books, the more formal and abstract concepts in that theory deserve in and by themselves several hundred pages. Now that optimal transport has gradually established itself as an applied tool (for instance, in economics, as put forward recently by~\citet{galichon2016optimal}), we have tried to balance that rich literature with a computational viewpoint, centered on applications to data science, notably imaging sciences and machine learning. We follow in that sense the motivation of the recent review by~\citet{kolouri2017optimal} but try to cover more ground.
Ultimately, our goal is to present an overview of the main theoretical insights that support the practical effectiveness of OT and spend more time explaining how to turn these insights into fast computational schemes. 
The main body of Chapters \ref{c-continuous}, \ref{c-algo-basics}, \ref{c-entropic}, \ref{c-variational}, and \ref{c-extensions} is devoted solely to the study of the geometry induced by optimal transport in the space of probability vectors or discrete histograms. 
Targeting more advanced readers, we also give in the same chapters,  in light gray boxes, a more general mathematical exposition of optimal transport tailored for discrete measures. Discrete measures are defined by their probability weights, but also by the location at which these weights are defined. These locations are usually taken in a continuous metric space, giving a second important degree of freedom to model random phenomena.
%
Lastly, the third and most technical layer of exposition is indicated in dark gray boxes and deals with arbitrary measures that need not be discrete, and which can have in particular a density w.r.t. a base measure. This is traditionally the default setting for most classic textbooks on OT theory, but one that plays a less important role in general for practical applications.
Chapters~\ref{c-algo-semidiscr} to~\ref{c-statistical} deal with the interplay between continuous and discrete measures and are thus targeting a more mathematically inclined audience. 

The field of computational optimal transport is at the time of this writing still an extremely active one. There are therefore a wide variety of topics that we have not touched upon in this survey. Let us cite in no particular order the subjects of distributionally robust optimization \citep{NIPS2015_5745,esfahani2018data,NIPS2018_7534,NIPS2018_8015}, in which parameter estimation is carried out by minimizing the worst posssible empirical risk of any data measure taken within a certain Wasserstein distance of the input data; convergence of the Langevin Monte Carlo sampling algorithm in the Wasserstein geometry \citep{dalalyan2017user,pmlr-v65-dalalyan17a,pmlr-v75-bernton18a}; other numerical methods to solve OT with a squared Euclidian cost in low-dimensional settings using the Monge-Amp\`ere equation~\citep{froese2011convergent,benamou2014numerical,sulman2011efficient} which are only briefly mentioned in Remark~\ref{rem:MA}.


\section*{Notation}

\begin{itemize}
	\item $\range{n}$: set of integers $\{1,\dots,n\}$.
	\item $\ones_{n,m}$: matrix of $\RR^{n\times m}$ with all entries identically set to $1$. $\ones_{n}$: vector of ones.
	\item $\Identity_n$: identity matrix of size $n\times n$. 
	\item For $u\in\RR^n$, $\diag(u)$ is the $n\times n$ matrix with diagonal $u$ and zero otherwise.
	\item $\simplex_n$: probability simplex with $n$ bins, namely the set of probability vectors in $\RR^n_+$.	
	\item $(\a,\b)$:  histograms in the simplices $\simplex_n \times \simplex_m$. 
	\item $(\al,\be)$:  measures, defined on spaces $(\X,\Y)$.
	\item $\frac{\d\al}{\d\be}$: relative density of a measure $\al$ with respect to $\be$.
	\item $\density{\al} = \frac{\d\al}{\d x}$: density of a measure $\al$ with respect to Lebesgue measure.
	\item $(\al=\sum_i \a_i \delta_{x_i},\be=\sum_j \b_j \delta_{y_j})$: discrete measures supported on $x_1,\dots,x_n\in\X$ and $y_1,\dots,y_m\in\Y$.
	\item $\c(x,y)$: ground cost, with associated pairwise cost matrix $\C_{i,j}=(\c(x_i,y_j))_{i,j}$ evaluated on the support of $\al,\be$.
	\item $\pi$: coupling measure between $\al$ and $\be$, namely such that for any $A\subset\X, \pi(A\times \Y)= \al(A)$, and for any subset $B\subset \Y, \pi(\X\times B)= \be(B)$. For discrete measures $\pi = \sum_{i,j} \P_{i,j}\delta_{(x_i,y_j)}$.
	\item $\Couplings(\al,\be)$: set of coupling measures, for discrete measures $\CouplingsD(\a,\b)$.
	\item $\Potentials(\c)$: set of admissible dual potentials; for discrete measures $\PotentialsD(\C)$.
	\item $\T : \X \rightarrow \Y$: Monge map, typically such that $\T_\sharp \al = \be$.
	\item $(\al_t)_{t=0}^1$: dynamic measures, with $\al_{t=0}=\al_0$ and $\al_{t=1}=\al_1$.
	\item $\speed$: speed for Benamou--Brenier formulations; $\Moment=\al \speed$: momentum.
	\item $(\f,\g)$: dual potentials, for discrete measures $(\fD,\gD)$ are dual variables.
	\item $(\uD,\vD) \eqdef (e^{\fD/\varepsilon},e^{\gD/\varepsilon})$: Sinkhorn scalings.
	\item $\K \eqdef e^{-\C/\varepsilon}$: Gibbs kernel for Sinkhorn.
	\item $\flow$: flow for $\Wass_1$-like problem (optimization under divergence constraints).
	\item $\MKD_\C(\a,\b)$ and $\MK_\c(\al,\be)$: value of the optimization problem associated to the OT with cost $\C$ (histograms) and $\c$ (arbitrary measures).
	\item $\WassD_p(\a,\b)$ and $\Wass_p(\al,\be)$: $p$-Wasserstein distance associated to ground distance matrix $\distD$ (histograms) and distance $\dist$ (arbitrary measures).
	\item $\lambda\in\simplex_S$: weight vector used to compute the barycenters of $S$ measures.
	\item $\dotp{\cdot}{\cdot}$: for the usual Euclidean dot-product between vectors; for two matrices of the same size $A$ and $B$, $\dotp{A}{B} \eqdef \trace(A^\top B)$ is the Frobenius dot-product.
	\item $\f\oplus \g(x,y) \eqdef f(x)+g(y)$, for two functions $\f : \X \rightarrow \RR, \g : \Y \rightarrow \RR$, defines
		$\f \oplus \g : \X\times \Y \rightarrow \RR$.
	\item $\fD \oplus \gD \eqdef \fD \ones_m^\top + \ones_n \gD^\top \in \RR^{n\times m}$ for two vectors $\fD\in\RR^n$, $\gD\in\RR^m$.
	\item $\al \otimes \be$ is the product measure on $\X \times \Y$, \ie 
		$\int_{\X \times \Y} g(x,y) \d(\al\otimes\be)(x,y) \eqdef \int_{\X \times \Y} g(x,y) \d\al(x) \d\be(y)$.
	\item $\a \otimes \b \eqdef \a \b^\top \in \RR^{n \times m}$. 
	\item $\uD \odot \vD = (\uD_i \vD_i) \in \RR^n$ for $(\uD, \vD) \in (\RR^n)^2$.
\end{itemize}

\chapter{Theoretical Foundations}
\label{c-continuous}

This chapter describes the basics of optimal transport, introducing first the related notions of optimal matchings and couplings between probability vectors $(\a,\b)$, generalizing gradually this computation to transport between discrete measures $(\al,\be)$, to cover lastly the general setting of arbitrary measures. At first reading, these last nuances may be omitted and the reader can only focus on computations between probability vectors, namely histograms, which is the only requisite to implement algorithms detailed in Chapters~\ref{c-algo-basics} and~\ref{c-entropic}. More experienced readers will reach a better understanding of the problem by considering the formulation that applies to arbitrary measures, and will be able to apply it for more advanced problems (\eg in order to move positions of clouds of points, or in a statistical setting where points are sampled from continuous densities).

\section{Histograms and Measures}

We will use interchangeably the terms histogram and probability vector for any element $\a \in \simplex_n$ that belongs to the probability simplex
\eq{
	\simplex_n \eqdef \enscond{\a \in \RR_+^n}{ \sum_{i=1}^n \a_i = 1 }.
}
A large part of this review focuses exclusively on the study of the geometry induced by optimal transport on the simplex. 

\begin{rem1}{Discrete measures}
A discrete measure with weights $\a$ and locations $x_1,\dots,x_n\in\X$ reads
\eql{\label{eq-discr-meas}
	\al = \sum_{i=1}^n \a_i \de_{x_i},
}
where $\de_x$ is the Dirac at position $x$, intuitively a unit of mass which is infinitely concentrated at location $x$. Such a measure describes a probability measure if, additionally, $\a\in\simplex_n$ and more generally a positive measure if all the elements of vector $\a$ are nonnegative. To avoid degeneracy issues where locations with no mass are accounted for, we will assume when considering discrete measures that all the elements of $\a$ are positive.
\end{rem1}

\begin{rem2}{General measures}
A convenient feature of OT is that it can deal with measures that are either or both discrete and continuous within the same framework. To do so, one relies on the set of Radon measures $\Mm(\X)$ on the space $\X$. The formal definition of that set
requires that $\X$ is equipped with a distance, usually denoted $\dist$, because one can access a measure only by ``testing'' (integrating) it against continuous functions, denoted $f \in \Cc(\X)$.

Integration of $f \in \Cc(\X)$ against a discrete measure $\al$ computes a sum
\eq{
	\int_\X f(x) \d\al(x) = \sum_{i=1}^n \a_i f(x_i).
}
More general measures, for instance on $\X=\RR^\dim$ (where $\dim \in \NN^*$ is the dimension), can have a density $\d\al(x)=\density{\al}(x)\d x$ w.r.t. the Lebesgue measure, often denoted $\density{\al} = \frac{\d\al}{\d x}$, which means that
\eq{
	\foralls h \in \Cc(\RR^\dim), \quad
	\int_{\RR^\dim} h(x) \d\al(x) =  \int_{\RR^\dim} h(x) \density{\al}(x) \d x.
}
An arbitrary measure $\al \in \Mm(\X)$ (which need not have a density nor be a sum of Diracs) is defined by the fact that it can be integrated against any continuous function $f \in \Cc(\X)$ and obtain $\int_\X f(x) \d\al(x) \in \RR$. If $\X$ is not compact, one should also impose that $f$ has compact support or at least has $0$ limit at infinity.
Measures are thus in some sense ``less regular'' than functions but more regular than distributions (which are dual to smooth functions). For instance, the derivative of a Dirac is not a measure.
We denote $\Mm_+(\X)$ the set of all positive measures on $\X$. The set of probability measures is denoted $\Mm_+^1(\X)$, which means that any $\al \in \Mm_+^1(\X)$ is positive, and that $\al(\X)=\int_\Xx \d\al = 1$.
Figure~\ref{fig-measures} offers a visualization of the different classes of measures, beyond histograms, considered in this work.
\end{rem2}

\begin{figure}[h!]
\centering
\begin{tabular}{@{}c@{\hspace{1mm}}c@{\hspace{1mm}}c@{\hspace{1mm}}c@{}}
\includegraphics[width=.24\linewidth]{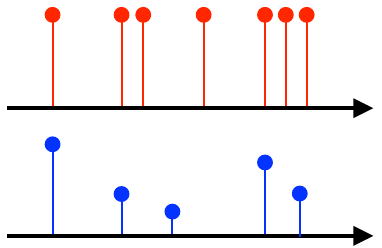}&
\includegraphics[width=.15\linewidth]{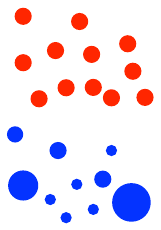}&
\includegraphics[width=.24\linewidth]{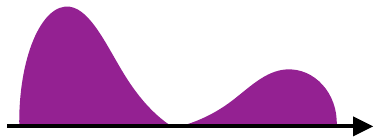}&
\includegraphics[width=.22\linewidth]{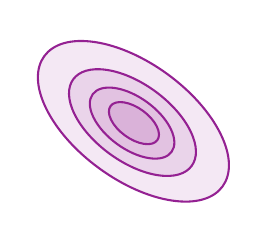}\\
Discrete $d=1$ & Discrete $d=2$ & Density $d=1$ & Density $d=2$
\end{tabular}
\caption{\label{fig-measures}
Schematic display of discrete distributions $\al = \sum_{i=1}^n \a_i \de_{x_i}$ (red corresponds to empirical uniform distribution $\a_i=1/n$, and blue to arbitrary distributions) and densities $\d\al(x)=\density{\al}(x)\d x$ (in purple), in both one and two dimensions. Discrete distributions in one-dimension are displayed as stem plots (with length equal to $\a_i$) and in two dimensions using point clouds (in which case their radius might be equal to $\a_i$ or, for a more visually accurate representation, their area).
}
\end{figure}

\section{Assignment and Monge Problem}

Given a cost matrix $(\C_{i,j})_{i \in \range{n}, j \in \range{m}}$, assuming $n=m$, the optimal assignment problem seeks for a bijection $\si$ in the set $\Perm(n)$ of permutations of $n$ elements solving
\eql{\label{eq-optimal-assignment}
	\umin{\si \in \Perm(n)} \frac{1}{n}\sum_{i=1}^n \C_{i,\si(i)}.
}
One could naively evaluate the cost function above using all permutations in the set $\Perm(n)$. However, that set has size $n!$, which is gigantic even for small $n$. Consider, for instance, that such a set has more than $10^{100}$ elements~\citep{Dantzig1983} when $n$ is as small as 70. That problem can therefore be solved only if there exist efficient algorithms to optimize that cost function over the set of permutations, which is the subject of~\S\ref{s-auction}.

\begin{rem}[Uniqueness] Note that the optimal assignment problem may have several optimal solutions. Suppose, for instance, that $n=m=2$ and that the matrix $\C$ is the pairwise distance matrix between the four corners of a 2-D square of side length $1$, as represented in the left plot of Figure~\ref{fig-non-unique-matching}. In that case only two assignments exist, and they are both optimal.
\end{rem}

\begin{figure}[h!]
\centering
\includegraphics[width=.8\linewidth]{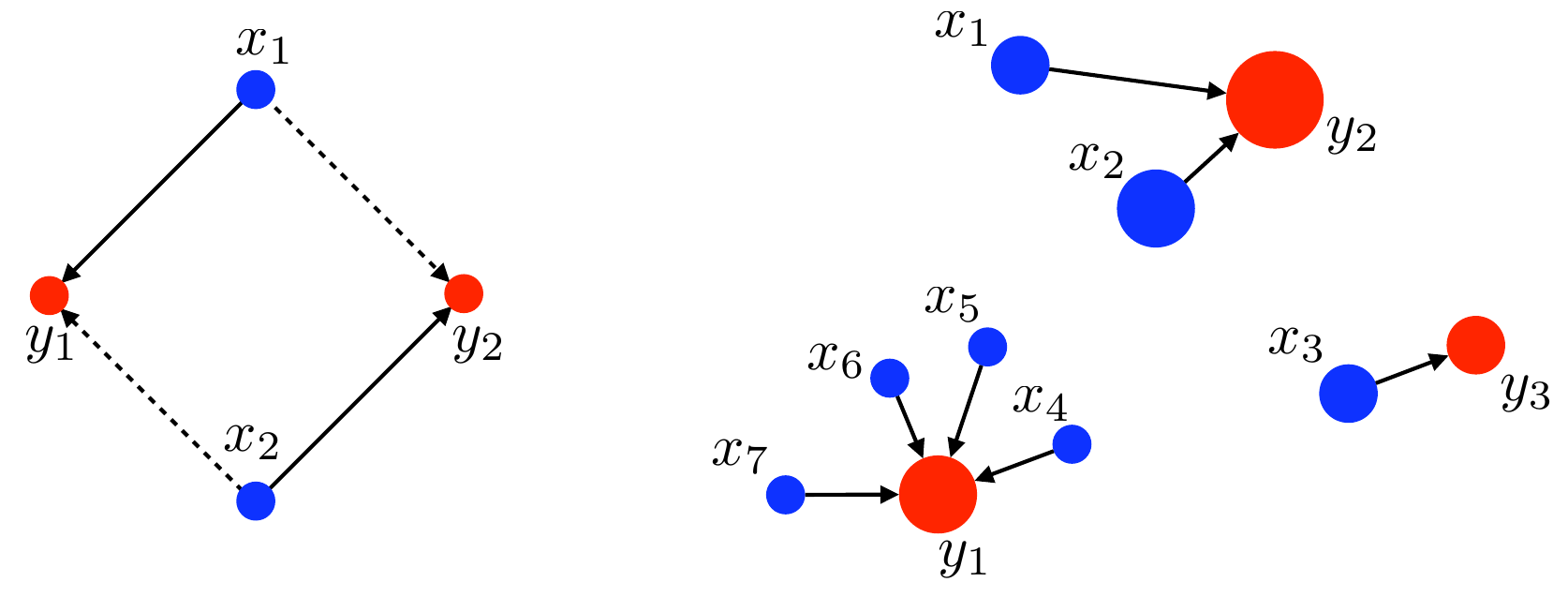}
\caption{\label{fig-non-unique-matching}
Left: blue dots from measure $\alpha$ and red dots from measure $\beta$ are pairwise equidistant. Hence, either matching $\sigma=(1,2)$ (full line) or $\sigma=(2,1)$ (dotted line) is optimal. Right: a Monge map can associate the blue measure $\alpha$ to the red measure $\beta$. The weights $\alpha_i$ are displayed proportionally to the area of the disk marked at each location. The mapping here is such that $T(x_1)=T(x_2)=y_2$, $T(x_3)=y_3$, whereas for $4\leq i\leq 7$ we have $T(x_i)=y_1$.
}
\end{figure}

\begin{rem1}{Monge problem between discrete measures}\label{rem-monge1}
For dis\-cre\-te measures
\eql{\label{eq-pair-discr}
	\al = \sum_{i=1}^n \a_i \de_{x_i}
	\qandq
	\be = \sum_{j=1}^m \b_j \de_{y_j},
}
the~\citeauthor{Monge1781} problem~\citeyearpar{Monge1781} seeks a map that associates to each point $x_i$ a single point $y_j$ and which must push the mass of $\al$ toward the mass of $\be$, namely, such a map $T:\{x_1,\dots, x_n\}\rightarrow \{y_1,\dots,y_m\}$ must verify that
\eql{\label{eq-monge-constr}
	\foralls j \in \range{m}, \quad
	\b_j = \sum_{ i : \T(x_i) = y_j } \a_i ,
}
which we write in compact form  as $\T_\sharp \al = \be$. Because all the elements of $\b$ are positive, that map is necessarily surjective.
This map should minimize some transportation cost, which is parameterized by a function $\c(x,y)$ defined for points $(x,y) \in \X \times \Y$,
\eql{\label{eq-monge-discr}
	\umin{\T} \enscond{ \sum_{i} \c(x_i,\T(x_i))  }{ \T_\sharp \al = \be } .
}
Such a map between discrete points can be of course encoded, assuming all $x$'s and $y$'s are distinct, using indices $\si : \range{n} \rightarrow \range{m}$ so that $j=\si(i)$, and the mass conservation is written as
\eq{
	\sum_{i \in \si^{-1}(j)} \a_i = \b_j,
}
where the inverse $\si^{-1}(j)$ is to be understood as the preimage set of $j$. In the special case when $n=m$ and all weights are uniform, that is, $\a_i=\b_j=1/n$, then the mass conservation constraint implies that $\T$ is a bijection, such that $\T(x_i)=y_{\si(i)}$, and the Monge problem is equivalent to the optimal matching problem~\eqref{eq-optimal-assignment}, where the cost matrix is
\eq{
	\C_{i,j} \eqdef \c(x_i,y_j).
}
When $n\ne m$, note that, optimality aside, Monge maps may not even exist between a discrete measure to another. This happens when their weight vectors are not compatible, which is always the case when the target measure has more points than the source measure, $n<m$. For instance, the right plot in Figure~\ref{fig-non-unique-matching} shows an (optimal) Monge map between $\al$ and $\be$, but there is no Monge map from $\be$ to $\al$.
\end{rem1}

\begin{rem2}{Push-forward operator}\label{rem-push-f}
For a continuous map $\T : \X \rightarrow \Y$, we define its corresponding push-forward operator $\T_\sharp : \Mm(\X) \rightarrow \Mm(\Y)$.
For discrete measures~\eqref{eq-discr-meas}, the push-forward operation consists simply in moving the positions of all the points in the support of the measure
\eq{
	\T_{\sharp} \al \eqdef \sum_i \a_i \de_{\T(x_i)}.
}
For more general measures, for instance, for those with a density, the notion of push-forward plays a fundamental role to describe the spatial modification (or transport) of a probability measure. The formal definition reads as follows.

\begin{defn}[Push-forward]\label{defn-pushfwd}
For $\T : \X \rightarrow \Y$, the push-forward measure $\be = \T_\sharp \al \in \Mm(\Y)$ of some $\al \in \Mm(\X)$ satisfies
\eql{\label{eq-push-fwd}
	\foralls h \in \Cc(\Y), \quad \int_\Y h(y) \d \be(y) = \int_\X h(\T(x)) \d\al(x).
}
Equivalently, for any measurable set $B \subset \Y$, one has
\eql{\label{eq-equiv-pushfwd}
	\be(B) = \al( \enscond{x \in \X}{\T(x) \in B} ) = \al( T^{-1}(B) ).
}
Note that $\T_\sharp$ preserves positivity and total mass, so that if $\al \in \Mm_+^1(\X)$ then $\T_\sharp \al \in \Mm_+^1(\Y)$.
\end{defn}

Intuitively, a measurable map $T: \X\rightarrow \Y$ can be interpreted as a function moving a single point from a measurable space to another. $T_\sharp$ is an extension of $T$ that can move an entire probability measure on $\X$ toward a new probability measure on $\Y$. The operator $T_\sharp$ \emph{pushes forward} each elementary mass of a measure $\al$ on $\X$ by applying the map $T$ to obtain then an elementary mass in $\Y$.  Note that a push-forward operator $\T_\sharp : \Mm_+^1(\X) \rightarrow \Mm_+^1(\Y)$ is \emph{linear} in the sense that for two measures $\al_1,\al_2$ on $\X$, $T_\sharp(\al_1+\al_2)=T_\sharp\al_1+ T_\sharp\al_2$.
\end{rem2}

\begin{rem2}{Push-forward for multivariate densities}
Explicitly doing the change of variables in formula~\eqref{eq-push-fwd} for measures with densities $(\density{\al},\density{\be})$ on $\RR^\dim$ (assuming $\T$ is smooth and bijective) shows that a push-forward acts on densities linearly as a change of variables in the integration formula. Indeed, one has
\eql{\label{eq-pfwd-density}
	\density{\al}(x) = |\det(\T'(x))|  \density{\be}(\T(x)),
}
where $\T'(x) \in \RR^{\dim \times \dim}$ is the Jacobian matrix of $T$ (the matrix formed by taking the gradient of each coordinate of $T$).
This implies
\eq{
	|\det(\T'(x))| = \frac{ \density{\al}(x) }{ \density{\be}(\T(x)) }.
}
\end{rem2}

\begin{rem2}{Monge problem between arbitrary measures}\label{rem-monge2}
The Mon\-ge problem~\eqref{eq-monge-discr} can be extended to the case where two arbitrary probability measures $(\al,\be)$, supported on two spaces $(\X,\Y)$ can be linked through a map $\T : \X \rightarrow \Y$ that minimizes
\eql{\label{eq-monge-continuous}
	\umin{\T} \enscond{ \int_{\X} \c(x,\T(x)) \d \al(x)  }{  \T_\sharp \al = \be }.
}
The constraint $\T_\sharp \al = \be$ means that $\T$ pushes forward the mass of $\al$ to $\be$, using the push-forward operator defined in Remark~\ref{rem-push-f}
\end{rem2}

\begin{rem2}{Push-forward vs. pull-back}
The push-forward $\T_\sharp$ of measures should not be confused with the pull-back of functions $\T^\sharp : \Cc(\Y) \rightarrow \Cc(\X)$ which corresponds to ``warping'' between functions, defined as the linear map which to $g \in \Cc(\Y)$ associates $\T^\sharp g = g \circ \T$. Push-forward and pull-back are actually adjoint to one another, in the sense that
\eq{
	\foralls (\al,g) \in \Mm(\X) \times \Cc(\Y), \quad
	\int_\Y g \d( \T_\sharp\al ) = \int_\X (\T^\sharp g) \d\al.
}
Note that even if $(\al,\be)$ have densities $(\density{\al},\density{\be})$ with respect to a fixed measure (\emph{e.g.} Lebesgue on $\RR^\dim$), $\T_\sharp \al$ does not have $\T^\sharp \density{\be}$ as density, because of the presence of the Jacobian in~\eqref{eq-pfwd-density}.
This explains why OT should be used with caution to perform image registration, because it does not operate as an image warping method.
Figure~\ref{fig-push-pull} illustrates the distinction between these push-forward and pull-back operators.
\end{rem2}

\begin{figure}[h!]
\centering
\begin{tabular}{@{}c@{\hspace{5mm}}c@{}}
\includegraphics[width=.3\linewidth]{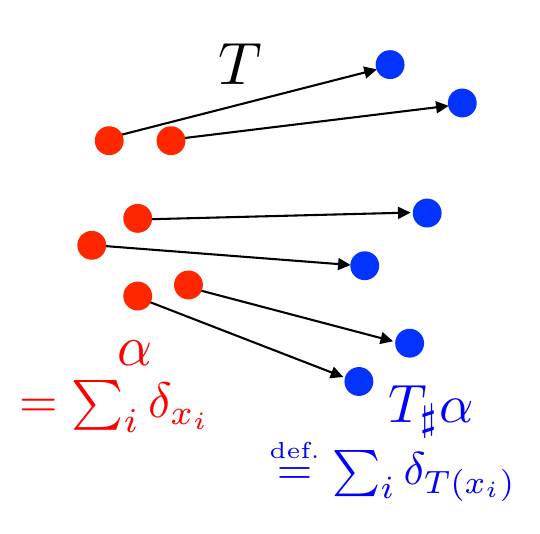}&
\includegraphics[width=.3\linewidth]{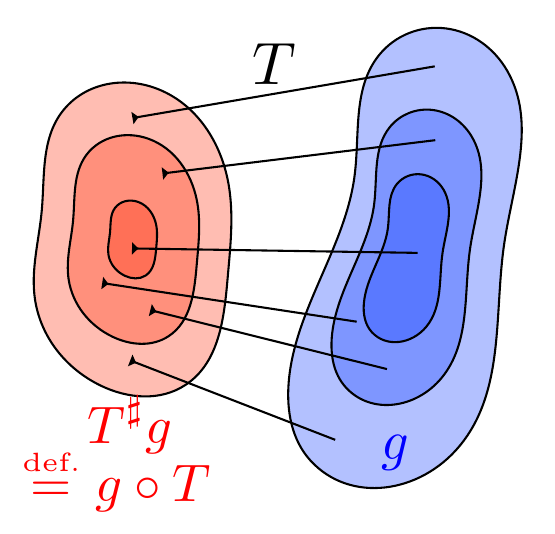}\\
Push-forward of measures & Pull-back of functions
\end{tabular}
\caption{\label{fig-push-pull}
Comparison of the push-forward operator $\T_\sharp$, which can take as an input any measure, and the pull-back operator $\T^\sharp$, which operates on functions, notaly densities.
}
\end{figure}

\begin{rem2}{Measures and random variables}\label{rem-meas-random}
	Radon measures can also be viewed as representing the distributions of random variables. A random variable $X$ on $\X$ is actually a map $X : \Om \rightarrow \X$ from some abstract (often unspecified) probability space $(\Om,\PP)$, and its distribution $\al$ is the Radon measure $\al \in \Mm_+^1(\X)$ such that $\PP(X \in A) = \al(A)=\int_A \d\al(x)$.
	Equivalently, it is the push-forward of $\PP$ by $X$, $\al=X_\sharp\PP$.
	Applying another push-forward $\be = \T_\sharp\al$ for $\T : \X \rightarrow \Y$, following~\eqref{eq-push-fwd}, is equivalent to defining another random variable $Y=\T(X) : \om \in \Om \rightarrow \T(X(\om)) \in Y$, so that $\be$ is the distribution of $Y$.
	Drawing a random sample $y$ from $Y$ is thus simply achieved by computing $y=\T(x)$, where $x$ is drawn from~$X$.
\end{rem2}

\section{Kantorovich Relaxation}\label{sec-kantorovich-relaxation}


The assignment problem, and its generalization found in the Monge problem laid out in Remark~\ref{rem-monge1}, is not always relevant to studying discrete measures, such as those found in practical problems. Indeed, because the assignment problem is formulated as a permutation problem, it can only be used to compare \emph{uniform} histograms of the \emph{same} size. A direct generalization to discrete measures with nonuniform weights can be carried out using Monge's formalism of push-forward maps, but that formulation may also be degenerate in the absence of feasible solutions satisfying the mass conservation constraint~\eqref{eq-monge-constr} (see the end of Remark~\ref{rem-monge1}). Additionally, the assignment problem~\eqref{eq-monge-discr} is combinatorial, and the feasible set for the Monge problem~\eqref{eq-monge-continuous}, despite being continuously parameterized as the set consisting in all push-forward measures that satisfy the mass conservation constraint, is \emph{nonconvex}. Both are therefore difficult to solve when approached in their original formulation.

The key idea of~\citet{Kantorovich42} is to relax the deterministic nature of transportation, namely the fact that a source point $x_i$ can only be assigned to another point or location $y_{\sigma_i}$ or $T(x_i)$ only. Kantorovich proposes instead that the mass at any point $x_i$ be potentially dispatched across several locations. Kantorovich moves away from the idea that mass transportation should be \emph{deterministic} to consider instead a \emph{probabilistic} transport, which allows what is commonly known now as \emph{mass splitting} from a source toward several targets. This flexibility is encoded using, in place of a permutation $\sigma$ or a map $T$, a coupling matrix $\P  \in \RR_+^{n \times m}$, where $\P_{i,j}$ describes the amount of mass flowing from bin $i$ toward bin $j$, or from the mass found at $x_i$ toward $y_j$ in the formalism of discrete measures~\eqref{eq-pair-discr}. Admissible couplings admit a far simpler characterization than Monge maps,
\eql{\label{eq-discr-couplings}
	\CouplingsD(\a,\b) \eqdef \enscond{ \P \in \RR_+^{n \times m} }{
		\P \ones_m = \a \qandq
		\transp{\P} \ones_n = \b
	},
}
where we used the following matrix-vector notation:
\eq{
	\P \ones_m = \left(\sum_j \P_{i,j}\right)_i \in \RR^n
	\qandq
	\transp{\P} \ones_n = \left(\sum_i \P_{i,j}\right)_j \in \RR^m.
}
The set of matrices $\CouplingsD(\a,\b)$ is bounded and defined by $n+m$ equality constraints, and therefore is a convex polytope (the convex hull of a finite set of matrices)~\citep[\S8.1]{brualdi2006combinatorial}.

Additionally, whereas the Monge formulation (as illustrated in the right plot of Figure~\ref{fig-non-unique-matching}) was intrisically asymmetric, Kantorovich's relaxed formulation is always symmetric, in the sense that a coupling $\P$ is in  $\CouplingsD(\a,\b)$ if and only if  $\transp{\P}$ is in $\CouplingsD(\b,\a)$. Kantorovich's optimal transport problem now reads
\eql{\label{eq-mk-discr}
	\MKD_{\C}(\a,\b) \eqdef
	\umin{\P \in \CouplingsD(\a,\b)}
		\dotp{\C}{\P} \eqdef \sum_{i,j} \C_{i,j} \P_{i,j}.
}
This is a linear program (see Chapter~\ref{c-algo-basics}), and as is usually the case with such programs, its optimal solutions are not necessarily unique.

\begin{rem}[Mines and factories]\label{rem-kantorovich-primal} The Kantorovich problem finds a very natural illustration in the following resource allocation problem (see also \citet{Hitchcock41}). Suppose that an operator runs $n$ warehouses and $m$ factories. Each warehouse contains a valuable raw material that is needed by the factories to run properly. More precisely, each warehouse is indexed with an integer $i$ and contains $\a_i$ units of the raw material. These raw materials must all be moved to the factories, with a prescribed quantity $\b_j$ needed at factory $j$ to function properly. To transfer resources from a warehouse $i$ to a factory $j$, the operator can use a transportation company that will charge $\C_{i,j}$ to move a single unit of the resource from location $i$ to location $j$. We assume that the transportation company has the monopoly to transport goods and applies the same linear pricing scheme to all actors of the economy: the cost of shipping $a$ units of the resource from $i$ to $j$ is equal to $a\times \C_{i,j}$.

Faced with the problem described above, the operator chooses to solve the linear program described in Equation~\eqref{eq-mk-discr} to obtain a transportation plan $\P^\star$ that quantifies for each pair $i,j$ the amount of goods $\P_{i,j}$ that must transported from warehouse $i$ to factory $j$. The operator pays on aggregate a total of $\dotp{\P^\star}{\C}$ to the transportation company to execute that plan.
\end{rem}

\paragraph{Permutation matrices as couplings.}

For a permutation $\si\in\Perm(n)$, we write $\P_{\si}$ for the corresponding permutation matrix,
	\eql{\label{eq-perm-matrices}
		\foralls (i,j) \in \range{n}^2, \quad
		(\P_{\si})_{i,j} = \choice{
			1/n \qifq j=\si_i, \\
			0 \quad\text{otherwise.}
		}
	}
One can check that in that case
\eq{\dotp{\C}{\P_{\si}}=\frac{1}{n}\sum_{i=1}^n \C_{i,\si_i},}
which shows that the assignment problem~\eqref{eq-optimal-assignment} can be recast as a Kantorovich problem~\eqref{eq-mk-discr} where the couplings $\P$ are restricted to be exactly permutation matrices:
\eq{\umin{\si \in \Perm(n)} \frac{1}{n}\sum_{i=1}^n \C_{i,\si(i)} = \umin{\si \in \Perm(n)} \dotp{\C}{\P_{\si}}.}
Next, one can easily check that the set of permutation matrices is strictly included in the \citeauthor{birkhoff} polytope $\CouplingsD(\ones_n/n,\ones_n/n)$. Indeed, for any permutation $\si$ we have $\P_{\si}\ones=\ones_n$ and $\transp{\P_{\si}}\ones=\ones_n$, whereas $\ones_n\transp{\ones_n}/n^2$ is a valid coupling but not a permutation matrix. Therefore, the minimum of $\dotp{\C}{\P}$ is necessarily smaller when considering all transportation than when considering only permutation matrices:
$$ \MKD_{\C}(\ones_n/n,\ones_n/n) \leq \umin{\si \in \Perm(n)} \dotp{\C}{\P_{\si}}.$$

The following proposition shows that these problems result in fact in the same optimum, namely that one can always find a permutation matrix that minimizes Kantorovich's problem~\eqref{eq-mk-discr} between two uniform measures $\a=\b=\ones_n/n$. The Kantorovich relaxation is therefore \emph{tight} when considered on assignment problems. 
Figure~\ref{fig-matching-kantorovitch} shows on the left a 2-D example of optimal matching corresponding to this special case.

\begin{prop}[Kantorovich for matching]\label{prop-matching-kanto}
	If $m=n$ and $\a=\b=\ones_n/n$, then there exists an optimal solution for Problem~\eqref{eq-mk-discr} $\P_{\si^\star}$, which is a permutation matrix associated to an optimal permutation $\si^\star \in \Perm(n)$ for Problem~\eqref{eq-optimal-assignment}.
\end{prop}

\begin{proof}
	\citeauthor{birkhoff}'s theorem \citeyearpar{birkhoff} states that the set of extremal points of $\CouplingsD(\ones_n/n,\ones_n/n)$ is equal to the set of permutation matrices. A fundamental theorem of linear programming \citep[Theorem 2.7]{bertsimas1997introduction} states that the minimum of a linear objective in a nonempty polyhedron, if finite, is reached at an extremal point of the polyhedron.
\end{proof}

\begin{figure}[h!]
\centering
\begin{tabular}{@{}c@{\hspace{5mm}}c@{}}
\includegraphics[width=.4\linewidth]{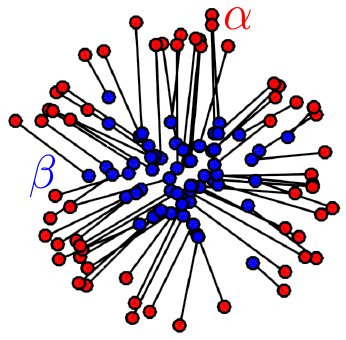}&
\includegraphics[width=.4\linewidth]{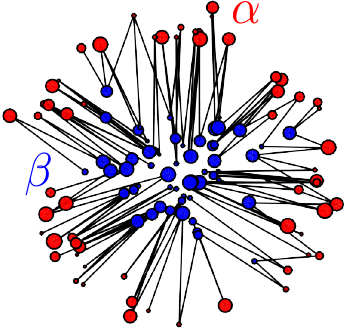}
\end{tabular}
\caption{\label{fig-matching-kantorovitch}
Comparison of optimal matching and generic couplings. A black segment between $x_i$ and $y_j$ indicates a nonzero element in the displayed optimal coupling $\P_{i,j}$ solving~\eqref{eq-mk-discr}.
Left: optimal matching, corresponding to the setting of Proposition~\ref{prop-matching-kanto} (empirical measures with the same number $n=m$ of points).
Right: these two weighted point clouds cannot be matched; instead a Kantorovich coupling can be used to associate two arbitrary discrete measures.
}
\end{figure}

\begin{rem1}{Kantorovich problem between discrete measures}
For discrete measures $\al,\be$ of the form~\eqref{eq-pair-discr}, we store in the matrix $\C$ all pairwise costs between points in the supports of $\al,\be$, namely $\C_{i,j} \eqdef \c(x_i,y_j)$, to define
\eql{\label{eq-kanto-discr}
	\MK_\c(\al,\be) \eqdef \MKD_{\C}(\a,\b).
}
Therefore, the Kantorovich formulation of optimal transport between discrete measures is the same as the problem between their associated probability weight vectors $\a,\b$  except that the cost matrix $\C$ depends on the support of $\al$ and $\be$. The notation $\MK_\c(\al,\be)$, however, is useful in some situations, because it makes explicit the dependency with respect to \emph{both} probability weights and supporting points, the latter being exclusively considered through the cost function $\c$.
\end{rem1}

\begin{figure}[h!]
\centering
\begin{tabular}{@{}c@{\hspace{5mm}}c@{\hspace{5mm}}c@{}}
\includegraphics[width=.3\linewidth]{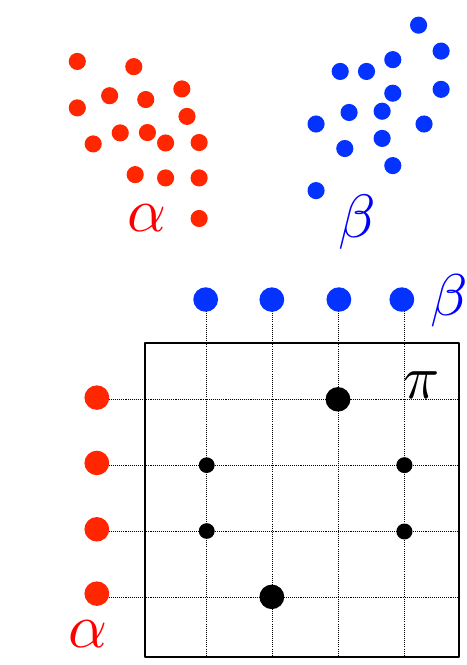}&
\includegraphics[width=.3\linewidth]{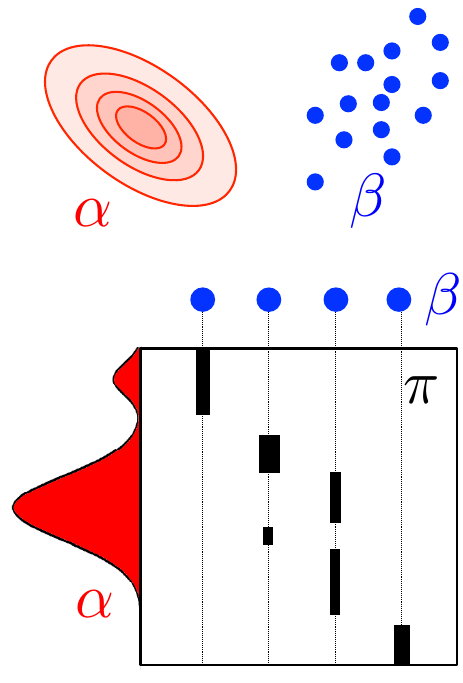}&
\includegraphics[width=.3\linewidth]{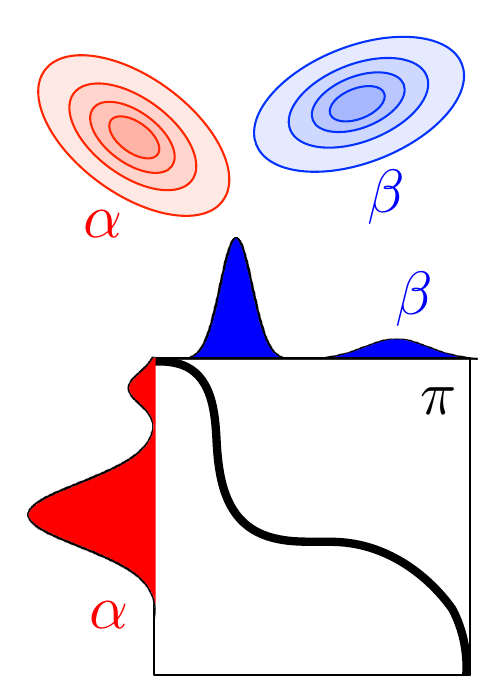}\\
Discrete & Semidiscrete & Continuous
\end{tabular}
\caption{\label{fig-settings}
Schematic viewed of input measures $(\al,\be)$ and couplings $\Couplings(\al,\be)$ encountered in the three main scenarios for Kantorovich OT. Chapter~\ref{c-algo-semidiscr} is dedicated to the semidiscrete setup.
}
\end{figure}

\begin{rem1}{Using optimal assignments and couplings}
The op\-ti\-mal transport plan itself (either as a coupling $\P$ or a Monge map $\T$ when it exists) has found many applications in data sciences, and in particular image processing. It has, for instance, been used for contrast equalization~\citep{delon2004midway} and texture synthesis~\citet{gutierrez2017optimal}.
A significant part of applications of OT to imaging sciences is for image matching~\citep{zhu2007image,wang2013linear,museyko2009application,li2013novel}, image fusion~\citep{courty2016optimal}, medical imaging~\citep{wang2011optimal} and shape registration~\citep{makihara2010earth,lai2014multi,su2015optimal}, and image watermarking~\citep{mathon2014optimal}. In astrophysics, OT has been used for reconstructing the early universe~\citep{FrischNaturee}.
OT has also been used for music transcription~\citep{flamary2016optimal}, and finds numerous applications in economics to interpret matching data~\citep{galichon2016optimal}.
Lastly, let us note that the computation of transportation maps computed using OT techniques (or inspired from them) is also useful to perform sampling~\citep{reich2013nonparametric,oliver2014minimization} and Bayesian inference~\citep{kim2013efficient,el2012bayesian}.
\end{rem1}

\begin{rem2}{Kantorovich problem between arbitrary measures} De\-fi\-ni\-tion~\eqref{eq-kanto-discr} of $\MK_\c$ is extended to arbitrary measures by considering couplings $\pi \in \Mm_+^1(\X \times \Y)$ which are joint distributions over the product space. The discrete case is a special situation where one imposes this product measure to be of the form $\pi = \sum_{i,j} \P_{i,j} \de_{(x_i,y_j)}$. In the general case, the mass conservation constraint~\eqref{eq-discr-couplings} should be rewritten as a marginal constraint on joint probability distributions
\eql{\label{eq-coupling-generic}
	\Couplings(\al,\be) \eqdef
	\enscond{
		\pi \in \Mm_+^1(\X \times \Y)
	}{
		P_{\X\sharp} \pi = \al
		\qandq
		P_{\Y\sharp} \pi = \be
	}.
}
Here $P_{\X\sharp}$ and $P_{\Y\sharp}$ are the push-forwards (see Definition~\ref{defn-pushfwd}) of the projections $P_\X(x,y)=x$ and $P_\Y(x,y)=y$.
Figure~\ref{fig-settings} shows how these coupling constraints translate for different classes of problems (discrete measures and densities).
Using~\eqref{eq-equiv-pushfwd}, these marginal constraints are equivalent to imposing that $\pi(A \times \Y)=\al(A)$ and $\pi(\X \times B)=\be(B)$ for sets $A \subset \X$ and $B \subset \Y$.
The Kantorovich problem~\eqref{eq-mk-discr} is then generalized as
\eql{\label{eq-mk-generic}
	\MK_\c(\al,\be) \eqdef
	\umin{\pi \in \Couplings(\al,\be)}
		\int_{\X \times \Y} \c(x,y) \d\pi(x,y).
}
This is an infinite-dimensional linear program over a space of measures. If $(\X,\Y)$ are compact spaces and $c$ is continuous, then it is easy to show that it always has solutions. Indeed $\Couplings(\al,\be)$ is compact for the weak topology of measures (see Remark~\ref{dfn-weak-conv}), $\pi \mapsto \int c \d\pi$ is a continuous function for this topology and the constraint set is nonempty (for instance, $\al \otimes \be \in \Couplings(\al,\be)$).
Figure~\ref{fig-couplings} shows examples of discrete and continuous optimal coupling solving~\eqref{eq-mk-generic}.
Figure~\ref{fig-couplings-simple} shows other examples of optimal 1-D couplings, involving discrete and continuous marginals.
\todoK{Explain that proof of existence of minimizers (\ie the inf is a min) follows from classical reasoning in variational calculus, especially if the spaces $(\X,\Y)$ are compact (or the support of the input measure) and the cost $\c$ is continuous. Indeed, the topology to use is the weak topology of measure (Definition~\ref{dfn-weak-conv}) and then the function is continuous on a compact. }
\end{rem2}

\begin{figure}[h!]
\centering
\begin{tabular}{@{}c@{\hspace{10mm}}c@{}}
\includegraphics[width=.3\linewidth]{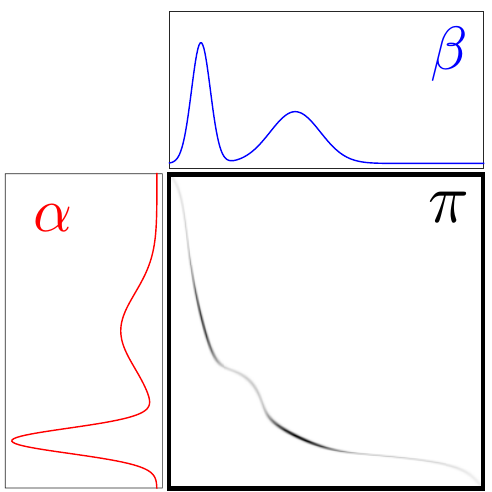}&
\includegraphics[width=.3\linewidth]{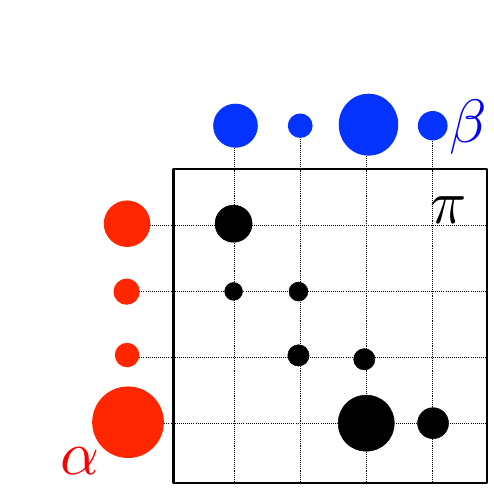}
\end{tabular}
\caption{\label{fig-couplings}
Left: ``continuous'' coupling $\pi$ solving~\eqref{eq-coupling-generic} between two 1-D measures with density. The coupling is localized along the graph of the Monge map $(x,\T(x))$ (displayed in black).
Right: ``discrete'' coupling $\T$ solving~\eqref{eq-mk-discr} between two discrete measures of the form~\eqref{eq-pair-discr}. The positive entries $\T_{i,j}$  are displayed as black disks at position $(i,j)$ with radius proportional to $\T_{i,j}$.
}
\end{figure}

\begin{figure}[h!]
\centering
\begin{tabular}{@{}c@{}c@{}c@{}c@{}}
\includegraphics[width=.24\linewidth]{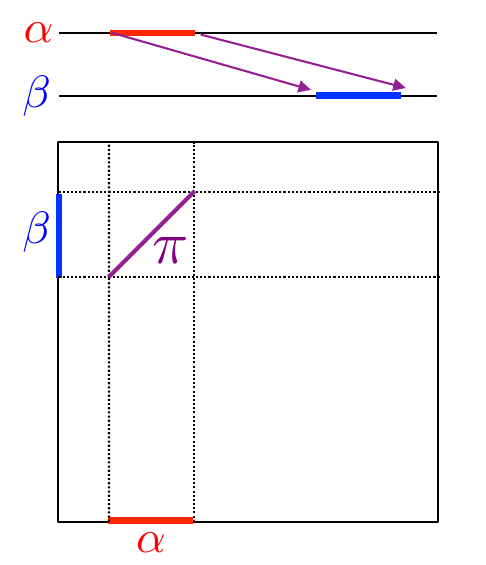}&
\includegraphics[width=.24\linewidth]{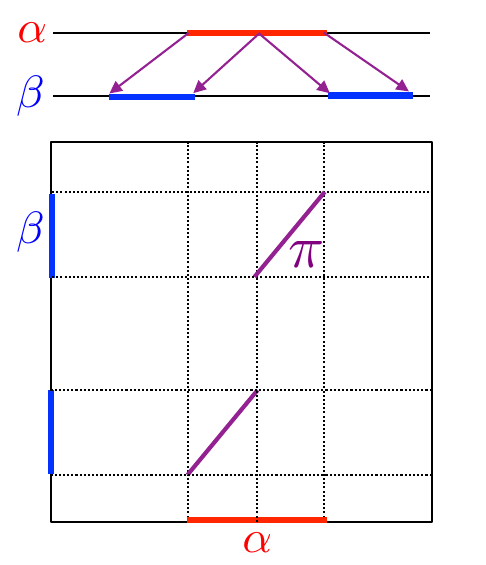}&
\includegraphics[width=.24\linewidth]{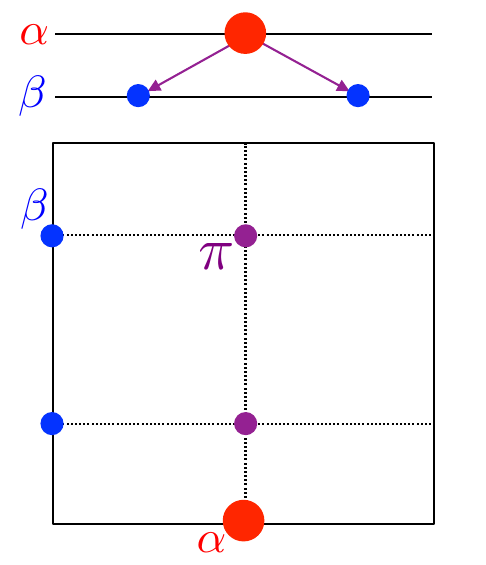}&
\includegraphics[width=.24\linewidth]{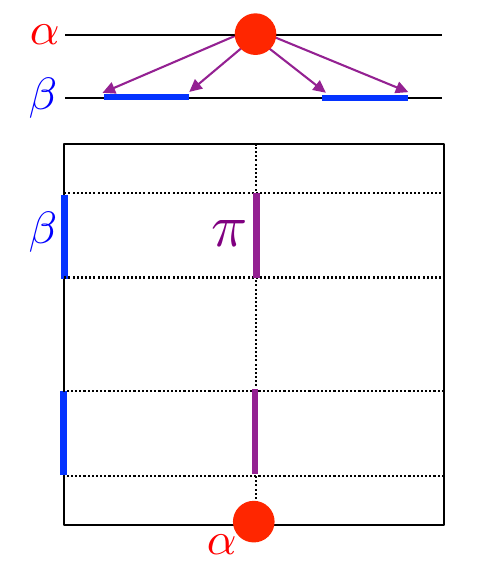}
\end{tabular}
\caption{\label{fig-couplings-simple}
Four simple examples of optimal couplings between 1-D distributions, represented as maps above (arrows) and couplings below. Inspired by~\citet{Levy2017review}.
}
\end{figure}

\begin{rem2}{Probabilistic interpretation}
Kantorovich's problem can be reinterpreted through the prism of random variables, following Remark~\ref{rem-meas-random}. Indeed, Problem~\eqref{eq-mk-generic} is equivalent to
\eql{\label{eq-ot-proba-interpretation}
	\MK_\c(\al,\be) = \umin{(X,Y)} \enscond{ \EE_{(X,Y)}(c(X,Y)) }{ X \sim \al, Y \sim \be },
}
where $(X,Y)$ is a couple of random variables over $\X \times \Y$ and $X \sim \al$ (resp., $Y \sim \be$) means that the law of $X$ (resp., $Y$), represented as a measure, must be $\al$ (resp., $\be$). The law of the couple $(X,Y)$ is then $\pi \in \Couplings(\al,\be)$ over the product space $\X \times \Y$.
\end{rem2}

\section{Metric Properties of Optimal Transport}

An important feature of OT is that it defines a distance between histograms and probability measures as soon as the cost matrix satisfies certain suitable properties. Indeed, OT can be understood as a canonical way to lift a ground distance between points to a distance between histogram or measures.

We first consider the case where, using a term first introduced by~\citet{RubTomGui00}, the ``ground metric'' matrix $\C$ is fixed, representing substitution costs between bins, and shared across several histograms we would like to compare. The following proposition states that OT provides a valid distance between histograms supported on these bins.

\begin{prop}\label{prop-metric-histo}
We suppose $n=m$ and that for some $p \geq 1$, $\C=\distD^p=(\distD_{i,j}^p)_{i,j} \in \RR^{n \times n}$, where $\distD \in \RR_+^{n \times n}$ is a distance on $\range{n}$, \ie
\begin{enumerate}[label=(\roman*)]
	\item $\distD \in \RR_+^{n \times n}$ is symmetric;
	\item $\distD_{i,j}=0$ if and only if $i=j$;
	\item $\foralls (i,j,k) \in \range{n}^3, \distD_{i,k} \leq \distD_{i,j}+\distD_{j,k}$.
\end{enumerate}
Then
\eql{\label{eq-wass-p-disc}
	\WassD_p(\a,\b) \eqdef \MKD_{\distD^p}(\a,\b)^{1/p}
}
(note that $\WassD_p$ depends on $\distD$) defines the $p$-Wasserstein distance on $\Si_n$, \ie $\WassD_p$ is symmetric, positive, $\WassD_p(\a,\b)=0$ if and only if $\a = \b$, and it satisfies the triangle inequality
\eq{
	\foralls \a,\b,\VectMode{c} \in \Si_n, \quad \WassD_p(\a,\VectMode{c}) \leq \WassD_p(\a,\b) + \WassD_p(\b,\VectMode{c}).
}
\end{prop}

\begin{proof}
Symmetry and definiteness of the distance are easy to prove: since $\C = \distD^p$ has a null diagonal, $\WassD_p(\a,\a)=0$, with corresponding optimal transport matrix $\P^\star=\diag(\a)$; by the positivity of all off-diagonal elements of $\distD^p$, $\WassD_p(\a,\b)>0$ whenever $\a\ne \b$ (because in this case, an admissible coupling necessarily has a nonzero element outside the diagonal); by symmetry of $\distD^p$, $\WassD_p(\a,\b)$ is itself a symmetric function.

To prove the triangle inequality of Wasserstein distances for arbitrary measures, \citet[Theorem 7.3]{Villani03} uses the gluing lemma, which stresses the existence of couplings with a prescribed structure.
In the discrete setting, the explicit constuction of this glued coupling is simple.
Let $\a,\b,\VectMode{c} \in\simplex_n$. Let $\P$ and $\Q$ be two optimal solutions of the transport problems between $\a$ and $\b$, and $\b$ and $\VectMode{c}$, respectively.
To avoid issues that may arise from null coordinates in $\b$, we define a vector $\tilde{\b}$ such that $\tilde\b_j \eqdef \b_j$ if $\b_j>0$, and $\tilde{\b}_j\eqdef 1$ otherwise, to write
\eq{
	\SS \eqdef \P \diag(1/\tilde{\b}) \Q \in \RR_+^{n \times n},
}
and notice that $\SS \in \CouplingsD(\a,\VectMode{c})$ because
\eq{
	\SS \ones_n  = \P \diag(1/\tilde{\b}) \Q \ones_n = \P (\b / \tilde{\b}) = \P \ones_{\Supp(\b)} = \a,
}
where we denoted $\ones_{\Supp(\b)}$ the vector of size $n$ with ones located at those indices $j$ where $\b_j>0$ and zero otherwise, and we use the fact that $\P \ones_{\Supp(\b)} = \P \ones = \a$ because necessarily $\P_{i,j} = 0$ for those $j$ where $\b_j=0$. Similarly one verifies that $\transp{\SS} \ones_n = \VectMode{c}$. The triangle inequality follows then from
\begin{align*}
\WassD_p(\a,\VectMode{c})&=\left(\min_{\P\in U(\a,\VectMode{c})}\dotp{\P}{\distD^p}\right)^{1/p} \leq \dotp{\SS}{\distD^p}^{1/p}\\
&= \left(\sum_{ik} \distD^p_{ik}\sum_{j} \frac{\P_{ij}\Q_{jk}}{\tilde{\b}_j}\right)^{1/p} \leq \left(\sum_{ijk} \left(\distD_{ij}+\distD_{jk}\right)^p \frac{\P_{ij}\Q_{jk}}{\tilde{\b}_j}\right)^{1/p} \\
& \leq \left(\sum_{ijk} \distD^p_{ij} \frac{\P_{ij}\Q_{jk}}{\tilde{\b}_j}\right)^{1/p} + \left(\sum_{ijk}\distD^p_{jk} \frac{\P_{ij}\Q_{jk}}{\tilde{\b}_j}\right)^{1/p}.
\end{align*}
The first inequality is due to the suboptimality of $\SS$, the second is the triangle inequality for elements in $\distD$, and the third comes from Minkowski's inequality. One thus has
\begin{align*}
\WassD_p(\a,\VectMode{c})& \leq \left(\sum_{ij} \distD^p_{ij}\P_{ij} \sum_k \frac{\Q_{jk}}{\tilde{\b}_j}\right)^{1/p} + \left(\sum_{jk} \distD^p_{jk} \Q_{jk} \sum_i \frac{\P_{ij}}{\tilde{\b}_j}\right)^{1/p}\\
&= \left(\sum_{ij} \distD^p_{ij}\P_{ij}\right)^{1/p} + \left(\sum_{jk} \distD^p_{jk} \Q_{jk}\right)^{1/p}\\
&= \WassD_p(\a,\b) +\WassD_p(\b,\c),
\end{align*}
which concludes the proof.
\end{proof}

\begin{rem}[The cases $0<p\leq 1$]
Note that if $0 < p \leq 1$, then $\distD^p$ is itself distance. This implies that while for $p\geq 1$, $\WassD_p(\a,\b)$ is a distance, in the case $p \leq 1$, it is actually $\WassD_p(\a,\b)^p$ which defines a distance on the simplex.
\end{rem}

\begin{rem1}{Applications of Wasserstein distances}
The fact that the OT distance automatically ``lifts'' a ground metric between bins to a metric between histograms on such bins makes it a method of choice for applications in computer vision and machine learning to compare histograms.
In these fields, a classical approach is to ``pool'' local features (for instance, image descriptors) and compute a histogram of the empirical distribution of features (a so-called bag of features) to perform retrieval, clustering or classification; see, for instance,~\citep{oliva2001modeling}.
Along a similar line of ideas, OT distances can be used over some lifted feature spaces to perform signal and image analysis~\citep{thorpe2017transportation}.
Applications to retrieval and clustering were initiated by the landmark paper~\citep{RubTomGui00}, with renewed applications following faster algorithms for threshold matrices $\C$ that fit for some applications, for example, in computer vision~\citep{pele2008linear,Pele-iccv2009}.
More recent applications stress the use of the earth mover's distance for bags-of-words, either to carry out dimensionality reduction~\citep{pmlr-v51-rolet16} and classify texts~\citep{kusner2015word,huang2016supervised}, or to define an alternative loss to train multiclass classifiers that output bags-of-words~\citep{FrognerNIPS}.
\citet{kolouri2017optimal} provides a recent overview of such applications to signal processing and machine learning.
\end{rem1}

\begin{rem1}{Wasserstein distance between measures}
Pro\-po\-si\-tion \ref{prop-metric-histo} can be generalized to deal with arbitrary measures that need not be discrete.

\begin{prop}\label{prop-metric-measure}
We assume $\X=\Y$ and that for some $p \geq 1$, $\c(x,y)=\dist(x,y)^p$, where $\dist$ is a distance on $\X$, \ie
\begin{enumerate}[label=(\roman*)]
	\item  $\dist(x,y) = \dist(y,x) \geq 0$;
	\item  $\dist(x,y)=0$ if and only if $x=y$;
	\item  $\foralls (x,y,z) \in \X^3, \dist(x,z) \leq \dist(x,y)+\dist(y,z)$.
\end{enumerate}
Then the $p$-Wasserstein distance on $\X$,
\eql{\label{eq-defn-wass-dist}
	\Wass_p(\al,\be) \eqdef \MK_{\dist^p}(\al,\be)^{1/p}
}
(note that $\Wass_p$ depends on $\dist$),  is indeed a distance, namely $\Wass_p$ is symmetric, nonnegative, $\Wass_p(\al,\be)=0$ if and only if $\al = \be$, and it satisfies the triangle inequality
\eq{
	\foralls (\al,\be,\ga) \in  \Mm_+^1(\X)^3, \quad \Wass_p(\al,\ga) \leq \Wass_p(\al,\be) + \Wass_p(\be,\ga).
}
\end{prop}
\begin{proof}
The proof follows the same approach as that for Proposition~\ref{prop-metric-histo} and relies on the existence of a coupling between $(\al,\ga)$ obtained by ``gluing'' optimal couplings between $(\al,\be)$ and $(\be,\ga)$.
\end{proof}
\end{rem1}

\begin{rem2}{Geometric intuition and weak convergence}
The Was\-ser\-stein distance $\Wass_p$ has many important properties, the most important being that it is a weak distance, \ie it allows one to compare singular distributions (for instance, discrete ones) whose supports do not overlap and to quantify the spatial shift between the supports of two distributions.
In particular, ``classical'' distances (or divergences) are not even defined between discrete distributions (the $L^2$ norm can only be applied to continuous measures with a density with respect to a base measure, and the discrete $\ell^2$ norm requires that positions $(x_i,y_j)$ take values in a predetermined discrete set to work properly). In sharp contrast, one has that for any $p> 0$, $\Wass_p^p(\de_x,\de_y) = \dist(x,y)$. Indeed, it suffices to notice that $\Couplings(\de_x,\de_y)=\{ \delta_{x,y}\}$ and therefore the Kantorovich problem having only one feasible solution, $\Wass_p^p(\de_x,\de_y)$ is necessarily $(\dist(x,y)^p)^{1/p}=\dist(x,y)$. This shows that $\Wass_p(\de_x,\de_y) \rightarrow 0$ if $x \rightarrow y$.
This property corresponds to the fact that $\Wass_p$ is a way to quantify the weak convergence, as we now define.

\begin{defn}[Weak convergence]\label{dfn-weak-conv}
	On a compact domain $\Xx$, $(\al_k)_k$ converges weakly to $\al$ in $\Mm_+^1(\Xx)$ (denoted $\al_k \rightharpoonup \al$) if and only if for any continuous function $g \in \Cc(\Xx)$, $\int_\Xx g \d\al_k \rightarrow \int_\Xx g \d\al$.
	One needs to add additional decay conditions on $g$ on noncompact domains.
	This notion of weak convergence corresponds to the convergence in the law of random vectors.
\end{defn}

This convergence can be shown to be equivalent to $\Wass_p(\al_k,\al) \rightarrow 0$~\citep[Theorem 6.8]{Villani09} (together with a convergence of the moments up to order $p$ for unbounded metric spaces).
\end{rem2}

\begin{rem1}{Translations}
A nice feature of the Wasserstein distance over a Euclidean space $\X=\RR^\dim$ for the ground cost $c(x,y)=\norm{x-y}^2$ is that one can factor out translations; indeed, denoting $T_{\tau} : x \mapsto x-\tau$ the translation operator, one has
\eq{
	\Wass_2(T_{\tau\sharp} \al,T_{\tau'\sharp} \be)^2  =
	\Wass_2(\al,\be)^2 - 2\dotp{\tau-\tau'}{ \mean_{\al}-\mean_{\be} } + \norm{\tau-\tau'}^2,
}
\todoK{Check this!}
where $\mean_{\al} \eqdef \int_\X x \d\al(x) \in \RR^\dim$ is the mean of $\al$.
In particular, this implies the nice decomposition of the distance as
\eq{
	\Wass_2(\al,\be)^2  = \Wass_2(\tilde\al,\tilde{\be})^2  + \norm{\mean_{\al}-\mean_{\be}}^2,
}
where $(\tilde\al,\tilde{\be})$ are the ``centered'' zero mean measures $\tilde\al=T_{\mean_{\al}\sharp}\al$.
\end{rem1}
\pagebreak
\begin{rem1}{The case $p = +\infty$}\label{rem-p-inf}
Informally, the limit of $\Wass_p^p$ as $p \rightarrow +\infty$ is
\eql{\label{eq-wass-infty}
	\Wass_{\infty}(\al,\be) \eqdef
	\umin{\pi \in \Couplings(\al,\be)}
		\usup{(x,y) \in \Supp(\pi)} d(x,y),
}
where the sup should be understood as the essential supremum according to the measure $\pi$ on $\X^2$.
In contrast to the cases $p<+\infty$, this is a nonconvex optimization problem, which is difficult to solve numerically and to study theoretically. The $\Wass_{\infty}$ distance is related to the Hausdorff distance between the supports of $(\al,\be)$; see \S~\ref{sec-hausdorff}. We refer to~\citep{champion2008wasserstein} for details.
\end{rem1}

\section{Dual Problem}

The Kantorovich problem~\eqref{eq-mk-discr} is a constrained convex minimization problem, and as such, it can be naturally paired with a so-called dual problem, which is a constrained concave maximization problem. The following fundamental proposition explains the relationship between the primal and dual problems.

\begin{prop}\label{prop-duality-discr}
The Kantorovich problem~\eqref{eq-mk-discr} admits the dual
\eql{\label{eq-dual}
	\MKD_\C(\a,\b) =
	\umax{(\fD,\gD) \in \PotentialsD(\C)} \dotp{\fD}{\a} + \dotp{\gD}{\b},
}
where the set of admissible dual variables is
\eql{\label{eq-feasible-potential}
	\PotentialsD(\C) \eqdef \enscond{
		(\fD,\gD) \in \RR^n \times \RR^m
	}{ \foralls (i,j) \in \range{n} \times \range{m}, \fD \oplus \gD \leq \C }.
}
Such dual variables are often referred to as ``Kantorovich potentials.''
\end{prop}

\begin{proof}
This result is a direct consequence of the more general result on the strong duality for linear programs~\citep[p. 148, Theo. 4.4]{bertsimas1997introduction}. The easier part of the proof, namely, establishing that the right-hand side of Equation~\eqref{eq-dual} is a lower bound of $\MKD_\C(\a,\b)$, is discussed in Remark~\ref{rem-duality} in the next section.
For the sake of completeness, let us derive our result using Lagrangian duality. The Lagangian associated to~\eqref{eq-mk-discr} reads
\eql{\label{eq-mk-lagr}
	\umin{\P \geq 0} \umax{ (\fD,\gD) \in \RR^n \times \RR^m }
		\dotp{\C}{\P} + \dotp{\a - \P\ones_m}{\fD} + \dotp{\b - \transp{\P} \ones_n}{\gD}.
}
We exchange the min and the max above, which is always possible when considering linear programs (in finite dimension)\todoK{donner une ref}, to obtain
\eq{
	\umax{ (\fD,\gD) \in \RR^n \times \RR^m }
	\dotp{\a}{\fD} + \dotp{\b}{\gD}
	+ \umin{\P \geq 0}
		\dotp{\C - \fD\transp{\ones_m} - \ones_n \transp{\gD}}{\P}.
}
We conclude by remarking that
\eq{
	\umin{\P \geq 0} \dotp{\Q}{\P} =
	\choice{
		0 \qifq \Q \geq 0,\\
		-\infty \quad \text{otherwise}
	}
}
so that the constraint reads $\C - \fD\transp{\ones_m} - \ones_n \transp{\gD} = \C-\fD\oplus \gD \geq 0$.
\end{proof}

The primal-dual optimality relation for the Lagrangian~\eqref{eq-mk-lagr} allows us to locate the support of the optimal transport plan (see also~\S\ref{s-complementary})
\eql{\label{eq-mk-pd-rel}
\{(i,j) \in \range{n} \times \range{m} : \P_{i,j}>0\} \subset \enscond{(i,j) \in \range{n} \times \range{m}}{ \fD_i+\gD_j=\C_{i,j} }.
}

\begin{rem}\label{rem-kantorovich-dual}
Following the interpretation given to the Kantorovich problem in Remark~\ref{rem-kantorovich-primal}, we follow with an intuitive presentation of the dual. Recall that in that setup, an operator wishes to move at the least possible cost an overall amount of resources from warehouses to factories. The operator can do so by solving~\eqref{eq-mk-discr}, follow the instructions set out in $\P^\star$, and pay $\dotp{\P^\star}{\C}$ to the transportation company.

\emph{Outsourcing logistics.} Suppose that the operator does not have the computational means to solve the linear program~\eqref{eq-mk-discr}. He decides instead to outsource that task to a vendor. The vendor chooses a pricing scheme with the following structure: the vendor splits the logistic task into that of collecting and then delivering the goods and will apply a collection price $\fD_i$ to collect a unit of resource at each warehouse $i$ (no matter where that unit is sent to) and a price $\gD_j$ to deliver a unit of resource to factory $j$ (no matter from which warehouse that unit comes from). On aggregate, since there are exactly $\a_i$ units at warehouse $i$ and $\b_j$ needed at factory $j$, the vendor asks as a consequence of that pricing scheme a price of $\dotp{\fD}{\a}+\dotp{\gD}{\b}$ to solve the operator's logistic problem.

\emph{Setting prices.} Note that the pricing system used by the vendor allows quite naturally for arbitrarily negative prices. Indeed, if the vendor applies a price vector $\fD$ for warehouses and a price vector $\gD$ for factories, then the total bill will not be changed by simultaneously decreasing all entries in $\fD$ by an arbitrary number and increasing all entries of $\gD$ by that same number, since the total amount of resources in all warehouses is equal to those that have to be delivered to the factories. In other words, the vendor can give the illusion of giving an extremely good deal to the operator by paying him to collect some of his goods, but compensate that loss by simply charging him more for delivering them. Knowing this, the vendor, wishing to charge as much as she can for that service, sets vectors $\fD$ and $\gD$ to be as high as possible.

\emph{Checking prices.} In the absence of another competing vendor, the operator must therefore think of a quick way to check that the vendor's prices are reasonable. A possible way to do so would be for the operator to compute the price $\MKD_{\C}(\a,\b)$ of the most efficient plan by solving problem~\eqref{eq-mk-discr} and check if the vendor's offer is at the very least no larger than that amount. However, recall that the operator cannot afford such a lengthy computation in the first place. Luckily, there is a far more efficient way for the operator to check whether the vendor has a competitive offer. Recall that $\fD_i$ is the price charged by the vendor for picking a unit at $i$ and $\gD_j$ to deliver one at $j$. Therefore, the vendor's pricing scheme implies that transferring one unit of the resource from $i$ to $j$ costs exactly $\fD_i+\gD_j$. Yet, the operator also knows that the cost of shipping one unit from $i$ to $j$ as priced by the transporting company is $\C_{i,j}$. Therefore, if for any pair $i,j$ the aggregate price $\fD_i+\gD_j$ is strictly larger that $\C_{i,j}$, the vendor is charging more than the fair price charged by the transportation company for that task, and the operator should refuse the vendor's offer.

\begin{figure}[h!]
\centering
\includegraphics[width=\linewidth]{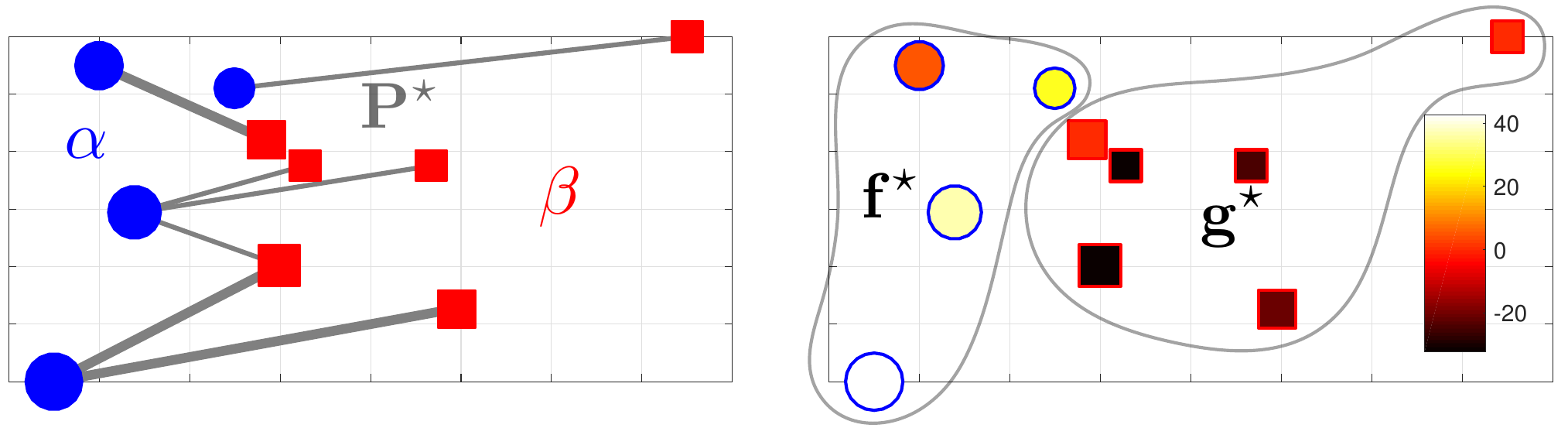}
\caption{\label{dual-solutions-overlay}
Consider in the left plot the optimal transport problem between two discrete measures $\al$ and $\be$, represented respectively by blue dots and red squares. The area of these markers is proportional to the weight at each location. That plot also displays the optimal transport $\P^\star$ using a quadratic Euclidean cost. The corresponding dual (Kantorovich) potentials $\fD^\star$ and $\gD^\star$ that correspond to that configuration are also displayed on the right plot. Since there is a ``price'' $\fD_i^\star$ for each point in $\alpha$ (and conversely for $\gD$ and $\beta$), the color at that point represents the obtained value using the color map on the right. These potentials can be interpreted as relative prices in the sense that they indicate the individual cost, under the best possible transport scheme, to move a mass away at each location in $\al$, or on the contrary to send a mass toward any point in $\be$. The optimal transport cost is therefore equal to the sum of the squared lengths of all the arcs on the left weighted by their thickness or, alternatively, using the dual formulation, to the sum of the values (encoded with colors) multiplied by the area of each marker on the right plot.}
\end{figure}

\emph{Optimal prices as a dual problem.} It is therefore in the interest of the operator to check that for all pairs $i,j$ the prices offered by the vendor verify $\fD_i+\gD_j \leq \C_{i,j}$. Suppose that the operator does check that the vendor has provided price vectors that do comply with these $n\times m$ inequalities. Can he conclude that the vendor's proposal is attractive? Doing a quick back of the hand calculation, the operator does indeed conclude that it is in his interest to accept that offer. Indeed, since any of his transportation plans $\P$ would have a cost
$\dotp{\P}{\C}=\sum_{i,j} \P_{i,j}\C_{i,j}$, the operator can conclude by applying these $n\times m$ inequalities that for any transport plan $\P$ (including the optimal one $\P^\star$), the marginal constraints imply
\eq{\begin{aligned}
	\sum_{i,j} \P_{i,j}\C_{i,j} &\geq \sum_{i,j} \P_{i,j} \left(\fD_{i} + \gD_j\right) = \left(\sum_{i} \fD_{i} \sum_{j} \P_{i,j}\right) + \left(\sum_{j} \gD_{j} \sum_{i} \P_{i,j}\right)\\ 	& =\dotp{\fD}{\a}+\dotp{\gD}{\b},
\end{aligned}}
and therefore observe that \emph{any} attempt at doing the job by himself would necessarily be more expensive than the vendor's price.

Knowing this, the vendor must therefore find a set of prices $\fD,\gD$ that maximize $\dotp{\fD}{\a}+\dotp{\gD}{\b}$ but that must satisfy at the very least for all $i,j$ the basic inequality that $\fD_i+\gD_j\leq \C_{i,j}$ for his offer to be accepted, which results in Problem~\eqref{eq-dual}. One can show, as we do later in~\S\ref{s-kantorovich}, that the best price obtained by the vendor is in fact exactly equal to the best possible cost the operator would obtain by computing $\MKD_{\C}(\a,\b)$.

Figure~\ref{dual-solutions-overlay} illustrates the primal and dual solutions resulting from the same transport problem. On the left, blue dots represent warehouses and red dots stand for factories; the areas of these dots stand for the probability weights $\a,\b$, links between them represent an optimal transport, and their width is proportional to transfered amounts. Optimal prices obtained by the vendor as a result of optimizing Problem~\eqref{eq-dual} are shown on the right. Prices have been chosen so that their mean is equal to 0. The highest relative prices come from collecting goods at an isolated warehouse on the lower left of the figure, and delivering goods at the factory located in the upper right area.
\end{rem}

\begin{rem2}{Dual problem between arbitrary measures}
	To extend this primal-dual construction to arbitrary measures, it is important to realize that measures are naturally paired in duality with continuous functions (a measure can be accessed only through integration against continuous functions). The duality is formalized in the following proposition, which boils down to Proposition~\ref{prop-duality-discr} when dealing with discrete measures.

	\begin{prop}
	One has
	\eql{\label{eq-dual-generic}
		\MK_\c(\al,\be) =
		\usup{(\f,\g) \in \Potentials(\c)}
			\int_\X \f(x) \d\al(x) + \int_\Y \g(y) \d\be(y),
	}
	where the set of admissible dual potentials is
	\eql{\label{eq-dfn-pot-dual}
		\Potentials(\c) \eqdef \enscond{
			(\f,\g) \in \Cc(\X) \times \Cc(\Y)
		}{
			\forall (x,y), \f(x)+\g(y) \leq \c(x,y)
		}.
	}
	Here, $(\f,\g)$ is a pair of continuous functions and are also called, as in the discrete case, ``Kantorovich potentials.''
	\end{prop}

	\todoK{\begin{proof}
	TODO.
	\end{proof}
	}

	The discrete case~\eqref{eq-dual} corresponds to the dual vectors being samples of the continuous potentials, \ie $(\fD_i,\gD_j)=(\f(x_i),\g(y_j))$.
	The primal-dual optimality conditions allow us to track the support of the optimal plan, and~\eqref{eq-mk-pd-rel} is generalized as
	\eql{\label{eq-mk-pd-rel-cont}
		\Supp(\pi) \subset \enscond{(x,y) \in \X \times \Y}{ \f(x)+\g(y)=\c(x,y) }.
	}

	Note that in contrast to the primal problem~\eqref{eq-mk-generic}, showing the existence of solutions to~\eqref{eq-dual-generic} is nontrivial, because the constraint set $\Potentials(\c)$ is not compact and the function to minimize noncoercive.
	Using the machinery of $c$-transform detailed in \S~\ref{s-c-transform}, in the case $c(x,y)=\dist(x,y)^p$ with $p\geq 1$, one can, however, show that optimal $(\f,\g)$ are necessarily Lipschitz regular, which enables us to replace the constraint by a compact one.
\end{rem2}

\begin{rem2}{Unconstrained dual}\label{rem-uncons-dual}
In the case $\int_\X \d\mu=\int_\Y\d\nu = 1$, the constrained dual problem~\eqref{eq-dual-generic} can be replaced by an unconstrained one,
\eql{\label{eq-dual-uncons}
	\MK_\c(\al,\be) =
		\usup{(\f,\g) \in \Cc(\X) \times \Cc(\Y)}
			\int_\X \f\d\al + \int_\Y \g \d\be
			+ \umin{\X \otimes \Y} ( c - f \oplus g ),
}
where we denoted $(f\oplus g)(x,y)=f(x)+g(y)$. Here the minimum should be considered as the essential supremum associated to the measure $\mu \otimes \nu$, i.e., it does not change if $f$ or $g$ is modified on sets of zero measure for $\mu$ and $\nu$.
This alternative dual formulation was pointed out to us by Francis Bach.
It is obtained from the primal problem~\eqref{eq-mk-generic} by adding the redundant constraint $\int \d\pi=1$.
\end{rem2}

\begin{rem2}{Monge--Kantorovich equivalence---Brenier theorem}\label{rem-exist-mongemap}
The following theorem is often attributed to~\citet{Brenier91} and ensures that in $\RR^\dim$ for $p=2$, if at least one of the two input measures has a density, and for measures with second order moments, then the Kantorovich and Monge problems are equivalent. The interested reader should also consult variants of the same result published more or less at the same time by \citet{cuesta1989,ruschendorf1990characterization}, including notably the original result in~\citep{MR923203} and a precursor by~\citet{knott1984optimal}.

\begin{thm}[Brenier]
	In the case $\X=\Y=\RR^\dim$ and $c(x,y)=\norm{x-y}^2$, if at least one of the two input measures (denoted $\al$) has a density $\density{\al}$ with respect to the Lebesgue measure, then the optimal $\pi$ in the Kantorovich formulation~\eqref{eq-mk-generic} is unique and is supported on the graph $(x,\T(x))$ of a ``Monge map'' $\T : \RR^\dim \rightarrow \RR^\dim$. This means that $\pi = (\Id,\T)_{\sharp} \mu$, \ie
\eql{\label{eq-brenier-map}
	\foralls h \in \Cc(\X \times \Y), \quad
		\int_{\X \times \Y} h(x,y) \d \pi(x,y) = \int_{\X} h(x,\T(x)) \d\mu(x).
}
Furthermore, this map $\T$ is uniquely defined as the gradient of a convex function $\phi$, $\T(x) = \nabla\phi(x)$, where $\phi$ is the unique (up to an additive constant) convex function such that $(\nabla\phi)_\sharp \mu=\nu$. This convex function is related to the dual potential $\f$ solving~\eqref{eq-dual-generic} as $\phi(x)=\frac{\norm{x}^2}{2}-\f(x)$.
\end{thm}

\begin{proof}
	We sketch the main ingredients of the proof; more details can be found, for instance, in~\citep{SantambrogioBook}.
	We remark that $\int c \d\pi = C_{\al,\be} - 2 \int \dotp{x}{y}\d\pi(x,y)$, where the constant is $C_{\al,\be} = \int \norm{x}^2\d\al(x) + \int \norm{y}^2\d\be(y)$. Instead of solving~\eqref{eq-mk-generic}, one can thus consider the problem
	\eq{
		\umax{\pi \in \Couplings(\al,\be)}
			\int_{\X \times \Y} \dotp{x}{y} \d\pi(x,y),
	}
	whose dual reads
	\eql{\label{eq-brenier-proof-1}
		\umin{(\phi,\psi)}
		\enscond{
			\int_\X \phi \d\al + \int_\Y \psi \d\be
		}{
			\forall (x,y), \quad \phi(x)+\psi(y) \geq \dotp{x}{y}
		}.
	}
	The relation between these variables and those of~\eqref{eq-dfn-pot-dual} is
	$(\phi,\psi) = (\frac{\norm{\cdot}^2}{2}-\f,\frac{\norm{\cdot}^2}{2}-\g)$.
	One can replace the constraint by
	\eql{\label{eq-brenier-proof-leg}
		\foralls y, \quad \psi(y) \geq \phi^*(y) \eqdef \usup{x} \dotp{x}{y} - \phi(x).
	}
	Here $\phi^*$ is the Legendre transform of $\phi$ and is a convex function as a supremum of linear forms (see also~\eqref{eq-legendre}). Since the objective appearing in~\eqref{eq-brenier-proof-2} is linear and the integrating measures positive, one can minimize explicitly with respect to $\psi$ and set $\psi=\phi^*$ in order to consider the unconstrained problem
	\eql{\label{eq-brenier-proof-2}
		\umin{\phi} \int_\X \phi \d\al + \int_\Y \phi^* \d\be;
	}
	see also~\S\ref{sec-c-transforms} and~\S\ref{s-c-transform}, where that idea is applied respectively in the discrete setting and for generic costs $\c(x,y)$.
	By iterating this argument twice, one can replace $\phi$ by $\phi^{**}$, which is a convex function, and thus impose in~\eqref{eq-brenier-proof-2} that $\phi$ is convex.
	Condition~\eqref{eq-mk-pd-rel-cont} shows that an optimal $\pi$ is supported on $\enscond{(x,y)}{\phi(x)+\phi^*(y)=\dotp{x}{y}}$, which shows that such a $y$ is optimal for the minimization~\eqref{eq-brenier-proof-leg} of the Legendre transform, whose optimality condition reads $y \in \partial \phi(x)$.
	Since $\phi$ is convex, it is differentiable almost everywhere, and since $\al$ has a density, it is also differentiable $\al$-almost everywhere.
	This shows that for each $x$, the associated $y$ is uniquely defined $\al$-almost everywhere as $y = \nabla\phi(x)$, and it shows that necessarily $\pi = (\Id,\nabla\phi)_\sharp \al$.
\end{proof}

This result shows that in the setting of $\Wass_2$ with no-singular densities, the Monge problem~\eqref{eq-monge-continuous} and its Kantorovich relaxation~\eqref{eq-mk-generic} are equal (the relaxation is tight). This is the continuous counterpart of Proposition~\ref{prop-matching-kanto} for the assignment case~\eqref{prop-matching-kanto}, which states that the minimum of the optimal transport problem is achieved at a permutation matrix (a discrete map) when the marginals are equal and uniform.
Brenier's theorem, stating that an optimal transport map must be the gradient of a convex function, provides a useful generalization of the notion of increasing functions in dimension more than one. This is the main reason why optimal transport can be used to define quantile functions in arbitrary dimensions, which is in turn useful for applications to quantile regression problems~\citep{carlier2016vector}.

Note also that this theorem can be extended in many directions.
The condition that $\al$ has a density can be weakened to the condition that it does not give mass to ``small sets'' having Hausdorff dimension smaller than $\dim-1$ (\emph{e.g.} hypersurfaces).
One can also consider costs of the form $\c(x,y)=h(x-y)$, where $h$ is a strictly convex function.
\end{rem2}

\begin{rem2}{Monge--Amp\`ere equation}\label{rem:MA}
For measures with densities, using~\eqref{eq-pfwd-density}, one obtains that $\phi$ is the unique (up to the addition of a constant) convex function which solves the following Monge--Amp\`ere-type equation:
\eql{\label{eq-monge-ampere}
	\det(\partial^2\phi(x))  \density{\be}(\nabla\phi(x)) = \density{\al}(x)
}
where $\partial^2\phi(x) \in \RR^{\dim \times \dim}$ is the Hessian of $\phi$. The Monge--Amp\`ere operator $\det(\partial^2\phi(x))$ can be understood as a nonlinear degenerate Laplacian. In the limit of small displacements, $\phi=\Id + \epsilon\psi$, one indeed recovers the Laplacian $\Delta$ as a linearization since for smooth maps
\eq{
	\det(\partial^2\phi(x)) = 1 + \epsilon \Delta \psi(x) + o(\epsilon).
}
The convexity constraint forces $\det(\partial^2\phi(x)) \geq 0$ and is necessary for this equation to have a solution.
There is a large body of literature on the theoretical analysis of the Monge--Amp\`ere equation, and in particular the regularity of its solution---see, for instance,~\citep{gutierrez2016monge}; we refer the interested read to the review paper by~\citet{caffarelli2003monge}.
A major difficulty is that in full generality, solutions need not be smooth, and one has to resort to the machinery of Alexandrov solutions when the input measures are arbitrary (\emph{e.g.} Dirac masses).
Many solvers have been proposed in the simpler case of the Monge--Amp\`ere equation  $\det(\partial^2 \phi(x)) = f(x)$ for a fixed right-hand-side $f$; see, for instance,~\citep{benamou2016monotone} and the references therein. In particular, capturing anisotropic convex functions requires special care, and usual finite differences can be inaccurate.
For optimal transport, where $f$ actually depends on $\nabla \phi$, the discretization of Equation~\eqref{eq-monge-ampere}, and the boundary condition result in technical challenges outlined in~\citep{benamou2014numerical} and the references therein.
Note also that related solvers based on fixed-point iterations have been applied to image registration~\citep{haker2004optimal}.
\end{rem2}

\section{Special Cases}\label{sec:specialcases}

In general, computing OT distances is numerically involved. Before detailing in \S\S \ref{c-algo-basics},\ref{c-entropic}, and \ref{c-dynamic} different numerical solvers, we first review special favorable cases where the resolution of the OT problem is relatively easy.

\begin{rem}[Binary cost matrix and 1-norm]\label{rem-binary}
One can easily check that when the cost matrix $\C$ is 0 on the diagonal and $1$ elsewhere, namely, when $\C=\ones_{n\times n}-\Identity_n$, the 1-Wasserstein distance between $\a$ and $\b$ is equal to the 1-norm of their difference, $\MKD_\C(\a,\b)=\norm{\a-\b}_1$.
\end{rem}

\begin{rem1}{Kronecker cost function and total variation}
In addition to Remark~\ref{rem-binary} above, one can also easily check that this result extends to arbitrary measures in the case where $c(x,y)$ is $0$ if $x=y$ and 1 when $x\ne y$. The OT distance between two discrete measures $\al$ and $\be$ is equal to their total variation distance (see also Example~\ref{exmp-tv}).
\end{rem1}

\begin{rem1}{1-D case---Empirical measures}\label{rem-1d-empir}
Here $\X=\RR$. Assuming $\al = \frac{1}{n}\sum_{i=1}^n \de_{x_i}$ and $\be = \frac{1}{n}\sum_{j=1}^n \de_{y_j}$, and assuming (without loss of generality) that the points are ordered, \ie $x_1 \leq x_2 \leq \cdots \leq x_n$ and $y_1 \leq y_2 \leq \cdots \leq y_n$, then one has the simple formula
\eql{\label{eq-1d-empirical}
	\Wass_p(\al,\be)^p = \frac{1}{n} \sum_{i=1}^n |x_i-y_i|^p,
}
\ie locally (if one assumes distinct points), $\Wass_p(\al,\be)$ is the $\ell^p$ norm between two vectors of ordered values of $\al$ and $\be$. That statement is valid only locally, in the sense that the order (and those vector representations) might change whenever some of the values change. That formula is a simple consequence of the more general setting detailed in Remark~\ref{rem-1d-ot-generic}.
Figure~\ref{fig-1d-discrete}, top row, illustrates the 1-D transportation map between empirical measures with the same number of points.  The bottom row shows how this monotone map generalizes to arbitrary discrete measures.

It is also possible to leverage this 1-D computation to also compute efficiently OT on the circle as shown by~\citet{delon-circle}.
Note that if the cost is a concave function of the distance, notably when $p<1$, the behavior of the optimal transport plan is very different, yet efficient solvers also exist~\citep{delon-concave}.
\end{rem1}

\begin{figure}[h!]
\centering
\includegraphics[width=.65\linewidth]{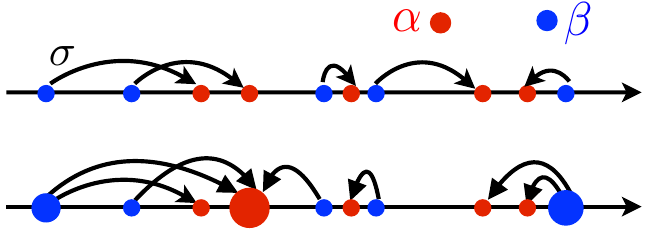}
\caption{\label{fig-1d-discrete}
1-D optimal couplings: each arrow $x_i \rightarrow y_j$ indicates a nonzero $\P_{i,j}$ in the optimal coupling.
Top: empirical measures with same number of points (optimal matching).
Bottom: generic case.
This corresponds to monotone rearrangements, if $x_i \leq x_{i'}$ are such that $\P_{i,j} \neq 0, \P_{i',j'} \neq 0$, then necessarily $y_j \leq y_{j'}$.
}
\end{figure}

\begin{rem1}{Histogram equalization}
One-dimensional op\-ti\-mal tr\-ans\-port can be used to perform histogram equalization, with applications to the normalization of the palette of grayscale images, see Figure~\ref{fig-hist-eq}. In this case, one denotes $(\bar x_i)_i$ and $(\bar y_j)_j$ the gray color levels ($0$ for black, $1$ for white, and all values in between) of all pixels of the two input images enumerated in a predefined order (\ie columnwise). Assuming the number of pixels in each image is the same and equal to $n\times m$, sorting these color levels defines $x_i = \bar x_{\si_1(i)}$ and $y_j = \bar y_{\si_2(j)}$ as in Remark~\ref{rem-1d-empir}, where $\si_1, \si_2 : \{1,\ldots,nm\} \rightarrow \{1,\ldots,nm\}$ are permutations, so that $\si \eqdef \si_2 \circ \si_1^{-1}$ is the optimal assignment between the two discrete distributions. For image processing applications, $(\bar y_{\si(i)})_i$ defines the color values of an equalized version of $\bar x$, whose empirical distribution matches exactly the one of $\bar y$. The equalized version of that image can be recovered by folding back that $nm$-dimensional vector as an image of size $n\times m$. Also, $t \in [0,1] \mapsto (1-t) \bar x_{i} + t \bar y_{\si(i)}$ defines an interpolation between the original image and the equalized one, whose empirical distribution of pixels is the displacement interpolation (as defined in~\eqref{eq-displacement-1d-cumul}) between those of the inputs.
\end{rem1}

\begin{figure}[h!]
\centering
\begin{tabular}{@{}c@{\hspace{1mm}}c@{\hspace{1mm}}c@{\hspace{1mm}}c@{\hspace{1mm}}c@{\hspace{1mm}}@{}}
\includegraphics[width=.19\linewidth]{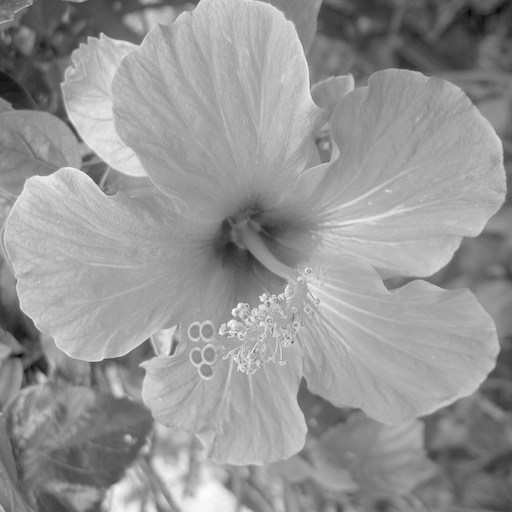}&
\includegraphics[width=.19\linewidth]{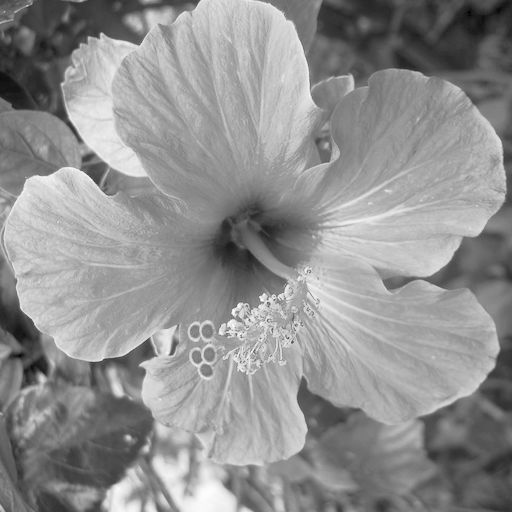}&
\includegraphics[width=.19\linewidth]{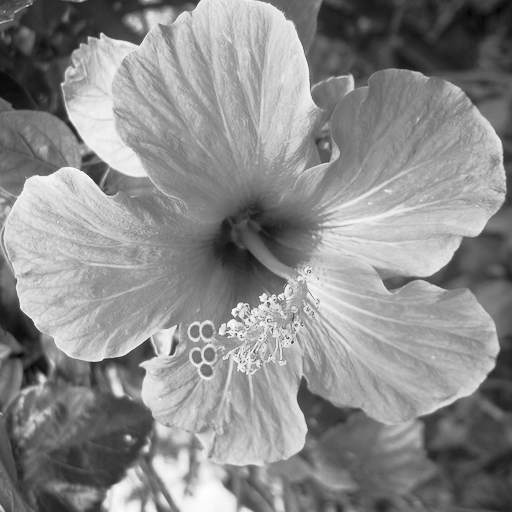}&
\includegraphics[width=.19\linewidth]{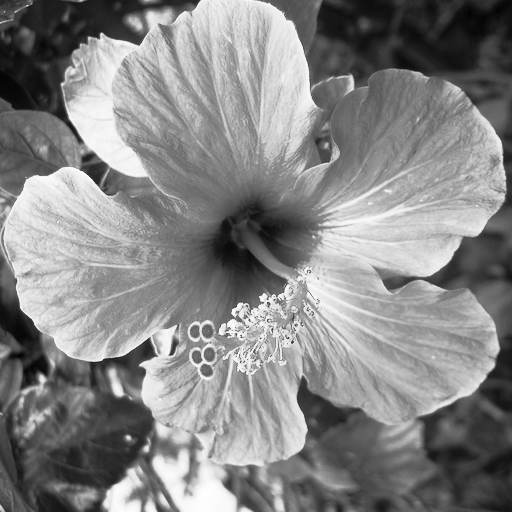}&
\includegraphics[width=.19\linewidth]{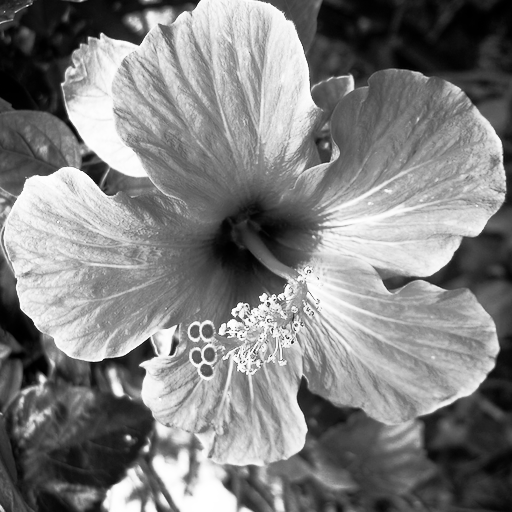}\\
\includegraphics[width=.19\linewidth]{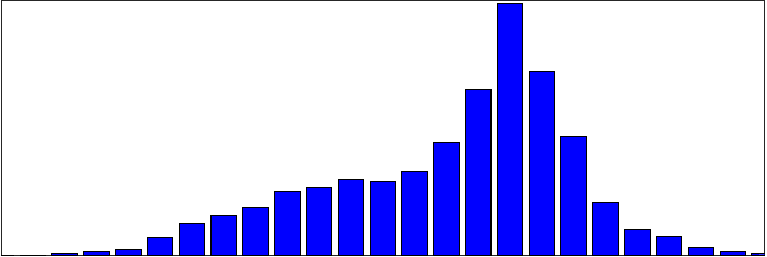}&
\includegraphics[width=.19\linewidth]{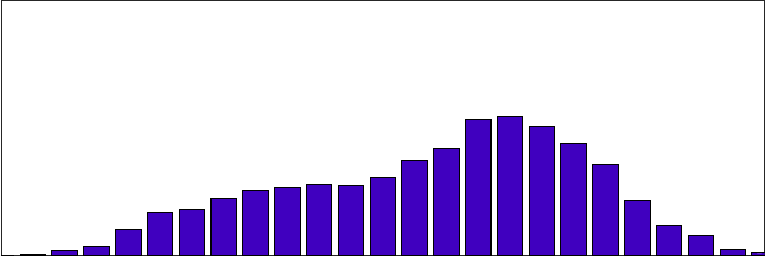}&
\includegraphics[width=.19\linewidth]{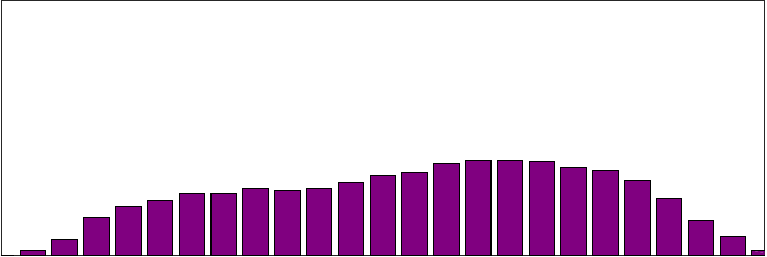}&
\includegraphics[width=.19\linewidth]{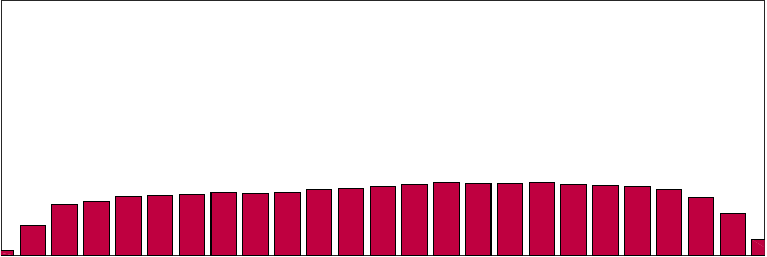}&
\includegraphics[width=.19\linewidth]{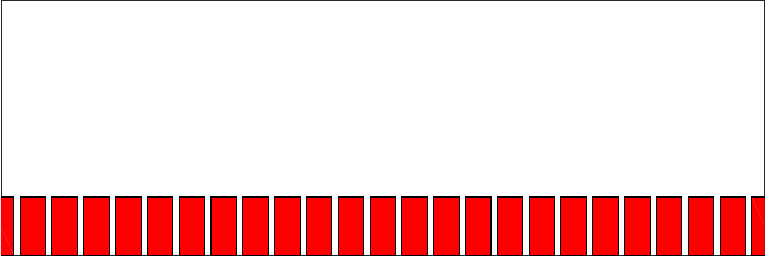}\\
$t=0$ & $t=0.25$& $t=0.5$& $t=.75$& $t=1$
\end{tabular}
\caption{\label{fig-hist-eq}
Histogram equalization for image processing, where $t$ parameterizes the displacement interpolation between the histograms.
}
\end{figure}

\begin{rem2}{1-D case---Generic case}\label{rem-1d-ot-generic}
For a measure $\al$ on $\RR$, we introduce the cumulative distribution function from $\RR$ to $\rightarrow [0,1]$ defined as
\eql{\label{eq-cumul-defn}
	\foralls x \in \RR, \quad \cumul{\al}(x) \eqdef \int_{-\infty}^x \d\al,
}
and its pseudoinverse  $\cumul{\al}^{-1} : [0,1] \rightarrow \RR \cup \{-\infty\}$
\eql{\label{eq-pseudo-inv-cum}
	\foralls r \in [0,1], \quad \cumul{\al}^{-1}(r) = \umin{x} \enscond{x \in \RR \cup \{-\infty\} }{ \cumul{\al}(x) \geq r }.
}
That function is also called the generalized quantile function of $\alpha$. For any $p \geq 1$, one has
\eql{\label{eq-wass-cumul}
	\Wass_p(\al,\be)^p = \norm{ \cumul{\al}^{-1} - \cumul{\be}^{-1} }_{L^p([0,1])}^p = \int_0^1 | \cumul{\al}^{-1}(r) - \cumul{\be}^{-1}(r) |^p \d r.
}
This means that through the map $\al \mapsto \cumul{\al}^{-1}$, the Wasserstein distance is isometric to a linear space equipped with the $L^p$ norm or, equivalently, that the Wasserstein distance for measures on the real line is a Hilbertian metric.
This makes the geometry of 1-D optimal transport very simple but also very different from its geometry in higher dimensions, which is not Hilbertian as discussed in Proposition~\ref{prop-negative-definite} and more generally in~\S\ref{sec-non-embeddability}.
For $p=1$, one even has the simpler formula
\begin{align}\label{eq-w1-1d}
	\Wass_1(\al,\be) &= \norm{ \cumul{\al} - \cumul{\be} }_{L^1(\RR)} =
	\int_\RR | \cumul{\al}(x) - \cumul{\be}(x) | \d x \\
	&= \int_\RR \abs{ \int_{-\infty}^x \d(\al-\be) } \d x,
\end{align}
which shows that $\Wass_1$ is a norm (see~\S\ref{sec-w1-eucl} for the generalization to arbitrary dimensions).
An optimal Monge map $\T$ such that $\T_\sharp \al=\be$ is then defined by
\eql{\label{eq-OT-map-1d}
 	\T = \cumul{\be}^{-1} \circ \cumul{\al}.
}
Figure~\ref{fig-1d-ot} illustrates the computation of 1-D OT through cumulative functions. It also displays displacement interpolations, computed as detailed in~\eqref{eq-displacement-1d-cumul}; see also Remark~\ref{rem-bary-1d}. For a detailed survey of the properties of optimal transport in one dimension, we refer the reader to~\cite[Chapter 2]{SantambrogioBook}.
\end{rem2}

\newcommand{\MyFigCumulMeas}[1]{\includegraphics[width=.33\linewidth]{1d-cumulative/#1}}
\newcommand{\MyFigCumulCum}[1]{\includegraphics[width=.24\linewidth]{1d-cumulative/#1}}
\begin{figure}[ht!]
\centering
\begin{tabular}{@{}c@{\hspace{1mm}}c@{\hspace{1mm}}c@{}}
\MyFigCumulMeas{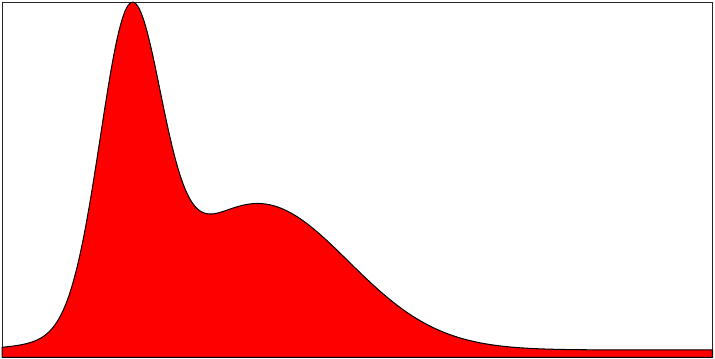}&
\MyFigCumulMeas{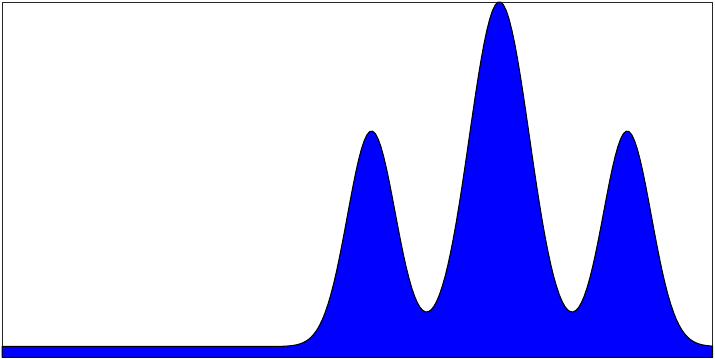}&
\MyFigCumulMeas{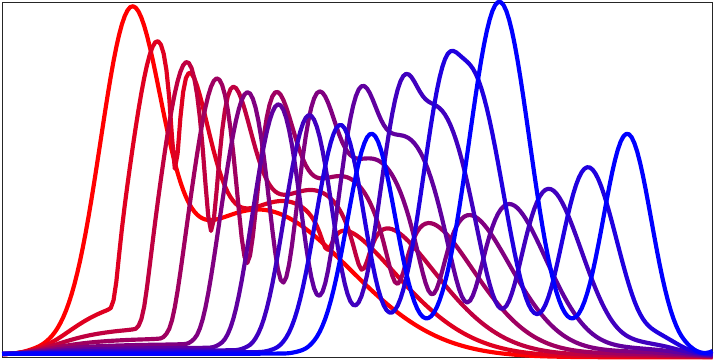}\\
$\mu$ & $\nu$ & ${ (t\T+(1-t)\Id)_\sharp \mu}$
\end{tabular}
\begin{tabular}{@{}c@{\hspace{2mm}}c@{\hspace{2mm}}c@{\hspace{2mm}}c@{}}
\MyFigCumulCum{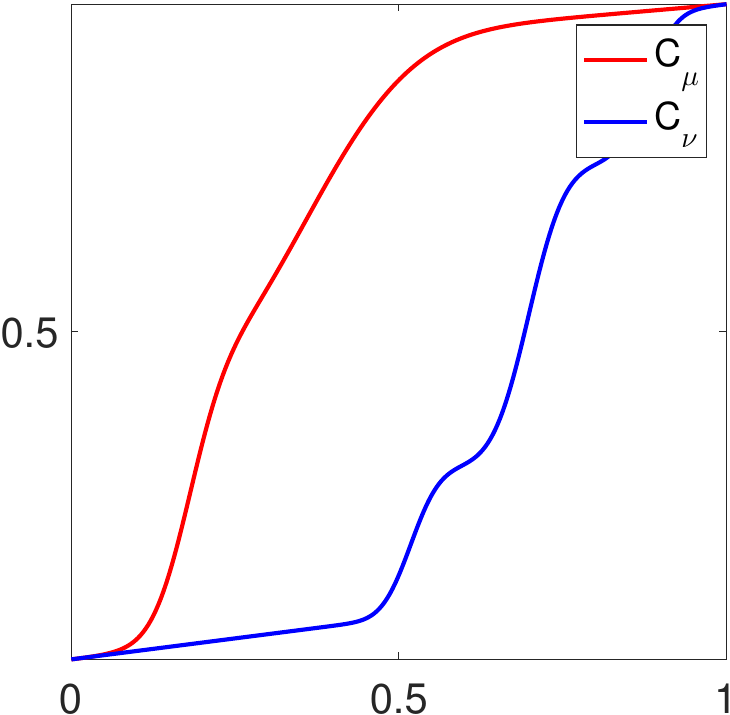}&
\MyFigCumulCum{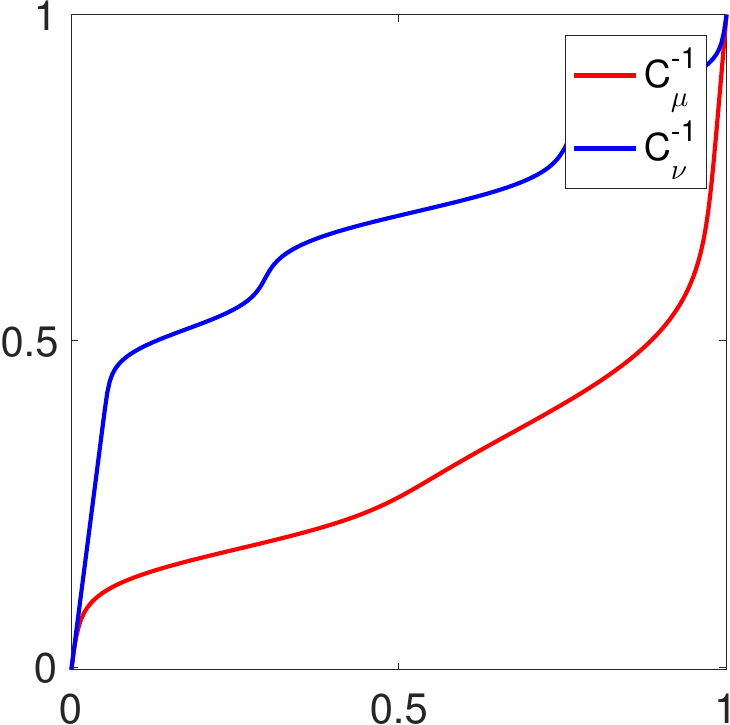}&
\MyFigCumulCum{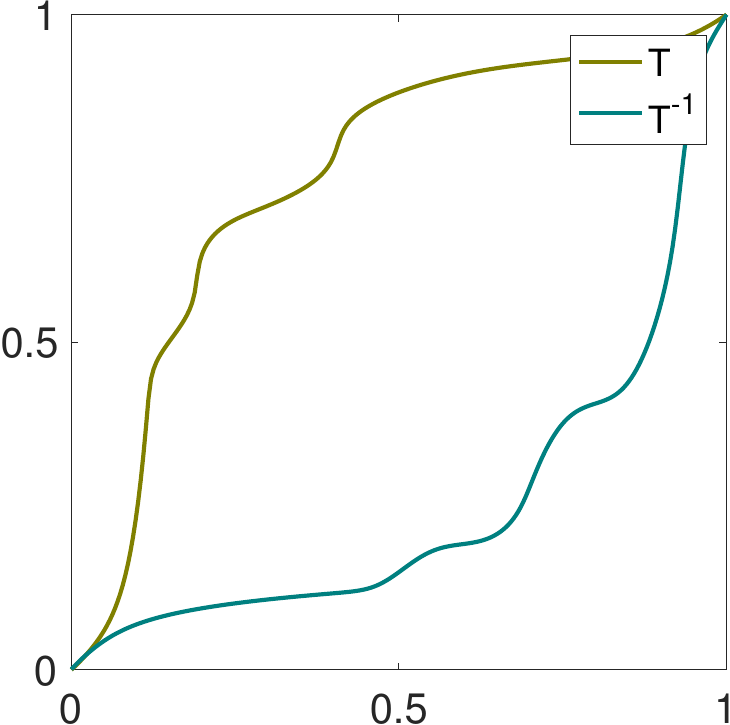}&
\MyFigCumulCum{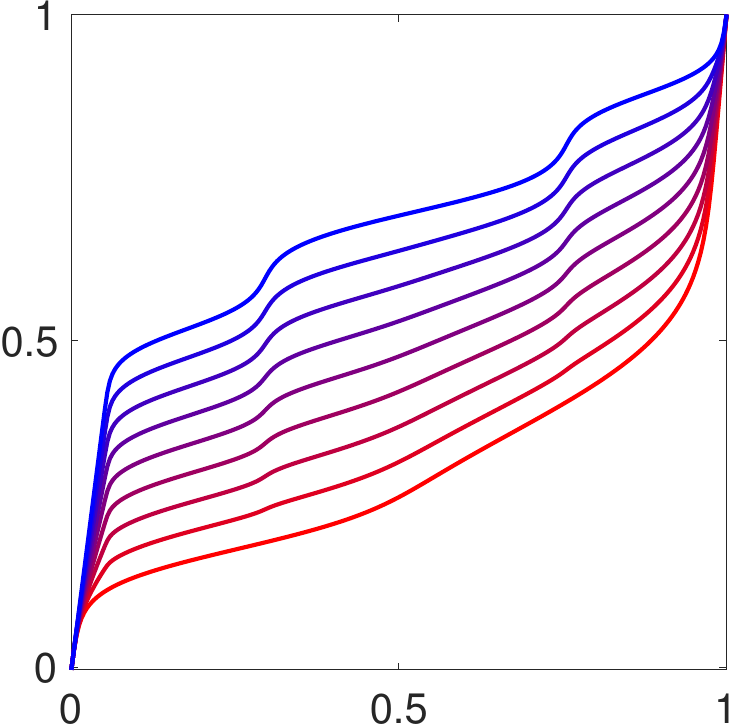}\\
$(\cumul{\al},\cumul{\be})$ &
$(\cumul{\al}^{-1},\cumul{\be}^{-1})$ &
$(T,T^{-1})$ &
$(1-t)\cumul{\al}^{-1}+t\cumul{\be}^{-1}$
\end{tabular}
\caption{\label{fig-1d-ot}
Computation of OT and displacement interpolation between two 1-D measures, using cumulant function as detailed in~\eqref{eq-OT-map-1d}.
}
\end{figure}

\begin{rem2}{Distance between Gaussians}\label{rem-dist-gaussians}
If $\al = \Nn(\mean_\al,\cov_\al)$ and $\be = \Nn(\mean_\be,\cov_\be)$ are two Gaussians in $\RR^\dim$, then one can show that the following map
\eql{\label{eq-transport-Bures}T:x\mapsto \mean_\be + A(x-\mean_\al),}
where
$$A=\cov_\al^{-\tfrac{1}{2}}\Big(\cov_\al^{\tfrac{1}{2}}\cov_\be\cov_\al^{\tfrac{1}{2}}\Big)^{\tfrac{1}{2}}\cov_\al^{-\tfrac{1}{2}}=\transp{A},$$
is such that $T_\sharp \rho_\al = \rho_\be$. Indeed, one simply has to notice that the change of variables formula~\eqref{eq-pfwd-density} is satisfied since
$$
\begin{aligned}\rho_\be(T(x))&=\det(2\pi\cov_\be)^{-\tfrac{1}{2}} \exp(-\dotp{T(x)-\mean_\be}{\cov_\be^{-1}(T(x)-\mean_\be)})\\
&= \det(2\pi\cov_\be)^{-\tfrac{1}{2}} \exp(-\dotp{ x-\mean_\al}{\transp{A}\cov_\be^{-1}A(x-\mean_\al)}) \\
&= \det(2\pi\cov_\be)^{-\tfrac{1}{2}} \exp(-\dotp{ x-\mean_\al}{\cov_\al^{-1}(x-\mean_\al)}),
\end{aligned}$$
and since $T$ is a linear map we have that
$$|\det T'(x)|= \det A = \left(\frac{\det\cov_\be}{\det\cov_\al}\right)^{\tfrac{1}{2}}$$
 and we therefore recover $\rho_\al=|\det T'| \rho_\be$ meaning $T_\sharp \al = \be$. Notice now that $T$ is the gradient of the convex function $\psi:x\mapsto \tfrac{1}{2}\dotp{x-\mean_\al}{A (x-\mean_\al)} + \dotp{\mean_\be}{x}$ to conclude, using~\citeauthor{Brenier91}'s theorem~\citeyearpar{Brenier91} (see Remark~\ref{rem-exist-mongemap}), that $T$ is optimal. Both that map $T$ and the corresponding potential $\psi$ are illustrated in Figures~\ref{fig-gaussians-2d-T} and~\ref{fig-gaussians-2d-psi}

With additional calculations involving first and second order moments of $\rho_\al$, we obtain that the transport cost of that map is
\eql{\label{eq-dist-gauss}
	\Wass_2^2( \al,\be ) = \norm{ \mean_\al - \mean_\be }^2 + \Bb(\cov_\al,\cov_\be)^2,
}
where $\Bb$ is the so-called~\citeauthor{bures1969extension} metric~\citeyearpar{bures1969extension} between positive definite matrices (see also~\citet{,forrester2016relating}),
\eql{\label{eq-bure-defn}
	\Bb(\cov_\al,\cov_\be)^2 \eqdef \tr\pa{
		\cov_\al + \cov_\be - 2 ( \cov_\al^{1/2} \cov_\be \cov_\al^{1/2} )^{1/2}
	},
}
where $\cov^{1/2}$ is the matrix square root. One can show that $\Bb$ is a distance on covariance matrices and that $\Bb^2$ is convex with respect to both its arguments. \todoK{ref ? proof ? }
In the case where $\cov_\al = \diag(r_i)_i$ and $\cov_\be = \diag(s_i)_i$ are diagonals, the Bures metric is the Hellinger distance
\eq{
	\Bb(\cov_\al,\cov_\be) = \norm{ \sqrt{r}-\sqrt{s} }_2.
}
For 1-D Gaussians, $\Wass_2$ is thus the Euclidean distance on the 2-D plane plotting the mean and the standard deviation of a Gaussian $(\mean,\sqrt{\cov})$, as illustrated in Figure~\ref{fig-1d-gaussian}.
For a detailed treatment of the Wasserstein geometry of Gaussian distributions, we refer to~\citet{takatsu2011wasserstein}, and for additional considerations on the Bures metric the reader can consult the very recent references~\citep{malago2018wasserstein,bhatia2018bures}. One can also consult~\citep{NIPS2018_8226} for a a recent application of this metric to compute probabilistic embeddings for words,~\citep{NIPS2018_8067} to see how it is used to compute a robust extension to Kalman filtering, or~\citep{NIPS2017_7149} in which it is applied to covariance functions in reproducing kernel Hilbert spaces.
\end{rem2}
\begin{figure}[h!]
\centering
\includegraphics[width=.7\linewidth]{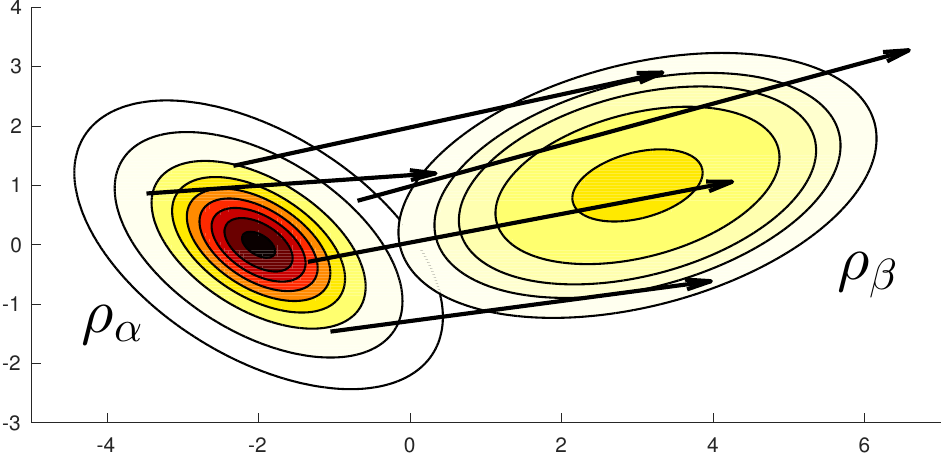}
\caption{\label{fig-gaussians-2d-T} Two Gaussians $\rho_\al$ and $\rho_\be$, represented using the contour plots of their densities, with respective mean and variance matrices $\mean_\al=(-2,0),\cov_\al=\frac{1}{2}\left(1 -\tfrac{1}{2};-\tfrac{1}{2}  1\right)$ and $\mean_\be=(3,1), \cov_\be=\left(2, \tfrac{1}{2}; \tfrac{1}{2}, 1\right)$. The arrows originate at random points $x$ taken on the plane and end at the corresponding mappings of those points $T(x)=\mean_\be + A(x-\mean_\al)$.}
\end{figure}

\begin{rem2}{Distance between elliptically contoured distributions}\label{rem-dist-elliptic}
\citeauthor{gelbrich1990formula} provides a more general result than that provided in Remark~\ref{rem-dist-gaussians}: the Bures metric between Gaussians extends more generally to \emph{elliptically contoured distributions}~\citeyearpar{gelbrich1990formula}.
In a nutshell, one can first show that for two measures with given mean and covariance matrices, the distance between the two Gaussians with these respective parameters is a lower bound of the Wasserstein distance between the two measures~\citep[Theorem 2.1]{gelbrich1990formula}. Additionally, the closed form~\eqref{eq-dist-gauss} extends to families of elliptically contoured densities: If two densities $\rho_\al$ and $\rho_\be$ belong to such a family, namely when $\rho_\al$ and $\rho_\be$ can be written for any point $x$ using a mean and positive definite parameter,

\eq{\begin{aligned}\rho_\al(x) = \frac{1}{\sqrt{\det(\A)}}h(\dotp{x-\mean_\al}{\A^{-1}(x-\mean_\al)})\\ \rho_\be(x)= \frac{1}{\sqrt{\det(\B)}} h(\dotp{x-\mean_\be}{\B^{-1}(x-\mean_\be)}),\end{aligned}}
for the same nonnegative valued function $h$ such that the integral $$\int_{\RR^\dim} h(\dotp{x}{x})dx=1,$$ then their optimal transport map is also the linear map \eqref{eq-transport-Bures} and their Wasserstein distance is also given by the expression $\eqref{eq-dist-gauss}$, with a slightly different scaling of the Bures metric that depends only the generator function $h$. For instance, that scaling is $1$ for Gaussians ($h(t)=e^{-t/2}$) and $1/(\dim+2)$ for uniform distributions on ellipsoids ($h$ the indicator function for $[0,1]$). This result follows from the fact that the covariance matrix of an elliptic distribution is a constant times its positive definite parameter~\citep[Theo. 4(ii)]{gomez2003survey} and that the Wasserstein distance between elliptic distributions is a function of the Bures distance between their covariance matrices~\citep[Cor. 2.5]{gelbrich1990formula}.
\end{rem2}
\begin{figure}[h!]
\centering
\includegraphics[width=.85\linewidth]{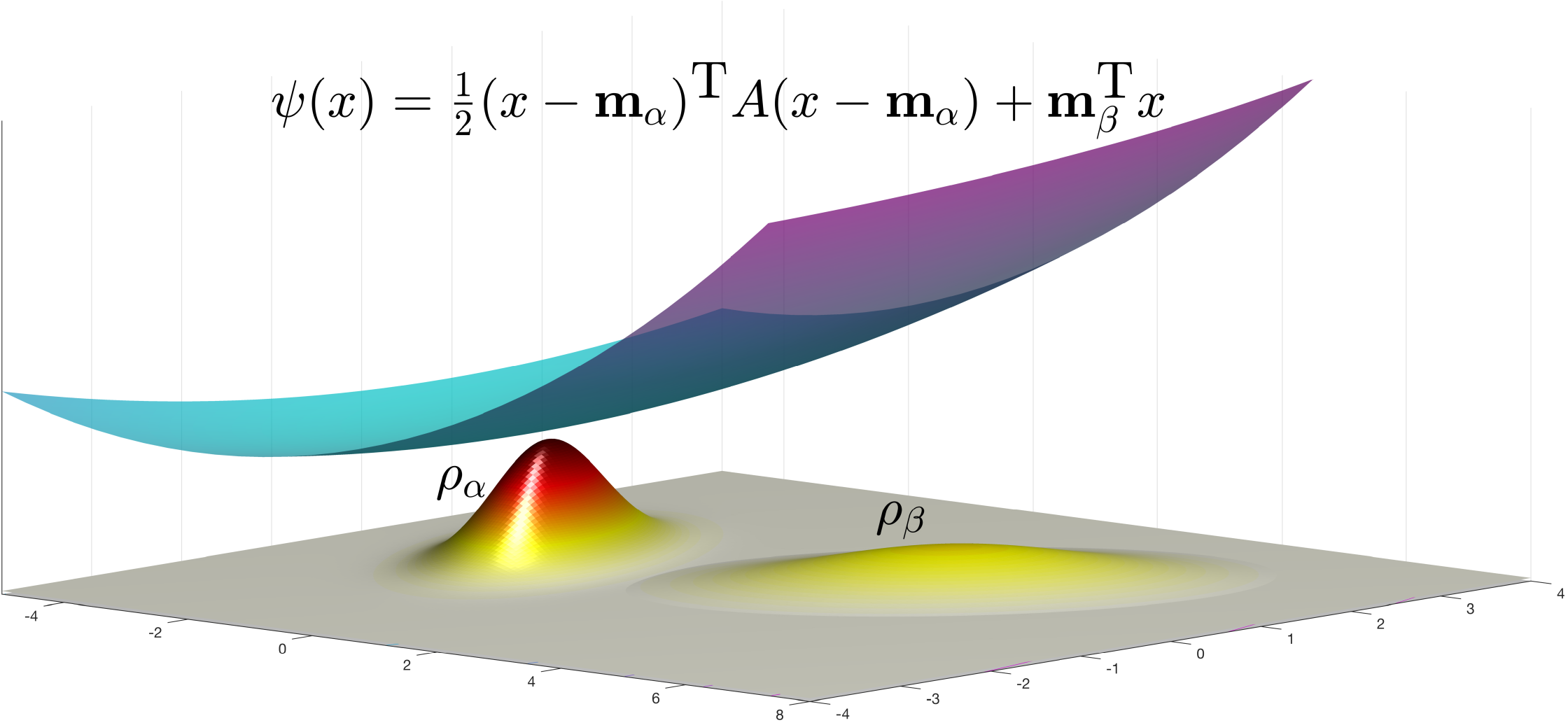}
\caption{\protect\label{fig-gaussians-2d-psi} Same Gaussians $\rho_\al$ and $\rho_\be$ as defined in Figure~\ref{fig-gaussians-2d-T}, represented this time as surfaces. The surface above is the Brenier potential $\psi$ defined up to an additive constant (here +50) such that $T=\nabla \psi$. For visual purposes, both Gaussian densities have been multiplied by a factor of 100.}
\end{figure}

\begin{figure}[h!]
\centering
\begin{tabular}{@{}c@{\hspace{1mm}}c@{}}
\includegraphics[width=.35\linewidth]{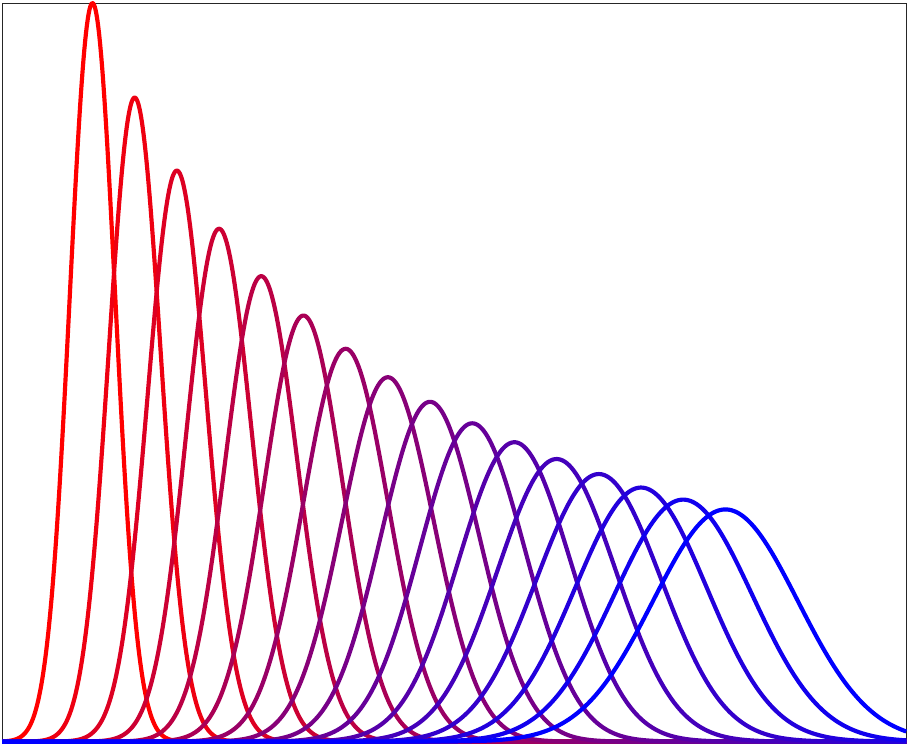}
\includegraphics[width=.25\linewidth]{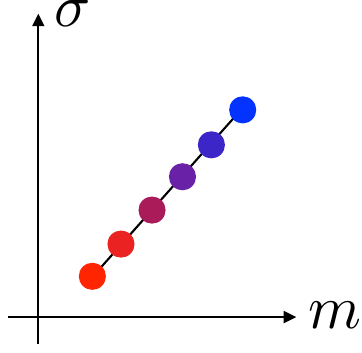}
\end{tabular}
\caption{\label{fig-1d-gaussian}
Computation of displacement interpolation between two 1-D Gaussians.
Denoting $\Gg_{m,\si}(x) \eqdef \frac{1}{\sqrt{2\pi}s}e^{-\frac{(x-m)^2}{2s^2}}$ the Gaussian density, it thus shows
the interpolation $\Gg_{(1-t)m_0+t m_1,(1-t)\si_0+t \si_1}$.
}
\end{figure}

\todoK{Maybe cover the case of translation / scaling of the same input. Explain that $\Wass$ is invariant to isometries. }


\chapter{Algorithmic Foundations}
\label{c-algo-basics} 

This chapter describes the most common algorithmic tools from combinatorial optimization and linear programming that can be used to solve the discrete formulation of optimal transport, as described in the primal problem~\eqref{eq-mk-discr} or alternatively its dual~\eqref{eq-dual}.

The origins of these algorithms can be traced back to World War II, either right before with \citeauthor{tolstoi1930methods}'s seminal work \citeyearpar{tolstoi1930methods} or during the war itself, when~\citet{Hitchcock41} and \citet{Kantorovich42} formalized the generic problem of dispatching available resources toward consumption sites in an optimal way. Both of these formulations, as well as the later contribution by~\citet{koopmans1949optimum}, fell short of providing a \emph{provably} correct algorithm to solve that problem (the cycle violation method was already proposed as a heuristic by~\citet{tolstoi1939metody}). One had to wait until the field of linear programming fully blossomed, with the proposal of the simplex method, to be at last able to solve rigorously these problems.

The goal of linear programming is to solve optimization problems whose objective function is linear and whose constraints are linear (in)equalities in the variables of interest. The optimal transport problem fits that description and is therefore a particular case of that wider class of problems. One can argue, however, that optimal transport is truly special among all linear program. First,~\citeauthor{dantzig1991}'s early motivation to solve linear programs was greatly related to that of solving transportation problems~\citep[p. 210]{dantzig49econometrica}. Second, despite being only a particular case, the optimal transport problem remained in the spotlight of optimization, because it was understood shortly after that optimal transport problems were related, and in fact equivalent, to an important class of linear programs known as minimum cost network flows~\citep[p. 213, Lem. 9.3]{korte2012combinatorial} thanks to a result by~\citet{ford1962flows}. As such, the OT problem has been the subject of particular attention, ever since the birth of mathematical programming~\citep{Dantzig51}, and is still widely used to introduce optimization to a new audience~\citep[\S1, p. 4]{nocedal}. 

\section{The Kantorovich Linear Programs}\label{s-kantorovich}
We have already introduced in Equation~\eqref{eq-mk-discr} the primal OT problem:
\eql{\label{eq-mk-discr-algo}
	\MKD_{\C}(\a,\b) = 
	\umin{\P \in \CouplingsD(\a,\b)}
		\sum_{i\in \range{n}, j \in \range{m}} \C_{i,j} \P_{i,j}. 
}
To make the link with the linear programming literature, one can cast the equation above as a linear program in \emph{standard} form, that is, a linear program with a linear objective; equality constraints defined with a matrix and a constant vector; and nonnegative constraints on variables. Let $\Identity_n$ stand for the identity matrix of size $n$ and let $\otimes$ be Kronecker's product. The ${(n+m) \times nm}$ matrix
$$
\mathbf{A}= \begin{bmatrix}
	\transp{\ones_{n}} \otimes \Identity_m \\
	\Identity_n \otimes \transp{\ones_{m}}
\end{bmatrix}\in\RR^{(n+m)\times nm}
$$
can be used to encode the row-sum and column-sum constraints that need to be satisfied for any $\P$ to be in $\CouplingsD(\a,\b)$. To do so, simply cast a matrix $\P\in\RR^{n\times m}$ as a vector $\p\in\RR^{nm}$ such that the $i+n(j-1)$'s element of $\p$ is equal to $\P_{ij}$ ($\P$ is enumerated columnwise) to obtain the following equivalence:
$$\P\in \RR^{n\times m}\in \CouplingsD(\a,\b) \Leftrightarrow \p\in\RR^{nm}_+, \mathbf{A}\p = \bigl[\begin{smallmatrix}\a\\ \b \end{smallmatrix} \bigr].$$
Therefore we can write the original optimal transport problem as
\eql{\label{eq-mk-discr-primal}
	\MKD_{\C}(\a,\b) =
	\umin{\substack{\p \in \RR^{nm}_+\\
		\mathbf{A}\p = \bigl[\begin{smallmatrix}\a\\ \b \end{smallmatrix} \bigr]}} \transp{\cc}\p,
}
where the $nm$-dimensional vector $\cc$ is equal to the stacked columns contained in the cost matrix $\C$. 

\begin{rem}\label{rem-transportation-polytope} Note that one of the $n+m$ constraints described above is redundant or that, in other words, the line vectors of matrix $A$ are not linearly independent. Indeed, summing all $n$ first lines and the subsequent $m$ lines results in the same vector (namely $A \bigl[\begin{smallmatrix} \ones_n\\\zeros_m \end{smallmatrix}\bigr] = A \bigl[\begin{smallmatrix} \zeros_n\\\ones_m\end{smallmatrix}\bigr]=\transp{\ones_{nm}}$). One can show that removing a line in $A$ and the corresponding entry in $\bigl[\begin{smallmatrix}\a\\ \b \end{smallmatrix}\bigr]$ yields a properly defined linear system. For simplicity, and to avoid treating asymmetrically $\a$ and $\b$, we retain in what follows a redundant formulation, keeping in mind that degeneracy will pop up in some of our computations.\end{rem}

The dual problem corresponding to Equation~\eqref{eq-mk-discr-primal} is, following duality in linear programming~\citep[p. 143]{bertsimas1997introduction} defined as
\eql{\label{eq-mk-discr-dual}
	\MKD_{\C}(\a,\b) = 
	\umax{\substack{\hD \in \RR^{n+m}\\
		\transp{A}\hD \leq \cc}}  \transp{\bigl[\begin{smallmatrix}\a\\ \b \end{smallmatrix} \bigr]}\hD.
}

Note that this program is exactly equivalent to that presented in Equation~\eqref{prop-duality-discr}.

\begin{rem}\label{rem-duality} We provide a simple derivation of the duality result above, which can be seen as a direct formulation of the arguments developed in Remark~\ref{rem-kantorovich-dual}. Strong duality, namely the fact that the optima of both primal~\eqref{eq-mk-discr-primal} and dual~\eqref{eq-mk-discr-dual} problems do indeed coincide, requires a longer proof~\citep[\S4.10]{bertsimas1997introduction}. To simplify notation, we write $\q=\bigl[\begin{smallmatrix}\a\\ \b \end{smallmatrix} \bigr]$. 
Consider now a relaxed primal problem of the optimal transport problem, where the constraint $\mathbf{A}\p=\q$ is no longer necessarily enforced but bears instead a cost $\transp{\hD}(\mathbf{A}\p-\q)$ parameterized by an arbitrary vector of costs $\hD\in\RR^{n+m}$. 
This relaxation, whose optimum depends directly on the cost vector $\hD$, can be written as
$$\LagrangeMKD(\hD)\eqdef \umin{\p \in \RR^{nm}_+} \transp{\cc}\p-\transp{\hD}(\mathbf{A}\p-\q).$$ 
Note first that this relaxed problem has no marginal constraints on~$\p$. Because that minimization allows for many more $\p$ solutions, we expect $\LagrangeMKD(\hD)$ to be smaller than $\bar{z}=\MKD_{\C}(\a,\b)$. Indeed, writing $\p^\star$ for any optimal solution of the primal problem \eqref{eq-mk-discr-algo}, we obtain
$$\umin{\p \in \RR^{nm}_+} \transp{\cc}\p-\transp{\hD}(\mathbf{A}\p-\q) \leq \transp{\cc}\p^\star-\transp{\hD}(\mathbf{A}\p^\star-\q)=\transp{\cc}\p^\star=\bar{z}.$$
The approach above defines therefore a problem which can be used to compute an optimal upper bound for the original problem~\eqref{eq-mk-discr-algo}, for any cost vector $\hD$; that function is called the Lagrange dual function of $\MKD$. 
The goal of duality theory is now to compute the best lower bound $\underline{z}$ by \emph{maximizing} $\LagrangeMKD$ over \emph{any} cost vector $\hD$, namely
$$
	\underline{z}=\umax{\hD} \pa{ \LagrangeMKD(\hD)=\umax{\hD} \transp{\hD}\q + \umin{\p \in \RR^{nm}_+} \transp{(\cc-\transp{A}\hD)}\p }.
$$
The second term involving a minimization on $\p$ can be easily shown to be $-\infty$ if any coordinate of $\transp{\cc}- \transp{A}\hD$ is negative. 
Indeed, if for instance for a given index $i\leq n+m$ we have $\cc_i-(\transp{A}\hD)_i<0$, then it suffices to take for $\p$ the canonical vector $\e_i$ multiplied by any arbitrary large positive value to obtain an unbounded value. 
When trying to maximize the lower bound $\LagrangeMKD(\hD)$ it therefore makes sense to restrict vectors $\hD$ to be such that $\transp{A}\hD\leq \cc$, in which case the best possible lower bound becomes $$\underline{z}=	\umax{\substack{\hD \in \RR^{n+m}\\\transp{A}\hD \leq \cc}}\transp{\hD}\q.$$ We have therefore proved a weak duality result, namely that $\underline{z}\leq \bar{z}$.
\end{rem}

\section{$\C$-Transforms}\label{sec-c-transforms} We present in this section an important property of the dual optimal transport problem~\eqref{eq-mk-discr-dual} which takes a more important meaning when used for the semidiscrete optimal transport problem in~\S\ref{s-c-transform}. This section builds upon the original formulation~\eqref{eq-dual} that splits dual variables according to row and column sum constraints:
\eql{\label{eq-dual-discrete-split}\MKD_\C(\a,\b) = \umax{(\fD,\gD) \in \PotentialsD(\C)} \dotp{\fD}{\a} + \dotp{\gD}{\b}. 
}
Consider any dual feasible pair $(\fD,\gD)$. If we ``freeze'' the value of $\fD$, we can notice that there is no better vector solution for $\gD$ than the $\C$-transform vector of $\fD$, denoted $\fD^{\,\C}\in\RR^m$ and defined as
\eq{(\fD^{\,\C})_j = \min_{i\in \range{n}} \C_{ij}-\fD_i,}
since it is indeed easy to prove that $(\fD,\fD^{\,\C})\in\PotentialsD(\C)$ and that $\fD^{\,\C}$ is the largest possible vector such that this constraint is satisfied. We therefore have that
\eq{\dotp{\fD}{\a} + \dotp{\gD}{\b} \leq \dotp{\fD}{\a} + \dotp{\fD^{\,\C}}{\b}.}
This result allows us first to reformulate the dual problem as a piecewise affine concave maximization problem expressed in a single variable $\fD$ as
\eql{\label{eq-semidual-discret}\MKD_\C(\a,\b) = \umax{\fD\in\RR^{n}} \dotp{\fD}{\a} + \dotp{\fD^{\,\C}}{\b}.}

Putting that result aside, the same reasoning applies of course if we now ``freeze'' the values of $\gD$ and consider instead the $\bar{\C}$-transform of $\gD$, namely vector $\gD^{\bar{\C}}\in\RR^n$ defined as
\eq{(\gD^{\bar{\C}})_i = \min_{j\in \range{m}} \C_{ij}-\gD_j,}
with a different increase in objective
\eq{\dotp{\fD}{\a} + \dotp{\gD}{\b} \leq \dotp{\gD^{\bar{\C}}}{\a} + \dotp{\gD}{\b}.}
Starting from a given $\fD$, it is therefore tempting to alternate $\C$ and $\bar{\C}$ transforms several times to improve $\fD$. Indeed, we have the sequence of inequalities
\eq{
	\dotp{\fD}{\a} + \dotp{\fD^{\,\C}}{\b} \leq \dotp{\fD^{\,\C\bar{\C}}}{\a} + \dotp{\fD^{\,\C}}{\b}\leq  \dotp{\fD^{\,\C\bar{\C}}}{\a} + \dotp{\fD^{\,\C\bar{\C}\C}}{\b}\leq  \dots
} 
One may hope for a strict increase in the objective at each of these iterations. However, this does not work because alternating $\C$ and $\bar{\C}$ transforms quickly hits a plateau.

\begin{prop}\label{prop-ccc-2}The following identities, in which the inequality sign between vectors should be understood elementwise, hold:
	\begin{enumerate}[label=(\roman*)]
	\item $\fD\leq \fD\,' \Rightarrow \fD^{\,\C}\geq \fD\,'^{\,\C}$, 
	\item $\fD^{\,\C\bar{\C}} \geq \fD$, $\gD^{\bar{\C}\C} \geq \gD$, 
	\item $\fD^{\,\C\bar{\C}\C}=\fD^{\,\C}.$
\end{enumerate}	
\end{prop}
\begin{proof} The first inequality follows from the definition of $\C$-transforms. Expanding the definition of $\fD^{\,\C\bar{\C}}$ we have
\eq{\left(\fD^{\,\C\bar{\C}}\right)_i= \min_{j\in \range{m}} \C_{ij}-\fD^{\,\C}_j = \min_{j\in \range{m}} \C_{ij}- \min_{i'\in \range{n}} \C_{i'j}-\fD_{i
'}.}
Now, since $-\min_{i'\in \range{n}} \C_{i'j}-\fD_{i'} \geq -(\C_{ij}-\fD_i)$, we recover 
\eq{\left(\fD^{\,\C\bar{\C}}\right)_i \geq \min_{j\in \range{m}} \C_{ij}- \C_{ij}+\fD_i = \fD_i.}
The relation $\gD^{\bar{\C}\C} \geq \gD$ is obtained in the same way. Now, set $\gD=\fD^{\,\C}$. Then, $\gD^{\bar{\C}}=\fD^{\,\C\bar{\C}}\geq \fD$. Therefore, using result (i) we have $\fD^{\,\C\bar{\C}\C}\leq \fD^{\,\C}$. Result (ii) yields $\fD^{\,\C\bar{\C}\C}\geq \fD^{\,\C}$, proving the equality.
\end{proof}

\section{Complementary Slackness}\label{s-complementary}

Primal~\eqref{eq-mk-discr-primal} and dual~\eqref{eq-mk-discr-dual},~\eqref{eq-dual} problems can be solved independently to obtain optimal primal $\P^{\star}$ and dual $(\fD^{\star},\gD^{\star})$ solutions. The following proposition characterizes their relationship. 

\begin{prop}\label{prop-primal-dual-optimal}Let $\P^\star$ and $\fD^\star,\gD^\star$ be optimal solutions for the primal~\eqref{eq-dual-generic} and dual~\eqref{eq-mk-discr} problems, respectively. Then, for any pair $(i,j)\in\range{n}\times\range{m}$, $\P^\star_{i,j}(\C_{i,j}-\fD^\star_i+\gD^\star_j)=0$ holds. In other words, if $\P^\star_{i,j}>0$, then necessarily $\fD^\star_i+\gD^\star_j=\C_{i,j}$; if $\fD^\star_i+\gD^\star_j<\C_{i,j}$ then necessarily $\P^\star_{i,j}=0$.
\end{prop}
\begin{proof} We have by strong duality that $\dotp{\P^\star}{\C} = \dotp{\fD^\star}{\a}+\dotp{\gD^\star}{\b}$. Recall that $\P^\star\ones_m = \a$ and $\transp{\P^\star}\ones_n = \b$; therefore 
	$$\begin{aligned}\dotp{\fD^\star}{\a}+\dotp{\gD^\star}{\b} &= \dotp{\fD^\star}{\P^\star\ones_m}+\dotp{\gD^\star}{\transp{\P^\star}\ones_n}\\&=  \dotp{\fD^\star\transp{\ones_m}}{\P^\star}+\dotp{\ones_n\transp{\gD^\star}}{\P^\star},\end{aligned}$$
	which results in $$\dotp{\P^\star}{\C- \fD^\star\oplus \gD^\star}=0.$$
Because $(\fD^\star,\gD^\star)$ belongs to the polyhedron of dual constraints~\eqref{eq-feasible-potential}, each entry of the matrix $\C- \fD^\star\oplus \gD^\star$ is necessarily nonnegative. Therefore, since all the entries of $\P$ are nonnegative, the constraint that the dot-product above is equal to $0$ enforces that, for any pair of indices $(i,j)$ such that $\P_{i,j}>0$, $\C_{i,j}-(\fD_i+\gD_j)$ must be zero, and for any pair of indices $(i,j)$ such that $\C_{i,j}>\fD_i+\gD_j$ that $\P_{i,j}=0$.
\end{proof}

The converse result is also true. We define first the idea that two variables for the primal and dual problems are complementary.

\begin{defn}\label{def-complementary}
A matrix $\P\in\RR^{n\times m}$ and a pair of vectors $(\fD,\gD)$ are complementary w.r.t. $\C$ if for all pairs of indices $(i,j)$ such that  $\P_{i,j}>0$ one also has $\C_{i,j}=\fD_i+\gD_j$.
\end{defn}

If a pair of feasible primal and dual variables is complementary, then we can conclude they are optimal.

\begin{prop}\label{prop-primal-dual-optimality} If $\P$ and $(\fD,\gD)$ are complementary and feasible solutions for the primal~\eqref{eq-dual-generic} and dual~\eqref{eq-mk-discr} problems, respectively, then $\P$ and $(\fD,\gD)$ are both primal and dual optimal.
\end{prop}
\begin{proof}
	By weak duality, we have that
$$\MKD_{\C}(\a,\b)\leq \dotp{\P}{\C}=\dotp{\P}{\fD\oplus\gD}=\dotp{\a}{\fD}+\dotp{\b}{\gD}\leq \MKD_{\C}(\a,\b)$$
and therefore $\P$ and $(\fD,\gD)$ are respectively primal and dual optimal.
\end{proof}

\todoK{
\begin{prop}[Primal-dual relationships]
	\todo{something on this?}
\end{prop}
}

\section{Vertices of the Transportation Polytope}\label{s-extremal}

Recall that a vertex or an extremal point of a convex set is formally a point $\mathbf{x}$ in that set such that, if there exiss $\mathbf{y}$ and $\mathbf{z}$ in that set with $\mathbf{x}=(\mathbf{y}+\mathbf{z})/2$, then necessarily $\mathbf{x}=\mathbf{y}=\mathbf{z}$. A linear program with a nonempty and bounded feasible set attains its minimum at a vertex (or extremal point) of the feasible set~\citep[p.~65, Theo.~2.7]{bertsimas1997introduction}. Since the feasible set $\CouplingsD(\a,\b)$ of the primal optimal transport problem~\eqref{eq-mk-discr-primal} is bounded, one can restrict the search for an optimal $\P$ to the set of extreme points of the polytope $\CouplingsD(\a,\b)$. Matrices $\P$ that are extremal in $\CouplingsD(\a,\b)$ have an interesting structure that has been the subject of extensive research~\citep[\S8]{brualdi2006combinatorial}. That structure requires describing the transport problem using the formalism of bipartite graphs.
 
\subsection{Tree Structure of the Support of All Vertices of $\CouplingsD(\a,\b)$} 
Let $V=(1,2,\dots,n)$ and $V'=(1',2',\dots,m')$ be two sets of nodes. Note that we add a prime to the labels of set $V'$ to disambiguate them from those of $V$. Consider their union $V\cup V'$, with $n+m$ nodes, and the set $\alledges$ of all $nm$ directed edges $\{ (i,j'), i \in \range{n}, j\in \range{m}\}$ between them (here we just add a prime to an integer $j\leq m$ to form $j'$ in $V'$). To each edge $(i,j')$ we associate the corresponding cost value $\C_{ij}$. The complete bipartite graph $\mathcal{G}$ between $V$ and $V'$ is $(V\cup V',E)$. A transport plan is a flow on that graph satisfying source ($\a_i$ flowing out of each node $i$) and sink ($\b_j$ flowing into each node $j'$) constraints, as described informally in Figure~\ref{fig-simplex}. An extremal point in $\CouplingsD(\a,\b)$ has the following property~\citep[p.~338,~Theo.~8.1.2]{brualdi2006combinatorial}.

\begin{figure}[h!]
\centering
\includegraphics[width=\linewidth]{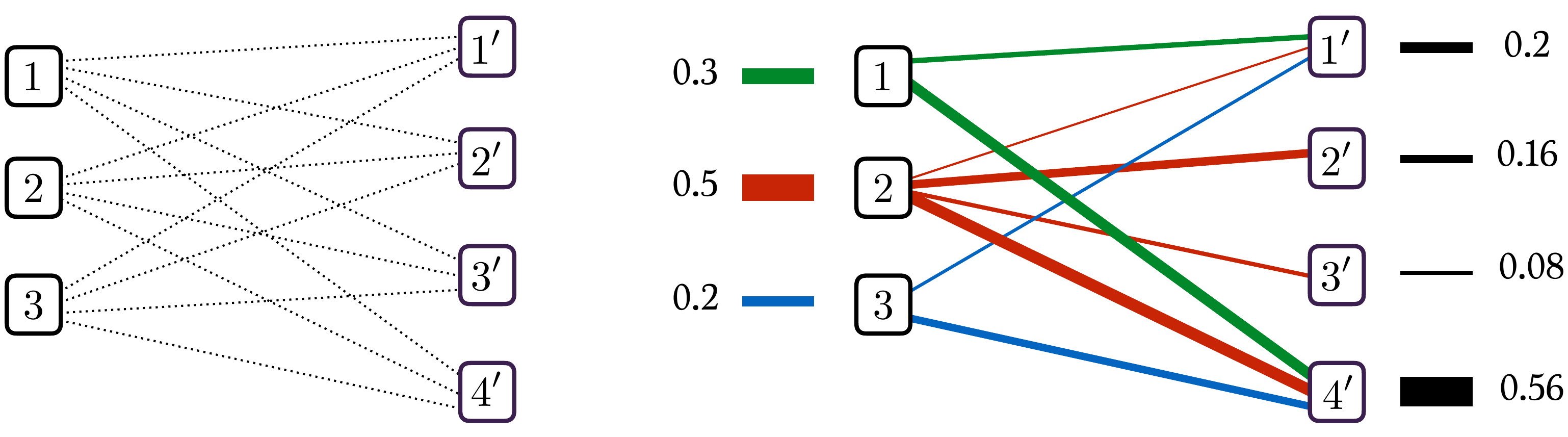}
\caption{\label{fig-simplex}
The optimal transport problem as a bipartite network flow problem. Here $n=3,m=4$. All coordinates of the source histogram, $\a$, are depicted as source nodes on the left labeled $1,2,3$, whereas all coordinates of the target histogram $\b$ are labeled as nodes $1',2',3',4'$. The graph is bipartite in the sense that all source nodes are connected to all target nodes, with no additional edges. To each edge $(i,j')$ is associated a cost $\C_{ij}$. A feasible flow is represented on the right. Proposition~\ref{prop-extremal} shows that this flow is not extremal since it has at least one cycle given by $((1,1'),(2,1'),(2,4'),(1,4'))$. 
}
\end{figure}

\begin{figure}[h!]
\centering
\includegraphics[width=\linewidth]{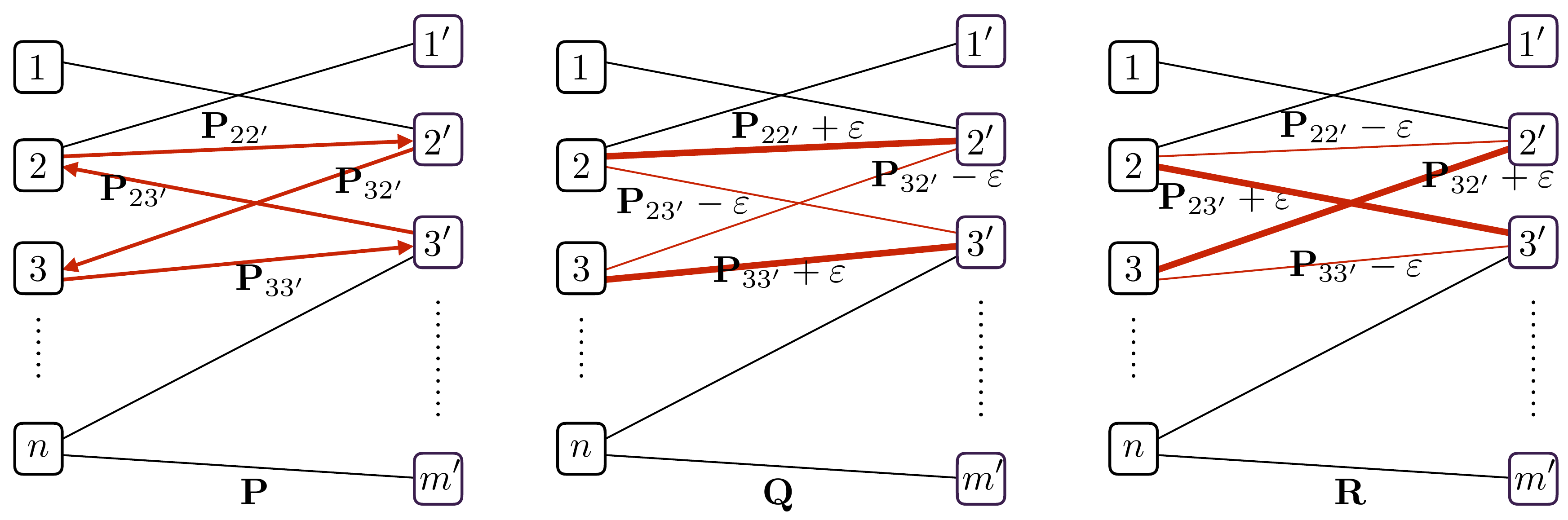}
\caption{\label{fig-perturb}
A solution $\P$ with a cycle in the graph of its support can be perturbed to obtain two feasible solutions $\Q$ and $\RRd$ such that $\P$ is their average, therefore disproving that $\P$ is extremal. 
}
\end{figure}

\begin{prop}[Extremal solutions]\label{prop-extremal}
Let $\P$ be an extremal point of the polytope $\CouplingsD(\a,\b)$. Let $S(\P)\subset \alledges$ be the subset of edges $\{(i,j'), i\in\range{n}, j\in\range{m} \text{ such that } \P_{ij}>0\}$. Then the graph $G(\P)\eqdef (V\cup V',S(\P))$ has no cycles. In particular, $\P$ cannot have more than $n+m-1$ nonzero entries.
 \end{prop}
\begin{proof}
We proceed by contradiction. Suppose that $\P$ is an extremal point of the polytope $\CouplingsD(\a,\b)$ and that its corresponding set $S(\P)$ of edges, denoted $F$ for short, is such that the graph $G=(V\cup V',F)$ contains a cycle, namely there exists $k>1$ and a sequence of distinct indices $i_1,\dots,i_{k-1}\in\range{n}$ and $j_1,\dots,j_{k-1}\in\range{m}$ such that the set of edges $H$ given below forms a subset of $F$. $$H=\left\{(i_1,j_1'), (i_2,j_1'), (i_2,j_2'),\dots,(i_k,j_k'),(i_1,j_k')\right\}.$$ 
	We now construct two feasible matrices $\Q$ and $\RRd$ such that $\P=(\Q+\RRd)/2$. To do so, consider a \emph{directed} cycle $\bar{H}$ corresponding to $H$, namely the sequence of pairs $i_1\rightarrow j_1', j_1' \rightarrow i_2, i_2 \rightarrow j_2',\dots, i_k \rightarrow j_k' , j_k' \rightarrow i_1$, as well as the elementary amount of flow $\varepsilon < \min_{(i,j')\in F}\P_{ij}$.
	Consider a perturbation matrix $\E$ whose $(i,j)$ entry is equal to $\varepsilon$ if $i\rightarrow j' \in \bar{H}$, $-\varepsilon$ if $ j\rightarrow i'\in \bar{H}$, and zero otherwise. Define matrices $\Q=\P+\E$ and $\RRd=\P-\E$ as illustrated in Figure~\ref{fig-perturb}. Because $\varepsilon$ is small enough, all elements in $\Q$ and $\RRd$ are nonnegative. By construction, $\E$ has either lines (resp., columns) with all entries equal to $0$ or exactly one entry equal to $\varepsilon$ and another equal to $-\varepsilon$ for those indexed by $i_1,\dots,i_k$ (resp., $j_1,\dots,j_k$). Therefore, $\E$ is such that $\E\ones_{m}=\zeros_{n}$ and $\transp{\E}\ones_{n}=\zeros_{m}$, and we have that $\Q$ and $\RRd$ have the same marginals as $\P$, and are therefore feasible. Finally $\P=(\Q+\RRd)/2$ which, since $\Q,\RRd\ne \P$, contradicts the fact that $\P$ is an extremal point. Since a graph with $k$ nodes and no cycles cannot have more than $k-1$ edges, we conclude that $S(\P)$ cannot have more than $n+m-1$ edges, and therefore $\P$ cannot have more than $n+m-1$ nonzero entries. 
\end{proof}

\subsection{The North-West Corner Rule}\label{subsec-northwest}
The north-west (NW) corner rule is a heuristic that produces a vertex of the polytope $\CouplingsD(\a,\b)$ in up to $n+m$ operations. This heuristic can play a role in initializing any algorithm working on the primal, such as the network simplex outlined in the next section. 

The rule starts by giving the highest possible value to $\P_{1,1}$ by setting it to $\min(\a_1,\b_1)$. At each step, the entry $\P_{i,j}$ is chosen to saturate either the row constraint at $i$, the column constraint at $j$, or both if possible. The indices $i,j$ are then updated as follows: $i$ is incremented in the first case, $j$ is in the second, and both $i$ and $j$ are in the third case. The rule proceeds until $\P_{n,m}$ has received a value. 

Formally, the algorithm works as follows: $i$ and $j$ are initialized to $1$, $r\leftarrow\a_1,c\leftarrow \b_1$. While $i\leq n$ and $j\leq m$, set $t\leftarrow \min(r,c)$, $\P_{i,j}\leftarrow t$, $r\leftarrow r-t$, $c\leftarrow s-t$; if $r=0$ then increment $i$, and update $r\leftarrow\a_i$ if $i\leq n$; if $c=0$ then increment $j$, and update $c\leftarrow \b_j$ if $j\leq n$; repeat. Here is an example of this sequence assuming $\a=[0.2,0.5,0.3]$  and $\b=[0.5,0.1,0.4]$:
$$\begin{aligned}\begin{bmatrix} \bullet & 0 & 0 \\ 0 & 0 & 0 \\ 0& 0 & 0\end{bmatrix} &\rightarrow \begin{bmatrix} 0.2 & 0 & 0 \\ \bullet & 0 & 0 \\ 0& 0 & 0\end{bmatrix} &\rightarrow \begin{bmatrix} 0.2 & 0 & 0 \\ 0.3 & \bullet & 0 \\ 0& 0 & 0\end{bmatrix}\\ &\rightarrow \begin{bmatrix} 0.2 & 0 & 0 \\ 0.3 &0.1 &\bullet \\ 0& 0 & 0\end{bmatrix} &\rightarrow \begin{bmatrix} 0.2 & 0 & 0 \\ 0.3 &0.1 &0.1 \\ 0& 0 & \bullet\end{bmatrix} & \rightarrow \begin{bmatrix} 0.2 & 0 & 0 \\ 0.3 &0.1 &0.1 \\ 0& 0 & 0.3\end{bmatrix}\end{aligned}$$
We write $\NW(\a,\b)$ for the unique plan that can be obtained through this heuristic. 

Note that there is, however, a much larger number of NW corner solutions that can be obtained by permuting arbitrarily the order of $\a$ and $\b$ first, computing the corresponding NW corner table, and recovering a table of $\CouplingsD(\a,\b)$ by inverting again the order of columns and rows: setting $\sigma=(3,1,2),\sigma'=(3,2,1)$ gives $\a_\sigma=[0.3,0.2,0.5], \b_{\sigma'}=[0.4,0.1,0.5]$, and $\sigma^{-1}=(2,3,1),\sigma'=(3,2,1)$. Observe that
\begin{gather*}
\NW(\a_\sigma,\b_{\sigma'}) = \begin{bmatrix} 0.3 & 0 & 0 \\ 0.1 & 0.1 & 0 \\ 0& 0 & 0.5\end{bmatrix} \in \CouplingsD(\a_\sigma,\b_{\sigma'}),\\
\NW_{\sigma^{-1}\sigma'^{-1}}(\a_\sigma,\b_{\sigma'})= \begin{bmatrix} 0 & 0.1 & 0.1 \\ 0.5 & 0 & 0 \\ 0& 0 & 0.3\end{bmatrix}\in \CouplingsD(\a,\b).
\end{gather*}

Let $\mathcal{N}(\a,\b)$ be the set of all NW corner solutions that can be produced this way:
$$\mathcal{N}(\a,\b)\eqdef\{ \NW_{\sigma^{-1}\sigma'^{-1}}(r_\sigma,c_{\sigma'}), \sigma,\sigma'\in S_d\}.$$
All NW corner solutions have by construction up to $n+m-1$ nonzero elements. The NW corner rule produces a table which is by construction unique for $\a_\sigma$ and $\b_\sigma'$, but there is an exponential number of pairs or row/column permutations $(\sigma,\sigma')$ that may yield the same table~\citep[p. 2]{stougie2002polynomial}. $\mathcal{N}(\a,\b)$ forms a subset of (usually strictly included in) the set of extreme points of $\CouplingsD(\a,\b)$~\citep[Cor. 8.1.4]{brualdi2006combinatorial}.

\section{A Heuristic Description of the Network Simplex}\label{s-networksimplex}

Consider a feasible matrix $\P$ whose graph $G(\P)=(V\cup V',S(\P))$ has no cycles. $\P$ has therefore no more than $n+m-1$ nonzero entries and is a vertex of $\CouplingsD(\a,\b)$ by Proposition~\ref{prop-extremal}. 
Following Proposition~\ref{prop-primal-dual-optimality}, it is therefore sufficient to obtain a dual solution $(\fD,\gD)$ which is feasible (\ie $\C-\fD\oplus\gD$ has nonnegative entries) and complementary to $\P$ (pairs of indices $(i,j')$ in $S(\P)$ are such that $\C_{i,j}=\fD_i+\gD_j$), to prove that $\P$ is optimal. 
The network simplex relies on two simple principles: to each feasible primal solution $\P$ one can associate a complementary pair $(\fD,\gD)$. 
If that pair is feasible, then we have reached optimality. If not, one can consider a modification of $\P$ that remains feasible and whose complementary pair $(\fD,\gD)$ is modified so that it becomes closer to feasibility.

\subsection{Obtaining a Dual Pair Complementary to $\P$}\label{subsec-obtaining}
The simplex proceeds by associating first to any extremal solution $\P$ a pair of $(\fD,\gD)$ complementary dual variables. This is simply carried out by finding two vectors $\fD$ and $\gD$ such that for any $(i,j')$ in $S(\P)$, $\fD_i+\gD_j$ is equal to $\C_{i,j}$. Note that this, in itself, does not guarantee that $(\fD,\gD)$ is feasible.

\begin{figure}[h!]
	\centering
	\includegraphics[width=\linewidth]{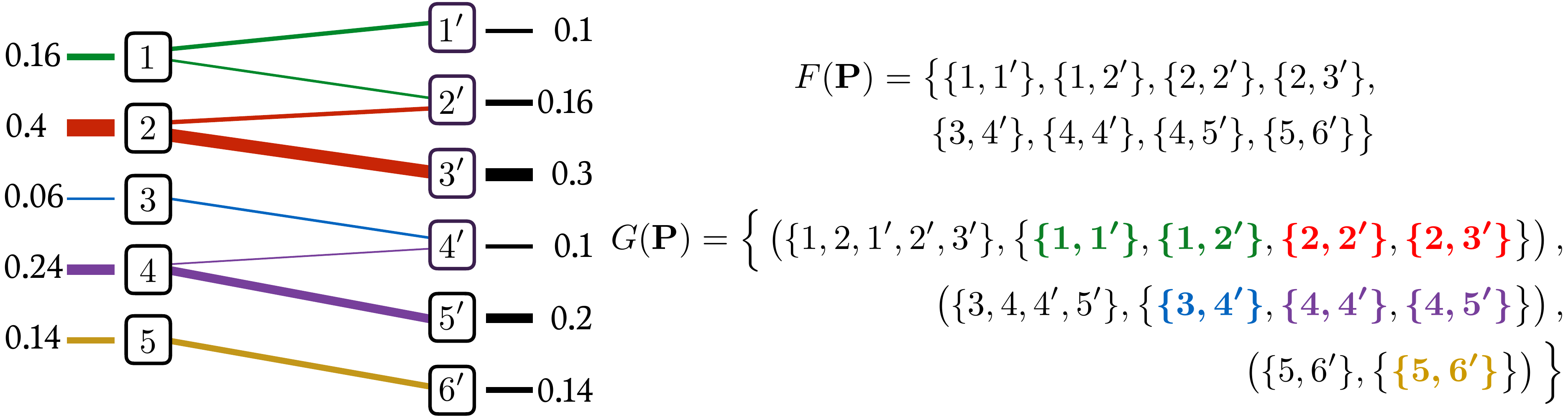}
	\caption{\label{fig-f_of_P}
	A feasible transport $\P$ and its corresponding set of edges $S(\P)$ and graph $G(\P)$. As can be seen, the graph $G(\P)=(\{1,\dots,5,1',\dots,6'\},S(\P))$ is a forest, meaning that it can be expressed as the union of tree graphs, three in this case.}
\end{figure}

Let $s$ be the cardinality of $S(\P)$. Because $\P$ is extremal, $s\leq n+m-1$. Because $G(\P)$ has no cycles, $G(\P)$ is either a tree or a forest (a union of trees), as illustrated in Figure~\ref{fig-f_of_P}. Aiming for a pair $(\fD,\gD)$ that is complementary to $\P$, we consider the following set of $s$ linear equality constraints on $n+m$ variables:
\begin{equation}\label{eq-dual-variables}
	\begin{array}{ccc}
\fD_{i_1}+\gD_{j_1}&=&\C_{i_1,j_1}\\
\fD_{i_2}+\gD_{j_1}&=&\C_{i_2,j_1}\\
\vdots & = & \vdots\\
\fD_{i_s}+\gD_{j_s}&=&\C_{i_s,j_s},\\
\end{array}
\end{equation}
where the elements of $S(\P)$ are enumerated as $(i_1,j_1'),\dots,(i_s,j_s')$.

Since $s\leq n+m-1 < n+m$, the linear system~\eqref{eq-dual-variables} above is always undetermined. This degeneracy can be interpreted in part because the parameterization of $\CouplingsD(\a,\b)$ with $n+m$ constraints results in $n+m$ dual variables. A more careful formulation, outlined in Remark~\ref{rem-transportation-polytope}, would have resulted in an equivalent formulation with only $n+m-1$ constraints and therefore $n+m-1$ dual variables. However, $s$ can also be strictly smaller than $n+m-1$: This happens when $G(\P)$ is the disjoint union of two or more trees. For instance, there are $5+6=11$ dual variables (one for each node) in Figure~\ref{fig-f_of_P}, but only $8$ edges among these $11$ nodes, namely $8$ linear equations to define $(\fD,\gD)$. Therefore, there will be as many undetermined dual variables under that setting as there will be connected components in $G(\P)$. 

Consider a tree among those listed in $G(\P)$. Suppose that tree has $k$ nodes $i_1,\dots,i_k$ among source nodes and $l$ nodes $j_1',\dots,j_{l}'$ among target nodes, resulting in $r\eqdef k+l$, and $r-1$ edges, corresponding to $k$ variables in $\fD$ and $l$ variables in $\gD$, linked with $r-1$ linear equations. To lift an indetermination, we can choose arbitrarily a root node in that tree and assign the value $0$ to its corresponding dual variable. From there, we can traverse the tree using a breadth-first or depth-first search to obtain a sequence of simple variable assignments that determines the values of all other dual variables in that tree, as illustrated in Figure~\ref{fig-tree-linear-solve}. That procedure can then be repeated for all trees in the graph of $\P$ to obtain a pair of dual variables $(\fD,\gD)$ that is complementary to $\P$.

\begin{figure}[h!]
	\centering
	\includegraphics[width=\linewidth]{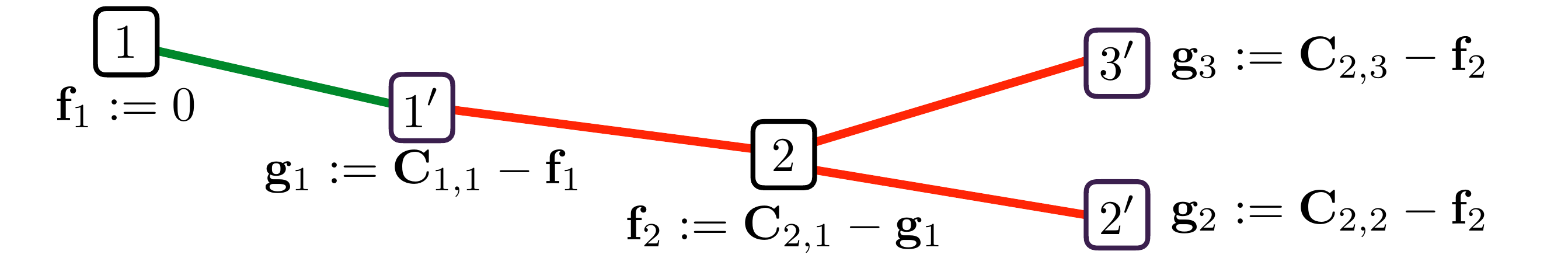}
	\caption{\label{fig-tree-linear-solve}
 The five dual variables $\fD_1,\fD_2,\gD_1,\gD_2,\gD_3$ corresponding to the five nodes appearing in the first tree of the graph $G(\P)$ illustrated in Figure \ref{fig-f_of_P} are linked through four linear equations that involve corresponding entries in the cost matrix $\C$. Because that system is degenerate, we choose a root in that tree (node 1 in this example) and set its corresponding variable to $0$ and proceed then by traversing the tree (either breadth-first or depth-first) from the root to obtain iteratively the values of the four remaining dual variables.}
\end{figure}

\subsection{Network Simplex Update}
The dual pair $(\fD,\gD)$ obtained previously might be feasible, in the sense that for all $i,j$ we have $\fD_i+\gD_j\leq \C_{i,j}$, in which case we have reached the optimum by Proposition~\ref{prop-primal-dual-optimality}. When that is not the case, namely when there exists $i,j$ such that $\fD_i+\gD_j> \C_{i,j}$, the network simplex algorithm kicks in. We first initialize a graph $G$ to be equal to the graph $G(\P)$ corresponding to the feasible solution $\P$ and add the violating edge $(i,j')$ to $G$. Two cases can then arise:
\begin{enumerate}[label={(\alph*)}]
	\item $G$ is (still) a forest, which can happen if $(i,j')$ links two existing subtrees. The approach outlined in~\S\ref{subsec-obtaining} can be used on graph $G$ to recover a new complementary dual vector $(\fD,\gD)$. Note that this addition simply removes an indetermination among the $n+m$ dual variables and does not result in any change in the primal variable $\P$. That update is usually called degenerate in the sense that $(i,j')$ has now entered graph $G$ although $\P_{i,j}$ remains $0$. $G(\P)$ is, however, contained in $G$.
	
	\item $G$ now has a cycle. In that case, we need to remove an edge in $G$ to ensure that $G$ is still a forest, yet also modify $\P$ so that $\P$ is feasible and $G(\P)$ remains included in $G$. These operations can all be carried out by increasing the value of $\P_{i,j}$ and modifying the other entries of $\P$ appearing in the detected cycle, in a manner very similar to the one we used to prove Proposition~\ref{prop-extremal}. To be more precise, let us write that cycle $(i_1,j_1'), (j_1',i_2), (i_2,j_2'),\dots, (i_l,j_l'),(j_l',i_{l+1})$ with the convention that $i_1=i_{l+1}=i$ to ensure that the path is a cycle that starts and ends at $i$, whereas $j_1=j$, to highlight the fact that the cycle starts with the added edge $\{i,j\}$, going in the right direction. Increase now the flow of all ``positive'' edges $(i_k,j_k')$ (for $k\leq l$), and decrease that of ``negative'' edges $(j_k',i_{k+1})$ (for $k\leq l$), to obtain an updated primal solution $\tilde{\P}$, equal to $\P$ for all but the following entries:
$$\forall k\leq l,\quad \tilde{\P}_{i_k,j_k}:= \P_{i_k,j_k}+\theta;\quad \tilde{\P}_{i_{k+1},j_k}:= \P_{i_{k+1},j_k}-\theta.$$
Here, $\theta$ is the largest possible increase at index $i,j$ using that cycle. The value of $\theta$ is controlled by the smallest flow negatively impacted by the cycle, namely $\min_{k} \P_{i_{k+1},j_k}$. That update is illustrated in Figure~\ref{fig-simplex-update}. Let $k^\star$ be an index that achieves that minimum. We then close the update by removing $(i_{k^\star+1},j_{k^\star})$ from $G$, to compute new dual variables $(\fD,\gD)$ using the approach outlined in~\S\ref{subsec-obtaining}.  
\end{enumerate}
	
\subsection{Improvement of the Primal Solution} Although this was not necessarily our initial motivation, one can show that the manipulation above can only improve the cost of $\P$. If the added edge has not created a cycle, case {(a)} above, the primal solution remains unchanged. When a cycle is created, case {(b)}, $\P$ is updated to $\tilde{\P}$, and the following equality holds:
\eq{
\dotp{\tilde{\P}}{\C}-\dotp{\P}{\C}= \theta \left(\;\sum_{k=1}^{l} \C_{i_k,j_k} - \sum_{k=1}^{l}\C_{i_{k+1},j_k}\right).
}
We now use the dual vectors $(\fD,\gD)$ computed at the end of the previous iteration. They are such that $\f_{i_k}+\g_{i_k}=\C_{i_k,j_k}$ and $\f_{i_{k+1}}+\g_{i_k}=\C_{i_{k+1},j_k}$ for \emph{all} edges initially in $G$, resulting in the identity
\eq{
\begin{aligned}\sum_{k=1}^{l} \C_{i_k,j_k} - \sum_{k=1}^{l}\C_{i_{k+1},j_k} &= \C_{i,j} + \sum_{k=2}^{l} \fD_{i_k}+\gD_{j_k} - \sum_{k=1}^{l} \fD_{i_{k+1}}+\gD_{j_k}\\
	&= \C_{i,j} - (\fD_{i} +\gD_{j}).
\end{aligned}
}
That term is, by definition, negative, since $i,j$ were chosen because $C_{i,j} < \fD_{i} -\gD_{j}$. Therefore, if $\theta>0$, we have that \eq{\dotp{\tilde{\P}}{\C}=\dotp{\P}{\C}+\theta\left(\C_{i,j}-(\fD_{i} -\fD_{g})\right)<\dotp{\P}{\C}.} 
If $\theta=0$, which can happen if $G$ and $G(\P)$ differ, the graph $G$ is simply changed, but $\P$ is not.

\begin{figure}[h!]
	\centering
	\includegraphics[width=\linewidth]{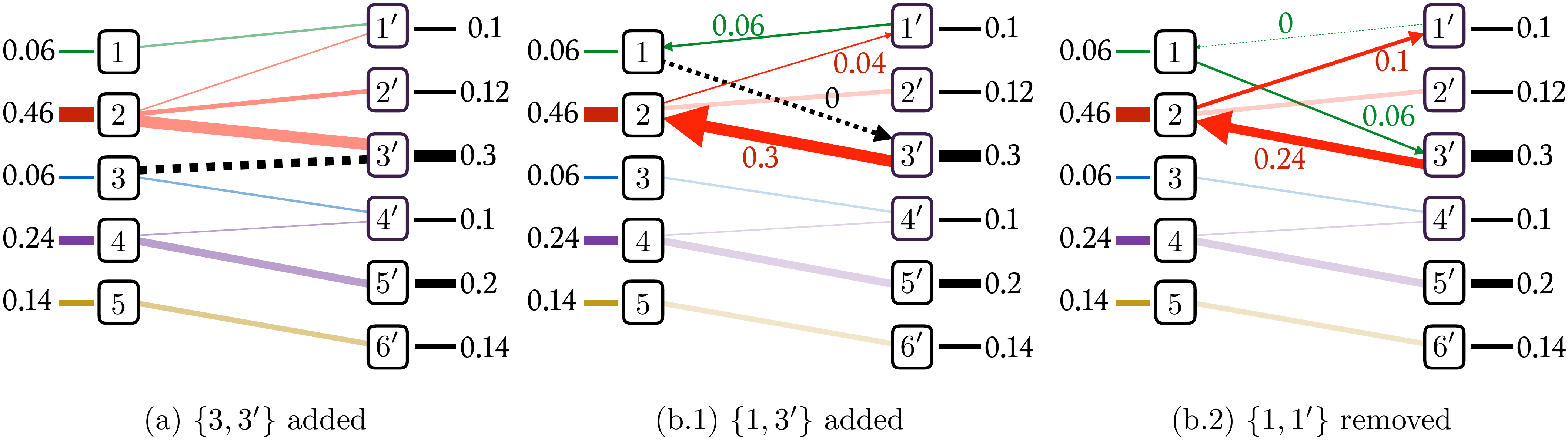}
	\caption{\label{fig-simplex-update}
Adding an edge $\{i,j\}$ to the graph $G(\P)$ can result in either (a) the graph remains a forest after this addition, in which case $\fD,\gD$ can be recomputed following the approach outlined in~\S\ref{subsec-obtaining}; (b.1) the addition of that edge creates a cycle, from which we can define a directed path; (b.2) the path can be used to increase the value of $\P_{i,j}$ and propagate that change along the cycle to maintain the flow feasibility constraints, until the flow of one of the edges that is negatively impacted by the cycle is decreased to $0$. This removes the cycle and updates $\P$.}
\end{figure}

The network simplex algorithm can therefore be summarized as follows: Initialize the algorithm with an extremal solution $\P$, given for instance by the NW corner rule as covered in~\S\ref{subsec-northwest}. Initialize the graph $G$ with $G(\P)$. Compute a pair of dual variables $(\fD,\gD)$ that are complementary to $\P$ using the linear system solve using the tree structure(s) in $G$ as described in~\S\ref{subsec-obtaining}. {(i)} Look for a violating pair of indices to the constraint $\C- \fD\oplus \gD\geq 0$; if none, $\P$ is optimal and stop. If there is a violating pair $(i,j')$, {(ii)} add the edge $(i,j')$ to $G$. If $G$ still has no cycles, update $(\fD,\gD)$ accordingly; if there is a cycle, direct it making sure $(i,j')$ is labeled as positive, and remove a negative edge in that cycle with the smallest flow value, updating $\P,G$ as illustrated in Figure~\ref{fig-simplex-update}, then build a complementary pair $\fD,\gD$ accordingly; return to {(i)}. Some of the operations above require graph operations (cycle detection, tree traversals) which can be implemented efficiently in this context, as described in (\cite[\S5]{bertsekas1998network}).

\citet{Orlin1997} was the first to prove the polynomial time complexity of the network simplex.~\citet{Tarjan1997} provided shortly after an improved bound in $O\left(\,(n+m)nm \log(n+m)\log\left((n+m)\|\C\|_\infty\right)\,\right)$ which relies on more efficient data structures to help select pivoting edges.


\section{Dual Ascent Methods}\label{s-dual-ascent}

Dual ascent methods precede the network simplex by a few decades, since they can be traced back to work by~\citet{borchardt1865investigando} and later K\"onig and Egerv\'ary, as recounted by~\citet{Kuhn1955}. The Hungarian algorithm is the best known algorithm in that family, and it can work only in the particular case when $\a$ and $\b$ are equal and are both uniform, namely $\a=\b=\ones_n/n$. We provide in what follows a concise description of the more general family of dual ascent methods. This requires the knowledge of the maximum flow problem (\cite[\S7.5]{bertsimas1997introduction}). By contrast to the network simplex, presented above in the primal, dual ascent methods maintain at each iteration dual feasible solutions whose objective is progressively improved by adding a sparse vector to $\fD$ and $\gD$. Our presentation is mostly derived from that of (\cite[\S7.7]{bertsimas1997introduction}) and starts with the following definition.

\begin{defn} For $S\subset\range{n},S'\subset\range{m}'\eqdef \{1',\dots,m'\}$ we write $\ones_S$ for the vector in $\RR^n$ of zeros except for ones at the indices enumerated in $S$, and likewise for the vector $\ones_{S'}$ in $\RR^m$ with indices in $S'$. 
\end{defn}	

In what follows, $(\fD,\gD)$ is a feasible dual pair in $\PotentialsD(\C)$. Recall that this simply means that for all pairs $(i,j')\in\range{n}\times\range{m}'$, $\fD_i+\gD_j\leq\C_{ij}$. We say that $(i,j')$ is a \emph{balanced} pair (or edge) if $\fD_i+\gD_j=\C_{ij}$ and \emph{inactive} otherwise, namely if $\fD_i+\gD_j<\C_{ij}$. With this convention, we start with a simple result describing how a feasible dual pair $(\fD,\gD)$ can be perturbed using sparse vectors indexed by sets $S$ and $S'$ and still remain feasible.
	
\begin{prop}\label{prop-feasibility-dual-pair}
$(\tilde{\fD},\tilde{\gD})\eqdef(\fD,\gD)+\varepsilon(\ones_S,-\ones_{S'})$ is dual feasible for a small enough $\varepsilon>0$ if for all $i\in S$, the fact that $(i,j')$ is balanced implies that $j'\in S'$.
\end{prop}
\begin{proof}
For any $i\in S$, consider the set $\mathcal{I}_{i}$ of all $j'\in\range{m}'$ such that $(i,j')$ is inactive, namely such that $\fD_i+\gD_j<\C_{ij}$. Define $\varepsilon_i \eqdef \min_{j\in I_{i}} \C_{i,j}-\fD_i-\gD_j$, the smallest margin by which $\fD_i$ can be increased without violating the constraints corresponding to $j'\in \mathcal{I}_i$. Indeed, one has that if $\varepsilon\leq \varepsilon_i$ then $\tilde{\fD}_i+\tilde{\gD}_j<\C_{i,j}$ for any $j'\in \mathcal{I}_i$. Consider now the set $\mathcal{B}_i$ of balanced edges associated with $i$. Note that $\mathcal{B}_i=\range{m}'\setminus \mathcal{I}_i$. The assumption above is that $j'\in\mathcal{B}_i\Rightarrow j'\in S'$. Therefore, one has that for $j'\in\mathcal{B}_i$, $\tilde{\fD}_i+\tilde{\gD}_j=\fD_i+\gD_j=\C_{i,j}$. As a consequence, the inequality $\tilde{\fD}_i+\tilde{\gD}_j\leq \C_{i,j}$ is ensured for any $j\in \range{m}'$. Choosing now an increase $\varepsilon$ smaller than the smallest possible allowed, namely $\min_{i\in S} \varepsilon_i$, we recover that $(\tilde{\fD},\tilde{\gD})$ is dual feasible.
\end{proof}

The main motivation behind the iteration of the network simplex presented in \S\ref{subsec-obtaining} is to obtain, starting from a feasible primal solution $\P$, a complementary feasible dual pair $(\fD,\gD)$. To reach that goal, $\P$ is progressively modified such that its complementary dual pair reaches dual feasibility. A symmetric approach, starting from a feasible dual variable to obtain a feasible primal $\P$, motivates dual ascent methods. The proposition below is the main engine of dual ascent methods in the sense that it guarantees (constructively) the existence of an ascent direction for $(\fD,\gD)$ that maintains feasibility. That direction is built, similarly to the network simplex, by designing a candidate primal solution $\P$ whose infeasibility guides an update for $(\fD,\gD)$.

\begin{prop}\label{prop-dual-ascent-dir} Either $(\fD,\gD)$ is optimal for Problem~\eqref{eq-dual-discrete-split} or there exists $S\subset\range{n},S'\subset\range{m}'$ such that $(\tilde{\fD},\tilde{\gD})\eqdef(\fD,\gD)+\varepsilon(\ones_S,-\ones_{S'})$ is feasible for a small enough $\varepsilon>0$ and has a strictly better objective.
\end{prop}
\begin{proof}
We consider first a complementary primal variable $\P$ to $(\fD,\gD)$. To that effect, let $\mathcal{B}$ be the set of balanced edges, namely all pairs $(i,j')\in\range{n}\times\range{m}'$ such that $\fD_i+\gD_j=\C_{i,j}$, and form the bipartite graph whose vertices $\{1,\dots,n,1',\dots,m'\}$ are linked with edges in $\mathcal{B}$ only, complemented by a source node $s$ connected with \emph{capacitated} edges to all nodes $i\in\range{n}$ with respective capacities $\a_i$, and a terminal node $t$ also connected to all nodes $j'\in\range{m}'$ with edges of respective capacities $\b_j$, as seen in Figure~\ref{fig-dualascent}. The Ford--Fulkerson algorithm (\cite[p. 305]{bertsimas1997introduction}) can be used to compute a maximal flow $\mathbf{F}$ on that network, namely a family of $n+m+|\mathcal{B}|$ nonnegative values indexed by $(i,j')\in \mathcal{B}$ as $f_{si}\leq \a_i, f_{ij'}, f_{j't}\leq\b_j$ that obey flow constraints and such that $\sum_i f_{si}$ is maximal. If the throughput of that flow $\mathbf{F}$ is equal to $1$, then a feasible primal solution $\P$, complementary to $\fD,\gD$ by construction, can be extracted from $\mathbf{F}$ by defining $\P_{i,j}=f_{ij'}$ for $(i,j')\in \mathcal{B}$ and zero elsewhere, resulting in the optimality of $(\fD,\gD)$ and $\P$ by Proposition~\ref{prop-primal-dual-optimality}. If the throughput of $\mathbf{F}$ is strictly smaller than $1$, the labeling algorithm proceeds by labeling (identifying) those nodes reached iteratively from $s$ for which $\mathbf{F}$ does not saturate capacity constraints, as well as those nodes that contribute flow to any of the labeled nodes. Labeled nodes are stored in a nonempty set $Q$, which does not contain the terminal node $t$ per optimality of $\mathbf{F}$ (see \citealt[p. 308]{bertsimas1997introduction}, for a rigorous presentation of the algorithm). $Q$ can be split into two sets $S=Q\cap\range{n}$ and $S'=Q\cap\range{m}'$. Because we have assumed that the total throughput is strictly smaller than $1$, $S\ne\emptyset$. Note first that if $i\in S$ and $(i,j)$ is balanced, then $j'$ is necessarily in $S'$. Indeed, since all edges $(i,j')$ have infinite capacity by construction, the labeling algorithm will necessarily reach $j'$ if it includes $i$ in $S$. By Proposition~\ref{prop-feasibility-dual-pair}, there exists thus a small enough $\varepsilon$ to ensure the feasibility of $\tilde{\fD},\tilde{\gD}$. One still needs to prove that $\transp{\ones_S}{\a}-\transp{\ones_{S'}}{\b}>0$ to ensure that $(\tilde{\fD},\tilde{\gD})$ has a better objective than $(\fD,\gD)$. Let $\bar{S}=\range{n}\setminus S$ and $\bar{S'}=\range{m}'\setminus S'$ and define
$$A = \sum_{i\in S} f_{si}, \quad B = \sum_{i\in \bar{S}} f_{si}, \quad C = \sum_{j'\in S'} f_{j't}, \quad D = \sum_{j'\in \bar{S'}} f_{j't}.$$
The total maximal flow starts from $s$ and is therefore equal to $A+B$, but also arrives at $t$ and is therefore equal to $C+D$. Flow conservation constraints also impose that the very same flow is equal to $B+C$, therefore $A=C$. On the other hand, by definition of the labeling algorithm, we have for all $i$ in $S$ that $f_{si}<\a_i$, whereas $f_{j't}=\b_j$ for $j'\in \bar{S}'$ because $t$ cannot be in $S'$ by optimality of the considered flow. We therefore have $A<\transp{\ones_S}\a$ and $C=\transp{\ones_S'}\b$. Therefore $\transp{\ones_S}\a-\transp{\ones_S'}\b>A-C=0$.
\end{proof}

\begin{figure}[h!]
\centering
\includegraphics[width=\linewidth]{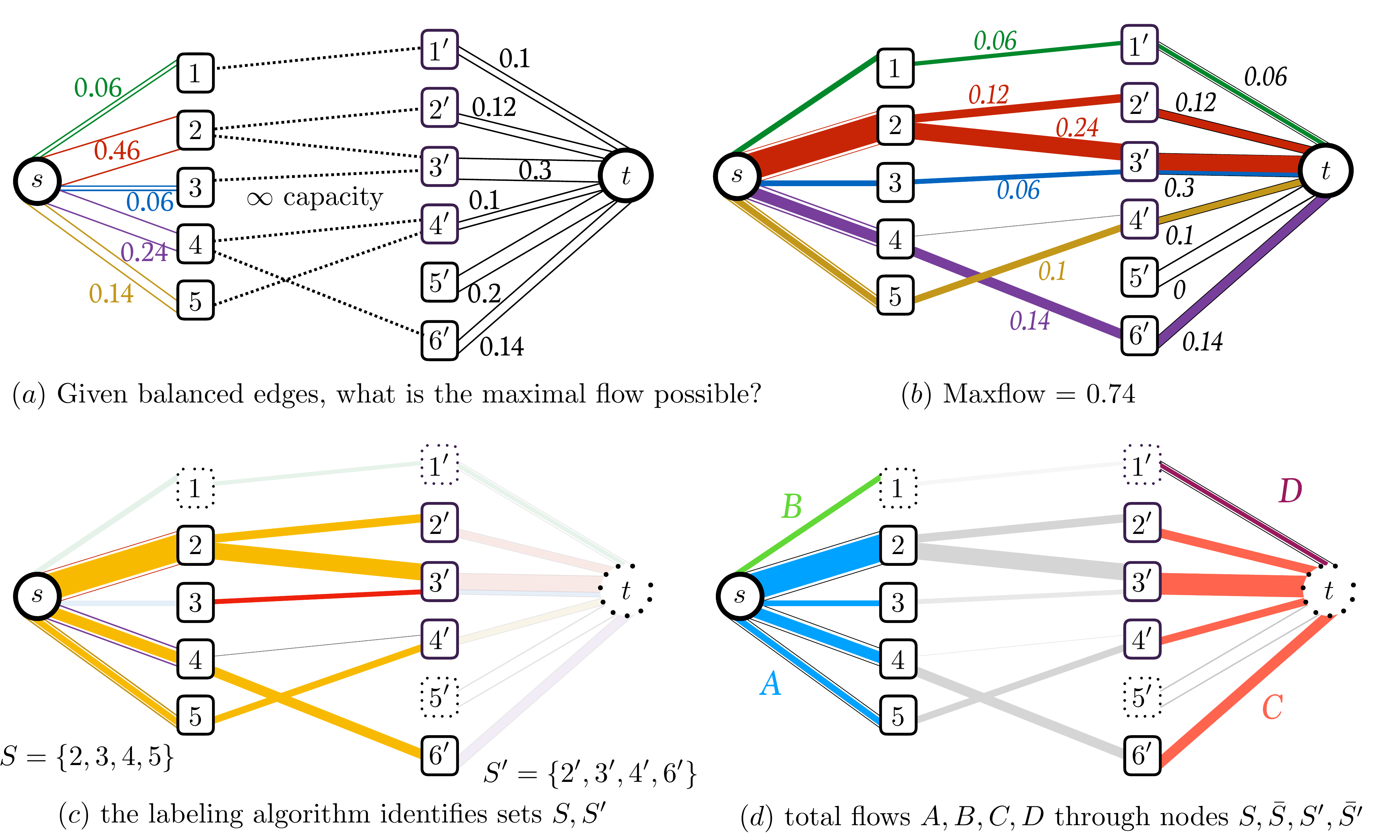}
\caption{\label{fig-dualascent}
Consider a transportation problem involving the marginals introduced first in Figure~\ref{fig-f_of_P}, with $n=5, m=6$. Given two feasible dual vectors $\fD,\gD$, we try to obtain the ``best'' flow matrix $P$ that is complementary to $(\fD,\gD)$. Recall that this means that $\P$ can only take positive values on those edges $(i,j')$ corresponding to indices for which $\fD_i+\gD_j=\C_{i,j}$, here represented with dotted lines in plot (a). The best flow that can be achieved with that graph structure can be formulated as a max-flow problem in a capacitated network, starting from an abstract source node $s$ connected to all nodes labeled $i\in\range{n}$, terminating at an abstract terminal node $t$ connected to all nodes labeled $j'$, where $j\in\range{m'}$, and such that the capacities of edge $(s,i),(j',t), i\in \range{n},j\in\range{m}$ are respectively $\a_i,\b_j$ and all others infinite. The Ford--Fulkerson algorithm (\cite[p. 305]{bertsimas1997introduction}) can be applied to compute such a max-flow, which, as represented in plot (b), only achieves $0.74$ units of mass out of $1$ needed to solve the problem. One of the subroutines used by max-flow algorithms, the labeling algorithm (\cite[p. 308]{bertsimas1997introduction}), can be used to identify nodes that receive an unsaturated flow from $s$ (and recursively, all of its successors), denoted by orange lines in plot (c). The labeling algorithm also adds by default nodes that send a positive flow to any labeled node, which is the criterion used to select node $3$, which contributes with a red line to $3'$. Labeled nodes can be grouped in sets $S,S'$ to identify nodes which can be better exploited to obtain a higher flow, by modifying $\fD,\gD$ to obtain a different graph. The proof involves partial sums of flows described in plot (d)}
\end{figure}

The dual ascent method proceeds by modifying any feasible solution $(\fD,\gD)$ by any vector generated by sets $S,S'$ that ensure feasibility and improve the objective. When the sets $S,S'$ are those given by construction in the proof of Proposition~\ref{prop-dual-ascent-dir}, and the steplength $\varepsilon$ is defined as in the proof of Proposition~\ref{prop-feasibility-dual-pair}, we recover a method known as the \emph{primal-dual} method. That method reduces to the Hungarian algorithm for matching problems. Dual ascent methods share similarities with the dual variant of the network simplex, yet they differ in at least two important aspects. Simplex-type methods always ensure that the current solution is an \emph{extreme point} of the feasible set, $\PotentialsD(\C)$ for the dual, whereas dual ascent as presented here does not make such an assumption, and can freely produce iterates that lie in the interior of the feasible set. Additionally, whereas the dual network simplex would proceed by modifying $(\fD,\gD)$ to produce a primal solution $\P$ that satisfies linear (marginal constraints) but only nonnegativity upon convergence, dual ascent builds instead a primal solution $\P$ that is always nonnegative but which does not necessarily satisfy marginal constraints.

\section{Auction Algorithm}\label{s-auction}
The auction algorithm was originally proposed by~\citet{bertsekas1981new} and later refined in~\citep{bertsekas1988dual}. Several economic interpretations of this algorithm have been proposed (see \emph{e.g.} \citet{bertsekas1992auction}). The algorithm can be adapted for arbitrary marginals, but we present it here in its formulation to solve optimal assignment problems.

\paragraph{Complementary slackness.} Notice that in the optimal assignment problem, the primal-dual conditions presented for the optimal transport problem become easier to formulate, because any extremal solution $\P$ is \todoK{there always exist a permutation solution, but the converse it not true} necessarily a permutation matrix $\P_\si$ for a given $\si$ (see Equation~\eqref{prop-primal-dual-optimality}). Given primal $\P_{\si^\star}$ and dual $\fD^\star,\gD^\star$ optimal solutions we necessarily have that
$$
\fD^\star_i + \gD^\star_{\si^\star_i} = \C_{i,{\si_i^\star}}.
$$
Recall also that, because of the principle of $\C$-transforms enunciated in~\S\ref{sec-c-transforms}, that one can choose $\fD^\star$ to be equal to $\gD^{\bar\C}$. We therefore have that
\eql{\label{eq-dual-conditions-assign-1}
\C_{i,\si^\star_i} - \gD^\star_{\si_i}  = \min_{j} \C_{i,j} - \gD^\star_{j}. 
}
On the contrary, it is easy to show that if there exists a vector $\gD$ and a permutation $\sigma$ such that 
\eql{\label{eq-dual-conditions-assign}
\C_{i,\si_i} - \gD_{\si_i}  = \min_{j} \C_{i,j} - \gD_{j} 
}
holds, then they are both optimal, in the sense that $\si$ is an optimal assignment and $\gD^{\bar\C},\gD$ is an optimal dual pair.

\paragraph{Partial assignments and $\varepsilon$-complementary slackness.} The goal of the auction algorithm is to modify iteratively a triplet $S,\xi,\gD$, where $S$ is a subset of $\range{n}$, $\xi$ a partial assignment vector, namely an injective map from $S$ to $\range{n}$, and $\gD$ a dual vector. The dual vector is meant to converge toward a solution satisfying an \emph{approximate} complementary slackness property~\eqref{eq-dual-conditions-assign}, whereas $S$ grows to cover $\range{n}$ as $\xi$ describes a permutation. The algorithm works by maintaining the three following properties after each iteration:
\begin{enumerate}[label={(\alph*)}]
	\item $\forall i\in S,\quad \C_{i,\xi_i} - \gD_{\xi_i}  \leq \varepsilon + \min_{j} \C_{i,j} - \gD_{j}$ ($\varepsilon$-CS).
	\item The size of $S$ can only increase at each iteration.
	\item There exists an index $i$ such that $\gD_i$ decreases by at least $\varepsilon$.
\end{enumerate}

\paragraph{Auction algorithm updates.} Given a point $j$ the auction algorithm uses not only the optimum appearing in the usual $\C$-transform but also a second best,
$$j^1_i \in \argmin_{j} \C_{i,j} - \gD_{j}, \quad j^2_i \in \argmin_{j\ne j^1_i} \C_{i,j} - \gD_{j}, $$
to define the following updates on $\gD$ for an index $i\notin S$, as well as on $S$ and $\xi$:
\begin{enumerate}\item \textbf{update $\gD$}: Remove to the $j^1_i$th entry of $\gD$ the sum of $\varepsilon$ and the difference between the second lowest and lowest adjusted cost $\{\C_{i,j}-\gD_j\}_j$,
\eql{
\begin{aligned}
	\gD_{j^1_{i}} &\leftarrow \gD_{j^1_{i}} - \underbrace{\left((\C_{i,j^2_{i}}-\gD_{j^2_{i}})-(\C_{i,j^1_{i}}-\gD_{j^1_{i}})+\varepsilon\right)}_{\geq \varepsilon >0}\label{eq-auction-update}\\
	&= \C_{i,j^1_{i}} - (\C_{i,j^2_{i}}-\gD_{j^2_{i}}) -\varepsilon.\\
\end{aligned}
}
\item \textbf{update $S$ and $\xi$}: If there exists an index $i'\in S$ such that $\xi_{i'}=j^1_{i}$, remove it by updating $S\leftarrow S\setminus \{i'\}$. Set $\xi_{i}=j^1_{i}$ and add $i$ to $S$, $S\leftarrow S \cup \{i\}.$
\end{enumerate}

\paragraph{Algorithmic properties.} The algorithm proceeds by starting from an empty set of assigned points $S=\emptyset$ with no assignment and empty partial assignment vector $\xi$, and $\gD=\zeros_n$, terminates when $S=\range{n}$, and loops through both steps above until it terminates. The fact that properties {(b)} and {(c)} are valid after each iteration is made obvious by the nature of the updates (it suffices to look at Equation~\eqref{eq-auction-update}). $\varepsilon$-complementary slackness is easy to satisfy at the first iteration since in that case $S=\emptyset$. The fact that iterations preserve that property is shown by the following proposition.
\begin{prop} The auction algorithm maintains $\varepsilon$-complementary slackness at each iteration. 
\end{prop}
\begin{proof} 
	Let $\gD,\xi,S$ be the three variables at the beginning of a given iteration. We therefore assume that for any $i'\in S$ the relationship 
	$$\C_{i,\xi_{i'}} - \gD_{\xi_{i'}}  \leq \varepsilon + \min_{j} \C_{{i'},j} - \gD_{j} $$
holds. Consider now the particular $i\notin S$ considered in an iteration. Three updates happen: $\gD,\xi,S$ are updated to $\gD^{\text{n}},\xi^{\text{n}},S^{\text{n}}$ using indices $j^1_i$ and $j^2_i$. More precisely, $\gD^{\text{n}}$ is equal to $\gD$ except for element $j^1_i$, whose value is equal to
$$\gD_{j^1_{i}}^\text{n} = \gD_{j^1_{i}} - \left((\C_{i,j^2_{i}}-\gD_{j^2_{i}})-(\C_{i,j^1_{i}}-\gD_{j^1_{i}})\right)-\varepsilon \leq \gD_{j^1_{i}}-\varepsilon$$,
$\xi^{\text{n}}$ is equal to $\xi$ except for its $i$th element equal to $j^1_i$, and $S^{\text{n}}$ is equal to the union of $\{i\}$ with $S$ (with possibly one element removed). The update of $\gD^{\text{n}}$ can be rewritten
$$\gD^{\text{n}}_{j^1_i}= \C_{i,j^1_{i}} - (\C_{i,j^2_{i}}-\gD_{j^2_{i}}) -\varepsilon;$$
therefore we have
$$\C_{i,j^1_{i}} - \gD^{\text{n}}_{j^1_i}=  \varepsilon + (\C_{i,j^2_{i}}-\gD_{j^2_{i}}) = \varepsilon+\min_{j\ne j^1_i} (\C_{i,j}-\gD_{j}).$$
Since $-\gD\leq -\gD^{\text{n}}$ this implies that 
$$\C_{i,j^1_{i}} - \gD^{\text{n}}_{j^1_i} = \varepsilon + \min_{j\ne j^1_i} (\C_{i,j}-\gD_{j})\leq \varepsilon+\min_{j\ne j^1_i} (\C_{i,j}-\gD^{\text{n}}_{j}),$$
and since the inequality is also obviously true for $j=j^1_i$ we therefore obtain the $\varepsilon$-complementary slackness property for index $i$. For other indices $i'\ne i$, we have again that since $\gD^{\text{n}}\leq\gD$ the sequence of inequalities holds,
$$\C_{i,\xi^{\text{n}}_{i'}} - \gD^{\text{n}}_{\xi^{\text{n}}_{i'}} = \C_{i,\xi_{i'}} - \gD_{\xi_{i'}} \leq \varepsilon + \min_{j} \C_{{i'},j} - \gD_{j} \leq \varepsilon + \min_{j} \C_{{i'},j} - \gD^{n}_{j}. $$
\end{proof}

\begin{prop} The number of steps of the auction algorithm is at most $N=n\|\C\|_\infty/\varepsilon$.
\end{prop}
\begin{proof}Suppose that the algorithm has not stopped after $T>N$ steps. Then there exists an index $j$ which is not in the image of $\xi$, namely whose price coordinate $\gD_j$ has never been updated and is still $\gD_j=0$. In that case, there cannot exist an index $j'$ such that $\gD_{j'}$ was updated $n$ times with $n>\|\C\|_\infty/\varepsilon$. Indeed, if that were the case then for any index $i$
$$\gD_{j'}\leq-n\varepsilon<-\|\C\|_\infty\leq -\C_{i,j} = \gD_j-\C_{i,j},$$
which would result in, for all $i$,
$$\C_{i,j'}-\gD_{j'}> \C_{i,j}+ (\C_{i,j}-\gD_j),$$
which contradicts $\varepsilon$-CS. Therefore, since there cannot be more than $\|C\|_\infty/\varepsilon$ updates for each variable, the total number of iterations $T$ cannot be larger than $n\|\C\|_\infty/\varepsilon=N$.
\end{proof}

\begin{rem} Note that this result yields a naive number of operations of $N^3\|\C\|_\infty/\varepsilon$ for the algorithm to terminate. That complexity can be reduced to $N^3 \log\|\C\|_\infty$ when using a clever method known as $\varepsilon$-scaling, designed to decrease the value of $\varepsilon$ with each iteration (\cite[p. 264]{bertsekas1998network}).
\end{rem}

\begin{prop} The auction algorithm finds an assignment whose cost is $n\varepsilon$ suboptimal.
\end{prop}
\begin{proof}
	Let $\sigma,\gD^\star$ be the primal and dual optimal solutions of the assignment problem of matrix $\C$, with optimum $$t^\star=\sum \C_{i,\sigma_i} = \sum_i \min_{j}\C_{i,j}-\gD^\star_j+\sum_j\gD^\star_j.$$ 
Let $\xi,\gD$ be the solutions output by the auction algorithm upon termination. The $\varepsilon$-CS conditions yield that for any $i\in S$, 
$$\min_{j} \C_{i,j} - \gD_{j}\geq \C_{i,\xi_i} - \gD_{\xi_i} -\varepsilon.$$
Therefore by simple suboptimality of $\gD$ we first have
	$$\begin{aligned}t^\star&\geq \sum_i \left(\min_{j} \C_{i,j}-\gD_{j}\right) + \sum_j \gD_j\\
&\geq 
	\sum_i -\varepsilon+\left(\C_{i,\xi_i}-\gD_{\xi_i}\right) + \sum_j \gD_j = -n\varepsilon + \sum_i \C_{i,\xi_j} \geq  - n\varepsilon + t^\star.\end{aligned},$$
where the second inequality comes from $\varepsilon$-CS, the next equality by cancellation of the sum of terms in $\gD_{\xi_i}$ and $\gD_{j}$, and the last inequality by the suboptimality of $\xi$ as a permutation.	
\end{proof}

The auction algorithm can therefore be regarded as an alternative way to use the machinery of $\C$-transforms. Next we explore another approach grounded on regularization, the so-called Sinkhorn algorithm, which also bears similarities with the auction algorithm as discussed in~\citep{schmitzer2016stabilized}.
%

Note finally that, on low-dimensional regular grids in Euclidean space, it is possible to couple these classical linear solvers with multiscale strategies, to obtain a significant speed-up~\citep{schmitzer2016sparse,oberman2015efficient}.

\chapter{Entropic Regularization of Optimal Transport}
\label{c-entropic}

This chapter introduces a family of numerical schemes to approximate solutions to Kantorovich formulation of optimal transport and its many generalizations. It operates by adding an entropic regularization penalty to the original problem. This regularization has several important advantages, which make it, when taken altogether, a very useful tool: the minimization of the regularized problem can be solved using a simple alternate minimization scheme; that scheme translates into iterations that are simple matrix-vector products, making them particularly suited to execution of GPU; for some applications, these matrix-vector products do not require storing an $n\times m$ cost matrix, but instead only require access to a kernel evaluation; in the case where a large group of measures share the same support, all of these matrix-vector products can be cast as matrix-matrix products with significant speedups; the resulting approximate distance is smooth with respect to input histogram weights and positions of the Diracs and can be differentiated using automatic differentiation.

\section{Entropic Regularization}

The discrete entropy of a coupling matrix is defined as
\eql{\label{eq-discr-entropy}
	\HD(\P) \eqdef -\sum_{i,j} \P_{i,j} (\log(\P_{i,j})-1), 
}
with an analogous definition for vectors, with the convention that $\HD(\a) = -\infty$ if one of the entries $\a_j$ is 0 or negative. The function $\HD$ is 1-strongly concave, because its Hessian is $\partial^2 \HD(P)=-\diag(1/\P_{i,j})$ and $\P_{i,j} \leq 1$. The idea of the entropic regularization of optimal transport is to use $-\HD$ as a regularizing function to obtain approximate solutions to the original transport problem~\eqref{eq-mk-discr}:
\eql{\label{eq-regularized-discr}
	\MKD_\C^\varepsilon(\a,\b) \eqdef 
	\umin{\P \in \CouplingsD(\a,\b)}
		\dotp{\P}{\C} - \varepsilon \HD(\P). 
} 
Since the objective is an $\varepsilon$-strongly convex function, Problem (\ref{eq-regularized-discr}) has a unique optimal solution. The idea to regularize the optimal transport problem by an entropic term can be traced back to modeling ideas in transportation theory~\citep{wilson1969use}: Actual traffic patterns in a network do not agree with those predicted by the solution of the optimal transport problem. Indeed, the former are more diffuse than the latter, which tend to rely on a few routes as a result of the sparsity of optimal couplings for~\eqref{eq-mk-discr}. To mitigate this sparsity, researchers in transportation proposed a model, called the ``gravity'' model~\citep{erlander1980optimal}, that is able to form a more ``blurred'' prediction of traffic given marginals and transportation costs.

Figure~\ref{fig-impact-eps} illustrates the effect of the entropy to regularize a linear program over the simplex $\simplex_3$ (which can thus be visualized as a triangle in two dimensions). Note how the entropy pushes the original LP solution away from the boundary of the triangle. The optimal $\P_\varepsilon$ progressively moves toward an ``entropic center'' of the triangle. This is further detailed in the proposition below. The convergence of the solution of that regularized problem toward an optimal solution of the original linear program has been studied by~\citet{CominettiAsympt}, with precise asymptotics.

\begin{figure}[h!]
\centering
\includegraphics[width=.7\linewidth]{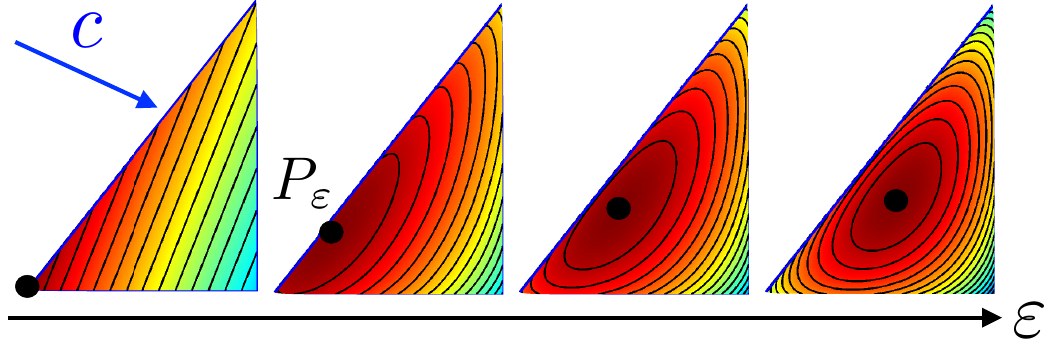}
\caption{\label{fig-impact-eps}
Impact of $\varepsilon$ on the optimization of a linear function on the simplex, solving $\P_\varepsilon = \argmin_{\P \in \simplex_3} \dotp{\C}{\P}-\varepsilon\HD(\P)$ for a varying $\varepsilon$. 
}
\end{figure}

\begin{prop}[Convergence with $\varepsilon$]\label{prop-convergence-eps}
The unique solution $\P_\varepsilon$ of~\eqref{eq-regularized-discr} converges to the optimal solution with maximal entropy within the set of all optimal solutions of the Kantorovich problem, namely
\eql{\label{eq-entropy-conv-1}
	\P_\varepsilon \overset{\varepsilon \rightarrow 0}{\longrightarrow}
	\uargmin{\P} \enscond{ -\HD(\P) }{
		\P \in \CouplingsD(\a,\b), \dotp{\P}{\C} = \MKD_\C(\a,\b),
	}
}
so that in particular
\eq{
	\MKD_\C^\varepsilon(\a,\b) \overset{\varepsilon \rightarrow 0}{\longrightarrow} \MKD_\C(\a,\b).
}
One also has
\eql{\label{eq-entropy-conv-2}
	\P_\varepsilon \overset{\varepsilon \rightarrow \infty}{\longrightarrow}
	\a \otimes \b = \a \transp{\b} = (\a_i \b_j)_{i,j}.
}
\end{prop}
\begin{proof}
	 We consider a sequence $(\varepsilon_\ell)_\ell$ such that $\varepsilon_\ell \rightarrow 0$ and $\varepsilon_\ell > 0$.	
 	We denote $\P_\ell$ the solution of~\eqref{eq-regularized-discr} for $\varepsilon=\varepsilon_\ell$. 
	Since $\CouplingsD(\a,\b)$ is bounded, we can extract a sequence (that we do not relabel for the sake of simplicity) such that $\P_\ell \rightarrow \P^\star$. Since $\CouplingsD(\a,\b)$ is closed, $\P^\star \in \CouplingsD(\a,\b)$. We consider any $\P$ such that $\dotp{\C}{\P} = \MKD_\C(\a,\b)$. By optimality of $\P$ and $\P_\ell$ for their respective optimization problems (for $\varepsilon=0$ and $\varepsilon=\varepsilon_\ell$), one has
 	\eql{\label{eq-proof-gamma-conv}
 		0 \leq \dotp{\C}{\P_\ell} - \dotp{\C}{\P} \leq \varepsilon_\ell ( \HD(\P_\ell)-\HD(\P) ).
 	}
 	Since $\HD$ is continuous, taking the limit $\ell \rightarrow +\infty$ in this expression shows that 
 	$\dotp{\C}{\P^\star} = \dotp{\C}{\P}$ so that $\P^\star$ is a feasible point of~\eqref{eq-entropy-conv-1}. Furthermore, dividing by $\varepsilon_\ell$ in~\eqref{eq-proof-gamma-conv} and taking the limit shows that 
 	$\HD(\P) \leq \HD(\P^\star)$, which shows that $\P^\star$ is a solution of~\eqref{eq-entropy-conv-1}. Since the solution $\P_0^\star$ to this program is unique by strict convexity of $-\HD$, one has $\P^\star = \P_0^\star$, and the whole sequence is converging. 
	In the limit $\varepsilon\rightarrow+\infty$, a similar proof shows that one should rather consider the problem
	\eq{
		\umin{\P \in \CouplingsD(\a,\b)} -\HD(\P),
	}
	the solution of which is $\a \otimes \b$.
\end{proof}

Formula~\eqref{eq-entropy-conv-1} states that for a small regularization $\varepsilon$, the solution converges to the maximum entropy optimal transport coupling.
In sharp contrast,~\eqref{eq-entropy-conv-2} shows that for a large regularization, the solution converges to the coupling with maximal entropy between two prescribed marginals $\a,\b$, namely the joint probability between two independent random variables distributed following $\a,\b$.
A refined analysis of this convergence is performed in~\citet{CominettiAsympt}, including a first order expansion in $\varepsilon$ (resp., $1/\varepsilon$) near $\varepsilon=0$ (resp., $\varepsilon=+\infty$).
Figures~\ref{fig-entropic-densities} and~\ref{fig-entropic} show visually the effect of these two convergences. A key insight is that, as $\varepsilon$ increases, the optimal coupling becomes less and less sparse (in the sense of having entries larger than a prescribed threshold), which in turn has the effect of both accelerating computational algorithms (as we study in~\S\ref{sec-sinkhorn}) and leading to faster statistical convergence (as shown in~\S\ref{sec-entropy-ot-mmd}). 

\newcommand{\myfigSinkEps}[1]{%
\imgBox{\includegraphics[width=.23\linewidth]{sinkhorn-eps/evol-img-#1}}}
\newcommand{\myfigSinkEpsV}[1]{\includegraphics[width=.24\linewidth]{sinkhorn-eps/evol-3d-#1}}

\begin{figure}[h!]
\centering
\begin{tabular}{@{}c@{}c@{}c@{}c@{}}
\myfigSinkEps{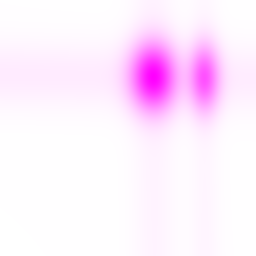} &
\myfigSinkEps{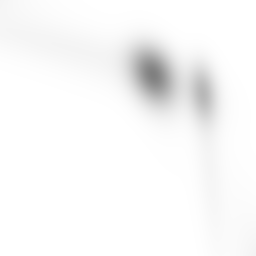} &
\myfigSinkEps{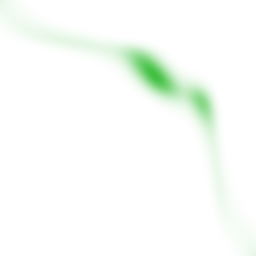} &
\myfigSinkEps{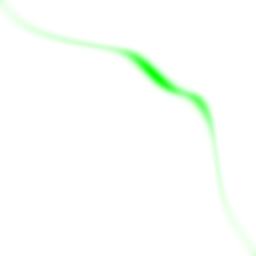} \\
\myfigSinkEpsV{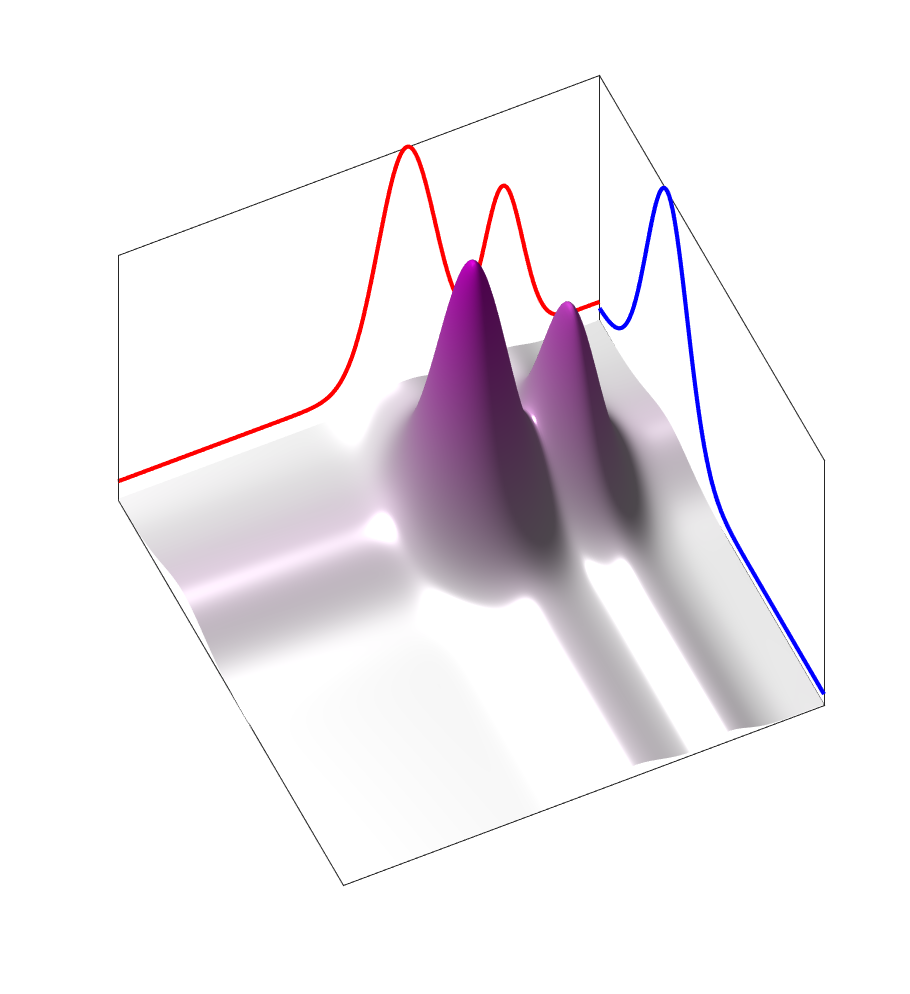} &
\myfigSinkEpsV{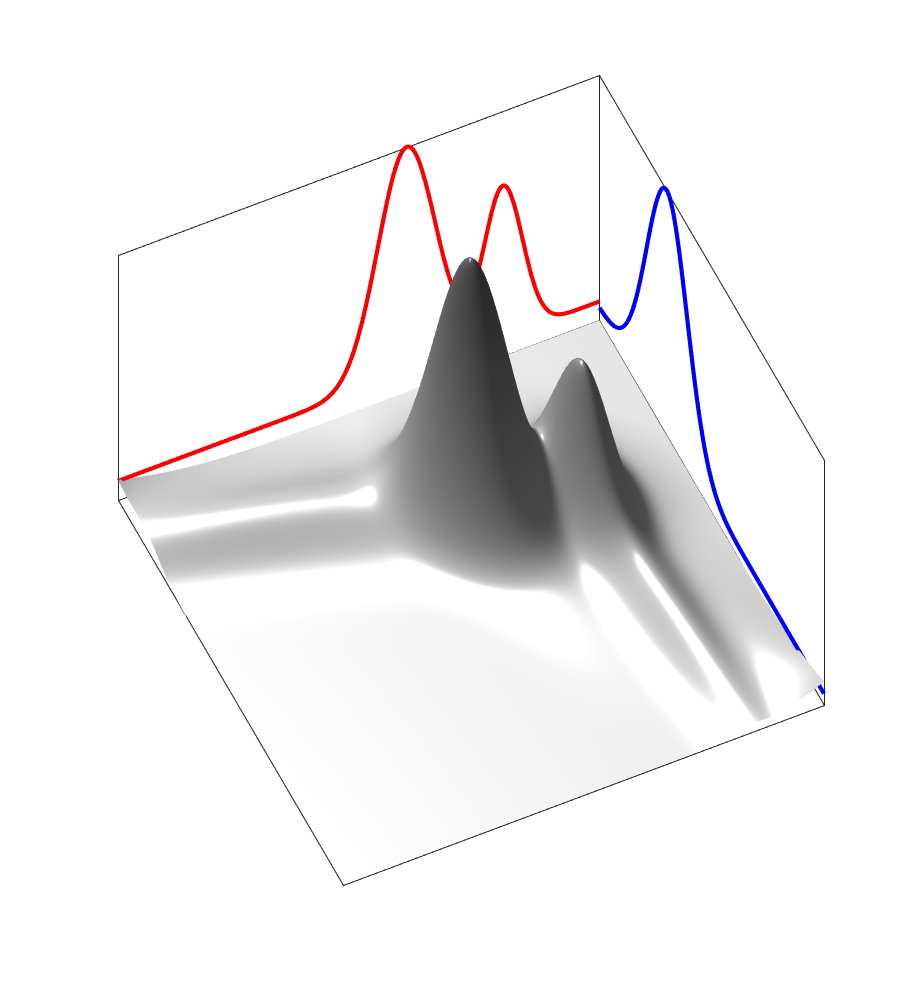} &
\myfigSinkEpsV{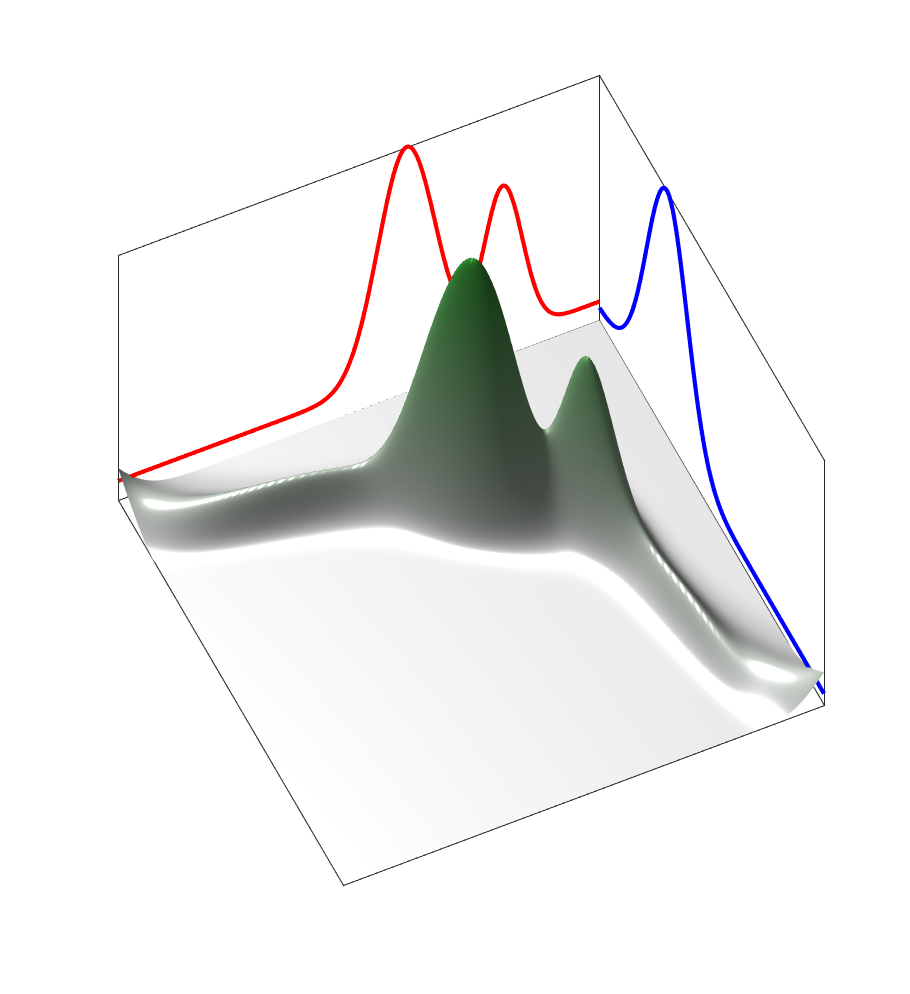} &
\myfigSinkEpsV{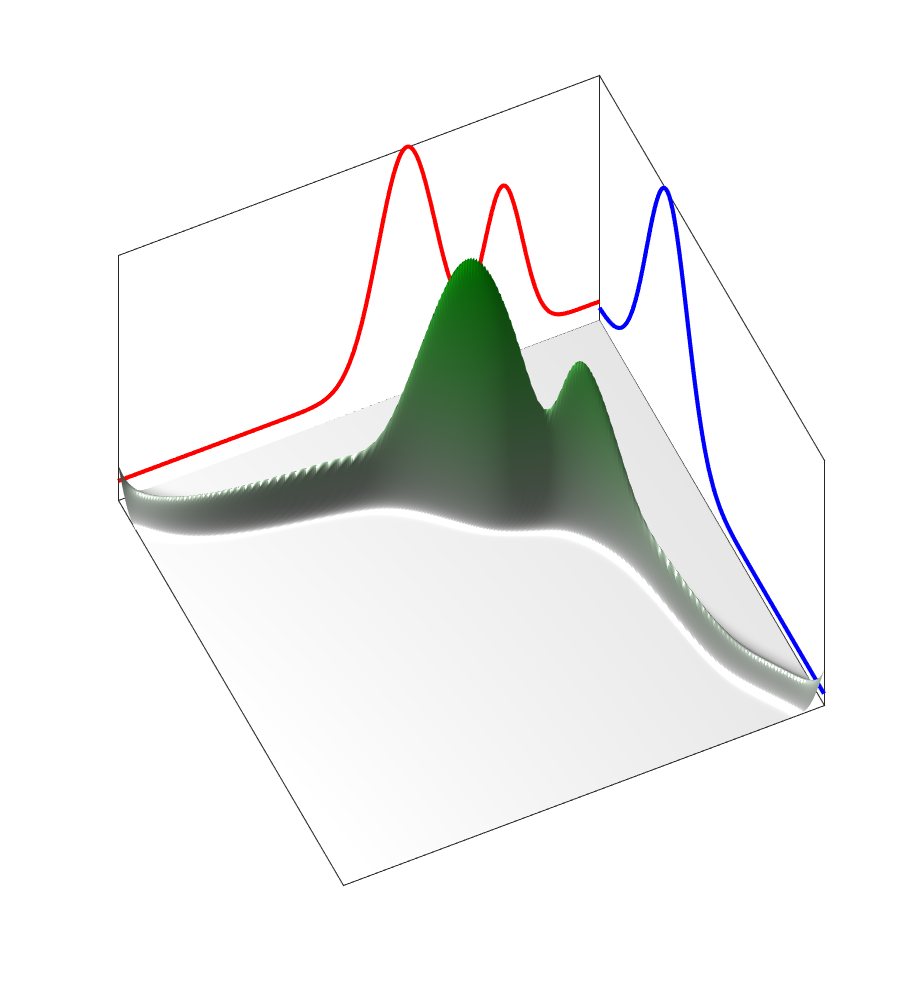} \\
$\varepsilon=10$ &
$\varepsilon=1$ &
$\varepsilon=10^{-1}$ &
$\varepsilon=10^{-2}$ 
\end{tabular}
\caption{\label{fig-entropic-densities}
Impact of $\varepsilon$ on the couplings between two 1-D densities, illustrating Proposition~\ref{prop-convergence-eps}.
Top row: between two 1-D densities. Bottom row: between two 2-D discrete empirical densities with the same number $n=m$ of points (only entries of the optimal $(\P_{i,j})_{i,j}$ above a small threshold are displayed as segments between $x_i$ and $y_j$).
}
\end{figure}

\begin{figure}[h!]
\centering
\includegraphics[width=\linewidth]{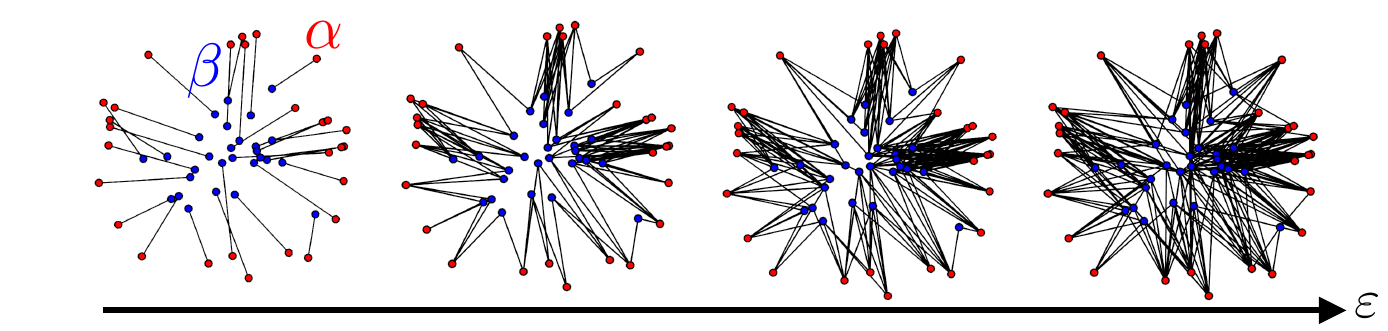}
\caption{\label{fig-entropic}
Impact of $\varepsilon$ on coupling between two 2-D discrete empirical densities with the same number $n=m$ of points (only entries of the optimal $(\P_{i,j})_{i,j}$ above a small threshold are displayed as segments between $x_i$ and $y_j$).
}
\end{figure}

Defining the Kullback--Leibler divergence between couplings as
\eql{\label{eq-kl-defn}
	\KLD(\P|\K) \eqdef \sum_{i,j}  \P_{i,j} \log\pa{\frac{\P_{i,j}}{\K_{i,j}}} - \P_{i,j} + \K_{i,j},
}
the unique solution $\P_\varepsilon$ of~\eqref{eq-regularized-discr} is a projection onto $\CouplingsD(\a,\b)$ of the Gibbs kernel associated to the cost matrix $\C$ as
\eq{
	\K_{i,j} \eqdef e^{-\frac{\C_{i,j}}{\varepsilon}}.
}
Indeed one has that using the definition above
\eql{\label{eq-kl-proj}
	\P_\varepsilon = \Proj_{\CouplingsD(\a,\b)}^\KLD(\K) \eqdef \uargmin{\P \in \CouplingsD(\a,\b)} \KLD(\P|\K).
}

\begin{rem1}{Entropic regularization between discrete measures}\label{rem:entrop-reg}
For discrete measures of the form~\eqref{eq-discr-meas}, the definition of regularized transport extends naturally to
\eql{\label{eq-entropic-disc-meas}
	\MK_\c^\varepsilon(\al,\be) \eqdef \MKD_\C^\varepsilon(\a,\b), 
} 
with cost $\C_{i,j}=\c(x_i,y_j)$, to emphasize the dependency with respect to the positions $(x_i,y_j)$ supporting the input measures.
\end{rem1}

\begin{rem2}{General formulation}
One can consider arbitrary measures by replacing the discrete entropy by the relative entropy with respect to the product measure $\d\al\otimes\d\be(x,y) \eqdef \d\al(x)\d\be(y)$, and propose a regularized counterpart to~\eqref{eq-mk-generic} using
\eql{\label{eq-entropic-generic}
	\MK_\c^\varepsilon(\al,\be) \eqdef 
	\umin{\pi \in \Couplings(\al,\be)}
		\int_{\X \times \Y} c(x,y) \d\pi(x,y) + \varepsilon \KL(\pi|\al\otimes\be),
}
where the relative entropy is a generalization of the discrete Kullback--Leibler divergence~\eqref{eq-kl-defn}
\eql{\label{eq-defn-rel-entropy}
\begin{aligned}	 \KL(\pi|\xi) \eqdef \int_{\X \times \Y} \log\Big( \frac{\d \pi}{\d\xi}(x,y) \Big) \d\pi(x,y)\\
	  + \int_{\X \times \Y} (\d\xi(x,y)-\d\pi(x,y)), 
	  \end{aligned}
}
and by convention $\KL(\pi|\xi)=+\infty$ if $\pi$ does not have a density $\frac{\d \pi}{\d\xi}$ with respect to $\xi$. 
It is important to realize that the reference measure $\al\otimes\be$ chosen in~\eqref{eq-entropic-generic} to define the entropic regularizing term $\KL(\cdot|\al\otimes\be)$ plays no specific role; only its support matters, as noted by the following proposition.

\begin{prop}
	For any $\pi \in \Couplings(\al,\be)$, and for any $(\al',\be')$  having the same 0 measure sets as $(\al,\be)$ (so that they have both densities with respect to one another) one has
	\eq{
		\KL(\pi|\al\otimes\be) = \KL(\pi|\al'\otimes\be') - \KL(\al\otimes\be|\al'\otimes\be').
	}
\end{prop}
\todoK{
\begin{proof}
	\todo{Prove it}
\end{proof}
}

This proposition shows that choosing $\KL(\cdot|\al'\otimes\be')$ in place of $\KL(\cdot|\al\otimes\be)$ in~\eqref{eq-entropic-generic} results in the same solution.

Formula~\eqref{eq-entropic-generic} can be refactored as a projection problem
\eql{\label{eq-entropic-generic-proj}
	\umin{\pi \in \Couplings(\al,\be)} \KL(\pi|\Kk),
}
where $\Kk$ is the Gibbs distributions $\d\Kk(x,y) \eqdef e^{-\frac{c(x,y)}{\varepsilon}} \d\mu(x)\d\nu(y)$.
This problem is often referred to as the ``static Schr\"odinger problem''~\citep{LeonardSchroedinger,RuschendorfThomsen}, since it was initially considered by Schr\"odinger in statistical physics~\citep{Schroedinger31}. 
As $\varepsilon \rightarrow 0$, the unique solution to~\eqref{eq-entropic-generic-proj} converges to the maximum entropy solution to~\eqref{eq-mk-generic}; see~\citep{leonard2012schrodinger,2017-carlier-SIMA}.
Section \ref{sec-entropic-dynamic} details an alternate ``dynamic'' formulation of the Schr\"odinger problem over the space of paths connecting the points of two measures.
\end{rem2}

\begin{rem2}{Mutual entropy}
Similarly to~\eqref{eq-ot-proba-interpretation}, one can re\-phra\-se \eqref{eq-entropic-generic} using random variables 
\eq{
	\MK_\c^\varepsilon(\al,\be) = \umin{(X,Y)} \enscond{ \EE_{(X,Y)}(c(X,Y)) + \epsilon \text{I}(X,Y) }{ X \sim \al, Y \sim \be },
}
where, denoting $\pi$ the distribution of $(X,Y)$, $\text{I}(X,Y) \eqdef \KL(\pi|\al\otimes\be)$ is the so-called mutual information between the two random variables. One has $\text{I}(X,Y) \geq 0$ and $\text{I}(X,Y)=0$ if and only if the two random variables are independent.  
\end{rem2}

\begin{rem2}{Independence and couplings}
	A coupling $\pi \in \Couplings(\al,\be)$ describes the distribution of a couple of random variables $(X,Y)$ defined on $(\X,\Y)$, where $X$ (resp., $Y$) has law $\al$ (resp., $\be$).
	Proposition~\ref{prop-convergence-eps} carries over for generic (nonnecessary discrete) measures, so that the solution $\pi_\varepsilon$ of~\eqref{eq-entropic-generic} converges to the tensor product coupling $\al \otimes \be$ as $\varepsilon \rightarrow+\infty$. This coupling $\al \otimes \be$ corresponds to the random variables $(X,Y)$ being independent. 
	In contrast, as $\varepsilon \rightarrow 0$, $\pi_\varepsilon$ convergence to a solution $\pi_0$ of the OT problem~\eqref{eq-mk-generic}. On $\X=\Y=\RR^\dim$, if $\al$ and $\be$ have densities with respect to the Lebesgue measure, as detailed in Remark~\ref{rem-exist-mongemap}, then $\pi_0$ is unique and supported on the graph of a bijective Monge map $\T : \RR^\dim \rightarrow \RR^\dim$. In this case, $(X,Y)$ are in some sense fully dependent, since $Y=\T(X)$ and $X=\T^{-1}(Y)$. 
	In the simple 1-D case $\dim=1$, a convenient way to visualize the dependency structure between $X$ and $Y$ is to use the copula $\xi_\pi$ associated to the joint distribution $\pi$. The cumulative function defined in~\eqref{eq-cumul-defn} is extended to couplings as 
	\eq{
		\foralls (x,y) \in \RR^2, \quad
		\cumul{\pi}(x,y) \eqdef \int_{-\infty}^x \int_{-\infty}^y \d\pi.  
	}
	The copula is then defined as
	\eq{
		\foralls (s,t) \in [0,1]^2, \quad
		\xi_{\pi}(s,t) \eqdef \cumul{\pi}( \cumul{\al}^{-1}(s),\cumul{\be}^{-1}(t) ),
	}
	where the pseudoinverse of a cumulative function is defined in~\eqref{eq-pseudo-inv-cum}. 
	For independent variables, $\varepsilon=+\infty$, \ie $\pi=\al\otimes \be$, one has $\xi_{\pi_{+\infty}}(s,t)=st$.
	In contrast, for fully dependent variables, $\varepsilon=+\infty$, one has $\xi_{\pi_0}(s,t)=\min(s,t)$.
	Figure~\ref{fig-entropic-copula} shows how entropic regularization generates copula $\xi_{\pi_\varepsilon}$ interpolating between these two extreme cases. 
\end{rem2}

\newcommand{\myfigSinkCopul}[1]{\imgBox{\includegraphics[width=.185\linewidth,trim=65 55 45 40,clip]{sinkhorn-copula/evol-#1}}}
\begin{figure}[h!]
\centering
\begin{tabular}{@{}c@{\hspace{5mm}}c@{}}
\includegraphics[width=.4\linewidth]{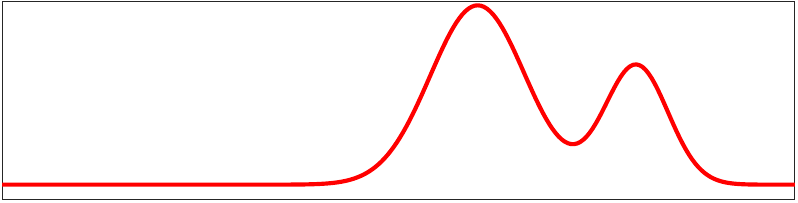} & 
\includegraphics[width=.4\linewidth]{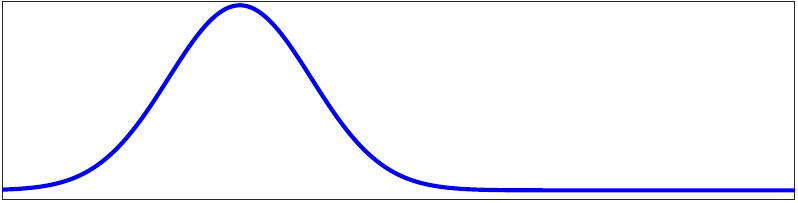}\\
{\color{red}$\al$} & {\color{blue}$\be$}
\end{tabular}
\begin{tabular}{@{}c@{}c@{}c@{}c@{}c@{}}
\myfigSinkCopul{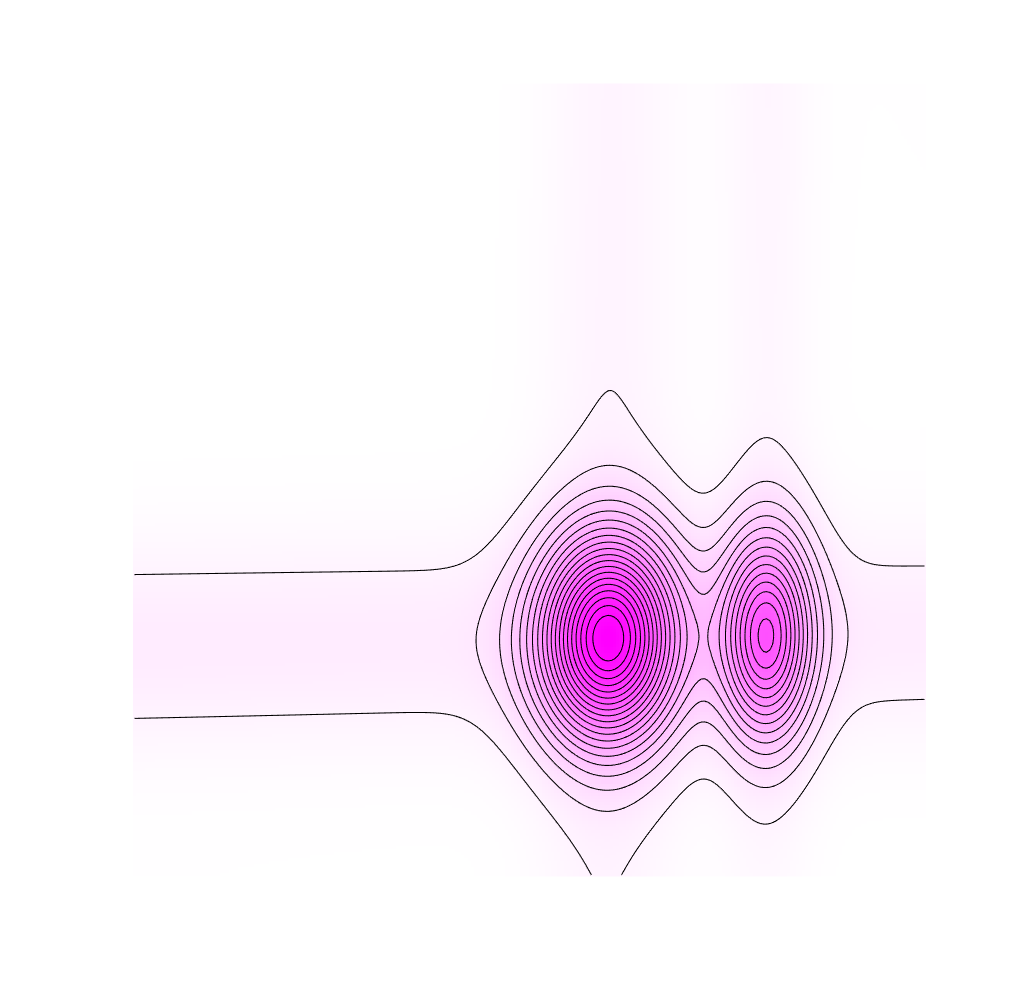} &
\myfigSinkCopul{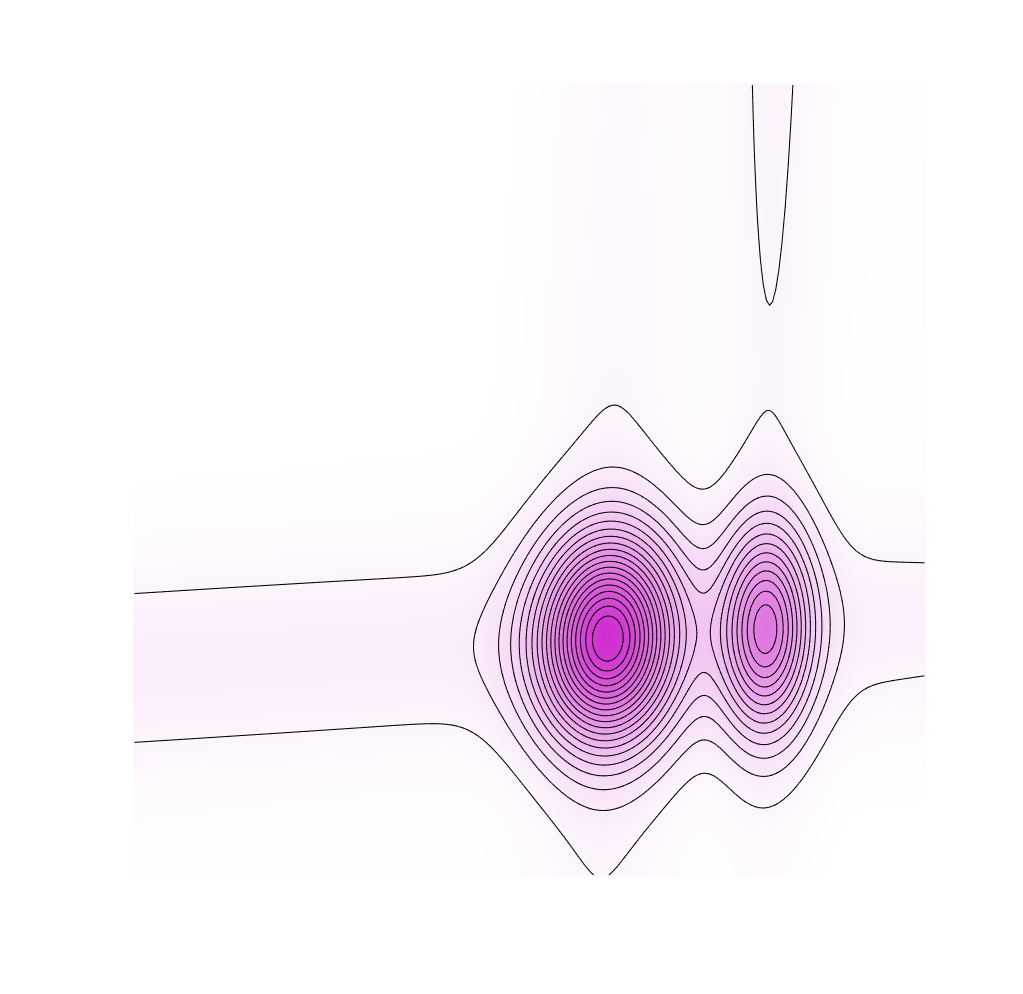} &
\myfigSinkCopul{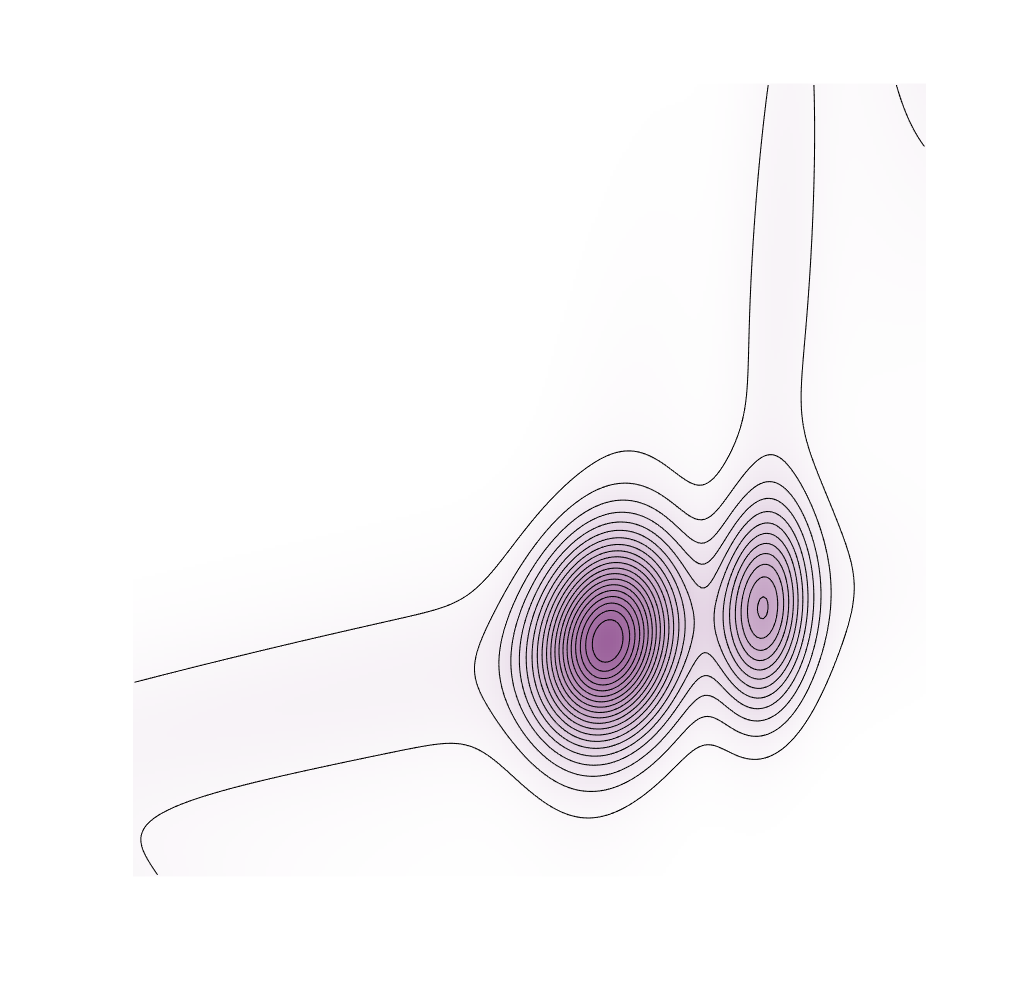} &
\myfigSinkCopul{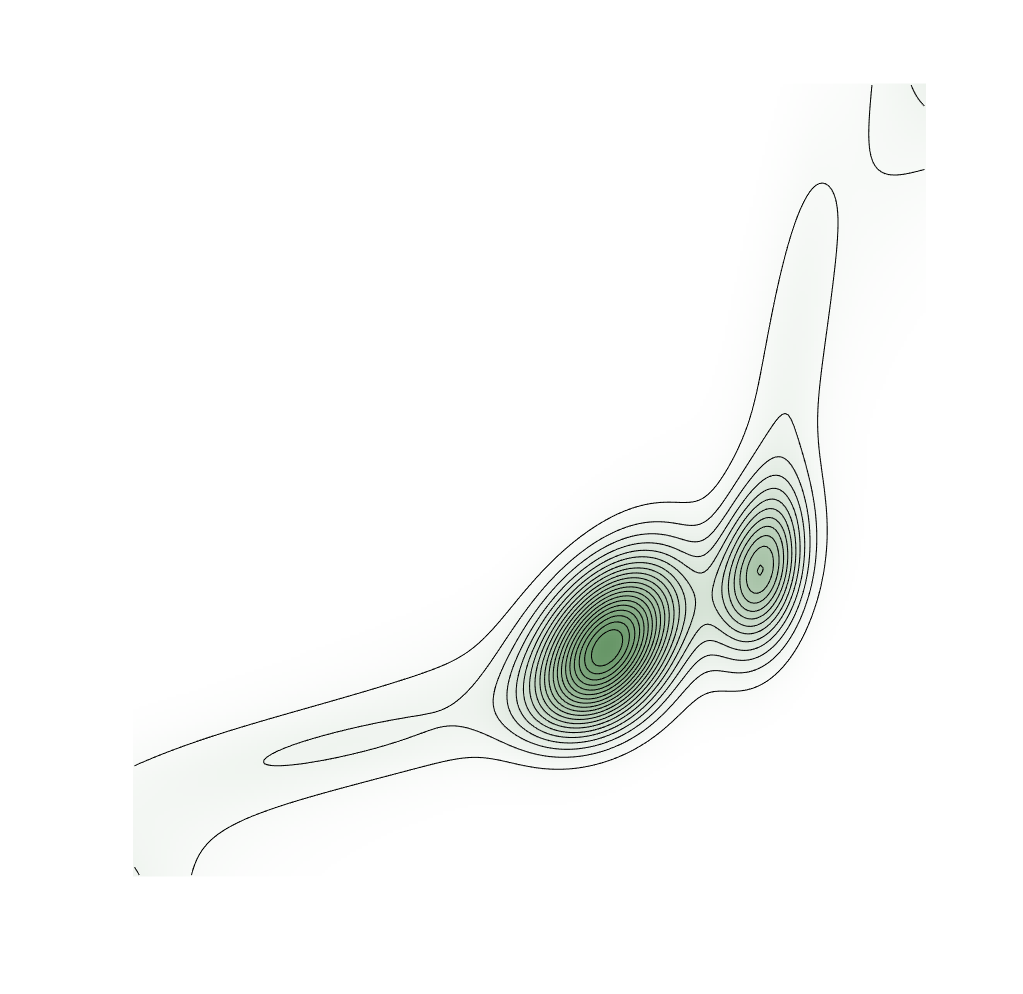} &
\myfigSinkCopul{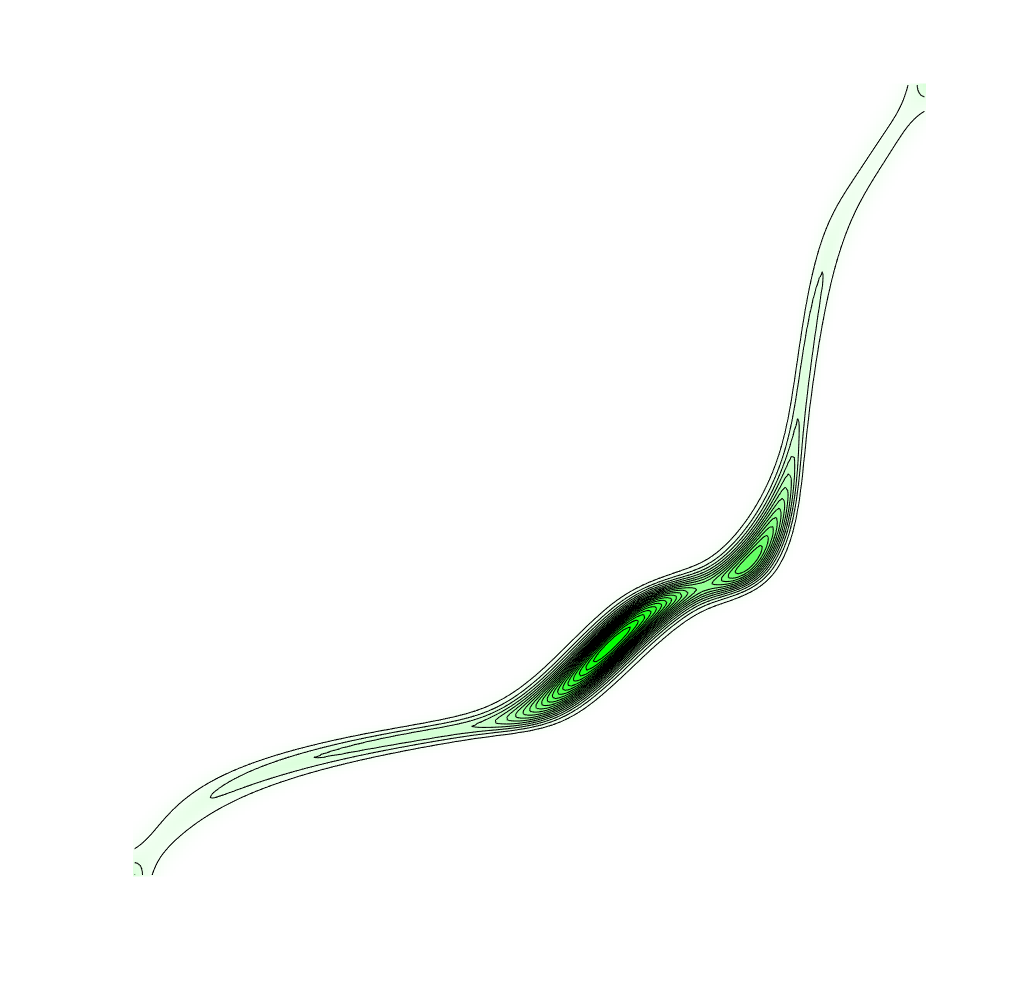} \\
\myfigSinkCopul{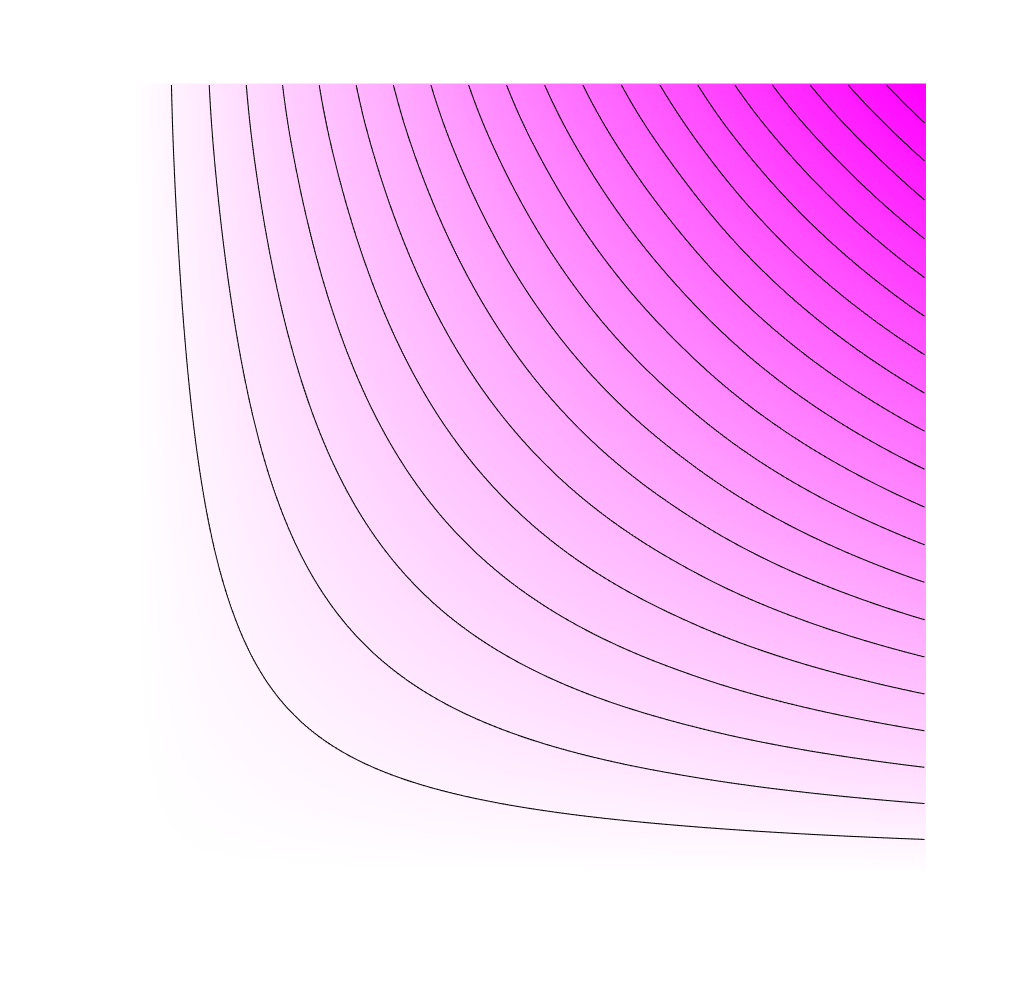} &
\myfigSinkCopul{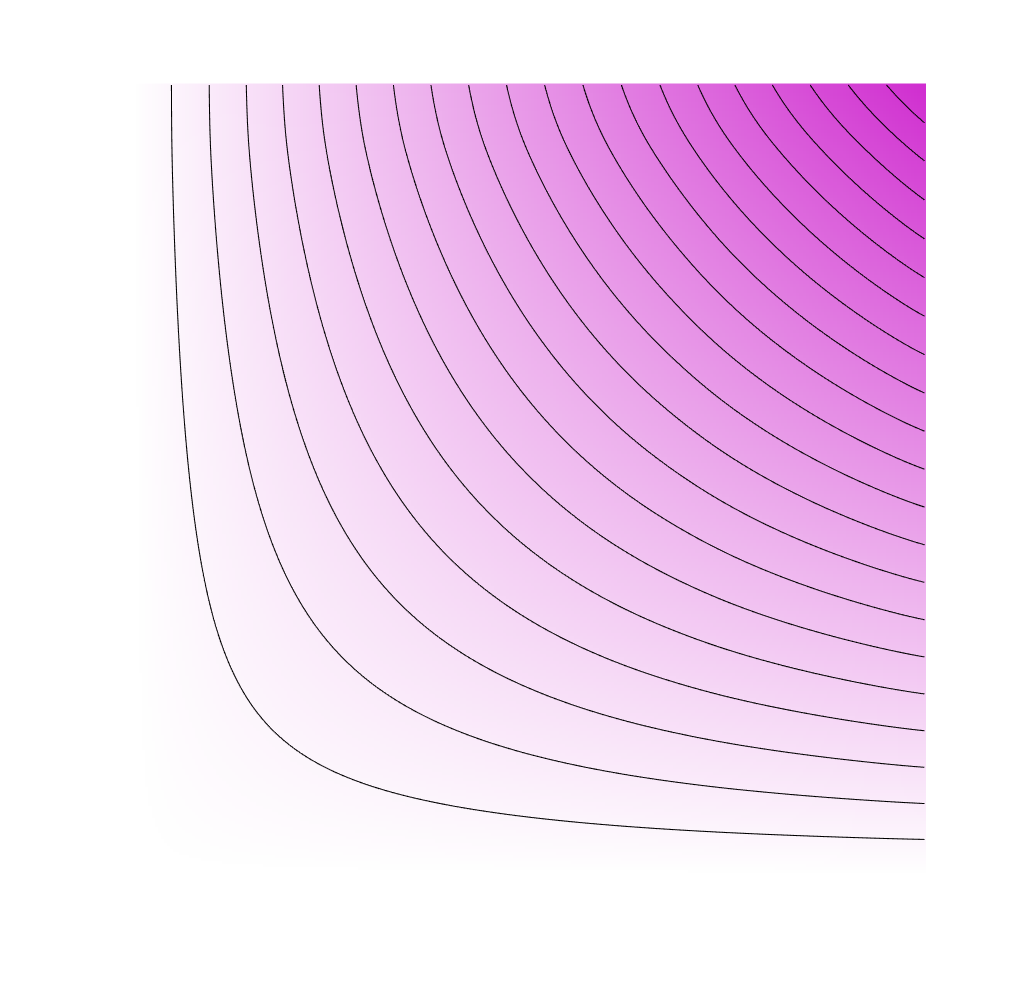} &
\myfigSinkCopul{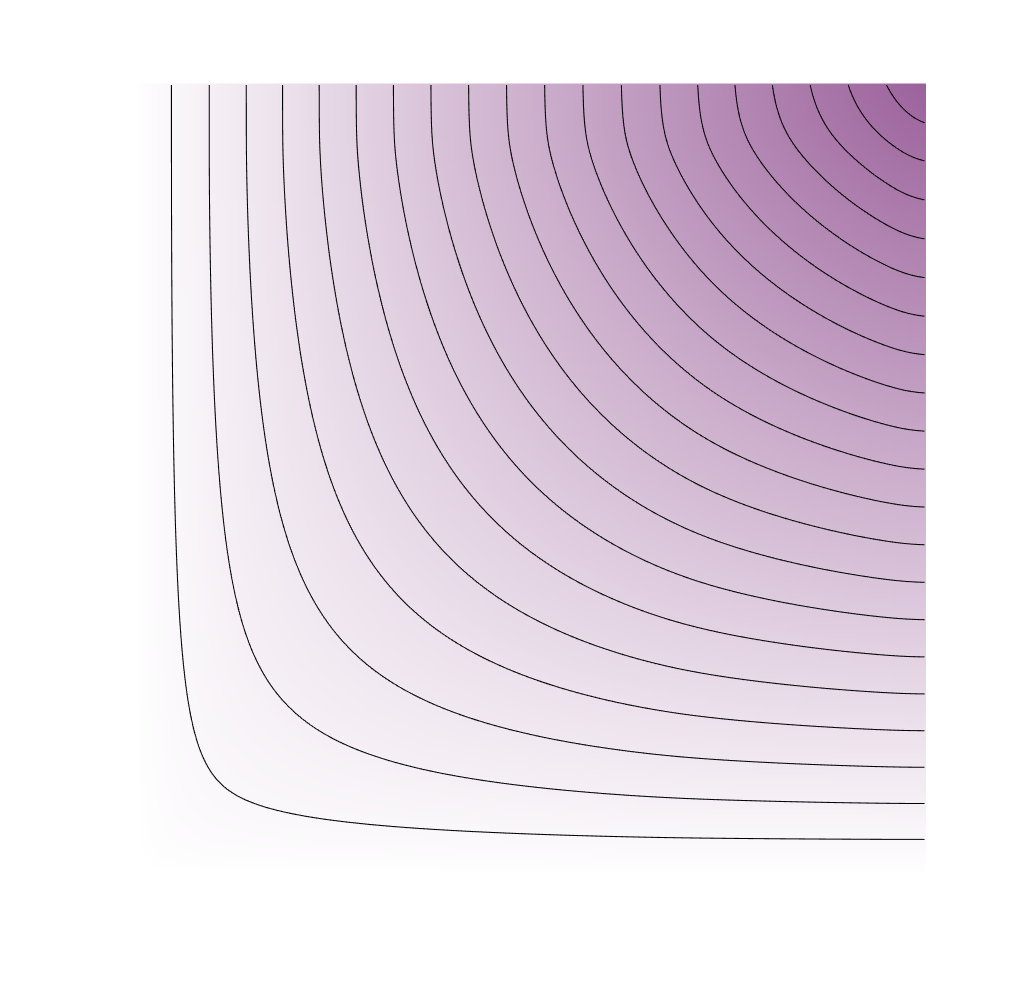} &
\myfigSinkCopul{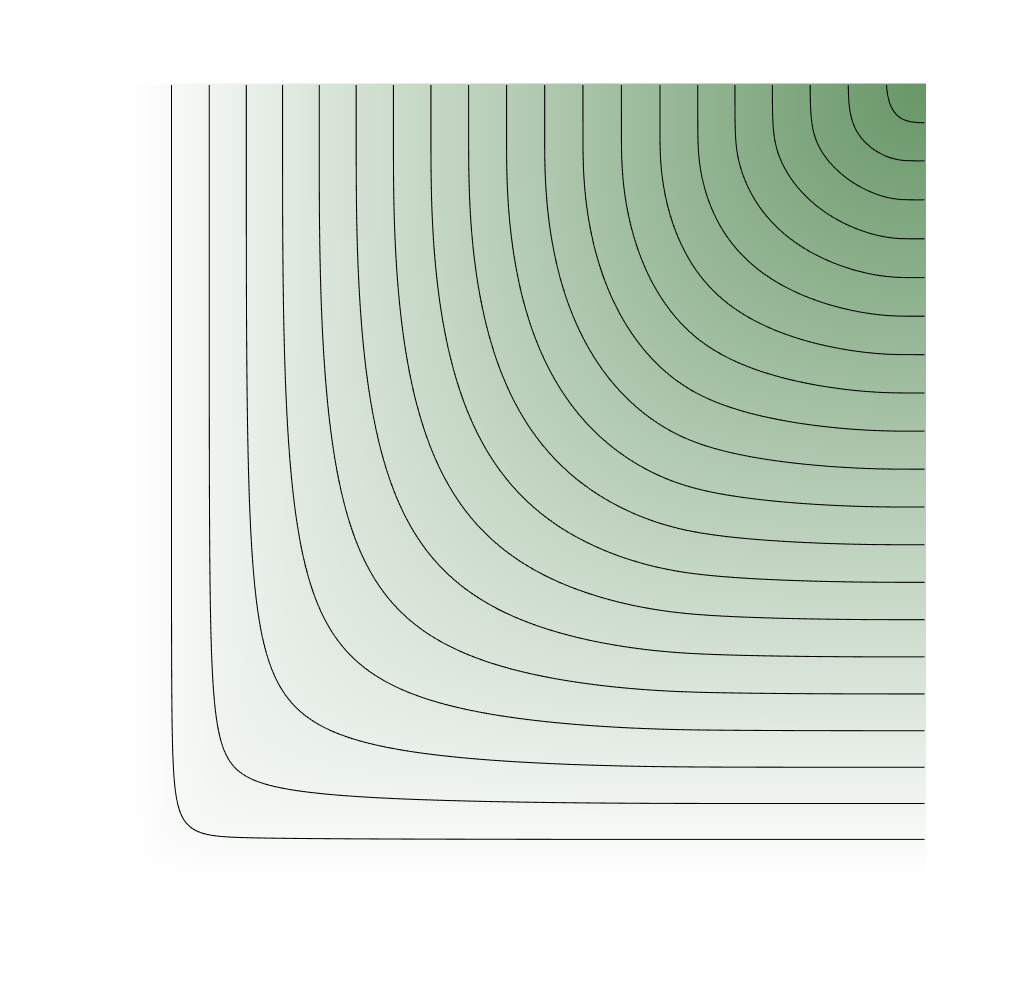} &
\myfigSinkCopul{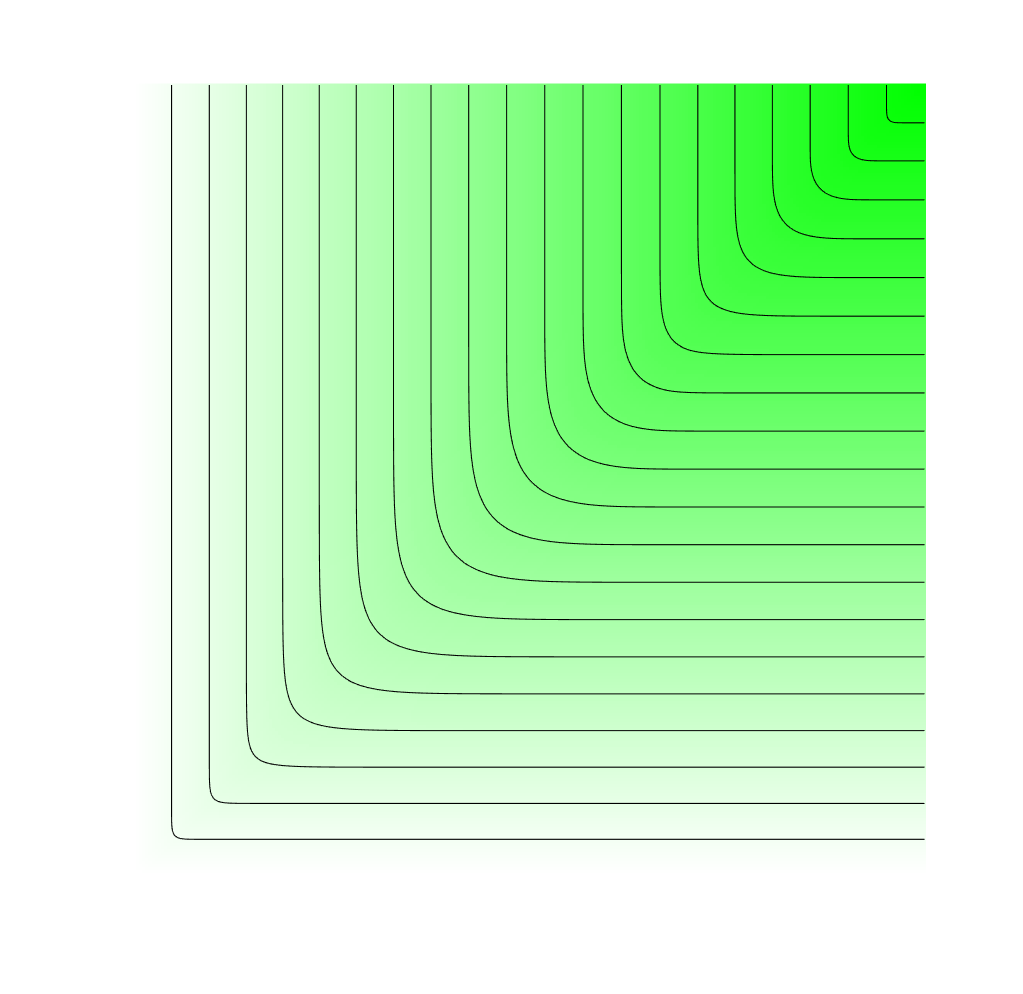} \\
{\color[rgb]{1,0,1} $\varepsilon=10$} &
{\color[rgb]{.75,.25,.75} $\varepsilon=1$} &
{\color[rgb]{.5,.5,.5} $\varepsilon=0.5 \cdot 10^{-1}$} &
{\color[rgb]{.25,.75,.25} $\varepsilon=10^{-1}$} &
{\color[rgb]{0,1,0} $\varepsilon=10^{-3}$} 
\end{tabular}
\caption{\label{fig-entropic-copula}
Top: evolution with $\varepsilon$ of the solution $\pi_\varepsilon$ of~\eqref{eq-entropic-generic}. 
Bottom: evolution of the copula function $\xi_{\pi_\varepsilon}$. 
}
\end{figure}

\section{Sinkhorn's Algorithm and Its Convergence}
\label{sec-sinkhorn}

The following proposition shows that the solution of~\eqref{eq-regularized-discr} has a specific form, which can be parameterized using $n+m$ variables. That parameterization is therefore essentially dual, in the sense that a coupling $\P$ in $\CouplingsD(\a,\b)$ has $nm$ variables but $n+m$ constraints.

\begin{prop}\label{prop-regularized-primal}
The solution to~\eqref{eq-regularized-discr} is unique and has the form
\eql{\label{eq-scaling-form}
	\foralls (i,j) \in \range{n} \times \range{m}, \quad \P_{i,j} = \uD_i \K_{i,j} \vD_j
}
for two (unknown) scaling variable $(\uD,\vD) \in \RR_+^n \times \RR_+^m$. 
\end{prop} 

\begin{proof} 
Introducing two dual variables $\fD\in\RR^n,\gD\in\RR^m$ for each marginal constraint, the Lagrangian of~\eqref{eq-regularized-discr} reads
\eq{\label{eq-sinkhorn-lagrangian}
	\Lag(\P,\fD,\gD)= \dotp{\P}{\C} - \varepsilon \HD(\P) - \dotp{\fD}{\P\ones_m-\a}-\dotp{\gD}{\transp{\P}\ones_n-\b}.
}
First order conditions then yield
$$
	\frac{\partial\Lag(\P,\fD,\gD)}{\partial \P_{i,j}}= \C_{i,j} + \varepsilon \log(\P_{i,j}) - \fD_i -\gD_j=0,
$$
which result, for an optimal $\P$ coupling to the regularized problem, in the expression $\P_{i,j}=e^{\fD_i/\varepsilon}e^{-\C_{i,j}/\varepsilon}e^{\gD_j/\varepsilon}$, which can be rewritten in the form provided above using nonnegative vectors $\uD$ and $\vD$.
\end{proof} 

\paragraph{Regularized OT as matrix scaling.} 

The factorization of the optimal solution exhibited in Equation~\eqref{eq-scaling-form} can be conveniently rewritten in matrix form as $\P=\diag(\uD)\K\diag(\vD)$.
The variables $(\uD,\vD)$ must therefore satisfy the following nonlinear equations which correspond to the mass conservation constraints inherent to $\CouplingsD(\a,\b)$:
\eql{\label{eq-dualsinkhorn-constraints}
	\diag(\uD)\K\diag(\vD)\ones_m=\a,
	\qandq
	\diag(\vD)\K^\top \diag(\uD)\ones_n=\b.
}
These two equations can be further simplified, since $\diag(\vD)\ones_m$ is simply $\vD$, and the multiplication of $\diag(\uD)$ times $\K \vD$ is 
\eql{\label{eq-dualsinkhorn-constraints2}
	\uD \odot (\K \vD) = \a
	\qandq
	\vD \odot (\transp{\K}\uD) = \b,
}
where $\odot$ corresponds to entrywise multiplication of vectors. That problem is known in the numerical analysis community as the matrix scaling problem (see~\citep{nemirovski1999complexity} and references therein).
An intuitive way to handle these equations is to solve them iteratively, by modifying first $\uD$ so that it satisfies the left-hand side of Equation~\eqref{eq-dualsinkhorn-constraints2} and then $\vD$ to satisfy its right-hand side. These two updates define Sinkhorn's algorithm,
\eql{\label{eq-sinkhorn}	
	\itt{\uD} \eqdef \frac{\a}{\K \it{\vD}}
	\qandq
	\itt{\vD} \eqdef \frac{\b}{\transp{\K}\itt{\uD}},
}
initialized with an arbitrary positive vector $\init{\vD} = \ones_m$. The division operator used above between two vectors is to be understood entrywise. Note that a different initialization will likely lead to a different solution for $\uD,\vD$, since $\uD,\vD$ are only defined up to a multiplicative constant (if $\uD,\vD$ satisfy \eqref{eq-dualsinkhorn-constraints} then so do $\lambda\uD,\vD/\lambda$ for any $\lambda>0$).
It turns out, however, that these iterations converge (see Remark~\ref{rem-iterative-projection} for a justification using iterative projections, and see Remark~\ref{rem-global-conv-sinkh} for a strict contraction result) and all result in the same optimal coupling $\diag(\uD)\K\diag(\vD)$. 
Figure~\ref{fig-sinkhorn-convergence}, top row, shows the evolution of the coupling $\diag(\it{\uD})\K\diag(\it{\vD})$ computed by Sinkhorn iterations. It  evolves from the Gibbs kernel $\K$ toward the optimal coupling solving~\eqref{eq-regularized-discr} by progressively shifting the mass away from the diagonal.

\begin{rem}[Historical perspective] The iterations~\eqref{eq-sinkhorn} first appeared in~\citep{yule1912methods,kruithof}. They were later known as the iterative proportional fitting procedure (IPFP)~\citet{DemingStephanIPFP} and RAS~\citep{bacharach1965estimating} methods~\citep{ReviewSinkhorn}. The proof of their convergence is attributed to~\citet{Sinkhorn64}, hence the name of the algorithm. This algorithm was later extended in infinite dimensions by~\citet{Ruschendorf95}. This regularization was used in the field of economics to obtain approximate solutions to optimal transport problems, under the name of gravity models~\citep{wilson1969use,erlander1980optimal,erlander1990gravity}.
It was rebranded as ``softassign'' by~\citet{kosowsky1994invisible} in the assignment case, namely when $\a=\b=\ones_n/n$, and used to solve matching problems in economics more recently by~\citet{Galichon-Entropic}.
This regularization has received renewed attention in data sciences (including machine learning, vision, graphics and imaging) following~\citep{CuturiSinkhorn}, who showed that Sinkhorn's algorithm provides an efficient and scalable approximation to optimal transport, thanks to seamless parallelization when solving several OT problems simultaneously (notably on GPUs; see Remark~\ref{rem-parallel}), and that this regularized quantity also defines, unlike the linear programming formulation, a differentiable loss function (see~\S\ref{sec-regularized-cost}).
There exist countless extensions and generalizations of the Sinkhorn algorithm (see for instance~\S\ref{sec-generalized}). For instance, when $\a=\b$, one can use averaged projection iterations to maintain symmetry~\citep{knight2014symmetry}. 
\end{rem}

\begin{figure}[h!]
\centering
\includegraphics[width=\linewidth]{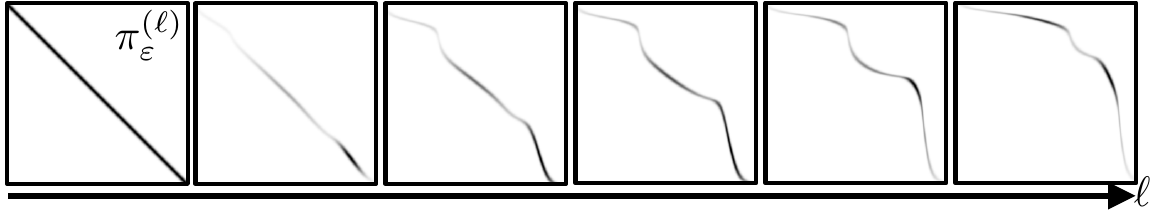}
\begin{tabular}{c}
\includegraphics[width=.5\linewidth]{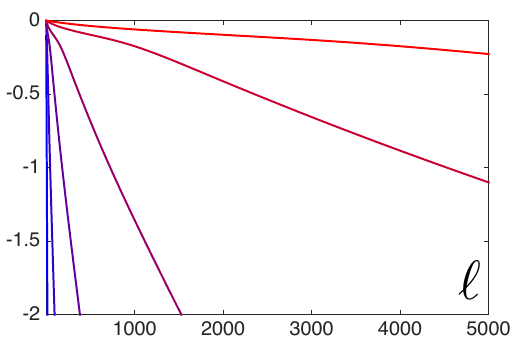}\\
{\color{blue}$\varepsilon=10$} \quad {\color[rgb]{0.5,0,0.5}$\varepsilon=0.1$} \quad {\color{red}$\varepsilon=10^{-3}$}
\end{tabular}
\caption{\label{fig-sinkhorn-convergence}
Top: evolution of the coupling $\it{\pi_\varepsilon}=\diag(\it{\uD})\K\diag(\it{\vD})$ computed at iteration $\ell$ of Sinkhorn's iterations, for 1-D densities on $\X=[0,1]$, $c(x,y)=|x-y|^2$, and $\varepsilon=0.1$.
Bottom: impact of $\varepsilon$ the convergence rate of Sinkhorn, as measured in term of marginal constraint violation $\log( \|\it{\pi_\varepsilon} \ones_m - \b \|_1 )$.
}
\end{figure}

\begin{rem1}{Overall complexity}\label{rem-complexity-rounding}
	By doing a careful convergence analysis (assuming $n=m$ for the sake of simplicity), \citet{altschuler2017near} showed that by setting $\varepsilon = \frac{4\log(n)}{\tau}$, $O(\norm{\C}_\infty^3 \log(n) \tau^{-3})$ Sinkhorn iterations (with an additional rounding step to compute a valid coupling $\hat \P \in \CouplingsD(\a,\b)$) are enough to ensure that $\dotp{\hat \P}{\C} \leq \MKD_{\C}(\a,\b) + \tau$. This implies that Sinkhorn computes a $\tau$-approximate solution of the unregularized OT problem in $O(n^2\log(n) \tau^{-3})$ operations. The rounding scheme consists in, given two vectors $\uD\in\RR^n,\vD\in\RR^m$ to carry out the following updates (\cite[Alg. 2]{altschuler2017near}):
	$$\begin{aligned}
		\uD'& \eqdef \uD\odot\min\left(\frac{\a}{\uD\odot(\K\vD)},\ones_n\right), 
		\vD' \eqdef \vD\odot \min\left(\frac{\b}{\vD\odot(\transp{\K}\uD')},\ones_n\right), \\
	\Delta_{\a}& \eqdef \a- \uD'\odot (\K \vD'), \Delta_{\b} \eqdef \b- \vD'\odot (\transp{\K} \uD), \\
	\hat\P& \eqdef \diag(\uD')\K\diag(\vD')+\Delta_{\a} \transp{(\Delta_{\b})}/\norm{\Delta_{\a}}_1.
	\end{aligned}$$
	This yields a matrix $\hat\P \in \CouplingsD(\a,\b)$ such that the $1$-norm between $\hat\P$ and $\diag(\uD)\K\diag(\vD)$ is controlled by the marginal violations of $\diag(\uD)\K\diag(\vD)$, namely
	$$ \norm{\hat\P-\diag(\uD)\K\diag(\vD)}_1 \leq \norm{\a-\uD\odot(\K\vD)}_1+\norm{\b-\vD\odot(\transp{\K}\uD)}_1.$$ 
This field remains active, as shown by the recent improvement on the result above by \citet{pmlr-v80-dvurechensky18a}.
\end{rem1}

\begin{rem}[Numerical stability of Sinkhorn iterations]\label{rem-stability}
As we discuss in Remarks~\ref{rem-global-conv-sinkh} and \ref{rem-local-conv}, the convergence of Sinkhorn's algorithm deteriorates as $\varepsilon\rightarrow 0$. In numerical practice, however, that slowdown is rarely observed in practice for a simpler reason: Sinkhorn's algorithm will often fail to terminate as soon as some of the elements of the kernel $\K$ become too negligible to be stored in memory as positive numbers, and become instead null. This can then result in a matrix product $\K\vD$ or $\transp{\K}\uD$ with ever smaller entries that become null and result in a division by $0$ in the Sinkhorn update of Equation~\eqref{eq-sinkhorn}. Such issues can be partly resolved by carrying out computations on the multipliers $\uD$ and $\vD$ in the log domain. That approach is carefully presented in Remark~\ref{rem-log-sinkh} and is related to a direct resolution of the dual of Problem~\eqref{eq-regularized-discr}.
\end{rem}

\begin{rem}[Relation with iterative projections]\label{rem-iterative-projection}
Denoting 
\eq{
	\Cc^1_\a \eqdef \enscond{\P}{\P\ones_m=\a}
	\qandq
	\Cc^2_\b \eqdef \enscond{\P}{\transp{\P}\ones_m=\b}
}
the rows and columns constraints, one has $\CouplingsD(\a,\b) = \Cc^1_\a \cap \Cc^2_\b$. One can use Bregman iterative projections~\citep{bregman1967relaxation},
\eql{\label{eq-kl-sinkh-proj}
	\itt{\P} \eqdef \Proj_{\Cc^1_\a}^{\KLD}(\it{\P})
	\qandq
	\ittt{\P} \eqdef \Proj_{\Cc^2_\b}^{\KLD}(\itt{\P}).
}
Since the sets $\Cc^1_\a$ and $\Cc^2_\b$ are affine, these iterations are known to converge to the solution of~\eqref{eq-kl-proj}; see~\citep{bregman1967relaxation}. These iterates are equivalent to Sinkhorn iterations~\eqref{eq-sinkhorn} since defining 
\eq{\label{eq-sink-matrix}\P^{(2\ell)} \eqdef \diag(\it{\uD}) \K \diag(\it{\vD}),}
one has
\begin{align*}
	\P^{(2\ell+1)} &\eqdef \diag(\itt{\uD}) \K \diag(\it{\vD}) \\
	\qandq
	\P^{(2\ell+2)} &\eqdef \diag(\itt{\uD}) \K \diag(\itt{\vD}).
\end{align*}
In practice, however, one should prefer using~\eqref{eq-sinkhorn}, which only requires manipulating scaling vectors and multiplication against a Gibbs kernel, which can often be accelerated (see Remarks~\ref{rem-separable} and~\ref{rem-geod-heat} below). 
\end{rem}

\begin{rem}[Proximal point algorithm]
In order to approximate a solution of the unregularized ($\epsilon=0$) problem~\eqref{eq-mk-discr}, it is possible to use iteratively the Sinkhorn algorithm, using the so-called proximal point algorithm for the $\KL$ metric. 
We denote $F(\P) \eqdef \dotp{\P}{\pi} + \iota_{\CouplingsD(\a,\b)}(\P)$ the unregularized objective function.
The proximal point iterations for the $\KLD$ divergence computes a minimizer of $F$, and hence a solution of the unregularized OT problem~\eqref{eq-mk-discr}, by computing iteratively
\eql{\label{eq-prox-point}
	\itt{\P} \eqdef \Prox_{\frac{1}{\epsilon} F}^{\KLD}(\it{\P})
	\eqdef 
	\uargmin{\P \in \RR_+^{n \times m}} \KLD(\P|\it{\P}) + \frac{1}{\epsilon} F(\P)
}
starting from an arbitrary $\init{\P}$ (see also~\eqref{eq-prox-kl}).
The proximal point algorithm is the most basic proximal splitting method. 
Initially introduced for the Euclidean metric (see, for instance, (\citealt{rockafellar1976monotone})), it extends to any Bregman divergence~\citep{censor1992proximal}, so in particular it can be applied here for the $\KLD$ divergence (see Remark~\ref{rem-bregman}).   
The proximal operator is usually not available in closed form, so some form of subiterations are required. 
The optimization appearing in~\eqref{eq-prox-point} is very similar to the entropy regularized problem~\eqref{eq-regularized-discr}, with the relative entropy $\KLD(\cdot|\it{\P})$ used in place of the negative entropy $-\HD$. 
Proposition~\ref{prop-regularized-primal} and Sinkhorn iterations~\eqref{eq-sinkhorn} carry over to this more general setting when defining the Gibbs kernel as $\K=e^{-\frac{\C}{\epsilon}} \odot \it{\P} = ( e^{-\frac{\C_{i,j}}{\epsilon}} \it{\P}_{i,j} )_{i,j}$.
Iterations~\eqref{eq-prox-point} can thus be implemented by running the Sinkhorn algorithm at each iteration.
Assuming for simplicity $\init{\P}=\ones_n \ones_m^\top$, these iterations thus have the form 
\begin{align*}
	\itt{P} &= 
		\diag(\it{\uD}) ( e^{-\frac{\C}{\epsilon}} \odot \it{\P} ) \diag(\it{\vD})\\
		&=
		\diag(\it{\uD} \odot \cdots \odot \init{\uD})  
			e^{-\frac{(\ell+1)\C}{\epsilon}} \odot \it{\P} ) 
		\diag(\it{\vD} \odot \cdots \odot \init{\vD}).		
\end{align*}
The proximal point iterates apply therefore iteratively Sinkhorn's algorithm with a kernel $e^{-\frac{\C}{\epsilon/\ell}}$, i.e., with a decaying regularization parameter $\epsilon/\ell$.
This method is thus tightly connected to a series of works which combine Sinkhorn with some decaying schedule on the regularization; see, for instance,~\citep{kosowsky1994invisible}. They are efficient in small spacial dimension, when combined with a multigrid strategy to approximate the coupling on an adaptive sparse grid~\citep{schmitzer2016stabilized}.  
\end{rem}

\begin{rem}[Other regularizations]
It is possible to replace the entropic term $-\HD(\P)$ in~\eqref{eq-regularized-discr} by any strictly convex penalty $R(\P)$, as detailed, for instance, in~\citep{dessein2016regularized}. A typical example is the squared $\ell^2$ norm 
\eql{\label{eq-quad-regul}
	R(\P) = \sum_{i,j} \P_{i,j}^2 + \iota_{\RR_+}(\P_{i,j}); 
} 
see~\citep{essid2017quadratically}. 
Another example is the family of Tsallis entropies~\citep{muzellec2017tsallis}.
Note, however, that if the penalty function is defined even when entries of $\P$ are nonpositive, which is, for instance, the case for a quadratic regularization~\eqref{eq-quad-regul}, then one must add back a nonnegativity constraint $\P \geq 0$, in addition to the marginal constraints $\P\ones_m=\a$ and $\P^\top \ones_n=\b$. Indeed, one can afford to ignore the nonnegativity constraint using entropy because that penalty incorporates a logarithmic term which forces the entries of $\P$ to stay in the positive orthant. This implies that the set of constraints is no longer affine and iterative Bregman projections do not converge anymore to the solution. 
A workaround is to use instead~\citeauthor{Dykstra83}'s algorithm~(\citeyear{Dykstra83,Dykstra85}) (see also \citealt{bauschke-lewis}), as detailed in~\citep{2015-benamou-cisc}. This algorithm uses projections according to the Bregman divergence associated to $R$. We refer to Remark~\ref{rem-bregman} for more details regarding Bregman divergences. An issue is that in general these projections cannot be computed explicitly. 
For the squared norm~\eqref{eq-quad-regul}, this corresponds to computing the Euclidean projection on $(\Cc^1_\a,\Cc^2_\b)$ (with the extra positivity constraints), which can be solved efficiently using projection algorithms on simplices~\citep{condat2015fast}. The main advantage of the quadratic regularization over entropy is that it produces sparse approximation of the optimal coupling, yet this comes at the expense of a slower algorithm that cannot be parallelized as efficiently as Sinkhorn to compute several optimal transports simultaneously (as discussed in~\S\ref{rem-parallel}). Figure~\ref{fig-quad-regul} contrasts the approximation achieved by entropic and quadratic regularizers. 
\end{rem}

\newcommand{\myfigRegQuad}[1]{\imgBox{\includegraphics[width=.185\linewidth,trim=65 55 45 40,clip]{quadratic-regul/quadratic-levelsets-#1}}}
\newcommand{\myfigRegEntrop}[1]{\imgBox{\includegraphics[width=.185\linewidth,trim=65 55 45 40,clip]{quadratic-regul/entropy-levelsets-#1}}}

\begin{figure}[h!]
\begin{tabular}{@{}c@{}c@{}c@{}c@{}c@{}}
\myfigRegEntrop{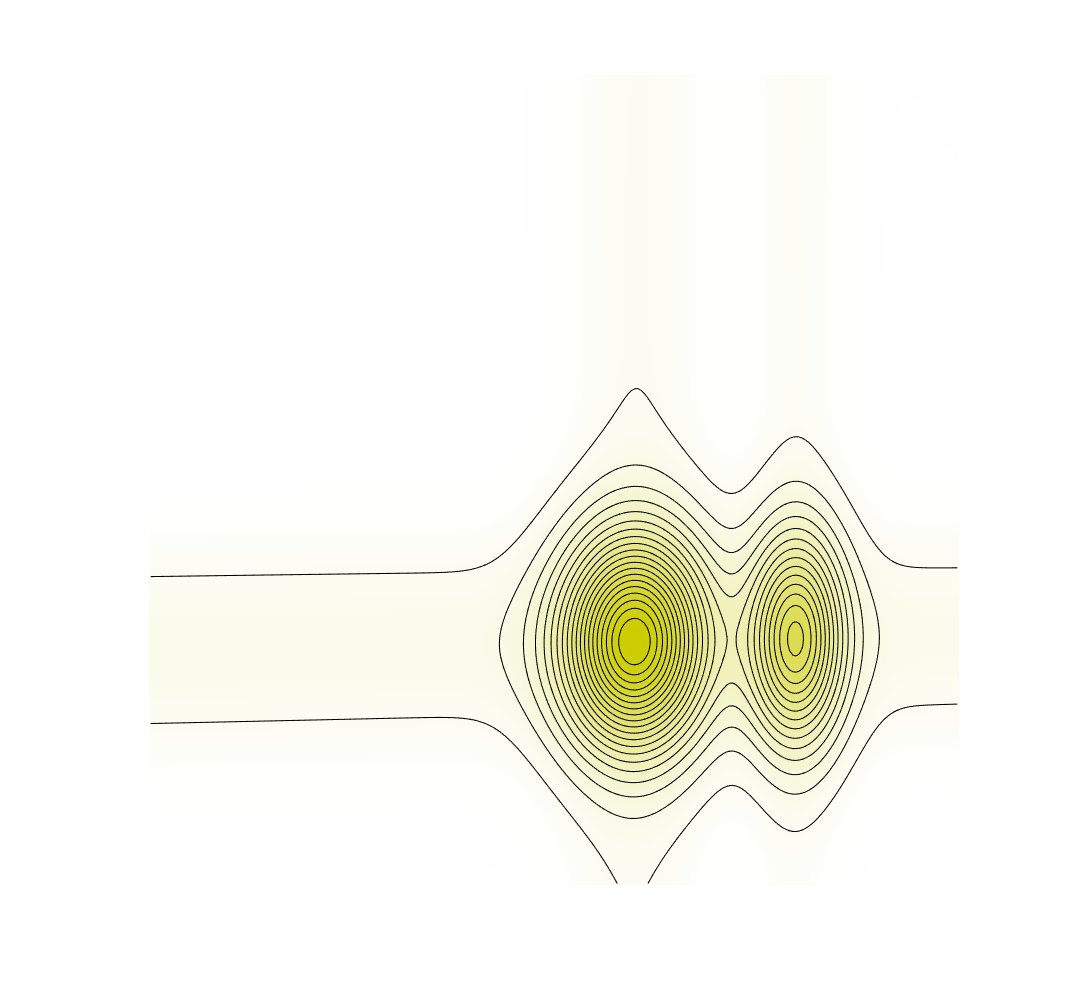} &
\myfigRegEntrop{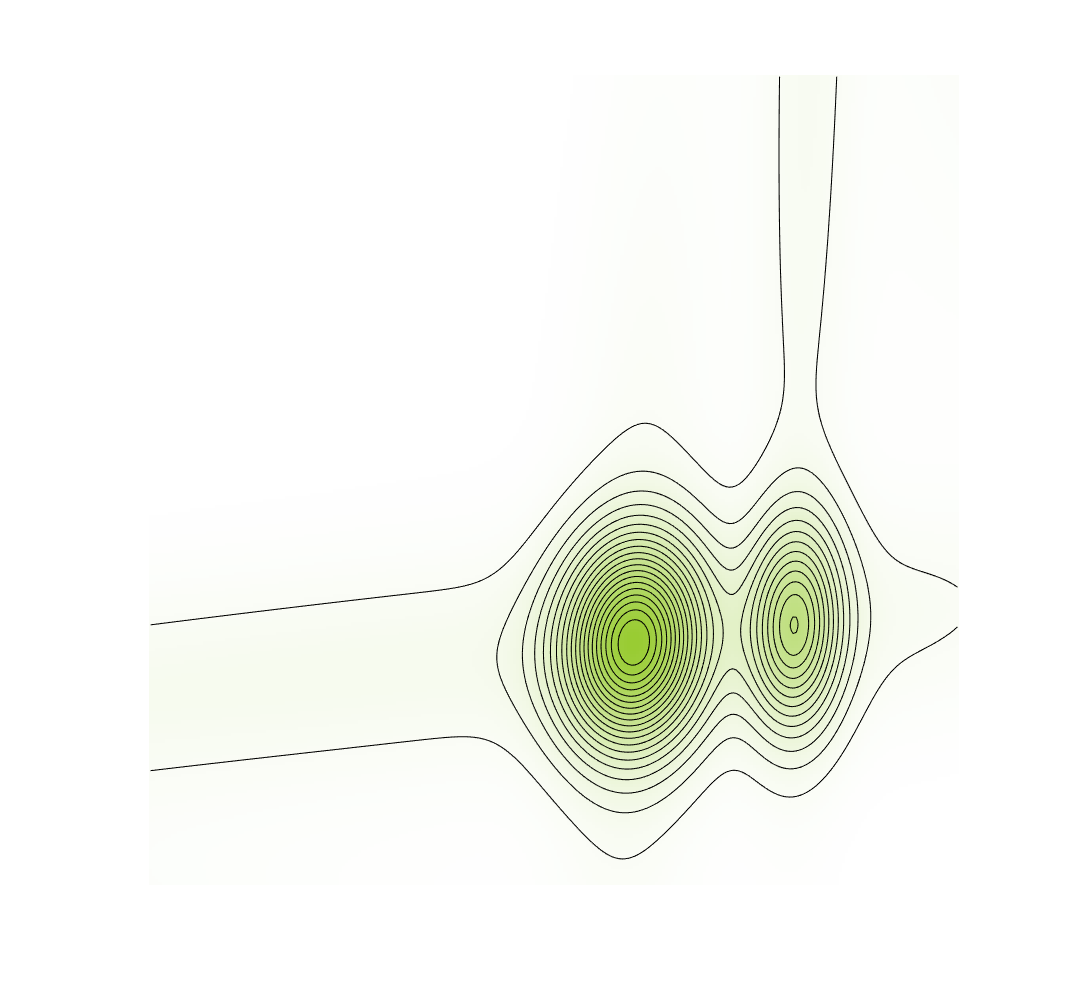} &
\myfigRegEntrop{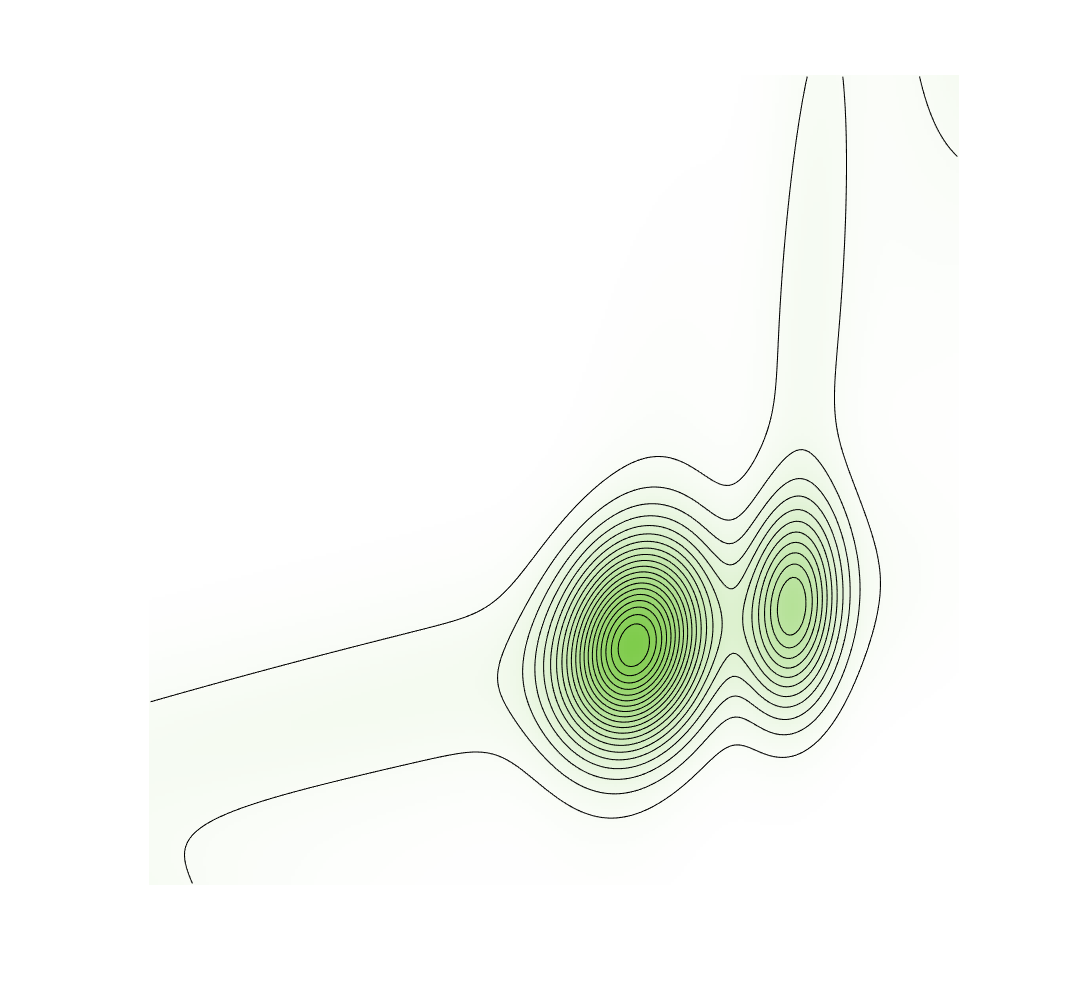} &
\myfigRegEntrop{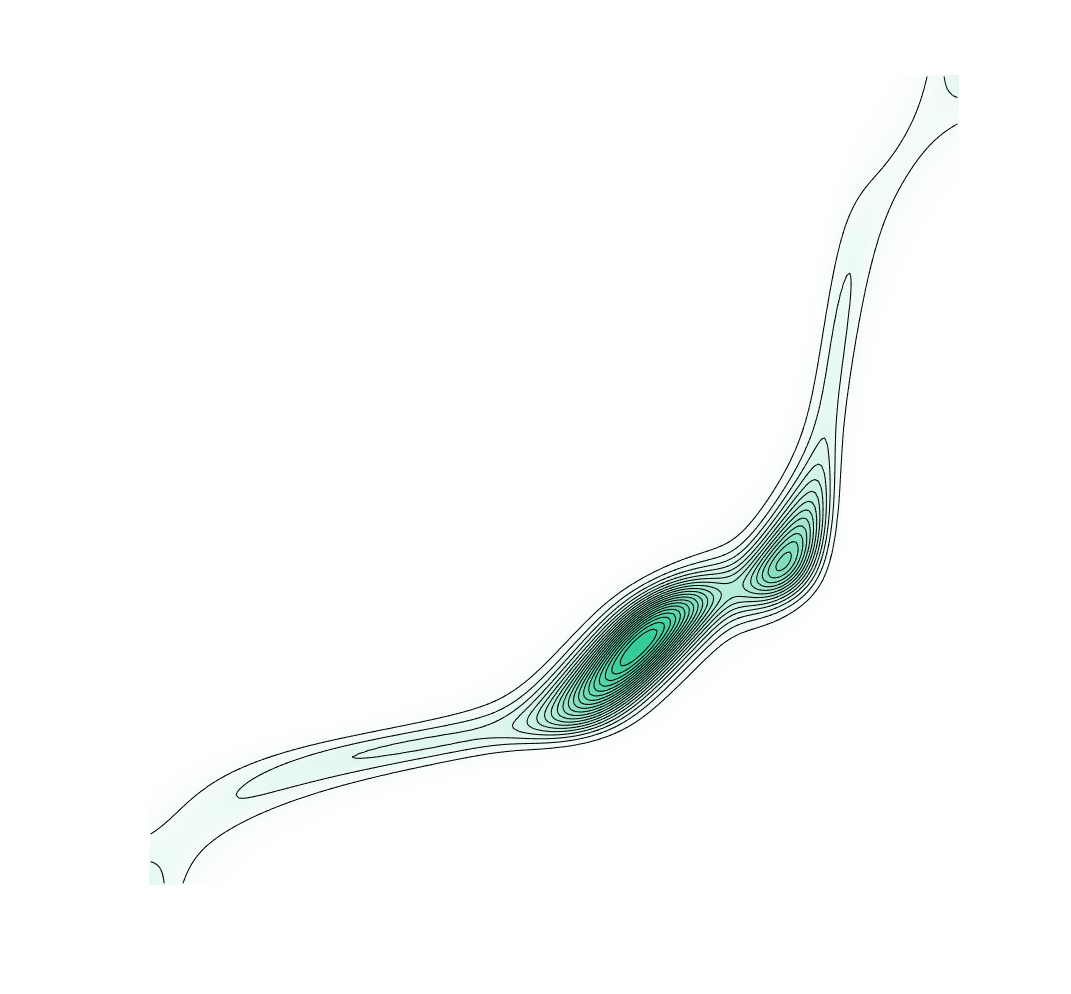} &
\myfigRegEntrop{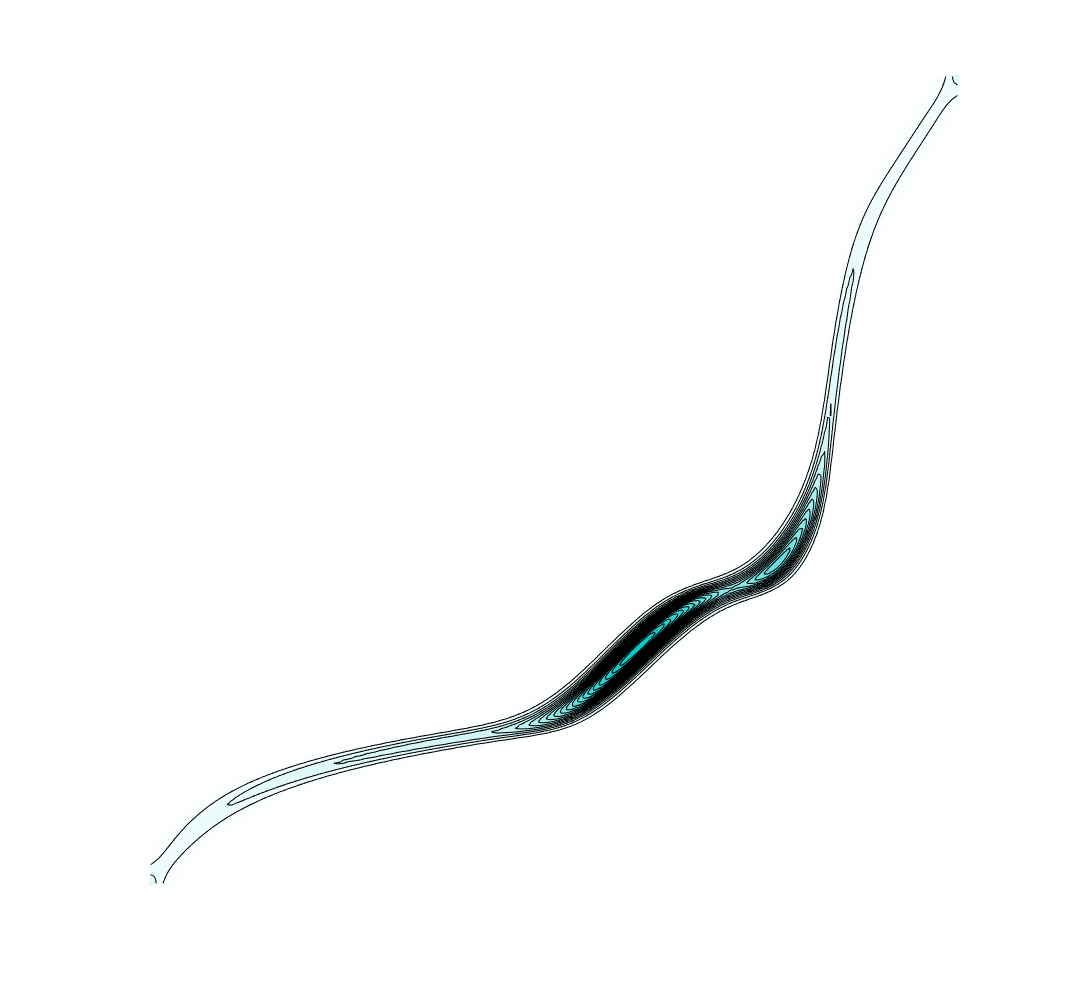} \\
{\color[rgb]{.6,.6,0} $\varepsilon=10$} &
{\color[rgb]{.45,.6,.15} $\varepsilon=1$} &
{\color[rgb]{.3,.6,.3} $\varepsilon=0.5 \cdot 10^{-1}$} &
{\color[rgb]{.15,.6,.45} $\varepsilon=10^{-1}$} &
{\color[rgb]{0,.6,.6} $\varepsilon=10^{-3}$} \\
\myfigRegQuad{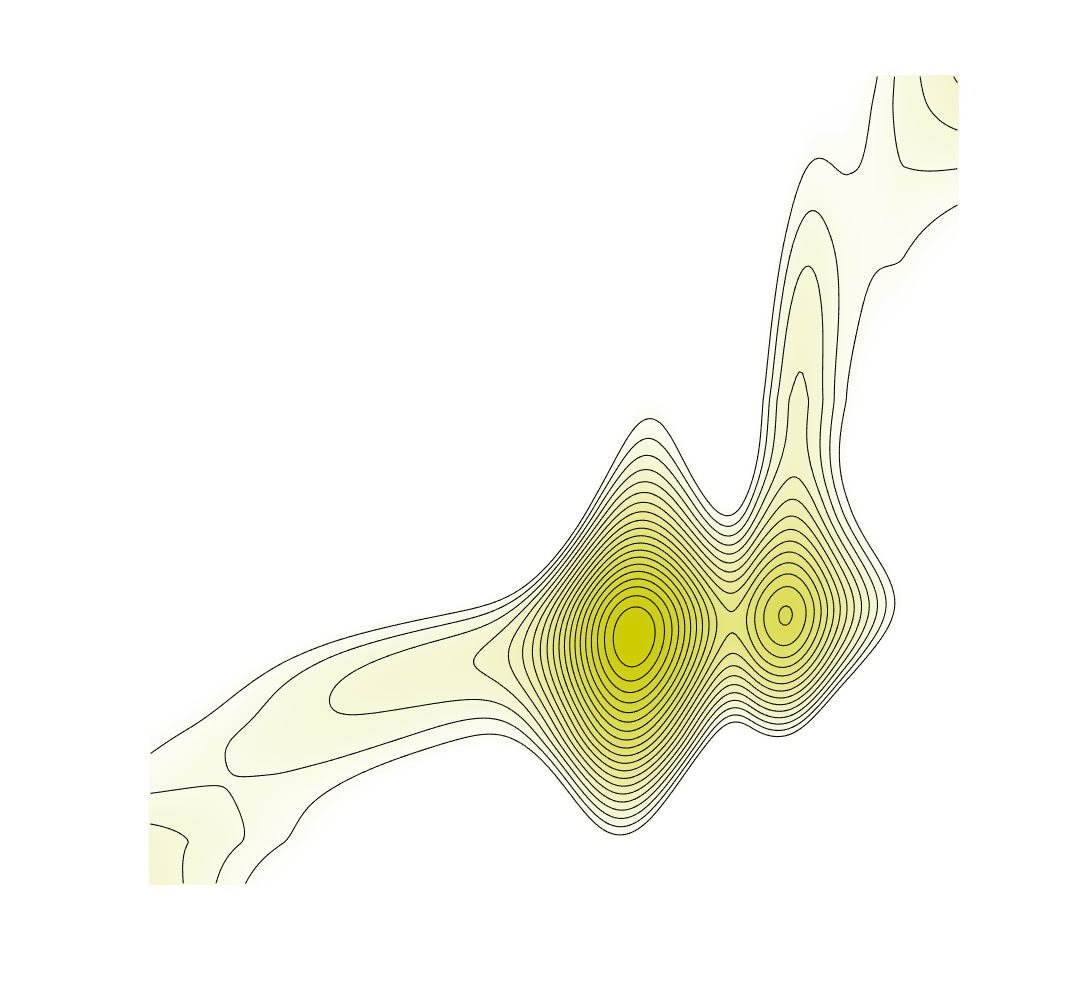} &
\myfigRegQuad{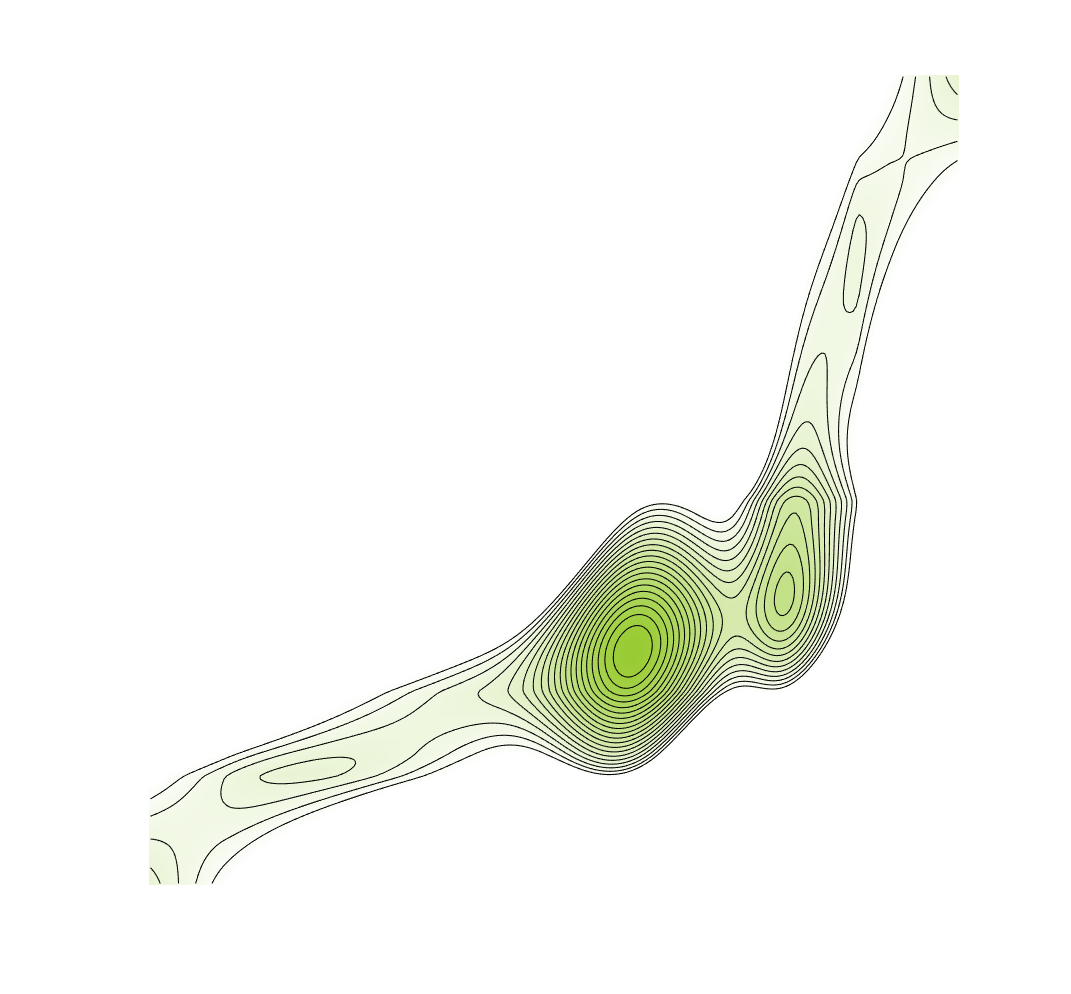} &
\myfigRegQuad{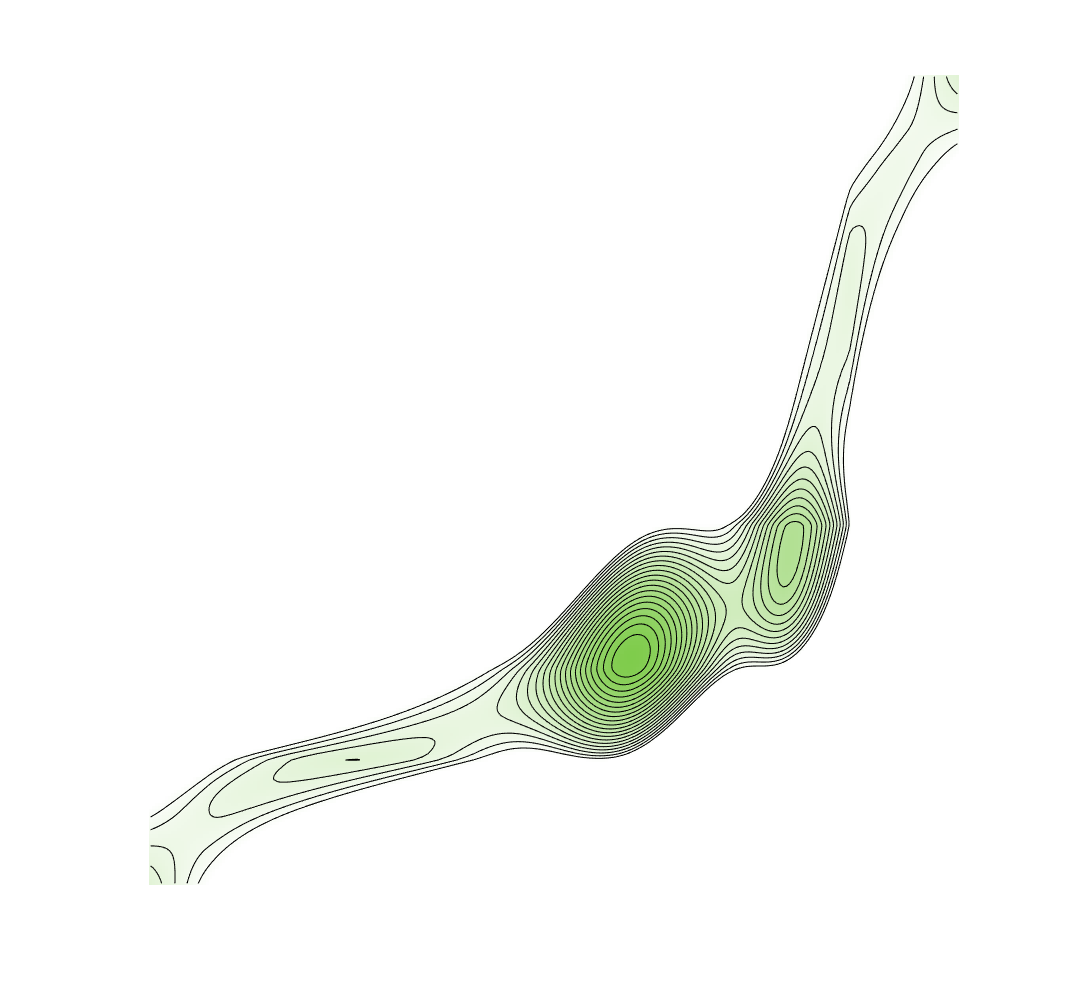} &
\myfigRegQuad{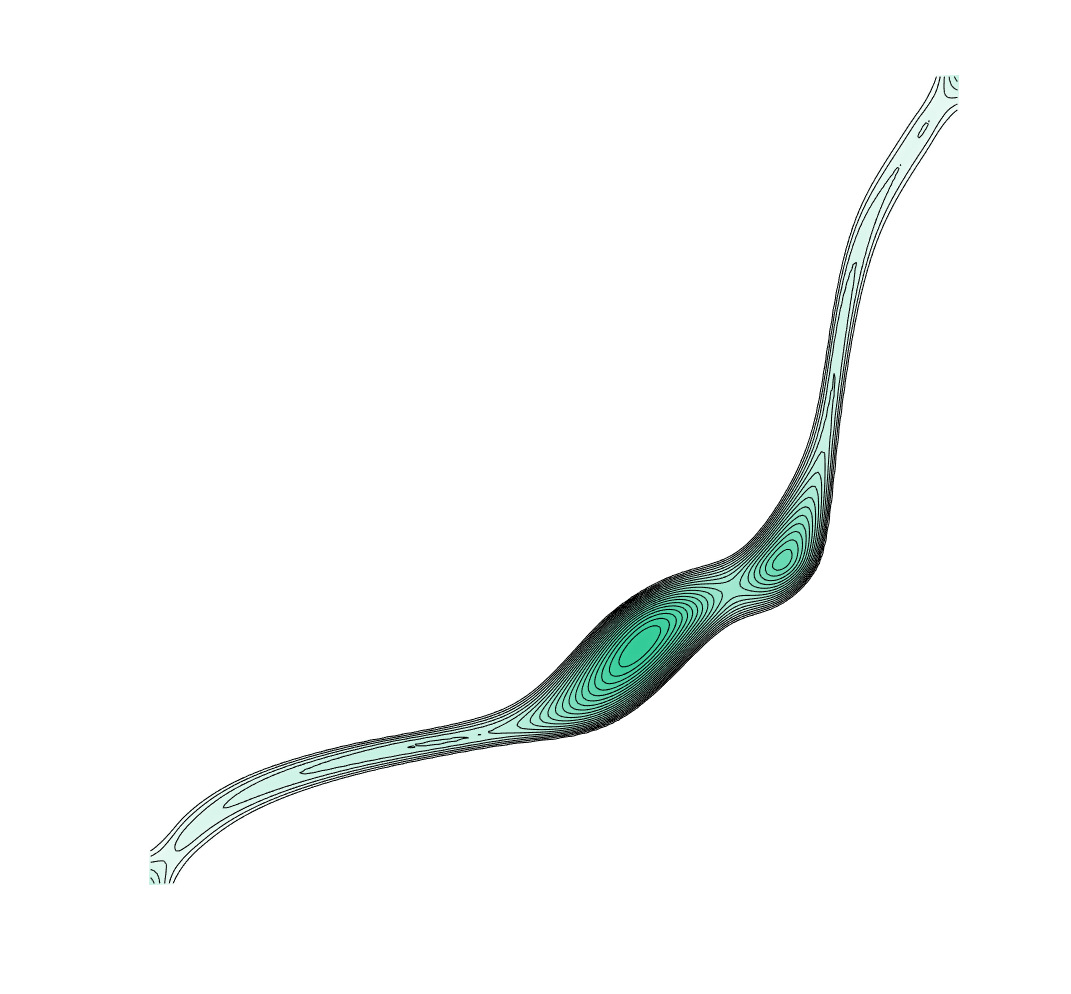} &
\myfigRegQuad{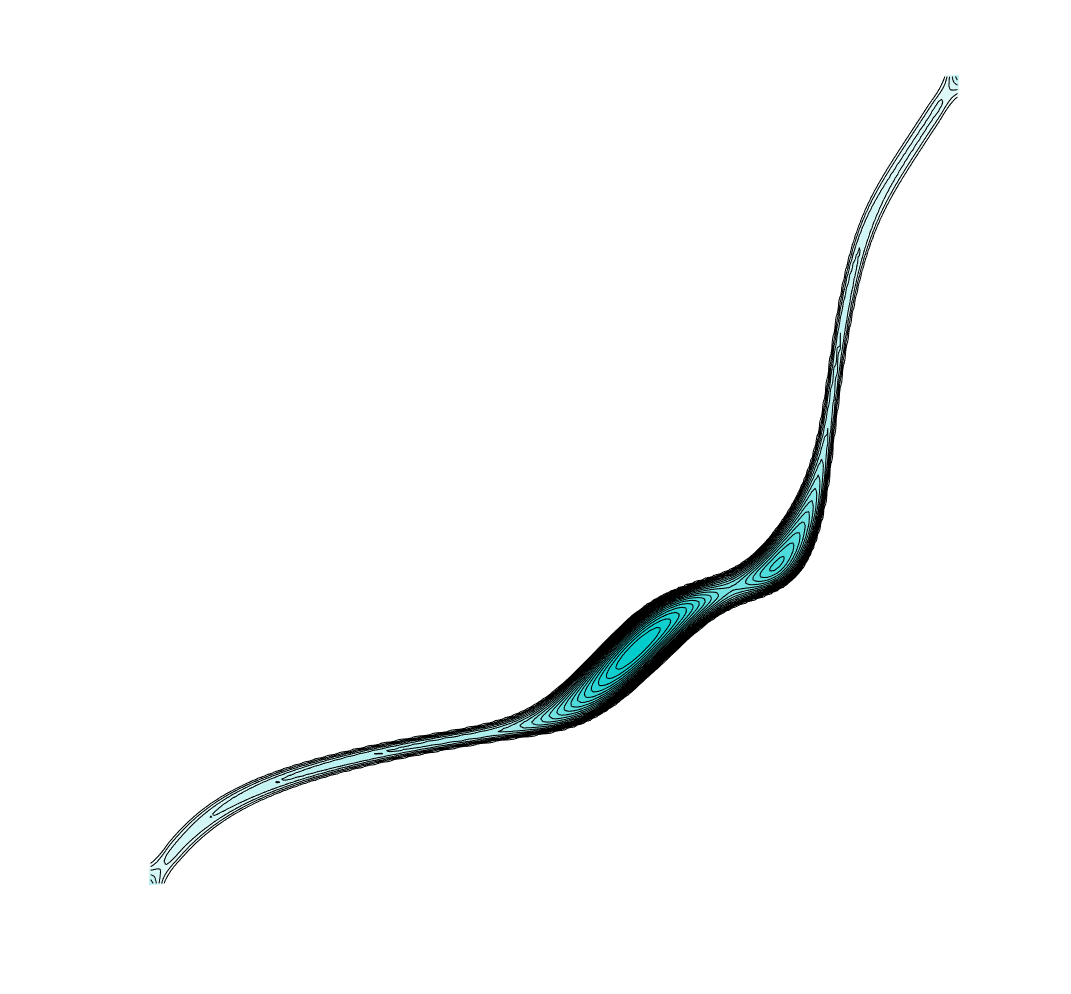} \\
{\color[rgb]{.6,.6,0} $\varepsilon=5 \cdot 10^3$} &
{\color[rgb]{.45,.6,.15}$\varepsilon=10^3$} &
{\color[rgb]{.3,.6,.3} $\varepsilon=10^2$} &
{\color[rgb]{.15,.6,.45} $\varepsilon=10$} &
{\color[rgb]{0,.6,.6} $\varepsilon=1$} 
\end{tabular}
\caption{\label{fig-quad-regul}
Comparison of entropic regularization $R=-\HD$ (top row) and quadratic regularization $R=\norm{\cdot}^2+\iota_{\RR_+}$ (bottom row). 
The $(\al,\be)$ marginals are the same as for Figure~\ref{fig-entropic-copula}.
}
\end{figure}

\begin{rem1}{Barycentric projection}\label{rem-barycenric-proj}
	Consider again the setting of Remark~\ref{rem:entrop-reg} in which we use entropic regularization to approximate OT between discrete measures. The Kantorovich formulation in \eqref{eq-mk-discr} and its entropic regularization~\eqref{eq-regularized-discr} both yield a coupling $\P \in \CouplingsD(\a,\b)$. In order to define a transportation map $\T : \X \rightarrow \Y$, in the case where $\Y=\RR^\dim$, one can define the so-called barycentric projection map
	\eql{\label{eq-baryproj}
		\T : x_i \in \X \longmapsto \frac{1}{\a_i} \sum_j \P_{i,j} y_j \in \Y,
	}
	where the input measures are discrete of the form~\eqref{eq-pair-discr}. Note that this map is only defined for points $(x_i)_i$ in the support of $\al$. 
	In the case where $\T$ is a permutation matrix (as detailed in Proposition~\ref{prop-matching-kanto}), then $\T$ is equal to a Monge map, and as $\varepsilon \rightarrow 0$, the barycentric projection progressively converges to that map if it is unique.
	For arbitrary (not necessarily discrete) measures, solving~\eqref{eq-mk-generic} or its regularized version~\eqref{eq-entropic-generic} defines a coupling $\pi \in \Couplings(\al,\be)$. Note that this coupling $\pi$ always has a density $\frac{\d\pi(x,y)}{\d\al(x)\d\be(y)}$ with respect to $\al \otimes \be$. A map can thus be retrieved by the formula
	\eql{\label{eq-bary-proj}
		\T : x \in \X \longmapsto \int_\Yy y \frac{\d\pi(x,y)}{\d\al(x)\d\be(y)} \d\be(y).
	}
	In the case where, for $\varepsilon=0$, $\pi$ is supported on the graph of the Monge map (see Remark~\ref{rem-exist-mongemap}), then using $\varepsilon>0$ produces a smooth approximation of this map.
	Such a barycentric projection is useful to apply the OT Monge map to solve problems in imaging; see Figure~\ref{fig-colors} for an application to color modification. It has also been used to compute approximations of principal geodesics in the space of probability measures endowed with the Wasserstein metric; see~\citep{SeguyCuturi}.
\end{rem1}

\begin{rem1}{Hilbert metric}
As initially explained by~\citep{franklin1989scaling}, the global convergence analysis of Sinkhorn is greatly simplified using the Hilbert projective metric on $\RR_{+,*}^n$ (positive vectors), defined as
\eq{
	\foralls (\uD,\uD') \in (\RR_{+,*}^n)^2, \quad
	\Hilbert(\uD,\uD') \eqdef \log \umax{i,j} \frac{ \uD_i \uD_{j}' }{ \uD_{j} \uD_{i}'  }.
}
It can be shown to be a distance on the projective cone $\RR_{+,*}^n/\sim$, where $\uD \sim \uD'$ means that $\exists r>0, \uD=r\uD'$ (the vectors are equal up to rescaling, hence the name ``projective'').  
This means that $\Hilbert$ satisfies the triangular inequality and $\Hilbert(\uD,\uD')=0$ if and only if $\uD \sim \uD'$. 
This is a projective version of Hilbert's original distance on bounded open convex sets~\citep{hilbert1895gerade}.
The projective cone $\RR_{+,*}^n/\sim$ is a complete metric space for this distance. 
By a logarithmic change of variables, the Hilbert metric on the rays of the positive cone is isometric to the variation seminorm (it is a norm between vectors that are defined up to an additive constant)
\eql{\label{eq-hilbert-var}
	\Hilbert(\uD,\uD') = \norm{\log(\uD)-\log(\uD')}_{\text{var}}
}
\eq{
	\qwhereq
	\norm{\fD}_{\text{var}} \eqdef (\max_i \fD_i) - (\min_i \fD_i).
}
This variation seminorm is closely related to the $\ell^\infty$ norm since one always has $\norm{\fD}_{\text{var}}  \leq 2 \norm{\fD}_\infty$. If one imposes that $\fD_i=0$ for some fixed $i$, then a converse inequality also holds since $\norm{\fD}_\infty \leq \norm{\fD}_{\text{var}}$. These bounds are especially useful to analyze Sinkhorn convergence (see Remark~\ref{rem-global-conv-sinkh} below), because dual variables $\fD=\log(\uD)$ solving~\eqref{eq-dualsinkhorn-constraints2} are defined up to an additive constant, so that one can impose that $\fD_i=0$ for some $i$.
The Hilbert metric was introduced independently by~\citep{birkhoff1957extensions} and~\citep{samelson1957perron}. They proved the following fundamental theorem, which shows that a positive matrix is a strict contraction on the cone of positive vectors.

\begin{thm}\label{thm-birkoff}
	Let $\K \in \RR_{+,*}^{n \times m}$; then for $(\vD,\vD') \in (\RR_{+,*}^m)^2$
	\eq{
		\Hilbert(\K \vD,\K \vD') \leq \la(\K) \Hilbert(\vD,\vD'),
		\text{ where }
		\choice{
			\la(\K) \eqdef \frac{ \sqrt{\eta(\K)}-1 }{ \sqrt{\eta(\K)}+1 } < 1, \\
			\eta(\K) \eqdef \umax{i,j,k,\ell} \frac{ \K_{i,k} \K_{j,\ell} }{ \K_{j,k} \K_{i,\ell} }.
		}
	}
\end{thm}

Figure~\ref{fig-h-m} illustrates this theorem. 
\end{rem1}

\begin{figure}[h!]
\centering
\includegraphics[width=.22\linewidth]{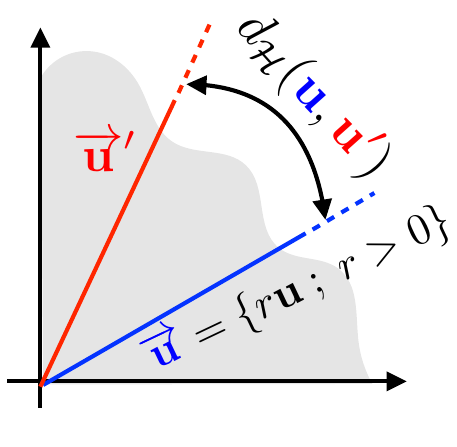} \qquad
\includegraphics[width=.7\linewidth]{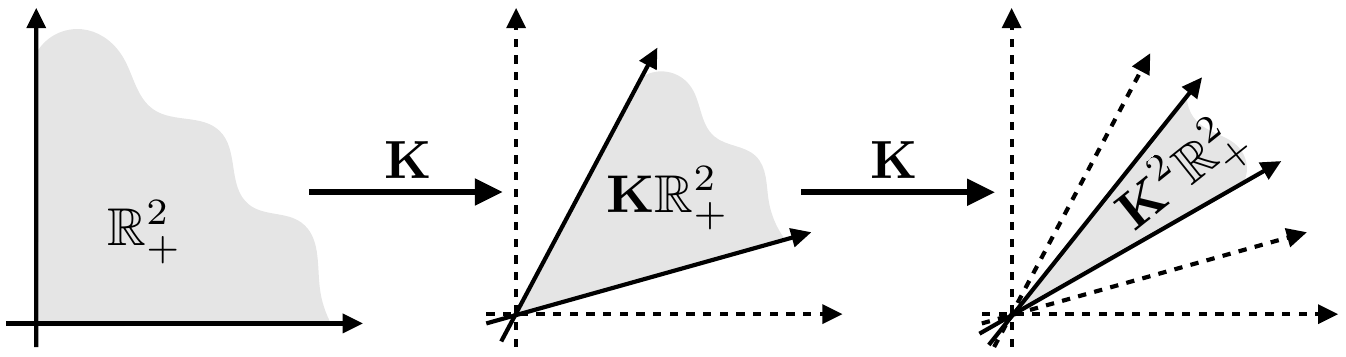}\\
\caption{\label{fig-h-m}
Left: the Hilbert metric $\Hilbert$ is a distance over rays in cones (here positive vectors).
Right: visualization of the contraction induced by the iteration of a positive matrix $\K$. 
}
\end{figure}

\begin{rem1}{Perron--Frobenius}
A typical application of Theorem~\ref{thm-birkoff} is to provide a quantitative proof of the Perron--Frobenius theorem, which, as explained in Remark~\ref{rem-local-conv}, is linked to a local linearization of Sinkhorn's iterates. A matrix $\K \in \RR_+^{n \times n}$ with $\K^\top \ones_n = \ones_n$ maps $\simplex_n$ into $\simplex_n$. If furthermore $\K>0$, then according to Theorem~\ref{thm-birkoff}, it is strictly contractant for the metric $\Hilbert$, hence there exists a unique invariant probability distribution $p^\star \in \simplex_n$ with $\K p^\star=p^\star$. Furthermore, for any $p_0 \in \simplex_n$, $\Hilbert(\K^\ell p_0,p^\star) \leq \la(\K)^\ell \Hilbert(p_0,p^\star)$, \ie one has linear convergence of the iterates of the matrix toward $p^\star$. This is illustrated in Figure~\ref{fig-perron}.
\end{rem1}

\begin{figure}[h!]
\centering
\begin{tabular}{@{}c@{}c@{}c@{\hspace{10mm}}c@{}}
\includegraphics[width=.22\linewidth]{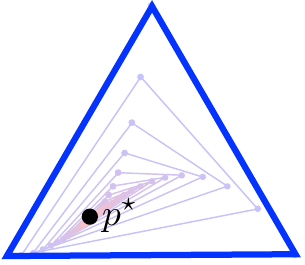}&
\includegraphics[width=.22\linewidth]{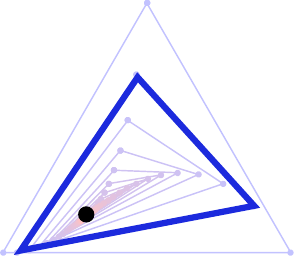}&
\includegraphics[width=.22\linewidth]{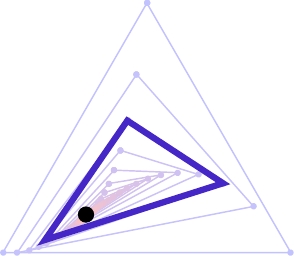}&
\includegraphics[width=.22\linewidth]{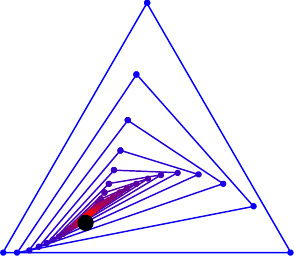}\\
$\simplex_3$ & $\K \simplex_3$ & $\K^2 \simplex_3$ &  $\{ \K^\ell \simplex_3 \}_\ell$
\end{tabular}
\caption{\label{fig-perron}
Evolution of $\K^\ell \simplex_3 \rightarrow \{p^\star\}$ the invariant probability distribution of $\K \in \RR_{+,*}^{3 \times 3}$ with $\K^\top \ones_3=\ones_3$.
}
\end{figure}

\begin{rem11}{Global convergence}{rem-global-conv-sinkh}
The following theorem, proved by~\citep{franklin1989scaling}, makes use of Theorem~\ref{thm-birkoff} to show the linear convergence of Sinkhorn's iterations.

\begin{thm}
	One has $(\it{\uD},\it{\vD}) \rightarrow (\uD^\star,\vD^\star)$ and
	\eql{\label{eq-convlin-sinkh}
		\Hilbert(\it{\uD}, \uD^\star) = O(\la(\K)^{2\ell}), \quad
		\Hilbert(\it{\vD}, \vD^\star) = O(\la(\K)^{2\ell}).
	}
	One also has
	\eql{\label{eq-convsinkh-control}
		\begin{split}
		\Hilbert(\it{\uD}, \uD^\star) &\leq \frac{\Hilbert( \it{\P}\ones_m,\a )}{1-\la(\K)^2}, \\
		\Hilbert(\it{\vD}, \vD^\star) &\leq \frac{\Hilbert( \P^{(\ell),\top} \ones_n,\b )}{1-\la(\K)^2},
		\end{split}
	}
	where we denoted $\it{\P} \eqdef \diag(\it{\uD}) \K \diag(\it{\vD})$. Last, one has
	\eql{\label{eq-convlin-sinkh-prim}
		\|\log(\it{\P}) - \log(\P^\star)\|_\infty \leq \Hilbert(\it{\uD}, \uD^\star) + \Hilbert(\it{\vD}, \vD^\star),
	}
	where $\P^\star$ is the unique solution of~\eqref{eq-regularized-discr}. 
\end{thm}

\begin{proof}
	One notices that for any $(\vD,\vD') \in (\RR_{+,*}^m)^2$, one has 
	\eq{	
		\Hilbert(\vD,\vD') = \Hilbert(\vD/\vD',\ones_m) = \Hilbert(\ones_m/\vD,\ones_m/\vD').
	}
	This shows that
	\begin{align*}
		\Hilbert(\itt{\uD},\uD^\star) &= \Hilbert\pa{ \frac{\a}{\K \it{\vD}}, \frac{\a}{\K \vD^\star} } \\
		&= \Hilbert( \K \it{\vD}, \K \vD^\star ) \leq \la(\K) \Hilbert( \it{\vD}, \vD^\star ),
	\end{align*}
	where we used Theorem~\ref{thm-birkoff}. This shows~\eqref{eq-convlin-sinkh}.  One also has, using the triangular inequality,
	\begin{align*}
		\Hilbert(\it{\uD},\uD^\star) &\leq \Hilbert(\itt{\uD},\it{\uD}) + \Hilbert(\itt{\uD},\uD^\star) \\
		&\leq \Hilbert\pa{ \frac{\a}{\K \it{\vD}},\it{\uD} } + \la(\K)^2 \Hilbert(\it{\uD},\uD^\star) \\
		&= \Hilbert\pa{ \a,\it{\uD} \odot  ( \K \it{\vD} ) } + \la(\K)^2 \Hilbert(\it{\uD},\uD^\star), 
	\end{align*}
	which gives the first part of~\eqref{eq-convsinkh-control} since 
	$\it{\uD} \odot  ( \K \it{\vD} ) = \it{\P}\ones_m$ (the second one being similar).
	The proof of~\eqref{eq-convlin-sinkh-prim} follows from~\citep[Lem. 3]{franklin1989scaling}.
\end{proof}
 
The bound~\eqref{eq-convsinkh-control} shows that some error measures on the marginal constraints violation, for instance, $\| \it{\P} \ones_m - \a \|_1$ and $\|\transp{\it{\P}} \ones_n - \b \|_1$, are useful stopping criteria to monitor the convergence.
Note that thanks to~\eqref{eq-hilbert-var}, these Hilbert metric rates on the scaling variable $(\it{\uD},\it{\vD})$ give a linear rate on the dual variables $(\it{\fD},\it{\gD}) \eqdef (\varepsilon\log(\it{\uD}),\varepsilon\log(\it{\vD}))$ for the variation norm $\norm{\cdot}_{\text{var}}$.

Figure~\ref{fig-sinkhorn-convergence}, bottom row, highlights this linear rate on the constraint violation and shows how this rate degrades as $\varepsilon\rightarrow 0$. 
These results are proved in~\citep{franklin1989scaling} and are tightly connected to nonlinear Perron--Frobenius theory~\citep{lemmens2012nonlinear}. Perron--Frobenius theory corresponds to the linearization of the iterations; see~\eqref{eq-linearized-sinkh}. This convergence analysis is extended by~\citep{linial1998deterministic}, who show that each iteration of Sinkhorn increases the permanence of the scaled coupling matrix. 
\end{rem11}

\begin{rem1}{Local convergence}\label{rem-local-conv}
The global linear rate~\eqref{eq-convlin-sinkh-prim} is often quite pessimistic, typically in $\X=\Y=\RR^d$ for cases where there exists a Monge map when $\varepsilon=0$ (see Remark~\ref{rem-monge2}). The global rate is in contrast rather sharp for more difficult situations where the cost matrix $\C$ is close to being random, and in these cases, the rate scales exponentially bad with $\varepsilon$, $1-\la(\K) \sim e^{-1/\varepsilon}$.
To obtain a finer asymptotic analysis of the convergence (\eg if one is interested in a high-precision solution and performs a large number of iterations), one usually rather studies the local convergence rate. 
One can write a Sinkhorn update as iterations of a fixed-point map $\itt{\fD} = \Phi( \it{\fD} )$, where
\eq{
	\Phi \eqdef \Phi_2 \odot \Phi_1
	\qwhereq
	\choice{
		\Phi_1(\fD) = \varepsilon \log \transp{\K} ( e^{\fD/\varepsilon} ) - \log(\b), \\
		\Phi_2(\gD) = \varepsilon \log \K ( e^{\gD/\varepsilon} ) - \log(\a).
	}
}
For optimal $(\fD,\gD)$ solving~\eqref{eq-dual-formulation}, denoting $\P = \diag(e^{\fD/\varepsilon}) \K \diag(e^{\gD/\varepsilon})$ the optimal coupling solving~\eqref{eq-regularized-discr}, one has the following Jacobian:
\eql{\label{eq-linearized-sinkh}
	\partial \Phi(\fD) = \diag(\a)^{-1} \odot \P \odot \diag(\b)^{-1} \odot \transp{\P}.
} 
This Jacobian is a positive matrix with $\partial \Phi(\fD) \ones_n = \ones_n$, and thus by the Perron--Frobenius theorem, it has a single dominant eigenvector $\ones_m$ with associated eigenvalue $1$. Since $\fD$ is defined up to a constant, it is actually the second eigenvalue $1-\kappa<1$ which governs the local linear rate, and this shows that for $\ell$ large enough, 
\eq{
	\|\it{\fD}-\fD \| = O( (1-\kappa)^\ell ).
} 
Numerically, in ``simple cases'' (such as when there exists a smooth Monge map when $\varepsilon=0$), this rate scales like $\kappa \sim \varepsilon$. 
We refer to~\citep{knight2008sinkhorn} for more details in the bistochastic (assignment) case.
\end{rem1}

\section{Speeding Up Sinkhorn's Iterations}

The main computational bottleneck of Sinkhorn's iterations is the vector-matrix multiplication against kernels $\K$ and $\K^\top$, with complexity $O(nm)$ if implemented naively. We now detail several important cases where the complexity can be improved significantly.

\begin{rem}[Parallel and GPU friendly computation]\label{rem-parallel} The simplicity of Sinkhorn's algorithm yields an extremely efficient approach to compute simultaneously several regularized Wasserstein distances between pairs of histograms. 
Let $N$ be an integer, $\a_1,\ldots,\a_N$ be histograms in $\simplex_n$, and $\b_1,\ldots,\b_N$ be histograms in $\Sigma_m$. 
We seek to compute all $N$ approximate distances $\MKD_\C^\varepsilon(\a_1,\b_1),\dots,\MKD_\C^\varepsilon(\a_N,\b_N)$. 
In that case, writing $\A=[\a_1,\dots,\a_N]$ and $\B=[\b_1,\dots,\b_N]$ for the $n\times N$ and $m\times N$ matrices storing all histograms, one can notice that all Sinkhorn iterations for all these $N$ pairs can be carried out in parallel, by setting, for instance,
\eql{\label{eq-sinkhorn-par}	
	\itt{\UD} \eqdef \frac{\A}{\K \it{\VD}}
	\qandq
	\itt{\VD} \eqdef \frac{\B}{\transp{\K}\itt{\UD}},
}
initialized with $\init{\VD} = \ones_{m \times N}$. Here $\frac{\cdot}{\cdot}$ corresponds to the entrywise division of matrices.
One can further check that upon convergence of $\VD$ and $\UD$, the (row) vector of regularized distances simplifies to
$$ \transp{\ones_n} (    \UD\odot \log \UD \odot (   (\K\odot \C) \VD) +  \UD\odot(   (\K\odot \C)(\VD\odot \log \VD)    ) ) \in\RR^{N}.$$
Note that the basic Sinkhorn iterations described in Equation~\eqref{eq-sinkhorn} are intrinsically GPU friendly, since they only consist in matrix-vector products, and this was exploited, for instance, to solve matching problems in~\citet{slomp2011gpu}). However, the matrix-matrix operations presented in Equation~\eqref{eq-sinkhorn-par} present even better opportunities for parallelism, which explains the success of Sinkhorn's algorithm to compute OT distances between histograms at large scale.
\end{rem}

\begin{rem}[Speed-up for separable kernels]\label{rem-separable} We consider in this section an important particular case for which the complexity of each Sinkhorn iteration can be significantly reduced. That particular case happens when each index $i$ and $j$ considered in the cost-matrix can be described as a $d$-uple taken in the cartesian product of $d$ finite sets $\range{n_1},\dots,\range{n_d}$,
$$i=(i_k)_{k=1}^d, j=(j_k)_{k=1}^d \in \range{n_1}\times\dots\times\range{n_d}.$$
In that setting, if the cost $\C_{ij}$ between indices $i$ and $j$ is additive along these sub-indices, namely if there exists $d$ matrices $\C^1,\dots,\C^d$, each of respective size $n_1\times n_,\dots,n_d\times n_d$, such that
$$\C_{ij} = \sum_{k=1}^d \C^k_{i_k,j_k},$$
then one obtains as a direct consequence that the kernel appearing in the Sinkhorn iterations has a separable multiplicative structure,
\eql{\label{eq-speedup-separable}
\K_{i,j} = \prod_{k=1}^d \K^k_{i_k,j_k}.
}
Such a separable multiplicative structure allows for a very fast (exact) evaluation of $\K \uD$. Indeed, instead of instantiating $\K$ as a matrix of size $n\times n$, which would have a prohibitive size since $n=\prod_k n_k$ is usually exponential in the dimension $d$, one can instead recover $\K\uD$ by simply applying $\K^k$ along each ``slice'' of $\uD$. If $n=m$, the complexity reduces to $O(n^{1+1/d})$ in place of $O(n^2)$.

An important example of this speed-up arises when $\Xx=\Yy=[0,1]^d$; the ground cost is the $q$-th power of the $q$-norm, $$\c(x,y)=\norm{x-y}^q_q=\sum_{i=1}^d |x_i-y_i|^q, \; q>0;$$ and the space is discretized using a regular grid in which only points $x_i = (i_1/n_1,\ldots,i_d/n_d)$ for $i=(i_1,\dots,i_d)\in\range{n_1}\times \dots\times\range{n_d}$ are considered. In that case a multiplication by $\K$ can be carried out more efficiently by applying each 1-D $n_k\times n_k$ convolution matrix $$\K^k = \begin{bmatrix}\exp(-\left|\frac{r-s}{n_k}\right|^q/ \varepsilon)\end{bmatrix}_{1 \leq r,s\leq n_k}$$ to $\uD$ reshaped as a tensor whose first dimension has been permuted to match the $k$-th set of indices.
For instance, if $d=2$ (planar case) and $q=2$ (2-Wasserstein, resulting in Gaussian convolutions), histograms $\a$ and as a consequence Sinkhorn multipliers $\uD$ can be instantiated as $n_1\times n_2$ matrices. We write $\mathbf{U}$ to underline the fact that the multiplier $\uD$ is reshaped as a $n_1\times n_2$ matrix, rather than a vector of length $n_1n_2$. Then, computing $\K\uD$, which would naively require $(n_1 n_2)^2$ operations with a naive implementation, can be obtained by applying two 1-D convolutions separately, as $$(\K^2(\K^1 \mathbf{U})^T)^T=\K^1\mathbf{U}\K^2,$$ to recover a $n_1\times n_2$ matrix in $(n_1^2)n_2+n_1(n_2^2)$ operations instead of $n_1^2n_2^2$ operations. Note that this example agrees with the exponent $(1+1/d)$ given above. With larger $d$, one needs to apply these very same 1-D convolutions to each slice of $\uD$ (reshaped as a tensor of suitable size) an operation which is extremely efficient on GPUs.

This important observations underlies many of the practical successes found when applying optimal transport to shape data in 2-D and 3-D, as highlighted in~\citep{2015-solomon-siggraph,2016-bonneel-barycoord}, in which distributions supported on grids of sizes as large as $200^3=8 \times 10^6$ are handled.
\end{rem}

\begin{rem}[Approximated convolutions]\label{rem-convol-sinkh}
The main computational bottleneck of Sinkhorn's iterations~\eqref{eq-sinkhorn} lies in the multiplication of a vector by $\K$ or by its adjoint. Besides using separability~\eqref{eq-speedup-separable}, it is also possible to exploit other special structures in the kernel. The simplest case is for translation invariant kernels $\K_{i,j} = k_{i-j}$, which is typically the case when discretizing the measure on a fixed uniform grid in Euclidean space $\X=\RR^\dim$. Then $\K \vD = k \star \vD$ is a convolution, and there are several algorithms to approximate the convolution in nearly linear time. 
The most usual one is by Fourier transform $\Ff$, assuming for simplicity periodic boundary conditions, because $\Ff(k \star \vD)=\Ff(k) \odot \Ff(\vD)$. This leads, however, to unstable computations and is often unacceptable for small $\varepsilon$.
Another popular way to speed up computation is by approximating the convolution using a succession of autoregressive filters, using, for instance, the Deriche filtering method~\citet{deriche1993recursively}. We refer to~\citep{getreuer2013survey} for a comparison of various fast filtering methods.
\end{rem}

\begin{rem}[Geodesic in heat approximation]\label{rem-geod-heat}
For nonplanar domains, the kernel $\K$ is not a convolution, but in the case where the cost is $\C_{i,j}=d_\Mm(x_i,y_j)^p$ where $d_\Mm$ is a geodesic distance on a surface $\Mm$ (or a more general manifold), it is also possible to perform fast approximations of the application of $\K=e^{-\frac{d_\Mm}{\varepsilon}}$ to a vector. Indeed, \citeauthor{varadhan-1967}'s formulas~\citeyearpar{varadhan-1967} assert that this kernel is close to the Laplacian kernel (for $p=1$) and the heat kernel (for $p=2$). 
The first formula of Varadhan states
\eql{\label{eq-varadhan-1}
	-\frac{\sqrt{t}}{2} \log( \Pp_t(x,y) ) = d_\Mm(x,y) + o(t)
	\qwhereq
	\Pp_t \eqdef (\Id - t\De_{\Mm})^{-1},
}
where $\De_{\Mm}$ is the Laplace--Beltrami operator associated to the manifold $\Mm$ (which is negative semidefinite), so that 
$\Pp_t$ is an integral kernel and $g = \int_\Mm \Pp_t(x,y) f(y) \d y$ is the solution of $g - t \De_{\Mm} g = f$. 
The second formula of Varadhan states
\eql{\label{eq-varadhan-2}
	\sqrt{ -4 t \log( \Hh_t(x,y) ) } = d_\Mm(x,y) + o(t),
}
where  $\Hh_t$ is the integral kernel defined so that $g_t = \int_\Mm \Hh_t(x,y) f(y) \d y$ is the solution at time $t$ of the heat equation
\eq{
	\frac{\partial g_t(x)}{\partial t} = (\Delta_\Mm g_t)(x).
}
The convergence in these formulas~\eqref{eq-varadhan-1} and~\eqref{eq-varadhan-2} is uniform on compact manifolds.
Numerically, the domain $\Mm$ is discretized (for instance, using finite elements) and $\De_\Mm$ is approximated by a discrete Laplacian matrix $L$. A typical example is when using piecewise linear finite elements, so that $L$ is the celebrated cotangent Laplacian (see~\citep{botsch-2010} for a detailed account for this construction).
These formulas can be used to approximate efficiently the multiplication by the Gibbs kernel $\K_{i,j}=e^{-\frac{d(x_i,y_j)^p}{\varepsilon}}$.
Equation~\eqref{eq-varadhan-1} suggests, for the case $p=1$, to use $\varepsilon=\frac{\sqrt{t}}{2}$ and to replace the multiplication by $\K$ by the multiplication by $(\Id-t L)^{-1}$, which necessitates the resolution of a positive symmetric linear system.
Equation~\eqref{eq-varadhan-2}, coupled with $R$ steps of implicit Euler for the stable resolution of the heat flow, suggests for $p=2$ to trade the multiplication by $\K$ by the multiplication by $(\Id-\frac{t}{R} L)^{-R}$ for $4 t=\varepsilon$, which in turn necessitates $R$ resolutions of linear systems. Fortunately, since these linear systems are supposed to be solved at each Sinkhorn iteration, one can solve them efficiently by precomputing a sparse Cholesky factorization. By performing a reordering of the rows and columns of the matrix~\citep{george1989evolution}, one obtains a nearly linear sparsity for 2-D manifolds and thus each Sinkhorn iteration has linear complexity (the performance degrades with the dimension of the manifold). 
The use of Varadhan's formula to approximate geodesic distances was initially proposed in~\citep{Crane2013} and its use in conjunction with Sinkhorn iterations in~\citep{2015-solomon-siggraph}.
\todoK{Maybe add some pedagological figure illustrating the method.}
\end{rem}

\begin{rem}[Extrapolation acceleration]
	Since the Sinkhorn algorithm is a fixed-point algorithm (as shown in Remark~\ref{rem-local-conv}), one can use standard linear or even nonlinear extrapolation schemes to enhance the conditioning of the fixed-point mapping near the solution, and improve the linear convergence rate. 
	This is similar to the successive overrelaxation  method (see, for instance,~\citep{hadjidimos2000successive}), so that the local linear rate of convergence is improved from $O((1-\kappa)^\ell)$ to $O((1-\sqrt{\kappa})^\ell)$ for some $\kappa>0$ (see Remark~\ref{rem-local-conv}). 
	We refer to~\citep{2016-peyre-qot} for more details\todoK{citer papier recent Lenaic}.
\end{rem}

\section{Stability and Log-Domain Computations}

As briefly mentioned in Remark~\ref{rem-stability}, the Sinkhorn algorithm suffers from numerical overflow when the regularization parameter $\epsilon$ is small compared to the entries of the cost matrix $\C$. This concern can be alleviated to some extent by carrying out computations in the log domain. The relevance of this approach is made more clear by considering the dual problem associated to~\eqref{eq-regularized-discr}, in which these log-domain computations arise naturally.

\begin{prop}
One has
\eql{\label{eq-dual-formulation}
	\MKD_\C^\varepsilon(\a,\b) = \umax{\fD \in \RR^n,\gD \in \RR^m}
		 \dotp{\fD}{\a} + \dotp{\gD}{\b} 
		- \varepsilon \dotp{e^{\fD/\varepsilon} }{ \K e^{\gD/\varepsilon}}.
} 
The optimal $(\fD,\gD)$ are linked to scalings $(\uD,\vD)$ appearing in~\eqref{eq-scaling-form} through 
\eql{\label{eq-entropy-pd}
	(\uD,\vD)=(e^{\fD/\varepsilon},e^{\gD/\varepsilon}).
}
\end{prop}

\begin{proof}
We start from the end of the proof of Proposition~\ref{prop-regularized-primal}, which links the optimal primal solution $\P$ and dual multipliers $\fD$ and $\gD$ for the marginal constraints as 
\eq{
	\label{eq-prim-dual-ent}\P_{i,j}=e^{\fD_i/\varepsilon}e^{-\C_{i,j}/\varepsilon}e^{\gD_j/\varepsilon}.
} 
Substituting in the Lagrangian $\Lag(\P,\fD,\gD)$ of Equation~\eqref{eq-sinkhorn-lagrangian} the optimal $\P$ as a function of $\fD$ and $\gD$, we obtain that the Lagrange dual function equals
\eql{\label{eq-lagrange-dual-naive}
	\fD,\gD \mapsto \dotp{e^{\fD/\varepsilon}}{\left(\K\odot \C\right)e^{\gD/\varepsilon}} - \varepsilon \HD(\diag(e^{\fD/\varepsilon}) \K \diag(e^{\gD/\varepsilon})).
}
The neg-entropy of $\P$ scaled by $\varepsilon$, namely $\varepsilon \dotp{\P}{\log \P - \ones_{n\times m}}$, can be stated explicitly as a function of $\fD,\gD, \C$,
\begin{align*}
&\dotp{\diag(e^{\fD/\varepsilon}) \K \diag(e^{\gD/\varepsilon})}{\fD\transp{\ones_m}+\ones_n\transp{\gD}-\C-\varepsilon\ones_{n\times m}}\\
&= -\dotp{e^{\fD/\varepsilon}}{\left(\K\odot \C\right)e^{\gD/\varepsilon}} + \dotp{\fD}{\a}+ \dotp{\gD}{\b}-\varepsilon \dotp{e^{\fD/\varepsilon}}{\K e^{\gD/\varepsilon}};
\end{align*}
therefore, the first term in~\eqref{eq-lagrange-dual-naive} cancels out with the first term in the entropy above. The remaining terms are those appearing in~\eqref{eq-dual-formulation}.
\end{proof}

\begin{rem}[Sinkhorn as a block coordinate ascent on the dual problem]\label{rem-sinkh-block}
A simple approach to solving the unconstrained maximization problem~\eqref{eq-dual-formulation} is to use an exact \emph{block coordinate ascent} strategy, namely to update alternatively $\fD$ and $\gD$ to cancel the respective gradients in these variables of the objective of \eqref{eq-dual-formulation}. Indeed, one can notice after a few elementary computations that, writing $Q(\fD,\gD)$ for the objective of~\eqref{eq-dual-formulation},
\begin{align}
	\label{eq-dualupadate-sinkh-1}\nabla|_\fD\, Q(\fD,\gD) &=  \a - e^{\fD/\varepsilon}\odot \left(\K e^{\gD/\varepsilon}\right),\\
	\label{eq-dualupadate-sinkh-2}\nabla|_\gD\, Q(\fD,\gD) &=  \b - e^{\gD/\varepsilon}\odot \left(\transp{\K} e^{\fD/\varepsilon}\right).
\end{align}
Block coordinate ascent can therefore be implemented in a closed form by applying successively the following updates, starting from any arbitrary $\init{\gD}$, for $l\geq 0$:
\begin{align}
	\label{eq-slse-sinkh-1}\itt{\fD} &= \varepsilon \log \a -\varepsilon\log\left(\K e^{\it{\gD}/\varepsilon}\right), \\
	\label{eq-slse-sinkh-2}\itt{\gD} &= \varepsilon \log \b- \varepsilon\log\left(\transp{\K} e^{\itt{\fD}/\varepsilon}\right).
\end{align}
Such iterations are mathematically equivalent to the Sinkhorn iterations~\eqref{eq-sinkhorn} when considering the primal-dual relations highlighted in~\eqref{eq-entropy-pd}. Indeed, we recover that at any iteration
\eq{
	(\it{\fD},\it{\gD}) = \varepsilon ( \log(\it{\uD}), \log(\it{\vD}) ). 
}
\end{rem}

\begin{rem}[Soft-min rewriting]
Iterations~\eqref{eq-slse-sinkh-1} and~\eqref{eq-slse-sinkh-2} can be given an alternative interpretation, using the following notation. Given a vector $\z$ of real numbers we write $\smine\z$ for the \emph{soft-minimum} of its coordinates, namely
\eql{\label{eq-defn-softmin}
	\smine\z= -\varepsilon \log \sum_i e^{-\z_i/\varepsilon}.
}
Note that $\smine(\z)$ converges to $\min\,\z$ for any vector $\z$ as $\varepsilon\rightarrow 0$. Indeed, $\smine$ can be interpreted as a differentiable approximation of the $\min$ function, as shown in Figure~\ref{fig-softmin}. 

\begin{figure}[h!]
\centering
\begin{tabular}{@{}c@{\hspace{1mm}}c@{\hspace{1mm}}c@{\hspace{1mm}}c@{\hspace{1mm}}c@{}}
\includegraphics[width=.19\linewidth]{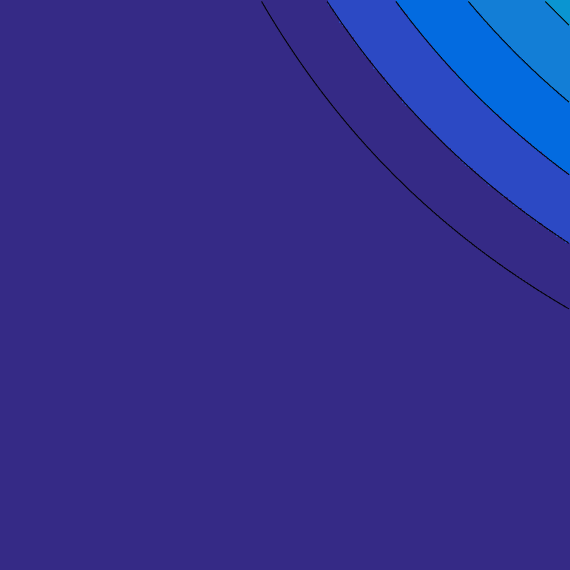}&
\includegraphics[width=.19\linewidth]{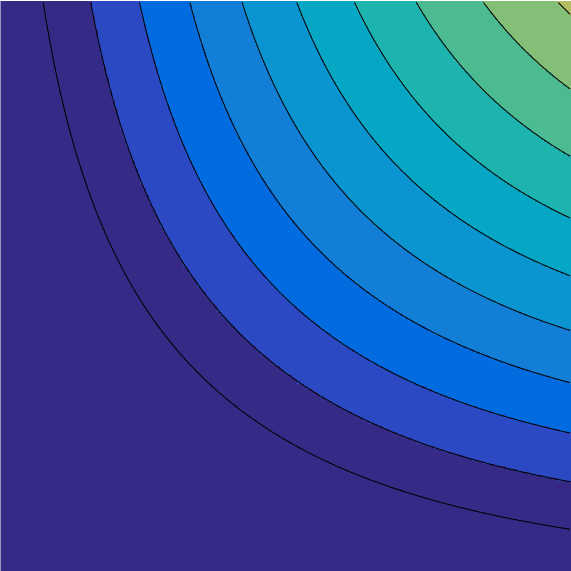}&
\includegraphics[width=.19\linewidth]{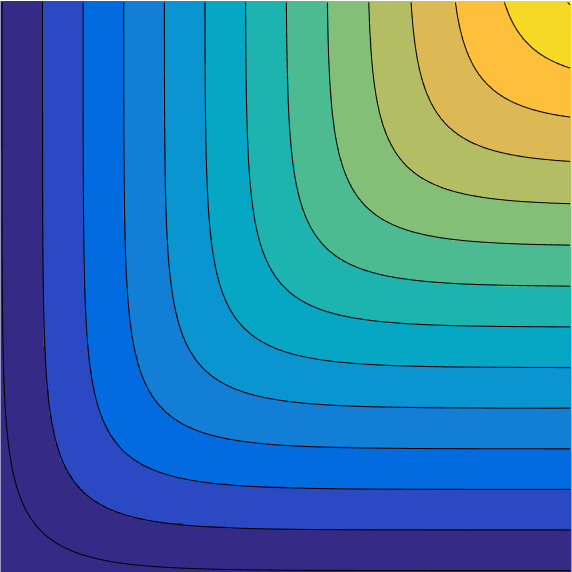}&
\includegraphics[width=.19\linewidth]{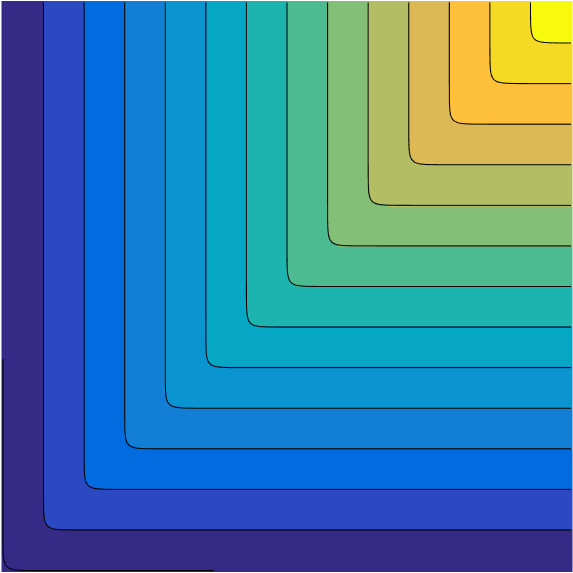}&
\includegraphics[width=.19\linewidth]{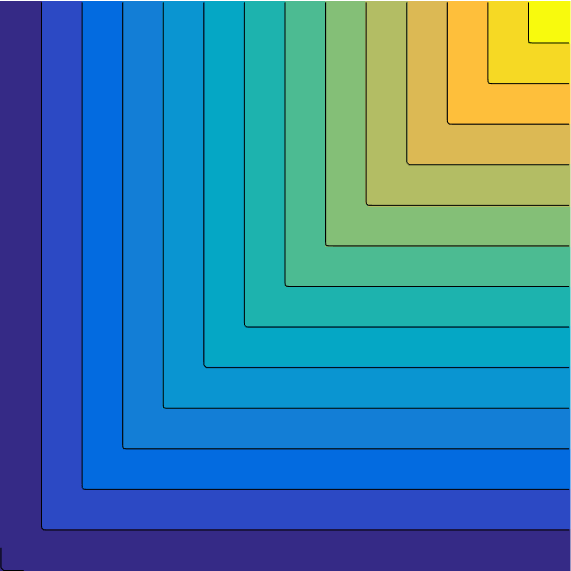} \\
$\varepsilon=1$ &
$\varepsilon=0.5$ &
$\varepsilon=10^{-1}$ &
$\varepsilon=10^{-2}$ &
$\varepsilon=10^{-3}$ 
\end{tabular}
\caption{\label{fig-softmin}
Display of the function $\smine(\z)$ in 2-D, $\z \in \RR^2$, for varying~$\varepsilon$.
}
\end{figure}

Using this notation, Equations~\eqref{eq-slse-sinkh-1} and~\eqref{eq-slse-sinkh-2} can be rewritten
\begin{align}
\label{eq-slse-smin-1}(\itt{\fD})_i &= \smine\, (\C_{ij}-\it{\gD}_j)_j + \varepsilon \log \a_i, \\
\label{eq-slse-smin-2}(\itt{\gD})_j &= \smine\, (\C_{ij}-\it{\fD}_i)_i + \varepsilon \log \b_j.
\end{align}
Here the term $\smine\, (\C_{ij}-\it{\gD}_j)_j$ denotes the soft-minimum of all values of the $j$th column of matrix $(\C-\ones_n (\it{\gD})^\top )$. To simplify notations, we introduce an operator that takes a matrix as input and outputs now a column vector of the soft-minimum values of its columns or rows. Namely, for any matrix $A\in\RR^{n\times m}$, we define
\begin{align*}
	\SMINEr(\A) \eqdef \pa{\smine \pa{\A_{i,j}}_j}_i\in\RR^n,\\ 
	\SMINEc(\A) \eqdef \pa{\smine \pa{\A_{i,j}}_i}_j\in\RR^m.
\end{align*}
Note that these operations are equivalent to the entropic $\c$-transform introduced in~\S\ref{sec-semi-discr-entropy} (see in particular~\eqref{eq-sinkh-c-transf}).
Using this notation, Sinkhorn's iterates read
\begin{align}
	\label{eq-lse-dual-1}\itt{\fD} &= \SMINEr\, (\C-\ones_n\transp{\it{\gD}}) + \varepsilon \log \a, \\
	\label{eq-lse-dual-2}\itt{\gD} &= \SMINEc\, (\C-\it{\fD}\transp{\ones_m})  + \varepsilon \log \b.
\end{align}
Note that as $\varepsilon \rightarrow 0$, $\smine$ converges to $\min$, but the iterations do not converge anymore in the limit $\varepsilon=0$, because alternate minimization does not converge for constrained problems, which is the case for the unregularized dual~\eqref{eq-dual}.
\end{rem}

\begin{rem}[Log-domain Sinkhorn]\label{rem-log-sinkh}
While mathematically equivalent to the Sinkhorn updates~\eqref{eq-sinkhorn}, iterations~\eqref{eq-slse-smin-1} and~\eqref{eq-slse-smin-2} suggest using the \emph{log-sum-exp} stabilization trick to avoid underflow for small values of $\varepsilon$. Writing $\underbar{z}=\min \z$, that trick suggests evaluating $\smine \z$ as
\eql{\label{eq-log-domain-min}
	\smine \z = \underbar{z} -\varepsilon \log \sum_i e^{-(\z_i-\underbar{z})/\varepsilon}.
}
Instead of substracting $\underbar{z}$ to stabilize the log-domain iterations as in~\eqref{eq-log-domain-min}, one can actually substract the previously computed scalings. 
This leads to the stabilized iteration
\begin{align}
	\label{eq-lse-sinkh-1}\itt{\fD} &= \SMINEr( \logP(\it{\fD},\it{\gD}))  + \it{\fD} + \varepsilon\log(\a), \\
	\label{eq-lse-sinkh-2}\itt{\gD} &= \SMINEc( \logP(\itt{\fD},\it{\gD})) + \it{\gD} + \varepsilon \log(\b), 
\end{align}
where we defined
\eq{
	\logP(\fD,\gD) = \pa{\C_{i,j} - \fD_i - \gD_j}_{i,j}.
}
In contrast to the original iterations~\eqref{eq-sinkhorn}, these log-domain iterations~\eqref{eq-lse-sinkh-1} and~\eqref{eq-lse-sinkh-2} are stable for arbitrary $\varepsilon>0$,
because the quantity $\logP(\fD,\gD)$ stays bounded during the iterations. 
The downside is that it requires $nm$ computations of $\exp$ at each step. 
Computing a $\SMINEr$ or $\SMINEc$ is typically substantially slower than matrix multiplications and requires computing line by line soft-minima of matrices $\logP$. There is therefore no efficient way to parallelize the application of Sinkhorn maps for several marginals simultaneously.
In Euclidean domains of small dimension, it is possible to develop efficient multiscale solvers with a decaying $\varepsilon$ strategy to significantly speed up the computation using sparse grids~\citep{schmitzer2016stabilized}.
\end{rem}

\begin{rem2}{Dual for generic measures}
For generic and not necessarily discrete input measures $(\al,\be)$, the dual problem~\eqref{eq-dual-formulation} reads
\eql{\label{eq-dual-entropic}
	\usup{(\f,\g) \in \Cc(\X)\times\Cc(\Y)} \int_\X \f \d\al + \int_\Y \g \d\be 
		 - \varepsilon \int_{\X\times\Y} e^{ \frac{-c(x,y)+f(x)+g(y)}{\varepsilon} } \d\al(x)\d\be(y).
}
This corresponds to a smoothing of the constraint $\Potentials(\c)$ appearing in the original problem~\eqref{eq-dual-generic}, which is retrieved in the limit $\varepsilon \rightarrow 0$.
Proving existence (\ie the sup is actually a max) of these Kantorovich potentials $(\f,\g)$ in the case of entropic transport is less easy than for classical OT, because one cannot use the $c$-transform and potentials are not automatically Lipschitz. Proof of existence can be done using the convergence of Sinkhorn iterations; see~\citep{2016-chizat-sinkhorn} for more details. 
\end{rem2}

\begin{rem2}{Unconstrained entropic dual}
As in Remark~\ref{rem-uncons-dual}, in the case $\int_\X \d\mu=\int_\Y\d\nu = 1$, one can consider an  alternative dual formulation
\eql{\label{eq-dual-entropic-alt}
	\usup{(\f,\g) \in \Cc(\X)\times\Cc(\Y)} \int_\X \f \d\al + \int_\Y \g \d\be 
		+ \smine( c - f \oplus g ),
}
which achieves the same optimal value as~\eqref{eq-dual-entropic}.
Similarly to~\eqref{eq-defn-softmin}, the soft-minimum (here on $\X \times \Y$) is defined as
\eq{
	\foralls S \in \Cc(\X\times\Y), \quad
	\smine S \eqdef 
	- \varepsilon \int_{\X\times\Y} e^{ \frac{-S(x,y)}{\varepsilon} } \d\al(x)\d\be(y)
}
(note that it depends on $(\al,\be)$).
As $\epsilon \rightarrow 0$, $\smine \rightarrow \min$, as used in the unregularized and unconstrained formulation~\eqref{eq-dual-uncons}.
Note that while both~\eqref{eq-dual-entropic} and~\eqref{eq-dual-entropic-alt} are unconstrained problems, 
a chief advantage of~\eqref{eq-dual-entropic-alt} is that it is better conditioned, 
in the sense that the Hessian of the functional is uniformly bounded by $\epsilon$. 
Another way to obtain such a conditioning improvement is to consider semidual problems; 
see \S\ref{sec-semi-discr-entropy} and in particular Remark~\ref{rem-second-order-smoothness}.
A disadvantage of this alternative dual formulation is that the presence of a log prevents the use of stochastic optimization methods as detailed in \S\ref{sec-sgd}; see in particular Remark~\ref{rem-sgd-cont-cont}.
\end{rem2}

\section{Regularized Approximations of the Optimal Transport Cost}
\label{sec-regularized-cost}

The entropic dual~\eqref{eq-dual-formulation} is a smooth unconstrained concave maximization problem, which approximates the original Kantorovich dual~\eqref{eq-dual}, as detailed in the following proposition. \todoK{explain soft indicator constraint, transition to result below.}

\begin{prop}\label{prop-feasibility-dual}
Any pair of optimal solutions $(\fD^\star,\gD^\star)$ to~\eqref{eq-dual-formulation} are such that $(\fD^\star,\gD^\star) \in \PotentialsD(\C)$, the set of feasible Kantorovich potentials defined in~\eqref{eq-feasible-potential}. As a consequence, we have that for any $\varepsilon$,
\eq{
	\dotp{\fD^\star}{\a} + \dotp{\gD^\star}{\b} \leq  \MKD_\C(\a,\b).
}
\end{prop}
\begin{proof}
	Primal-dual optimality conditions in~\eqref{eq-prim-dual-ent} with the constraint that $\P$ is a probability and therefore $\P_{i,j}\leq 1$ for all $i,j$ yields that $\exp(-(\fD^\star_i+\gD^\star_j-\C_{i,j})/\varepsilon)\leq 1$ and therefore that $\fD^\star_i+\gD^\star_j\leq\C_{i,j}$.
\end{proof}

A chief advantage of the regularized transportation cost $\MKD_\C^\varepsilon$ defined in~\eqref{eq-regularized-discr} is that it is smooth and convex, which makes it a perfect fit for integrating as a loss function in variational problems (see Chapter~\ref{c-variational}).

\begin{prop}\label{prop-convexity-dual}
$\MKD_\C^\varepsilon(\a,\b)$ is a jointly convex function of $\a$ and $\b$ for $\epsilon\geq 0$. When $\epsilon>0$, its gradient is equal to
\eq{
	\nabla\MKD_\C^\varepsilon(\a,\b)= \begin{bmatrix} \fD^\star \\ \gD^\star \end{bmatrix},
}
where $\fD^\star$ and $\gD^\star$ are the optimal solutions of Equation~\eqref{eq-dual-formulation} chosen so that their coordinates sum to 0.
\end{prop}

\todoK{
\begin{proof}
	\todo{TODO}
\end{proof}
}

In \citep{CuturiSinkhorn}, lower and upper bounds to approximate the Wasserstein distance between two histograms were proposed. These bounds consist in evaluating the primal and dual objectives at the solutions provided by the Sinkhorn algorithm.

\begin{defn}[Sinkhorn divergences]
	Let $\fD^\star$ and $\gD^\star$ be optimal solutions to~\eqref{eq-dual-formulation} and $\P^\star$ be the solution to~\eqref{eq-regularized-discr}. The Wasserstein distance is approximated using the following primal and dual Sinkhorn divergences:
\begin{align*}
	\SINKHORNP_\C^\varepsilon(\a,\b)& \eqdef \dotp{\C}{\P^\star} =  \dotp{e^{\frac{\fD^\star}{\varepsilon}}}{(\K\odot \C)e^{\frac{\gD^\star}{\varepsilon}}}, \\
	\SINKHORND_\C^\varepsilon(\a,\b) & \eqdef \dotp{\fD^\star}{\a} + \dotp{\gD^\star}{\b},
\end{align*}
where $\odot$ stands for the elementwise product of matrices, 
\end{defn}

\begin{prop}\label{prop-sinkhorn-div}
The following relationship holds:
\eq{
	\SINKHORND_\C^\varepsilon(\a,\b)\leq \MKD_\C^\varepsilon(\a,\b) \leq \SINKHORNP_\C^\varepsilon(\a,\b).
}
Furthermore
\eql{\label{eq-sinkh-div-gap}
	\SINKHORNP_\C^\varepsilon(\a,\b) - \SINKHORND_\C^\varepsilon(\a,\b) = \varepsilon ( \HD(\P^\star)+1 ) .
}
\end{prop}
\begin{proof}
Equation~\eqref{eq-sinkh-div-gap} is obtained by writing that the primal and dual problems have the same values at the optima (see~\eqref{eq-dual-formulation}), and hence
\eq{
	\MKD_\C^\varepsilon(\a,\b) = \SINKHORNP_\C^\varepsilon(\a,\b) - \varepsilon \HD(\P^\star) = 
	\SINKHORND_\C^\varepsilon(\a,\b) - \varepsilon \dotp{e^{\fD^\star/\varepsilon} }{ \K e^{\gD^\star/\varepsilon}}
}
The final result can be obtained by remarking that $\dotp{e^{\fD^\star/\varepsilon} }{ \K e^{\gD^\star/\varepsilon}}=1$, since the latter amounts to computing the sum of all entries of $\P^\star$.
\end{proof}

The relationships given above suggest a practical way to bound the actual OT distance, but they are, in fact, valid only upon convergence of the Sinkhorn algorithm and therefore never truly useful in practice. Indeed, in practice Sinkhorn iterations are always terminated after a certain accuracy threshold is reached. When a predetermined number of $L$ iterations is set and used to evaluate $\SINKHORND_\C^\varepsilon$ using iterates $\itL{\fD}$ and $\itL{\gD}$ instead of optimal solutions $\fD^\star$ and $\gD^\star$, one recovers, however, a lower bound: Using notation appearing in Equations~\eqref{eq-lse-sinkh-1} and~\eqref{eq-lse-sinkh-2}, we thus introduce the following finite step approximation of $\MKD_\C^\varepsilon$:
\eql{\label{eq-algorithmic-loss}
	\itL{\SINKHORND_\C}(\a,\b) \eqdef \dotp{\itL{\fD}}{\a} + \dotp{\itL{\gD}}{\b}.
}
This ``algorithmic'' Sinkhorn functional lower bounds the regularized cost function as soon as $L\geq 1$.

\begin{prop}[Finite Sinkhorn divergences]\label{defn-sinkhorn-div}
The following relationship holds:
\eq{
	\itL{\SINKHORND_\C}(\a,\b) \leq \MKD_\C^\varepsilon(\a,\b).
}
\end{prop}
\begin{proof}Similarly to the proof of Proposition~\ref{prop-feasibility-dual}, we exploit the fact that after even just one single Sinkhorn iteration, we have, following~\eqref{eq-slse-sinkh-1} and~\eqref{eq-slse-sinkh-2}, that $\itL{\fD}$ and $\itL{\gD}$ are such that the matrix with elements $\exp(-(\itL{\fD}_i+\itL{\gD}_j-\C_{i,j})/\varepsilon)$ has column sum $\b$ and its elements are therefore each upper bounded by $1$, which results in the dual feasibility of $(\itL{\fD}_i,\itL{\gD})$.
\end{proof}

\begin{rem}[Primal infeasibility of the Sinkhorn iterates]
Note that the primal iterates provided in \eqref{eq-sink-matrix} are not primal feasible, since, by definition, these iterates are designed to satisfy upon convergence marginal constraints. Therefore, it is not valid to consider $\dotp{\C}{\P^{(2L+1)}}$ as an approximation of $\MKD_\C(\a,\b)$ since $\P^{(2L+1)}$ is not feasible. Using the rounding scheme of~\citet{altschuler2017near} laid out in Remark~\ref{rem-complexity-rounding} one can, however, yield an upper bound on $\MKD_\C^\varepsilon(\a,\b)$ that can, in addition, be conveniently computed using matrix operations in parallel for several pairs of histograms, in the same fashion as Sinkhorn's algorithm~\citep{NIPS2018_8184}.
\end{rem}

\begin{rem}[Nonconvexity of finite dual Sinkhorn divergence] Unlike the regularized expression $\MKD_\C^\varepsilon$ in~\eqref{eq-dual-formulation}, the finite Sinkhorn divergence $\itL{\SINKHORND_\C}(\a,\b)$ is \emph{not}, in general, a convex function of its arguments (this can be easily checked numerically). $\itL{\SINKHORND_\C}(\a,\b)$ is, however, a differentiable function which can be differentiated using automatic differentiation techniques (see Remark~\ref{rem-auto-diff}) with respect to any of its arguments, notably $\C,\a$, or $\b$.
\end{rem}
\todoK{ automatically differentiated using the following recursion:\todo{write recursion}.}

\section{Generalized Sinkhorn}
\label{sec-generalized}

The regularized OT problem~\eqref{eq-regularized-discr} is a special case of a structured convex optimization problem of the form
\eql{\label{eq-generalized-ot-regul}
	\umin{\P}
		\sum_{i,j} \C_{i,j} \P_{i,j} - \varepsilon \HD(\P) + F( \P\ones_m ) + G( \transp{\P}\ones_n ).
} 
Indeed, defining $F=\iota_{\{\a\}}$ and $G=\iota_{\{\b\}}$, where the indicator function of a closed convex set $\Cc$ is
\eql{\label{eq-iota-function}
	\iota_{\Cc}(x) = \choice{
		0 \qifq x \in \Cc, \\
		+\infty \quad\text{otherwise},
	}
}
one retrieves the hard marginal constraints defining $\Couplings(\a,\b)$.
The proof of Proposition~\ref{prop-regularized-primal} carries to this more general problem~\eqref{eq-generalized-ot-regul}, so that  the unique solution of~\eqref{eq-generalized-ot-regul} also has the form~\eqref{eq-scaling-form}. 

As shown in~\citep{2015-Peyre-siims,FrognerNIPS,2016-chizat-sinkhorn,karlsson2016generalized}, Sinkhorn iterations~\eqref{eq-sinkhorn} can hence be extended to this problem, and they read
\eql{\label{eq-gen-sinkh}
	\uD \leftarrow \frac{\Prox_{F}^{\KLD}(\K \vD)}{\K \vD}
	\qandq
	\vD \leftarrow \frac{\Prox_{G}^{\KLD}(\transp{\K}\uD)}{\transp{\K}\uD},	
}
where the proximal operator for the $\KLD$ divergence is
\eql{\label{eq-prox-kl}
	\foralls \uD \in \RR_+^N, \quad
	\Prox_{F}^{\KLD}(\uD) = \uargmin{\uD' \in \RR_+^N}	\KLD(\uD'|\uD) + F(\uD').
}
For some functions $F,G$ it is possible to prove the linear rate of convergence for iterations~\eqref{eq-gen-sinkh}, and these schemes can be generalized to arbitrary measures; see~\citep{2016-chizat-sinkhorn} for more details.

Iterations~\eqref{eq-gen-sinkh} are thus interesting in the cases where $\Prox_{F}^{\KLD}$ and $\Prox_{G}^{\KLD}$ can be computed in closed form or very efficiently. This is in particular the case for separable functions of the form $F(\uD) = \sum_i F_i(\uD_i)$ since in this case
\eq{
	\Prox_{F}^{\KLD}(\uD) = \pa{ \Prox_{F_i}^{\KLD}(\uD_i) }_i.
}
Computing each $\Prox_{F_i}^{\KLD}$ is usually simple since it is a scalar optimization problem.
Note that, similarly to the initial Sinkhorn algorithm, it is also possible to stabilize the computation using log-domain computations~\citep{2016-chizat-sinkhorn}.

This algorithm can be used to approximate the solution to various generalizations of OT, and in particular unbalanced OT problems of the form~\eqref{eq-unbalanced-pbm} (see~\S\ref{sec-unbalanced} and in particular iterations~\eqref{eq-iterate-gen-sinkh}) and gradient flow problems of the form~\eqref{eq-grad-flow-discr} (see~\S\ref{sec-grad-flows}).

\begin{rem1}{Duality and Legendre transform}
The dual problem to~\eqref{eq-generalized-ot-regul} reads
\eql{\label{eq-dual-generalized}
	\umax{\fD,\gD}
		- F^*( \fD ) -G^*(\gD)
		- \varepsilon\sum_{i,j} e^{ \frac{\fD_i+\gD_j-\C_{i,j}}{\varepsilon} }
}
so that $(\uD,\vD) = (e^{\fD/\varepsilon},e^{\gD/\varepsilon})$ are the associated scalings appearing in~\eqref{eq-scaling-form}.
Here, $F^*$ and $G^*$ are the Fenchel--Legendre conjugate, which are convex functions defined as
\eql{\label{eq-legendre}
		\foralls \fD \in \RR^n, \quad
		F^*(\fD) \eqdef \umax{\a \in \RR^n} \dotp{\fD}{\a} - F(\a). 
}
The generalized Sinkhorn iterates~\eqref{eq-gen-sinkh} are a special case of Dykstra's algorithm~\citep{Dykstra83,Dykstra85} (extended to Bregman divergence~\citep{bauschke-lewis,CensorReich-Dykstra}; see also Remark~\ref{rem-bregman}) and is an alternate maximization scheme on the dual problem~\eqref{eq-dual-generalized}. 
\end{rem1}

The formulation~\eqref{eq-generalized-ot-regul} can be further generalized to more than two functions and more than a single coupling; we refer to~\citep{2016-chizat-sinkhorn} for more details. This includes as a particular case the Sinkhorn algorithm~\eqref{eq-sinkh-multimarg} for the multimarginal problem, as detailed in~\S\ref{sec-multimarginal}. 
It is also possible to rewrite the regularized barycenter problem~\eqref{eq-entropic-bary} this way, and the iterations~\eqref{eq-sinkhorn-bary} are in fact a special case of this generalized Sinkhorn.

\todoK{Write down the generalized iterations}



\chapter{Semidiscrete Optimal Transport}
\label{c-algo-semidiscr} 

This chapter studies methods to tackle the optimal transport problem when one of the two input measures is discrete (a sum of Dirac masses) and the other one is arbitrary, including notably the case where it has a density with respect to the Lebesgue measure.
When the ambient space has low dimension, this problem has a strong geometrical flavor because one can show that the optimal transport from a continuous density toward a discrete one is a piecewise constant map, where the preimage of each point in the support of the discrete measure is a union of disjoint cells.
When the cost is the squared Euclidean distance, these cells correspond to an important concept from computational geometry, the so-called Laguerre cells, which are Voronoi cells offset by a constant. This connection allows us to borrow tools from computational geometry to obtain fast computational schemes.
In high dimensions, the semidescrete formulation can also be interpreted as a stochastic programming problem, which can also benefit from a bit of regularization, extending therefore the scope of applications of the entropic regularization scheme presented in Chapter~\ref{c-entropic}.
All these constructions rely heavily on the notion of the $c$-transform, this time for general cost functions and not only matrices as in \S\ref{sec-c-transforms}. The $c$-transform is a generalization of the Legendre transform from convex analysis and plays a pivotal role in the theory and algorithms for OT. 

\section{$c$-Transform and $\bar c$-Transform}
\label{s-c-transform}

Recall that the dual OT problem~\eqref{eq-dual-generic} reads
\eq{
	\usup{(\f,\g)}
			\Ee(\f,\g) \eqdef
			\int_\X \f(x) \d\al(x) + \int_\Y \g(y) \d\be(y) + \iota_{\Potentials(\c)}(\f,\g),
}
where we used the useful indicator function notation~\eqref{eq-iota-function}. Keeping either dual potential $\f$ or $\g$ fixed and optimizing w.r.t. $\g$ or $\f$, respectively, leads to closed form solutions that  provide the definition of the $\c$-transform:
\begin{align}\label{eq-c-transform}
	\foralls y \in \Y, \quad
	\f^\c(y) &\eqdef \uinf{x \in \X} \c(x,y) - \f(x), \\ 
	\foralls x \in \X, \quad
	\g^{\bar\c}(x) &\eqdef \uinf{y \in \Y} \c(x,y) - \g(y), 
\end{align}
where we denoted $\bar\c(y,x) \eqdef c(x,y)$.
Indeed, one can check that 
\eql{\label{eq-alternate-c-transf}
	\f^\c \in \uargmax{\g} \Ee(\f,\g)
	\qandq
	\g^{\bar\c} \in \uargmax{\f} \Ee(\f,\g).
}
Note that these partial minimizations define maximizers on the support of respectively $\al$ and $\be$, while the definitions~\eqref{eq-c-transform} actually define functions on the whole spaces $\X$ and $\Y$. This is thus a way to extend in a canonical way solutions of~\eqref{eq-dual-generic} on the whole spaces. 
When $\X=\RR^d$ and $c(x,y)=\norm{x-y}_2^p=(\sum_{i=1}^d \abs{x_i-y_i})^{p/2}$, then the $\c$-transform~\eqref{eq-c-transform} $\f^\c$ is the so-called inf-convolution between $-\f$ and $\norm{\cdot}^p$. The definition of $\f^\c$ is also often referred to as a ``Hopf--Lax formula.'' 

The map $(\f,\g) \in \Cc(\X) \times \Cc(\Y) \mapsto (\g^{\bar\c},\f^\c) \in \Cc(\X) \times \Cc(\Y)$ replaces dual potentials by ``better'' ones (improving the dual objective $\Ee$). Functions that can be written in the form $\f^\c$ and $\g^{\bar\c}$ are called $c$-concave and $\bar c$-concave functions. 
In the special case $\c(x,y)=\dotp{x}{y}$ in $\X=\Y=\RR^\dim$, this definition coincides with the usual notion of concave functions.
Extending naturally Proposition~\ref{prop-ccc-2} to a continuous case, one has the property that
\eq{
	\f^{c\bar c c} = \f^c
	\qandq
	\g^{\bar c c \bar c} = \g^{\bar c},
}
where we denoted $\f^{c\bar c} = (\f^{c})^{\bar c}$. This invariance property shows that one can ``improve'' only once the dual potential this way. Alternatively, this means that alternate maximization does not converge (it immediately enters a cycle), which is classical for functionals involving a nonsmooth (a constraint) coupling of the optimized variables. This is in sharp contrast with entropic regularization of OT as shown in Chapter~\ref{c-entropic}. In this case, because of the regularization, the dual objective~\eqref{eq-dual-formulation} is smooth, and alternate maximization corresponds to Sinkhorn iterations~\eqref{eq-lse-sinkh-1} and~\eqref{eq-lse-sinkh-2}. These iterates, written over the dual variables, define entropically smoothed versions of the $c$-transform, where $\min$ operations are replaced by a ``soft-min.'' 

Using~\eqref{eq-alternate-c-transf}, one can reformulate~\eqref{eq-dual-generic} as an unconstrained convex program over a single potential,
\begin{align}\label{eq-semi-dual-cont}
	\MK_\c(\al,\be)  
		&= \usup{\f \in \Cc(\X)}
			\int_\X \f(x) \d\al(x) + \int_\Y \f^{c}(y) \d\be(y) \\
		&= \usup{\g \in \Cc(\Y)}
			\int_\X \g^{\bar\c}(x) \d\al(x) + \int_\Y \g(y) \d\be(y).
\end{align}
Since one can iterate the map $(\f,\g) \mapsto (\g^{\bar\c},\f^\c)$, it is possible to add the constraint that $\f$ is $\bar c$-concave and $\g$ is $\c$-concave, which is important to ensure enough regularity on these potentials and show, for instance, existence of solutions to~\eqref{eq-dual-generic}.  

\section{Semidiscrete Formulation}
\label{s-semidiscrete}

A case of particular interest is when $\be = \sum_j \b_j \de_{y_j}$ is discrete (of course the same construction applies if $\al$ is discrete by exchanging the role of $\al,\be$).
One can adapt the definition of the $\bar c$ transform~\eqref{eq-c-transform} to this setting by restricting the minimization to the support $(y_j)_j$ of $\be$,
\eql{\label{eq-disc-c-transfo}
	\foralls \gD \in \RR^m, \;
	\foralls x \in \Xx, \quad
	\gD^{\bar \c}(x) \eqdef \umin{j \in \range{m}} \c(x,y_j) - \gD_j.
}
This transform maps a vector $\gD$ to a continuous function $\gD^{\bar \c} \in \Cc(\Xx)$.
Note that this definition coincides with~\eqref{eq-c-transform} when imposing that the space $\X$ is equal to the support of $\be$. 
Figure~\ref{fig-c-transform-discrete} shows some examples of such discrete $\bar c$-transforms in one and two dimensions.

\newcommand{\MyFigCTrans}[1]{\includegraphics[width=.24\linewidth,trim=63 40 48 30,clip]{c-transform-2d/c-transform-p#1}}
\begin{figure}[h!]
\centering
\includegraphics[width=.8\linewidth]{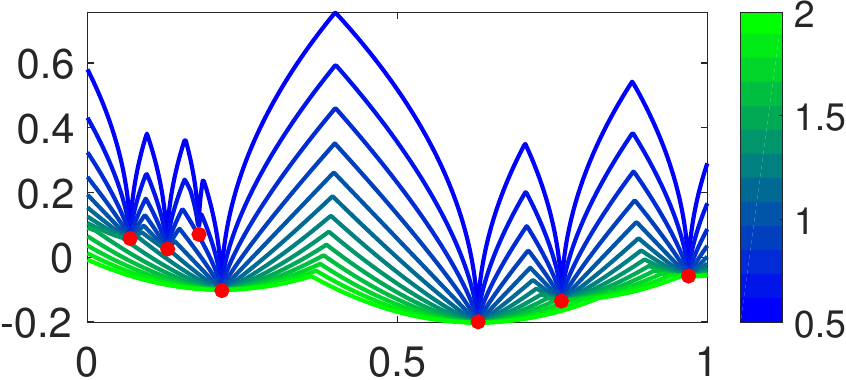}\vspace{2mm}
\begin{tabular}{@{}c@{\hspace{1mm}}c@{\hspace{1mm}}c@{\hspace{1mm}}c@{}}
\MyFigCTrans{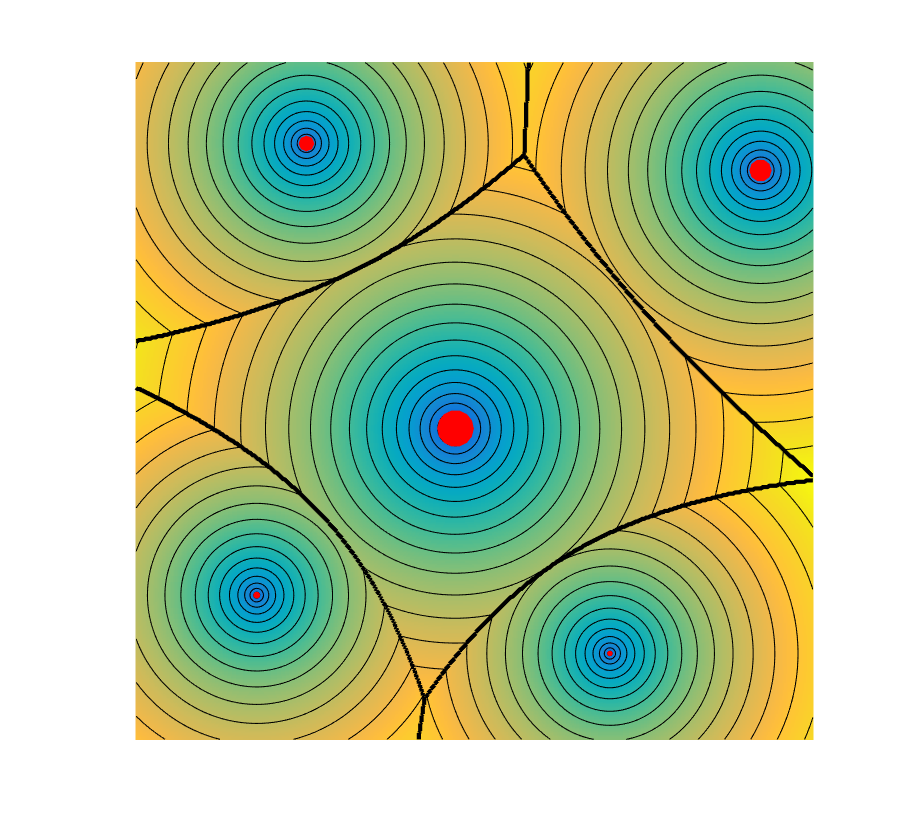}&
\MyFigCTrans{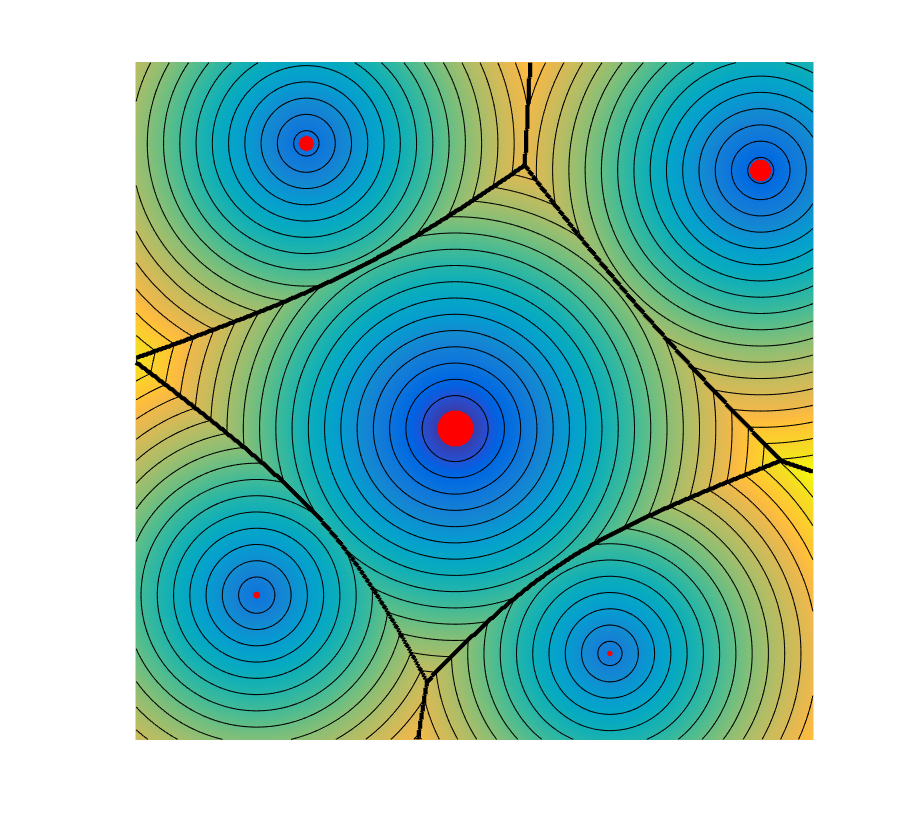}&
\MyFigCTrans{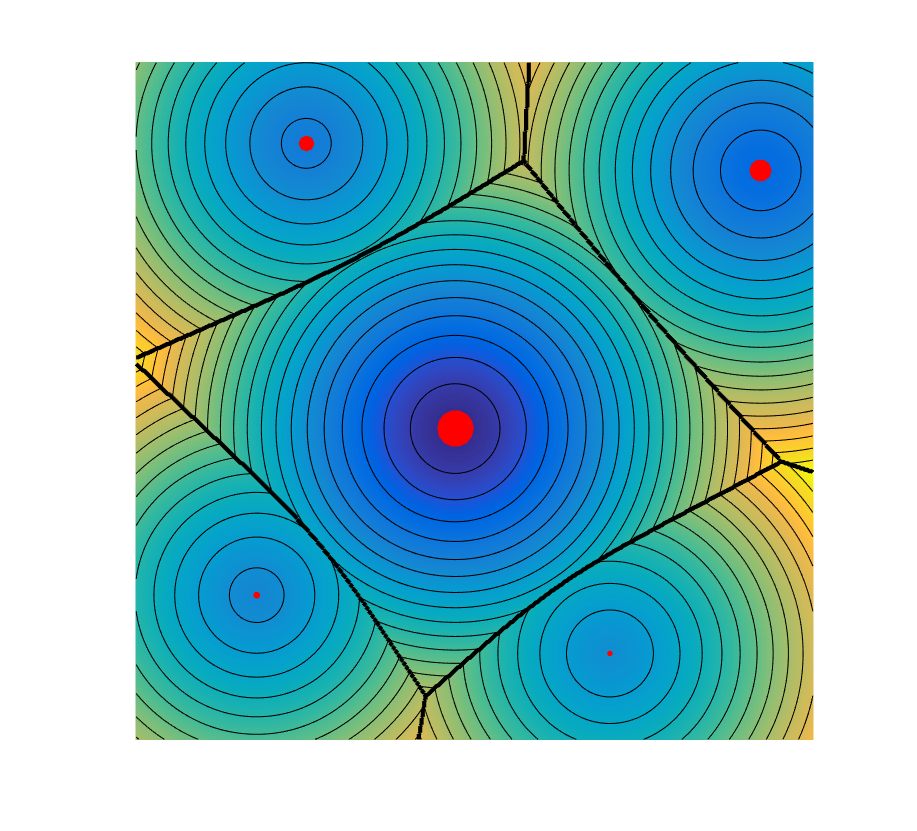}&
\MyFigCTrans{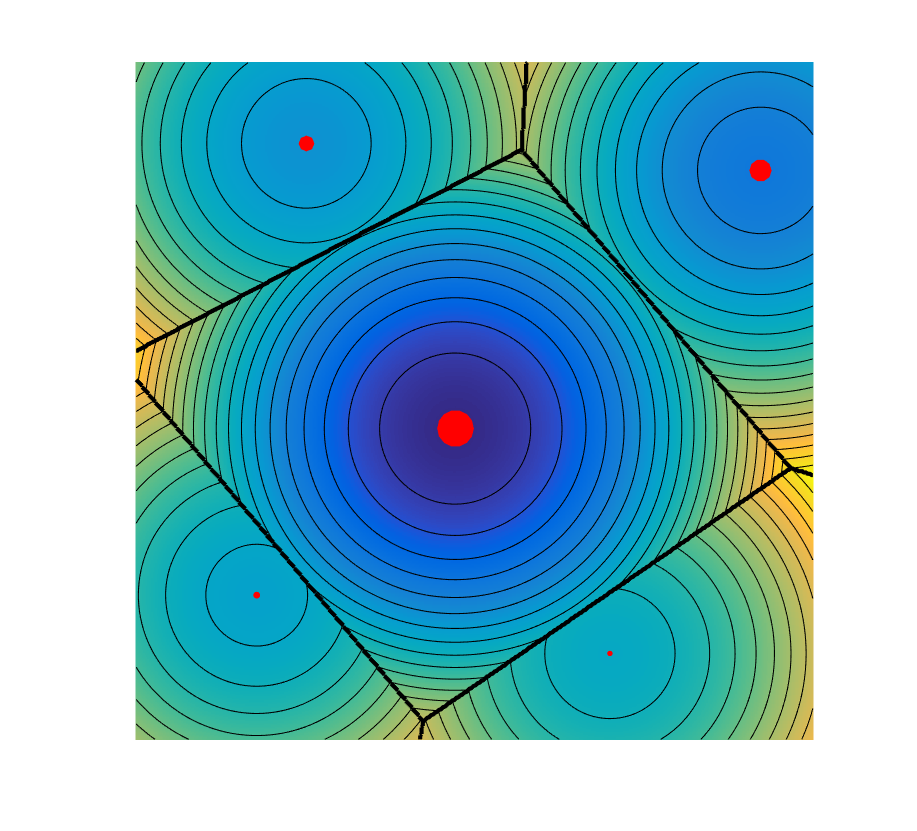}\\
$p=1/2$ & $p=1$ & $p=3/2$ & $p=2$ 
\end{tabular}
\caption{\label{fig-c-transform-discrete}
Top: examples of semidiscrete $\bar c$-transforms $\gD^{\bar \c}$ in one dimension, for ground cost $c(x,y)=|x-y|^p$ for varying $p$ (see colorbar). 
The red points are at locations $(y_j,-\gD_j)_j$. 
Bottom: examples of semidiscrete $\bar c$-transforms $\gD^{\bar \c}$ in two dimensions, for ground cost $c(x,y)=\norm{x-y}_2^p=(\sum_{i=1}^d \abs{x_i-y_i})^{p/2}$ for varying $p$. 
The red points are at locations $y_j \in \RR^2$, and their size is proportional to $\gD_j$. 
The regions delimited by bold black curves are the Laguerre cells $(\Laguerre_{j}(\gD))_j$ associated to these points $(y_j)_j$. 
}
\end{figure}

Crucially, using the discrete $\bar c$-transform in the semidiscrete problem~\eqref{eq-semi-dual-cont} yields a finite-dimensional optimization, 
\eql{\label{eq-semi-dual-discr}
	\MK_\c(\al,\be) = 
		\umax{\gD \in \RR^m}
			\Ee(\gD) \eqdef 
			\int_\X \gD^{\bar \c}(x) \d\al(x) + \sum \gD_y \b_j.
}

The Laguerre cells associated to the dual weights $\gD$
\eq{
	\Laguerre_{j}(\gD) \eqdef \enscond{x \in \X}{ \foralls j' \neq j, \c(x,y_j) - \gD_j \leq \c(x,y_{j'}) - \gD_{j'} }
}
induce a disjoint decomposition of $\X = \bigcup_j \Laguerre_{j}(\gD)$. When $\gD$ is constant, the Laguerre cells decomposition corresponds to the Voronoi diagram partition of the space. 
Figure~\ref{fig-c-transform-discrete}, bottom row, shows examples of Laguerre cells segmentations in two dimensions. 

This allows one to conveniently rewrite the minimized energy as
\eql{\label{eq-semi-disc-energy}
	\Ee(\gD) = \sum_{j=1}^m \int_{\Laguerre_{j}(\gD)} \pa{ c(x,y_j) - \gD_j } \d\al(x) + \dotp{\gD}{\b}.
}
The gradient of this function can be computed as follows:
\eq{
	\foralls j \in \range{m}, \quad
	\nabla\Ee(\gD)_j = - \int_{\Laguerre_{j}(\gD)} \d\al(x) + \b_j.
}
Figure~\ref{fig-semi-discr} displays iterations of a gradient descent to minimize $\Ee$.
Once the optimal $\gD$ is computed, then the optimal transport map $\T$ from $\al$ to $\be$ is mapping any $x \in \Laguerre_{j}(\gD)$ toward $y_j$, so it is piecewise constant. 

In the special case $\c(x,y)=\norm{x-y}^2$, the decomposition in Laguerre cells is also known as a ``power diagram.'' 
The cells are polyhedral and can be computed efficiently using computational geometry algorithms; see~\citep{aurenhammer1987power}. 
The most widely used algorithm relies on the fact that the power diagram of points in $\RR^\dim$ is equal to the projection on $\RR^\dim$ of the convex hull of the set of points $( (y_j,\norm{y_j}^2 - \gD_j) )_{j=1}^m \subset \RR^{\dim+1}$. There are numerous algorithms to compute convex hulls; for instance, that of~\citet{chan1996optimal} in two and three dimensions has complexity $O(m\log(Q))$, where $Q$ is the number of vertices of the convex hull.
\todoK{Maybe make a drawing of this projection.}


\newcommand{\MyFigSemiD}[1]{\includegraphics[width=.19\linewidth]{semidiscrete-gd/#1}}
\begin{figure}[h!]
\centering
\begin{tabular}{@{}c@{\hspace{1mm}}c@{\hspace{1mm}}c@{\hspace{1mm}}c@{\hspace{1mm}}c@{}}
\MyFigSemiD{inputs} &
\MyFigSemiD{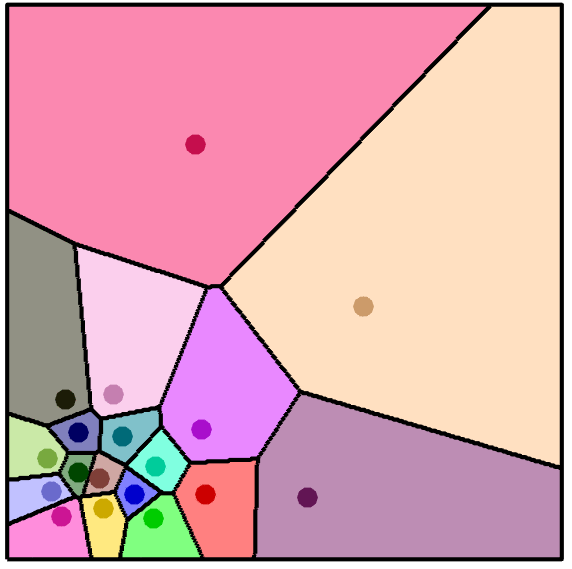} &
\MyFigSemiD{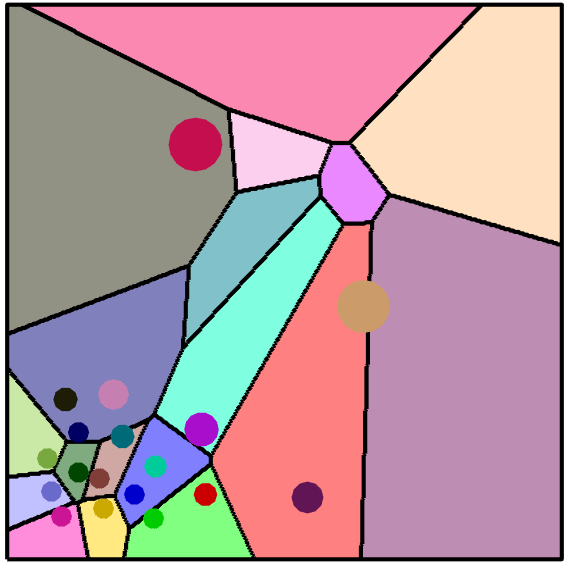} &
\MyFigSemiD{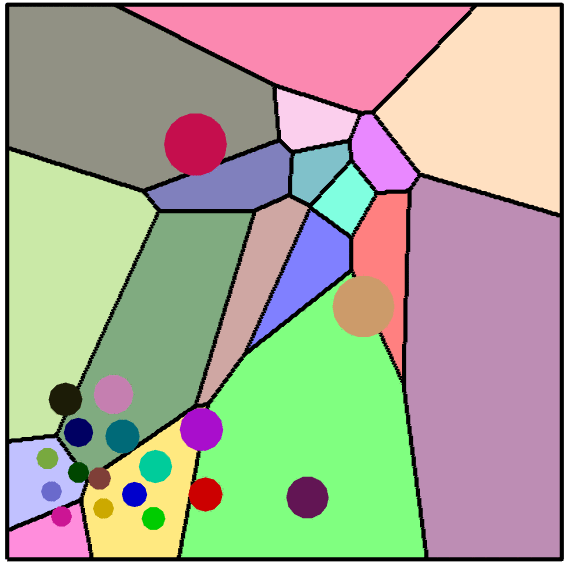} &
\MyFigSemiD{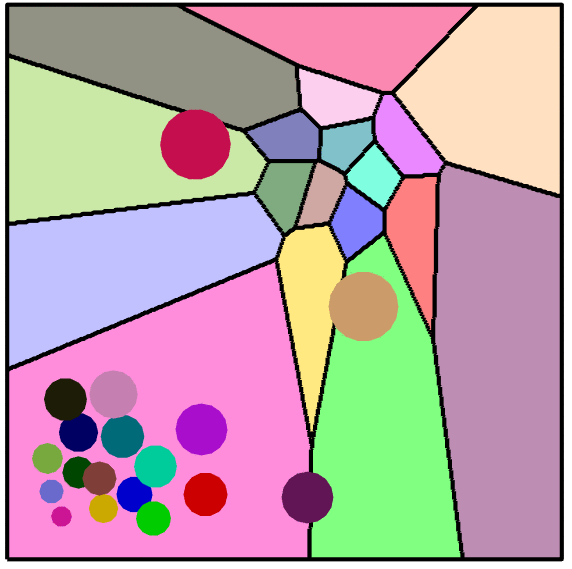}   \\ 
$\al$ and $\be$ &
$\ell=1$ &
$\ell=20$ &
$\ell=40$ &
$\ell=100$ 
\end{tabular}
\caption{\label{fig-semi-discr}
Iterations of the semidiscrete OT algorithm minimizing~\eqref{eq-semi-disc-energy} (here a simple gradient descent is used).
The support $(y_j)_j$ of the discrete measure $\be$ is indicated by the colored points, while the continuous measure $\al$ is the uniform measure on a square. 
The colored cells display the Laguerre partition $( \Laguerre_{j}( \it{\gD} ) )_j$ where $\it{\gD}$ is the discrete dual potential computed at iteration $\ell$. 
}
\end{figure}

The initial idea of a semidiscrete solver for Monge--Amp\`ere equations was proposed by~\citet{oliker1989numerical}, and its relation to the dual variational problem was shown by~\citet{AurenhammerHA98}.
A theoretical analysis and its application to the reflector problem in optics is detailed in~\citep{caffarelli1999problem}. 
The semidiscrete formulation was used in~\citep{carlier2010knothe} in conjunction with a continuation approach based on Knothe's transport. 
The recent revival of this methods in various fields is due to~\citet{Merigot11}, who proposed a quasi-Newton solver and clarified the link with concepts from computational geometry. We refer to~\citep{Levy2017review} for a recent overview. The use of a Newton solver which is applied to sampling in computer graphics is proposed in~\citep{de2012blue}; see also~\citep{levy2015numerical} for applications to 3-D volume and surface processing. 
An important area of application of the semidiscrete method is for the resolution of the incompressible fluid dynamic (Euler's equations) using Lagrangian methods~\citep{deGoes2015,gallouet2017lagrangian}. The semidiscrete OT solver enforces incompressibility at each iteration by imposing that the (possibly weighted) points cloud approximates a uniform distribution inside the domain.  
The convergence (with linear rate) of damped Newton iterations is proved in~\citep{mirebeau2015discretization} for the Monge--Amp\`ere equation and is refined in~\citep{kitagawa2016newton} for optimal transport. Semidiscrete OT finds important applications to illumination design, notably reflectors; see~\citep{merigot2017light}.

\section{Entropic Semidiscrete Formulation}
\label{sec-semi-discr-entropy}

The dual of the entropic regularized problem between arbitrary measures~\eqref{eq-entropic-generic} is a smooth unconstrained optimization problem:
\begin{equation}\label{eq-dual-entropic-generic}
	\MK_\c^\epsilon(\al,\be) = 
	\usup{(\f,\g) \in \Cc(\X) \times \Cc(\Y)}
		 \int_\X \f\d\al + \int_\Y \g\d\be 
		- \epsilon \int_{\X \times \Y}  e^{ \frac{-c + \f \oplus \g }{\epsilon} } \d\al\d\be,
\end{equation}
where we denoted $(f\oplus g)(x,y) \eqdef f(x)+g(y)$.

Similarly to the unregularized problem~\eqref{eq-c-transform}, one can minimize explicitly with respect to either $\f$ or $\g$ in~\eqref{eq-dual-entropic-generic}, which yields a smoothed $\c$-transform
\begin{align*}
	\foralls y \in \Y, \quad
	\f^{\c,\epsilon}(y) &\eqdef - \epsilon \log\pa{ 
			\int_{\X} e^{\frac{-\c(x,y) + \f(x)}{\epsilon}} \d\al(x)
	},\\
	\foralls x \in \X, \quad
	\g^{\bar\c,\epsilon}(x) &\eqdef - \epsilon \log\pa{ 
			\int_{\Y} e^{\frac{-\c(x,y) + \g(y)}{\epsilon}} \d\be(y)
	}.
\end{align*}
In the case of a discrete measure $\be = \sum_{j=1}^m \b_j\de_{y_j}$, the problem simplifies as with~\eqref{eq-semi-dual-discr} to a finite-dimensional problem expressed as a function of the discrete dual potential $\gD \in \RR^m$, 
\eql{\label{eq-entropic-smoothed-c}
	\foralls x \in \X, \quad
	\gD^{\bar\c,\epsilon}(x) \eqdef - \epsilon \log\pa{ 
		\sum_{j=1}^m e^{\frac{-\c(x,y_j) + \gD_j}{\epsilon}} \b_j
	}.
}
One defines similarly $\fD^{\bar\c,\epsilon}$ in the case of a discrete measure $\al$.
Note that the rewriting~\eqref{eq-lse-dual-1} and~\eqref{eq-lse-dual-2} of Sinkhorn using the soft-min operator $\smine$ corresponds to the alternate computation of entropic smoothed $\c$-transforms,
\eql{\label{eq-sinkh-c-transf}
	\itt{\fD}_i = \gD^{\bar\c,\epsilon}(x_i) 
	\qandq
	\itt{\gD}_j = \fD^{\,\c,\epsilon}(y_j) .
}

Instead of maximizing~\eqref{eq-dual-entropic-generic}, one can thus solve the following finite-dimensional optimization problem:
\eql{\label{eq-semi-disc-energy-entropy}
	\umax{\gD \in \RR^n}
		\Ee^\epsilon(\gD) \eqdef \int_\X \gD^{\bar\c,\epsilon}(x)\d\al(x) + \dotp{\gD}{\b}.
}
Note that this optimization problem is still valid even in the unregularized case $\epsilon=0$ and in this case $\gD^{\bar\c,\epsilon=0}=\gD^{\bar\c}$ is the $\bar \c$-transform defined in~\eqref{eq-disc-c-transfo} so that~\eqref{eq-semi-disc-energy-entropy} is in fact~\eqref{eq-semi-disc-energy}.
The gradient of this functional reads
\eql{\label{eq-grad-semid-entrop}
	\foralls j \in \range{m}, \quad
	\nabla\Ee^\epsilon(\gD)_j = - \int_{\X} \chi_j^\epsilon(x) \d\al(x) + \b_j, 
}
where $\chi_j^\epsilon$ is a smoothed version of the indicator $\chi_j^0$ of the Laguerre cell $\Laguerre_{j}(\gD)$,
\eq{
	\chi_j^\epsilon(x) = 
	\frac{
		e^{\frac{-\c(x,y_j) + \gD_j}{\epsilon}}
	}{
		\sum_\ell e^{\frac{-\c(x,y_\ell) + \gD_\ell}{\epsilon}}
	}.
}
Note once again that this formula~\eqref{eq-grad-semid-entrop} is still valid for $\epsilon=0$.
Note also that the family of functions $( \chi_j^\epsilon )_j$ is a partition of unity, \ie $\sum_j \chi_j^\epsilon=1$ and $\chi_j^\epsilon \geq 0$. Figure~\ref{fig-c-transform-discrete-eps}, bottom row, illustrates this. 

\begin{rem}[Second order methods and connection with logistic regression]\label{rem-second-order-smoothness}
	A crucial aspect of the smoothed semidiscrete formulation~\eqref{eq-semi-disc-energy-entropy} is that it corresponds to the minimization of a smooth function.
	Indeed, as shown in~\citep{genevay2016stochastic}, the Hessian of $\Ee^\epsilon$ is upper bounded by $1/\epsilon$, so that $\nabla \Ee^\epsilon$ is $\frac{1}{\epsilon}$-Lipschitz continuous.
	In fact, that problem is very closely related to a multiclass logistic regression problem (see Figure~\ref{fig-c-transform-discrete-eps} for a display of the resulting fuzzy classification boundary) and enjoys the same favorable properties (see~\citep{hosmer2013applied}), which are generalizations of self-concordance; see~\citep{bach2010self}. 
	In particular, the Newton method converges quadratically, and one can use in practice quasi-Newton techniques, such as L-BFGS, as advocated in~\citep{2016-Cuturi-siims}. Note that~\citep{2016-Cuturi-siims} studies the more general barycenter problem detailed in~\S\ref{sec-bary}, but it is equivalent to this semidiscrete setting when considering only a pair of input measures. 
	The use of second order schemes (Newton or L-BFGS) is also advocated in the unregularized case $\epsilon=0$ by~\citep{Merigot11,de2012blue,levy2015numerical}. In \citep[Theo. 5.1]{kitagawa2016newton}, the Hessian of $\Ee^0(\gD)$ is shown to be uniformly bounded as long as the volume of the Laguerre cells is bounded by below and $\al$ has a continuous density. \citeauthor{kitagawa2016newton} proceed by showing the linear convergence of a damped Newton algorithm with a backtracking to ensure that the Laguerre cells never vanish between two iterations. This result justifies the use of second order methods even in the unregularized case. The intuition is that, while the conditioning of the entropic regularized problem scales like $1/\epsilon$, when $\epsilon=0$, this conditioning is rather driven by $m$, the number of samples of the discrete distribution (which controls the size of the Laguerre cells).
	Other methods exploiting second order schemes were also recently studied by~\citep{knight2013fast,sugiyama2017tensor,cohen2017matrix,allen2017much}.
\end{rem}

\newcommand{\MyFigCTransEps}[1]{\includegraphics[width=.245\linewidth,trim=63 40 48 30,clip]{c-transform-2d/c-transform-eps#1}}
\begin{figure}[h!]
\centering
\includegraphics[width=.9\linewidth]{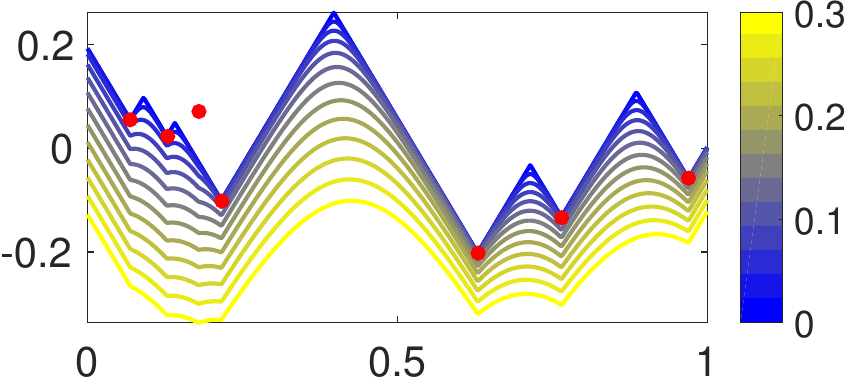}\vspace{2mm}
\begin{tabular}{@{}c@{\hspace{1mm}}c@{\hspace{1mm}}c@{\hspace{1mm}}c@{}}
\MyFigCTransEps{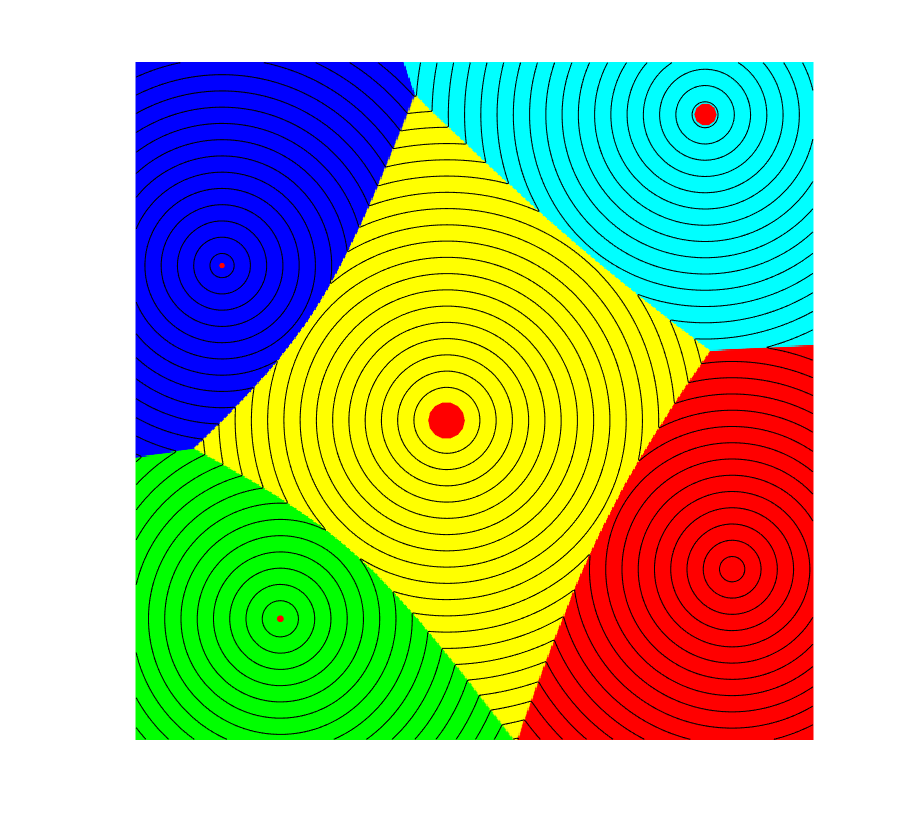}&
\MyFigCTransEps{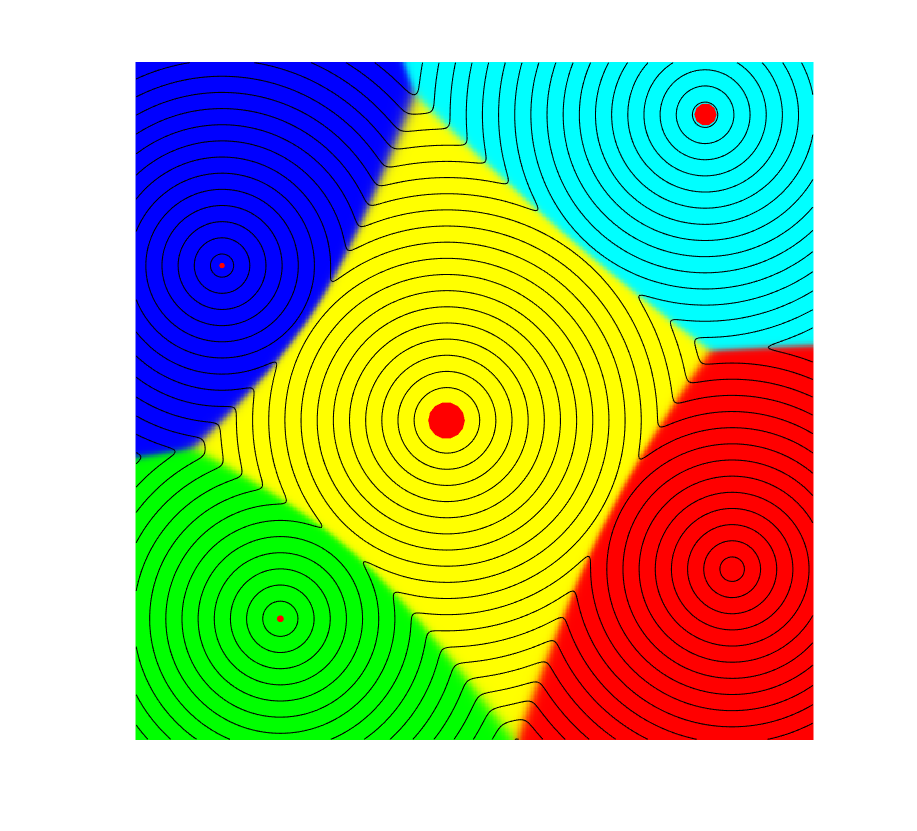}&
\MyFigCTransEps{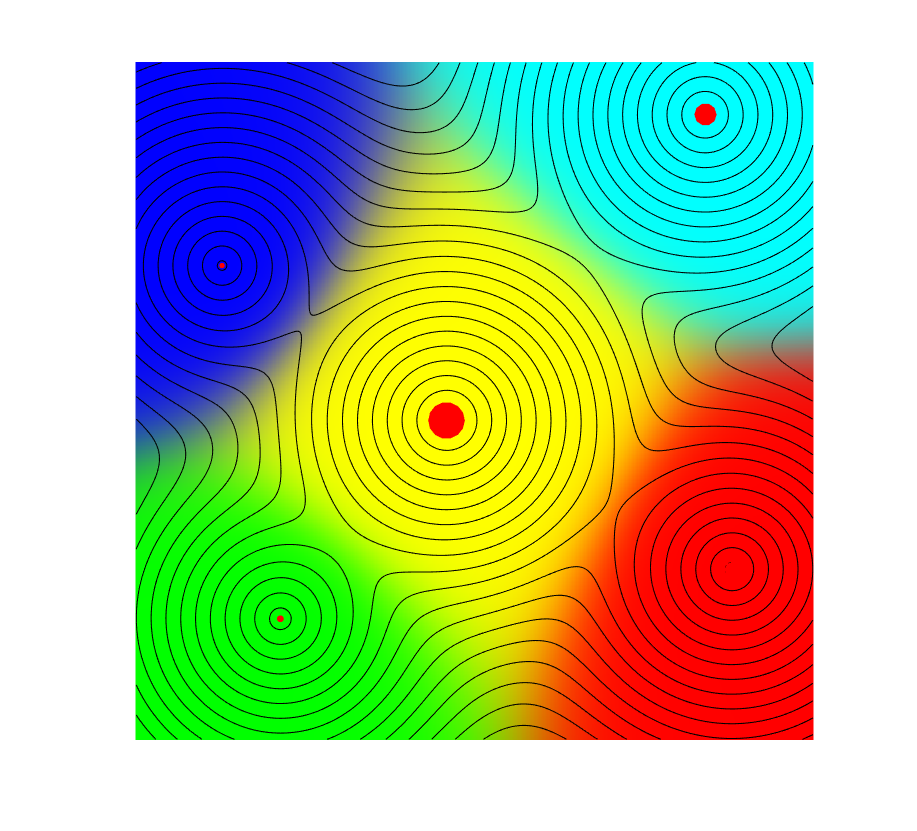}&
\MyFigCTransEps{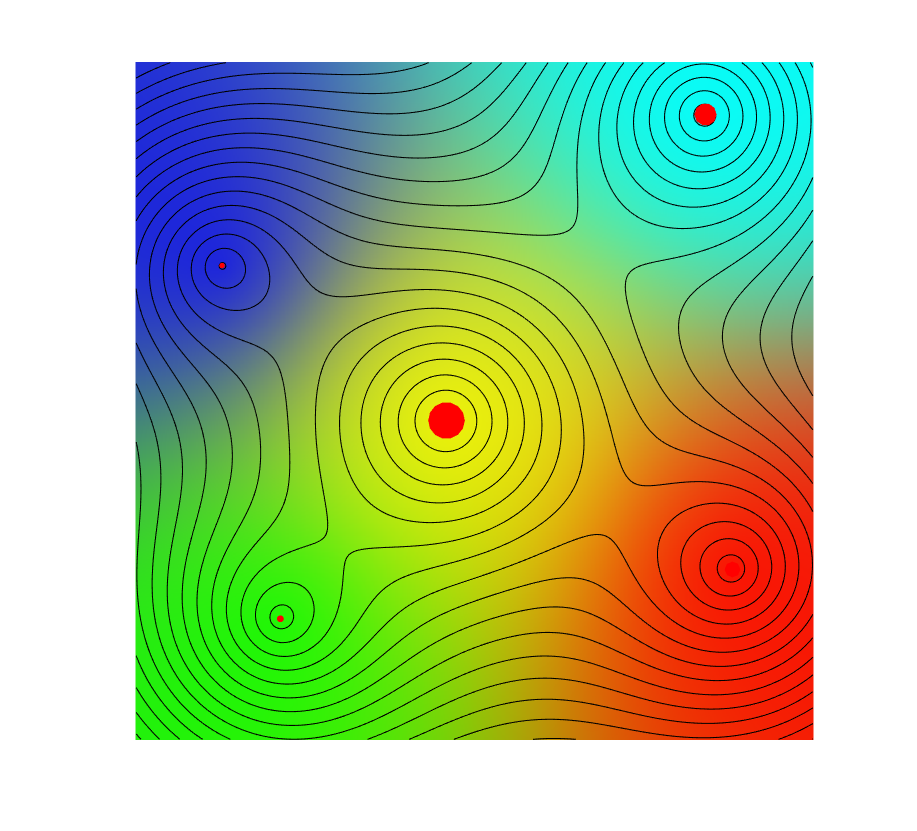}\\
$\epsilon=0$ & $\epsilon=0.01$ & $\epsilon=0.1$ & $\epsilon=0.3$ 
\end{tabular}
\caption{\label{fig-c-transform-discrete-eps}
Top: examples of entropic semidiscrete $\bar c$-transforms $\gD^{\bar \c,\epsilon}$ in one dimension, for ground cost $c(x,y)=|x-y|$ for varying $\epsilon$ (see colorbar). 
The red points are at locations $(y_j,-\gD_j)_j$. 
Bottom: examples of entropic semidiscrete $\bar c$-transforms $\gD^{\bar \c,\epsilon}$ in two dimensions, for ground cost $c(x,y)=\norm{x-y}_2$ for varying $\epsilon$. The black curves are the level sets of the function $\gD^{\bar \c,\epsilon}$, while the colors indicate the smoothed indicator function of the Laguerre cells $\chi_j^\epsilon$.
The red points are at locations $y_j \in \RR^2$, and their size is proportional to $\gD_j$. 
}
\end{figure}

\begin{rem}[Legendre transforms of OT cost functions] 
	As stated in Proposition~\ref{prop-convexity-dual}, $\MKD_\C^\epsilon(\a,\b)$ is a convex function of $(\a,\b)$ (which is also true in the unregularized case $\epsilon=0$). 
	It is thus possible to compute its Legendre--Fenchel transform, which is defined in~\eqref{eq-legendre}.
	Denoting $F_\a(\b) = \MKD_\C^\epsilon(\a,\b)$, one has, for a fixed $\a$, following~\citet{2016-Cuturi-siims}:
	\eq{
		F_\a^*(\gD) =  - \epsilon H(\a) + \sum_i \a_i \gD^{\bar\c,\epsilon}(x_i).
	}
	Here $\gD^{\bar\c,\epsilon}$ is the entropic-smoothed $c$-transform introduced in~\eqref{eq-entropic-smoothed-c}.
	In the unregularized case $\epsilon=0$, and for generic measures,~\citet{Carlier-NumericsBarycenters} show, 	
	denoting $\Ff_\al(\be) \eqdef \MK_\c(\al,\be)$,
	\eq{
		\foralls \g \in \Cc(\Y), \quad
		\Ff_\al^*(\g) = \int_\X \g^{\bar\c}(x)  \d\al(x), 		
	}
	where the $\bar c$-transform $\g^{\bar c} \in \Cc(\X)$ of $\g$ is defined in~\S\ref{s-c-transform}.
	%
	Note that here, since $\Mm(\X)$ is in duality with $\Cc(\X)$, the Legendre transform is a function of continuous functions.
	Denoting now $G(\a,\b) \eqdef \MKD_\C^\epsilon(\a,\b)$, one can derive as in~\citep{2016-Cuturi-siims,cuturi2018semidual} the Legendre transform for both arguments,
	\eq{
		\foralls (\fD,\gD) \in \RR^n \times \RR^m, \quad
		G^*(\fD,\gD) = - \epsilon \log \sum_{i,j} e^{ \frac{-\C_{i,j}+\fD_i+\gD_j}{\epsilon} },
	} 
	which can be seen as a smoothed version of the Legendre transform of $\Gg(\al,\be) \eqdef \MK_\c(\al,\be)$,
	\eq{
		\foralls (\f,\g) \in \Cc(\X) \times \Cc(\Y), \quad
		\Gg^*(\f,\g) = \uinf{(x,y) \in \X \times \Y} \c(x,y)-\f(x)-\g(y).	
	}
\end{rem}

\section{Stochastic Optimization Methods}
\label{sec-sgd}

The semidiscrete formulation~\eqref{eq-semi-disc-energy} and its smoothed version~\eqref{eq-semi-disc-energy-entropy} are appealing because the energies to be minimized are written as an expectation with respect to the probability distribution $\al$,
\eq{
	\Ee^\epsilon(\gD) = \int_\X E^\epsilon(\gD,x) \d\al(x) = \EE_X(  E^\epsilon(\gD,X) )
}
\eq{
	\qwhereq
	 E^\epsilon(\gD,x) \eqdef 
	\gD^{\bar\c,\epsilon}(x) - \dotp{\gD}{\b},
}
and $X$ denotes a random vector distributed on $\X$ according to $\al$.
Note that the gradient of each of the involved functional reads
\eq{
	\nabla_{\gD} E^\epsilon(x,\gD) = (  \chi_j^\epsilon(x) - \b_j )_{j=1}^m \in \RR^m.
}
One can thus use stochastic optimization methods to perform the maximization, as proposed in~\citet{genevay2016stochastic}.
This allows us to obtain provably convergent algorithms without the need to resort to an arbitrary discretization of $\al$ (either approximating $\al$ using sums of Diracs or using quadrature formula for the integrals).
The measure $\al$ is used as a black box from which one can draw independent samples, which is a natural computational setup for many high-dimensional applications in statistics and machine learning. 
This class of methods has been generalized to the computation of Wasserstein barycenters (as described in \S\ref{sec-bary}) in~\citep{staib2017parallel}.

\paragraph{Stochastic gradient descent.}

Initializing $\gD^{(0)} = \zeros_{P}$, the stochastic gradient descent algorithm (SGD; used here as a maximization method) draws at step $\ell$ a point $x_\ell \in \X$ according to distribution $\al$ (independently from all past and future samples $(x_\ell)_\ell$) to form the update
\eql{\label{eq-sgd}
	\itt{\gD} \eqdef \it{\gD} + \tau_\ell  \nabla_{\gD} E^\epsilon(\it{\gD},x_\ell).
} 
The step size $\tau_\ell$ should decay fast enough to zero in order to ensure that the ``noise'' created by using  $\nabla_{\gD} E^\epsilon(x_\ell,\gD)$ as a proxy for the true gradient $\nabla \Ee^\epsilon(\gD)$ is canceled in the limit. 
A typical choice of schedule is 
\eql{\label{eq-step-size-sgd} 
	\tau_\ell \eqdef \frac{\tau_0}{1 + \ell/\ell_0}, 
}
where $\ell_0$ indicates roughly the number of iterations serving as a warmup phase.
One can prove the convergence result
\eq{ 
	\Ee^\epsilon(\gD^\star) - \EE( \Ee^\epsilon(\it{\gD}) ) = O\pa{ \frac{1}{\sqrt{\ell}} }, 
}
where $\gD^\star$ is a solution of~\eqref{eq-semi-disc-energy-entropy} and where $\EE$ indicates an expectation with respect to the i.i.d. sampling of $(x_\ell)_\ell$ performed at each iteration.
Figure~\ref{fig-semi-discrete-sgd} shows the evolution of the algorithm on a simple 2-D example, where $\al$ is the uniform distribution on $[0,1]^2$.

\begin{figure}[h!]
\centering
\includegraphics[width=.95\linewidth]{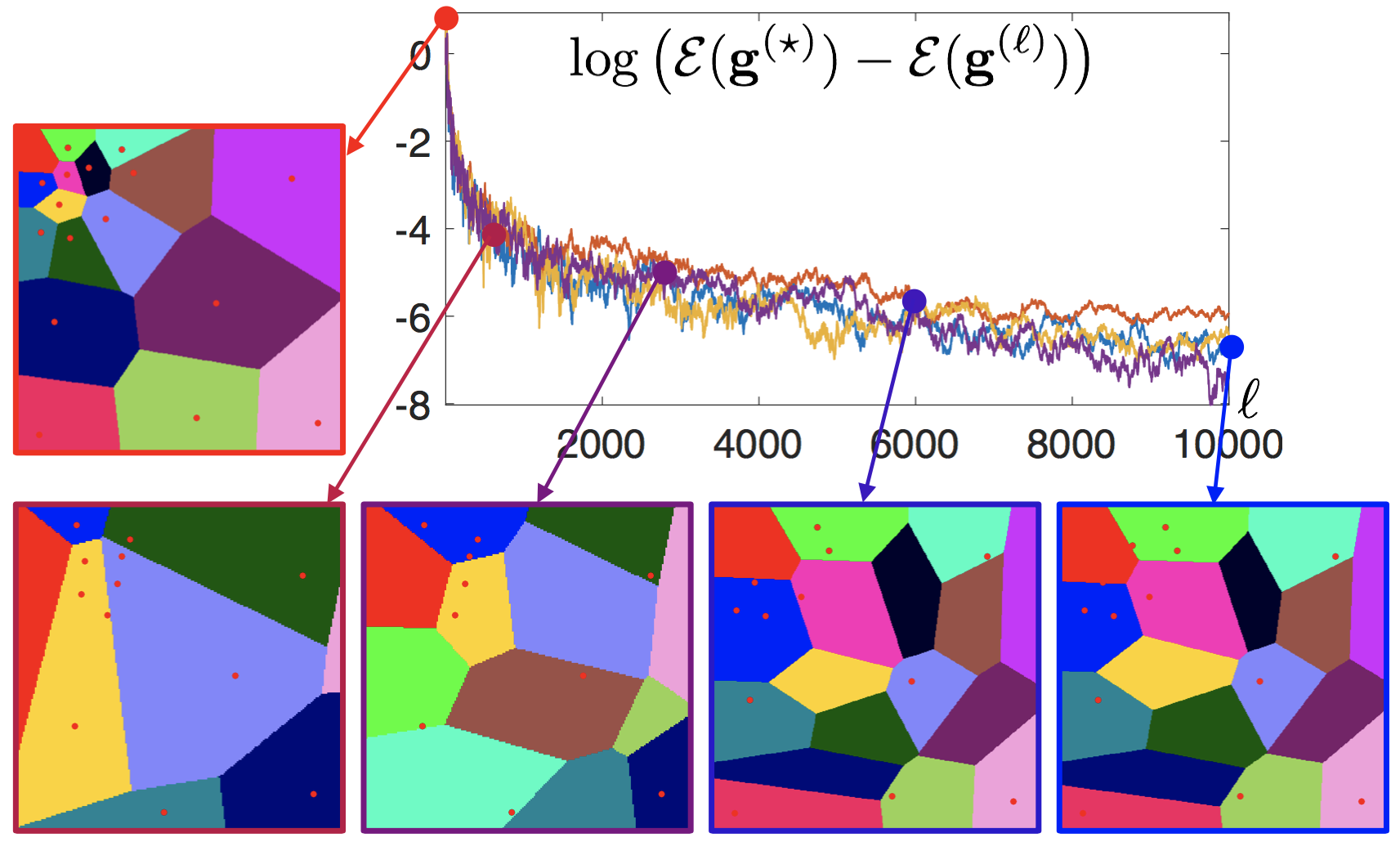}
\caption{\label{fig-semi-discrete-sgd}
Evolution of the energy $\Ee^\epsilon(\it{\gD})$, for $\epsilon=0$ (no regularization) during the SGD iterations~\eqref{eq-sgd}. Each colored curve shows a different randomized run.
The images display the evolution of the Laguerre cells $( \Laguerre_{j}(\it{\gD}) )_j$ through the iterations. 
}
\end{figure}

\paragraph{Stochastic gradient descent with averaging.}

SGD is slow because of the fast decay of
the stepsize $\tau_\ell$ toward zero.
To improve the convergence speed, it is possible to average the past
iterates, which is equivalent to running a ``classical'' SGD on auxiliary variables $(\it{\tilde\gD})_\ell$
\eq{ 	
	\itt{\tilde\gD} \eqdef  \it{\tilde\gD} + \tau_\ell \nabla_{\gD} E^\epsilon(\it{\tilde\gD},x_\ell),
}
where $x_\ell$ is drawn according to $\al$ (and all the $(x_\ell)_\ell$ are independent)
and output as estimated weight vector the average
\eq{ 
	\it{\gD} \eqdef \frac{1}{\ell} \sum_{k=1}^\ell  \tilde\gD^{(k)}. 
}
This defines the stochastic gradient descent with averaging (SGA)
algorithm.
Note that it is possible to avoid explicitly storing all the iterates by simply
updating a running average as follows:
\eq{ 
	\itt{\gD} = \frac{1}{\ell+1}  \itt{\tilde\gD} +  \frac{\ell}{\ell+1} \it{\gD}.
}
In this case, a typical choice of decay is rather of the form 
\eq{ 
	\tau_\ell \eqdef \frac{\tau_0}{1 + \sqrt{\ell/\ell_0}}. 
}
Notice that the step size now goes much slower to 0 than for~\eqref{eq-step-size-sgd}, at rate $\ell^{-1/2}$.
\citet{bach2014adaptivity} proves that SGA leads to a faster convergence (the constants involved are smaller) than SGD, since in contrast to SGD, SGA is adaptive to the local strong convexity (or concavity for maximization problems) of the functional.

\begin{rem}[Continuous-continuous problems]\label{rem-sgd-cont-cont}
When neither $\al$ nor $\be$ is a discrete measure, one cannot resort to semidiscrete strategies involving finite-dimensional dual variables, such as that given in Problem~\eqref{eq-semi-dual-discr}. The only option is to use stochastic optimization methods on the dual problem~\eqref{eq-dual-entropic}, as proposed in~\citep{genevay2016stochastic}. A suitable regularization of that problem is crucial, for instance by setting an entropic regularization strength $\epsilon>0$, to obtain an unconstrained problem that can be solved by stochastic descent schemes.
A possible approach to revisit Problem~\eqref{eq-dual-entropic} is to restrict that infinite-dimensional optimization problem over a space of continuous functions to a much smaller subset, such as that spanned by multilayer neural networks~\citep{seguy2018large}. This approach leads to nonconvex finite-dimensional optimization problems with no approximation guarantees, but this can provide an effective way to compute a proxy for the Wasserstein distance in high-dimensional scenarios.
Another solution is to use nonparametric families, which is equivalent to considering some sort of progressive refinement, as that proposed by~\citet{genevay2016stochastic} using reproducing kernel Hilbert spaces, whose dimension is proportional to the number of iterations of the SGD algorithm. 
\end{rem}


\todoK{results about embeddings, Will Leeb's work, Naor's work?}
\todoK{Indyk's embeddings for W1}

\chapter{$\Wass_1$ Optimal Transport}
\label{c-w1}

This chapter focuses on optimal transport problems in which the ground cost is equal to a distance.  Historically, this corresponds to the original problem posed by~\citeauthor{Monge1781} in \citeyear{Monge1781}; this setting was also that chosen in early applications of optimal transport in computer vision~\citep{RubTomGui00} under the name of ``earth mover's distances''.

Unlike the case where the ground cost is a \emph{squared} Hilbertian distance  (studied in particular in Chapter~\ref{c-dynamic}), transport problems where the cost is a metric are more difficult to analyze theoretically.
In contrast to Remark~\ref{rem-exist-mongemap} that states the uniqueness of a transport map or coupling between two absolutely continuous measures when using a squared metric, the optimal Kantorovich coupling is in general not unique when the cost is the ground distance itself.  Hence, in this regime it is often impossible to recover a uniquely defined Monge map, making this class of problems ill-suited for interpolation of measures.
We refer to works by~\citet{trudinger2001monge,caffarelli2002constructing,sudakov1979geometric,evans1999differential} for proofs of existence of optimal $\Wass_1$ transportation plans and detailed analyses of their geometric structure. 

Although more difficult to analyze in theory, optimal transport with a linear ground distance is usually more robust to outliers and noise than a quadratic cost. Furthermore, a cost that is a metric results in an elegant dual reformulation involving local flow, divergence constraints, or Lipschitzness of the dual potential, suggesting cheaper numerical algorithms that align with \emph{minimum-cost flow} methods over networks in graph theory.
This setting is also  popular because the associated OT distances define a norm that can compare arbitrary distributions, even if they are not positive; this property is shared by a larger class of so-called \emph{dual norms} (see \S\ref{sec-dual-norms} and Remark~\ref{rem-unb-dualnorms} for more details).

\section{$\Wass_1$ on Metric Spaces}
\label{sec-w1-metric}

Here we assume that $d$ is a distance on $\Xx=\Yy$, and we solve the OT problem with the ground cost $c(x,y)=d(x,y)$. The following proposition highlights key properties of the $c$-transform~\eqref{eq-c-transform} in this setup. In the following, we denote the Lipschitz constant of a function $f \in \Cc(\Xx)$ as
\eq{
	\Lip(f) \eqdef \sup \enscond{ \frac{|\f(x)-\f(y)|}{d(x,y)} }{ (x,y) \in \X^2, x \neq y }.
}
We define Lipschitz functions to be those functions $f$ satisfying $\Lip(f) < +\infty$; they form a convex subset of $\Cc(\X)$.

\begin{prop}Suppose $\X=\Y$ and $c(x,y)=d(x,y)$. 
Then, there exists $g$ such that $\f = g^c$ if and only $\Lip(f) \leq 1$. 
Furthermore, if $\Lip(f)\leq1$, then $f^c=-f$.
\end{prop}

\begin{proof}
First, suppose $f=g^c$.  Then, for $x,y\in \X$,
\begin{align*}
|f(x)-f(y)|&=\left|\inf_{z\in\X} d(x,z)-g(z) \;\; - \;\;\inf_{z\in \X} d(y,z)-g(z) \right|\\
&\leq \sup_{z\in\X} |d(x,z)-d(y,z)|\leq d(x,y).
\end{align*}
The first equality follows from the definition of $g^c$, the next inequality from the identity $|\inf f-\inf g|\leq\sup|f-g|$, and the last from the triangle inequality. This shows that $\Lip(f)\leq1$.

Now, suppose $\Lip(f)\leq1$, and define $g\eqdef -f$. By the Lipschitz property, for all $x,y\in\X$,
$f(y)-d(x,y)\leq f(x)\leq f(y)+d(x,y)$.
Applying these inequalities,
\begin{align*}
g^c(y)&=\inf_{x\in\X}\left[ d(x,y) + f(x) \right]
\geq\inf_{x\in \X} \left[  d(x,y)+f(y)-d(x,y)\right]=f(y),\\
g^c(y)&=\inf_{x\in\X}\left[ d(x,y) + f(x) \right]
\leq \inf_{x\in\X}\left[ d(x,y) + f(y)+d(x,y) \right]=f(y).
\end{align*}
Hence, $f=g^c$ with $g=-f$.  Using the same inequalities shows
\begin{align*}
f^c(y)&=\inf_{x\in\X}\left[ d(x,y) - f(x) \right]
\geq\inf_{x\in \X} \left[  d(x,y)-f(y)-d(x,y)\right]=-f(y),\\
f^c(y)&=\inf_{x\in\X}\left[ d(x,y) - f(x) \right]
\leq \inf_{x\in\X}\left[ d(x,y) - f(y)+d(x,y) \right]=-f(y).
\end{align*}
This shows $f^c=-f$.
\end{proof}

Starting from the single potential formulation~\eqref{eq-semi-dual-cont}, one can iterate the construction and replace the couple $(\g,\g^c)$ by $(\g^c,(\g^c)^c)$. The last proposition shows that one can thus use $(\g^c,-\g^c)$, which in turn is equivalent to any pair $(f,-f)$ such that $\Lip(f) \leq 1$. This leads to the following alternative expression for the $\Wass_1$ distance:
\eql{\label{eq-w1-metric}
	\Wass_1(\al,\be) = 
	\umax{f} \enscond{ \int_{\Xx} f(x) (\d\al(x)-\d\be(x)) }{ \Lip(f) \leq 1 }.
}
This expression shows that $\Wass_1$ is actually a norm, \ie $\Wass_1(\al,\be) = \norm{\al-\be}_{\Wass_1}$, and that it is still valid for any measures (not necessary positive) as long as $\int_{\Xx} \al=\int_{\Xx} \be$.  This norm is often called the \citeauthor{kantorovich1958space} norm~\citeyearpar{kantorovich1958space}.

For discrete measures of the form~\eqref{eq-discr-meas}, writing $\al-\be = \sum_{k} \VectMode{m}_k \de_{z_k}$ with $z_k \in \Xx$ and $\sum_k \VectMode{m}_k=0$, the optimization~\eqref{eq-w1-metric} can be rewritten as
\eql{\label{eq-w1-discr}	
	\Wass_1(\al,\be) = 
	\umax{ (\fD_k)_k } \enscond{ \sum_k \fD_k \VectMode{m}_k }{ \foralls (k,\ell), 
		\abs{\fD_k-\fD_\ell} \leq d(z_k,z_\ell), }
}
which is a finite-dimensional convex program with quadratic-cone constraints.  It can be solved using interior point methods or, as we detail next for a similar problem, using proximal methods. 

When using $d(x,y)=|x-y|$ with $\X=\RR$, we can reduce the number of constraints by ordering the $z_k$'s via $z_1 \leq z_2 \leq \ldots$.  In this case, we only have to solve
\eq{	
	\Wass_1(\al,\be) = 
	\umax{ (\fD_k)_k } \enscond{ \sum_k \fD_k \VectMode{m}_k }{ \foralls k, \abs{\fD_{k+1}-\fD_k} \leq z_{k+1}-z_k }, 
}
which is a linear program. 
Note that furthermore, in this 1-D case, a closed form expression for $\Wass_1$ using cumulative functions is given in~\eqref{eq-w1-1d}.

\begin{rem}[$\Wass_p$ with $0<p \leq 1$]
	If $0<p \leq 1$, then $\tilde d(x,y) \eqdef d(x,y)^p$ satisfies the triangular inequality, and hence $\tilde d$ is itself a distance. One can thus apply the results and algorithms detailed above for $\Wass_1$ to compute $\Wass_p$ by simply using $\tilde d$ in place of $d$. This is equivalent to stating that $\Wass_p$ is the dual of $p$-H\"older functions $\ensconds{f}{\Lip_p(f) \leq 1}$, where
	\eq{
		\Lip_p(f) \eqdef \sup \enscond{ \frac{|\f(x)-\f(y)|}{d(x,y)^p} }{ (x,y) \in \X^2, x \neq y }.
	}
\end{rem}

\section{$\Wass_1$ on Euclidean Spaces}
\label{sec-w1-eucl}

In the special case of Euclidean spaces $\X=\Y=\RR^\dim$, using $\c(x,y) = \norm{x-y}$, the global Lipschitz constraint appearing in~\eqref{eq-w1-metric} can be made local as a uniform bound on the gradient of $f$, 
\eql{\label{eq-w1-cont}
	\Wass_1(\al,\be) = 
	\umax{f} \enscond{ \int_{\RR^\dim} \f(x) (\d\al(x)-\d\be(x)) }{ \norm{\nabla \f}_\infty \leq 1 }.
}
Here the constraint $\norm{\nabla \f}_\infty \leq 1$ signifies that the norm of the gradient of $\f$ at any point $x$ is upper bounded by $1$, $\norm{\nabla \f(x)}_2 \leq 1$ for any $x$.

Considering the dual problem to~\eqref{eq-w1-cont}, one obtains an optimization problem under fixed divergence constraint
\eql{\label{eq-w1-cont-div}
	\Wass_1(\al,\be) = 
	\umin{\flow} \enscond{ \int_{\RR^\dim} \norm{\flow(x)}_2 \d x }{  \diverg(\flow)=\al-\be }, 
}
which is often called the Beckmann formulation~\citep{Beckmann52}.
Here the vectorial function $\flow(x) \in \RR^2$ can be interpreted as a flow field, describing locally the movement of mass. Outside the support of the two input measures, $\diverg(\flow)=0$, which is the conservation of mass constraint. 
Once properly discretized using finite elements, Problems~\eqref{eq-w1-cont} and~\eqref{eq-w1-cont-div} become nonsmooth convex optimization problems. 
It is possible to use an off-the-shelf interior points quadratic-cone optimization solver, but as advocated in~\S\ref{sec-prox-solvers}, large-scale problems require the use of simpler but more adapted first order methods. One can thus use, for instance, Douglas--Rachford (DR) iterations~\eqref{eq-dr-iters} or the related alternating direction method of multipliers method. Note that on a uniform grid, projecting on the divergence constraint is conveniently handled using the fast Fourier transform. We refer to~\citet{SolomonEMDSurfaces2014} for a detailed account for these approaches and application to OT on triangulated meshes.
See also~\citet{LiRyuOsherYinGangbo2017_parallel,Ryu2017a,Ryu2017b} for similar approaches using primal-dual splitting schemes.
Approximation schemes that relax the Lipschitz constraint on the dual potentials $\f$ have also been proposed, using, for instance, a constraint on wavelet coefficients leading to an explicit formula~\citep{shirdhonkar2008approximate}, or by considering only functions $\f$ parameterized as multilayer neural networks with ``rectified linear'' $\max(0,\cdot)$ activation function and clipped weights~\citep{WassersteinGAN}.


\section{$\Wass_1$ on a Graph}

The previous formulations~\eqref{eq-w1-cont} and~\eqref{eq-w1-cont-div} of $\Wass_1$ can be generalized to the setting where $\X$ is a geodesic space, \ie $\c(x,y)=\dist(x,y)$ where $\dist$ is a geodesic distance. We refer to~\citet{FeldmanMacCann02} for a theoretical analysis in the case where $\X$ is a Riemannian manifold.
When $\X=\range{1,n}$ is a discrete set, equipped with undirected edges $(i,j) \in \Edges \subset \X^2$ labeled with a weight (length) $\weight_{i,j}$, we recover the important case where $\X$ is a graph equipped with the geodesic distance (or shortest path metric):
\eq{
	\distD_{i,j} \eqdef \umin{K \geq 0, (i_k)_k : i \rightarrow j }
	 \enscond{ \sum_{k=1}^{K-1} \weight_{i_k,i_{k+1}} }{ \foralls k \in \range{1,K-1}, (i_k,i_{k+1}) \in \Ee  },
}
where $i \rightarrow j$ indicates that $i_1=i$ and $i_K=j$, namely that the path starts at $i$ and ends at $j$.

We consider two vectors $(\a,\b) \in (\RR^n)^2$ defining (signed) discrete measures on the graph $\X$ such that $\sum_i \a_i = \sum_i \b_i$ (these weights do not need to be positive). 
The goal is now to compute $\WassD_1(\a,\b)$, as introduced in~\eqref{eq-wass-p-disc} for $p=1$, when the ground metric is the graph geodesic distance. This computation should be carried out without going as far as having to compute a ``full'' coupling $\P$ of size $n \times n$, to rely instead on local operators thanks to the underlying connectivity of the graph. These operators are discrete formulations for the  gradient and divergence differential operators.

A discrete dual Kantorovich potential $\fD \in \RR^n$ is a vector indexed by all vertices of the graph. The gradient operator $\nabla : \RR^n \rightarrow \RR^{\Edges}$ is defined as
\eq{
	\foralls (i,j) \in \Edges, \quad (\nabla \fD)_{i,j} \eqdef \fD_i-\fD_j.
}
A flow $\flowD = (\flowD_{i,j})_{i,j}$ is defined on edges, and the divergence operator $\diverg : \RR^{\Edges} \rightarrow \RR^n$, which is the adjoint of the gradient $\nabla$, maps flows to vectors defined on vertices and is defined as
\eq{
	\foralls i \in \range{1,n}, \quad \diverg(\flowD)_i \eqdef \sum_{j: (i,j) \in \Edges} (\flowD_{i,j} - \flowD_{j,i}) \in \RR^n.
}

Problem~\eqref{eq-w1-cont} becomes, in the graph setting, 
\eql{\label{eq-w1-graph-potential}
	\WassD_1(\a,\b) = 
	\umax{\fD \in \RR^n} \enscond{ \sum_{i=1}^n \fD_i (\a_i-\b_i)   }{  \forall (i,j) \in \Edges, \, |(\nabla \fD)_{i,j}| \leq \weight_{i,j}  }.
}
The associated dual problem, which is analogous to Formula~\eqref{eq-w1-cont-div}, is then
\eql{\label{eq-w1-graph-div}
	\WassD_1(\a,\b) = 
	\umin{\flowD \in \RR_+^\Edges} \enscond{ \sum_{(i,j) \in \Ee} \weight_{i,j} \flowD_{i,j}  }{  \diverg(\flowD)=\a-\b }.
}
This is a linear program and more precisely an instance of min-cost flow problems. Highly efficient dedicated simplex solvers have been devised to solve it; see, for instance,~\citep{TreeEMD2007}. 
Figure~\ref{fig-planar-graph} shows an example of primal and dual solutions. 
Formulation~\eqref{eq-w1-graph-div} is the so-called Beckmann formulation~\citep{Beckmann52} and has been used and extended to define and study traffic congestion models; see, for instance,~\citep{CarlierSantambrogioOTCongestion2008}.

\begin{figure}[h!]
\centering
\begin{tabular}{@{}c@{\hspace{10mm}}c@{}}
\includegraphics[width=.45\linewidth]{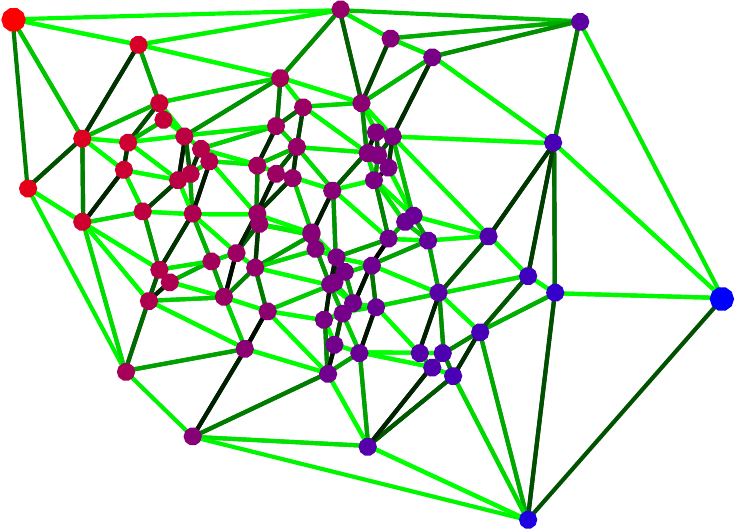}&
\includegraphics[width=.45\linewidth]{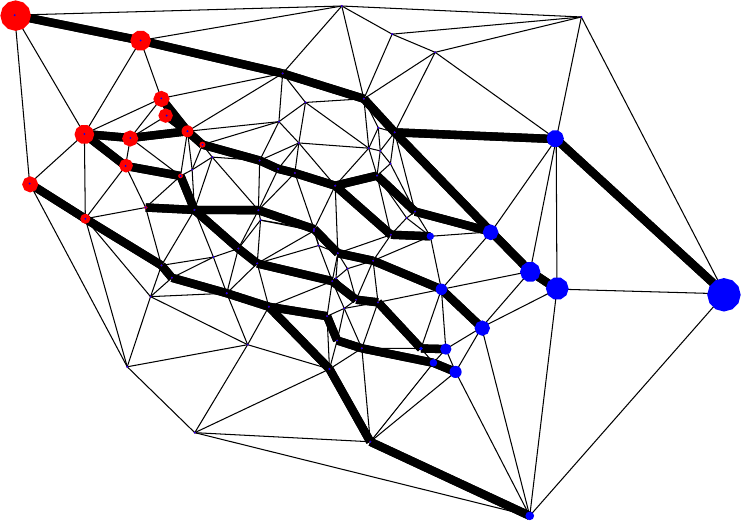}\\
$\fD$ & $({\color{red}\a},{\color{blue}\b})$ and $\flowD$ 
\end{tabular}
\caption{\label{fig-planar-graph}
Example of computation of $\WassD_1(\a,\b)$ on a planar graph with uniform weights $\weight_{i,j}=1$.
Left: potential $\fD$ solution of~\eqref{eq-w1-graph-potential} (increasing value from red to blue). The green color of the edges is proportional to $|(\nabla \fD)_{i,j}|$. 
Right: flow $\flowD$ solution of~\eqref{eq-w1-graph-div}, 
where bold black edges display nonzero $\flowD_{i,j}$, which saturate to $\weight_{i,j}=1$.
These saturating flow edge on the right match the light green edge on the left where $|(\nabla \fD)_{i,j}|=1$.
}
\end{figure}


\chapter{Dynamic Formulations}
\label{c-dynamic}

This chapter presents the geodesic (also called dynamic) point of view of optimal transport when the cost is a squared geodesic distance. This describes the optimal transport between two measures as a curve in the space of measures minimizing a total length. The dynamic point of view offers an alternative and intuitive interpretation of optimal transport, which not only allows us to draw links with fluid dynamics but also results in an efficient numerical tool to compute OT in small dimensions when interpolating between two densities. The drawback of that approach is that it cannot scale to large-scale sparse measures and works only in low dimensions on regular domains (because one needs to grid the space) with a squared geodesic cost.

In this chapter, we use the notation $(\al_0,\al_1)$ in place of $(\al,\be)$ in agreement with the idea that we start at time $t=0$ from one measure to reach another one at time $t=1$.

\section{Continuous Formulation}

In the case $\X=\Y=\RR^d$, and $c(x,y)=\norm{x-y}^2$, the optimal transport distance $\Wass_2^2(\al,\be) = \MK_\c(\al,\be)$ as defined in~\eqref{eq-mk-generic} can be computed by looking for a minimal length path $(\al_t)_{t=0}^1$ between these two measures. 
This path is described by advecting the measure using a vector field $\speed_t$ defined at each instant. The vector field $\speed_t$ and the path $\alpha_t$ must satisfy the conservation of mass formula, resulting in 
\eql{\label{eq-cons-bb-ncvx}
	\pd{\al_t}{t} + \diverg(\al_t \speed_t) = 0
	\qandq 	
	\al_{t=0}=\al_0, 
	\al_{t=1}=\al_1,
}
where the equation above should be understood in the sense of distributions on $\RR^d$. The infinitesimal length of such a vector field is measured using the $L^2$ norm associated to the measure $\al_t$, that is defined as 
\eq{\norm{\speed_t}_{L^2(\al_t)} = \left( \int_{\RR^\dim} \norm{\speed_t(x)}^2 \d\al_t(x) \right)^{1/2}.} 
This definition leads to the following minimal-path reformulation of $\Wass_2$, originally introduced by~\citet{benamou2000computational}:
\eql{\label{eq-bb-non-cvx}
	\Wass_2^2(\al_0,\al_1) = \umin{ (\al_t,\speed_t)_t \text{ sat. \eqref{eq-cons-bb-ncvx}} }
		\int_0^1 \int_{\RR^d}   \norm{\speed_t(x)}^2 \d\al_t(x) \d t,
}
where $\al_t$ is a scalar-valued measure and $\speed_t$ a vector-valued measure. Figure~\ref{fig-displacement} shows two examples of such paths of measures.

\newcommand{\MyFigDispl}[1]{\includegraphics[width=.19\linewidth]{displacement/#1}}
\begin{figure}[h!]
\centering
\begin{tabular}{@{}c@{\hspace{1mm}}c@{\hspace{1mm}}c@{\hspace{1mm}}c@{\hspace{1mm}}c@{}}
\MyFigDispl{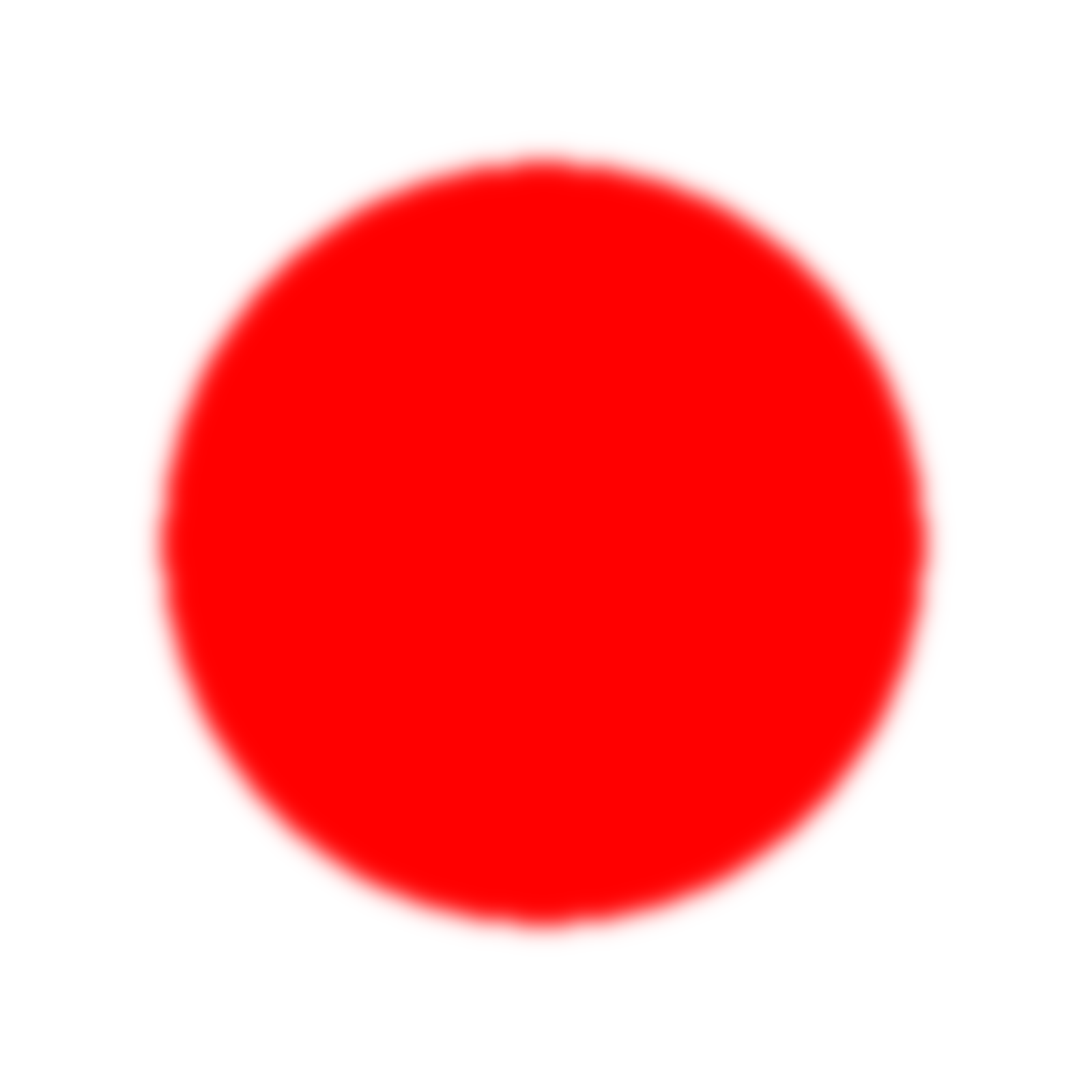} & 
\MyFigDispl{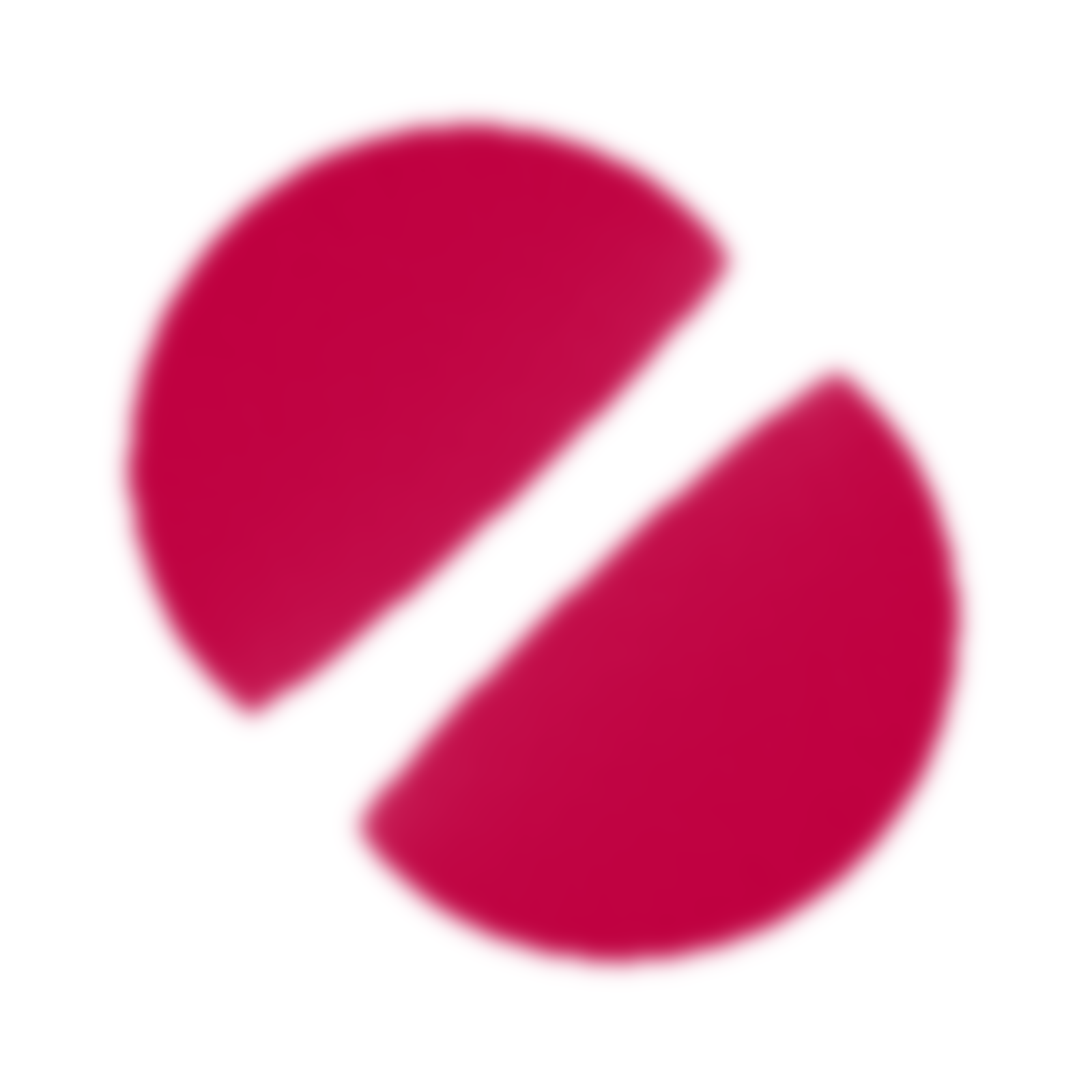} & 
\MyFigDispl{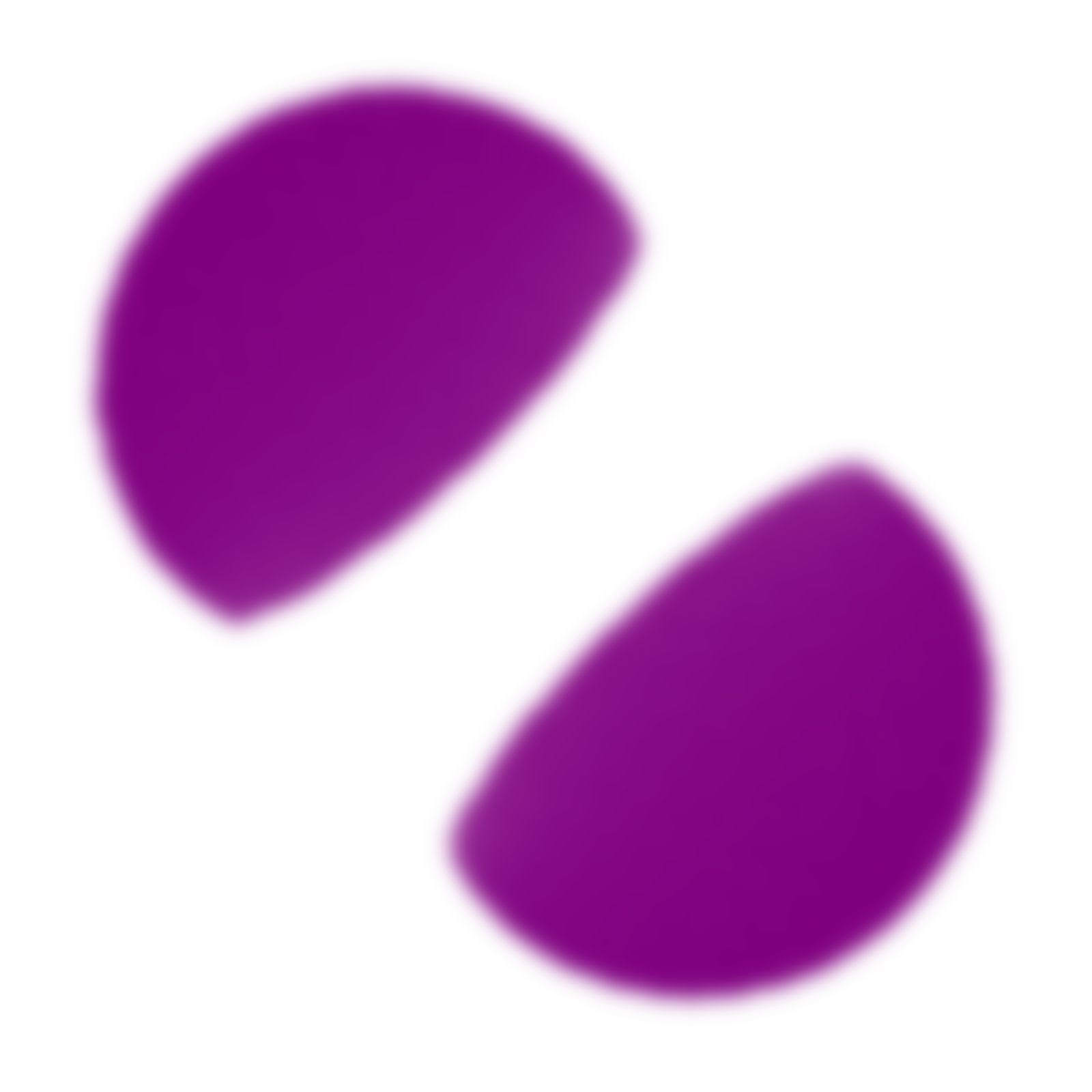} & 
\MyFigDispl{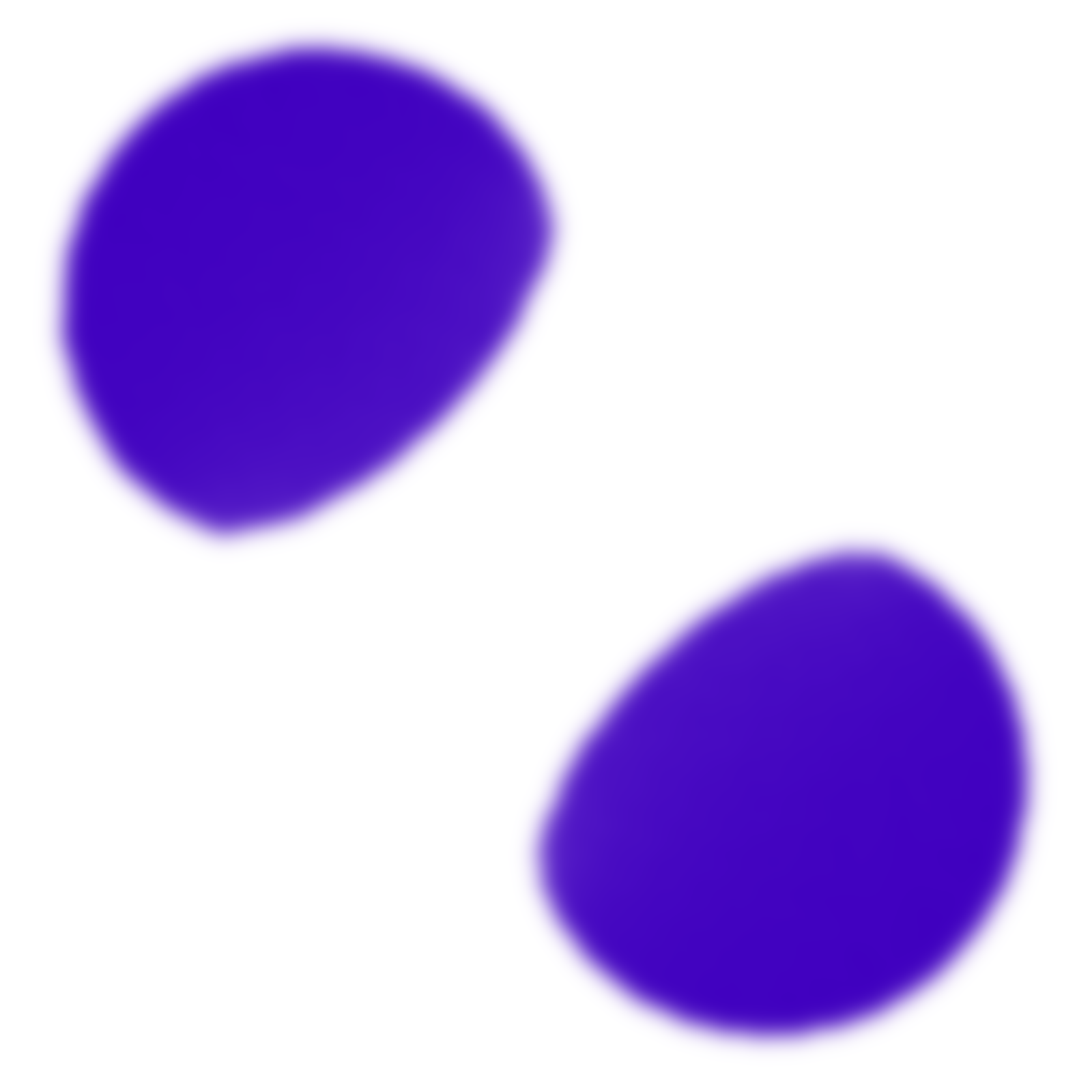} & 
\MyFigDispl{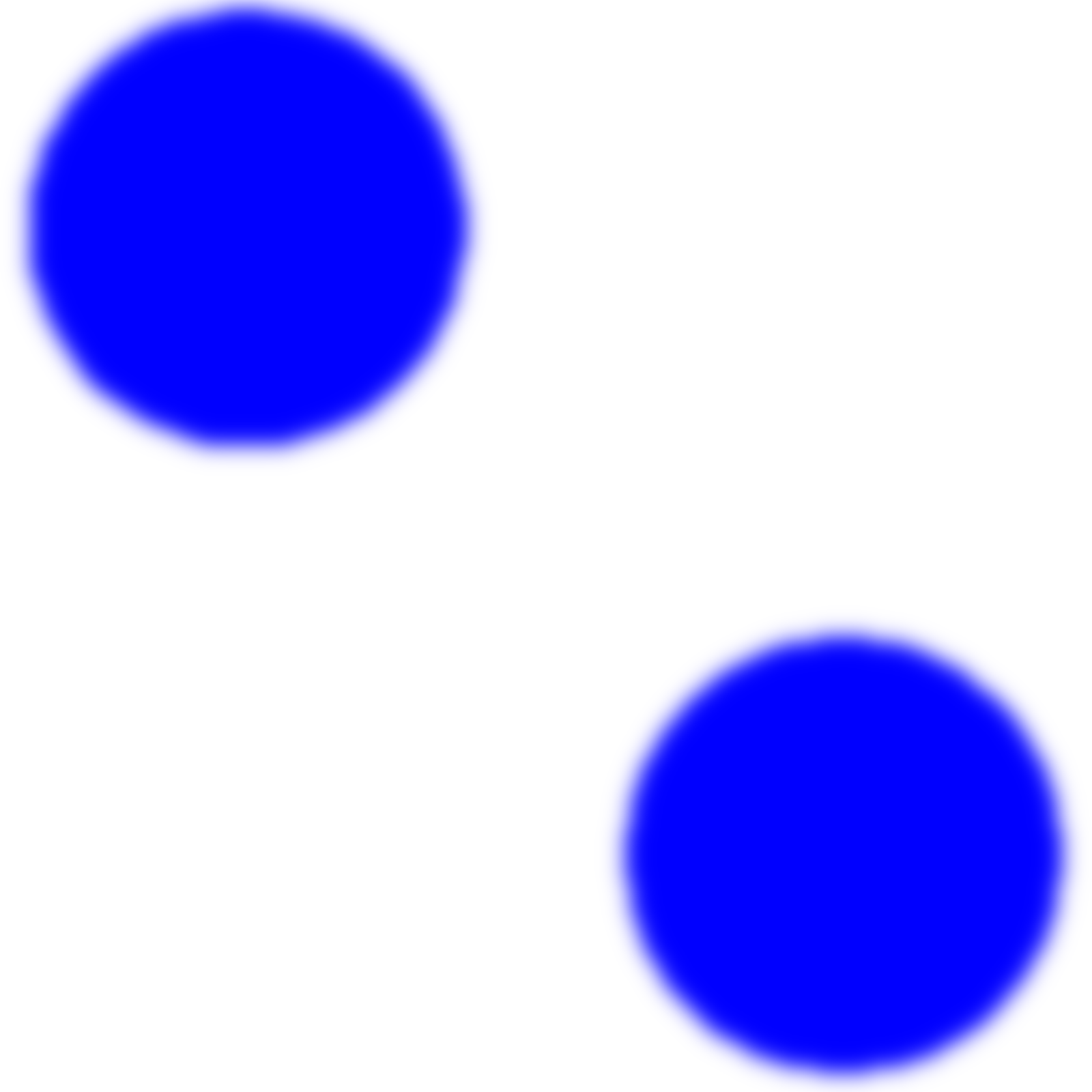} \\
\MyFigDispl{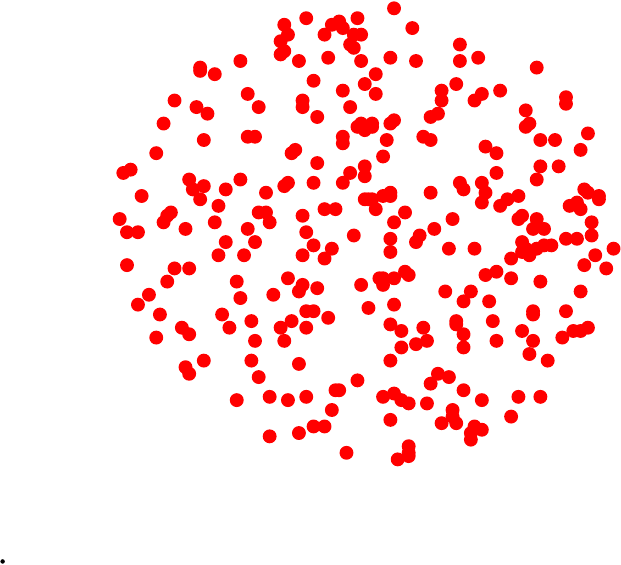} & 
\MyFigDispl{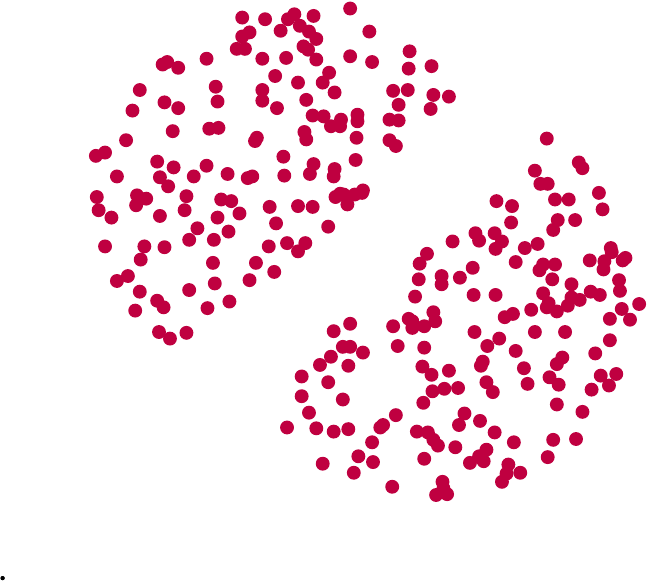} & 
\MyFigDispl{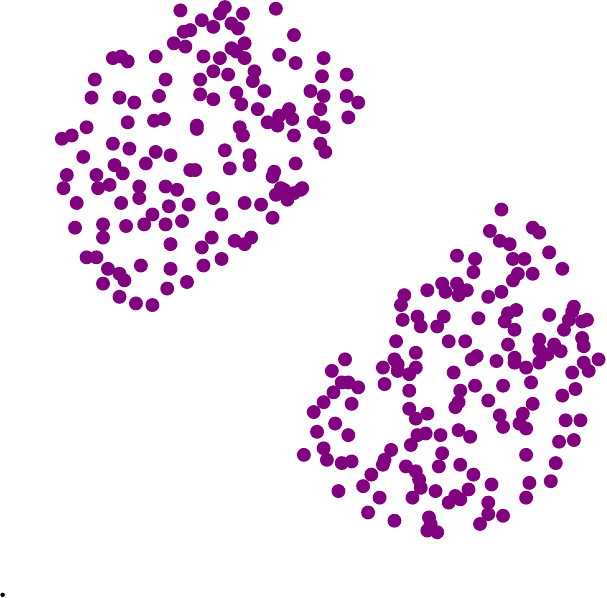} & 
\MyFigDispl{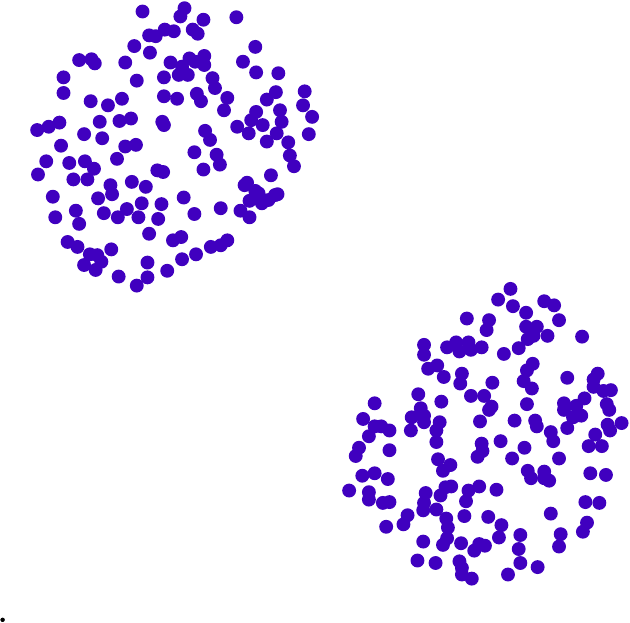} & 
\MyFigDispl{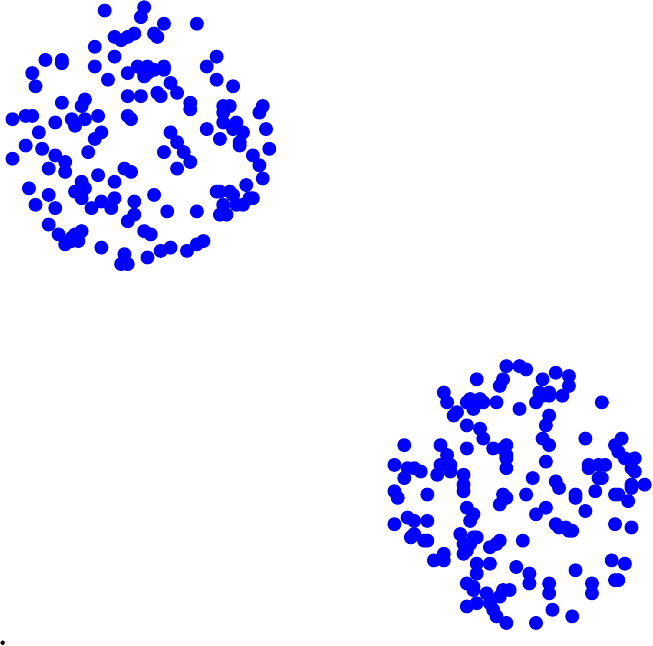} \\
$\al_0$ & $\al_{1/4}$ & $\al_{1/2}$ & $\al_{3/4}$ & $\al_{1}$
\end{tabular}
\caption{\label{fig-displacement}
Displacement interpolation $\al_t$ satisfying~\eqref{eq-bb-non-cvx}. 
Top: for two measures $(\al_0,\al_1)$ with densities with respect to the Lebesgue measure.
Bottom: for two discrete empirical measures with the same number of points (bottom).
}
\end{figure}

The formulation~\eqref{eq-bb-non-cvx} is a nonconvex formulation in the variables $(\al_t,\speed_t)_t$ because of the constraint~\eqref{eq-cons-bb-ncvx} involving the product $\al_t \speed_t$. Introducing a vector-valued measure (often called the ``momentum'')
\eq{
	\Moment_t \eqdef \al_t \speed_t, 
}
\citeauthor{benamou2000computational} showed in their landmark paper~\citeyearpar{benamou2000computational} that it is instead convex in the variable $(\al_t,\Moment_t)_t$ when writing 
\eql{\label{eq-bb-convex}
	\Wass_2^2(\al_0,\al_1) = \umin{ (\al_t,\Moment_t)_t  \in \BBConstr(\al_0,\al_1)  }
		\int_0^1 \int_{\RR^d} \theta(\al_t(x),\Moment_t(x)) \d x \d t,
}
where we define the set of constraints as
\eql{\label{eq-mass-conservation}
	\BBConstr(\al_0,\al_1) \eqdef 
	\enscond{ (\al_t,\Moment_t) }{
		\pd{\al_t}{t} + \diverg(\Moment_t) = 0, 
		\al_{t=0}=\al_0, 
		\al_{t=1}=\al_1
	}, 
}	
and where $\theta :  \rightarrow \RR^+ \cup \{+\infty\}$ is the following lower semicontinuous convex function
\eql{\label{eq-theta-dynamic}
	\foralls (a,b) \in \RR_+ \times \RR^\dim, \quad
	\th(a,b) = \choice{
		\frac{\norm{b}^2}{a} \qifq a>0, \\
		0 \qifq (a,b) = 0, \\
		+\infty \quad \text{otherwise}.
	}
}
This definition might seem complicated, but it is crucial to impose that the momentum $\Moment_t(x)$ should vanish when $\al_t(x)=0$. 
Note also that~\eqref{eq-bb-convex} is written in an informal way as if the measures $(\al_t,\Moment_t)$ were density functions, but this is acceptable because $\th$ is a 1-homogeneous function (and hence defined even if the measures do not have a density with respect to Lebesgue measure) and can thus be extended in an unambiguous way from density to functions. 

\begin{rem11}{Links with McCann's interpolation}{rem-displacement}
In the case (see Equation~\eqref{eq-brenier-map}) where there exists an optimal Monge map $\T : \RR^\dim \rightarrow \RR^\dim$ with $\T_\sharp \al_0=\al_1$, then $\al_t$ is equal to McCann's interpolation
\eql{\label{eq-mc-cann-interp}
	\al_t = ((1-t) \Id + t \T)_{\sharp} \al_0. 
}
In the 1-D case, using Remark~\ref{rem-1d-ot-generic}, this interpolation can be computed thanks to the relation
\eql{\label{eq-displacement-1d-cumul}
	\cumul{\al_t}^{-1} = (1-t)\cumul{\al_0}^{-1} + t \cumul{\al_1}^{-1}; 
}
see Figure~\ref{fig-1d-ot}. We refer to~\citet{gangbo1996geometry} for a detailed review on the Riemannian geometry of the Wasserstein space.
In the case that there is ``only'' an optimal coupling $\pi$ that is not necessarily supported on a Monge map, one can compute this interpolant as
\eql{\label{eq-displ-discr}
	\al_t = P_{t\sharp} \pi
	\qwhereq
	P_t : (x,y) \in \RR^\dim \times \RR^\dim \mapsto (1-t) x + t y.
}
For instance, in the discrete setup~\eqref{eq-pair-discr}, denoting $\P$ a solution to~\eqref{eq-mk-discr}, an interpolation is defined as
\eql{\label{eq-mccann-discrete}
	\al_t = \sum_{i,j} \P_{i,j} \de_{ (1-t) x_i + t y_j }.  
}   
Such an interpolation is typically supported on $n+m-1$ points, which is the maximum number of nonzero elements of $\P$.
Figure~\ref{fig-displacement-emp-weight} shows two examples of such displacement interpolation of discrete measures. 
This construction can be generalized to geodesic spaces $\Xx$ by replacing $P_t$ by the interpolation along geodesic paths.
McCann's interpolation finds many applications, for instance, color, shape, and illumination interpolations in computer graphics~\citep{Bonneel-displacement}. 
\end{rem11}

\newcommand{\MyFigDisplEmp}[1]{\includegraphics[width=.16\linewidth]{interpolation-discrete/empirical/interp-#1}}
\newcommand{\MyFigDisplWei}[1]{\includegraphics[width=.16\linewidth]{interpolation-discrete/weighted/interp-#1}}
\begin{figure}[h!]
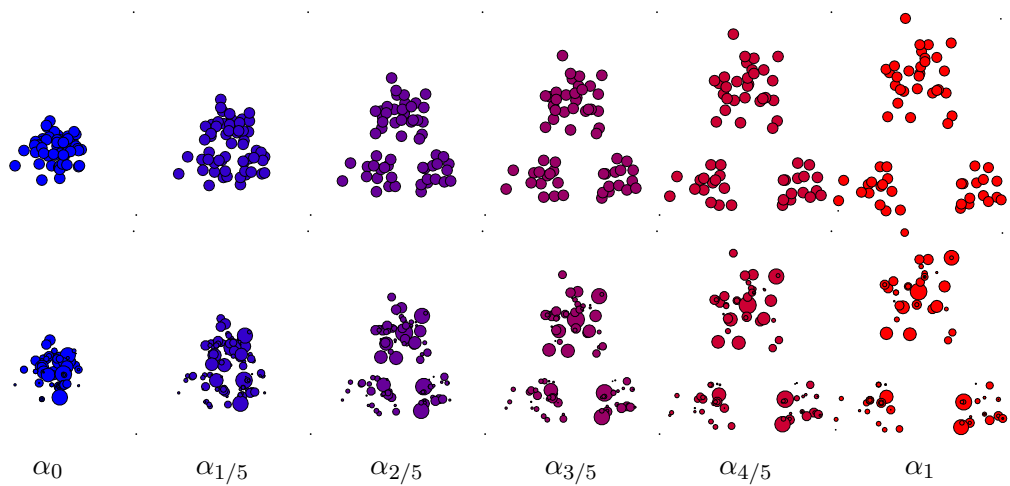

\centering
\begin{tabular}{@{}c@{}c@{}c@{}c@{}c@{}c@{}}
\MyFigDisplEmp{1}&
\MyFigDisplEmp{2}&
\MyFigDisplEmp{3}&
\MyFigDisplEmp{4}&
\MyFigDisplEmp{5}&
\MyFigDisplEmp{6}\\
\MyFigDisplWei{1}&
\MyFigDisplWei{2}&
\MyFigDisplWei{3}&
\MyFigDisplWei{4}&
\MyFigDisplWei{5}&
\MyFigDisplWei{6}\\
$\al_0$ & $\al_{1/5}$ & $\al_{2/5}$ & $\al_{3/5}$ & $\al_{4/5}$ & $\al_{1}$
\end{tabular}
\caption{\label{fig-displacement-emp-weight}
Comparison of displacement interpolation~\eqref{eq-displ-discr} of discrete measures. 
Top: point clouds (empirical measures $(\al_0,\al_1)$ with the same number of points).
Bottom: same but with varying weights. For $0<t<1$, the top example corresponds to an empirical measure interpolation $\mu_t$ with $N$ points, while the bottom one defines a measure supported on $2N-1$ points.
}
\end{figure}

\section{Discretization on Uniform Staggered Grids}

For simplicity, we describe the numerical scheme in dimension $\dim=2$; the extension to higher dimensions is straightforward. We follow the discretization method introduced by~\citet{FPapPeyOud13}, which is inspired by staggered grid techniques which are commonly used in fluid dynamics.
We discretize time as $t_k = k/T \in [0,1]$ and assume the space is uniformly discretized at points $x_i = (i_1/n_1,i_2/n_2) \in X = [0,1]^2$.
We use a staggered grid representation, so that $\al_t$ is represented using $\a \in \RR^{(T+1) \times n_1 \times n_2}$ associated to half grid points in time, whereas $\Moment$ is represented using $\MomentD=(\MomentD_1,\MomentD_2)$, where $\MomentD_1 \in \RR^{T \times (n_1+1) \times n_2}$ and $\MomentD_1 \in \RR^{T \times n_1 \times (n_2+1)}$ are stored at half grid points in each space direction.
Using this representation, for $(k,i_1,i_2) \in \range{1,T} \times \range{1,n_1} \times \range{1,n_2}$, the time derivative is computed as
\eq{
	(\partial_t \a)_{k,i} \eqdef \a_{k+1,i}-\a_{k,i}
}
and spatial divergence as
\eql{\label{eq-space-diverg-disc}
	\diverg(\MomentD)_{k,i} \eqdef
			\MomentD_{k,i_1+1,i_2}^1-\MomentD_{k,i_1,i_2}^1 
			+
			\MomentD_{k,i_1,i_2+1}^2-\MomentD_{k,i_1,i_2}^2,	
} 
which are both defined at grid points, thus forming arrays of $\RR^{T \times n_1 \times n_2}$.

In order to evaluate the functional to be optimized, one needs interpolation operators from midgrid points to grid points,
for all $(k,i_1,i_2) \in \range{1,T} \times \range{1,n_1} \times \range{1,n_2}$,  
\eq{
	\InterpBB_a( \a )_{k,i} \eqdef \InterpBB(\a_{k+1,i},\a_{k,i}),
}
\eq{
	\InterpBB_\Moment( \MomentD )_{k,i} \eqdef ( 
		\InterpBB(\MomentD_{k,i_1+1,i_2}^1,\MomentD_{k,i_1,i_2}^1),
		\InterpBB(\MomentD_{k,i_1,i_2+1}^2,\MomentD_{k,i_1,i_2}^2)	
	 ).
}
The simplest choice is to use a linear operator $\InterpBB(r,s)=\frac{r+s}{2}$, which is the one we consider next. The discrete counterpart to~\eqref{eq-bb-convex} reads
\eql{\label{eq-bb-discrete}
	\umin{(\a,\MomentD) \in \BBConstrD(\a_0,\a_1)} 
		\Theta(\InterpBB_a(\a),\InterpBB_\Moment( \MomentD )),
}
\eq{
		\qwhereq
		\Theta(\tilde\a,\tilde\MomentD) 
		\eqdef \sum_{k=1}^{T} \sum_{i_1=1}^{n_1} \sum_{i_2=1}^{n_2}
		\th( \tilde\a_{k,i}, \tilde\MomentD_{k,i} ), 
}
and where the constraint now reads
\eq{
	\BBConstrD(\a_0,\a_1) \eqdef 
	\enscond{ (\a, \MomentD)
	}{
		\partial_t \a + \diverg(\MomentD) = 0, 
		(\a_{0,\cdot},\a_{T,\cdot}) = (\a_0,\a_1)
	},
}	
where $\a \in \RR^{(T+1) \times n_1 \times n_2}$, $\MomentD=(\MomentD_1,\MomentD_2)$ with $\MomentD_1 \in \RR^{T \times (n_1+1) \times n_2}$, $\MomentD_2 \in \RR^{T \times n_1 \times (n_2+1)}$.
Figure~\ref{fig-dynamic-maze} shows an example of evolution $(\al_t)_t$ approximated using this discretization scheme. 

\begin{rem}[Dynamic formulation on graphs]
In the case where $\X$ is a graph and $c(x,y)=d_\X(x,y)^2$ is the squared geodesic distance, it is possible to derive faithful discretization methods that use a discrete divergence associated to the graph structure in place of the uniform grid discretization~\eqref{eq-space-diverg-disc}. In order to ensure that the heat equation has a gradient flow structure (see~\S\ref{sec-grad-flows} for more details about gradient flows) for the corresponding dynamic Wasserstein distance, ~\citet{Maas2011} and later ~\citet{MielkeCVPDE} proposed to use a logarithmic mean $\InterpBB(r,s)$ (see also~\citep{solomon2016continuous,ChowHuangLiZhou2012,chow2017entropy,chow2017discrete}). 
\end{rem}

\section{Proximal Solvers}
\label{sec-prox-solvers}

The discretized dynamic OT problem~\eqref{eq-bb-discrete} is challenging to solve because it requires us to minimize a nonsmooth optimization problem under affine constraints. Indeed, the function $\th$ is convex but nonsmooth for measures with vanishing mass $\a_{k,i}$. When interpolating between two compactly supported input measures $(\a_0,\a_1)$, one typically expects the mass of the interpolated measures $(\a_{k})_{k=1}^T$ to vanish as well, and the difficult part of the optimization process is indeed to track this evolution of the support. In particular, it is not possible to use standard smooth optimization techniques. 

There are several ways to recast~\eqref{eq-bb-discrete} into a quadratic-cone program, either by considering the dual problem or simply by replacing the functional $\th( \a_{k,i}, \MomentD_{k,i} )$ by a linear function under constraints,
\eq{
	\Theta(\tilde\a,\tilde\MomentD) = \umin{\tilde \zD} \enscond{
		\sum_{k,i} \tilde \zD_{k,i}
	}{
		\foralls (k,i), (\zD_{k,i},\tilde\a_{k,i},\tilde\MomentD_{i,j}) \in \Ll
	}, 
}
which thus requires the introduction of an extra variable $\tilde \zD$.
Here $\Ll \eqdef \ensconds{ (z,a,J) \in \RR \times \RR^+ \times \RR^\dim }{ \norm{J}^2 \leq z a }$ is a rotated Lorentz quadratic-cone.
With this extra variable, it is thus possible to solve the discretized problem using standard interior point solvers for quadratic-cone programs~\citep{nesterov1994interior}. These solvers have fast convergence rates and are thus capable of computing a solution with high precision. Unfortunately, each iteration is costly and requires the resolution of a linear system of dimension that scales with the number of discretization points. They are thus not applicable for large-scale multidimensional problems encountered in imaging applications.

An alternative to these high-precision solvers are low-precision first order methods, which are well suited for nonsmooth but highly structured problems such as~\eqref{eq-bb-discrete}. While this class of solvers is not new, it has recently been revitalized in the fields of imaging and machine learning because they are the perfect fit for these applications, where numerical precision is not the driving goal.
We refer, for instance, to the monograph~\citep{BauschkeCombettes11} for a detailed account on these solvers and their use for large-scale applications. 
We here concentrate on a specific solver, but of course many more can be used, and we refer to~\citep{FPapPeyOud13} for a study of several such approaches for dynamical OT. Note that the idea of using a first order scheme for dynamical OT was initially proposed by~\citet{benamou2000computational}.

The DR algorithm~\citep{Lions-Mercier-DR} is specifically tailored to solve nonsmooth structured problems of the form
\eql{\label{eq-splitting-opt}
	\umin{x \in \Hh} F(x) + G(x),
}
where $\Hh$ is some Euclidean space, 
and where $F, G : \Hh \rightarrow \RR \cup \{+\infty\}$ are two closed convex functions, for which one can ``easily '' (\eg in closed form or using a rapidly converging scheme) compute the so-called proximal operator 
\eql{\label{eq-prox-eucl}
	\foralls x \in \Hh, \quad
	\Prox_{\tau F}(x) \eqdef \uargmin{x' \in \Hh} \frac{1}{2}\norm{x-x'}^2 + \tau F(x)
}
for a parameter $\tau>0$. Note that this corresponds to the proximal map for the Euclidean metric and that this definition can be extended to more general Bregman divergence in place of $\norm{x-x'}^2$; see~\eqref{eq-prox-kl} for an example using the $\KLD$ divergence.
The iterations of the DR algorithm define a sequence $(\it{x}, \it{w}) \in \Hh^2$ using an initialization $(x^{(0)}, w^{(0)})\in \Hh^2$ and 
\begin{equation}\label{eq-dr-iters}
\begin{aligned}
	\itt{w} &\eqdef  \it{w} + \al (\Prox_{\ga F} (2\it{x}-\it{w})-\it{x}),\\
	\itt{x} &\eqdef  \Prox_{\ga G}(\itt{w}).
\end{aligned}
\end{equation}
If $0 < \al < 2$ and $\ga > 0$, one can show that $\it{x} \rightarrow z^\star$, where $z^\star$ is a solution of~\eqref{eq-splitting-opt}; see~\citep{Combettes2007} for more details. 
This algorithm is closely related to another popular method, the alternating direction method of multipliers~\citep{GabayMercier,GlowinskiMarroco} (see also~\citep{BoydADMM} for a review), which can be recovered by applying DR on a dual problem; see~\citep{FPapPeyOud13} for more details on the equivalence between the two, first shown by~\citep{Eckstein1992}.

There are many ways to recast Problem~\eqref{eq-bb-discrete} in the form~\eqref{eq-splitting-opt}, and we refer to~\citep{FPapPeyOud13} for a detailed account of these approaches. A simple way to achieve this is by setting $x=(\a,\MomentD,\tilde\a,\tilde\MomentD)$ and letting
\eq{
	F(x) \eqdef \Theta(\tilde\a,\tilde\MomentD) + \iota_{\BBConstrD(\a_0,\a_1)}(\a,\MomentD)
	\qandq
	G(x) = \iota_{\Dd}(\a,\MomentD,\tilde\a,\tilde\MomentD),
}
\eq{
	\qwhereq \Dd \eqdef \enscond{(\a,\MomentD,\tilde\a,\tilde\MomentD)}{
		\tilde\a=\InterpBB_a(\a), \tilde\MomentD=\InterpBB_\Moment( \MomentD )
	}.
}
The proximal operator of these two functions can be computed efficiently.
Indeed, one has
\eq{
	\Prox_{\tau F}(x) = ( \Prox_{\tau\Theta}(\tilde\a,\tilde\MomentD), \Proj_{\BBConstrD(\a_0,\a_1)}(\a,\MomentD) ).
}
The proximal operator $\Prox_{\tau\Theta}$ is computed by solving a cubic polynomial equation at each grid position. The orthogonal projection on the affine constraint $\BBConstrD(\a_0,\a_1)$ involves the resolution of a Poisson equation, which can be achieved in $O(N\log(N))$ operations using the fast Fourier transform, where $N=Tn_1 n_2$ is the number of grid points.
Lastly, the proximal operator $\Prox_{\tau G}$ is a linear projector, which requires the inversion of a small linear system.
We refer to~\citet{FPapPeyOud13} for more details on these computations. 
Figure~\ref{fig-dynamic-maze} shows an example in which that method is used to compute a dynamical interpolation inside a complicated planar domain. 
This class of proximal methods for dynamical OT has also been used to solve related problems such as mean field games~\citep{benamou2015augmented}.

\begin{figure}[h!]
\centering
\includegraphics[width=\linewidth]{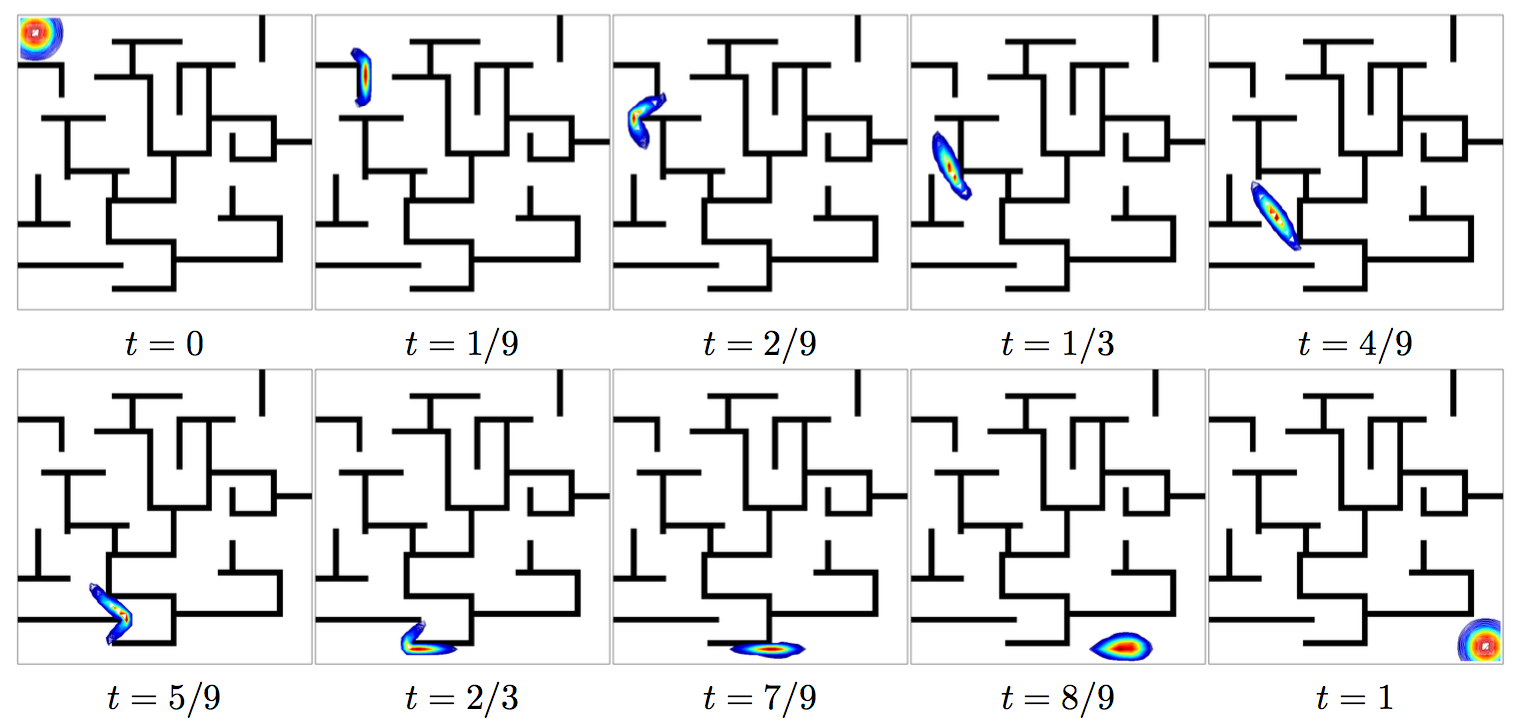}
\caption{\label{fig-dynamic-maze}
Solution $\al_t$ of dynamic OT computed with a proximal splitting scheme.
}
\end{figure}

\section{Dynamical Unbalanced OT}
\label{dynamic-unbalanced}

In order to be able to match input measures with different mass $\al_0(\Xx) \neq \al_1(\Xx)$ (the so-called ``unbalanced'' settings, the terminology introduced by~\citet{benamou2003numerical}), and also to cope with local mass variation,
several normalizations or relaxations have been proposed, in particular by relaxing the fixed marginal constraint; see~\S\ref{sec-unbalanced}. 
A general methodology consists in introducing a source term $s_t(x)$ in the continuity equation~\eqref{eq-mass-conservation}. We thus consider 
\eq{
	\bar\BBConstr(\al_0,\al_1) \eqdef 
	\enscond{ (\al_t,\Moment_t,s_t) }{
		\pd{\al_t}{t} + \diverg(\Moment_t) = s_t, 
		\al_{t=0}=\al_0, 
		\al_{t=1}=\al_1
	}.
}
The crucial question is how to measure the cost associated to this source term and introduce it in the original dynamic formulation~\eqref{eq-bb-convex}. Several proposals appear in the literature, for instance, using an $L^2$ cost~\citet{piccoli2014generalized}.  
In order to avoid having to ``teleport'' mass (mass which travels at infinite speed and suddenly grows in a region where there was no mass before), the associated cost should be infinite. It turns out that this can be achieved in a simple convex way, by also allowing $s_t$ to be an arbitrary measure (\emph{e.g.} using a 1-homogeneous cost) by penalizing $s_t$ in the same way as the momentum $\Moment_t$, 
\eql{\label{eq-wfr-dynamic}
	\WFR^2(\al_0,\al_1) = \umin{ (\al_t,\Moment_t,s_t)_t  \in \bar\BBConstr(\al_0,\al_1)  }
		\Theta(\al,\Moment,s),
}
\eq{
	\qwhereq \Theta(\al,\Moment,s) \eqdef 
		\int_0^1 \int_{\RR^d} \pa{ 
			\theta(\al_t(x),\Moment_t(x)) + \tau \theta(\al_t(x),s_t(x)) 
		}\d x \d t,
}
where $\th$ is the convex 1-homogeneous function introduced in~\eqref{eq-theta-dynamic}, and $\tau$ is a weight controlling the trade-off between mass transportation and mass creation/destruction. This formulation was proposed independently by several authors~\citep{LieroMielkeSavareShort,2017-chizat-focm,kondratyev2015}. This ``dynamic'' formulation has a ``static'' counterpart; see Remark~\ref{rem-wfr-static}. 
The convex optimization problem~\eqref{eq-wfr-dynamic} can be solved using methods similar to those detailed in~\S\ref{sec-prox-solvers}. 
Figure~\ref{fig-unbalanced-dynamic} displays a comparison of several unbalanced OT dynamic interpolations.
This dynamic formulation resembles ``metamorphosis'' models for shape registration~\citep{Metamorphosis2005}, and a more precise connection is detailed in~\citep{maas2015generalized,maas2016generalized}. 

As $\tau \rightarrow 0$, and if $\al_0(\Xx) = \al_1(\Xx)$, then one retrieves the classical OT problem, $\WFR(\al_0,\al_1) \rightarrow \Wass(\al_0,\al_1)$. In contrast, as $\tau \rightarrow +\infty$, this distance approaches the Hellinger metric over densities
\begin{align*}
	\frac{1}{\tau}\WFR(\al_0,\al_1)^2\; \overset{\tau \rightarrow +\infty} \longrightarrow 
	&\int_{\Xx} | \sqrt{\density{\al_0}(x)} - \sqrt{\density{\al_1}(x)} |^2 \d x \\
	=& \int_{\Xx} | 1 - \sqrt{ \frac{\d\al_1}{\d\al_0}(x) } |^2 \d \al_0(x).
\end{align*}

\begin{figure}[h!]
\centering
\includegraphics[width=\linewidth]{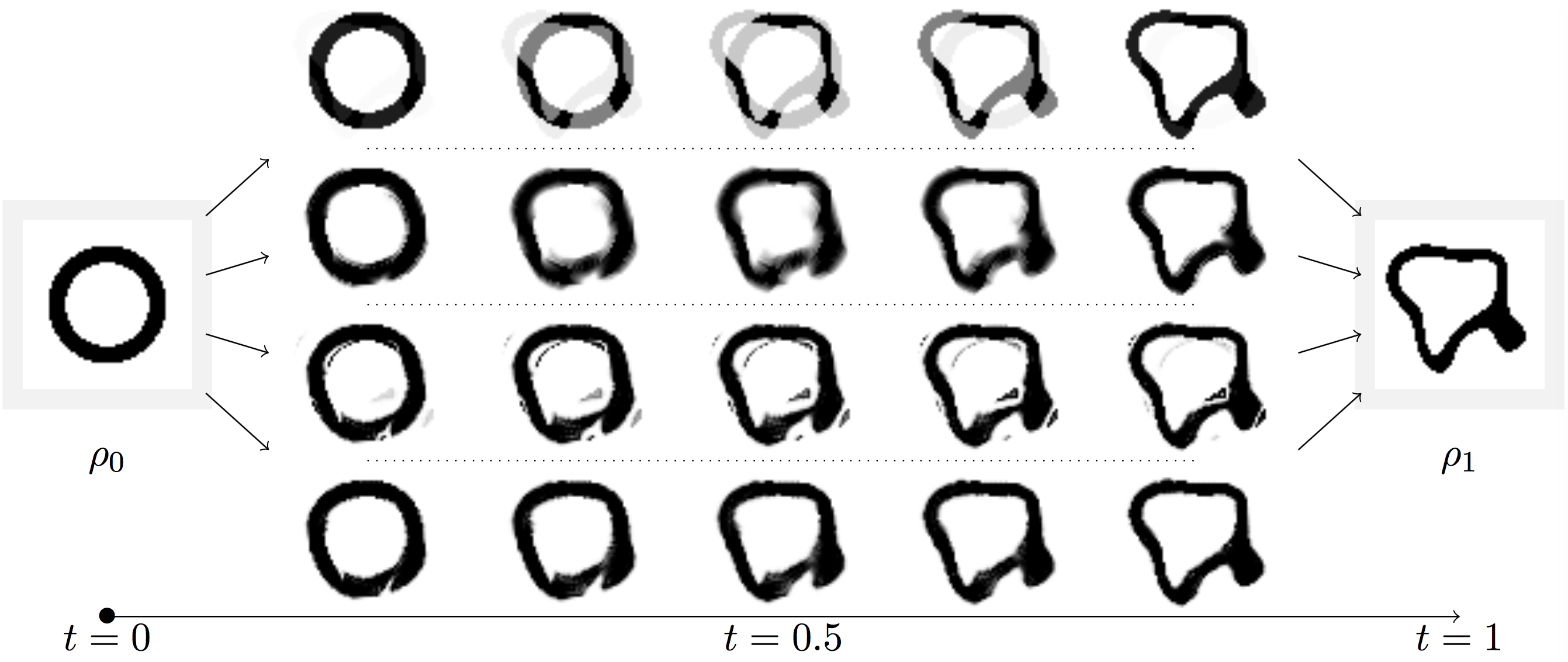}
\caption{\label{fig-unbalanced-dynamic}
Comparison of Hellinger (first row), Wasserstein (row 2), partial optimal transport (row 3), and Wasserstein--Fisher--Rao (row 4) dynamic interpolations.
}
\end{figure}

\section{More General Mobility Functionals}
\label{eq-generic-modibility}

It is possible to generalize the dynamic formulation~\eqref{eq-bb-convex} by considering other ``mobility functions'' $\th$ in place of the one defined in~\eqref{eq-theta-dynamic}. A possible choice for this mobility functional is proposed in~\citet{dolbeault2009new},
\eql{\label{eq?generic-mobility}
	\foralls (a,b) \in \RR_+ \times \RR^\dim, \quad
	\th(a,b) = a^{s-p} \norm{b}^{p},
}
where the parameter should satisfy $p \geq 1$ and $s \in [1,p]$ in order for $\th$ to be convex. 
Note that this definition should be handled with care in the case $1<s\leq p$ because $\th$ does not have a linear growth at infinity, so that solutions to~\eqref{eq-bb-convex} must be constrained to have a density with respect to the Lebesgue measure. 

The case $s=1$ corresponds to the classical OT problem and the optimal value of~\eqref{eq-bb-convex} defines $\Wass_p(\al,\be)$. In this case, $\th$ is 1-homogeneous, so that solutions to~\eqref{eq-bb-convex} can be arbitrary measures. The case $(s=1,p=2)$ is the initial setup considered in~\eqref{eq-bb-convex} to define $\Wass_2$. 

The limiting case $s=p$ is also interesting, because it corresponds to a dual Sobolev norm $W^{-1,p}$ and the value of~\eqref{eq-bb-convex} is then equal to
\eq{
	\norm{\al-\be}_{W^{-1,p}(\RR^\dim)}^p = 
	\umin{\f} \enscond{ \int_{\RR^\dim} \f  \d(\al-\be)}{  \int_{\RR^\dim} \norm{ \nabla f(x) }^q \d x \leq 1 }
}
for $1/q+1/p=1$.
In the limit $(p=s,q) \rightarrow (1,\infty)$, one recovers the $\Wass_1$ norm.
The case $s=p=2$ corresponds to the Sobolev $H^{-1}(\RR^\dim)$ Hilbert norm defined in~\eqref{eq-dual-sobolev-div}.

\todoK{Speak of the connexion with geodesic in shape space~\citet{beg2005computing}. }

\todoK{``Here many computation techniques related to optimal control and Hamilton-Jacobi equations have not mentioned. The connections between optimal transport and Mean field games may be the other interesting direction. Many techniques developed in numerical optimal transport can be applied to general control problems in the set of probability space. For example, a new Hopf-Lax formula is introduced for dynamical optimal transport problem. See related discussions in~\citet{gangbo2008hamilton}.''}

\section{Dynamic Formulation over the Paths Space}
\label{sec-entropic-dynamic}

There is a natural dynamical formulation of both classical and entropic regularized (see~\S\ref{c-entropic}) formulations of OT, which is based on studying abstract optimization problems on the space $\bar\Xx$ of all possible paths $\ga : [0,1] \rightarrow \Xx$ (\ie curves) on the space $\Xx$. 
For simplicity, we assume $\Xx=\RR^\dim$, but this extends to more general spaces such as geodesic spaces and graphs.
Informally, the dynamic of ``particles'' between two input measures $\al_0,\al_1$ at times $t=0,1$ is described by a probability distribution $\bar\pi \in \Mm_+^1(\bar\Xx)$. Such a distribution should satisfy that the distributions of starting and end points must match $(\al_0,\al_1)$, which is formally written using push-forward as
\eq{
	\bar\Couplings(\al_0,\al_1) \eqdef \enscond{ \bar\pi \in \Mm_+^1(\bar\Xx) }{ \bar P_{0\sharp} \bar\pi = \al_0, \bar P_{1\sharp} \bar\pi = \al_1  },
}
where, for any path $\ga \in \bar\Xx$, $P_0(\ga) = \ga(0), P_1(\ga) = \ga(1)$.

\paragraph{OT over the space of paths.}

The dynamical version of classical OT~\eqref{eq-mk-generic}, formulated over the space of paths, then reads
\eql{\label{eq-ot-pathsspace}
	\Wass_2(\al_0,\al_1)^2 =
	\umin{ \bar\pi \in \bar\Couplings(\al_0,\al_1) } \int_{\bar\Xx} \Ll(\ga)^2 \d\bar\pi(\ga),
}
where $\Ll(\ga) = \int_0^1 |\ga'(s)|^2 \d s$ is the kinetic energy of a path $s \in [0,1] \mapsto \ga(s) \in \Xx$. The connection between optimal couplings $\pi^\star$ and $\bar\pi^\star$ solving respectively~\eqref{eq-ot-pathsspace} and~\eqref{eq-mk-generic} is that $\bar\pi^\star$ only gives mass to geodesics joining pairs of points in proportion prescribed by $\pi^\star$. In the particular case of discrete measures, this means that 
\eq{
	\pi^\star = \sum_{i,j} \P_{i,j} \de_{(x_i,y_j)}
	\qandq
	\bar\pi^\star = \sum_{i,j} \P_{i,j} \de_{\ga_{x_i,y_j}},
} 
where $\ga_{x_i,y_j}$ is the geodesic between $x_i$ and $y_j$. Furthermore, the measures defined by the distribution of the curve points $\ga(t)$ at time $t$, where $\ga$ is drawn following $\bar\pi^\star$, \ie
\eql{\label{eq-interpolation-spacepaths}
	 t \in [0,1] \mapsto \al_t \eqdef P_{t\sharp} \bar\pi^\star 
	 \qwhereq
	 P_t(\ga) = \ga(t) \in \Xx,  
}
is a solution to the dynamical formulation~\eqref{eq-bb-convex}, \ie it is the displacement interpolation. 
In the discrete case, one recovers~\eqref{eq-mccann-discrete}. 

\paragraph{Entropic OT over the space of paths.}

We now turn to the re-interpretation of entropic OT, defined in Chapter~\ref{c-entropic}, using the space of paths. 
Similarly to~\eqref{eq-entropic-generic-proj}, this is defined using a Kullback--Leibler projection, but this time of a reference measure over the space of paths $\bar\Kk$ which is the distribution of a reversible Brownian motion (Wiener process), which has a uniform distribution at the initial and final times
\eql{\label{eq-shrodinger-paths}
	\umin{ \bar\pi \in \bar\Couplings(\al_0,\al_1) } \KL(\bar\pi|\bar\Kk).
}
We refer to the review paper by~\citet{LeonardSchroedinger} for an overview of this problem and an historical account of the work of~\citet{Schroedinger31}. 
One can show that the (unique) solution $\bar\pi_\varepsilon^\star$ to~\eqref{eq-shrodinger-paths} converges to a solution of~\eqref{eq-ot-pathsspace} as $\varepsilon \rightarrow 0$. Furthermore, this solution is linked to the solution of the static entropic OT problem~\eqref{eq-entropic-generic} using Brownian bridge $\bar\ga_{x,y}^\varepsilon \in \bar\Xx$ (which are similar to fuzzy geodesic and converge to $\de_{\ga_{x,y}}$ as $\varepsilon \rightarrow 0$). In the discrete setting, this means that 
\eql{\label{eq-brownian-bridge-interp}
	\pi_\varepsilon^\star = \sum_{i,j} \P_{\varepsilon,i,j}^\star \de_{(x_i,y_j)}
	\qandq
	\bar\pi_\varepsilon^\star = \sum_{i,j} \P_{\varepsilon,i,j}^\star \bar\ga_{x_i,y_j}^\varepsilon,
}
where $\P_{\varepsilon,i,j}^\star$ can be computed using Sinkhorn's algorithm. 
Similarly to~\eqref{eq-interpolation-spacepaths}, one then can define an entropic interpolation as 
\eq{
	\al_{\varepsilon,t} \eqdef \P_{t\sharp} \bar\pi_\varepsilon^\star. 
}
Since the law $\P_{t\sharp} \bar\ga_{x,y}^\varepsilon$ of the position at time $t$ along a Brownian bridge is a Gaussian $\Gg_{t(1-t)\varepsilon^2}(\cdot-\ga_{x,y}(t))$ of variance $t(1-t)\varepsilon^2$ centered at $\ga_{x,y}(t)$, one can deduce that $\al_{\varepsilon,t}$ is a Gaussian blurring of a set of traveling Diracs
\eq{
	\al_{\varepsilon,t} =
	\sum_{i,j} \P_{\varepsilon,i,j}^\star \Gg_{t(1-t)\varepsilon^2}( \cdot - \ga_{x_i,y_j}(t) ). 
}
The resulting mixture of Brownian bridges is displayed on Figure~\ref{fig-schrodinger-dynamic}.

\newcommand{\MyFigSchrDyn}[1]{\includegraphics[width=.245\linewidth,trim=50 45 38 35,clip]{schrodinger-dynamic/schrodinger-dynamic-eps#1}}
\begin{figure}[h!]\centering
\centering
\begin{tabular}{@{}c@{}c@{}c@{}c@{}}
\MyFigSchrDyn{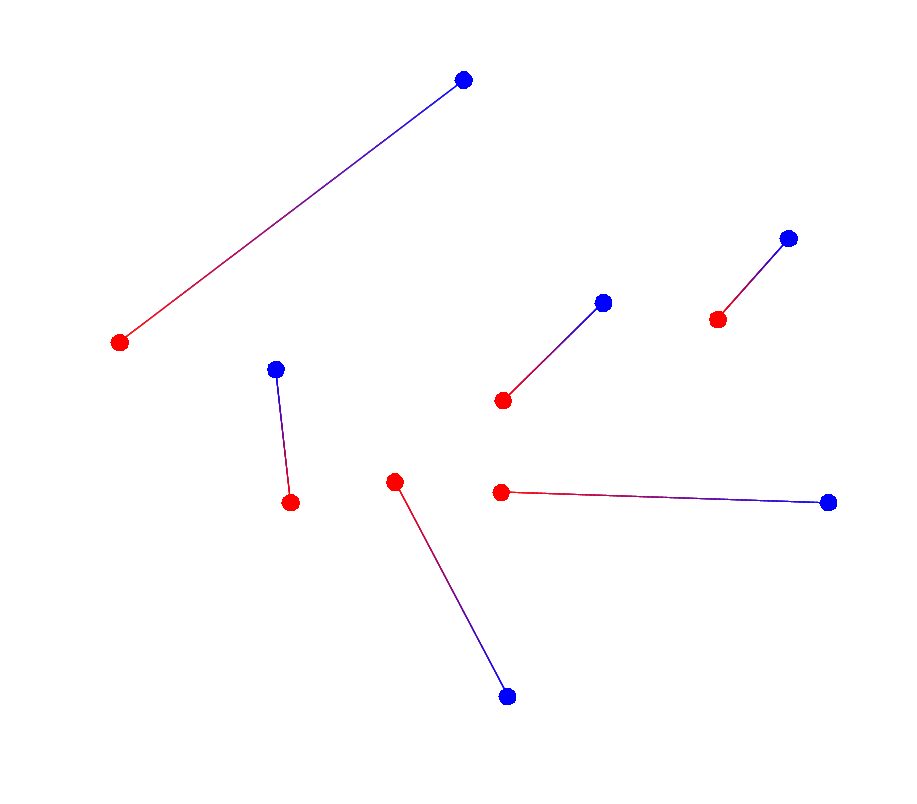}&
\MyFigSchrDyn{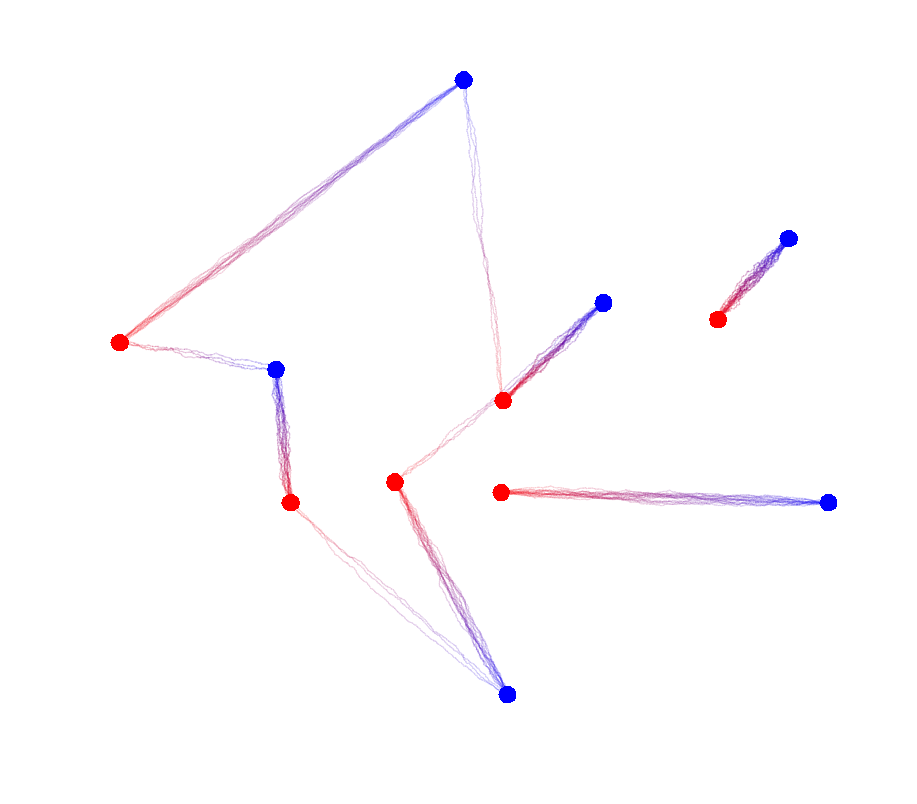}&
\MyFigSchrDyn{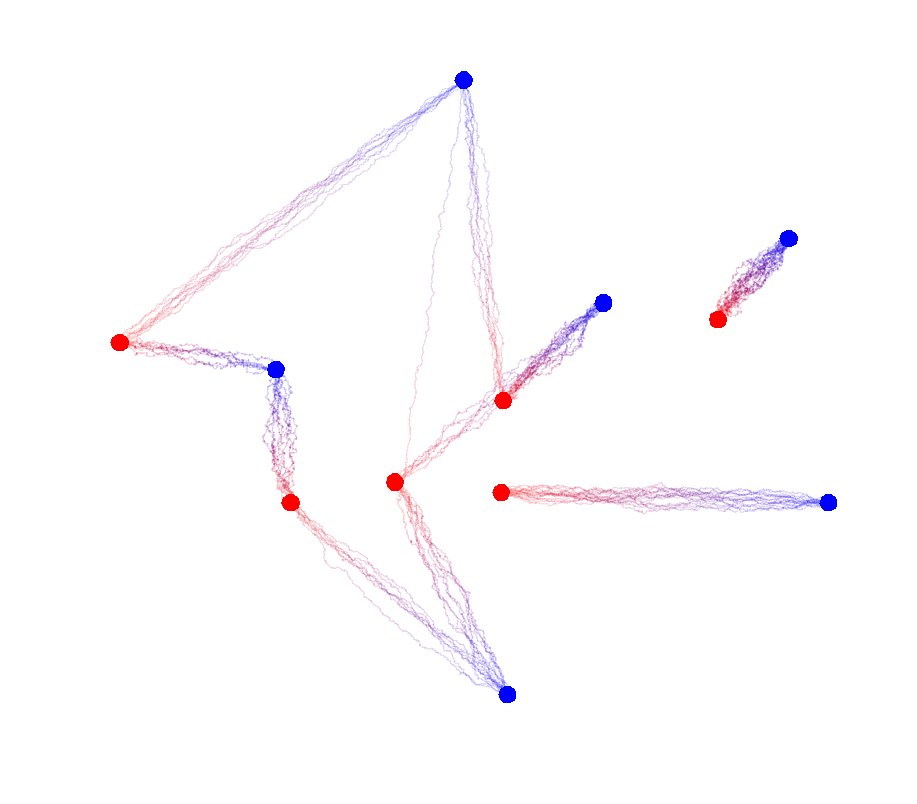}&
\MyFigSchrDyn{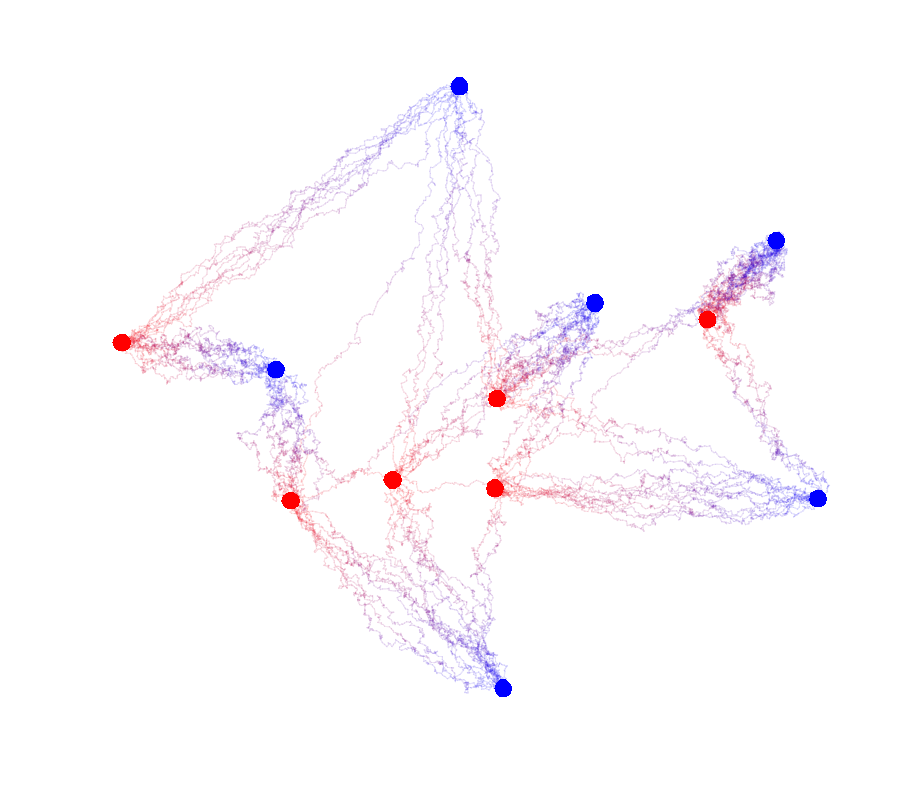}\\
$\varepsilon=0$ & $\varepsilon=.05$ & $\varepsilon=0.2$ & $\varepsilon=1$
\end{tabular}
\caption{
Samples from Brownian bridge paths associated to the Schr\"odinger entropic interpolation~\eqref{eq-brownian-bridge-interp} over path space. Blue corresponds to $t=0$ and red to $t=1$. 
} \label{fig-schrodinger-dynamic}
\end{figure}

Another way to describe this entropic interpolation $(\al_t)_t$ is using a regularization of the Benamou--Brenier dynamic formulation~\eqref{eq-bb-non-cvx}, namely \todoK{check how $\varepsilon$ enters the picture}
\eql{\label{eq-bb-non-cvx-entropy}
	\umin{ (\al_t,\speed_t)_t \text{ sat. \eqref{eq-cons-bb-ncvx}} }
		\int_0^1 \int_{\RR^d}  \pa{
			 \norm{\speed_t(x)}^2  + \frac{\varepsilon}{4} \norm{ \nabla \log(\al_t)(x) }^2
			 } 
			 \d\al_t(x) \d t; 
}
see~\citep{gentil2015analogy,chen2016relation}.

\newcommand{\myFigRKHS}[1]{\includegraphics[width=.195\linewidth]{rkhs/#1}}
\chapter{Statistical Divergences}
\label{c-statistical}

We study in this chapter the statistical properties of the Wasserstein distance. More specifically, we compare it to other major distances and divergences routinely used in data sciences. We quantify how one can approximate the distance between two probability distributions when having only access to samples from said distributions. To introduce these subjects,~\S\ref{sec-phi-div} and~\S\ref{sec-dual-norms} review respectively divergences and integral probability metrics between probability distributions. A divergence $D$ typically satisfies $D(\al,\be) \geq 0$ and $D(\al,\be)=0$ if and only if $\al=\be$, but it does not need to be symmetric or satisfy the triangular inequality. An integral probability metric for measures is a dual norm defined using a prescribed family of test functions. These quantities are sound alternatives to Wasserstein distances and are routinely used as loss functions to tackle inference problems, as will be covered in~\S\ref{c-variational}. We show first in~\S\ref{sec-non-embeddability} that the optimal transport distance is not Hilbertian, \ie one cannot approximate it efficiently using a Hilbertian metric on a suitable feature representation of probability measures. We show in \S\ref{sec-empirical-wass} how to approximate $D(\al,\be)$ from discrete samples $(x_i)_i$ and $(y_j)_j$ drawn from $\al$ and $\be$. A good statistical understanding of that problem is crucial when using the Wasserstein distance in machine learning. Note that this section will be chiefly concerned with the statistical approximation of optimal transport between distributions supported on \emph{continuous} sets. The very same problem when the ground space is finite has received some attention in the literature following the work of~\citet{sommerfeld2018inference}, extended to entropic regularized quantities by~\citet{bigot2017central}.

\section{$\phi$-Divergences}
\label{sec-phi-div}

Before detailing in the following section ``weak'' norms, whose construction shares similarities with $\Wass_1$, let us detail a generic construction of so-called divergences between measures, which can then be used as loss functions when estimating probability distributions. Such divergences compare two input measures by comparing their mass \emph{pointwise}, without introducing any notion of mass transportation. Divergences are functionals which, by looking at the pointwise ratio between two measures, give a sense of how close they are. They have nice analytical and computational properties and build  upon \emph{entropy functions}.

\begin{defn}[Entropy function]
\label{def_entropy}
A function $\phi : \RR \to \RR \cup \{\infty\}$ is an entropy function if it is lower semicontinuous, convex, $\dom \phi\subset [0,\infty[$, and satisfies the following feasibility condition:  $\dom \phi \; \cap\;  ]0, \infty[\; \neq \emptyset$. The speed of growth of $\phi$ at $\infty$ is described by 
\eq{
\phi'_\infty = \lim_{x\rightarrow +\infty} \varphi(x)/x \in \RR \cup \{\infty\} \, .
}
\end{defn}

If $\phi'_\infty = \infty$, then $\phi$ grows faster than any linear function and $\phi$ is said \emph{superlinear}. Any entropy function $\phi$ induces a $\phi$-divergence (also known as Cisz\'ar divergence~\citep{ciszar1967information,ali1966general} or $f$-divergence) as follows.

\begin{defn}[$\phi$-Divergences]
\label{def_divergence}
Let $\phi$ be an entropy function.
For $\al,\be \in \Mm(\X)$, let $\frac{\d \al}{\d \be} \be + \al^{\perp}$ be the Lebesgue decomposition\footnote{The Lebesgue decomposition theorem asserts that, given $\be$, $\al$ admits a unique decomposition as the sum of two measures $\al^s + \al^\perp$ such that $\al^s$ is absolutely continuous with respect to $\be$ and $\al^\perp$ and $\be$ are singular.} of $\al$ with respect to $\be$. The divergence $\Divergm_\phi$ is defined by
\eql{\label{eq-phi-div}
	\Divergm_\phi (\al|\be) \eqdef \int_\X \phi\left(\frac{\d \al}{\d \be} \right) \d \be 
+ \phi'_\infty \al^{\perp}(\X)
}
if $\al,\be$ are nonnegative and $\infty$ otherwise.
\end{defn}%

The additional term $\phi'_\infty \al^{\perp}(\X)$ in~\eqref{eq-phi-div} is important to ensure that $\Divergm_\phi$ defines a continuous functional (for the weak topology of measures) even if $\phi$ has a linear growth at infinity, as this is, for instance, the case for the absolute value~\eqref{eq-tv-entropy} defining the TV norm. If $\phi$ as a superlinear growth, \eg the usual entropy~\eqref{eq-shannon-entropy}, then $\phi'_\infty=+\infty$ so that $\Divergm_\phi (\al|\be) = +\infty$ if $\al$ does not have a density with respect to $\be$. 

In the discrete setting, assuming 
\eql{\label{eq-div-disc-meas}
	\al=\sum_i \a_i \de_{x_i}
	\qandq \be=\sum_i \b_i \de_{x_i}
} 
are supported on the same set of $n$ points $(x_i)_{i=1}^n \subset \X$,~\eqref{eq-phi-div} defines a divergence on $\simplex_n$
\eql{\label{eq-discr-diverg}
	\DivergmD_\phi(\a|\b) = \sum_{i \in \Supp(\b)} \phi\pa{ \frac{\a_i}{\b_i} } \b_i + \phi'_\infty \sum_{i \notin \Supp(\b)} \a_i,
}
where $\Supp(\b) \eqdef \enscond{i \in \range{n}}{ b_i \neq 0 }$.

The proof of the following proposition can be found in~\cite[Thm 2.7]{LieroMielkeSavareLong}.

\begin{proposition}
If $\phi$ is an entropy function, then $\Divergm_\phi$ is jointly $1$-homogeneous, convex and weakly* lower semicontinuous in $(\al,\be)$.
\end{proposition}

\begin{figure}[h!]
\centering
\includegraphics[width=.6\linewidth]{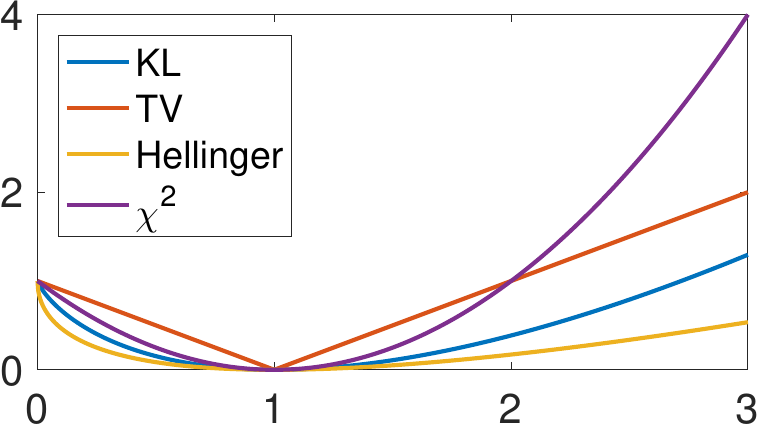}
\caption{\label{fig-divergences}
Example of entropy functionals.
}
\end{figure}

\begin{rem}[Dual expression]
	A $\phi$-divergence can be expressed using the Legendre transform 
	\eq{
		\phi^*(s) \eqdef \usup{t \in \RR} st - \phi(t)
	}
	of $\phi$ (see also~\eqref{eq-legendre}) as
	\eq{
		\Divergm_\phi (\al|\be) = \usup{f: \X \rightarrow \RR} \int_\X f(x) \d\al(x) - \int_\X \phi^*(f(x)) \d\be(x); 
	}
	see~\citet{LieroMielkeSavareLong} for more details.
\end{rem}

We now review a few popular instances of this framework. Figure~\ref{fig-divergences} displays the associated entropy functionals, while Figure~\ref{fig-divergences-relations} reviews the relationship between them.

\begin{figure}[h!]
\centering
\includegraphics[width=.6\linewidth]{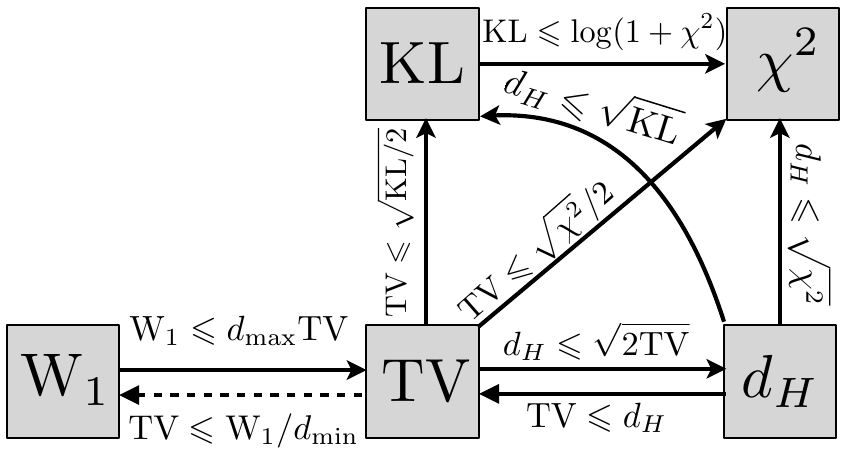}
\caption{\label{fig-divergences-relations}
Diagram of relationship between divergences (inspired by~\citet{gibbs2002choosing}).
For $\X$ a metric space with ground distance $d$, $d_{\max} = \sup_{(x,x')} d(x,x')$ is the diameter of $\X$. When $\X$ is discrete, $d_{\min} \eqdef \min_{x \neq x'} d(x,x')$. \todoK{Tripple check these bounds since they might use different definition/normalization. }\todoK{Change the notation for the Hellinger distance one the symbol is decided $\Hellinger$. }
}
\end{figure}

\begin{example}[Kullback--Leibler divergence]
\label{ex_KLdiv}
The Kullback--Leibler divergence $\KL \eqdef \Divergm_{\phi_{\KL}}$, also known as the relative entropy, was already introduced in~\eqref{eq-defn-rel-entropy} and~\eqref{eq-kl-defn}. It is the divergence associated to the Shannon--Boltzman entropy function $\phi_{\KL}$, given by
\eql{\label{eq-shannon-entropy}
	\phi_{\KL}(s)= \begin{cases}
		s\log(s)-s+1 & \textnormal{for } s>0 , \\
		1 & \textnormal{for } s=0 , \\
		+\infty & \textnormal{otherwise.}
		\end{cases}
}
\end{example}

\begin{remark}[Bregman divergence]\label{rem-bregman}
The discrete KL divergence, $\KLD \eqdef \DivergmD_{\phi_{\KL}}$, has the unique property of being both a $\phi$-divergence and a Bregman divergence.
For discrete vectors in $\RR^n$, a Bregman divergence~\citep{bregman1967relaxation} associated to a smooth strictly convex function $\psi : \RR^n \rightarrow \RR$ is defined as
\eql{\label{eq-bregman-discr}
	\VectMode{B}_\psi(\a|\b) \eqdef \psi(\a) - \psi(\b) - \dotp{\nabla \psi(\b)}{\a-\b}, 
}
where $\dotp{\cdot}{\cdot}$ is the canonical inner product on $\RR^n$. Note that $\VectMode{B}_\psi(\a|\b)$ is a convex function of $\a$ and a linear function of $\psi$.
Similarly to $\phi$-divergence, a Bregman divergence satisfies $\VectMode{B}_\psi(\a|\b) \geq 0$ and $\VectMode{B}_\psi(\a|\b)=0$ if and only if $\a=\b$. 
The KL divergence is the Bregman divergence for minus the entropy $\psi = -\HD$ defined in~\eqref{eq-discr-entropy}), \ie $\KLD = \VectMode{B}_{-\HD}$. 
A Bregman divergence is locally a squared Euclidean distance since
\eq{
	\VectMode{B}_\psi(\a + \varepsilon|\a + \eta) = \dotp{ \partial^2 \psi(\a)(\varepsilon-\eta) }{\varepsilon-\eta}
	+ o(\norm{\varepsilon-\eta}^2)
}
and the set of separating points $\enscond{\a}{\VectMode{B}_\psi(\a|\b) = \VectMode{B}_\psi(\a|\b')}$ is a hyperplane between $\b$ and $\b'$.
These properties make Bregman divergence suitable to replace Euclidean distances in first order optimization methods. The best know example is mirror gradient descent~\citep{beck2003mirror}, which is an explicit descent step of the form~\eqref{eq-explicit-metricflow}.
Bregman divergences are also important in convex optimization and can be used, for instance, to derive Sinkhorn iterations and study its convergence in finite dimension; see Remark~\ref{rem-iterative-projection}.
\end{remark}

\begin{remark}[Hyperbolic geometry of KL]
	It is interesting to contrast the geometry of the Kullback--Leibler divergence to that defined by quadratic optimal transport when comparing Gaussians. As detailed, for instance, by~\citet{costa2015fisher}, the Kullback--Leibler divergence has a closed form for Gaussian densities. In the univariate case, $\dim=1$, if $\al = \Nn(m_\al,\si_\al^2)$ and $\be = \Nn(m_\be,\si_\be^2)$, one has
	\eql{\label{eq-kl-gaussian}
		\KL(\al|\be) = \frac{1}{2}\pa{
			\frac{\si_\al^2}{\si_\be^2}  + 
			\log\pa{\frac{\si_\be^2}{\si_\al^2}}
			+ \frac{|m_\al-m_\be|}{\si_\be^2}
			 - 1
		}.
	}	
	This expression shows that the divergence between $\al$ and $\be$ diverges to infinity as $\si_\be$ diminishes to $0$ and $\be$ becomes a Dirac mass. In that sense, one can say that singular Gaussians are infinitely far from all other Gaussians in the KL geometry. That geometry is thus useful when one wants to avoid dealing with singular covariances. To simplify the analysis, one can look at the infinitesimal geometry of KL, which is obtained by performing a Taylor expansion at order 2,
	\eq{
		\KL( \Nn(m+\de_m,(\si+\de_\si)^2) | \Nn(m,\si^2) ) = \frac{1}{\si^2}\pa{
			\frac{1}{2}\de_m^2 + \de_\si^2
		} + 
		o( \de_m^2,\de_\si^2 ).
	}
	This local Riemannian metric, the so-called Fisher metric, expressed over $(m/\sqrt{2},\si) \in \RR \times \RR_{+,*}$, matches exactly that of the hyperbolic Poincar\'e half plane. 
	Geodesics over this space are half circles centered along the $\si=0$ line and have an exponential speed, \ie they only reach the limit $\si=0$ after an infinite time. 
	Note in particular that if $\si_\al=\si_\be$ but $m_\al \neq m_\al$, then the geodesic between $(\al,\be)$ over this hyperbolic half plane does not have a constant standard deviation. 
	
	The KL hyperbolic geometry over the space of Gaussian parameters $(m,\si)$ should be contrasted with the Euclidean geometry associated to OT as described in Remark~\ref{rem-dist-gaussians}, since in the univariate case
	\eql{\label{eq-univariate-gauss}
		\Wass_2^2( \al,\be ) = |m_\al - m_\be |^2 + |\si_\al-\si_\be|^2. 
	}  
	Figure~\ref{fig-kl-gaussians} shows a visual comparison of these two geometries and their respective geodesics.
	This interesting comparison was suggested to us by Jean Feydy.  
\end{remark}

\begin{figure}[h!]
\centering
\begin{tabular}{@{}c@{\hspace{3mm}}c@{}}
\includegraphics[width=.45\linewidth]{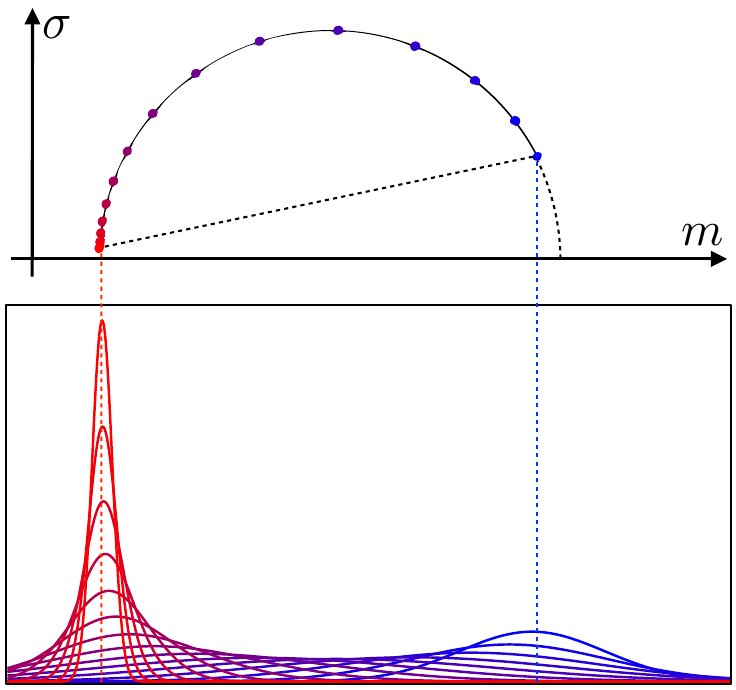}&
\includegraphics[width=.46\linewidth]{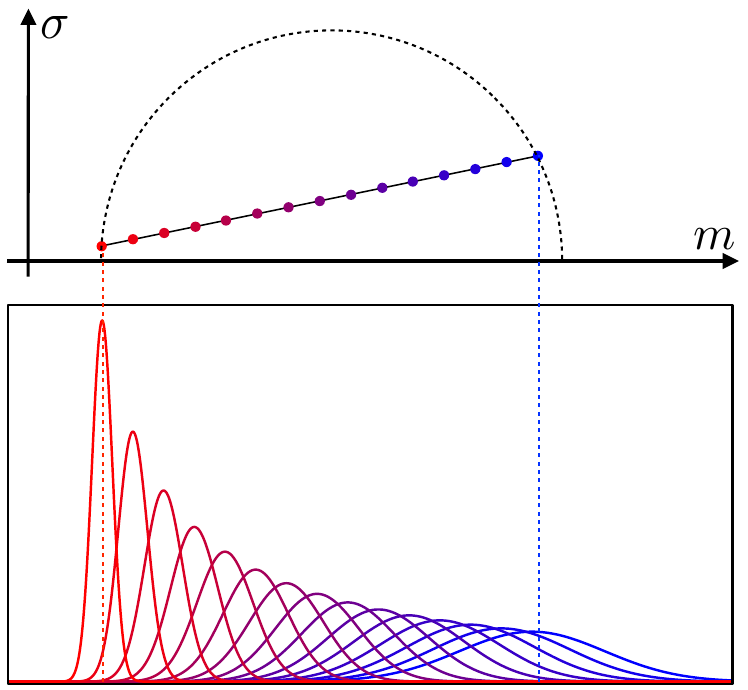}\\
KL & OT
\end{tabular}
\caption{\label{fig-kl-gaussians}
Comparisons of interpolation between Gaussians using KL (hyperbolic) and OT (Euclidean) geometries.
}
\end{figure}

\begin{example}[Total variation]\label{exmp-tv}
The total variation distance $\TV \eqdef \Divergm_{\phi_{\TV}}$ is the divergence associated to
\eql{\label{eq-tv-entropy}
	\phi_{\TV}(s)= \begin{cases}
		|s-1| & \textnormal{for } s\geq0 , \\
		+\infty & \textnormal{otherwise.}
		\end{cases}
}
It actually defines a norm on the full space of measure $\Mm(\X)$ where
\eql{\label{eq-defn-tv}
	\TV(\al|\be) = \norm{\al-\be}_{\TV},
	\qwhereq
	\norm{\al}_{\TV} = |\al|(\X) = \int_\X \d|\al|(x).
}
If $\al$ has a density $\density{\al}$ on $\X=\RR^\dim$, then the TV norm is the $L^1$ norm on functions, $\norm{\al}_{\TV} = \int_\X |\density{\al}(x)| \d x = \norm{\density{\al}}_{L^1}$.
If $\al$ is discrete as in~\eqref{eq-div-disc-meas}, then the TV norm is the $\ell^1$ norm of vectors in $\RR^n$, $\norm{\al}_{\TV}=\sum_i |\a_i| = \norm{\a}_{\ell^1}$.
\end{example}

\begin{rem}[Strong vs. weak topology]
	The total variation norm~\eqref{eq-defn-tv} defines the so-called ``strong'' topology on the space of measure. 
	On a compact domain $\X$ of radius $R$, one has 
	\eq{
		\Wass_1(\al,\be) \leq R \norm{\al-\be}_{\TV}
	}
	so that this strong notion of convergence implies the weak convergence metrized by Wasserstein distances. 
	The converse is, however, not true, since $\de_x$ does not converge strongly to $\de_y$ if $x \rightarrow y$ (note that
	$\norm{\de_x-\de_y}_{\TV}=2$ if $x \neq y$). 
	A chief advantage is that $\Mm_+^1(\Xx)$ (once again on a compact ground space $\X$) is compact for the weak topology, so that from any sequence of probability measures $(\al_k)_k$, one can always extract a converging subsequence, which makes it a suitable space for several optimization problems, such as those considered in Chapter~\ref{c-variational}.
\end{rem}

\begin{example}[Hellinger]\label{exmp-hellinger}
	The Hellinger distance $\Hellinger \eqdef \Divergm_{\phi_{H}}^{1/2}$ is the square root of the divergence associated to
	\eq{
		\phi_{H}(s)= \begin{cases}
			|\sqrt{s}-1|^2 & \textnormal{for } s\geq0 , \\
			+\infty & \textnormal{otherwise.}
			\end{cases}
	}
	As its name suggests, $\Hellinger$ is a distance on $\Mm_+(\X)$, which metrizes the strong topology as $\norm{\cdot}_{\TV}$. 
	If $(\al,\be)$ have densities $(\density{\al},\density{\be})$ on $\X=\RR^\dim$, then $\Hellinger(\al,\be) = \norms{\sqrt{\density{\al}}-\sqrt{\density{\be}}}_{L^2}$.
	If $(\al,\be)$ are discrete as in~\eqref{eq-div-disc-meas}, then $\Hellinger(\al,\be) = \norms{\sqrt{\a}-\sqrt{\b}}$.	
	Considering $\phi_{L^p}(s)=|s^{1/p}-1|^p$ generalizes the Hellinger ($p=2$) and total variation ($p=1$) distances
	and $\Divergm_{\phi_{L^p}}^{1/p}$ is a distance which metrizes the strong convergence for $0<p<+\infty$. 
\end{example}

\begin{example}[Jensen--Shannon distance]
	The KL divergence is not symmetric and, while being a Bregman divergence (which are locally quadratic norms), it is not the square of a distance. On the other hand, the Jensen--Shannon distance $\text{JS}(\al,\be)$, defined as
	\eq{
		\text{JS}(\al,\be)^2 \eqdef \frac{1}{2}\pa{
			\KL(\al|\xi) + \KL(\be|\xi)
		}
		\qwhereq
		\xi = \frac{\al+\be}{2}, 
	}
	is a distance~\citep{endres2003new,osterreicher2003new}. 
	$\text{JS}^2$ can be shown to be a $\phi$-divergence for $\phi(s) = t\log(t)-(t+1)\log(t+1)$.
	In sharp contrast with $\KL$, $\text{JS}(\al,\be)$ is always bounded; more precisely, it satisfies $0 \leq \text{JS}(\al,\be)^2 \leq \ln(2)$. 
	Similarly to the TV norm and the Hellinger distance, it metrizes the strong convergence. 
\end{example}

\begin{example}[$\chi^2$]\label{exmp-chisquare}
	The $\chi^2$-divergence $\chi^2 \eqdef \Divergm_{\phi_{\chi^2}}$ is the divergence associated to
	\eq{
		\phi_{\chi^2}(s)= \begin{cases}
			|s-1|^2 & \textnormal{for } s\geq0 , \\
			+\infty & \textnormal{otherwise.}
			\end{cases}
	}
	If $(\al,\be)$ are discrete as in~\eqref{eq-div-disc-meas} and have the same support, then 
	\eq{
		\chi^2(\al|\be) = \sum_i \frac{(\a_i-\b_i)^2}{\b_i}.
	}	
\end{example}

\section{Integral Probability Metrics}
\label{sec-dual-norms}

Formulation~\eqref{eq-w1-cont} is a special case of a dual norm. A dual norm is a convenient way to design ``weak'' norms that can deal with arbitrary measures. For a symmetric convex set $B$ of measurable functions, one defines 
\eql{\label{eq-dual-norm-cont}
	\norm{\al}_B \eqdef 
	\umax{\f} \enscond{ \int_{\X} \f(x) \d\al(x) }{ \f \in B}.
}
These dual norms are often called ``integral probability metrics'; see~\citep{sriperumbudur2012empirical}.


\begin{example}[Total variation]
The total variation norm (Example~\ref{exmp-tv}) is a dual norm associated to the whole space of continuous functions
\eq{
	B = \enscond{f \in \Cc(\X)}{\norm{f}_\infty \leq 1}.
}
The total variation distance is the only nontrivial divergence that is also a dual norm; see~\citep{sriperumbudur2009integral}. 
\end{example}

\begin{rem}[Metrizing the weak convergence]\label{rem-ipm-weak}
By using smaller ``balls'' $B$, which typically only contain continuous (and sometimes regular) functions, one defines weaker dual norms.
In order for $\norm{\cdot}_B$ to metrize the weak convergence (see Definition~\ref{dfn-weak-conv}), it is sufficient for the space spanned by $B$ to be dense in the set of continuous functions for the sup-norm $\norm{\cdot}_\infty$ (i.e. for the topology of uniform convergence); see~\citep[para. 5.1]{ambrosio2006gradient}.
\end{rem}

Figure~\ref{fig-dual-norms} displays a comparison of several such dual norms, which we now detail.

\begin{figure}[h!]
\centering
\begin{tabular}{@{}c@{\hspace{3mm}}c@{}}
\includegraphics[width=.45\linewidth]{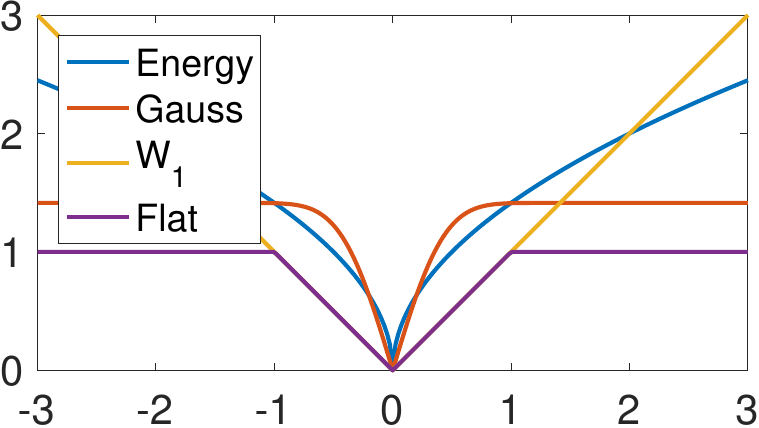} &
\includegraphics[width=.45\linewidth]{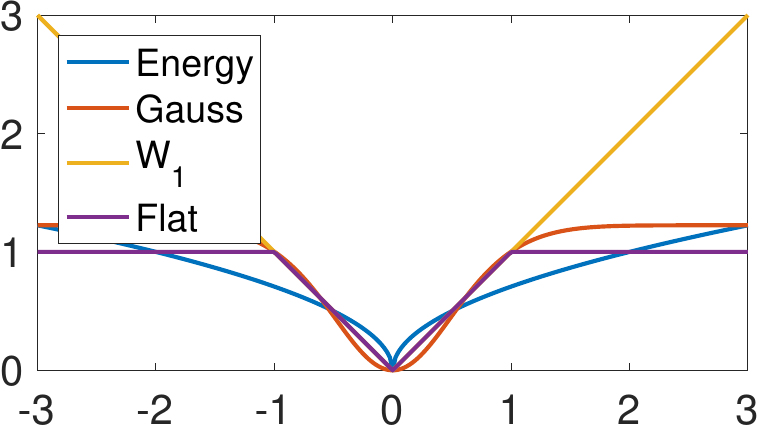} \\
$(\al,\be)=(\de_0,\de_t)$ &
$(\al,\be)=(\de_0,\frac{1}{2}(\de_{-t/2}+\de_{t/2}))$
\end{tabular}
\caption{\label{fig-dual-norms}
Comparison of dual norms. 
}
\end{figure}

\subsection{$\Wass_1$ and Flat Norm}

If the set $B$ is bounded, then $\norm{\cdot}_B$ is a norm on the whole space $\Mm(\Xx)$ of measures.
This is not the case of $\Wass_1$, which is only defined for $\al$ such that $\int_\X \d\al=0$ (otherwise $\norm{\al}_B=+\infty$). 
This can be alleviated by imposing a bound on the value of the potential $\f$, in order to define for instance the flat norm.

\begin{example}[$\Wass_1$ norm]
$\Wass_1$ as defined in~\eqref{eq-w1-cont}, is a special case of dual norm~\eqref{eq-dual-norm-cont}, using
\eq{ 
	B = \enscond{f}{\Lip(f) \leq 1}
} 
the set of 1-Lipschitz functions.
\end{example}

\begin{example}[Flat norm and Dudley metric]
The flat norm is defined using 
\eql{\label{eq-set-flatnorm}
	B=\enscond{f}{\norm{\nabla \f}_\infty \leq 1 \qandq \norm{\f}_\infty \leq 1}.
}
It metrizes the weak convergence on the whole space $\Mm(\X)$.
Formula~\eqref{eq-w1-discr} is extended to compute the flat norm by adding the constraint $\abs{\fD_k} \leq 1$.
The flat norm is sometimes called the ``Kantorovich--Rubinstein'' norm~\citep{hanin1992kantorovich} and has been used as a fidelity term for inverse problems in imaging~\citep{lellmann2014imaging}.
The flat norm is  similar to the Dudley metric, which uses
\eq{\label{eq-set-dudley}
	B=\enscond{f}{\norm{\nabla \f}_\infty + \norm{\f}_\infty \leq 1}.
}
\end{example}

\subsection{Dual RKHS Norms and Maximum Mean Discrepancies}
\label{sec-mmd}

It is also possible to define ``Euclidean'' norms (built using quadratic functionals) on measures using the machinery of kernel methods and more specifically reproducing kernel Hilbert spaces (RKHS; see~\citep{scholkopf2002learning} for a survey of their applications in data sciences), of which we recall first some basic definitions.

\begin{defn}\label{def-negativedefinitekernel}
A symmetric function $k$ (resp., $\varphi$) defined on a set $\X \times \X$ is said to be positive (resp., negative) definite if for any $n\geq0$, family $x_1,\dots,x_n\in\Z$, and vector $r\in\RR^n$ the following inequality holds:
\eql{\label{eq-dual-kern}
	\sum_{i,j=1}^n r_i r_j k(x_i,x_j)\geq 0, 
	\quad
	\left(\text{resp.}\quad\sum_{i,j=1}^n r_i r_j \varphi(x_i,x_j)\leq 0\right).
}
The kernel is said to be conditionally positive if positivity only holds in~\eqref{eq-dual-kern} for zero mean vectors $r$ (i.e. such that $\dotp{r}{\ones_n}=0$).  
\end{defn}

If $k$ is conditionally positive, one defines the following norm:  
\eql{\label{eq-kernel-dual}
	\norm{\al}^2_{\Krkhs} \eqdef \int_{\X \times \X} \Krkhs(x,y) \d \al(x) \d \al(y). 
}
These norms are often referred to as ``maximum mean discrepancy'' (MMD) (see~\citep{gretton2007kernel}) and have also been called ``kernel norms'' in shape analysis~\citep{glaunes2004diffeomorphic}.
This expression~\eqref{eq-kernel-dual} can be rephrased, introducing two independent random vectors $(X,X')$ on $\X$ distributed with law $\al$, as
\eq{
	\norm{\al}^2_{\Krkhs} = \EE_{X,X'}( \Krkhs(X,X') ). 
}
One can show that $\norm{\cdot}^2_{\Krkhs}$ is the dual norm in the sense of~\eqref{eq-dual-norm-cont} associated to the unit ball $B$ of the RKHS associated to $k$. We refer to~\citep{berlinet03reproducing,Hofmann2008,scholkopf2002learning} for more details on RKHS functional spaces.  

\begin{rem}[Universal kernels]
According to Remark~\ref{rem-ipm-weak}, the MMD norm $\norm{\cdot}_k$ metrizes the weak convergence if the span of the dual ball $B$ is dense in the space of continuous functions $\Cc(\Xx)$. This means that finite sums of the form $\sum_{i=1}^n a_i k(x_i,\cdot)$ (for arbitrary choice of $n$ and points $(x_i)_i$) are dense in $\Cc(\Xx)$ for the uniform norm $\norm{\cdot}_\infty$.  For translation-invariant kernels over $\X=\RR^d$, $\Krkhs(x,y)=\Krkhs_0(x-y)$, this is equivalent to having a nonvanishing Fourier transform, $\hat \Krkhs_0(\om) >0$.
\end{rem}

In the special case where $\al$ is a discrete measure of the form~\eqref{eq-pair-discr}, one thus has the simple expression
\eq{
	\norm{\al}_{\Krkhs}^2 = \sum_{i=1}^n \sum_{i'=1}^n \a_i\a_{i'} \KrkhsD_{i,i'}  = \dotp{\KrkhsD\a}{\a}
	\qwhereq
	\KrkhsD_{i,i'} \eqdef \Krkhs(x_i,x_{i'}).
}
In particular, when $\al=\sum_{i=1}^n \a_i \de_{x_i}$ and $\be=\sum_{i=1}^n \b_i \de_{x_i}$ are supported on the same set of points, $\norm{\al-\be}_{\Krkhs}^2 = \dotp{\KrkhsD(\a-\b)}{\a-\b}$, so that $\norm{\cdot}_{\Krkhs}$ is a Euclidean norm (proper if $\KrkhsD$ is positive definite, degenerate otherwise if $\KrkhsD$ is semidefinite) on the simplex $\simplex_n$.
To compute the discrepancy between two discrete measures of the form~\eqref{eq-pair-discr}, one can use
\eql{\label{eq-mmd-discr}
	\norm{\al-\be}_{\Krkhs}^2 = 
		\sum_{i,i'} \a_i \a_{i'} \Krkhs(x_i,x_{i'})	+
		\sum_{j,j'} \b_j \b_{j'} \Krkhs(y_j,y_{j'}) - 2
		\sum_{i,j} \a_i \b_j \Krkhs(x_i,y_j). 
}

\begin{example}[Gaussian RKHS]\label{exmp-gaussian-kernel}
One of the most popular kernels is the Gaussian one $\Krkhs(x,y) = e^{-\frac{\norm{x-y}^2}{2 \si^2}}$, which is a positive universal kernel on $\Xx=\RR^d$. 
An attractive feature of the Gaussian kernel is that it is separable as a product of 1-D kernels, which facilitates computations when working on regular grids (see also Remark~\ref{rem-separable}).
However, an important issue that arises when using the Gaussian kernel is that one needs to select the bandwidth parameter $\si$. This bandwidth should match the ``typical scale'' between observations in the measures to be compared. If the measures have multiscale features (some regions may be very dense, others very sparsely populated), a Gaussian kernel is thus not well adapted, and one should consider a ``scale-free'' kernel as we detail next. An issue with such scale-free kernels is that they are global (have slow polynomial decay), which makes them typically computationally more expensive, since no compact support approximation is possible.
Figure~\ref{fig-rkhs} shows a comparison between several kernels. 
\end{example}

\begin{figure}[h!]
\centering
\begin{tabular}{@{}c@{\hspace{0mm}}c@{\hspace{0mm}}c@{\hspace{0mm}}c@{\hspace{0mm}}c@{}}
 &
\myFigRKHS{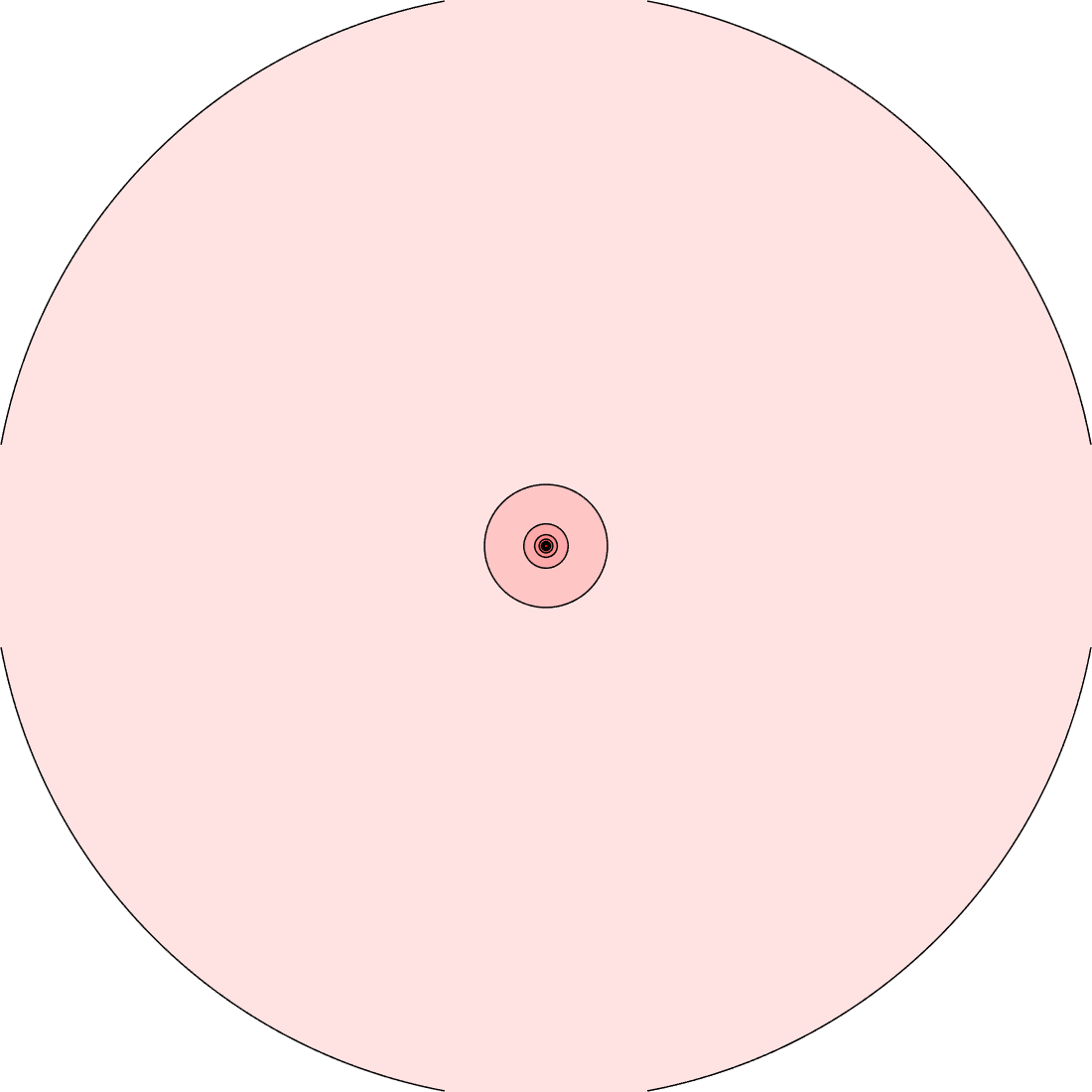}&
\myFigRKHS{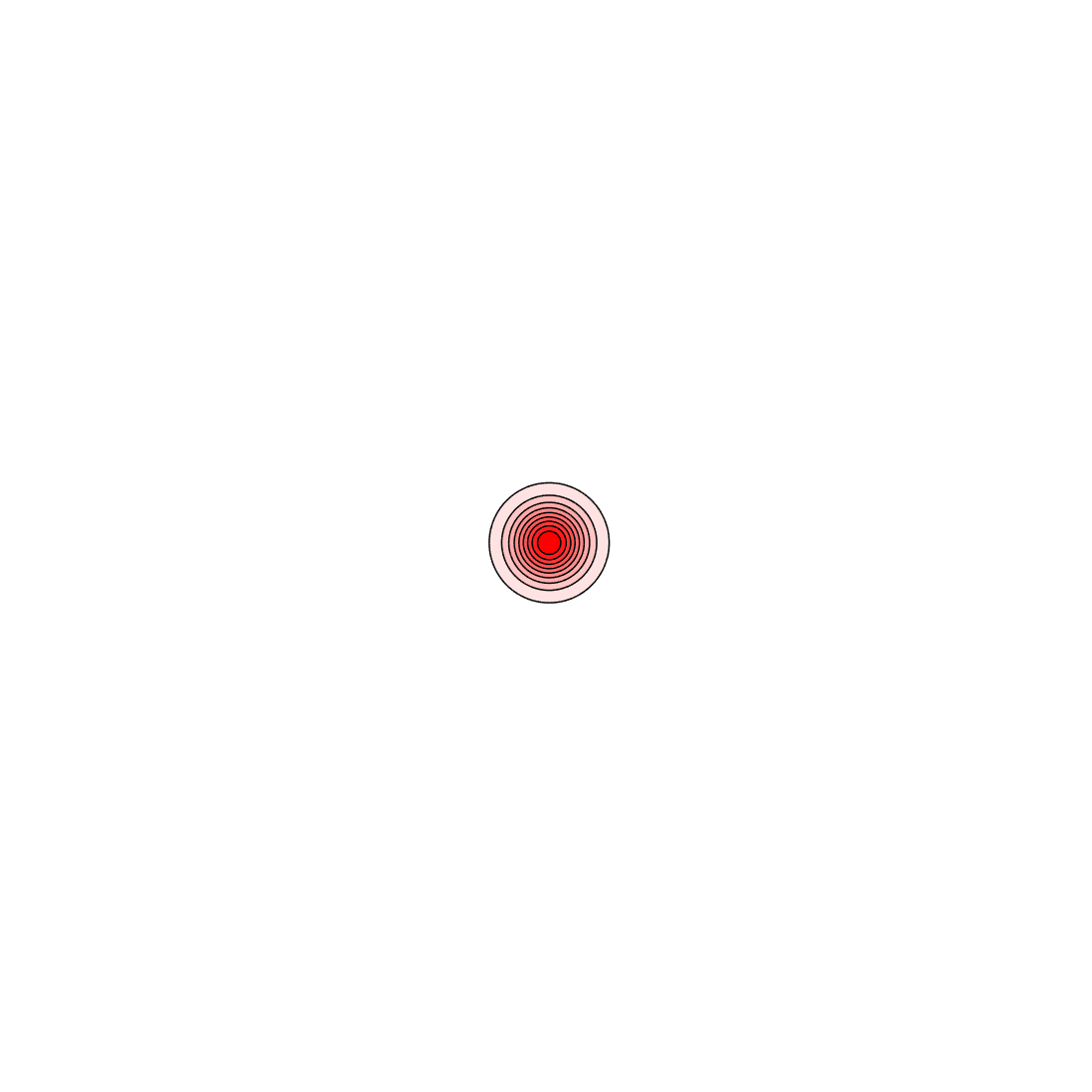}&
\myFigRKHS{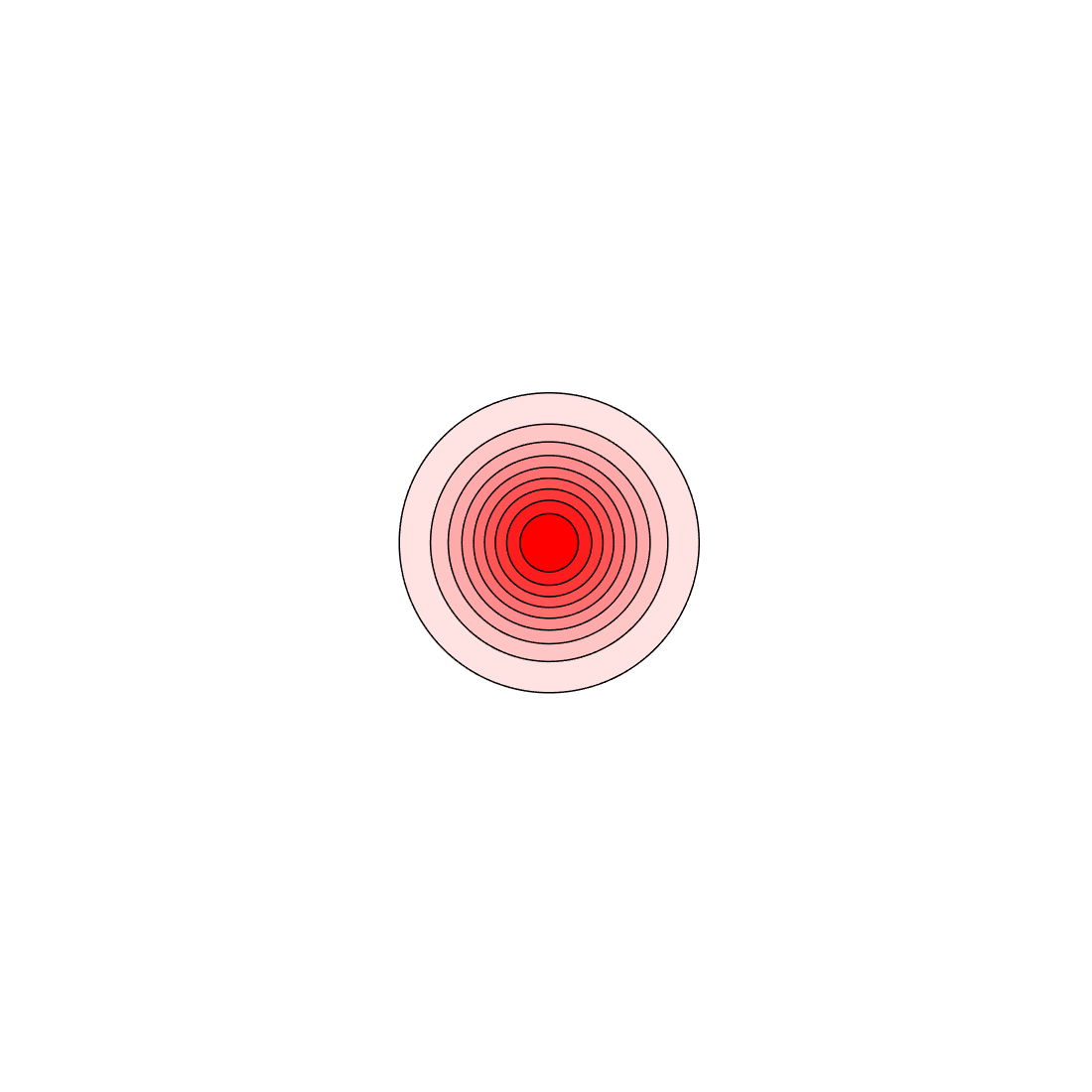}&
\myFigRKHS{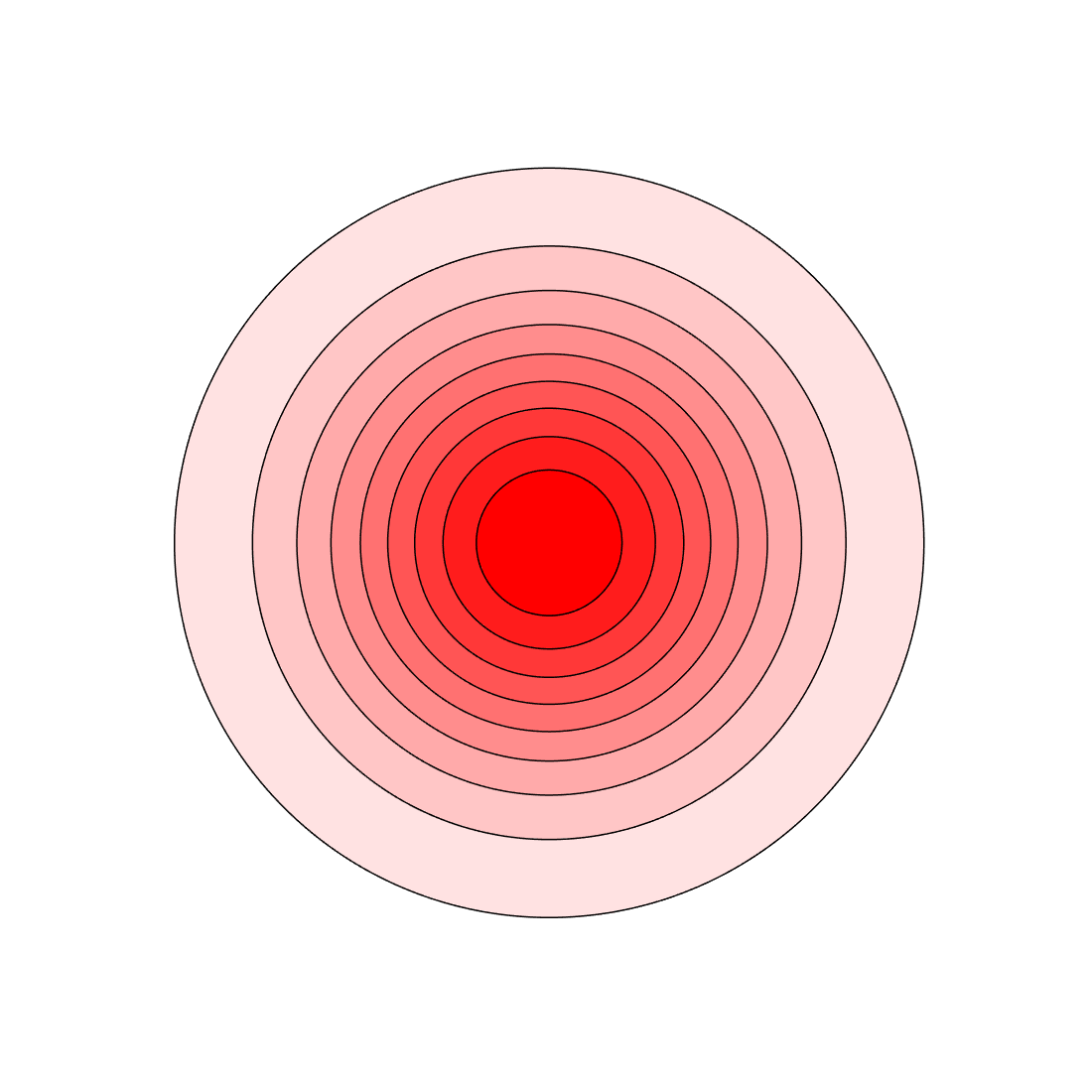}\\
\myFigRKHS{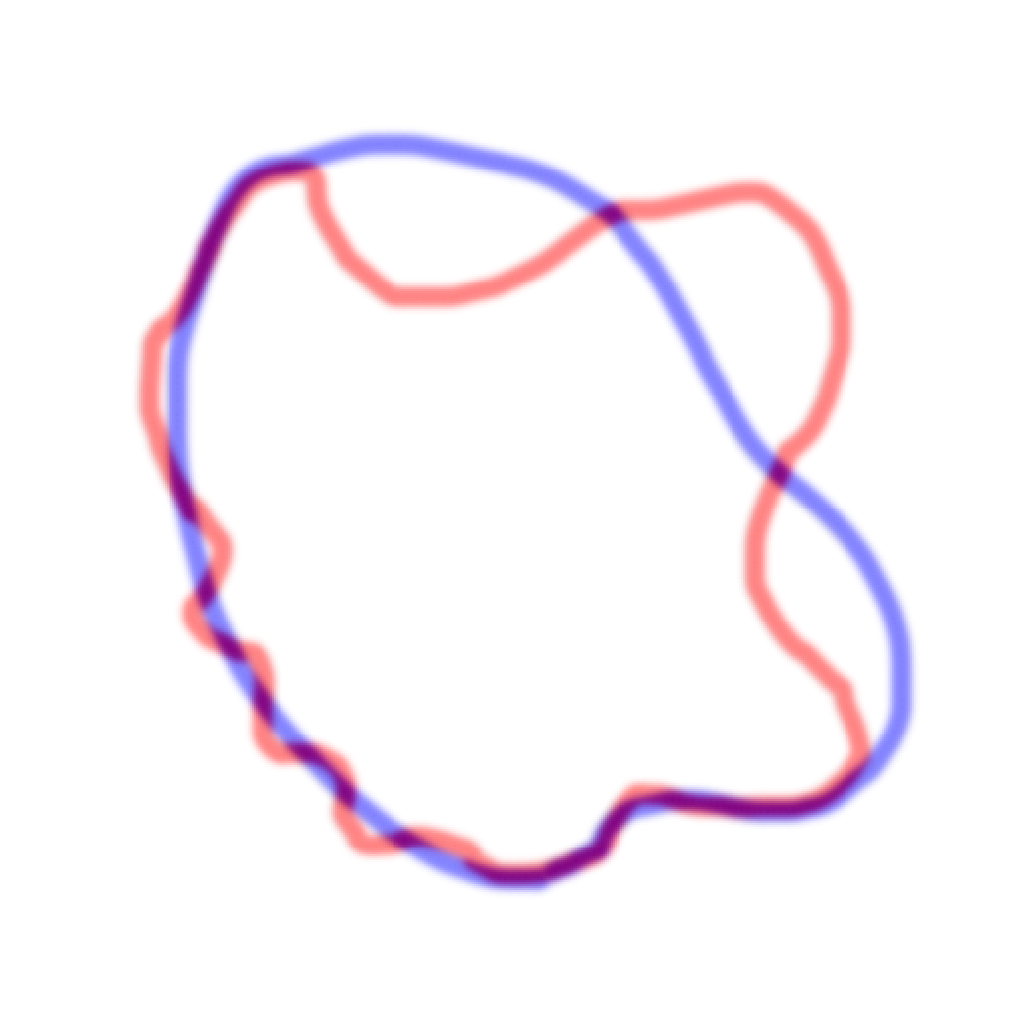} &
\myFigRKHS{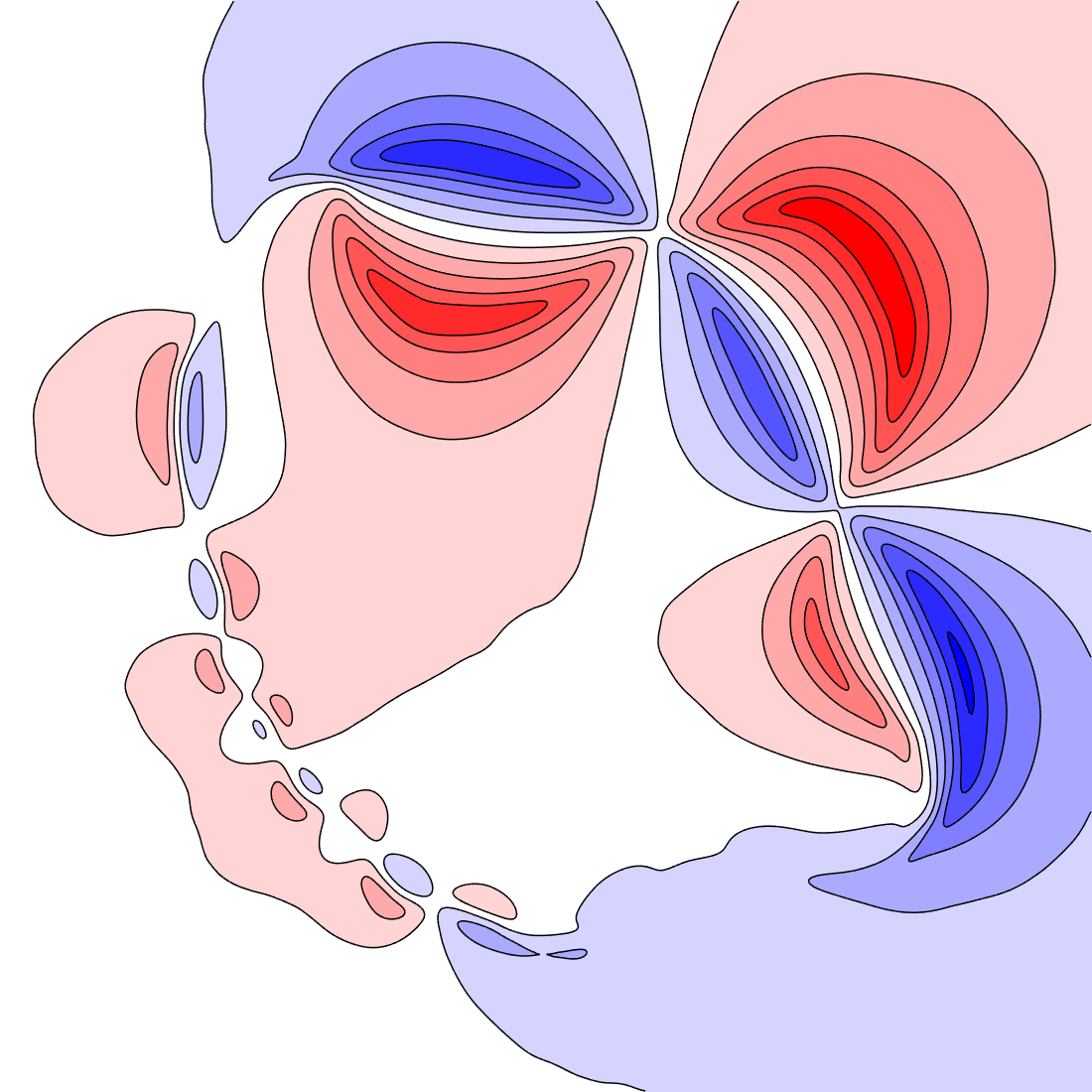} &
\myFigRKHS{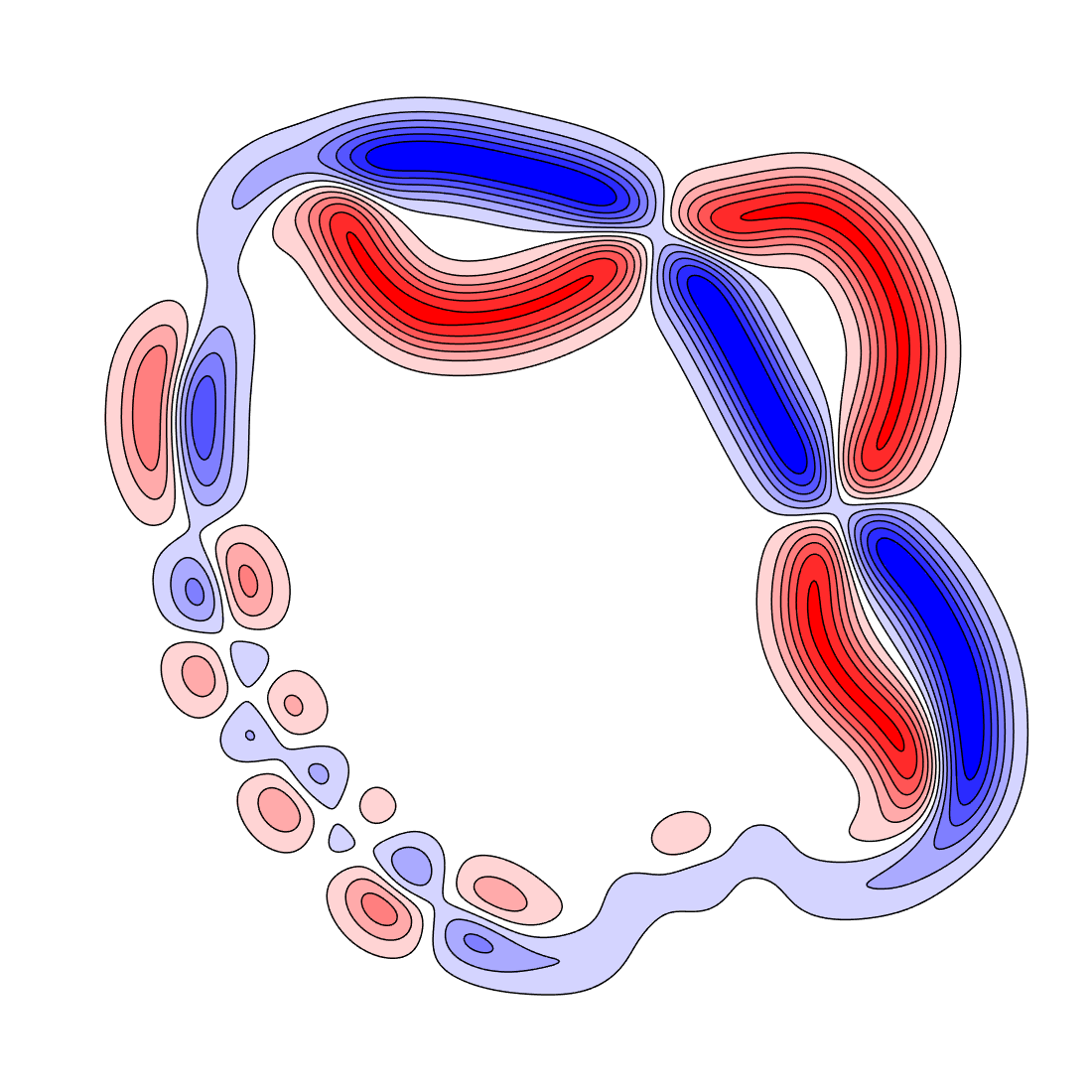}&
\myFigRKHS{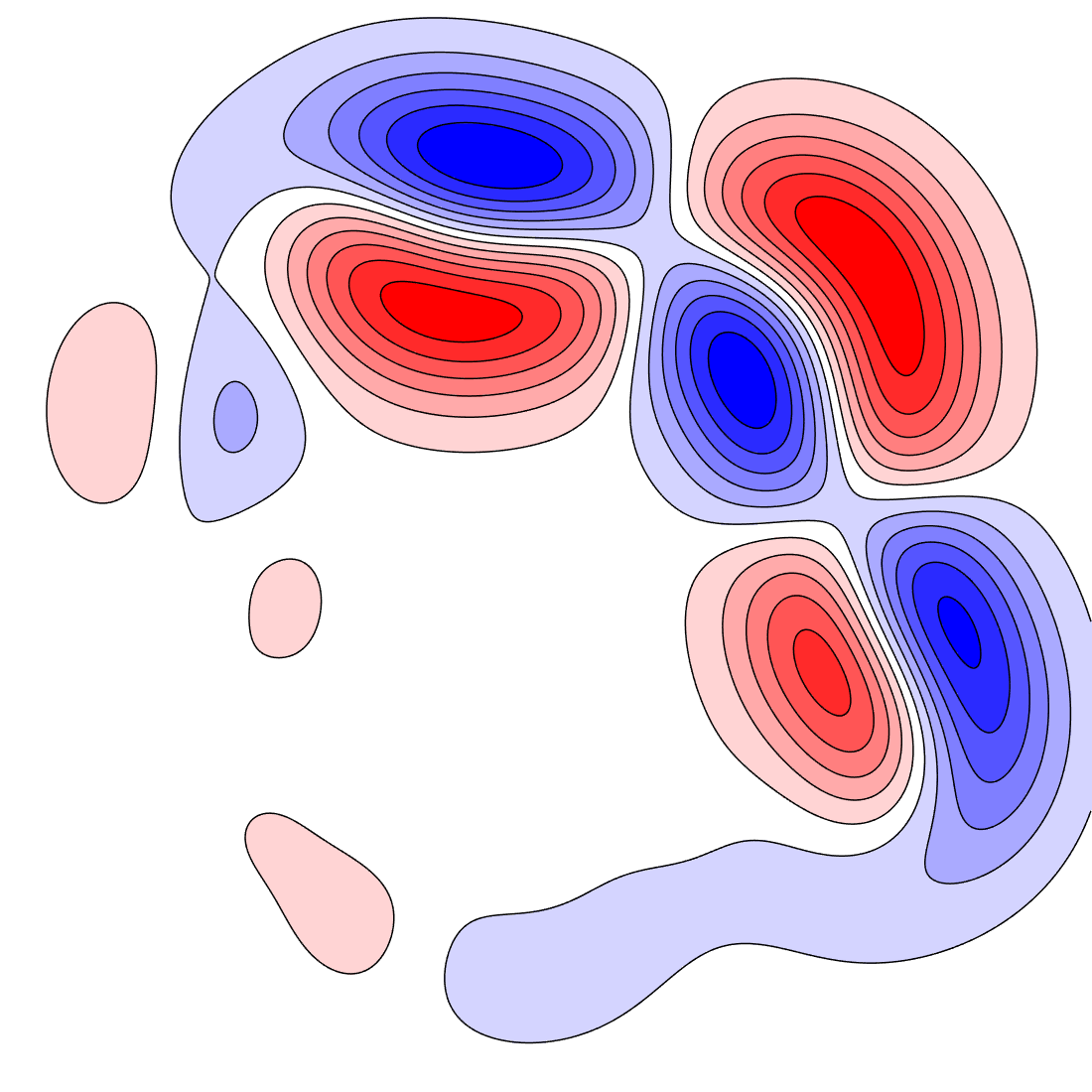}&
\myFigRKHS{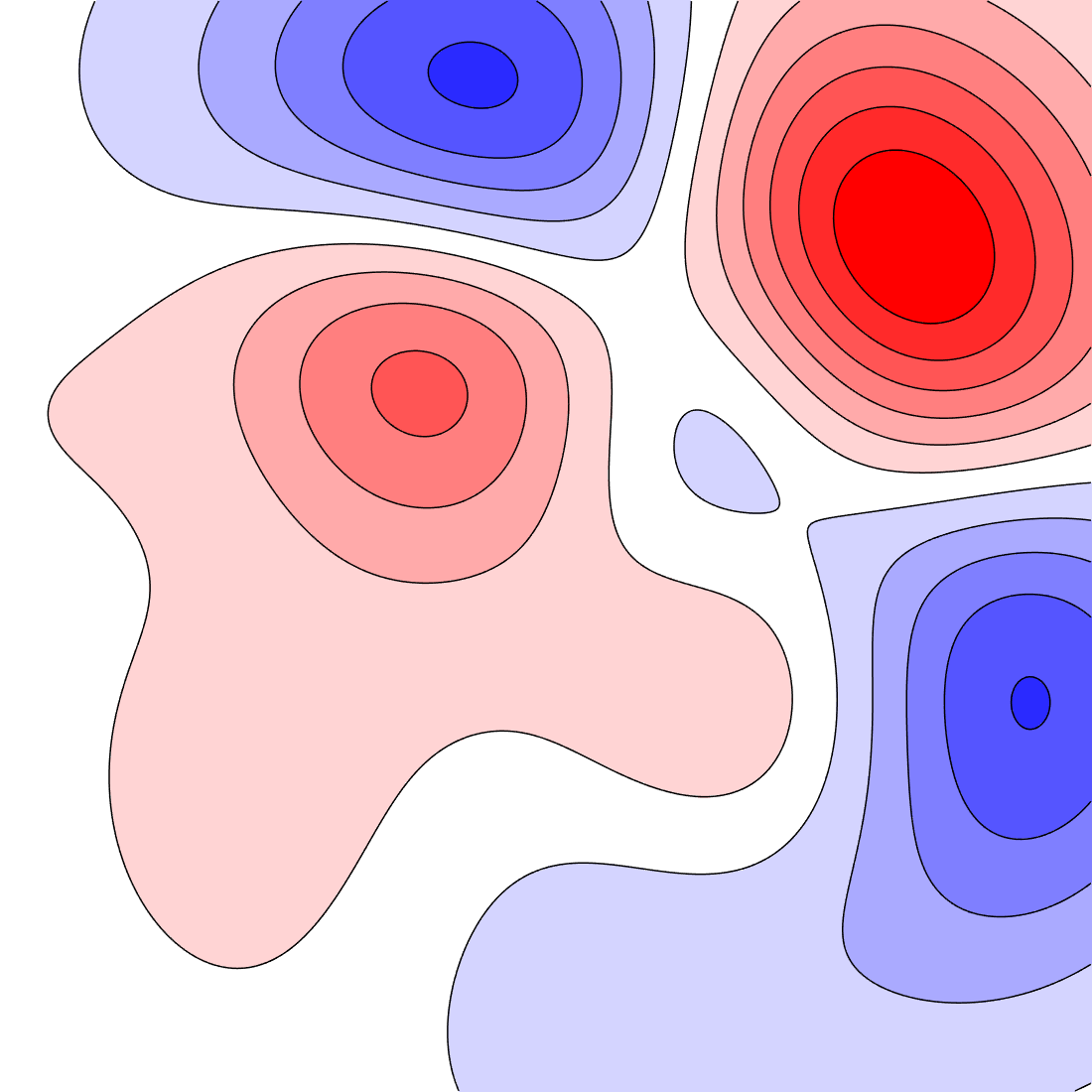}\\
$({\color{red} \al},{\color{blue} \be})$ & $\text{ED}(\RR^2,\norm{\cdot})$ & $(G,.005)$ &  $(G,.02)$ & $(G,.05)$
\end{tabular}
\caption{\label{fig-rkhs}
Top row: display of $\psi$ such that $\norm{\al-\be}_k = \norm{\psi \star (\al-\be)}_{L^2(\RR^2)}$, formally defined over Fourier as $\hat\psi(\om) = \sqrt{ \hat k_0(\om) }$, where $\Krkhs(x,x')=k_0(x-x')$.
Bottom row: display of $\psi \star (\al-\be)$.
(G,$\si$) stands for Gaussian kernel of variance $\si^2$. The kernel for $\text{ED}(\RR^2,\norm{\cdot})$ is $\psi(x)=1/\sqrt{\norm{x}}$.
\todoK{check that in 2D the ED sqrt kernel is $1/\sqrt{|x|}$, I think this is the case.} \todoK{Explain that ED and large sigma are too global to match the feature size, medium sigma only capture one scale of the details ; small sigma too small}
}
\end{figure}

\begin{example}[$H^{-1}(\RR^\dim)$]\label{exp-sobolev-neg1}
Another important dual norm is $H^{-1}(\RR^\dim)$, the dual (over distributions) of the Sobolev space $H^1(\RR^\dim)$ of functions having derivatives in $L^2(\RR^\dim)$.
It is defined using the primal RKHS norm $\norm{\nabla \f}_{L^2(\RR^\dim)}^2$. It is not defined for singular measures (\emph{e.g.} Diracs) unless $\dim=1$ because functions in the Sobolev space $H^1(\RR^\dim)$ are in general not continuous.
This $H^{-1}$ norm (defined on the space of zero mean measures with densities) can also be formulated in divergence form,
\eql{\label{eq-dual-sobolev-div}
	\norm{\al-\be}_{H^{-1}(\RR^\dim)}^2 = 
	\umin{\flow} \enscond{ \int_{\RR^\dim} \norm{\flow(x)}_2^2 \d x }{  \diverg(\flow)=\al-\be },
}
which should be contrasted with~\eqref{eq-w1-cont-div}, where an $L^1$ norm of the vector field $\flow$ was used in place of the $L^2$ norm used here.
The ``weighted'' version of this Sobolev dual norm,
\eq{
	\norm{\rho}_{H^{-1}(\al)}^2 = 
	\umin{\diverg(\flow)=\rho} \int_{\RR^\dim} \norm{\flow(x)}_2^2 \d \al(x),
}
can be interpreted as the natural ``linearization'' of the Wasserstein $\Wass_2$ norm, in the sense that the Benamou--Brenier dynamic formulation can be interpreted infinitesimally as 
\eql{\label{eq-wass-asymp-sob}
	\Wass_2(\al,\al+\varepsilon \rho) = \varepsilon \norm{\rho}_{H^{-1}(\al)} + o(\varepsilon).
}
The functionals $\Wass_2(\al,\be)$ and $\norm{\al-\be}_{H^{-1}(\al)}$ can be shown to be equivalent~\citep{peyre2011comparison}.
The issue is that $\norm{\al-\be}_{H^{-1}(\al)}$ is not a norm (because of the weighting by $\al$), and one cannot in general replace it by $\norm{\al-\be}_{H^{-1}(\RR^\dim)}$ unless $(\al,\be)$ have densities. 
In this case, if $\al$ and $\be$ have densities on the same support bounded from below by $a>0$ and from above by $b<+\infty$, then 
\eql{\label{eq-compar-w2-sob}
	b^{-1/2} \norm{\al-\be}_{H^{-1}(\RR^\dim)} \leq W_2(\al,\be) \leq a^{-1/2} \norm{\al-\be}_{H^{-1}(\RR^\dim)};
}
see~\citep[Theo. 5.34]{SantambrogioBook}, and see~\citep{peyre2011comparison} for sharp constants. 
%
\end{example}

\begin{example}[Negative Sobolev spaces]\label{exp-sobolev-neg}
	One can generalize this construction by considering the Sobolev space $H^{-r}(\RR^\dim)$ of arbitrary negative index, which is the dual of the functional Sobolev space $H^{r}(\RR^\dim)$ of functions having $r$ derivatives (in the sense of distributions) in $L^2(\RR^\dim)$.
	In order to metrize the weak convergence, one needs functions in $H^{r}(\RR^\dim)$  to be continuous, which is the case when $r>\dim/2$. As the dimension $\dim$ increases, one thus needs to consider higher regularity.
	For arbitrary $\al$ (not necessarily integers), these spaces are defined using the Fourier transform, and for a measure $\al$ with Fourier transform $\hat \al(\om)$ (written here as a density with respect to the Lebesgue measure $\d\om$)
	\eq{
		\norm{\al}_{H^{-r}(\RR^\dim)}^2 \eqdef \int_{\RR^\dim} \norm{\om}^{-2r} |\hat \al(\om)|^{2}\d\om.
	}
	This corresponds to a dual RKHS norm with a convolutive kernel $k(x,y)=k_0(x-y)$ with $\hat k_0(\om) = \pm \norm{\om}^{-2r}$. Taking the inverse Fourier transform, one sees that (up to constant) one has
	 \eql{\label{eq-sob-kern}
	 	\foralls x \in \RR^\dim, \quad  
		k_0(x) = 
		\choice{
			\frac{1}{\norm{x}^{d-2r}} \qifq r < d/2,  \\
			-\norm{x}^{2r-d} \qifq r>d/2. 
		}
	}
\end{example}

\begin{example}[Energy distance]\label{exp-energy-dist}
The energy distance (or Cramer distance when $\dim=1$)~\citep{szekely2004testing} associated to a distance $\dist$ is defined as
\eql{\label{eq-defn-ed}
	\norm{\al-\be}_{\text{ED}(\X,\dist^p)} \eqdef \norm{\al-\be}_{\Krkhs_{\text{ED}}}
	\qwhereq
	\Krkhs_{\text{ED}}(x,y) = -\dist(x,y)^p
}
for $0 < p < 2$. \todoK{is it correct?}
It is a valid MMD norm over measures if $\dist$ is negative definite (see Definition~\ref{def-negativedefinitekernel}), a typical example being the Euclidean distance $\dist(x,y)=\norm{x-y}$. 
For $\X=\RR^\dim$, $\dist(x,y)=\norm{\cdot}$, using~\eqref{eq-sob-kern}, one sees that the energy distance is a Sobolev norm
\eq{
	\norm{\cdot}_{\text{ED}(\RR^\dim,\norm{\cdot}^p)} = \norm{\cdot}_{H^{-\frac{d+p}{2}}(\RR^\dim)}.
}
A chief advantage of the energy distance over more usual kernels such as the Gaussian (Example~\ref{exmp-gaussian-kernel}) is that it is scale-free and does not depend on a bandwidth parameter $\si$. More precisely, one has the following scaling behavior on $\X=\RR^\dim$, when denoting $f_s(x)=sx$ the dilation by a factor $s>0$,
\eq{
	\norm{f_{s\sharp}(\al-\be)}_{\text{ED}(\RR^\dim,\norm{\cdot}^p)} = s^{\frac{p}{2}} \norm{\al-\be}_{\text{ED}(\RR^\dim,\norm{\cdot}^p)},
}
while the Wasserstein distance exhibits a perfect linear scaling,
\eq{
	\Wass_p( f_{s\sharp}\al,f_{s\sharp}\be) ) = s \Wass_p(\al,\be) ).
}
Note, however, that for the energy distance, the parameter $p$ must satisfy $0 < p < 2$, and that for $p=2$, it degenerates to the distance between the means 
\eq{
	\norm{\al-\be}_{\text{ED}(\RR^\dim,\norm{\cdot}^2)}
	= 
	\norm{\int_{\RR^\dim} x (\d\al(x)-\d\be(x))  }, 
}
so it is not a norm anymore. This shows that it is not possible to get the same linear scaling under $f_{s\sharp}$ with the energy distance as for the Wasserstein distance. 
\todoK{compare the sqrt kernel associated to energy and sobolev and comment its slow decay}
\end{example}

\section{Wasserstein Spaces Are Not Hilbertian}\label{sec-non-embeddability}

Some of the special cases of the Wasserstein geometry outlined earlier in~\S\ref{sec:specialcases} have highlighted the fact that the optimal transport distance can sometimes be computed in closed form. They also illustrate that in such cases the optimal transport distance is a \emph{Hilbertian} metric between probability measures, in the sense that there exists a map $\tphi$ from the space of input measures onto a Hilbert space, as defined below.

\begin{defn}\label{def-Hilbertmetric}
A distance $d$ defined on a set $\Z\times \Z$ is said to be Hilbertian if there exists a Hilbert space $\RKHS$ and a mapping $\tphi:\Z\rightarrow \RKHS$ such that for any pair $z,z'$ in $\Z$ we have that $d(z,z')=\|\tphi(z)-\tphi(z')\|_\RKHS$.
\end{defn}

For instance, Remark~\ref{rem-1d-ot-generic} shows that the Wasserstein metric is a Hilbert norm between univariate distributions, simply by defining $\tphi$ to be the map that associates to a measure its generalized quantile function. Remark~\ref{rem-dist-gaussians} shows that for univariate Gaussians, as written in~\eqref{eq-univariate-gauss} in this chapter, the Wasserstein distance between two univariate Gaussians is simply the Euclidean distance between their mean and standard deviation.

Hilbertian distances have many favorable properties when used in a data analysis context~\citep{dattorro2010convex}. First, they can be easily cast as radial basis function kernels: for any Hilbertian distance $\dist$, it is indeed known that $e^{-\dist^p/t}$ is a positive definite kernel for any value $0\leq p\leq 2$ and any positive scalar $t$ as shown in~\citep[Cor. 3.3.3, Prop. 3.2.7]{berg84harmonic}. The Gaussian $(p=2)$ and Laplace $(p=1)$ kernels are simple applications of that result using the usual Euclidean distance. The entire field of kernel methods~\citep{Hofmann2008} builds upon the positive definiteness of a kernel function to define convex learning algorithms operating on positive definite kernel matrices. Points living in a Hilbertian space can also be efficiently embedded in lower dimensions with low distortion factors \citep{Johnson84}, \citep[\S V.6.2]{barvinok2002course} using simple methods such as multidimensional scaling~\citep{borg2005modern}.

Because Hilbertian distances have such properties, one might hope that the Wasserstein distance remains Hilbertian in more general settings than those outlined above, notably when the dimension of $\X$ is 2 and more. This can be disproved using the following equivalence.

\begin{prop}\label{prop-negative-definite} A distance $d$ is Hilbertian if and only if $d^2$ is negative definite.
\end{prop}
\begin{proof} If a distance is Hilbertian, then $d^2$ is trivially negative definite. Indeed, given $n$ points in $\Z$, the sum $\sum r_i r_j d^2(z_i,z_j)$ can be rewritten as $\sum r_i r_j \|\tphi(z_i)-\tphi(z_j)\|_\RKHS^2$ which can be expanded, taking advantage of the fact that $\sum r_i=0$ to $-2\sum r_i r_j \dotp{\tphi(z_i)}{\tphi(z_j)}_\RKHS$ which is negative by definition of a Hilbert dot product. If, on the contrary, $d^2$ is negative definite, then the fact that $d$ is Hilbertian proceeds from a key result by \citet{schoenberg38} outlined in (\cite[p. 82, Prop. 3.2]{berg84harmonic}).
\end{proof}

It is therefore sufficient to show that the squared Wasserstein distance is not negative definite to show that it is not Hilbertian, as stated in the following proposition.

\begin{prop}\label{prop-negative-definite-negative} If $\X=\RR^\dim$ with $\dim\geq 2$ and the ground cost is set to $\dist(x,y)=\norm{x-y}_2$, then the $p$-Wasserstein distance is not Hilbertian for $p=1,2$.
\end{prop}

\begin{proof} It suffices to prove the result for $\dim=2$ since any counterexample in that dimension suffices to obtain a counterexample in any higher dimension. We provide a nonrandom counterexample which works using measures supported on four vectors $x^1, x^2,x^3,x^4\in\RR^{2}$ defined as follows: $x^{1}=[0,0], x^{2}=[1,0], x^{3}=[0,1], x^{4}=[1,1]$. We now consider all points on the regular grid on the simplex of four dimensions, with increments of $1/4$. There are $35=\bigl(\!\bigl(\substack{4\\4}\bigl)\!\bigl)=\binom{4+4-1}{4}$ such points in the simplex. Each probability vector $\a^i$ on that grid is such that for $j\leq 4$, we have that $\a^i_j$ is in the set $\{0,\tfrac{1}{4},\tfrac{1}{2},\tfrac{3}{4},1\}$ and such that $\sum_{j=1}^4\a^i_j=1$. For a given $p$, the $35\times 35$ pairwise Wasserstein distance matrix $\D_p$ between these histograms can be computed. $\D_p$ is not negative definite if and only if its elementwise square $\D_p^2$ is such that $\J \D_p^2 \J$ has positive eigenvalues, where $\J$ is the centering matrix $\J=\Identity_n-\tfrac{1}{n}\ones_{n,n}$, which is the case as illustrated in Figure~\ref{fig-counterexample}.
\end{proof}
	
\begin{figure}[h!]
\centering
\includegraphics[width=.6\linewidth]{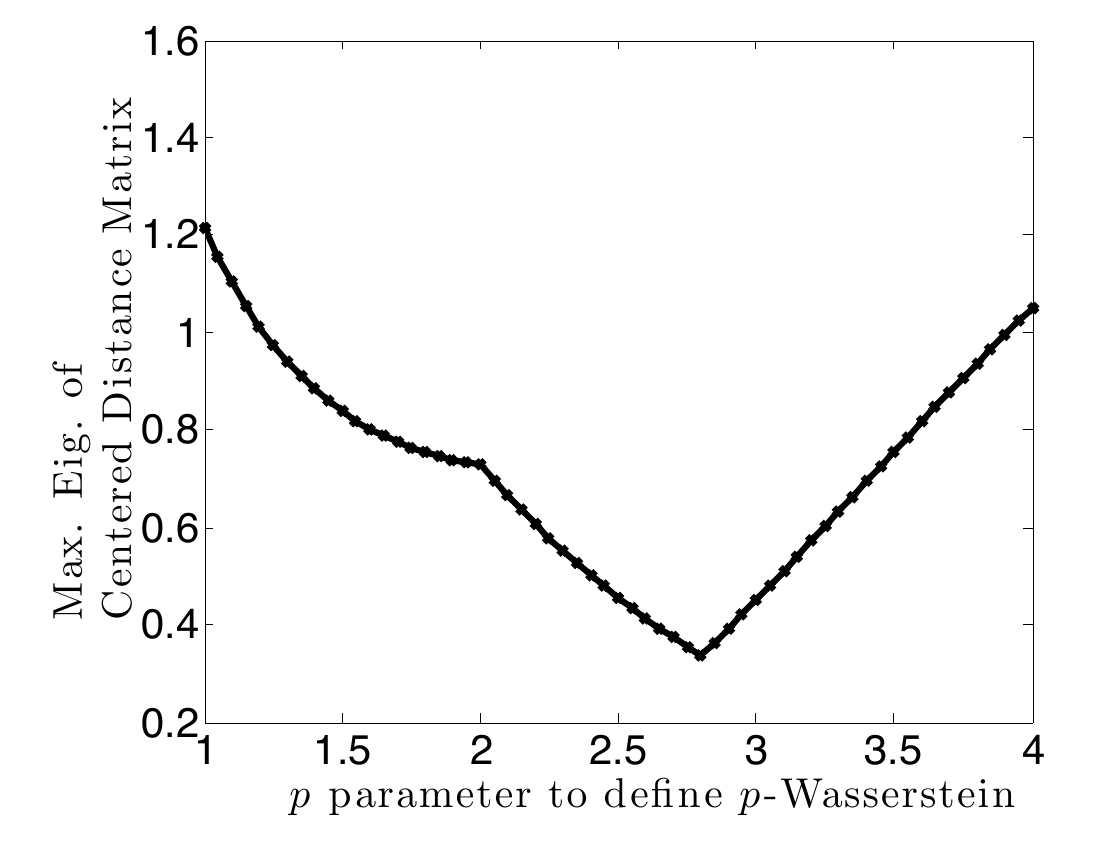}
\caption{\label{fig-counterexample}
One can show that a distance is \emph{not} Hilbertian by looking at the spectrum of the centered matrix $\J \D_p^2 \J$ corresponding to the pairwise squared-distance matrix $\D_p^2$ of a set of points. The spectrum of such a matrix is necessarily non-positive if the distance is Hilbertian. Here we plot the values of the maximal eigenvalue of that matrix for points selected in the proof of Proposition~\ref{prop-negative-definite-negative}. We do so for varying values of $p$, and display the maximal eigenvalues we obtain. These eigenvalues are all positive, which shows that for all these values of $p$, the $p$-Wasserstein distance is not Hilbertian.
}
\end{figure}

\subsection{Embeddings and Distortion}

An important body of work quantifies the hardness of approximating Wasserstein distances using Hilbertian embeddings. 
It has been shown that embedding measures in $\ell_2$ spaces incurs necessarily an important \todoK{need to be more precise here} distortion~(\citet{naor-2005,andoni2015snowflake}) as soon as $\X=\RR^\dim$ with $\dim\geq 2$.

It is possible to embed quasi-isometrically $p$-Wasserstein spaces for $0 < p \leq 1$ in $\ell_1$ (see~\citep{indyk,andoni2008earth,do2011sublinear}), but the equivalence constant between the distances grows fast with the dimension $\dim$. Note also that for $p=1$ the embedding is true only for discrete measures (\ie the embedding constant depends on the minimum distance between the spikes).  \todoK{need to be more precise here} 
A closely related embedding technique consists in using the characterization of $\Wass_1$ as the dual of Lipschitz functions $f$ (see~\S\ref{sec-w1-eucl}) and approximating the Lipschitz constraint $\norm{\nabla f}_1 \leq 1$ by a weighted $\ell_1$ ball over the wavelets coefficients; see~\citep{shirdhonkar2008approximate}. 
This weighted $\ell_1$ ball of wavelet coefficients defines a so-called Besov space of negative index~\citep{leeb2016holder}.
These embedding results are also similar to the bound on the Wasserstein distance obtained using dyadic partitions; see~\citep[Prop. 1]{weed2017sharp} and also~\citep{fournier2015rate}.
This also provides a quasi-isometric embedding in $\ell_1$ (this embedding being given by rescaled wavelet coefficients) and comes with the advantage that this embedding can be computed approximately in linear time when the input measures are discretized on uniform grids. We refer to~\citep{mallat2008wavelet} for more details on wavelets. Note that the idea of using multiscale embeddings to compute Wasserstein-like distances has been used extensively in computer vision; see, for instance,~\citep{ling2006diffusion,grauman2005pyramid,CuturiNIPS2006,lazebnik2006beyond}. 

\subsection{Negative/Positive Definite Variants of Optimal Transport}

We show later in~\S\ref{sec-sliced} that the \emph{sliced} approximation to Wasserstein distances, essentially a sum of 1-D directional transportation distance computed on random push-forwards of measures projected on lines, is negative definite as the sum of negative definite functions~\citep[\S 3.1.11]{berg84harmonic}. This result can be used to define a positive definite kernel~\citep{kolouri2016sliced}. Another way to recover a positive definite kernel is to cast the optimal transport problem as a soft-min problem (over all possible transportation tables) rather than a minimum, as proposed by~\citet{kosowsky1994invisible} to introduce entropic regularization. That soft-min defines a term whose neg-exponential (also known as a generating function) is positive definite~\citep{cuturi2012positivity}.

\section{Empirical Estimators for OT, MMD and $\phi$-divergences}
\label{sec-empirical-wass}

In an applied setting, given two input measures $(\al,\be) \in \Mm_+^1(\X)^2$, an important statistical problem is to approximate the (usually unknown) divergence $D(\al,\be)$ using only samples $(x_i)_{i=1}^n$ from $\al$ and $(y_j)_{j=1}^m$ from $\be$. These samples are assumed to be independently identically distributed from their respective distributions. 

\subsection{Empirical Estimators for OT and MMD}

For both Wasserstein distances $\Wass_p$ (see~\ref{eq-defn-wass-dist}) and MMD norms (see \S\ref{sec-dual-norms}), a straightforward estimator of the unknown distance between distriubtions is compute it directly between the empirical measures, hoping ideally that one can control the rate of convergence of the latter to the former,
\eq{
	D(\al,\be) \approx D(\hat \al_n,\hat \be_m) \qwhereq
	\choice{
		\hat \al_n \eqdef \frac{1}{n}\sum_i \de_{x_i},\\ 
		\hat \be_m \eqdef \frac{1}{m}\sum_j \de_{y_j}.
	}
}
Note that here both $\hat \al_n$ and $\hat \be_m$ are random measures, so $D(\hat \al_n,\hat \be_m)$ is a random number. 
For simplicity, we assume that $\X$ is compact (handling unbounded domain requires extra constraints on the moments of the input measures).

For such a dual distance that metrizes the weak convergence (see Definition~\ref{dfn-weak-conv}), since there is the weak convergence $\hat\al_n \rightarrow \al$, one has $D(\hat \al_n,\hat \be_n) \rightarrow D(\al,\be)$ as $n \rightarrow +\infty$.
But an important question is the speed of convergence of $D(\hat \al_n,\hat \be_n)$ toward $D(\al,\be)$, and this rate is often called the ``sample complexity'' of $D$. 

Note that for $D(\al,\be) = \norm{\cdot}_{\TV}$, since the TV norm does not metrize the weak convergence, $\norms{\hat \al_n - \hat \be_n}_{\TV}$ is not a consistent estimator, namely it does not converge toward $\norm{\al-\be}_{\TV}$. Indeed, with probability 1, $\norms{\hat \al_n - \hat \be_n}_{\TV}=2$ since the support of the two discrete measures does not overlap. Similar issues arise with other $\phi$-divergences, which cannot be estimated using divergences between empirical distributions.

\paragraph{Rates for OT.}

For $\X=\RR^\dim$ and measure supported on bounded domain, it is shown by~\citep{dudley1969speed} that for $\dim>2$, and $1 \leq p < +\infty$,  
\eq{
	\EE( |\Wass_p(\hat \al_n,\hat \be_n)-\Wass_p(\al,\be)| ) = O(n^{-\frac{1}{\dim}}),
}
where the expectation $\EE$ is taken with respect to the random samples $(x_i,y_i)_i$. This rate is tight in $\RR^\dim$ if one of the two measures has a density with respect to the Lebesgue measure. This result was proved for general metric spaces~\citep{dudley1969speed} using the notion of covering numbers and was later refined, in particular for $\X=\RR^d$ in~\citep{dereich2013constructive,fournier2015rate}. 
This rate can be refined when the measures are supported on low-dimensional subdomains:~\citet{weed2017sharp} show that, indeed, the rate depends on the intrinsic dimensionality of the support. \citeauthor{weed2017sharp} also study the nonasymptotic behavior of that convergence, such as for measures which are discretely approximated (\emph{e.g.} mixture of Gaussians with small variances).
It is also possible to prove concentration of $\Wass_p(\hat \al_n,\hat \be_n)$ around its mean $\Wass_p(\al,\be)$; see~\citep{bolley2007quantitative,boissard2011simple,weed2017sharp}.

\paragraph{Rates for MMD.}

For weak norms $\norm{\cdot}^2_{\Krkhs}$ which are dual of RKHS norms (also called MMD), as defined in~\eqref{eq-kernel-dual}, and contrary to Wasserstein distances, the sample complexity does not depend on the ambient dimension 
\eq{
	\EE( |\norms{\hat \al_n-\hat \be_n}_{\Krkhs} - \norm{\al-\be}_{\Krkhs} | ) = O(n^{-\frac{1}{2}}); 
}
see~\citep{sriperumbudur2012empirical}. 
Figure~\ref{fig-sample-complexity} shows a numerical comparison of the sample complexity rates for Wasserstein and MMD distances.  
Note, however, that $\norms{\hat \al_n-\hat \be_n}_{\Krkhs}^2$ is a slightly biased estimate of $\norm{\al-\be}_{\Krkhs}^2$. In order to define an unbiased estimator, and thus to be able to use, for instance, SGD when minimizing such losses, one should rather use the unbiased estimator 
\begin{align*}
	\text{MMD}_{\Krkhs}(\hat \al_n,\hat \be_n)^2 & \eqdef
		\frac{1}{n(n-1)} \sum_{i,i'} \Krkhs(x_i,x_{i'})	+
		\frac{1}{n(n-1)}\sum_{j,j'} \Krkhs(y_j,y_{j'}) \\
		 & \qquad - 2
		\frac{1}{n^2}\sum_{i,j} \Krkhs(x_i,y_j), 
\end{align*}
which should be compared to~\eqref{eq-mmd-discr}. It satisfies $\EE(\text{MMD}_{\Krkhs}(\hat \al_n,\hat \be_n)^2) = \norm{\al-\be}_{\Krkhs}^2$; see~\citep{gretton2012kernel}.

\begin{figure}[h!]
\centering
\begin{tabular}{@{}c@{\hspace{1mm}}c@{}}
\includegraphics[width=.49\linewidth]{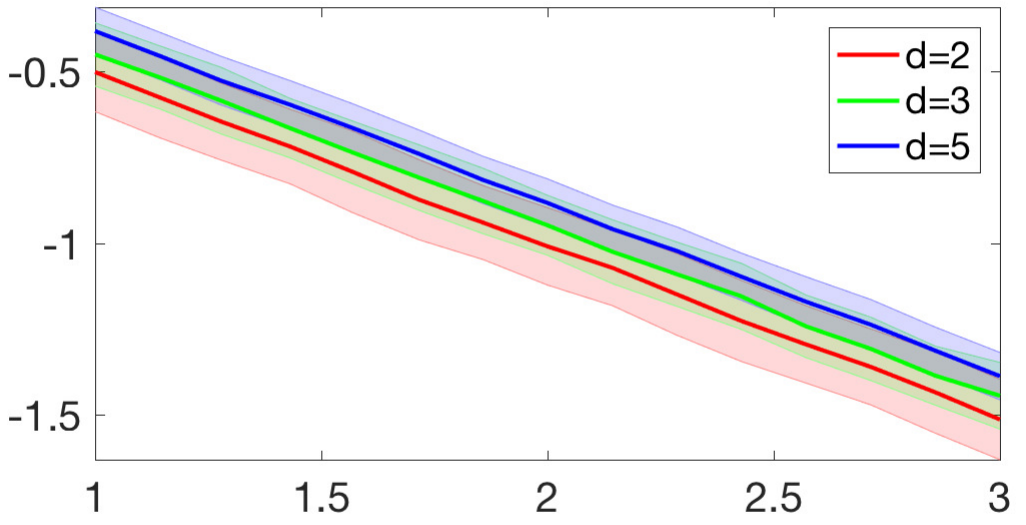}&
\includegraphics[width=.49\linewidth]{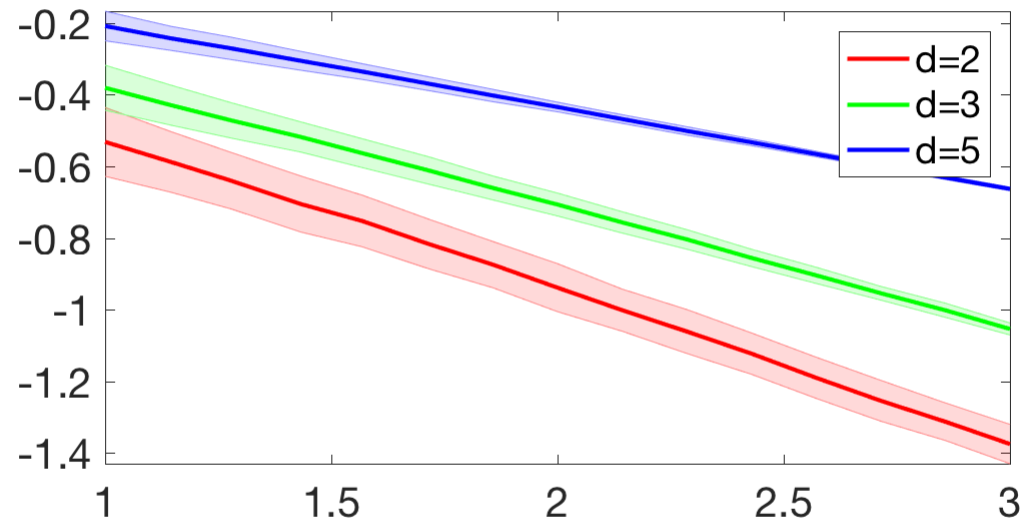}\\
Energy distance $\norm{\cdot}_{H^{-1}}$ & $\Wass_2$ 
\end{tabular}
\caption{\label{fig-sample-complexity}
Decay of $\log_{10}(D(\hat\al_n,\hat\al_n'))$ as a function of $\log_{10}(n)$ for $D$ being the energy distance $D=\norm{\cdot}_{H^{-1}}$ (\ie the $H^{-1}$ norm) as defined in Example~\ref{exp-energy-dist} (left) and the Wasserstein distance $D=\Wass_2$ (right).
Here $(\hat\al_n,\hat\al_n')$ are two independent empirical distributions of $\al$, the uniform distribution on the unit cube $[0,1]^\dim$, tested for several value of $\dim \in \{2,3,5\}$.
The shaded bar displays the confidence interval at $\pm$ the standard deviation of $\log(D(\hat\al_n,\al))$. 
}
\end{figure}

\subsection{Empirical Estimators for $\phi$-divergences}

It is not possible to approximate $\Divergm_\phi(\al|\be)$, as defined in~\eqref{def_divergence}, from discrete samples using $\Divergm_\phi(\hat \al_n|\hat \be_n)$. Indeed, this quantity is either $+\infty$ (for instance, for the $\KL$ divergence) or is not converging to $\Divergm_\phi(\al|\be)$ as $n \rightarrow +\infty$ (for instance, for the $\TV$ norm).
Instead, it is required to use a density estimator to somehow smooth the discrete empirical measures and replace them by densities; see~\citep{silverman1986density}. 
In a Euclidean space $\Xx=\RR^\dim$, introducing $h_\si=h(\cdot/\si)$ with a smooth windowing function and a bandwidth $\si>0$, a density estimator for $\al$ is defined using a convolution against this kernel, 
\eql{\label{eq-kernel-density-estimate}
	\hat \mu_n \star h_\si = \frac{1}{n}\sum_i h_\si(\cdot-x_i).
} 
One can then approximate the $\phi$ divergence using
\eq{
	\Divergm_\phi^\si(\hat \al_n |\hat \be_n)
	\eqdef
	\frac{1}{n} \sum_{j=1}^n
	\phi\pa{
		\frac{
			\sum_i h_\si(y_j-x_i)
		}{
			\sum_{j'} h_\si(y_j-y_{j'}),		
		}
	}
} 
where $\si$ should be adapted to the number $n$ of samples and to the dimension $\dim$. 
It is also possible to devise nonparametric estimators, bypassing the choice of a fixed bandwidth $\si$ to select instead a number $k$ of nearest neighbors. These methods typically make use of the distance between nearest neighbors~\citep{loftsgaarden1965nonparametric}, which is similar to locally adapting the bandwidth $\si$ to the local sampling density. Denoting $\De_k(x)$ the distance between $x \in \RR^\dim$ and its $k$th nearest neighbor among the $(x_i)_{i=1}^n$, a density estimator is defined as
\eql{\label{eq-knn-density-estimate}
	\rho_{\hat \al_n}^k(x) \eqdef \frac{k/n}{ |B_\dim| \De_k(x)^r },
}
where $|B_\dim|$ is the volume of the unit ball in $\RR^\dim$.
Instead of somehow ``counting'' the number of sample falling in an area of width $\si$ in~\eqref{eq-kernel-density-estimate}, this formula~\eqref{eq-knn-density-estimate} estimates the radius required to encapsulate $k$ samples. Figure~\ref{fig-density-estimation} compares the estimators~\eqref{eq-kernel-density-estimate} and~\eqref{eq-knn-density-estimate}.
A typical example of application is detailed in~\eqref{eq-discr-entropy} for the entropy functional, which is the KL divergence with respect to the Lebesgue measure. We refer to~\citep{BerishaHero} for more details.

\newcommand{\myFigDE}[1]{\includegraphics[width=.32\linewidth]{density-estimation/#1}}

\begin{figure}[h!]
\centering
\begin{tabular}{@{}c@{\hspace{1mm}}c@{\hspace{1mm}}c@{}}
\myFigDE{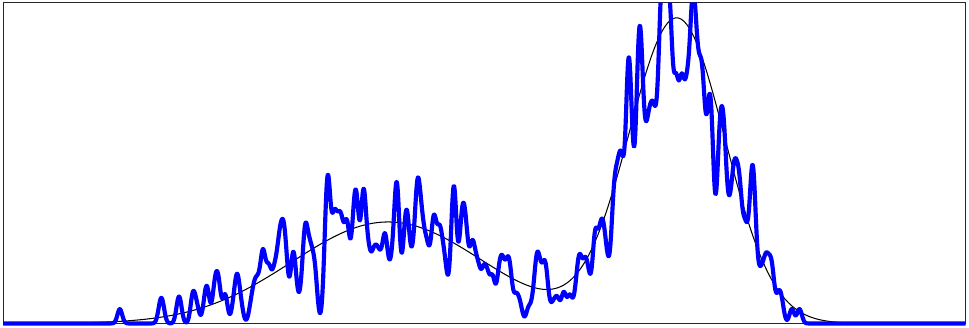} & \myFigDE{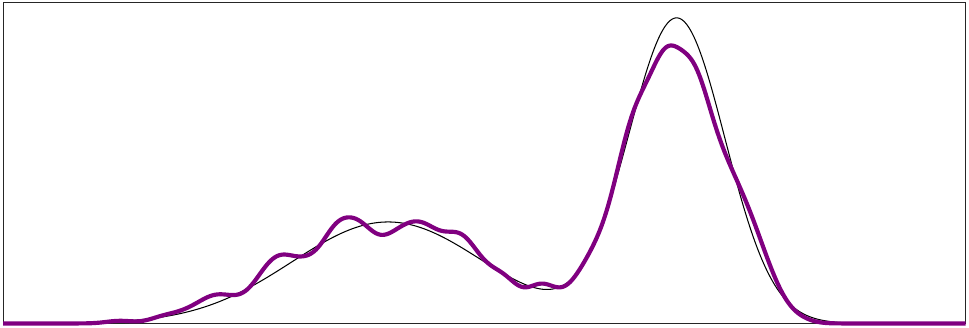} & \myFigDE{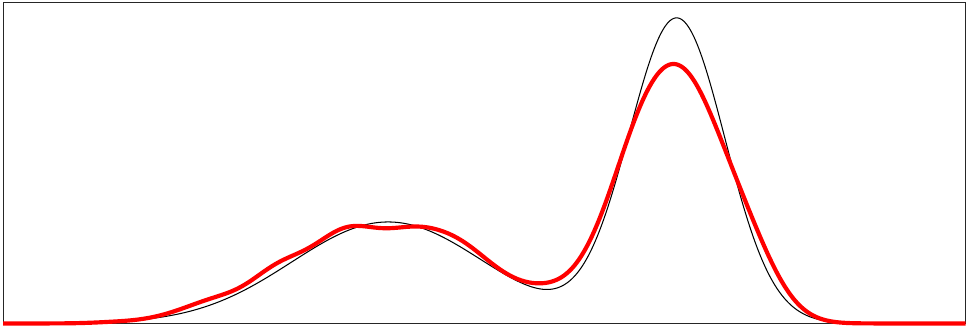} \\
$\si=2.5 \cdot 10^{-3}$ & $\si=15 \cdot 10^{-3}$ & $\si=25 \cdot 10^{-3}$ \\
\myFigDE{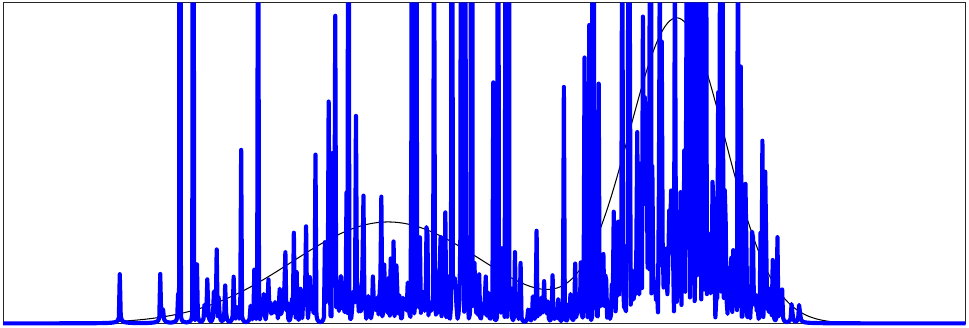} & \myFigDE{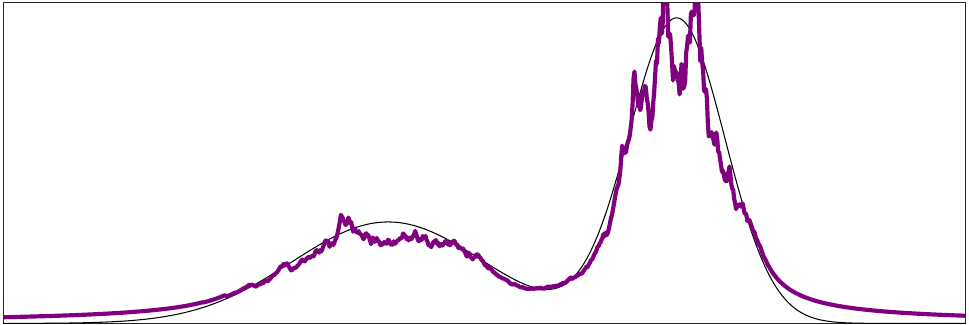} & \myFigDE{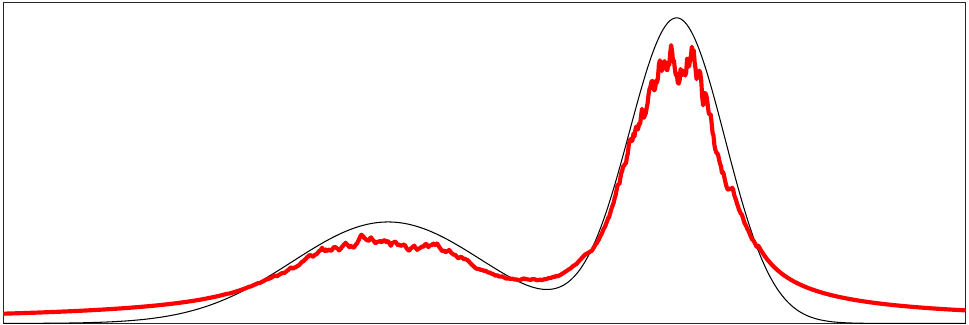} \\
$k=1$ & $k=50$ & $k=100$
\end{tabular}
\caption{\label{fig-density-estimation}
Comparison of kernel density estimation $\hat \mu_n \star h_\si$ (top, using a Gaussian kernel $h$) and $k$-nearest neighbors estimation $\rho_{\hat \al_n}^k$ (bottom) for $n=200$ samples from a mixture of two Gaussians.
}
\end{figure}

\section{Entropic Regularization: Between OT and MMD}
\label{sec-entropy-ot-mmd}

Following Proposition~\ref{prop-sinkhorn-div}, we recall that the Sinkhorn divergence is defined as 
\eq{
	\SINKHORNP_\C^\varepsilon(\a,\b) \eqdef 
	\dotp{\P^\star}{\C} = 
	\dotp{e^{\frac{\fD^\star}{\varepsilon}}}{(\K\odot \C)e^{\frac{\gD^\star}{\varepsilon}}},
}
where $\P^\star$ is the solution of~\eqref{eq-regularized-discr} while $(\fD^\star,\gD^\star)$ are solutions of~\eqref{eq-dual-formulation}.
Assuming $\C_{i,j}=\dist(x_i,x_j)^p$ for some distance $\dist$ on $\X$, for two discrete probability distributions of the form~\eqref{eq-pair-discr}, this defines a regularized Wasserstein cost
\eq{
	\Wass_{p,\varepsilon}(\al,\be)^p \eqdef \SINKHORNP_\C^\varepsilon(\a,\b).
}
This definition is generalized to any input distribution (not necessarily discrete) as
\eq{
	\Wass_{p,\varepsilon}(\al,\be)^p \eqdef \int_{\X \times \X} \dist(x,y)^p \d\pi^\star(x,y),
}	
where $\pi^\star$ is the solution of~\eqref{eq-entropic-generic}. 

In order to cancel the bias introduced by the regularization (in particular, $\Wass_{p,\varepsilon}(\al,\al) \neq 0$), we introduce a corrected regularized divergence
\eq{
	\tilde\Wass_{p,\varepsilon}(\al,\be)^p \eqdef
	2 \Wass_{p,\varepsilon}(\al,\be)^p
	-
	\Wass_{p,\varepsilon}(\al,\al)^p
	-
	\Wass_{p,\varepsilon}(\be,\be)^p.
}
It is proved in~\citep{feydy2018interpolating} that if $e^{-c/\epsilon}$ is a positive kernel, then a related corrected divergence (obtained by using $\MKD_\C^\varepsilon$ in place of $\SINKHORNP_\C^\varepsilon$) is positive. 
Note that it is possible to define other renormalization schemes using regularized optimal transport, as proposed, for instance, by~\citet{amari2017information}.

\begin{figure}[h!]
\centering
\begin{tabular}{@{}c@{\hspace{1mm}}c@{}}
\includegraphics[width=.49\linewidth]{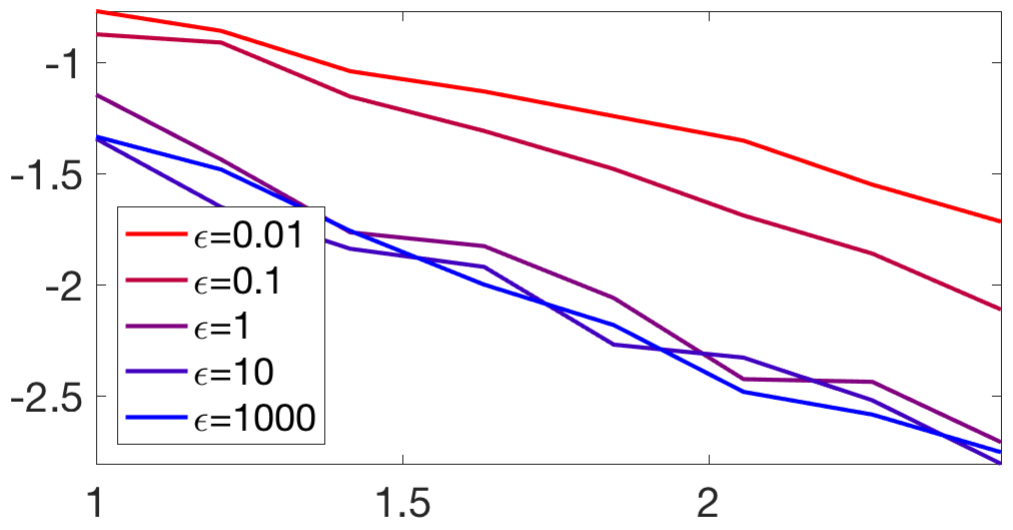}&
\includegraphics[width=.49\linewidth]{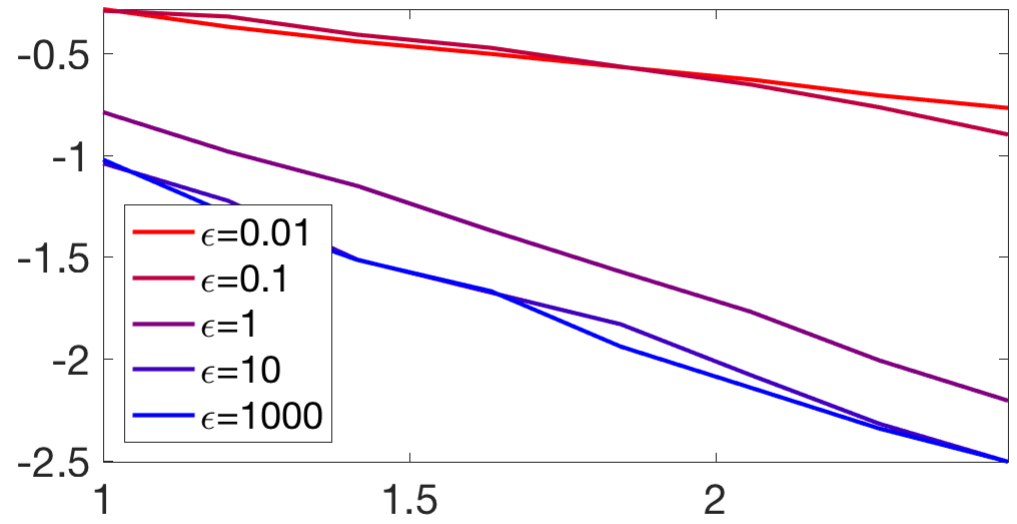}\\
$\dim=2$ & $\dim=5$
\end{tabular}
\caption{\label{fig-sample-complexity-regul}
Decay of $\EE(\log_{10}(\tilde\Wass_{p,\varepsilon}(\hat\al_n,\hat\al_n')))$, for $p=3/2$ for various $\varepsilon$, as a function of $\log_{10}(n)$ where $\al$ is the same as in Figure~\ref{fig-sample-complexity}. \todoK{redo the plot with higher precision}
}
\end{figure}

The following proposition, whose proof can be found in~\citep{ramdas2017wasserstein}, shows that this regularized divergence interpolates between the Wasserstein distance and the energy distance defined in Example~\ref{exp-energy-dist}.

\begin{prop}
	One has
	\eq{
		\tilde\Wass_{p,\varepsilon}(\al,\be)  \overset{\varepsilon \rightarrow 0}{\longrightarrow} 2 \Wass_{p}(\al,\be)
		\qandq
		\tilde\Wass_{p,\varepsilon}(\al,\be)^p \overset{\varepsilon \rightarrow +\infty}{\longrightarrow} \norm{\al-\be}_{\text{ED}(\X,\dist)}^2, 
	}
	where $\norm{\cdot}_{\text{ED}(\X,\dist)}$ is defined in~\eqref{eq-defn-ed}. 
\end{prop}

Figure~\ref{fig-sample-complexity-regul} shows numerically the impact of $\epsilon$ on the sample complexity rates. It is proved in~\cite{genevay2018sample}, in the case of $c(x,y)=\norm{x-y}^2$ on $\Xx=\RR^d$, that these rates interpolate between the ones of OT and MMD.


\chapter{Variational Wasserstein Problems}
\label{c-variational} 

In data analysis, common divergences between probability measures (\emph{e.g.} Euclidean, total variation, Hellinger, Kullback--Leibler) are often used to measure a fitting error or a loss in parameter estimation problems. Up to this chapter, we have made the case that the optimal transport geometry has a unique ability, not shared with other information divergences, to leverage physical ideas (mass displacement) and geometry (a ground cost between observations or bins) to compare measures. These two facts combined make it thus very tempting to use the Wasserstein distance as a loss function. This idea was recently explored for various applied problems. However, the main technical challenge associated with that idea lies in approximating and differentiating efficiently the Wasserstein distance. We start this chapter with a few motivating examples and show how the different numerical schemes presented in the first chapters of this book can be used to solve variational Wasserstein problems.

In image processing, the Wasserstein distance can be used as a loss to synthesize textures~\citep{2016-tartavel-siims}, to account for the discrepancy between statistics of synthesized and input examples. It is also used for image segmentation to account for statistical homogeneity of image regions~\citep{SchnorSegmentation,RabinPapadakisSSVM,peyre2012wasserstein,ni2009local,schmitzer2013object,liimage}.
The Wasserstein distance is also a very natural fidelity term for inverse problems when the measurements are probability measures, for instance, image restoration~\citep{lellmann2014imaging}, tomographic inversion~\citep{AbrahamRadon}, density regularization~\citep{Burger-JKO}, particle image velocimetry~\citep{saumier2015optimal}, sparse recovery and compressed sensing~\citep{indyk2011k}, and seismic inversion~\citep{metivier2016optimal}.
Distances between measures (mostly kernel-based as shown in~\S\ref{sec-mmd}) are routinely used for shape matching (represented as measures over a lifted space, often called currents) in computational anatomy~\citep{vaillant2005surface}, but OT distances offer an interesting alternative~\citep{2017-feydy-miccai}. 
To reduce the dimensionality of a dataset of histograms,~\citeauthor{lee1999learning} have shown that the nonnegative matrix factorization problem can be cast using the Kullback--Leibler divergence to quantify a reconstruction loss~\citep{lee1999learning}. When prior information is available on the geometry of the bins of those histograms, the Wasserstein distance can be used instead, with markedly different results~\citep{Sandler09,ZenICPR14,pmlr-v51-rolet16}.

All of these problems have in common that they require access to the gradients of Wasserstein distances, or approximations thereof. We start this section by presenting methods to approximate such gradients, then follow with three important applications that can be cast as variational Wasserstein problems.

\section{Differentiating the Wasserstein Loss}

In statistics, text processing or imaging, one must usually compare a probability distribution $\be$ arising from measurements to a model, namely a parameterized family of distributions $\{\al_\th,\th\in\Theta\}$, where $\Theta$ is a subset of a Euclidean space. Such a comparison is done through a ``loss'' or a ``fidelity'' term, which is the Wasserstein distance in this section.
In the simplest scenario, the computation of a suitable parameter $\th$ is obtained by minimizing directly
\eql{\label{eq-wloss-generic}
	\umin{\th \in \Theta} \Ee(\th) \eqdef \MK_\c(\al_\th,\be).
}
Of course, one can consider more complicated problems: for instance, the barycenter problem described in~\S\ref{sec-bary} consists in a sum of such terms. However, most of these more advanced problems can be usually solved by adapting tools defined for the basic case above, either using the chain rule to compute explicitly derivatives or using automatic differentiation as advocated in~\S\ref{rem-auto-diff}. 

\paragraph{Convexity.} The Wasserstein distance between two histograms or two densities is convex with respect to its two inputs, as shown by~\eqref{eq-dual} and~\eqref{eq-dual-generic}, respectively. Therefore, when the parameter $\theta$ is itself a histogram, namely $\Theta=\simplex_n$ and $\al_\theta=\theta$, or more generally when $\theta$ describes $K$ weights in the simplex, $\Theta=\simplex_K$, and $\alpha_\theta=\sum_{i=1}^K \theta_i \alpha_i$ is a convex combination of known atoms $\alpha_1,\dots,\alpha_K$ in $\simplex_N$, Problem~\eqref{eq-wloss-generic} remains convex (the first case corresponds to the barycenter problem, the second to one iteration of the dictionary learning problem with a Wasserstein loss~\citep{pmlr-v51-rolet16}). However, for more general parameterizations $\th\mapsto \al_\th$, Problem~\eqref{eq-wloss-generic} is in general not convex.

\paragraph{Simple cases.} For those simple cases where the Wasserstein distance has a closed form, such as univariate (see~\S\ref{rem-1d-ot-generic}) or elliptically contoured (see \S\ref{rem-dist-gaussians}) distributions, simple workarounds exist. They consist mostly in casting the Wasserstein distance as a simpler distance between suitable representations of these distributions (Euclidean on quantile functions for univariate measures, Bures metric for covariance matrices for elliptically contoured distributions of the same family) and solving Problem~\eqref{eq-wloss-generic} directly on such representations. 

In most cases, however, one has to resort to a careful discretization of $\al_\th$ to compute a local minimizer for Problem~\eqref{eq-wloss-generic}. Two approaches can be envisioned: Eulerian or Lagrangian. Figure~\ref{fig-eulerian-lagrangian} illustrates the difference between these two fundamental discretization schemes. At the risk of oversimplifying this argument, one may say that a Eulerian discretization is the most suitable when measures are supported on a low-dimensional space (as when dealing with shapes or color spaces), or for intrinsically discrete problems (such as those arising from string or text analysis). When applied to fitting problems where observations can take continuous values in high-dimensional spaces, a Lagrangian perspective is usually the only suitable choice.
%

\begin{figure}[h!]
\centering
\includegraphics[width=\linewidth]{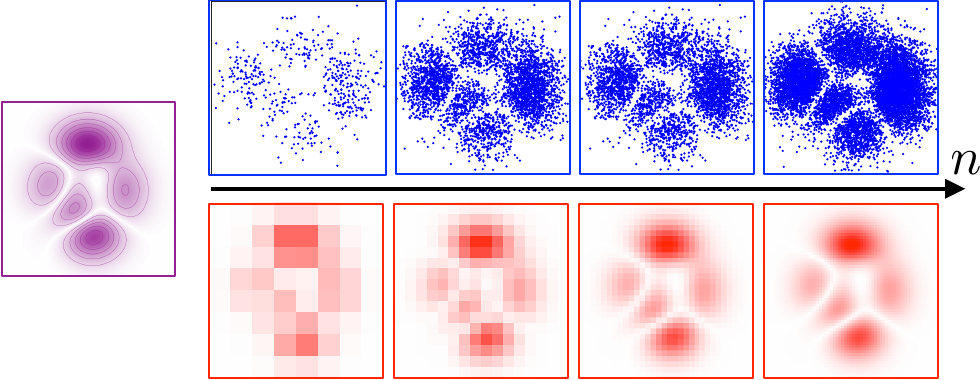}
\caption{\label{fig-eulerian-lagrangian}
Increasing fine discretization of a continuous distribution having a density (violet, left) using a Lagrangian representation $\frac{1}{n}\sum_i \de_{x_i}$ (blue, top) and an Eulerian representation $\sum_i \a_i \de_{x_i}$ with $x_i$ representing cells on a grid of increasing size (red, bottom). The Eulerian perspective starts from a pixelated image down to one with such fine resolution that it almost matches the original density. Weights $\a_i$ are directly proportional to each pixel-cell's intensity.
}
\end{figure}

\subsection{Eulerian Discretization}

A first way to discretize the problem is to suppose that both distributions $\be=\sum_{j=1}^m \b_j \de_{y_j}$ and $\al_\th=\sum_{i=1}^n \a(\th)_i \de_{x_i}$ are discrete distributions defined on fixed locations $(x_i)_i$ and $(y_j)_j$. Such locations might stand for cells dividing the entire space of observations in a grid, or a finite subset of points of interest in a continuous space (such as a family of vector embeddings for all words in a given dictionary~\citep{kusner2015word,pmlr-v51-rolet16}). The parameterized measure $\al_\th$ is in that case entirely represented through the weight vector $\a : \th \mapsto \a(\th) \in \Si_n$, which, in practice, might be very sparse if the grid is large. This setting corresponds to the so-called class of Eulerian discretization methods. In its original form, the objective of Problem~\eqref{eq-wloss-generic} is not differentiable. In order to obtain a smooth minimization problem, we use the entropic regularized OT and approximate~\eqref{eq-wloss-generic} using 
\eq{
	\umin{\th \in \Theta} \Ee_E(\th) \eqdef \MKD_\C^\epsilon(\a(\th),\b)
	\qwhereq
	\C_{i,j} \eqdef c(x_i,y_j).
}
We recall that Proposition~\ref{prop-convexity-dual} shows that the entropic loss function is differentiable and convex with respect to the input histograms, with gradient.

\begin{proposition}[Derivative with respect to histograms]
For $\epsilon>0$, 
$(\a,\b) \mapsto \MKD_\C^\epsilon(\a,\b)$ is convex and differentiable. Its gradient reads
\eql{\label{eq-diff-marginals}
		\nabla \MKD_\C^\epsilon(\a,\b) = (\fD,\gD),
}
where $(\fD,\gD)$ is the unique solution to~\eqref{eq-dual-formulation}, centered such that $\sum_i \fD_i=\sum_j \gD_j=0$. For $\epsilon=0$, this formula defines the elements of the sub-differential of $\MKD_\C^\epsilon$, and the function is differentiable if they are unique. 
\end{proposition}

The zero mean condition on $(\fD,\gD)$ is important when using gradient descent to guarantee conservation of mass.
%
Using the chain rule, one thus obtains that $\Ee_E$ is smooth and that its gradient is
\eql{\label{eq-diff-loss-eul}
	\nabla \Ee_E(\th) = [ \partial \a(\th) ]^\top( \fD\, ),
}
where $\partial \a(\th) \in \RR^{n \times \text{dim}(\Theta)}$ is the Jacobian (differential) of the map $\a(\th)$, and 
where $\fD \in \RR^n$ is the dual potential vector associated to the dual entropic OT~\eqref{eq-dual-formulation} between $\a(\th)$ and $\b$ for the cost matrix $\C$ (which is fixed in a Eulerian setting, and in particular independent of $\th$). 
This result can be used to minimize locally $\Ee_E$ through gradient descent.

\subsection{Lagrangian Discretization}

A different approach consists in using instead fixed (typically uniform) weights and approximating an input measure $\alpha$ as an empirical measure $\al_\th = \frac{1}{n}\sum_i \de_{x(\th)_i}$ for a point-cloud parameterization map $x : \th \mapsto x(\th) = (x(\th)_i)_{i=1}^n \in \Xx^n$, where we assume here that $\X$ is Euclidean. Problem~\eqref{eq-wloss-generic} is thus approximated as
\eql{\label{eq-lagr-sinkloss}
	\umin{\th} \Ee_L(\th) \eqdef \MKD_{\C(x(\th))}^\epsilon(\ones_n/n,\b)
	\qwhereq
	\C(x)_{i,j} \eqdef c(x(\theta)_i,y_j).
}
Note that here the cost matrix $\C(x(\th))$ now depends on $\th$ since the support of $\al_\th$ changes with $\theta$.
The following proposition shows that the entropic OT loss is a smooth function of the cost matrix and gives the expression of its gradient.

\begin{proposition}[Derivative with respect to the cost]
For fixed input histograms $(\a,\b)$, for $\epsilon>0$, the mapping $\C \mapsto \Rr(\C) \eqdef \MKD_\C^\epsilon(\a,\b)$ is concave and smooth, and  
\eql{\label{eq-diff-cost}
	\nabla \Rr(\C) = \P, 
}
where $\P$ is the unique optimal solution of~\eqref{eq-regularized-discr}. 
For $\epsilon=0$, this formula defines the set of upper gradients. 
\end{proposition}

\todoK{
\begin{proof}
	Use the primal formulation~\eqref{eq-regularized-discr}. \todoK{write me}
	Formula~\eqref{eq-diff-positions} follows using the chain rule. 
\end{proof}
}

Assuming $(\X,\Y)$ are convex subsets of $\RR^\dim$, for discrete measures $(\al,\be)$ of the form~\eqref{eq-pair-discr}, one obtains using the chain rule that $x = (x_i)_{i=1}^n  \in \X^n \mapsto \Ff(x) \eqdef  \MKD_{\C(x)}(\ones_n/n,\b)$ is smooth and that
\eql{\label{eq-diff-positions}
	\nabla \Ff(x) = \pa{ \sum_{j=1}^m \P_{i,j} \nabla_1 \c(x_i,y_j) }_{i=1}^n \in \X^n,
}
where $\nabla_1 \c$ is the gradient with respect to the first variable. 
For instance, for $\X=\Y=\RR^d$, for $\c(s,t) = \norm{s-t}^2$ on $\X=\Y=\RR^d$, one has
\eql{\label{eq-diff-positions-eucl}
	\nabla \Ff(x) = 2 \pa{ \a_i x_i - \sum_{j=1}^m \P_{i,j} y_j }_{i=1}^n,
}
where $\a_i=1/n$ here.
Note that, up to a constant, this gradient is $\Id-\T$, where $\T$ is the barycentric projection defined in~\eqref{eq-baryproj}.
Using the chain rule, one thus obtains that the Lagrangian discretized problem~\eqref{eq-lagr-sinkloss} is smooth and its gradient is
\eql{\label{eq-diff-loss-lagr}
	\nabla \Ee_L(\th) = [ \partial x(\th) ]^\top( \nabla \Ff( x(\th) ) ),
}
where $\partial x(\th) \in \RR^{\text{dim}(\Theta) \times (n\dim)}$ is the Jacobian of the map $x(\th)$ and
where $\nabla \Ff$ is implemented as in~\eqref{eq-diff-positions} or~\eqref{eq-diff-positions-eucl} using for $\P$ the optimal coupling matrix between $\al_\th$ and $\be$.
One can thus implement a gradient descent to compute a local minimizer of $\Ee_L$, as used, for instance, in~\citep{CuturiBarycenter}.

\subsection{Automatic Differentiation}
\label{rem-auto-diff}

The difficulty when applying formulas~\eqref{eq-diff-loss-eul} and~\eqref{eq-diff-loss-lagr} is that one needs to compute the exact optimal solutions $\fD$ or $\P$ for these formulas to be valid, which can only be achieved with acceptable precision using a very large number of Sinkhorn iterates.
In challenging situations in which the size and the quantity of histograms to be compared are large, the computational budget to compute a single Wasserstein distance is usually limited, therefore allowing only for a few Sinkhorn iterations. In that case, and rather than approximating the gradient~\eqref{eq-dual-formulation} using the value obtained at a given iterate, it is usually better to differentiate directly the output of Sinkhorn's algorithm, using reverse mode automatic differentiation. This corresponds to using the ``algorithmic'' Sinkhorn divergences as introduced in~\eqref{eq-algorithmic-loss}, rather than the quantity $\MKD_\C^\varepsilon$ in~\eqref{eq-regularized-discr} which incorporates the entropy of the regularized optimal transport, and differentiating it directly as a composition of simple maps using the inputs, either the histogram in the Eulerian case or the cost matrix in the Lagrangian cases. Using definitions introduced in~\S\ref{sec-regularized-cost}, this is equivalent to differentiating 
$$\itL{\SINKHORND_{\C}}(\a(\theta),\b) \quad\text{ or }\quad \itL{\SINKHORND_{\C(x(\theta))}}(\a,\b)$$
with respect to $\theta$, in, respectively, the Eulerian and the Lagrangian cases for $L$ large enough.

The cost for computing the gradient of functionals involving Sinkhorn divergences is the same as that of computation of the functional itself; see, for instance,~\citep{2016-bonneel-barycoord,2017-Genevay-AutoDiff} for some applications of this approach. We also refer to~\citep{adams2011ranking} for an early work on differentiating Sinkhorn iterations with respect to the cost matrix (as done in the Lagrangian framework), with applications to learning rankings.
Further details on automatic differentiation can be found in~\citep{griewank2008evaluating,rall1981automatic,Neidinger10}, in particular on the ``reverse mode,'' which is the fastest way to compute gradients. 
In terms of implementation, all recent deep-learning Python frameworks feature state-of-the-art reverse-mode differentiation and support for GPU/TPU computations~\citep{theano2016,abadi2016tensorflow,pytorch}, they should be adopted for any large-scale application of Sinkhorn losses.
We strongly encourage the use of such automatic differentiation techniques, since they have the same complexity as computing~\eqref{eq-diff-loss-eul} and~\eqref{eq-diff-loss-lagr}, these formulas being mostly useful to obtain a theoretical understanding of what automatic differentation is computing. The only downside is that reverse mode automatic differentation is memory intensive (the memory grows proportionally with the number of iterations).  There exist, however, subsampling strategies that mitigate this problem~\citep{griewank1992achieving}.

\section{Wasserstein Barycenters, Clustering and Dictionary Learning}
\label{sec-bary}

A basic problem in unsupervised learning is to compute the ``mean'' or ``barycenter'' of several data points. A classical way to define such a weighted mean of points $(x_s)_{s=1}^S \in \X^S$ living in a metric space $(\X,d)$ (where $d$ is a distance or more generally a divergence) is by solving a variational problem 
\eql{\label{eq-frechet-means}
	\umin{x \in \X} \sum_{s=1}^S \la_s d(x,x_s)^p 
}
for a given family of weights $(\la_s)_s \in \simplex_S$, where $p$ is often set to $p=2$. 
When $\X=\RR^d$ and $d(x,y)=\norm{x-y}_2$, this leads to the usual definition of the linear average $x=\sum_s \la_s x_s$ for $p=2$ and the more evolved median point when $p=1$. One can retrieve various notions of means (\emph{e.g.} harmonic or geometric means over $\X=\RR_+$) using this formalism.
This process is often referred to as the ``Fr\'echet''  or ``Karcher'' mean (see~\citet{karcher2014riemannian} for a historical account). For a generic distance $d$, Problem~\eqref{eq-frechet-means} is usually a difficult nonconvex optimization problem. Fortunately, in the case of optimal transport distances, the problem can be formulated as a convex program for which existence can be proved and efficient numerical schemes exist.

\paragraph{Fr\'echet means over the Wasserstein space.}

Given input histogram $\{\b_s\}_{s=1}^S$, where $b_s \in \simplex_{n_s}$, and weights $\la \in \simplex_S$, a Wasserstein barycenter is computed by minimizing
\eql{\label{eq-wass-discr}
	\umin{\a \in \simplex_n} \sum_{s=1}^S \la_s \MKD_{\C_s}(\a,\b_s),
}
where the cost matrices $\C_s \in \RR^{n \times n_s}$ need to be specified. 
A typical setup is ``Eulerian,'' so that all the barycenters are defined on the same grid, $n_s=n$, $\C_s=\C=\distD^p$ is set to be a distance matrix, to solve
\eq{
	\umin{\a \in \simplex_n} \sum_{s=1}^S \la_s \WassD_p^p(\a,\b_s). 
}

The barycenter problem~\eqref{eq-wass-discr} was introduced in a more general form involving arbitrary measures in~\citet{Carlier_wasserstein_barycenter} following earlier ideas of~\citet{carlierekelandmatching}. That presentation is deferred to Remark~\ref{rem-bary-carlier}. 
The barycenter problem for histograms~\eqref{eq-wass-discr} is in fact a linear program, since one can look for the $S$ couplings $(\P_s)_s$ between each input and the barycenter itself, which by construction must be constrained to share the same row marginal,
\eq{
	\umin{\a \in \simplex_n, (\P_s \in \RR^{n \times n_s})_s } \enscond{
		\sum_{s=1}^S \la_s \dotp{\P_s}{\C_s}
	}{
		\foralls s, \P_s^\top \ones_{n_s}=\a, \P_s^\top \ones_{n} = \b_s
	}.
}
Although this problem is an LP, its scale forbids the use of generic solvers for medium-scale problems. One can resort to using first order methods such as subgradient descent on the dual~\citep{Carlier-NumericsBarycenters}.

\begin{rem2}{Barycenter of arbitrary measures}\label{rem-bary-carlier}
	Given a set of input measure $(\be_s)_s$ defined on some space $\X$, the barycenter problem becomes 
	\eql{\label{eq-barycenter-generic}
		\umin{\al \in \Mm_+^1(\X)} \sum_{s=1}^S \la_s \MK_{\c}(\al,\be_s).
	}
	In the case where $\X=\RR^d$ and $c(x,y)=\norm{x-y}^2$,~\citet{Carlier_wasserstein_barycenter} show that if one of the input measures has a density, then this barycenter is unique. 
	Problem~\eqref{eq-barycenter-generic} can be viewed as a generalization of the problem of computing barycenters of points $(x_s)_{s=1}^S \in \X^S$ to arbitrary measures. Indeed, if $\be_s=\de_{x_s}$ is a single Dirac mass, then a solution to~\eqref{eq-barycenter-generic} is $\de_{x^\star}$, where $x^\star$ is a Fr\'echet mean solving~\eqref{eq-frechet-means}.
	Note that for $c(x,y)=\norm{x-y}^2$, the mean of the barycenter $\al^\star$ is necessarily the barycenter of the mean, \ie 
	\eq{
		\int_\Xx x \d\al^\star(x) =  \sum_s \la_s \int_\Xx x \d\al_s(x), 
	}
	and the support of $\al^\star$ is located in the convex hull of the supports of the $(\al_s)_s$.
	The consistency of the approximation of the infinite-dimensional optimization~\eqref{eq-barycenter-generic} when approximating the input distribution using discrete ones (and thus solving~\eqref{eq-wass-discr} in place) is studied in~\citet{Carlier-NumericsBarycenters}.
	Let us also note that it is possible to recast~\eqref{eq-barycenter-generic} as a multimarginal OT problem; see Remark~\ref{eq-multimarg-bary}. 
\end{rem2}

\begin{rem2}{$k$-means as a Wasserstein variational problem}\label{rem-bary-kmeans}
When the family of input measures $(\be_s)_s$ is limited to but one measure $\be$, this measure is supported on a discrete finite subset of $\X=\RR^{\dim}$, and the cost is the squared Euclidean distance, then one can show that the barycenter problem 
	\eql{\label{eq-barycenter-kmeans}
		\umin{\al \in \Mm_{k}^1(\X)} \MK_{\c}(\al,\be),
	}
where $\al$ is constrained to be a discrete measure with a finite support of size up to $k$, is equivalent to the usual $k$-means problem taking $\be$. Indeed, one can easily show that the centroids output by the $k$-means problem correspond to the support of the solution $\al$ and that its weights correspond to the fraction of points in $\be$ assigned to each centroid. One can show that approximating $\MK_{\c}$ using entropic regularization results in smoothed out assignments that appear in soft-clustering variants of $k$-means, such as mixtures of Gaussians~\citep{dessein2017parameter}.
\end{rem2}

\begin{rem2}{Distribution of distributions and consistency}
	It is possible to generalize~\eqref{eq-barycenter-generic} to a possibly infinite collection of measures. This problem is described by considering a probability distribution $M$ over the space $\Mm_+^1(\X)$ of probability distributions, \ie $M \in \Mm_+^1(\Mm_+^1(\X))$. A barycenter is then a solution of 
	\eql{\label{eq-bary-infinite} 
		\umin{\al \in \Mm_+^1(\X)} \EE_M( \MK_{\c}(\al,\be) ) = \int_{\Mm_+^1(\X)} \MK_{\c}(\al,\be) \d M(\be), 
	}
	where $\be$ is a random measure distributed according to $M$. Drawing uniformly at random a finite number $S$ of input measures $(\be_s)_{s=1}^S$ according to $M$, one can then define $\hat \be_S$ as being a solution of~\eqref{eq-barycenter-generic} for uniform weights $\la_s=1/S$ (note that here $\hat \be_S$ is itself a random measure). 
	Problem~\eqref{eq-barycenter-generic} corresponds to the special case of a ``discrete'' measure $M=\sum_s \la_s \de_{\be_s}$.
	The convergence (in expectation or with high probability) of $\MK_{\c}(\hat\be_S,\al)$ to zero (where $\al$ is the unique solution to~\eqref{eq-bary-infinite}) corresponds to the consistency of the barycenters, and is proved in~\citep{BigotBarycenter,leGouic2016existence,bigot2012characterization}.
	This can be interpreted as a law of large numbers over the Wasserstein space. The extension of this result to a central limit theorem is an important problem; see~\citep{panaretos2016amplitude} and~\citep{agueh2017vers} for recent formulations of that problem and solutions in particular cases (1-D distributions and Gaussian measures).
\end{rem2}

\begin{rem2}{Fixed-point map}
When dealing with the Euclidean space $\X=\RR^\dim$ with ground cost $\c(x,y)=\norm{x-y}^2$, it is possible to study the barycenter problem using transportation maps. Indeed, if $\al$ has a density, according to Remark~\ref{rem-exist-mongemap}, one can define optimal transportation maps $\T_s$ between $\al$ and $\al_s$, in particular such that $\T_{s,\sharp}\al = \al_s$. The average map 
\eq{
	\T^{(\al)} \eqdef \sum_{s=1}^S \la_s \T_s
} 
(the notation above makes explicit the dependence of this map on $\al$) is itself an optimal map between $\al$ and $\T^{(\al)}_\sharp \al$ (a positive combination of optimal maps is equal by Brenier's theorem, Remark~\ref{rem-exist-mongemap}, to the sum of gradients of convex functions, equal to the gradient of a sum of convex functions, and therefore optimal by Brenier's theorem again). 
As shown in~\citep{Carlier_wasserstein_barycenter}, first order optimality conditions of the barycenter problem~\eqref{eq-bary-infinite} actually read $\T^{(\al^\star)}=\Identity_{\RR^\dim}$ (the identity map) at the optimal measure $\al^\star$ (the barycenter), and it is shown in~\citep{alvarez2016fixed} that the barycenter $\al^\star$ is the unique (under regularity conditions clarified in~\citep[Theo. 2]{zemel2017fr}) to the fixed-point equation 
\eql{\label{eq-fixed-point-bary}
	G(\al) = \al
	\qwhereq 
	G(\al) \eqdef \T^{(\al)}_\sharp \al,
}  
Under mild conditions on the input measures,~\citet{alvarez2016fixed} and~\citet{zemel2017fr} have shown that $\al \mapsto G(\al)$ strictly decreases the objective function of~\eqref{eq-bary-infinite} if $\al$ is not the barycenter and that the fixed-point iterations $\itt{\al} \eqdef G( \it{\al} )$ converge to the barycenter $\al^\star$. 
This fixed point algorithm can be used in cases where the optimal transportation maps are known in closed form (\emph{e.g.} for Gaussians). Adapting this algorithm for empirical measures of the same size results in computing optimal assignments in place of Monge maps. For more general discrete measures of arbitrary size the scheme can also be adapted~\citep{CuturiBarycenter} using barycentric projections~\eqref{eq-baryproj}.
\end{rem2}

\paragraph{Special cases.}

In general, solving~\eqref{eq-wass-discr} or~\eqref{eq-barycenter-generic} is not straightforward, but there exist some special cases for which solutions are explicit or simple.

\begin{rem1}{Barycenter of Gaussians}
	It is shown in~\citep{Carlier_wasserstein_barycenter} that the barycenter of Gaussians distributions $\al_s = \Nn(\mean_s,\cov_s)$, for the squared Euclidean cost $c(x,y)=\norm{x-y}^2$, is itself a Gaussian $\Nn(\mean^\star,\cov^\star)$.
	Making use of~\eqref{eq-dist-gauss}, one sees that the barycenter mean is the mean of the inputs
	\eq{
		\mean^\star = \sum_s \la_s \mean_s
	}
	while the covariance minimizes 
	\eq{
		\umin{ \cov } \sum_s \la_s \Bb(\cov,\cov_s)^2,
	}
	where $\Bb$ is the Bure metric~\eqref{eq-bure-defn}. As studied in~\citep{Carlier_wasserstein_barycenter}, the first order optimality condition of this convex problem shows that $\cov^\star$ is the unique positive definite fixed point of the map
	\eq{
		\cov^\star = \Psi(\cov^\star)
		\qwhereq
		\Psi(\cov) \eqdef \sum_s \la_s ( \cov^{\frac{1}{2}} \cov_s \cov^{\frac{1}{2}}  )^{\frac{1}{2}},
	}
	where $\cov^{\frac{1}{2}}$ is the square root of positive semidefinite matrices. 
	This result was known from~\citep{KnottSmith,RuschendorfUckelmann} and is proved in~\citep{Carlier_wasserstein_barycenter}. 
	While $\Psi$ is not strictly contracting, iterating this fixed-point map, \ie defining $\itt{\cov} \eqdef \Psi( \it{\cov} )$ converges in practice to the solution $\cov^\star$.
		%
	This method has been applied to texture synthesis in~\citep{2014-xia-siims}. \citet{alvarez2016fixed} have also proposed to use an alternative map 
	\eq{
		\bar\Psi(\cov) \eqdef 
		\cov^{-\frac{1}{2}} 
		\Big( \sum_s \la_s ( \cov^{\frac{1}{2}} \cov_s \cov^{\frac{1}{2}}  )^{\frac{1}{2}} \Big)^2 
		\cov^{-\frac{1}{2}}
	}
	for which the iterations $\itt{\cov} \eqdef \bar\Psi( \it{\cov} )$ converge.
	This is because the fixed-point map $G$ defined in~\eqref{eq-fixed-point-bary} preserves Gaussian distributions, and in fact, 
	\eq{
		G( \Nn(\mean,\cov) ) = \Nn(\mean^\star,\bar\Psi(\cov)). 		
	}
	Figure~\ref{fig-bary-gaussian} shows two examples of computations of barycenters between four 2-D Gaussians.
\end{rem1}

\begin{figure}[h!]
\centering
\begin{tabular}{@{}c@{\hspace{20mm}}c@{}}
\includegraphics[width=.3\linewidth]{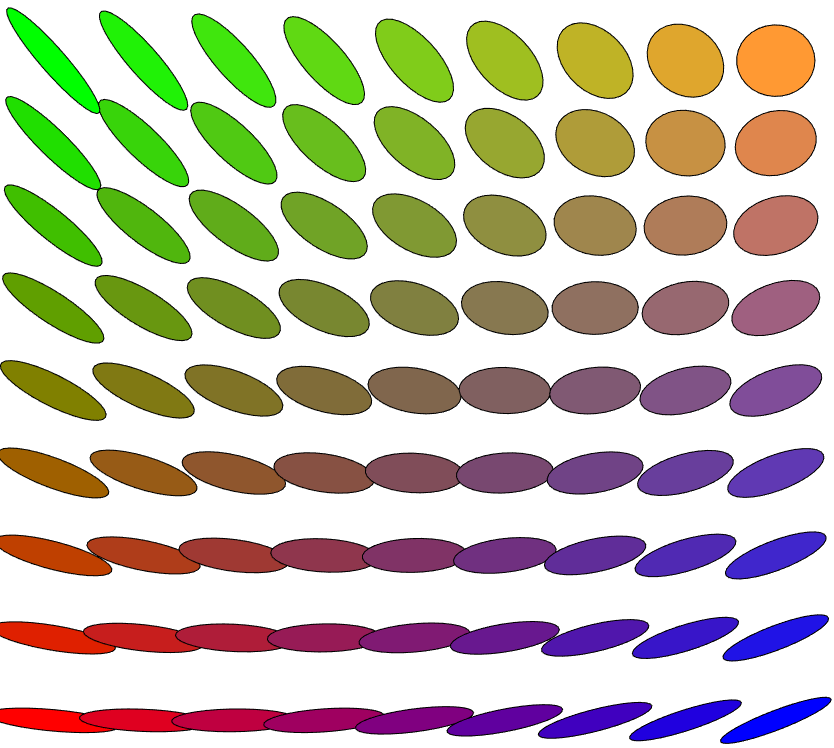} &
\includegraphics[width=.3\linewidth]{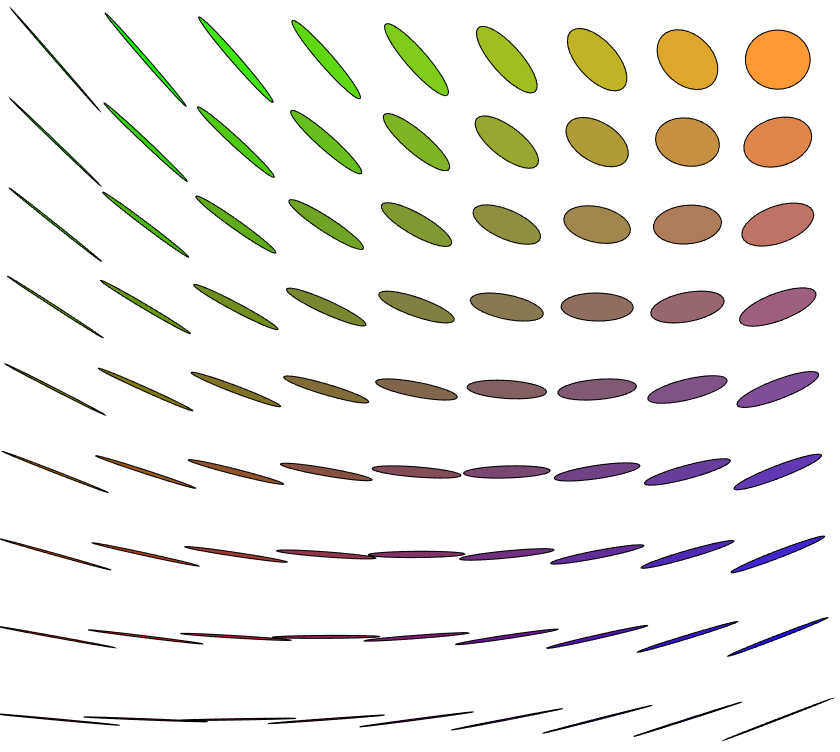} 
\end{tabular}
\caption{\label{fig-bary-gaussian}
Barycenters between four Gaussian distributions in 2-D. Each Gaussian is displayed using an ellipse aligned with the principal axes of the covariance, and with elongations proportional to the corresponding eigenvalues. 
}
\end{figure}

\begin{rem1}{1-D cases}\label{rem-bary-1d}
	For 1-D distributions, the $\Wass_p$ barycenter can be computed almost in closed form using the fact that the transport is the monotone rearrangement, as detailed in Remark~\ref{rem-1d-ot-generic}. 
	The simplest case is for empirical measures with $n$ points, \ie $\be_s = \frac{1}{n}\sum_{i=1}^n \de_{y_{s,i}}$, where the points are assumed to be sorted $y_{s,1} \leq y_{s,2} \leq \ldots$. Using~\eqref{eq-1d-empirical} the barycenter $\al_\la$ is also an empirical measure on $n$ points
	\eq{
		\al_\la = \frac{1}{n}\sum_{i=1}^n \de_{x_{\la,i}}
		\qwhereq
		x_{\la,i} = A_\la(x_{s,i})_s, 
	}
	where $A_\la$ is the barycentric map
	\eq{
	 	A_\la(x_s)_s \eqdef \uargmin{x \in \RR} \sum_{s=1}^S \la_s |x-x_{s}|^p.
	}
	For instance, for $p=2$, one has $x_{\la,i} = \sum_{s=1}^S \la_s x_{s,i}$. 
	In the general case, one needs to use the cumulative functions as defined in~\eqref{eq-cumul-defn}, and using~\eqref{eq-wass-cumul}, one has
	\eq{
		\foralls r \in [0,1], \quad
		\cumul{\al_\la}^{-1}(r) = A_\la( \cumul{\al_s}^{-1}(r) )_{s=1}^S, 
	}
	which can be used, for instance, to compute barycenters between discrete measures supported on less than $n$ points in $O(n\log(n))$ operations, using a simple sorting procedure.
\end{rem1}

\begin{rem1}{Simple cases}
	Denoting by $\T_{r,u} : x \mapsto rx+u$ a scaling and translation, and assuming that $\al_{s} = T_{r_s,u_s,\sharp} \al_0$ is obtained by scaling and translating an initial template measure, then a barycenter $\al_\la$ is also obtained using scaling and translation \todoK{to check}
	\eq{
		\al_\la = T_{r^\star,u^\star,\sharp} \al_0
		\qwhereq
		\choice{
			r^\star = (\sum_s \la_s/r_s)^{-1},  \\
			u^\star = \sum_s \la_s u_s.
		}
	}	
\end{rem1}

\begin{rem1}{Case $S=2$}
	In the case where $\X=\RR^d$ and $c(x,y)=\norm{x-y}^2$ (this can be extended more generally to geodesic spaces), the barycenter between $S=2$ measures $(\al_0,\al_1)$ is the McCann interpolant as already introduced in~\eqref{eq-mc-cann-interp}. Denoting $\T_\sharp \al_0=\al_1$ the Monge map, one has that the barycenter $\al_\la$ reads $\al_\la = (\la_1 \Id + \la_2 \T)_{\sharp} \al_0$.
	Formula~\eqref{eq-mccann-discrete} explains how to perform the computation in the discrete case.
\end{rem1}

\paragraph{Entropic approximation of barycenters.}

One can use entropic smoothing and approximate the solution of~\eqref{eq-wass-discr} using 
\eql{\label{eq-entropic-bary}
	\umin{\a \in \simplex_n} \sum_{s=1}^S \la_s \MKD_{\C_s}^\epsilon(\a,\b_s)
}
for some $\epsilon>0$. 
This is a smooth convex minimization problem, which can be tackled using gradient descent~\citep{CuturiBarycenter,GramfortPC15}. An alternative is to use descent methods (typically quasi-Newton) on the semi-dual~\citep{2016-Cuturi-siims}, which is useful to integrate additional regularizations on the barycenter, to impose, for instance, some smoothness w.r.t a given norm.
A simpler yet very effective approach, as remarked by~\citet{2015-benamou-cisc} is to rewrite~\eqref{eq-entropic-bary} as a (weighted) KL projection problem
\eql{\label{eq-bary-entropy-couplings}
	\umin{ (\P_s)_s } \enscond{ \sum_{s} \la_s \epsilon \KLD( \P_s|\K_s ) }{
		\foralls s, \transp{\P_s}\ones_m = \b_s, \:
		\P_1\ones_1 =  \cdots = \P_S\ones_S,
	}
}
where we denoted $\K_s \eqdef e^{-\C_s/\epsilon}$. Here, the barycenter $\a$ is implicitly encoded in the row marginals of all the couplings $\P_s \in \RR^{n \times n_s}$ as $\a = \P_1\ones_1 =  \cdots = \P_S\ones_S$.
As detailed by~\citet{2015-benamou-cisc}, one can generalize Sinkhorn to this problem, which also corresponds to iterative projections. This can also be seen as a special case of the generalized Sinkhorn detailed in~\S\ref{sec-generalized}.
The optimal couplings $(\P_s)_s$ solving~\eqref{eq-bary-entropy-couplings} are computed in scaling form as 
\eql{\label{eq-bary-opt}
	\P_s=\diag(\uD_s)\K\diag(\vD_s), 
}
and the scalings are sequentially updated as
\begin{align}\label{eq-sinkhorn-bary}
	\foralls s \in \range{1,S}, \quad \itt{\vD}_s &\eqdef \frac{\b_s}{\transp{\K}_s \it{\uD}_s}, \\
	\foralls s \in \range{1,S}, \quad  \itt{\uD}_s &\eqdef \frac{\itt{\a}}{\K_s \itt{\vD}_s}, \label{eq-sinkhorn-bary-2}\\
		\qwhereq
		\itt{\a} &\eqdef \prod_s (  \K_s \itt{\vD}_s )^{\la_s}. \label{eq-sinkhorn-bary-3}
\end{align}
An alternative way to derive these iterations is to perform alternate minimization on the variables of a dual problem, which is detailed in the following proposition.

\begin{prop}
	The optimal $(\uD_s,\vD_s)$ appearing in~\eqref{eq-bary-opt} can be written as $(\uD_s,\vD_s) = ( e^{\fD_s/\epsilon},e^{\gD_s/\epsilon} )$, where $( \fD_s,\gD_s )_s$ are the solutions of the following program (whose value matches the one of \eqref{eq-entropic-bary}):
	\eql{\label{eq-dual-bary-entropy}
		\umax{ ( \fD_s,\gD_s )_s } \enscond{
		\sum_s \la_s \pa{
			\dotp{\gD_s}{\b_s} - \epsilon \dotp{\K_s e^{\gD_s/\epsilon}}{e^{\fD_s/\epsilon}}
		}
		}{
			\sum_s \la_s \fD_s = 0
		}.
	}
\end{prop}

\begin{proof}	Introducing Lagrange multipliers in~\eqref{eq-bary-entropy-couplings} leads to
	\begin{align*}
		\umin{ (\P_s)_s, \a } 
		\umax{ ( \fD_s,\gD_s )_s }
		\sum_{s} \la_s \Big(
			\epsilon \KLD( \P_s|\K_s )
			+ 
			\dotp{\a - \P_s\ones_m}{ \fD_s } \\
		\qquad\qquad	+			
			\dotp{\b_s - \transp{\P_s}\ones_m}{ \gD_s }
		\Big).
	\end{align*}
	Strong duality holds, so that one can exchange the min and the max, to obtain
		\begin{align*}
		&\umax{ ( \fD_s,\gD_s )_s }
			\sum_s \la_s \pa{
				 \dotp{\gD_s}{\b_s}
		+ 
		\umin{\P_s} \epsilon \KLD( \P_s|\K_s ) - \dotp{\P_s}{\fD_s \oplus \gD_s }
		} \\
	&\qquad\qquad	+ \umin{\a}
			\dotp{ \sum_{s} \la_s \fD_s }{\a}.
	\end{align*}
	The explicit minimization on $\a$ gives the constraint $\sum_{s} \la_s \fD_s=0$ together with 
	\begin{align*}
		\umax{ ( \fD_s,\gD_s )_s }
			\sum_s \la_s 
				 \dotp{\gD_s}{\b_s}
		-
		\epsilon \KLD^*\pa{ \frac{\fD_s \oplus \gD_s}{\epsilon}|\K_s },
	\end{align*}
	where $\KLD^*(\cdot|\K_s)$ is the Legendre transform~\eqref{eq-legendre} of the function $\KLD^*(\cdot|\K_s)$.
	This Legendre transform reads
	\eql{\label{eq-legendre-kl}
		\KLD^*(\VectMode{U}|\K) = \sum_{i,j} \K_{i,j} (e^{\VectMode{U}_{i,j}}-1), 
	}
	which shows the desired formula. 
	To show~\eqref{eq-legendre-kl}, since this function is separable, one needs to compute
	\eq{
		\foralls (u,k) \in \RR_+^2, \quad
		\KLD^*(u|k) \eqdef \umax{r} ur -   \pa{ r \log(r/k) - r + k  }
	}
	whose optimality condition reads $u=\log(r/k)$, \ie $r = k e^{u}$, hence the result.
\end{proof}

Minimizing~\eqref{eq-dual-bary-entropy} with respect to each $\gD_s$, while keeping all the other variables fixed, is obtained in closed form by~\eqref{eq-sinkhorn-bary}. 
Minimizing~\eqref{eq-dual-bary-entropy} with respect to all the $(\fD_s)_s$ requires us to solve for $\a$ using~\eqref{eq-sinkhorn-bary-3} and leads to the expression~\eqref{eq-sinkhorn-bary-2}.

Figures~\ref{fig-barycenters-images} and~\ref{fig-barycenters-shapes} show applications to 2-D and 3-D shapes interpolation. Figure~\ref{fig-barycenters-surfaces} shows a computation of barycenters on a surface, where the ground cost is the square of the geodesic distance. For this figure, the computations are performed using the geodesic in heat approximation detailed in Remark~\ref{rem-geod-heat}. We refer to~\citep{2015-solomon-siggraph} for more details and other applications to computer graphics and imaging sciences.

\newcommand{\FigBaryImA}[2]{\includegraphics[width=.088\linewidth]{barycenters-2d/annulus-cross-heart-2disk/shape-#1-#2}}
\newcommand{\FigBaryImLineA}[1]{\FigBaryImA{#1}{1} & \FigBaryImA{#1}{2} & \FigBaryImA{#1}{3} & \FigBaryImA{#1}{4} & \FigBaryImA{#1}{5}}
\newcommand{\FigBaryImB}[2]{\includegraphics[width=.088\linewidth]{barycenters-2d/cat-star8-thinspiral-trefle/shape-#1-#2}}
\newcommand{\FigBaryImLineB}[1]{\FigBaryImB{#1}{1} & \FigBaryImB{#1}{2} & \FigBaryImB{#1}{3} & \FigBaryImB{#1}{4} & \FigBaryImB{#1}{5}}

\begin{figure}[h!]
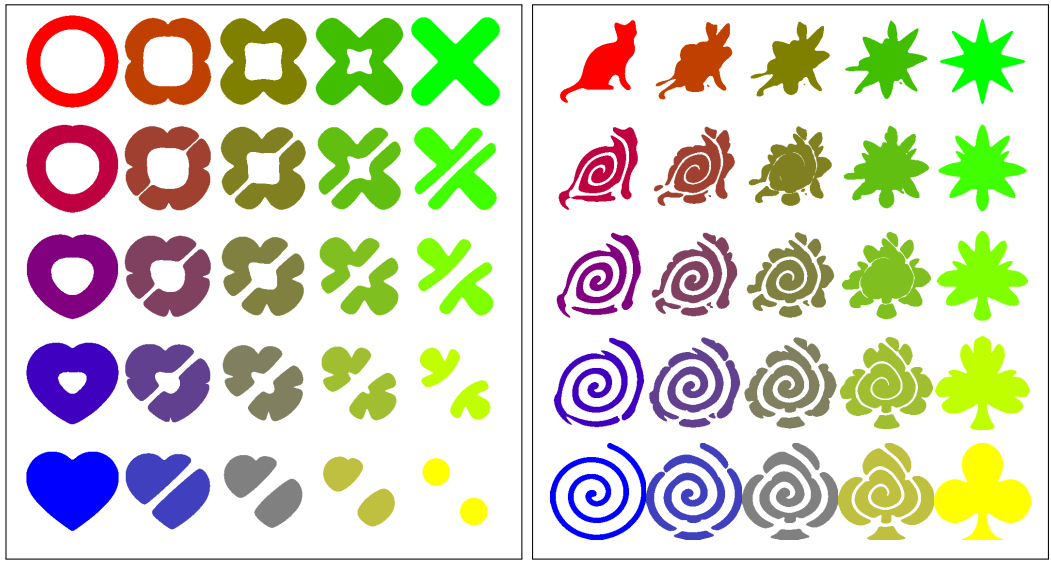

\centering
\fbox{
\begin{tabular}{@{}c@{}c@{}c@{}c@{}c@{}}
\FigBaryImLineA{1} \\
\FigBaryImLineA{2} \\
\FigBaryImLineA{3} \\
\FigBaryImLineA{4} \\
\FigBaryImLineA{5}
\end{tabular}
}
\fbox{
\begin{tabular}{@{}c@{}c@{}c@{}c@{}c@{}}
\FigBaryImLineB{1} \\
\FigBaryImLineB{2} \\
\FigBaryImLineB{3} \\
\FigBaryImLineB{4} \\
\FigBaryImLineB{5}
\end{tabular}
}
\caption{\label{fig-barycenters-images}
Barycenters between four input 2-D shapes using entropic regularization~\eqref{eq-entropic-bary}. To display a binary shape, the displayed images shows a thresholded density. 
The weights $(\la_s)_s$ are bilinear with respect to the four corners of the square.  
}
\end{figure}

\newcommand{\FigBaryShapeA}[2]{\includegraphics[width=.09\linewidth,trim=130 78 110 60,clip]{barycenters-shapes/duck-spiky-moomoo_s0-double-torus/barycenter-#1-#2}}
\newcommand{\FigBaryShapeLineA}[1]{\FigBaryShapeA{#1}{0} & \FigBaryShapeA{#1}{1} & \FigBaryShapeA{#1}{2} & \FigBaryShapeA{#1}{3} & \FigBaryShapeA{#1}{4}}
\newcommand{\FigBaryShapeB}[2]{\includegraphics[width=.087\linewidth,trim=160 65 150 55,clip]{barycenters-shapes/mushroom-torus-hand1-trim-star/barycenter-#1-#2}}
\newcommand{\FigBaryShapeLineB}[1]{\FigBaryShapeB{#1}{0} & \FigBaryShapeB{#1}{1} & \FigBaryShapeB{#1}{2} & \FigBaryShapeB{#1}{3} & \FigBaryShapeB{#1}{4}}

\begin{figure}[h!]
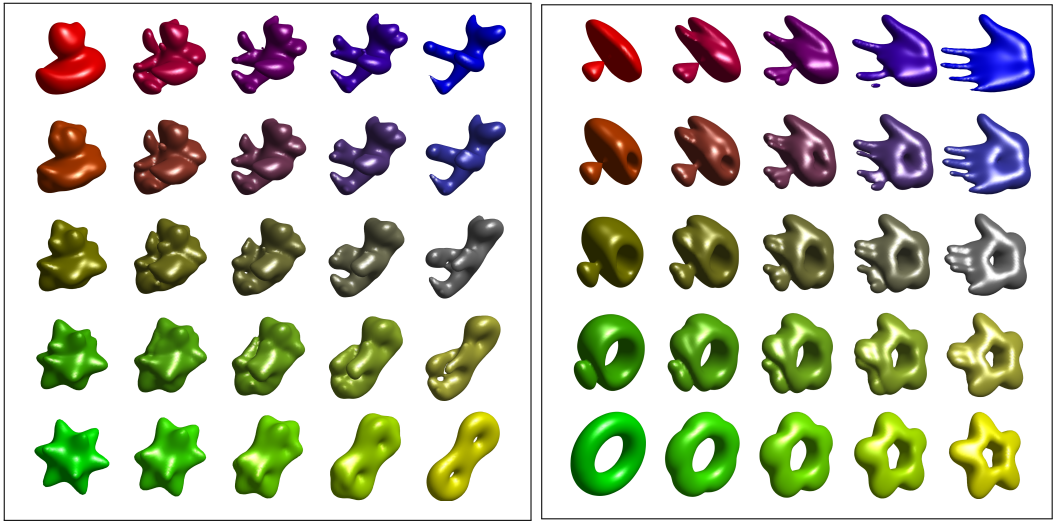

\centering
\fbox{
\begin{tabular}{@{}c@{}c@{}c@{}c@{}c@{}}
\FigBaryShapeLineA{0} \\
\FigBaryShapeLineA{1} \\
\FigBaryShapeLineA{2} \\
\FigBaryShapeLineA{3} \\
\FigBaryShapeLineA{4}
\end{tabular}
}
\fbox{
\begin{tabular}{@{}c@{}c@{}c@{}c@{}c@{}}
\FigBaryShapeLineB{0} \\
\FigBaryShapeLineB{1} \\
\FigBaryShapeLineB{2} \\
\FigBaryShapeLineB{3} \\
\FigBaryShapeLineB{4}
\end{tabular}
}
\caption{\label{fig-barycenters-shapes}
Barycenters between four input 3-D shapes using entropic regularization~\eqref{eq-entropic-bary}. The weights $(\la_s)_s$ are bilinear with respect to the four corners of the square.  
Shapes are represented as measures that are uniform within the boundaries of the shape and null outside.
}
\end{figure}
The efficient computation of Wasserstein barycenters remains at this time an active research topic~\citep{NIPS2017_6858,NIPS2018_8274}. Beyond their methodological interest, Wasserstein barycenters have found many applications outside the field of shape analysis. They have been used for image processing~\citep{rabin-ssvm-11}, in particular color modification~\citep{2015-solomon-siggraph} (see Figure~\ref{fig-colors}); Bayesian computations~\citep{srivastava2015scalable,srivastava2015wasp} to summarize measures; and nonlinear dimensionality reduction, to express an input measure as a Wasserstein barycenter of other known measures~\citep{2016-bonneel-barycoord}. All of these problems result in involved nonconvex objective functions which can be accurately optimized using automatic differentiation (see Remark~\ref{rem-auto-diff}). 
Problems closely related to the computation of barycenters include the computation of principal components analyses over the Wasserstein space (see, for instance,~\citep{SeguyCuturi,bigot2017geodesic}) and the statistical estimation of template models~\citep{boissard2015distribution}. The ability to compute barycenters enables more advanced clustering methods such as the $k$-means on the space of probability measures~\citep{del2016robust,ho2017multilevel}.

\begin{figure}[h!]
\centering
\begin{tabular}{@{}c@{}c@{}c@{}c@{}c@{}}
\includegraphics[width=.195\linewidth]{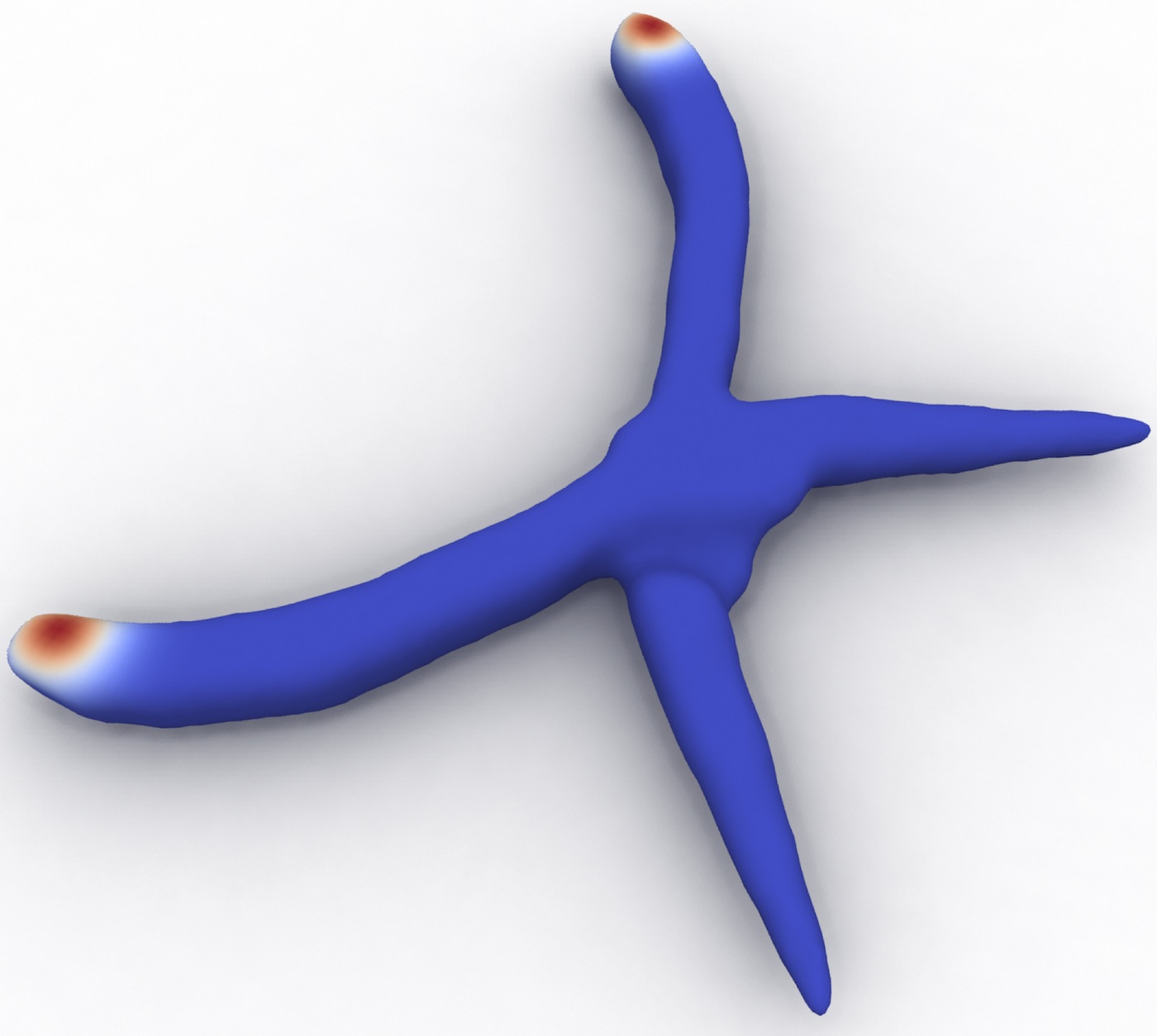}&
\includegraphics[width=.195\linewidth]{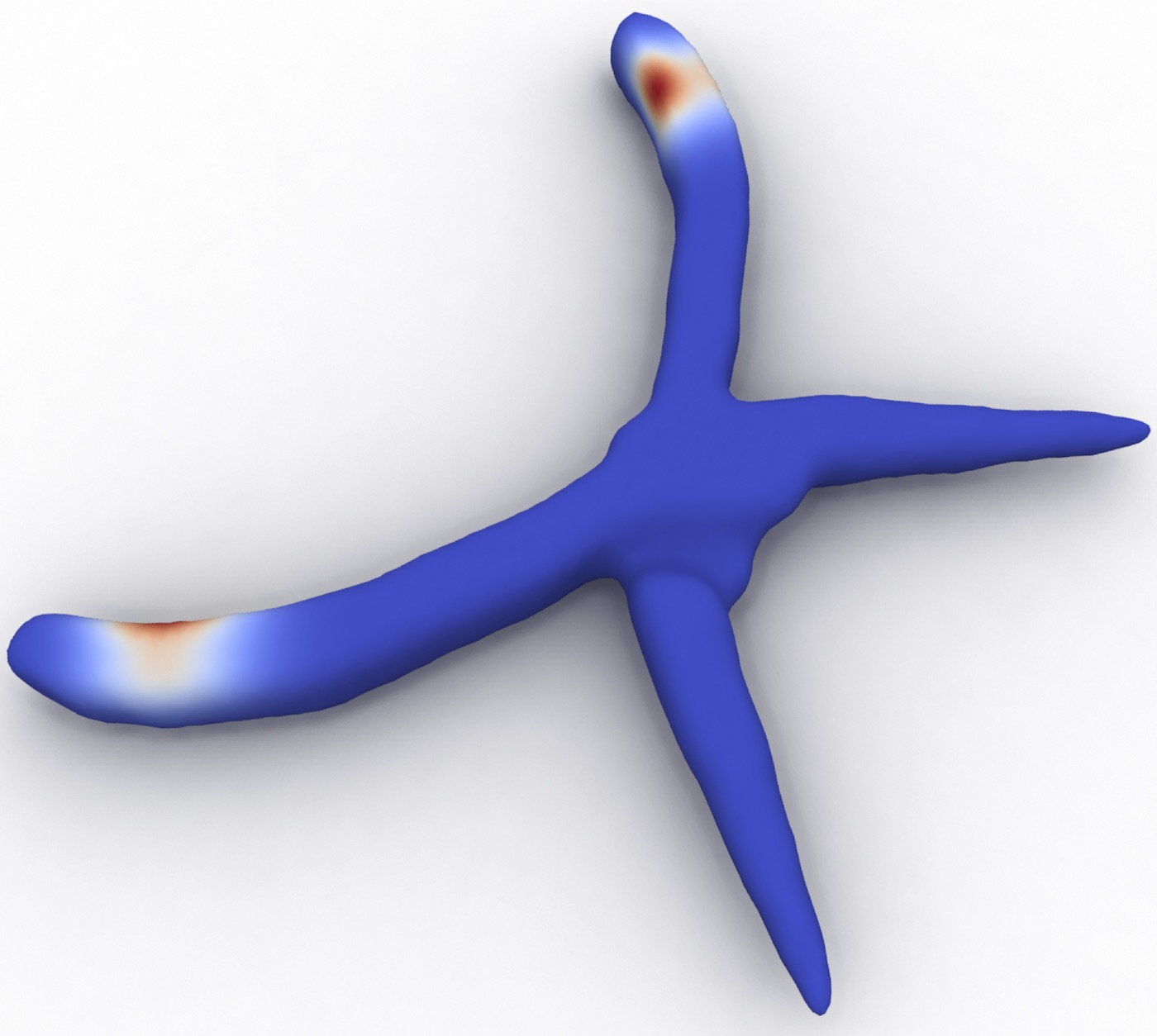}&
\includegraphics[width=.195\linewidth]{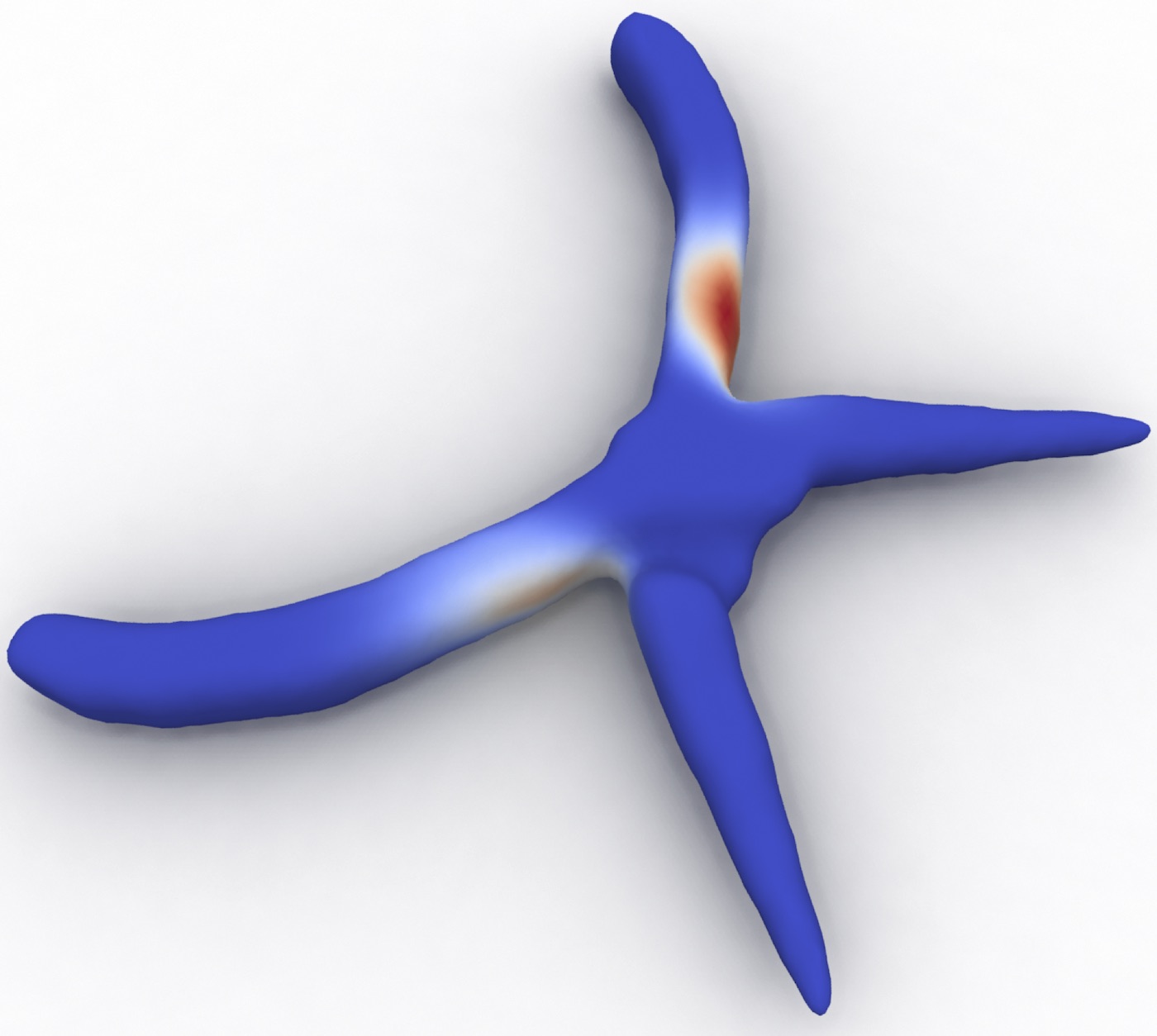}&
\includegraphics[width=.195\linewidth]{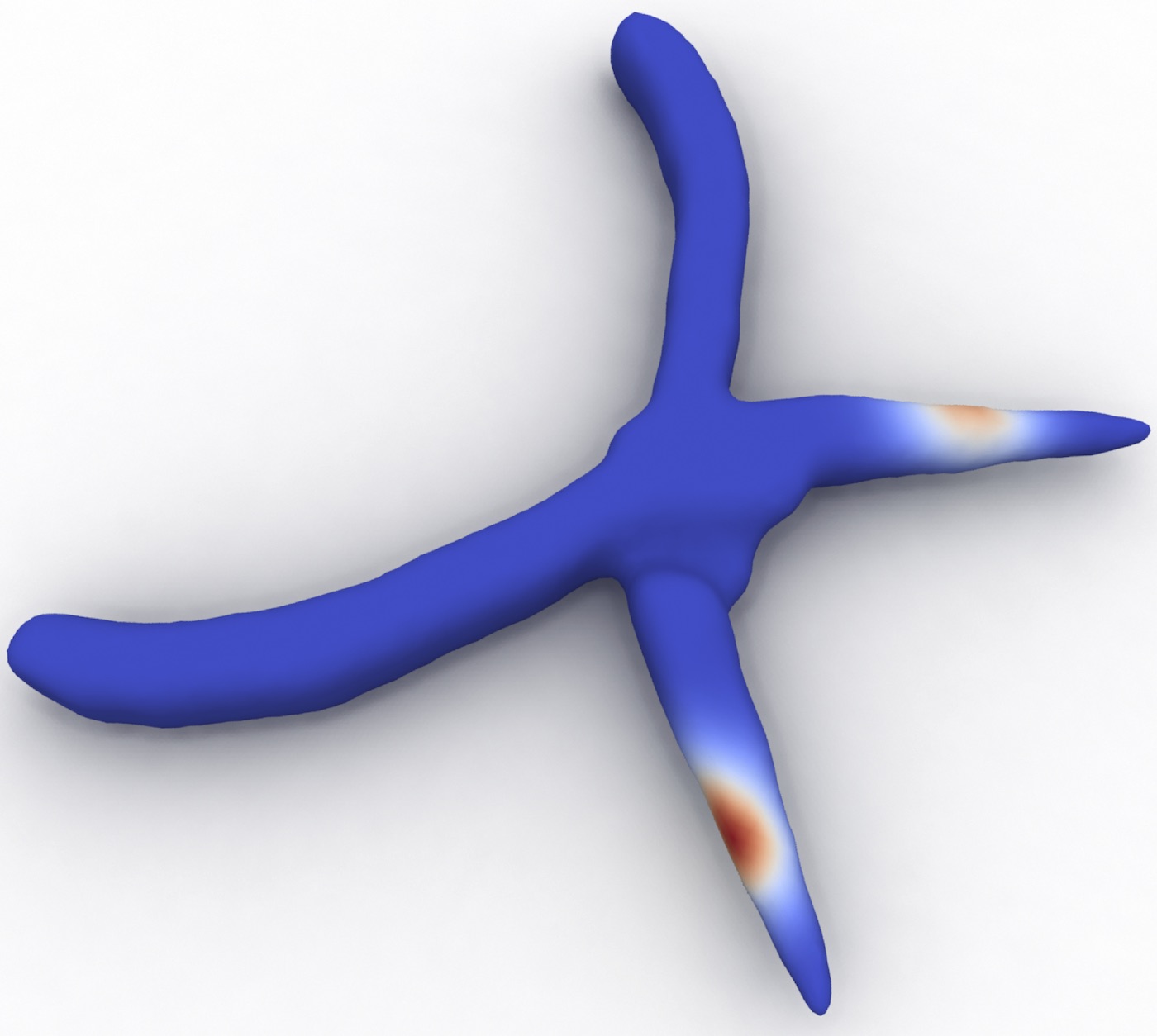}&
\includegraphics[width=.195\linewidth]{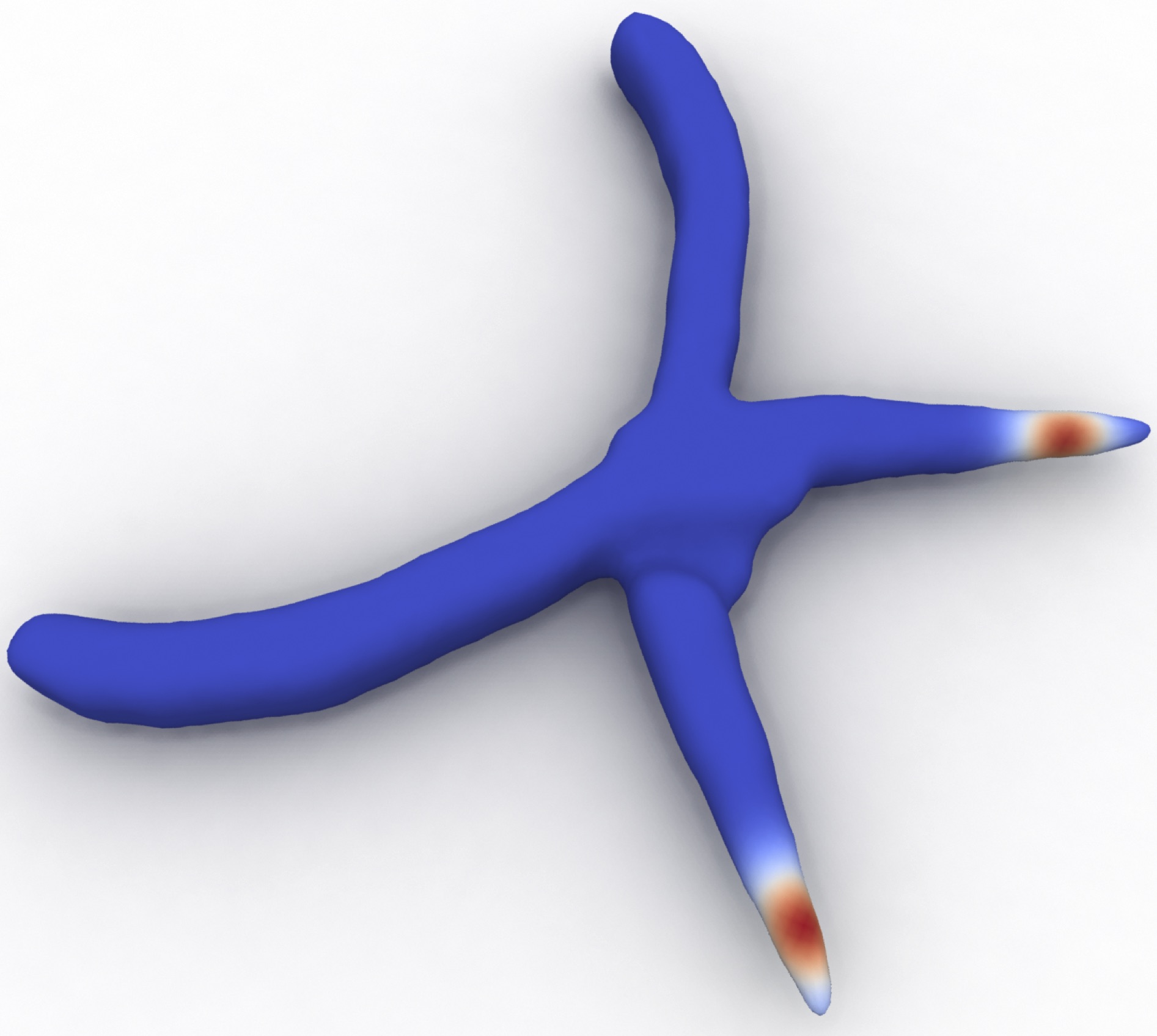}
\end{tabular}
\caption{\label{fig-barycenters-surfaces}
Barycenters interpolation between two input measures on surfaces, computed using the geodesic in heat fast kernel approximation (see Remark~\ref{rem-geod-heat}). Extracted from~\citep{2015-solomon-siggraph}.
}
\end{figure}

\begin{figure}[h!]
\centering
\includegraphics[width=1\linewidth]{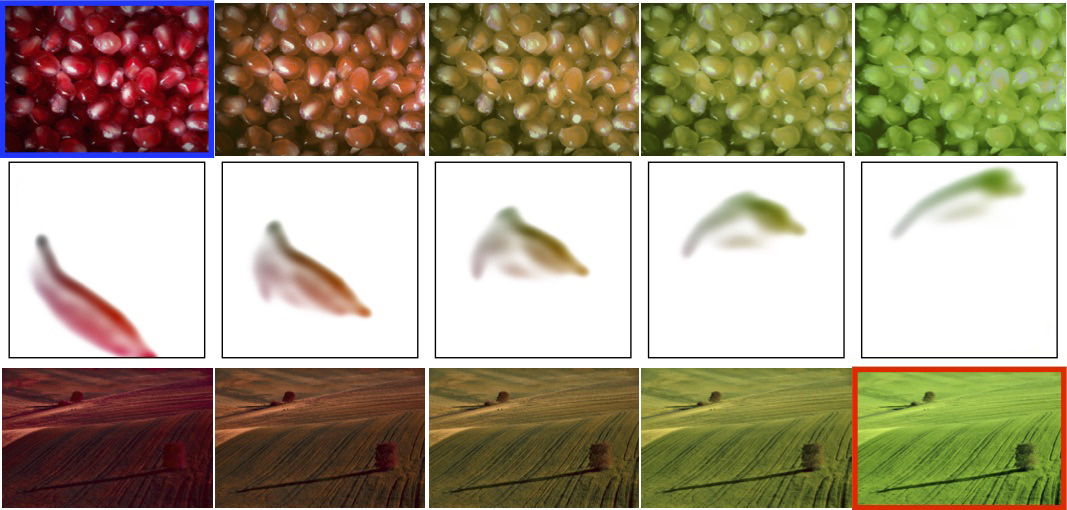}
\caption{\label{fig-colors}
Interpolation between the two 3-D color empirical histograms of two input images (here only the 2-D chromatic projection is visualized for simplicity). The modified histogram is then applied to the input images using barycentric projection as detailed in Remark~\ref{rem-barycenric-proj}. Extracted from~\citep{2015-solomon-siggraph}.
}
\end{figure}

\begin{rem1}{Wasserstein propagation}
	As studied in~\citet{Solomon-ICML}, it is possible to generalize the barycenter problem~\eqref{eq-wass-discr}, where one looks for distributions $(\b_u)_{u \in U}$ at some given set $U$ of nodes in a graph $\Gg$ given a set of fixed input distributions $(\b_v)_{v \in V}$ on the complementary set $V$ of the nodes. The unknown are determined by minimizing the overall transportation distance between all pairs of nodes $(r,s) \in \Gg$ forming edges in the graph
	\eql{\label{eq-w-propag}
		\umin{(\b_u  \in \simplex_{n_u})_{u \in U}} 
			\sum_{(r,s) \in \Gg}  \MKD_{\C_{r,s}}(\b_r,\b_s),
	}
	where the cost matrices $\C_{r,s} \in \RR^{n_r \times n_s}$ need to be specified by the user.
	The barycenter problem~\eqref{eq-wass-discr} is a special case of this problem where the considered graph $\Gg$ is ``star shaped,'' where $U$ is a single vertex connected to all the other vertices $V$ (the weight $\la_s$ associated to $\b_s$ can be absorbed in the cost matrix). 
	Introducing explicitly a coupling $\P_{r,s} \in \CouplingsD( \b_r,\b_s )$ for each edge $(r,s) \in \Gg$, and using entropy regularization, one can rewrite this problem similarly as in~\eqref{eq-bary-entropy-couplings}, and one extends Sinkhorn iterations~\eqref{eq-sinkhorn-bary} to this problem (this can also be derived by recasting this problem in the form of the generalized Sinkhorn algorithm detailed in~\S\ref{sec-generalized}). 
	\todoK{show a graphical illustration}
	This discrete variational problem~\eqref{eq-w-propag} on a graph can be generalized to define a Dirichlet energy when replacing the graph by a continuous domain~\citep{solomon2013dirichlet}. This in turn leads to the definition of measure-valued harmonic functions which finds application in image and surface  processing. We refer also to~\citet{Lavenant2017} for a theoretical analysis and to~\citet{vogt2017measure} for extensions to nonquadratic (total-variation) functionals and applications to imaging.
\end{rem1}

\section{Gradient Flows}
\label{sec-grad-flows}

Given a smooth function $\a \mapsto F(\a)$, one can use the standard gradient descent
\eql{\label{eq-explicit-euclidean}
	\itt{\a} \eqdef \it{\a} - \tau \nabla F(\it{\a}),
} 
where $\tau$ is a small enough step size. This corresponds to a so-called ``explicit'' minimization scheme and only applies for smooth functions $F$. For nonsmooth functions, one can use instead an ``implicit'' scheme, which is also called the proximal-point algorithm (see, for instance,~\citet{BauschkeCombettes11})
\eql{\label{eq-implicit-euclidean}
	\itt{\a} \eqdef \Prox_{\tau F}^{\norm{\cdot}}(\it{\a}) \eqdef \uargmin{\a} \frac{1}{2}\norm{\a-\it{\a}}^2 + \tau F(\a).
}
Note that this corresponds to the Euclidean proximal operator, already encountered in~\eqref{eq-prox-eucl}.
The update~\eqref{eq-explicit-euclidean} can be understood as iterating the explicit operator $\Id-\tau\nabla F$, while~\eqref{eq-implicit-euclidean} makes use of the implicit operator $(\Id+\tau\nabla F)^{-1}$. For convex $F$, iterations~\eqref{eq-implicit-euclidean} always converge, for any value of $\tau>0$. 

If the function $F$ is defined on the simplex of histograms $\simplex_n$, then it makes sense to use an optimal transport metric in place of the $\ell^2$ norm $\norm{\cdot}$ in~\eqref{eq-implicit-euclidean}, in order to solve
\eql{\label{eq-grad-flow-discr}
	\itt{\a} \eqdef \uargmin{\a} \WassD_p(\a,\it{\a})^p + \tau F(\a).
}

\begin{rem2}{Wasserstein gradient flows}
Equation~\eqref{eq-grad-flow-discr} can be generalized to arbitrary measures by defining the iteration
\eql{\label{eq-grad-flow-stepping-cont}
	\itt{\al} \eqdef \uargmin{\al} \Wass_p(\al,\it{\al})^p + \tau F(\al)
}
for some function $F$ defined on $\Mm_+^1(\X)$. This implicit time stepping is a useful tool to construct continuous flows, by formally taking the limit $\tau \rightarrow 0$ and introducing the time $t=\tau \ell$, so that $\it{\al}$ is intended to approximate a continuous flow $t \in \RR_+ \mapsto \al_t$. For the special case $p=2$ and $\X=\RR^d$, a formal calculus shows that $\al_t$ is expected to solve a PDE of the form
\eql{\label{eq-grad-flow-cont}
	\pd{\al_t}{t} = \diverg( \al_t \nabla( F'(\al_t) ) ),
}
where $F'(\al)$ denotes the derivative of the function $F$ in the sense that it is a continuous function $F'(\al) \in \Cc(\X)$ such that
\eq{
	F(\al+\epsilon\xi) = F(\al) + \epsilon \int_\X F'(\al) \d\xi(x) + o(\epsilon). 
}
A typical example is when using $F=-\H$, where $\H(\al)=\KL(\al|\Ll_{\RR^d})$ is the relative entropy with respect to the Lebesgue measure $\Ll_{\RR^d}$ on $\X=\RR^\dim$
\eql{\label{eq-entropy-cont}
	\H(\al) = - \int_{\RR^\dim} \density{\al}(x) (\log(\density{\al}(x)) - 1) \d x
}
(setting $\H(\al)=-\infty$ when $\al$ does not have a density), then~\eqref{eq-grad-flow-cont} shows that the gradient flow of this neg-entropy is the linear heat diffusion
\eql{\label{eq-heat}
	\pd{\al_t}{t} = \Delta \al_t,
}
where $\Delta$ is the spatial Laplacian.
The heat diffusion can therefore be interpreted either as the ``classical'' Euclidian flow (somehow performing ``vertical'' movements with respect to mass amplitudes) of the Dirichlet energy $\int_{\RR^\dim} \|\nabla \density{\al}(x)\|^2 \d x$ or, alternatively, as the entropy for the optimal transport flow (somehow a ``horizontal'' movement with respect to mass positions).
Interest in Wasserstein gradient flows was sparked by the seminal paper of Jordan, Kinderlehrer and Otto~\citep{jordan1998variational}, and these evolutions are often called ``JKO flows'' following their work. As shown in detail in the monograph by~\citet{ambrosio2006gradient}, JKO flows are a special case of gradient flows in metric spaces. We also refer to the recent survey paper~\citep{santambrogio2017euclidean}.
JKO flows can be used to study in particular nonlinear evolution equations such as the porous medium equation~\citep{otto2001geometry}, total variation flows~\citep{carlier2017total}, quantum drifts~\citep{GianazzaARMA}, or heat evolutions on manifolds~\citep{ErbarHeatManifold}. Their flexible formalism allows for constraints on the solution, such as the congestion constraint (an upper bound on the density at any point) that~\citeauthor{maury2010macroscopic} used to model crowd motion~\citep{maury2010macroscopic} (see also the review paper~\citep{SantambrogioCrowdReview}).
\end{rem2}

\begin{rem2}{Gradient flows in metric spaces}
The implicit stepping~\eqref{eq-grad-flow-stepping-cont} is a special case of a more general formalism to define gradient flows over metric spaces $(\Xx,\dist)$, where $\dist$ is a distance, as detailed in~\citep{ambrosio2006gradient}. 
For some function $F(x)$ defined for $x \in \Xx$, the implicit discrete minmization step is then defined as
\eql{\label{eq-implicit-metricflow}
	\itt{x} \in \uargmin{x \in \Xx} \dist(\it{x},x)^2 + \tau F(x). 
}
The JKO step~\eqref{eq-grad-flow-stepping-cont} corresponds to the use of the Wasserstein distance on the space of probability distributions. In some cases, one can show that~\eqref{eq-implicit-metricflow} admits a continuous flow limit $x_t$ as $\tau \rightarrow 0$ and $k\tau=t$. 
In the case that $\Xx$ also has a Euclidean structure, an explicit stepping is defined by linearizing $F$
\eql{\label{eq-explicit-metricflow}
	\itt{x} = \uargmin{x \in \Xx} \dist(\it{x},x)^2 + \tau \dotp{ \nabla F(\it{x}) }{ x }.
}
In sharp contrast to the implicit formula~\eqref{eq-implicit-metricflow} it is usually straightforward to compute but can be unstable. The implicit step is always stable, is also defined for nonsmooth $F$, but is usually not accessible in closed form.  
Figure~\ref{fig-gradflow-metric} illustrates this concept on the function $F(x) = \norm{x}^2$ on $\Xx=\RR^2$ for the distances $d(x,y)=\norm{x-y}_p=(|x_1-y_1|^p+|x_2-y_2|^p)^{\frac{1}{p}}$ for several values of $p$. 
The explicit scheme~\eqref{eq-explicit-metricflow} is unstable for $p=1$ and $p=+\infty$, and for $p=1$ it gives axis-aligned steps (coordinatewise descent). 
In contrast, the implicit scheme~\eqref{eq-implicit-metricflow} is stable. Note in particular how, for $p=1$, when the two coordinates are equal, the following step operates in the diagonal direction. 
\end{rem2}

\begin{figure}[h!]
\centering
\begin{tabular}{cc}
\includegraphics[width=.45\linewidth]{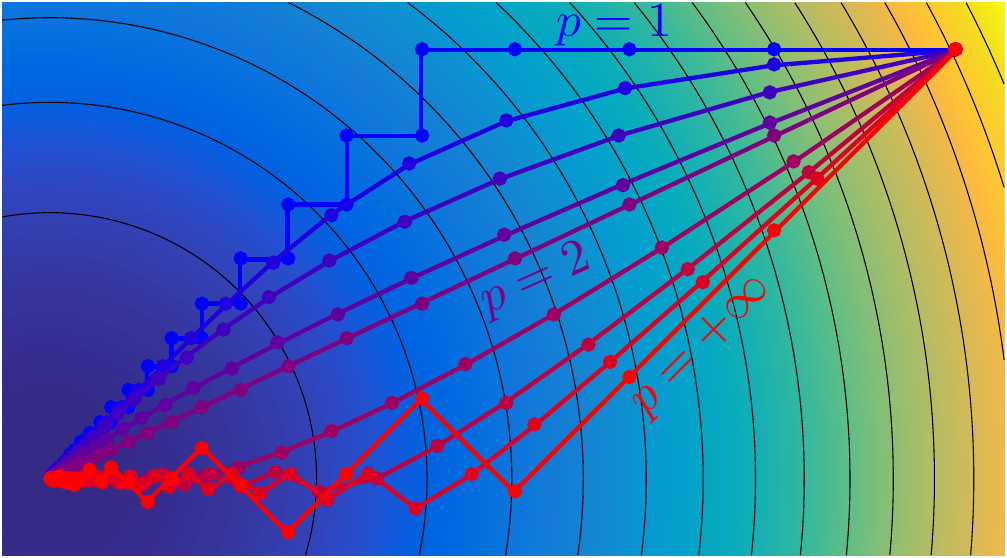}&
\includegraphics[width=.45\linewidth]{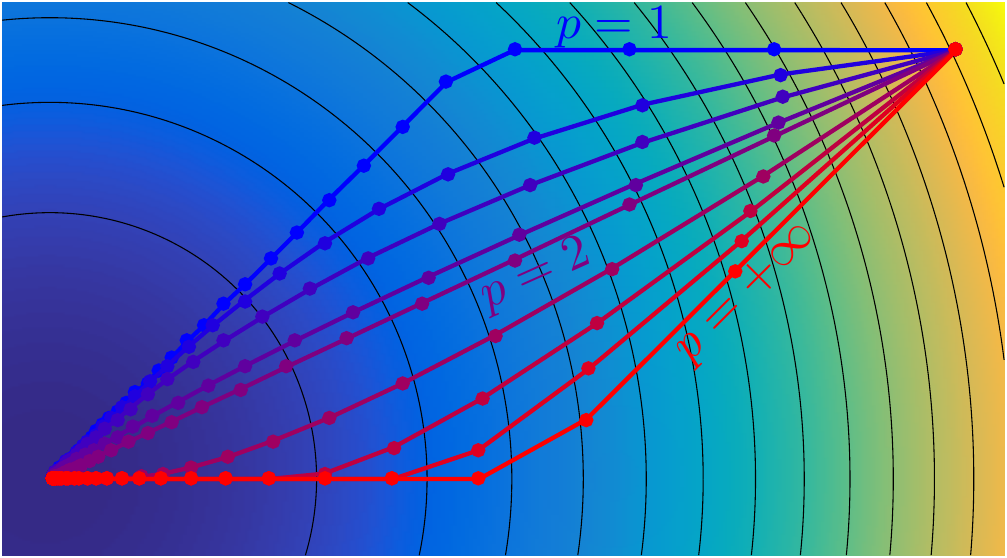}\\
Explicit & Implicit
\end{tabular}
\caption{\label{fig-gradflow-metric}
Comparison of explicit and implicit gradient flow to minimize the function $f(x) = \norm{x}^2$ on $\Xx=\RR^2$ for the distances $d(x,y)=\norm{x-y}_p$ for several values of $p$. 
}
\end{figure}

\begin{rem1}{Lagrangian discretization using particles systems}\label{rem-lagr-discretization}
The finite-dimensional problem in~\eqref{eq-grad-flow-discr} can be interpreted as the Eulerian discretization of a flow over the space of measures~\eqref{eq-grad-flow-stepping-cont}.
An alternative way to discretize the problem, using the so-called Lagrangian method using particles systems, is to parameterize instead the solution as a (discrete) empirical measure moving with time, where the locations of that measure (and not its weights) become the variables of interest. In practice, one can consider a dynamic point cloud of particles $\al_t = \frac{1}{n}\sum_{i=1}^n \de_{x_i(t)}$ indexed with time.
The initial problem~\eqref{eq-grad-flow-discr} is then replaced by a set of $n$ coupled ODE prescribing the dynamic of the points $X(t)=(x_i(t))_i \in \X^n$. 
If the energy $F$ is finite for discrete measures, then one can simply define $\Ff(X)=F( \frac{1}{n}\sum_{i=1}^n \de_{x_i} )$. Typical examples are linear functions $F(\al) = \int_\X V(x) \d\al(x)$ and quadratic interactions $F(\al) = \int_{\X^2} W(x,y) \d\al(x)\d\al(y)$, in which case one can use respectively 
\eq{
	\Ff(X) = \frac{1}{n} \sum_i V(x_i) \qandq \Ff(X) = \frac{1}{n^2} \sum_{i,j} W(x_i,x_j). 
}
For functions such as generalized entropy, which are only finite for measures having densities, one should apply a density estimator to convert the point cloud into a density, which allows us to also define function $\Ff(x)$ consistent with $F$ as $n \rightarrow +\infty$. A typical example is for the entropy $F(\mu) = \H(\al)$ defined in~\eqref{eq-entropy-cont}, for which a consistent estimator (up to a constant term) can be obtained by summing the logarithms of the distances to nearest neighbors
\eql{\label{eq-entropy-discretized}
	\Ff(X) = \frac{1}{n}\sum_{i} \log( d_X(x_i) )
	\qwhereq d_X(x) = \umin{x' \in X, x' \neq x} \norm{x-x'};
}
see~\citet{beirlant1997nonparametric} for a review of nonparametric entropy estimators.
For small enough step sizes $\tau$, assuming $\X=\RR^\dim$, the Wasserstein distance $\Wass_2$ matches the Euclidean distance on the points, \ie if $|t-t'|$ is small enough, $\Ww_2(\al_t,\al_{t'})=\norm{X(t)-X(t')}$. The gradient flow is thus equivalent to the Euclidean flow on positions $X'(t) = -\nabla \Ff(X(t))$, which is discretized for times $t_k=\tau k$ similarly to~\eqref{eq-explicit-euclidean} using explicit Euler steps
\eq{
	\itt{X} \eqdef \it{X} - \tau \nabla \Ff(\it{X}).
} 
Figure~\ref{fig-flow-lagr} shows an example of such a discretized explicit evolution for a linear plus entropy functional, resulting in a discretized version of a Fokker--Planck equation.   
Note that for this particular case of linear Fokker--Planck equation, it is possible also to resort to stochastic PDEs methods, and it can be approximated numerically by evolving a single random particle with a Gaussian drift. The convergence of these schemes (so-called Langevin Monte Carlo) to the stationary distribution can in turn be quantified in terms of Wasserstein distance; see, for instance,~\citep{dalalyan2017user}. \todoK{Detailed this much more.}
If the function $\Ff$ is not smooth, one should discretize similarly to~\eqref{eq-implicit-euclidean} using implicit Euler steps, \ie consider 
\eq{
	\itt{X} \eqdef \Prox_{\tau \Ff}^{\norm{\cdot}}(\it{X}) \eqdef \uargmin{Z \in \X^n} \frac{1}{2}\norm{Z-\it{X}}^2 + \tau \Ff(Z).
}
In the simplest case of a linear function $F(\al) = \int_X V(x) \d\al(x)$, the flow operates independently over each particule $x_i(t)$ and corresponds to a usual Euclidean flow for the function $V$, $x_i'(t)=-\nabla V(x_i(t))$ (and is an advection PDEs of the density along the integral curves of the flow).
\end{rem1}

\newcommand{\FigLagrFlow}[1]{\includegraphics[width=.19\linewidth]{lagrangian-flows/lagrange-flow-#1}}

\begin{figure}[h!]
\centering
\begin{tabular}{@{}c@{}c@{}c@{}c@{}c@{}}
\FigLagrFlow{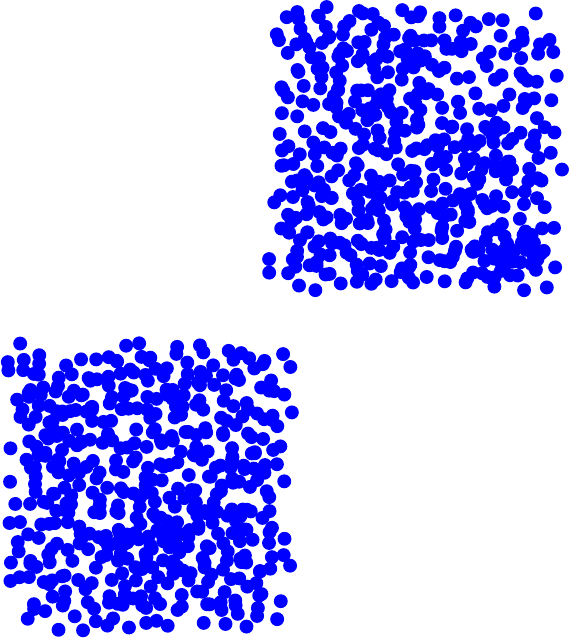} & 
\FigLagrFlow{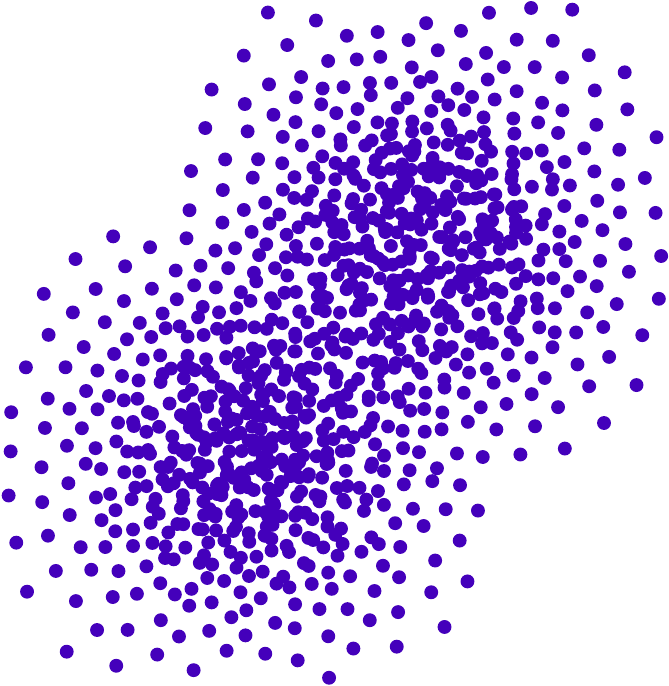} & 
\FigLagrFlow{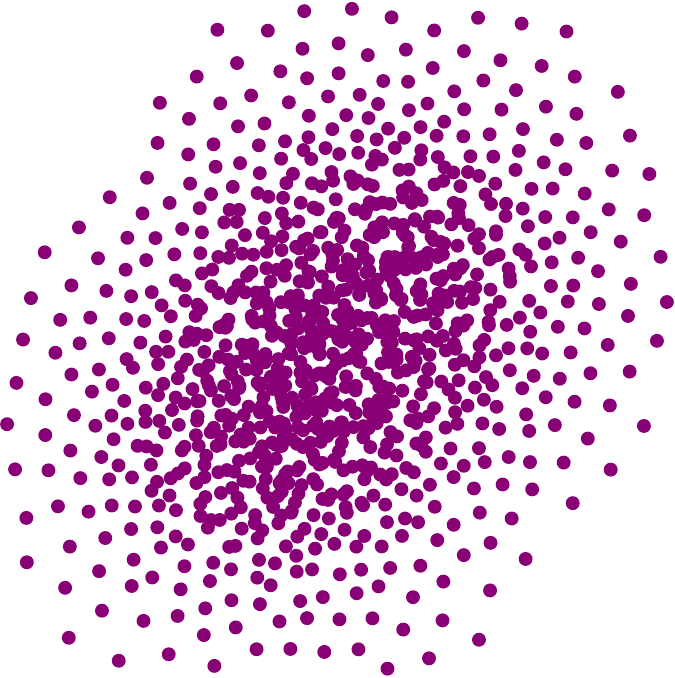} & 
\FigLagrFlow{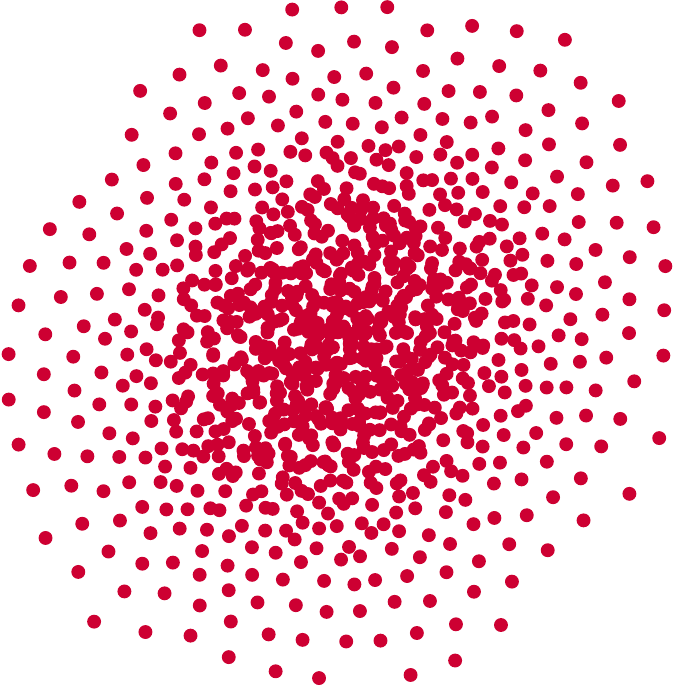} & 
\FigLagrFlow{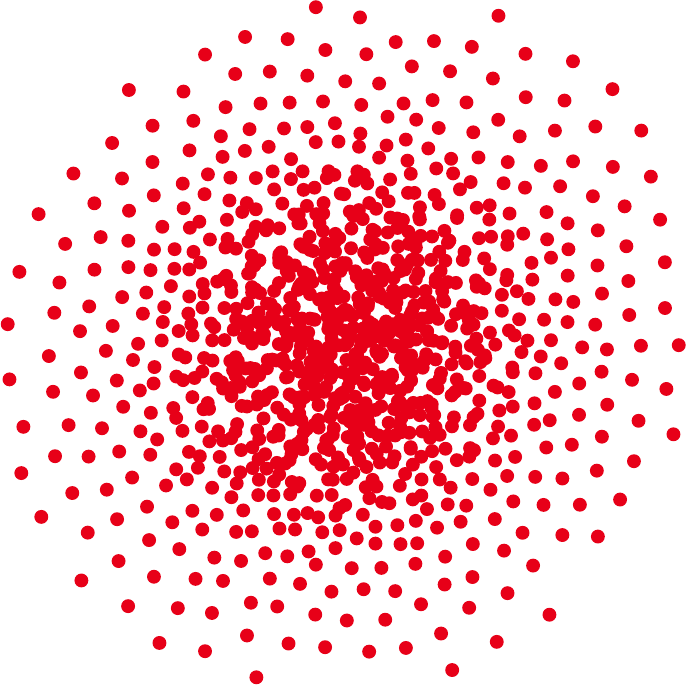} \\
$t=0$ & $t=0.2$  & $t=0.4$ & $t=0.6$ & $t=0.8$
\end{tabular}
\caption{\label{fig-flow-lagr}
Example of gradient flow evolutions using a Lagrangian discretization, for the function $F(\al) = \int V \d\al - H(\al)$, for $V(x)=\norm{x}^2$.
The entropy is discretized using~\eqref{eq-entropy-discretized}.
The limiting stationary distribution is a Gaussian. 
}
\end{figure}

\begin{rem2}{Geodesic convexity}
An important concept related to gradient flows is the convexity of the functional $F$ with respect to the Wasserstein-2 geometry, \ie the convexity of $F$ along Wasserstein geodesics (\ie displacement interpolations as shown in Remark~\ref{rem-displacement}). The Wasserstein gradient flow (with a continuous time) for such a function exists, is unique, and is the limit of the discrete stepping~\eqref{eq-grad-flow-stepping-cont} as $\tau \rightarrow 0$. It converges to a fixed stationary distribution as $t \rightarrow +\infty$.    
The entropy is a typical example of geodesically convex function, and so are linear functions of the form $F(\al) = \int_\X V(x) \d\al(x)$ and quadratic interaction functions $F(\al) = \int_{\X \times \X} W(x,y) \d\al(x) \d\al(y)$ for convex functions $V : \X \rightarrow \RR$, $W : \X \times \X \rightarrow \RR$. 
Note that while linear functions are convex in the classical sense, quadratic interaction functions might fail to be.
A typical example is $W(x,y)=\norm{x-y}^2$, which is a negative semi-definite kernel (see Definition~\ref{def-negativedefinitekernel}) and thus corresponds to $F(\al)$ being a concave function in the usual sense (while it is geodesically convex).  
An important result of~\citet{mccann1997convexity} is that generalized ``entropy'' functions of the form $F(\al) = \int_{\RR^\dim} \phi(\density{\al}(x)) \d x$ on $\X=\RR^\dim$ are geodesically convex if $\phi$ is convex, with $\phi(0)=0$, $\phi(t)/t \rightarrow +\infty$ as $t \rightarrow +\infty$ and such that  $s \mapsto s^{\dim} \phi(s^{-\dim})$ is convex decaying.
\end{rem2}

There is important literature on the numerical resolution of the resulting discretized flow, and we give only a few representative publications. For 1-D problems, very precise solvers have been developed because OT is a quadratic functional in the inverse cumulative function (see Remark~\ref{rem-1d-ot-generic}):~\citet{kinderlehrer1999approximation,blanchet2008convergence,agueh2013one,Matthes1D,blanchet2012optimal}. In higher dimensions, it can be tackled using finite elements and finite volume schemes:~\citet{CarrilloFiniteVolume,burger2010mixed}. Alternative solvers are obtained using Lagrangian schemes (\ie particles systems):~\citet{carrillo2009numerical,JDB-JKO,Westdickenberg2010}.
Another direction is to look for discrete flows (typically on discrete grids or graphs) which maintain some properties of their continuous counterparts; see~\citet{MielkeCVPDE,ErbarDCDS,ChowHuangLiZhou2012,Maas2011}.

An approximate approach to solve the Eulerian discretized problem~\eqref{eq-explicit-euclidean} relying on entropic regularization was initially proposed in~\citet{2015-Peyre-siims}, refined in~\citet{2016-chizat-sinkhorn} and theoretically analyzed in~\citet{2017-carlier-SIMA}.
With an entropic regularization, Problem~\eqref{eq-grad-flow-discr} has the form~\eqref{eq-generalized-ot-regul} when setting $G = \iota_{\it{\a}}$ and replacing $F$ by $\tau F$. One can thus use the iterations~\eqref{eq-gen-sinkh} to approximate $\itt{\a}$ as proposed initially in~\citet{2015-Peyre-siims}. The convergence of this scheme as $\epsilon \rightarrow 0$ is proved in~\citet{2017-carlier-SIMA}.
Figure~\ref{fig-jko} shows an example of evolution computed with this method.
An interesting application of gradient flows to machine learning is to learn the underlying function $F$ that best models some dynamical model of density. This learning can be achieved by solving a smooth nonconvex optimization using entropic regularized transport and automatic differentiation (see Remark~\ref{rem-auto-diff}); see~\citet{hashimoto2016learning}. 

Analyzing the convergence of gradient flows discretized in both time and space is difficult in general. Due to the polyhedral nature of the linear program defining the distance, using too-small step sizes leads to a ``locking'' phenomena (the distribution is stuck and does not evolve, so that the step size should be not too small, as discussed in~\citep{maury201713}). We refer to~\citep{Matthes1D,matthes2017convergent} for a convergence analysis of a discretization method for gradient flows in one dimension.

\begin{figure}[h!]
\centering
\begin{tabular}{@{}c@{\hspace{1mm}}c@{\hspace{1mm}}c@{\hspace{1mm}}c@{\hspace{1mm}}c@{}}
\includegraphics[width=.19\linewidth]{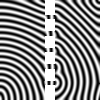}&
\includegraphics[width=.19\linewidth]{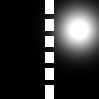}&
\includegraphics[width=.19\linewidth]{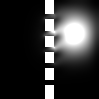}&
\includegraphics[width=.19\linewidth]{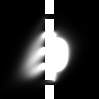}&
\includegraphics[width=.19\linewidth]{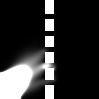}\\
\includegraphics[width=.19\linewidth]{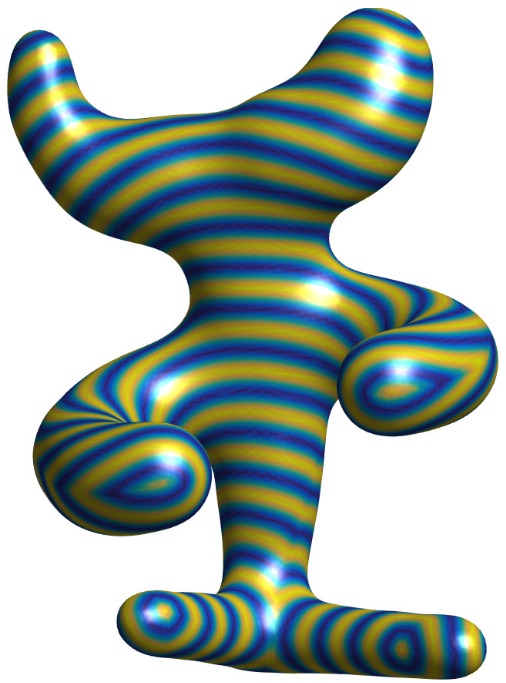}&
\includegraphics[width=.19\linewidth]{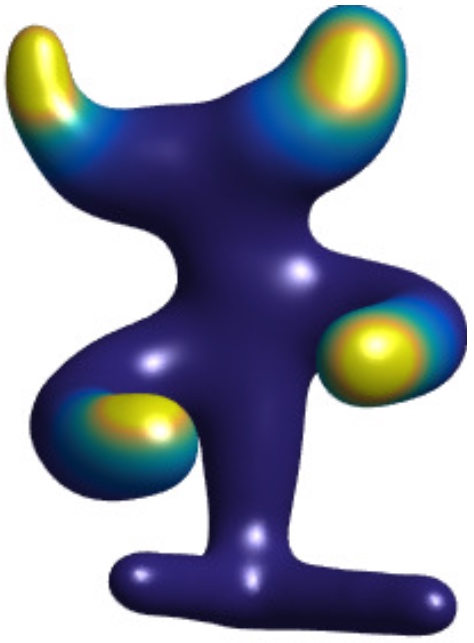}&
\includegraphics[width=.19\linewidth]{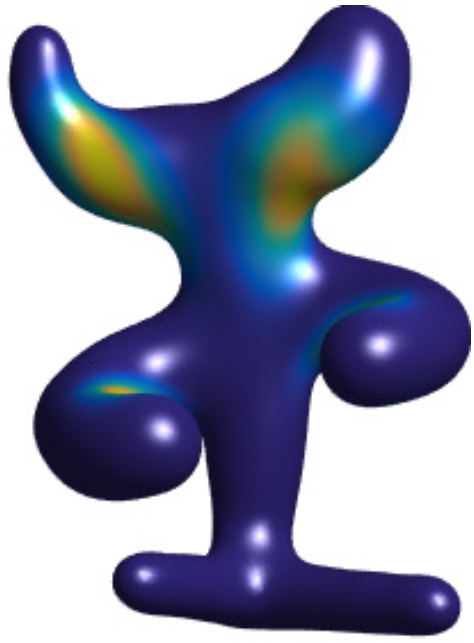}&
\includegraphics[width=.19\linewidth]{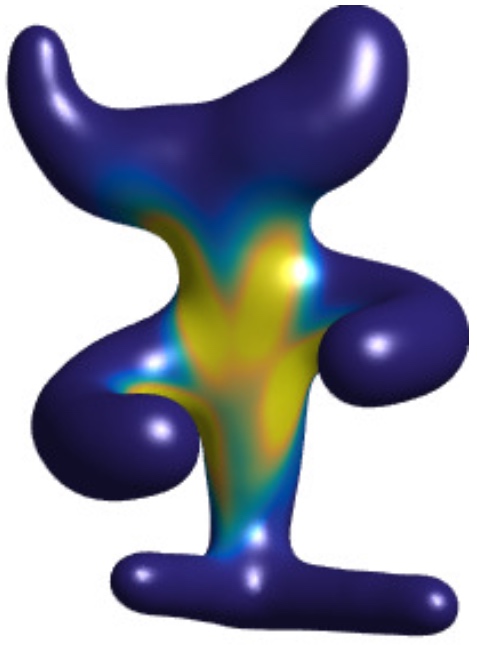}&
\includegraphics[width=.19\linewidth]{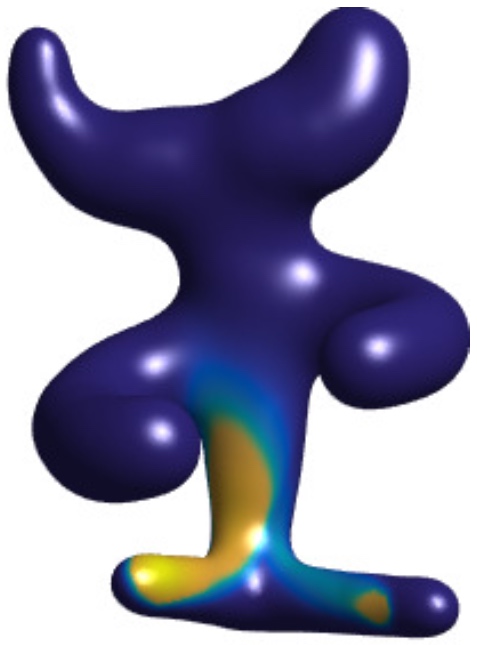}\\
$\cos(w)$ & $t=0$ & $t=5$ & $t=10$ & $t=20$
\end{tabular}
\caption{\label{fig-jko}
Examples of gradient flows evolutions, with drift $V$ and congestion terms (from~\citet{2015-Peyre-siims}), so that $F(\al) = \int_\X V(x) \d \al(x) + \iota_{\leq \kappa}(\density{\al})$.
}
\end{figure}

It is also possible to compute gradient flows for unbalanced optimal transport distances as detailed in~\S\ref{sec-unbalanced}. This results in evolutions allowing mass creation or destruction, which is crucial to model many physical, biological or chemical phenomena. 
An example of unbalanced gradient flow is the celebrated Hele-Shaw model for cell growth~\citep{PerthameTumor}, which is studied theoretically in~\citep{gallouet2017jko,marino2017tumor}. Such an unbalanced gradient flow also can be approximated using the generalized Sinkhorn algorithm~\citep{2016-chizat-sinkhorn}.


\section{Minimum Kantorovich Estimators}
\label{sec-mke}

Given some discrete samples $(x_i)_{i=1}^n \subset \Xx$ from some unknown distribution, the goal is to fit a parametric model $\th \mapsto \al_\th \in \Mm(\Xx)$ to the observed empirical input measure $\be$
\eql{\label{eq-density-fitting}
	\umin{\th \in \Theta} \Ll(\al_\th,\be)
	\qwhereq
	\be = \frac{1}{n} \sum_i \de_{x_i}, 
}
where $\Ll$ is some ``loss'' function between a discrete and a ``continuous'' (arbitrary) distribution (see Figure~\ref{fig-density-fitting}). 

In the case where $\al_\th$ as a density $\density{\th} \eqdef \density{\al_\th}$ with respect to the Lebesgue measure (or any other fixed reference measure), the maximum likelihood estimator (MLE) is obtained by solving
\eq{
	\umin{\th} \Ll_{\text{MLE}}(\al_\th,\be) \eqdef -\sum_i \log(\density{\th}(x_i)). 
}
This corresponds to using an empirical counterpart of a Kullback--Leibler loss since, assuming the $x_i$ are i.i.d. samples of some $\bar\be$, then 
\eq{
	\Ll_{\text{MLE}}(\al,\be) \overset{n \rightarrow +\infty}{\longrightarrow} \KL(\al|\bar\be).
}

\begin{figure}[h!]
\centering
\includegraphics[width=.4\linewidth]{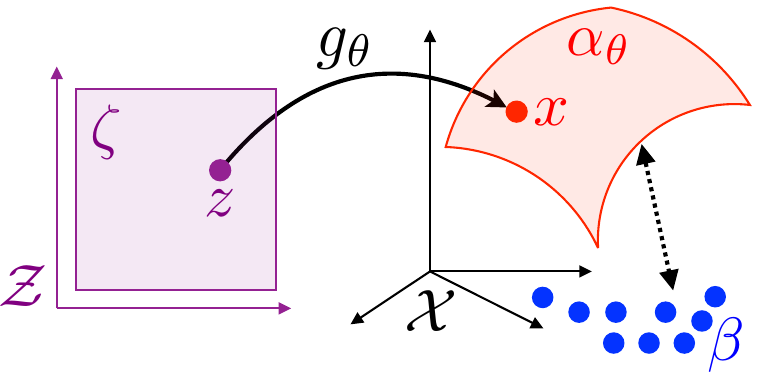}
\caption{\label{fig-density-fitting}
Schematic display of the density fitting problem~\ref{eq-density-fitting}.
}
\end{figure}

\newcommand{\fPF}{h}

This MLE approach is known to lead to optimal estimation procedures in many cases (see, for instance,~\citet{owen2001empirical}). However, it fails to work when estimating singular distributions, typically when the $\al_\th$ does not have a density (so that $\Ll_{\text{MLE}}(\al_\th,\be) = +\infty$) or when $(x_i)_i$ are samples from some singular $\bar\be$ (so that the $\al_\th$ should share the same support as $\be$ for $\KL(\al_\th|\bar\be)$ to be finite, but this support is usually unknown). Another issue is that in several cases of practical interest, the density $\density{\th}$ is inaccessible (or too hard to compute).

A typical setup where both problems (singular and unknown densities) occur is for so-called generative models, where the parametric measure is written as a push-forward of a fixed reference measure $\zeta \in \Mm(\Zz)$
\eq{
	\al_\th = \fPF_{\th,\sharp} \zeta \qwhereq \fPF_\th : \Zz \rightarrow \Xx,
}
where the push-forward operator is introduced in Definition~\ref{defn-pushfwd}. The space $\Zz$ is usually low-dimensional, so that the support of $\al_\th$ is localized along a low-dimensional ``manifold'' and the resulting density is highly singular (it does not have a density with respect to Lebesgue measure).
Furthermore, computing this density is usually intractable, while generating i.i.d. samples from $\al_\th$ is achieved by computing $x_i=\fPF_\th(z_i)$, where $(z_i)_i$ are i.i.d. samples from $\zeta$.

In order to cope with such a difficult scenario, one has to use weak metrics in place of the MLE functional $\Ll_{\text{MLE}}$, which needs to be written in dual form as 
\eql{\label{eq-dual-loss}
	\Ll(\al,\be) \eqdef 
	\umax{(\f,\g) \in \Cc(\X)^2} 
	\enscond{ \int_{\X} \f(x) \d\al(x) + \int_{\X} \g(x) \d\be(x) }{ (\f,\g) \in \Potentials}.
}
Dual norms shown in~\S\ref{sec-dual-norms} correspond to imposing 
\eq{
	\Potentials = \enscond{(\f,-\f)}{\f \in B}, 
}
while optimal transport~\eqref{eq-dual-generic} sets $\Potentials = \Potentials(\c)$ as defined in~\eqref{eq-dfn-pot-dual}. 

For a fixed $\th$, evaluating the energy to be minimized in~\eqref{eq-density-fitting} using such a loss function corresponds to solving a semidiscrete optimal transport, which is the focus of Chapter~\ref{c-algo-semidiscr}. Minimizing the energy with respect to $\th$ is much more involved and is typically highly nonconvex.

Denoting $\f_\th$ a solution to~\eqref{eq-dual-loss} when evaluating $\Ee(\th) = \Ll(\al_\th,\be)$, a subgradient is obtained using the formula
\eql{\label{eq-grad-wass-loss-dual}
	\nabla \Ee(\th) = \int_\Xx [\partial \fPF_\th(x)]^\top \nabla \f_\th(x) \d\al_\th(x), 
}
where $\partial \fPF_\th(x) \in \RR^{\text{dim}(\Theta) \times \dim}$ is the differential (with respect to $\th$) of $\th \in \RR^{\text{dim}(\Theta)} \mapsto \fPF_\th(x)$, while $\nabla \f_\th(x)$ is the gradient (with respect to $x$) of $\f_\th$. 
This formula is hard to use numerically, first because it requires first computing a continuous function $\f_\th$, which is a solution to a semi-discrete problem. As shown in~\S\ref{sec-entropy-ot-mmd}, for OT loss, this can be achieved using stochastic optimization, but this is hardly applicable in high dimension. Another option is to impose a parametric form for this potential, for instance expansion in an RKHS~(\citet{genevay2016stochastic}) or a deep-network approximation~(\citep{WassersteinGAN}). This, however, leads to important approximation errors that are not yet analyzed theoretically.
A last issue is that it is unstable numerically because it requires the computation of the gradient $\nabla \f_\th$ of the dual potential $\f_\th$.

For the OT loss, an alternative gradient formula is obtained when one rather computes a primal optimal coupling  for the following equivalent problem:
\eql{\label{eq-ot-couplings-fitting}
	\MK_\c(\al_\th,\be) = \umin{\ga \in \Mm(\Zz \times \Xx)} 
		\enscond{
			\int_{\Zz \times \Xx} c(\fPF_\th(z),x) \d\ga(z,x)
		}{
			\gamma \in \Couplings(\zeta,\be)
		}.
}
Note that in the semidiscrete case considered here, the objective to be minimized can be actually decomposed as 
\eql{\label{eq-ot-couplings-fitting-bis}
	\umin{(\ga_i)_{i=1}^n} \sum_{i=1}^n \int_\Zz c(\fPF_\th(z),x_i) \d \ga_i(z)
	\qwhereq
	\sum_{i=1}^n \ga_i = \zeta, \quad
	\int_\Zz \d\ga_i(z)  = \frac{1}{n}, 
}
where each $\ga_i \in \Mm_+^1(\Zz)$. 
Once an optimal $(\gamma_{\th,i})_i$ solving~\eqref{eq-ot-couplings-fitting-bis} is obtained, the gradient of $\Ee(\th)$ is computed as
\eq{
	\nabla \Ee(\th) = \sum_{i=1}^n \int_\Zz [\partial \fPF_\th(z)]^\top \nabla_1 c(\fPF_\th(z),x_i) \d \ga_i(z), 
}
where $\nabla_1 c(x,y) \in \RR^\dim$ is the gradient of $x \mapsto c(x,y)$. Note that as opposed to~\eqref{eq-grad-wass-loss-dual}, this formula does not involve computing the gradient of the potentials being solutions of the dual OT problem.

The class of estimators obtained using $\Ll=\MK_\c$, often called ``minimum Kantorovich estimators,'' was initially introduced in~\citep{bassetti2006minimum}; see also~\citep{CanasRosasco}. It has been used in the context of generative models by~\citep{CuturiBoltzman} to train restricted Boltzmann machines and in~\citep{bernton2017inference} in conjunction with approximate Bayesian computations.  Approximations of these computations using Deep Network are used to train deep generative models for both GAN~\citep{WassersteinGAN} and VAE~\citep{Bousquet2017}; see also~\citep{2017-Genevay-AutoDiff,2017-Genevay-gan-vae,salimans2018improving}. Note that the use of Sinkhorn divergences for parametric model fitting is used routinely for shape matching and registration, see~\citep{gold1998new,chui2000new,myronenko2010point,2017-feydy-miccai}.

\begin{rem1}{Metric learning and transfer learning}
	Let us insist on the fact that, for applications in machine learning, the success of OT-related methods very much depends on the choice of an adapted cost $c(x,y)$ which captures the geometry of the data. 
	While it is possible to embed many kinds of data in Euclidean spaces (see, for instance,~\citep{mikolov2013efficient} for words embedding), in many cases, some sort of adaptation or optimization of the metric is needed. 
	Metric learning for supervised tasks is a classical problem (see, for instance,~\citep{MAL-019,weinberger2009distance}) and it has been extended to the learning of the ground metric $c(x,y)$ when some OT distance is used in a learning pipeline~\citep{CuturiGroundMetric2014} (see also~\citealt{ZenICPR14,WangECCV12OLD,huang2016supervised}).
	Let us also mention the related inverse problem of learning the cost matrix from the observations of an optimal coupling $\P$, which can be regularized using a low-rank prior~\citep{dupuy2016estimating}.
	Related problems are transfer learning~\citep{pan2010survey} and domain adaptation~\citep{glorot2011domain}, where one wants to transfer some trained machine learning pipeline to adapt it to some new dataset. This problem can be modeled and solved using OT techniques; see~\citep{courty2017optimal,courty2017joint}.
\end{rem1}


\chapter{Extensions of Optimal Transport}
\label{c-extensions}

This chapter details several variational problems that are related to (and share the same structure of) the Kantorovich formulation of optimal transport. The goal is to extend optimal transport to more general settings: several input histograms and measures, unnormalized ones, more general classes of measures, and optimal transport between measures that focuses on local regularities (points nearby in the source measure should be mapped onto points nearby in the target measure) rather than a total transport cost, including cases where these two measures live in different metric spaces.

\section{Multimarginal Problems}
\label{sec-multimarginal}

Instead of coupling two input histograms using the Kantorovich formulation~\eqref{eq-mk-discr}, one can couple $S$ histograms $(\a_s)_{s=1}^S$, where $\a_s \in \simplex_{n_s}$, by solving the following multimarginal problem:
\eql{\label{eq-multimarginal-discr}
	\umin{\P \in \CouplingsD(\a_s)_s}
		\dotp{\C}{\P} \eqdef \sum_s \sum_{i_s=1}^{n_s} \C_{i_1,\ldots,i_S} \P_{i_1,\ldots,i_S},
}
where the set of valid couplings is
\eq{
	\CouplingsD(\a_s)_s = \enscond{\P \in \RR^{n_1 \times \ldots \times n_S}}{
		\foralls s, \foralls i_s, \sum_{\ell \neq s} \sum_{i_\ell=1}^{n_\ell}  \P_{i_1,\ldots,i_S} = \a_{s,i_s}
	}.
} 
The entropic regularization scheme~\eqref{eq-regularized-discr} naturally extends to this setting
\eq{
	\umin{\P \in \CouplingsD(\a_s)_s}
		\dotp{\P}{\C} - \varepsilon \HD(\P), 
}
and one can then apply Sinkhorn's algorithm to compute the optimal $\P$ in scaling form, where each entry indexed by a multi-index vector $i=(i_1,\ldots,i_S)$
\eq{	
		\P_i = \K_i \prod_{s=1}^S \uD_{s,i_s}
		\qwhereq
		\K \eqdef e^{-\frac{\C}{\varepsilon}}, 
}
where $\uD_s \in \RR_+^{n_s}$ are (unknown) scaling vectors, which are iteratively updated, by cycling repeatedly through $s=1,\ldots,S$, 
\eql{\label{eq-sinkh-multimarg}
	\uD_{s,i_s} \leftarrow 
	\frac{ \a_{s,i_s} }{
		\sum_{\ell \neq s} \sum_{i_\ell=1}^{n_\ell}  \K_i \prod_{r \neq s} \uD_{\ell,i_r}
	}
}.

\begin{rem2}{General measures}
The discrete multimarginal problem~\eqref{eq-multimarginal-discr} is generalized to measures $(\al_s)_s$ on spaces $(\X_1,\ldots,\X_S)$ by computing a coupling measure 
\eql{\label{eq-multi-marginal-generic}
	\umin{\pi \in \Couplings(\al_s)_s} \int_{\X_1 \times \ldots \times \X_S} c(x_1,\ldots,x_S) \d\pi(x_1,\ldots,x_S),
}
where the set of couplings is 
\eq{
	\Couplings(\al_s)_s \eqdef \enscond{ \pi \in \Mm_+^1(\X_1 \times \ldots \times \X_S) }{
		\foralls s=1,\ldots,S, P_{s,\sharp} \pi = \al_s,
	}
}
where $P_s : \X_1 \times \ldots \times \X_S \rightarrow \X_s$ is the projection on the $s$th component, $P_s(x_1,\ldots,x_S)=x_s$; see, for instance,~\citep{GangboSciech}.
We refer to~\citep{PassMultiReview,PassMultiMarginalStructure} for a review of the main properties of the multimarginal OT problem. 
A typical application of multimarginal OT is to compute approximation of solutions to quantum chemistry problems, and in particular, in density functional theory~\citep{CotarDFT,GorSeiVig,BuDePGor}. This problem is obtained when considering the singular Coulomb interaction cost
\eq{
	c(x_1,\ldots,x_S) = \sum_{i \neq j} \frac{1}{\norm{x_i-x_j}}. 
}
\end{rem2}

\begin{rem2}{Multimarginal formulation of the barycenter}\label{eq-multimarg-bary}
	It is pos\-si\-ble to recast the linear program optimization~\eqref{eq-barycenter-generic} as an optimization over a single coupling over $\X^{S+1}$ where the last marginal is the barycenter and the other ones are the input measure $(\al_s)_{s=1}^S$
	\eql{\label{eq-bary-multi-full}
		\umin{ \bar \pi \in \Mm_+^1(X^{S+1}) } 
			\int_{\X^{S+1}} \sum_{s=1}^S \la_s c(x,x_s) \d\bar\pi(x_1,\ldots,x_s,x)
	}
	\eq{
		\text{subject to} \quad
				\foralls s=1,\ldots,S, \quad P_{s,\sharp}\bar\pi=\al_s.
	}
	This stems from the ``gluing lemma,'' which states that given couplings $(\pi_s)_{s=1}^S$ where $\pi_s \in \Couplings(\al_s,\al)$, one can construct a higher-dimensional coupling $\bar \pi \in \Mm_+^1(X^{S+1})$ with marginals $\pi_s$, \ie such that $Q_{s\sharp}\bar\pi=\pi_s$, where $Q_s(x_1,\ldots,x_S,x) \eqdef (x_s,x) \in \X^2$.
	\todoK{also relate this to celebrated result in graphical model, see book of Wainright/Jordan}
	By explicitly minimizing in~\eqref{eq-bary-multi-full} with respect to the last marginal (associated to $x \in \X$), one obtains that solutions $\al$ of the barycenter problem~\eqref{eq-barycenter-generic} can be computed as $\al = A_{\la,\sharp} \pi$, where $A_\la$ is the ``barycentric map'' defined as
	\eq{
		A_\la : (x_1,\ldots,x_S) \in \X^S \mapsto 
		\uargmin{x \in \X} \sum_s \la_s c(x,x_s)
	}
	(assuming this map is single-valued), where $\pi$ is any solution of the multimarginal problem~\eqref{eq-multi-marginal-generic} with cost
	\eql{\label{eq-cost-bary-multi}
		c(x_1,\ldots,x_S) = \sum_\ell \la_\ell c(x_\ell, A_\la(x_1,\ldots,x_S)).
	}
	For instance, for $c(x,y)=\norm{x-y}^2$, one has, removing the constant squared terms, 
	\eq{
		c(x_1,\ldots,x_S) = -\sum_{r \leq s} \la_r \la_s \dotp{x_r}{x_s},
	} 
	which is a problem studied in~\citet{GangboSciech}.
	We refer to~\citet{Carlier_wasserstein_barycenter} for more details.
	This formula shows that if all the input measures are discrete $\be_s = \sum_{i_s=1}^{n_s} \a_{s,i_s} \de_{x_{s,i_s}}$, 
	then the barycenter $\al$ is also discrete and is obtained using the formula
	\eq{
		\al = \sum_{(i_1,\ldots,i_S)} \P_{(i_1,\ldots,i_S)}
			\de_{ A_\la(x_{i_1},\ldots,x_{i_S}) }, 
	}
	where $\P$ is an optimal solution of~\eqref{eq-multimarginal-discr} with cost matrix $\C_{i_1,\ldots,i_S} = c(x_{i_1},\ldots,x_{i_S})$
	as defined in~\eqref{eq-cost-bary-multi}. Since $\P$ is a nonnegative tensor of $\prod_s n_s$ dimensions obtained as the solution of a linear program with $\sum_s n_s-S+1$ equality constraints, an optimal solution $\P$ with up to $\sum_s n_s-S+1$ nonzero values can be obtained. A barycenter $\al$ with a support of up to $\sum_s n_s-S+1$ points can therefore be obtained. This result and other considerations in the discrete case can be found in~\citet{anderes2016discrete}.
\end{rem2}

\begin{rem2}{Relaxation of Euler equations}
A convex relaxation of Euler equations of incompressible fluid dynamics has been proposed by~\citeauthor{BrenierGeneralized} (\citeyear{BrenierEulerAMS}, \citeyear{BrenierEulerARMA}, \citeyear{BrenierEulerCPAM}, \citeyear{BrenierGeneralized}) and \citep{AmbrosioFigalliEuler}.
Similarly to the setting exposed in~\S\ref{sec-entropic-dynamic}, it corresponds to the problem of finding a probability distribution $\bar\pi \in \Mm_+^1(\bar\Xx)$ over the set $\bar\Xx$ of all paths $\ga : [0,1] \rightarrow \Xx$, which describes the movement of particules in the fluid.
This is a relaxed version of the initial partial differential equation model because, as in the Kantorovich formulation of OT, mass can be split. The evolution with time does not necessarily define a diffemorphism of the underlying space $\X$. 
The dynamic of the fluid is obtained by minimizing as in~\eqref{eq-ot-pathsspace} the energy $\int_0^1 \norm{\ga'(t)}^2 \d t$ of each path. 
The difference with OT over the space of paths is the additional incompressibilty of the fluid.
This incompressibilty is taken care of by imposing that the density of particules should be uniform at any time $t \in [0,1]$ (and not just imposed at initial and final times $t \in \{0,1\}$ as in classical OT). Assuming $\X$ is compact and denoting $\rho_{\Xx}$ the uniform distribution on $\Xx$, this reads $\bar P_{t,\sharp} \bar\pi = \rho_{\Xx}$ where $\bar P_t : \ga \in \bar\X \rightarrow \ga(t) \in \X$.
One can discretize this problem by replacing a continuous path $(\ga(t))_{t \in [0,1]}$ by a sequence of $S$ points $(x_{i_1}, x_{i_2},\ldots,x_{i_S})$ on a grid $(x_k)_{k=1}^n \subset \X$, and $\bar\Pi$ is represented by an $S$-way coupling $\P \in \RR^{n^S} \in \Couplings(\a_s)_s$, where the marginals are uniform $\a_s=n^{-1} \ones_n$. 
The cost of the corresponding multimarginal problem is then
\eql{\label{eq-separable-euler}
	\C_{i_1,\ldots,i_S} = \sum_{s=1}^{S-1} \norm{x_{i_s}-x_{i_{s+1}}}^2 + R \norm{x_{\sigma(i_1)}-x_{i_{S}}}^2.
}
Here $R$ is a large enough penalization constant, which is here to enforce the movement of particules between initial and final times, which is prescribed by a permutation $\sigma : \range{n} \rightarrow \range{n}$.
This resulting multimarginal problem is implemented efficiently in conjunction with Sinkhorn iterations~\eqref{eq-sinkh-multimarg} using the special structure of the cost, as detailed in~\citep{2015-benamou-cisc}.
Indeed, in place of the $O(n^S)$ cost required to compute the denominator appearing in~\eqref{eq-sinkh-multimarg}, one can decompose it as a succession of $S$ matrix-vector multiplications, hence with a low cost $Sn^2$.
Note that other solvers have been proposed, for instance, using the semidiscrete framework shown in~\S\ref{s-semidiscrete}; see~\citep{deGoes2015,gallouet2017lagrangian}.
\end{rem2}

\section{Unbalanced Optimal Transport}
\label{sec-unbalanced}

A major bottleneck of optimal transport in its usual form is that it requires the two input measures $(\al,\be)$ to have the same total mass. While many workarounds have been proposed (including renormalizing the input measure, or using dual norms such as detailed in \S~\ref{sec-dual-norms}), it is only recently that satisfying unifying theories have been developed. We only sketch here a simple but important particular case. 

Following~\citet{LieroMielkeSavareLong}, to account for arbitrary positive histograms $(\a,\b) \in \RR_+^n \times \RR_+^m$, the initial Kantorovich formulation~\eqref{eq-mk-discr} is ``relaxed'' by only penalizing marginal deviation using some divergence $\DivergmD_\phi$, defined in~\eqref{eq-discr-diverg}. This equivalently corresponds to minimizing an OT distance between approximate measures
\begin{align}\label{eq-unbalanced-pbm}
	\MKD_{\C}^\tau(\a,\b) &= \umin{\tilde\a,\tilde\b} \MKD_{\C}(\a,\b) + \tau_1 \DivergmD_\phi(\a,\tilde \a) + \tau_2 \DivergmD_\phi(\b,\tilde\b) \\
		&= \umin{\P \in \RR_+^{n \times m}} \dotp{\C}{\P} + \tau_1\DivergmD_\phi(\P\ones_m|\a) + \tau_2\DivergmD_\phi(\P^\top\ones_m|\b), 
\end{align}
where $(\tau_1,\tau_2)$ controls how much mass variations are penalized as opposed to transportation of the mass. In the limit $\tau_1=\tau_2 \rightarrow +\infty$, assuming $\sum_i \a_i=\sum_j \b_j$ (the ``balanced'' case), one recovers the original optimal transport formulation with hard marginal constraint~\eqref{eq-mk-discr}. 

This formalism recovers many different previous works, for instance introducing for $\DivergmD_\phi$ an $\ell^2$ norm~\citep{benamou2003numerical} or an $\ell^1$ norm as in partial transport~\citep{FigalliPartial,CaffarelliMcCannPartial}. 
A case of particular importance is when using $\DivergmD_\phi=\KLD$ the Kulback--Leibler divergence, as detailed in Remark~\ref{rem-wfr-static}.
For this cost, in the limit $\tau=\tau_1=\tau_2 \rightarrow 0$, one obtains the so-called squared Hellinger distance (see also Example~\ref{exmp-hellinger})
\eq{
	\MKD_{\C}^\tau(\a,\b) \overset{\tau \rightarrow 0}{\longrightarrow}
	\Hellinger^2(\a,\b) = \sum_i ( \sqrt{\a_i}-\sqrt{\b_i} )^2.
}

Sinkhorn's iterations~\eqref{eq-sinkhorn} can be adapted to this problem by making use of the generalized algorithm detailed in~\S\ref{sec-generalized}. The solution has the form~\eqref{eq-scaling-form} and the scalings are updated as
\eql{\label{eq-iterate-gen-sinkh}
	\uD \leftarrow \pa{ \frac{\a}{\K \vD} }^{\frac{\tau_1}{\tau_1+\varepsilon}}
	\qandq
	\vD \leftarrow \pa{ \frac{\b}{\transp{\K}\uD} }^{\frac{\tau_2}{\tau_2+\varepsilon}}.
}

\begin{rem2}{Generic measure}
For $(\al,\be)$ two arbitrary measures, the unbalanced version (also called ``log-entropic'') of~\eqref{eq-mk-generic} reads
\begin{align*}
	\MK_\c^\tau(\al,\be) \eqdef 
	&\umin{\pi \in \Mm_+(\X \times \Y)}
		\int_{\X \times \Y} \c(x,y) \d\pi(x,y) \\
		&+ \tau \Divergm_\phi(P_{1,\sharp}\pi|\al)
		+ \tau \Divergm_\phi(P_{2,\sharp}\pi|\be),
\end{align*}
where divergences $\Divergm_\phi$ between measures are defined in~\eqref{eq-phi-div}.
In the special case $c(x,y) = \norm{x-y}^2$, $\Divergm_\phi=\KL$, $\MK_\c^\tau(\al,\be)^{1/2}$ is the Gaussian--Hellinger distance~\citep{LieroMielkeSavareLong}, and it is shown to be a distance on $\Mm_+^1(\RR^d)$. 
\end{rem2}

\begin{rem2}{Wasserstein--Fisher--Rao}\label{rem-wfr-static}
For the particular choice of cost
\eq{
	c(x,y) = -\log \cos(\min(\dist(x,y)/\kappa,\pi/2)),
}
where $\kappa$ is some cutoff distance, and using $\Divergm_\phi=\KL$, then
\eq{
	\WFR(\al,\be) \eqdef \MK_\c^\tau(\al,\be)^{\frac{1}{2}}
} 
is the so-called Wasserstein--Fisher--Rao or Hellinger--Kantorovich distance. 
In the special case $\X=\RR^d$, this static (Kantorovich-like) formulation matches its dynamical counterparts~\eqref{eq-wfr-dynamic}, as proved independently by~\citet{LieroMielkeSavareLong,2015-chizat-unbalanced}. This dynamical formulation is detailed in~\S\ref{dynamic-unbalanced}.
%
\end{rem2}

The barycenter problem~\eqref{eq-barycenter-generic} can be generalized to handle an unbalanced setting by replacing $\MK_{\c}$ with $\MK_\c^\tau$. 
Figure~\ref{fig-interpolation-static} shows the resulting interpolation, providing a good illustration of the usefulness of the relaxation parameter $\tau$. 
The input measures are mixtures of two Gaussians with unequal mass. Classical OT requires the leftmost bump to be split in two and gives a nonregular interpolation. In sharp contrast, unbalanced OT allows the mass to vary during interpolation, so that the bumps are not split and local modes of the distributions are smoothly matched.
Using finite values for $\tau$ (recall that OT is equivalent to $\tau=\infty)$ is thus important to prevent irregular interpolations that arise because of mass splitting, which happens because of a ``hard'' mass conservation constraint.
The resulting optimization problem can be tackled numerically using entropic regularization and the generalized Sinkhorn algorithm detailed in~\S\ref{sec-generalized}. 

In practice, unbalanced OT techniques seem to outperform classical OT for applications (such as in imaging or machine learning) where the input data is noisy or not perfectly known. They are also crucial when the signal strength of a measure, as measured by its total mass, must be accounted for, or when normalization is not meaningful. This was the original motivation of~\citet{FrognerNIPS}, whose goal was to compare sets of word labels used to describe images. Unbalanced OT and the corresponding Sinkhorn iterations have also been used for applications to the dynamics of cells in~\citep{schiebinger2017reconstruction}.

\begin{figure}[h!]
\centering
\begin{tabular}{@{}c@{\hspace{10mm}}c@{}}
\includegraphics[width=.4\linewidth]{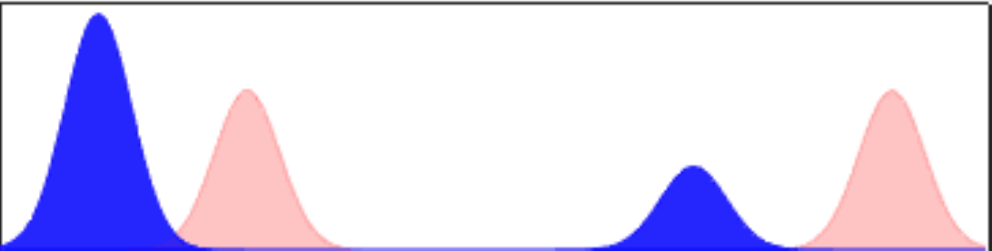} & 
\includegraphics[width=.4\linewidth]{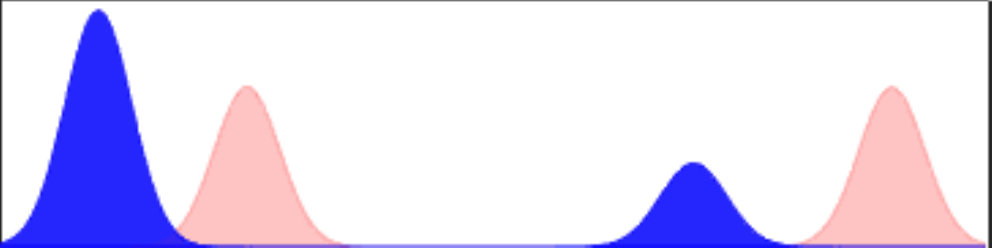} \\
\includegraphics[width=.4\linewidth]{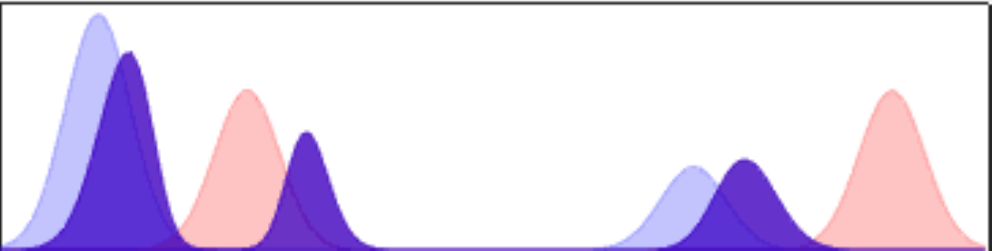} & 
\includegraphics[width=.4\linewidth]{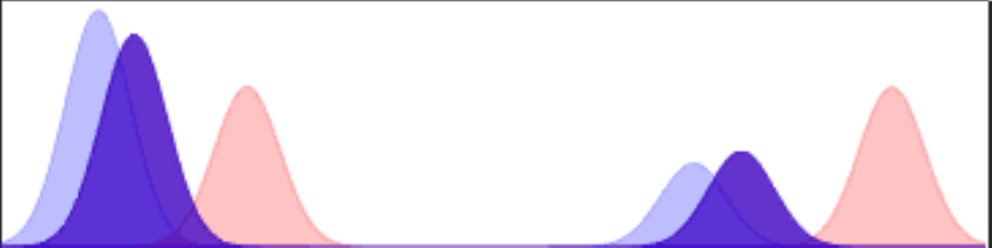} \\
\includegraphics[width=.4\linewidth]{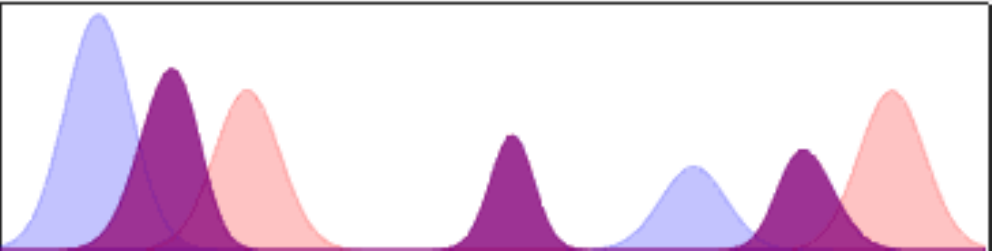} & 
\includegraphics[width=.4\linewidth]{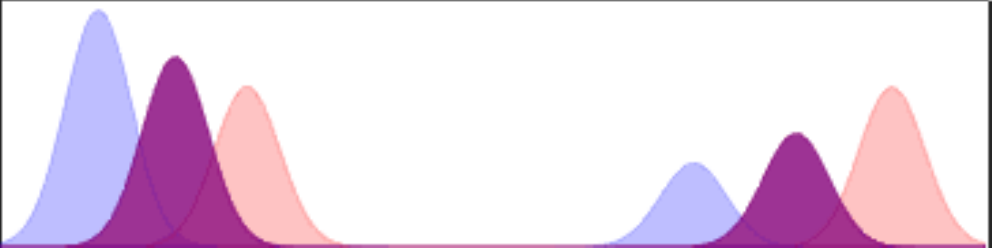} \\
\includegraphics[width=.4\linewidth]{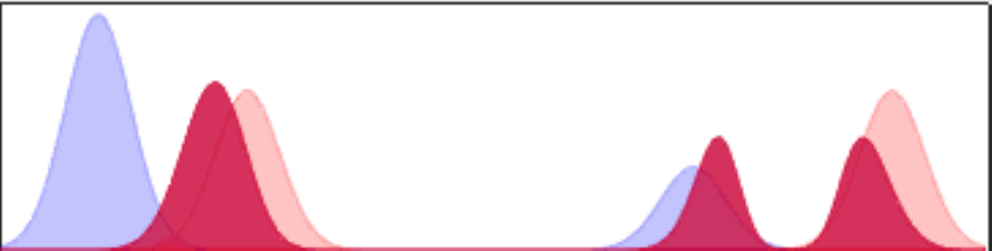} & 
\includegraphics[width=.4\linewidth]{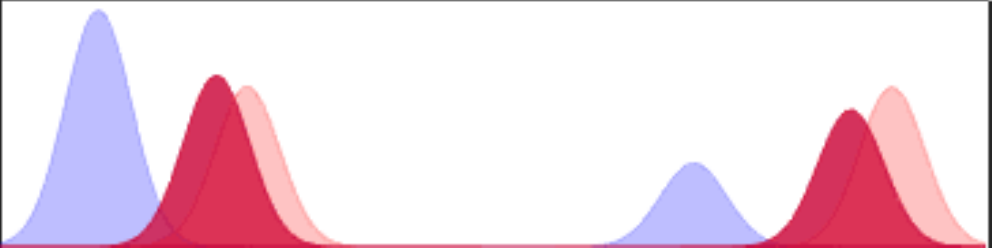} \\
\includegraphics[width=.4\linewidth]{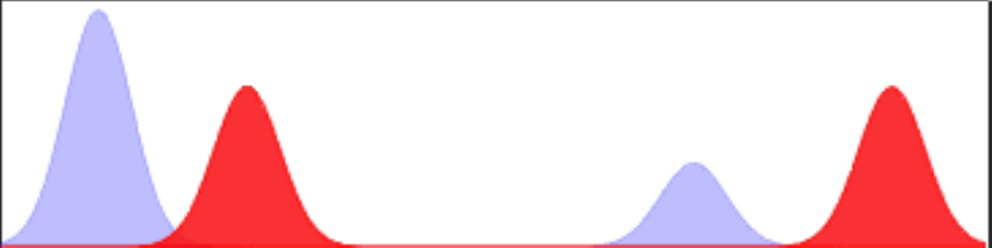} & 
\includegraphics[width=.4\linewidth]{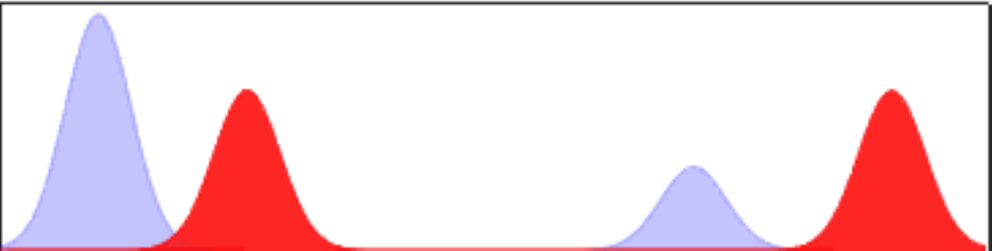} \\
Classical OT ($\tau=+\infty$) & Ubalanced OT ($\tau=1$)
\end{tabular}
\caption{\label{fig-interpolation-static}
Influence of relaxation parameter $\tau$ on unbalanced barycenters.
Top to bottom: the evolution of the barycenter between two input measures.
}
\end{figure}

\begin{rem}[Connection with dual norms]\label{rem-unb-dualnorms}
A particularly simple setup to account for mass variation is to use dual norms, as detailed in~\S\ref{sec-dual-norms}. By choosing a compact set $B \subset \Cc(\X)$ one obtains a norm defined on the whole space $\Mm(\X)$ (in particular, the measures do not need to be positive). A particular instance of this setting is the flat norm~\eqref{eq-set-flatnorm}, which is recovered as a special instance of unbalanced transport, when using $\Divergm_\phi(\al|\al')=\norm{\al-\al'}_{\TV}$ to be the total variation norm~\eqref{eq-defn-tv}; see, for instance,~\citep{hanin1992kantorovich,lellmann2014imaging}. 
We also refer to~\citep{schmitzer2017framework} for a general framework to define Wasserstein-1 unbalanced transport.
\end{rem}

\section{Problems with Extra Constraints on the Couplings}

Many other OT-like problems have been proposed in the literature. They typically correspond to adding extra constraints $\Cc$ on the set of feasible couplings appearing in the original OT problem~\eqref{eq-mk-generic}
\eql{\label{ot-extra-constr}
	\umin{\pi \in \Couplings(\al,\be)}
		\enscond{ \int_{\X \times \Y} \c(x,y) \d\pi(x,y) }{ \pi \in \Cc }.
}	
Let us give two representative examples.
The optimal transport with capacity constraint~\citep{km1} corresponds to imposing that the density $\density{\pi}$ (for instance, with respect to the Lebesgue measure) is upper bounded
\eql{\label{eq-capacity-constr}
	\Cc = \enscond{ \pi }{ \density{\pi} \leq \kappa }
}
for some $\kappa>0$. This constraint rules out singular couplings localized on Monge maps. 
The martingale transport problem (see, for instance,~\citet{GalichonMartingale,dolinsky2014martingale,TanTouzi,beiglbock2013model}), which finds many applications in finance, imposes the so-called martingale constraint on the conditional mean of the coupling, when $\X=\Y=\RR^d$:
\eql{\label{eq-martingale-constr}
	\Cc = \enscond{ \pi }{ \foralls x \in \RR^d,  \int_{\RR^d} y \frac{\d \pi(x,y)}{\d\al(x)\d\be(y)} \d\be(y) = x }.
}
This constraint imposes that the barycentric projection map~\eqref{eq-bary-proj} of any admissible coupling must be equal to the identity.
For arbitrary $(\al,\be)$, this set $\Cc$ is typically empty, but necessary and sufficient conditions exist ($\al$ and $\be$ should be in ``convex order'') to ensure $\Cc \neq \emptyset$ so that  $(\al,\be)$ satisfy a martingale constraint. This constraint can be difficult to enforce numerically when discretizing an existing problem.
It also forbids the solution to concentrate on a single Monge map, and can lead to couplings concentrated on the union of several graphs (a ``multivalued'' Monge map), or even more complicated support sets.
Using an entropic penalization as in~\eqref{eq-entropic-generic}, one can solve approximately~\eqref{ot-extra-constr} using the Dykstra algorithm as explained in~\citet{2015-benamou-cisc}, which is a generalization of Sinkhorn's algorithm shown in~\S\ref{sec-sinkhorn}. This requires computing the projection onto $\Cc$ for the $\KL$ divergence, which is straightforward for~\eqref{eq-capacity-constr} but cannot be done in closed form~\eqref{eq-martingale-constr} and thus necessitates subiterations; see~\citep{guo2017computational} for more details.

\section{Sliced Wasserstein Distance and Barycenters}\label{sec-sliced}

One can define a distance between two measures $(\al,\be)$ defined on $\RR^\dim$ by aggregating 1-D Wasserstein distances between their projections onto all directions of the sphere. This defines
\eql{\label{eq-sw-def}
	\SW(\al,\be)^2 \eqdef \int_{\SS^d} \Wass_2( P_{\th,\sharp} \al, P_{\th,\sharp} \be )^2 \d\th,
}
where $\SS^\dim = \ensconds{\th \in \RR^\dim}{\norm{\th}=1}$ is the $\dim$-dimensional sphere, and $P_\th : x \in \RR^\dim \rightarrow \RR$ is the projection. This approach is detailed in~\citep{2013-Bonneel-barycenter}, following ideas from Marc Bernot. It is related to the problem of Radon inversion over measure spaces~\citep{AbrahamRadon}. 

\paragraph{Lagrangian discretization and stochastic gradient descent. }

The advantage of this functional is that 1-D Wasserstein distances are simple to compute, as detailed in~\S\ref{sec:specialcases}. In the specific case where $m=n$ and 
\eql{\label{eq-lagrangian-sliced}
	\al=\frac{1}{n}\sum_{i=1}^n \de_{x_i}
	\qandq 
	\be=\frac{1}{n}\sum_{i=1}^m \de_{y_i}, 
}
this is achieved by simply sorting points
\eq{
	\SW(\al,\be)^2 = \int_{\SS^d} \pa{
		 \sum_{i=1}^n | \dotp{ x_{\si_\th(i)} - y_{\kappa_\th(i)} }{\th} |^2 
		}
		\d\th,
}
where $\si_\th,\kappa_\th \in \Perm(n)$ are the permutation ordering in increasing order, respectively, $(\dotp{x_i}{\th})_i$ and $(\dotp{y_i}{\th})_i$. 

Fixing the vector $y$, the function $\Ee_\be(x) \eqdef \SW(\al,\be)^2$ is smooth, and one can use this function to define a mapping by gradient descent 
\eql{\label{eq-grad-sliced-lagr}
	x \leftarrow x - \tau \nabla\Ee_\be(x)
	\qwhereq
}
\eq{
	\nabla\Ee_\be(x)_i = 
		2 \int_{\SS^d} \pa{
		 	\dotp{ x_{i} - y_{\kappa_\th \circ \si_\th^{-1}(i)} }{\th}  \th
		}
		\d\th
}
using a small enough step size $\tau>0$. 
To make the method tractable, one can use a stochastic gradient descent (SGD), replacing this integral with a discrete sum against randomly drawn directions $\th \in \SS^\dim$ (see~\S\ref{sec-sgd} for more details on SGD). 
The flow~\eqref{eq-grad-sliced-lagr} can be understood as (Langrangian implementation of) a Wasserstein gradient flow (in the sense of~\S\ref{sec-grad-flows}) of the function $\al \mapsto \SW(\al,\be)^2$. Numerically, one finds that this flow has no local minimizer and that it thus converges to $\al=\be$. The usefulness of the Lagrangian solver is that, at convergence, it defines a matching (similar to a Monge map) between the two distributions. This method has been used successfully for color transfer and texture synthesis in~\citep{rabin-ssvm-11} and is related to the alternate minimization approach detailed in~\citep{pitie2007automated}.

It is simple to extend this Lagrangian scheme to compute approximate ``sliced'' barycenters  of measures, by mimicking the Frechet definition of Wasserstein barycenters~\eqref{eq-barycenter-generic} and minimizing 
\eql{\label{eq-barycenter-sliced}
		\umin{\al \in \Mm_+^1(\X)} \sum_{s=1}^S \la_s \SW(\al,\be_s)^2,
}
given a set $(\be_s)_{s=1}^S$ of fixed input measure. Using a Lagrangian discretization of the form~\eqref{eq-lagrangian-sliced} for both $\al$ and the $(\be_s)_s$, one can perform the nonconvex minimization over the position $x=(x_i)_i$
\eql{\label{eq-bary-lagrangian-sliced}
	\umin{x} \Ee(x) \eqdef \sum_{s} \la_s \Ee_{\be_s}(x),
	\qandq 
	\nabla \Ee(x) = \sum_{s} \la_s \nabla \Ee_{\be_s}(x), 
}
by gradient descent using formula~\eqref{eq-grad-sliced-lagr} to compute $\nabla \Ee_{\be_s}(x)$ (coupled with a random sampling of the direction $\th$).

\paragraph{Eulerian discretization and Radon transform. }

A related way to compute an approximated sliced barycenter, without resorting to an iterative minimization scheme, is to use the fact that~\eqref{eq-sw-def} computes a distance between the Radon transforms $\Rr(\al)$ and $\Rr(\be)$ where
\eq{
	\Rr(\al) \eqdef ( P_{\th,\sharp} \al )_{\th \in \SS^d}.
}
A crucial point is that the Radon transform is invertible and that its inverse can be computed using a filtered backprojection formula. 
Given a collection of measures $\rho = (\rho_\th)_{\th \in \SS^d}$, one defines the filtered  backprojection operator as
\eql{\label{eq-filter-backproj}
	\Rr^+(\rho) = C_\dim \De^{\frac{d-1}{2}} \Bb(\rho),
}
where $\xi = \Bb(\rho) \in \Mm(\RR^d)$ is the measure defined through the relation 
\eql{\label{eq-backproj}
	\foralls g \in \Cc(\RR^d), \quad
	\int_{\RR^d} g(x) \d\xi(x) = 
	\int_{\SS^d}  \int_{\RR^{\dim-1}}   \int_{\RR}   
		g( r \th + U_\th z ) 
	\d \rho_\th(r) \d z \d \th,
}
where $U_\th$ is any orthogonal basis of $\th^\bot$, 
and where $C_\dim \in \RR$ is a normalizing constant which depends on the dimension.
Here $\De^{\frac{d-1}{2}}$ is a fractional Laplacian, which is the high-pass filter defined over the Fourier domain as 
$\hat \De^{\frac{d-1}{2}}(\om) = \norm{\om}^{d-1}$.
The definition of the backprojection~\eqref{eq-backproj} adds up the contribution of all the measures $(\rho_\th)_\th$ by extending each one as being constant in the directions orthogonal to $\th$.
One then has the left-inverse relation $\Rr^+ \circ \Rr = \Identity_{\Mm(\RR^\dim)}$, so that $\Rr^+$ is a valid reconstruction formula.

\newcommand{\MyFigRadon}[1]{\includegraphics[width=.19\linewidth]{radon-bary/#1}}
\begin{figure}[h!]
\centering
\begin{tabular}{@{}c@{\hspace{1mm}}c@{\hspace{1mm}}c@{\hspace{1mm}}c@{\hspace{1mm}}c@{}}
\MyFigRadon{bary-1} &
\MyFigRadon{bary-3} &
\MyFigRadon{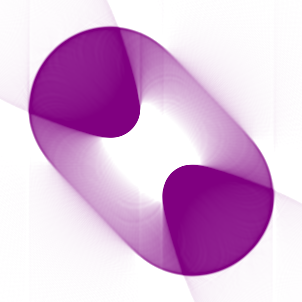} &
\MyFigRadon{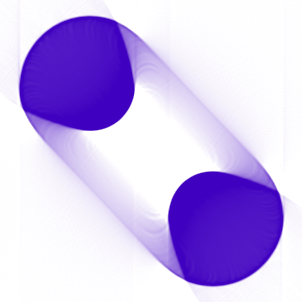} &
\MyFigRadon{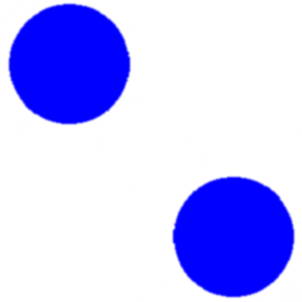} \\
\MyFigRadon{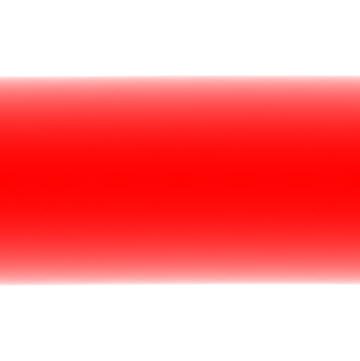} &
\MyFigRadon{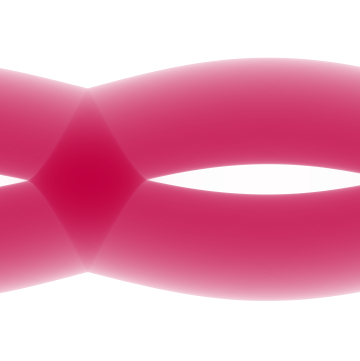} &
\MyFigRadon{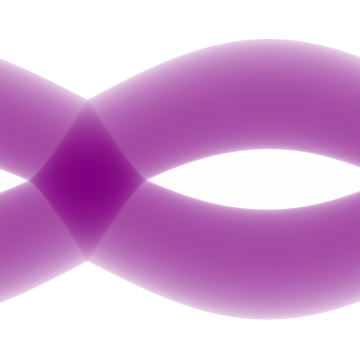} &
\MyFigRadon{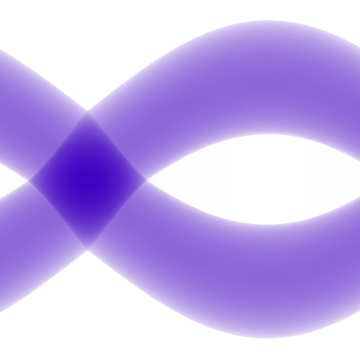} &
\MyFigRadon{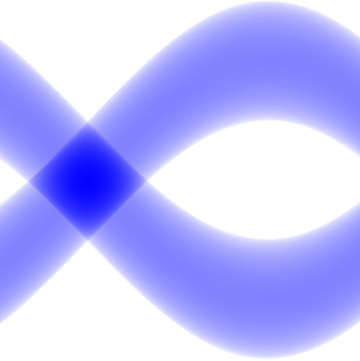} \\
$t=0$ & $t=1/4$ & $t=1/2$ & $t=3/4$ & $t=1$  
\end{tabular}
\caption{\label{fig-radon-bary}
Example of sliced barycenters computation using the Radon transform (as defined in~\eqref{eq-barycenter-radon-sliced}). 
Top: barycenters $\al_t$ for $S=2$ two input and weights $(\la_1,\la_2)=(1-t,t)$. 
Bottom: their Radon transform $\Rr(\al_t)$ (the horizontal axis being the orientation angle $\th$). 
}
\end{figure}

In order to compute barycenters of input densities, it makes sense to replace formula~\eqref{eq-barycenter-generic} by its equivalent using Radon transform, and thus consider independently for each $\th$ the 1-D barycenter problem
\eql{\label{eq-barycenter-radon-sliced}
		\rho_\th^\star \in \uargmin{ (\rho_\th \in \Mm_+^1(\RR)) } \sum_{s=1}^S \la_s \Wass_2( \rho_\th, P_{\th,\sharp} \be_s)^2.
}
Each 1-D barycenter problem is easily computed using the monotone rearrangement as detailed in Remark~\ref{rem-bary-1d}.
The Radon approximation $\al_{R} \eqdef \Rr^+( \rho^\star )$ of a sliced barycenter solving~\eqref{eq-barycenter-generic} is then obtained by the inverse Radon transform $\Rr^+$. Note that in general, $\al_R$ is not a solution to~\eqref{eq-barycenter-generic} because the Radon transform is not surjective, so that  $\rho^\star$, which is obtained as a barycenter of the Radon transforms $\Rr(\be_s)$ does not necessarily belong to the range of $\Rr$. But numerically it seems in practice to be almost the case~\citep{2013-Bonneel-barycenter}.
Numerically, this Radon transform formulation is very effective for input measures and barycenters discretized on a fixed grid (\emph{e.g.} a uniform grid for images), and $\Rr$ and well as $\Rr^+$ are computed approximately on this grid using fast algorithms (see, for instance,~\citep{FastSlantStack}).
Figure~\ref{fig-radon-bary} illustrates this computation of barycenters (and highlights the way the Radon transforms are interpolated), while Figure~\ref{fig-sliced-bary-compar} shows a comparison of the Radon barycenters~\eqref{eq-barycenter-radon-sliced} and the ones obtained by Lagrangian discretization~\eqref{eq-bary-lagrangian-sliced}.

\begin{figure}[h!]
\centering
\begin{tabular}{@{}c@{\hspace{1mm}}c@{\hspace{1mm}}c@{}}
\includegraphics[width=.32\linewidth]{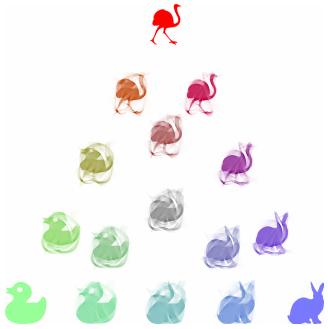}&
\includegraphics[width=.32\linewidth]{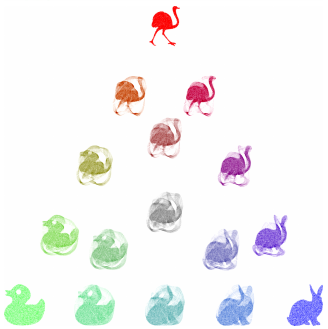}&
\includegraphics[width=.32\linewidth]{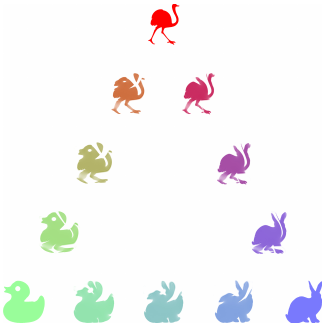} \\
Radon & Lagrangian & Wasserstein
\end{tabular}
\caption{\label{fig-sliced-bary-compar}
Comparison of barycenters computed using Radon transform~\eqref{eq-barycenter-radon-sliced} (Eulerian discretization), Lagrangian discretization~\eqref{eq-bary-lagrangian-sliced}, and Wasserstein OT (computed using Sinkhorn iterations~\eqref{eq-sinkhorn-bary}). 
}
\end{figure}

\paragraph{Sliced Wasserstein kernels.}

Beside its computational simplicity, another advantage of the sliced Wasserstein distance is that it is isometric to a Euclidean distance (it is thus a ``Hilbertian'' metric), as detailed in Remark~\ref{rem-1d-ot-generic}, and in particular formula~\eqref{eq-wass-cumul}. 
As highlighted in~\S\ref{sec-non-embeddability}, this should be contrasted with the Wasserstein distance $\Wass_2$ on $\RR^\dim$, which is not Hilbertian in dimension $\dim \geq 2$. 
It is thus possible to use this sliced distance to equip the space of distributions $\Mm_1^+(\RR^\dim)$ with a reproducing kernel Hilbert space structure (as detailed in~\S\ref{sec-non-embeddability}). 
One can, for instance, use the exponential and energy distance kernels
\eq{
	\Krkhs(\al,\be) = e^{-\frac{\SW(\al,\be)^p}{2\si^p}}
	\qandq
	\Krkhs(\al,\be) = -\SW(\al,\be)^p
}
for $1 \leq p \leq 2$ for the exponential kernels and $0 < p < 2$ for the energy distance kernels.
This means that for any collection $(\al_i)_i$ of input measures, the matrix $( \Krkhs(\al_i,\al_j) )_{i,j}$ is symmetric positive semidefinite. It is possible to use these kernels to perform a variety of machine learning tasks using the ``kernel trick,'' for instance, in regression, classification (SVM and logistic), clustering (K-means) and dimensionality reduction (PCA)~\citep{Hofmann2008}.
We refer to~\citet{kolouri2016sliced} for details and applications.

\section{Transporting Vectors and Matrices}

Real-valued measures $\al \in \Mm(\Xx)$ are easily generalized to vector-valued measures $\al \in \Mm(\Xx;\VV)$, where $\VV$ is some vector space. For notational simplicity, we assume $\VV$ is Euclidean and equipped with some inner product $\dotp{\cdot}{\cdot}$ (typically $\VV=\RR^\dim$ and the inner product is the canonical one). Thanks to this inner product, vector-valued measures are identified with the dual of continuous functions $g : \Xx \rightarrow \VV$, \ie for any such $g$, one defines its integration against the measure as 
\eql{\label{eq-integration-vec-valued}
	\int_\Xx g(x) \d\al(x) \in \RR,
}
which is a linear operation on $g$ and $\al$. A discrete measure has the form $\al = \sum_i \a_i \de_{x_i}$ where $(x_i,a_i) \in \Xx \times \VV$ and the integration formula~\eqref{eq-integration-vec-valued} simply reads
\eq{
	\int_\Xx g(x) \d\al(x) = \sum_i \dotp{\a_i}{g(x_i)} \in \RR.
}
Equivalently, if $\VV=\RR^\dim$, then such an $\al$ can be viewed as a collection $(\al_s)_{s=1}^\dim$ of $\dim$ ``classical'' real-valued measures (its coordinates), writing  
\eq{
	\int_\Xx g(x) \d\al(x) =  \sum_{s=1}^\dim \int_\Xx g_s(x) \d\al_s(x), 
}
where $g(x)=(g_s(x))_{s=1}^\dim$ are the coordinates of $g$ in the canonical basis. 

\paragraph{Dual norms.}

It is nontrivial, and in fact in general impossible, to extend OT distances to such a general setting. Even coping with real-valued measures taking both positive and negative values is difficult. The only simple option is to consider dual norms, as defined in~\S\ref{sec-dual-norms}. Indeed, formula~\eqref{eq-w1-cont} readily extends to $\Mm(\Xx;\VV)$ by considering $B$ to be a subset of $\Cc(\Xx;\VV)$. So in particular, $\Ww_1$, the flat norm and MMD norms can be computed for vector-valued measures. 

\paragraph{OT over cone-valued measures. }

It is possible to define more advanced OT distances when $\al$ is restricted to be in a subset $\Mm(\Xx;\Vv) \subset \Mm(\Xx;\VV)$. The set $\Vv$ should be a positively 1-homogeneous convex cone of $\VV$
\eq{
	\Vv \eqdef \enscond{ \la u }{ \la \in \RR^+,  u \in \Vv_0, }
}
where $\Vv_0$ is a compact convex set. 
A typical example is the set of positive measures where $\Vv=\RR_+^\dim$. Dynamical convex formulations of OT over such a cone have been proposed; see~\citep{MatthesPositive}. This has been applied to model the distribution of chemical components. 
Another important example is the set of positive symmetric matrices $\Vv=\Ss_+^\dim \subset \RR^{\dim \times \dim}$. It is of course possible to use dual norms over this space, by treating matrices as vectors; see, for instance,~\citep{Ning2014metrics}. Dynamical convex  formulations for OT over such a cone have been provided~\citep{Chen2016,JiangSpectral}. Some static (Kantorovich-like) formulations also have been proposed~\citep{ning2015matrix,2016-peyre-qot}, but a mathematically sound theoretical framework is still missing. In particular, it is unclear if these static approaches define distances for vector-valued measures and if they relate to some dynamical formulation.  Figure~\ref{fig-tensors} is an example of tensor interpolation obtained using the method detailed in~\citep{2016-peyre-qot}, which proposes a generalization of Sinkhorn algorithms using quantum relative entropy~\eqref{eq-quantum-entropy} to deal with tensor fields.

\paragraph{OT over positive matrices. }

A related but quite different setting is to replace discrete measures, \ie histograms $\a \in \Si_n$, by positive matrices with unit trace $A \in \Ss_n^+$ such that $\tr(A)=1$. The rationale is that the eigenvalues $\la(A) \in \Si_n$ of $A$ play the role of a histogram, but one also has to take care of the rotations of the eigenvectors, so that this problem is more complicated.

One can extend several divergences introduced in~\S\ref{sec-phi-div} to this setting. For instance, the Bures metric~\eqref{eq-bure-defn} is a generalization of the Hellinger distance (defined in Remark~\ref{exmp-hellinger}), since they are equal on positive diagonal matrices. 
One can also extend the Kullback--Leibler divergence~\eqref{eq-kl-defn} (see also Remark~\ref{ex_KLdiv}), which is generalized to positive matrices as
\eql{\label{eq-quantum-entropy}
	\KLD(A|B) \eqdef \tr\pa{ P \log(P)-P\log(Q)-P+Q, }
}
where $\log(\cdot)$ is the matrix logarithm. This matrix $\KLD$ is convex with both of its arguments. 

It is possible to solve convex dynamic formulations to define OT distances between such matrices~\citep{Carlen2014,Chen2016,ChenGangbo17}. There also exists an equivalent of Sinkhorn's algorithm, which is due to~\citet{gurvits2004classical} and has been extensively studied in~\citep{georgiou2015positive}; see also the review paper~\citep{ReviewSinkhorn}. It is known to converge only in some cases but seems empirically to always work.

\newcommand{\MyFigTensorMesh}[1]{\includegraphics[width=.100\linewidth,trim=140 10 120 0,clip]{tensors/mesh/interpol-#1}}
\newcommand{\MyFigTensorImg}[1]{\includegraphics[width=.105\linewidth,trim=80 20 80 20,clip]{tensors/2d/interpol-ellipses-#1}}
\begin{figure}[h!]\centering
\centering
\begin{tabular}{@{\hspace{1mm}}c@{\hspace{1mm}}c@{\hspace{1mm}}c@{\hspace{1mm}}c@{\hspace{1mm}}c@{\hspace{1mm}}c@{\hspace{1mm}}c@{\hspace{1mm}}c@{\hspace{1mm}}c@{}}
\MyFigTensorImg{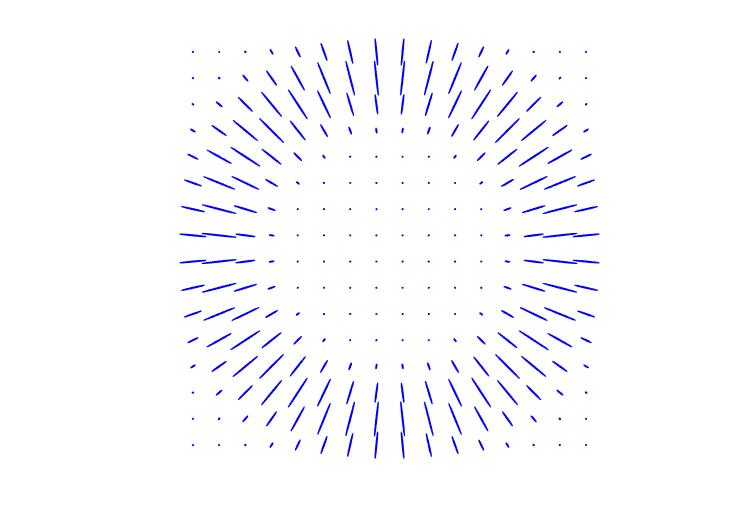}&
\MyFigTensorImg{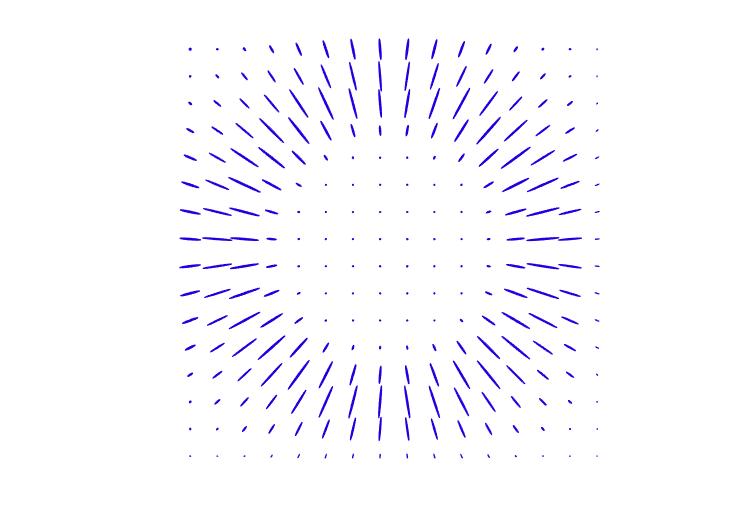}&
\MyFigTensorImg{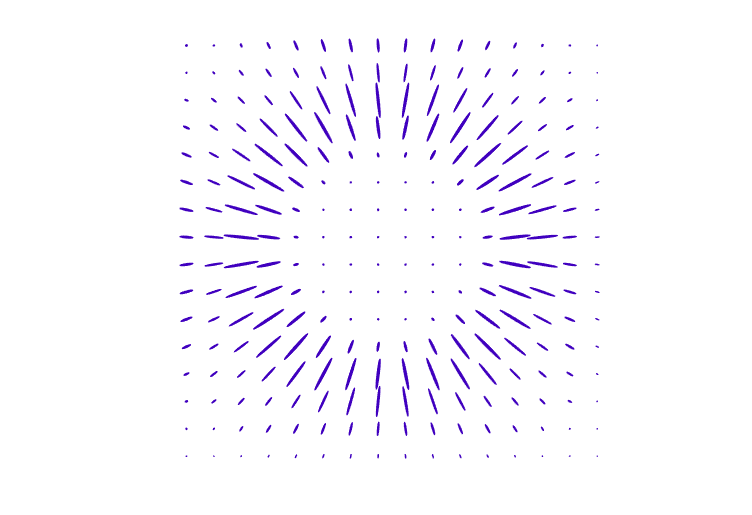}&
\MyFigTensorImg{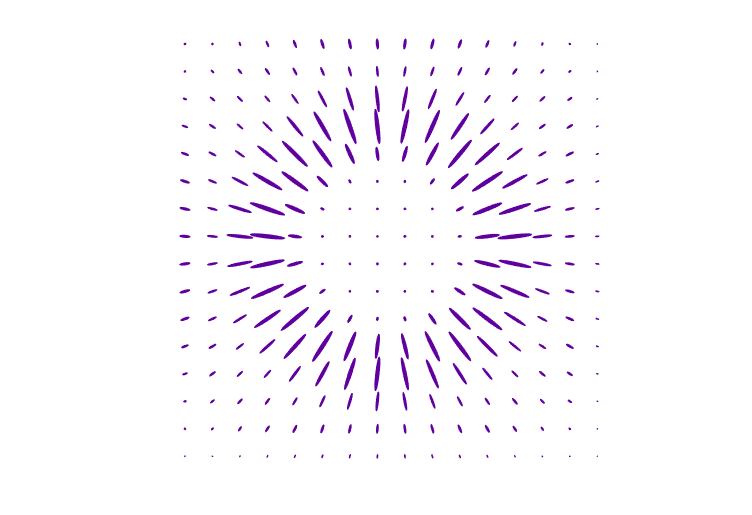}&
\MyFigTensorImg{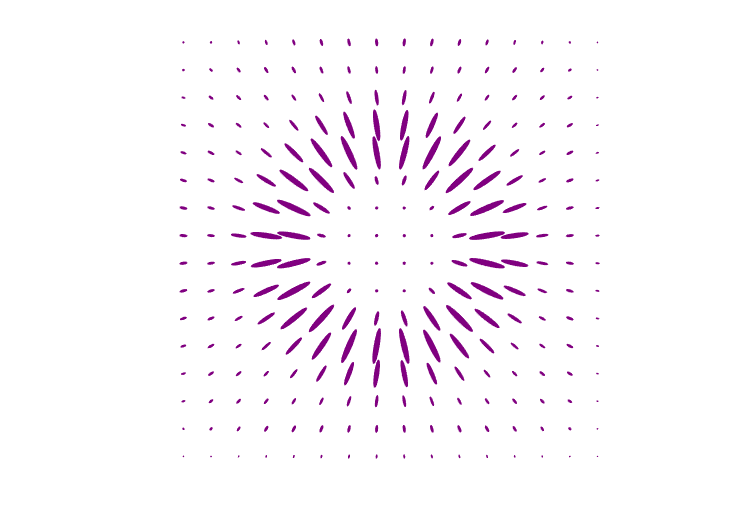}&
\MyFigTensorImg{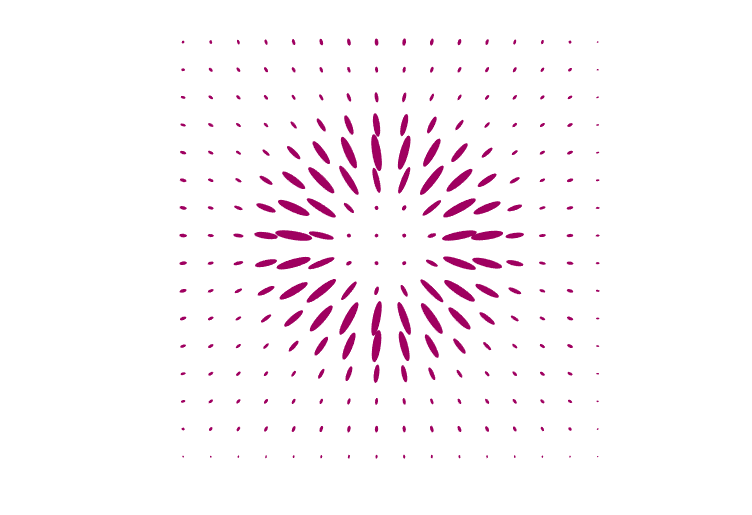}&
\MyFigTensorImg{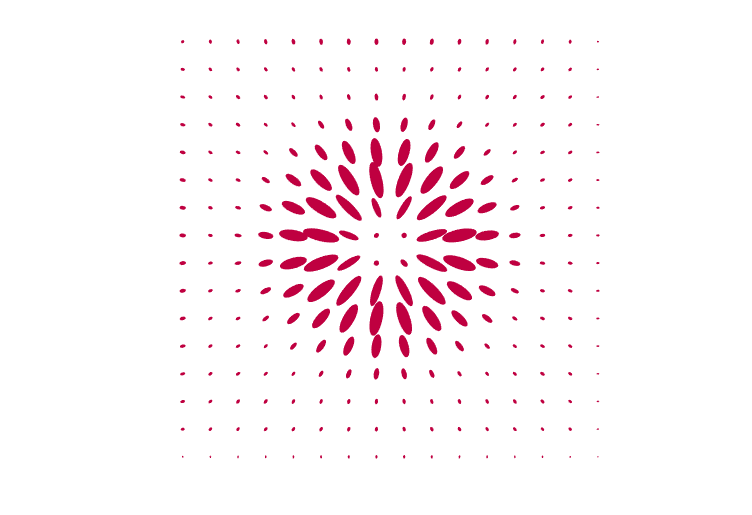}&
\MyFigTensorImg{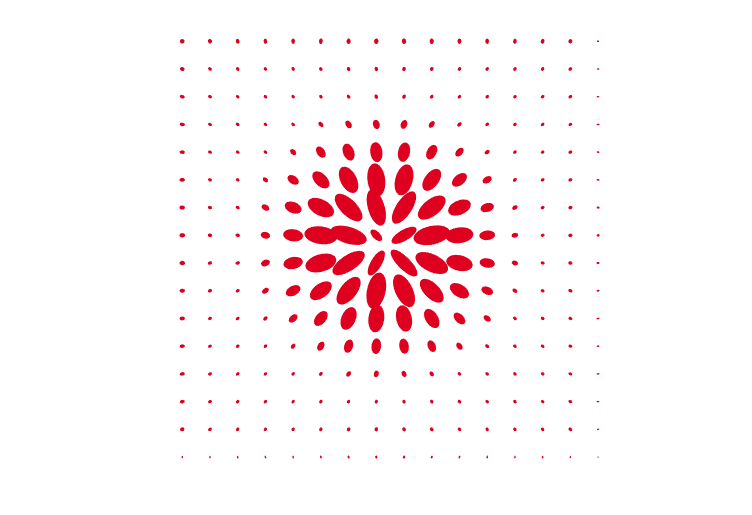}&
\MyFigTensorImg{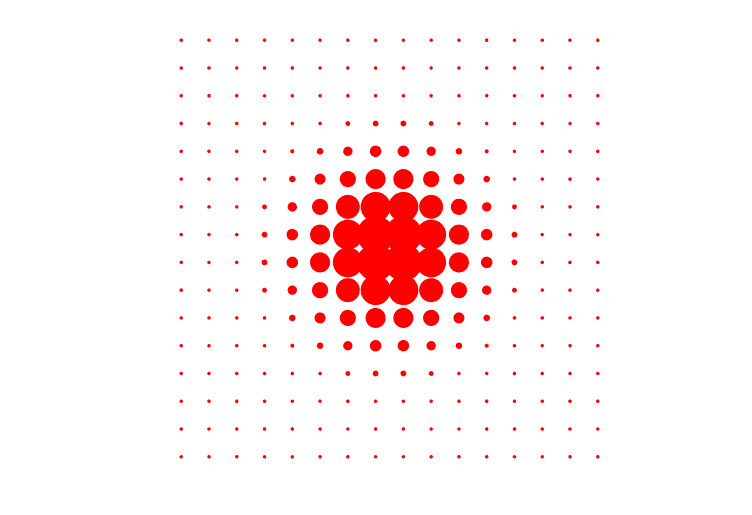}\\
\MyFigTensorMesh{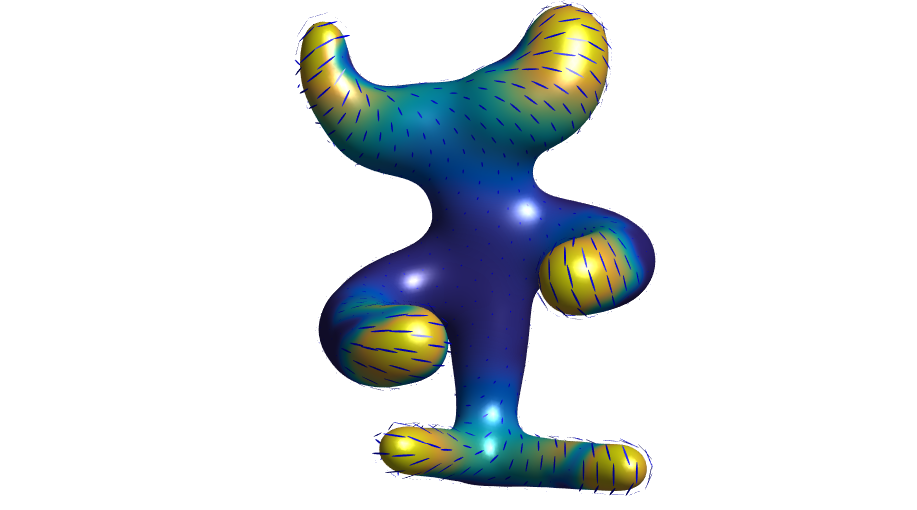}&
\MyFigTensorMesh{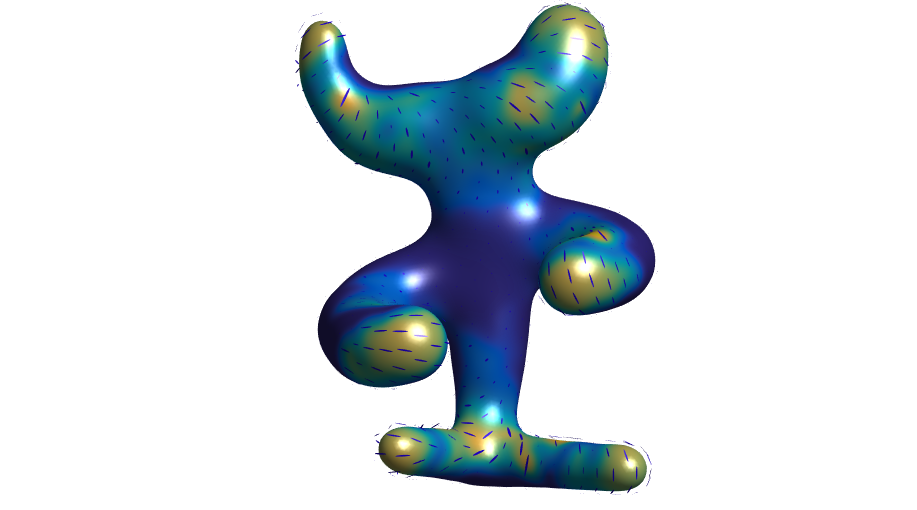}&
\MyFigTensorMesh{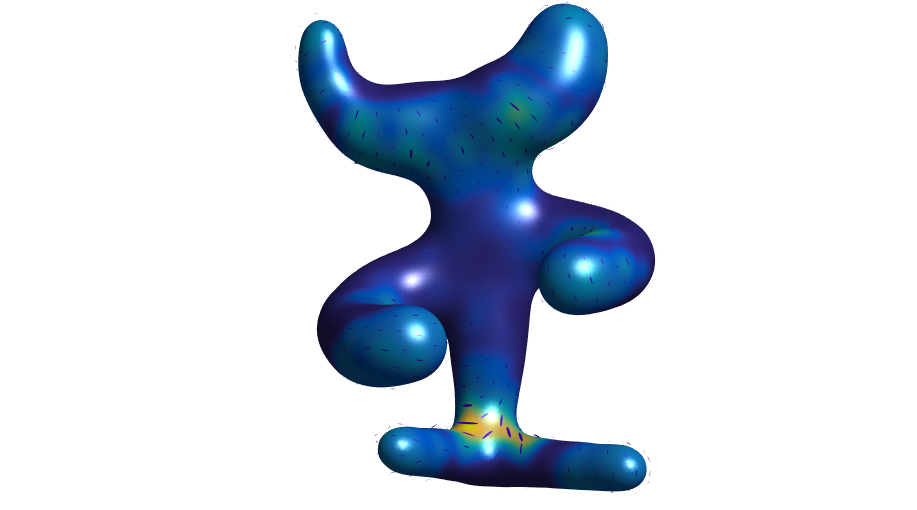}&
\MyFigTensorMesh{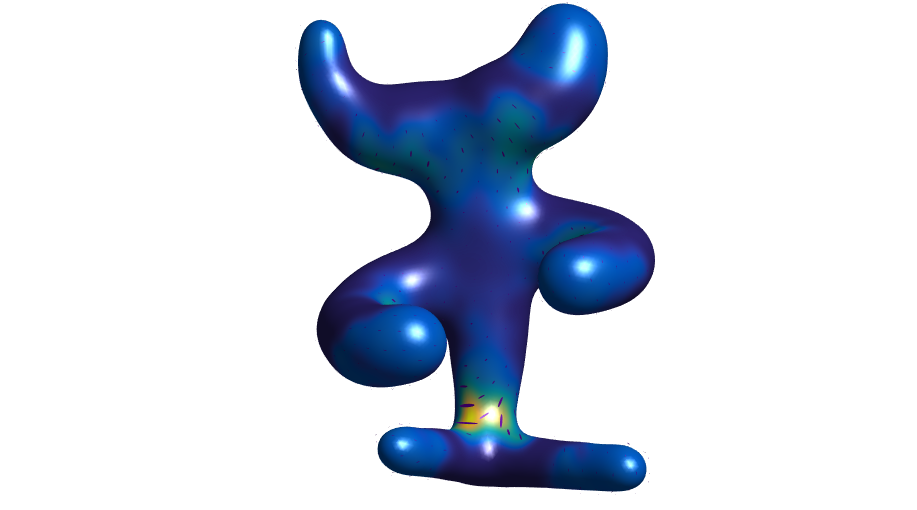}&
\MyFigTensorMesh{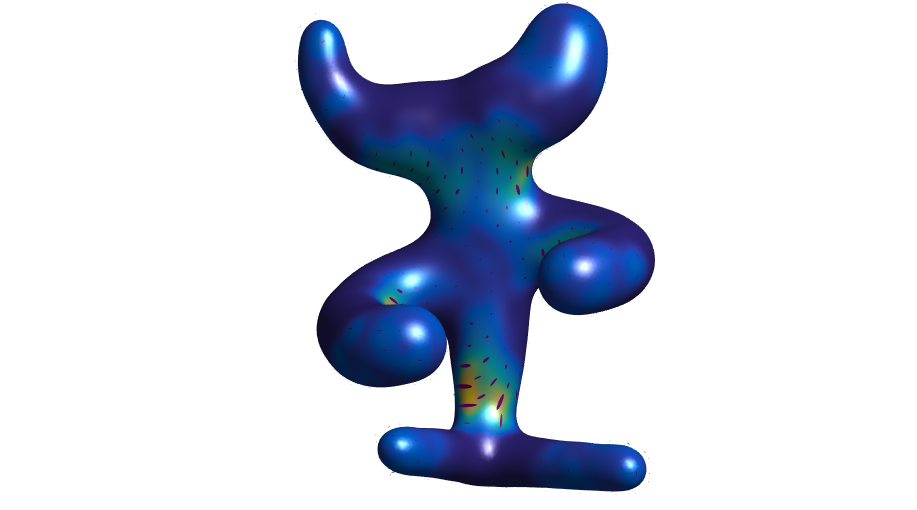}&
\MyFigTensorMesh{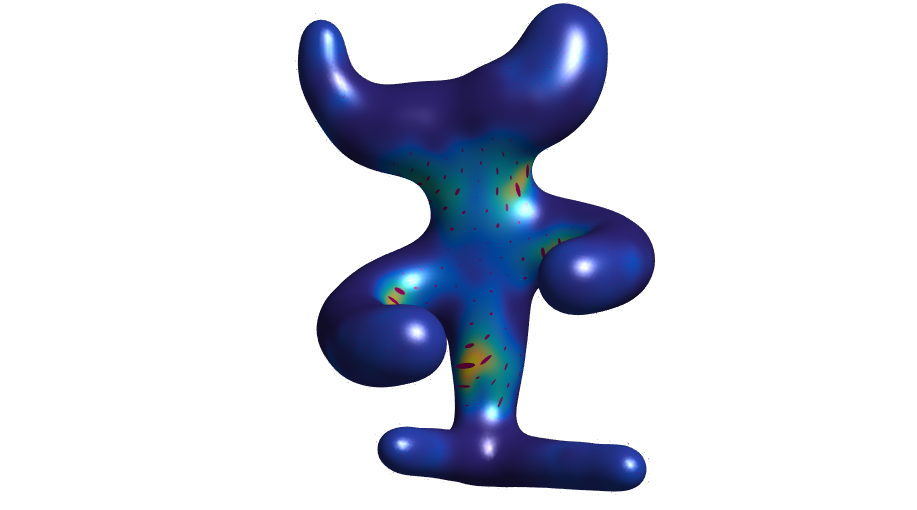}&
\MyFigTensorMesh{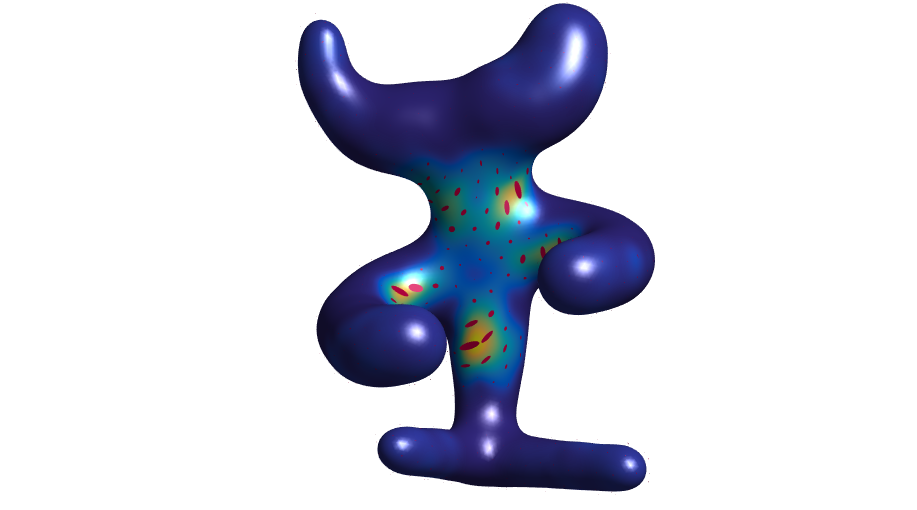}&
\MyFigTensorMesh{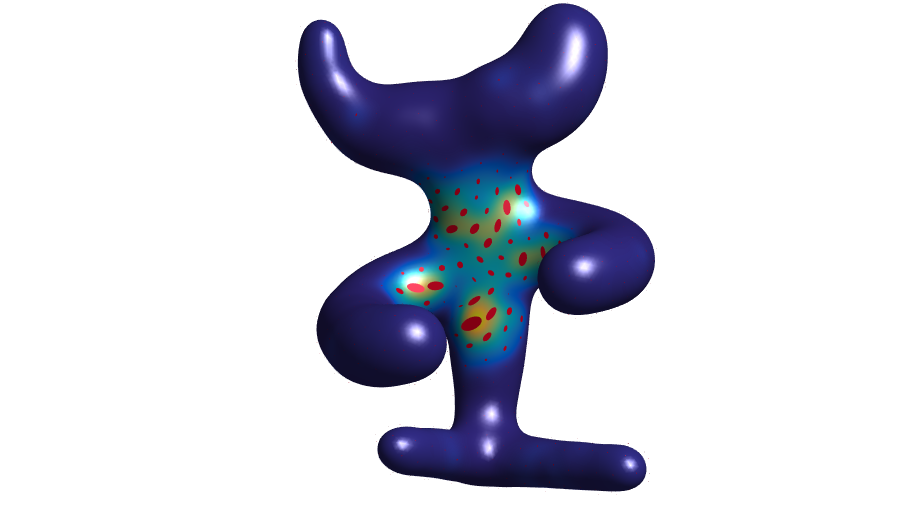}&
\MyFigTensorMesh{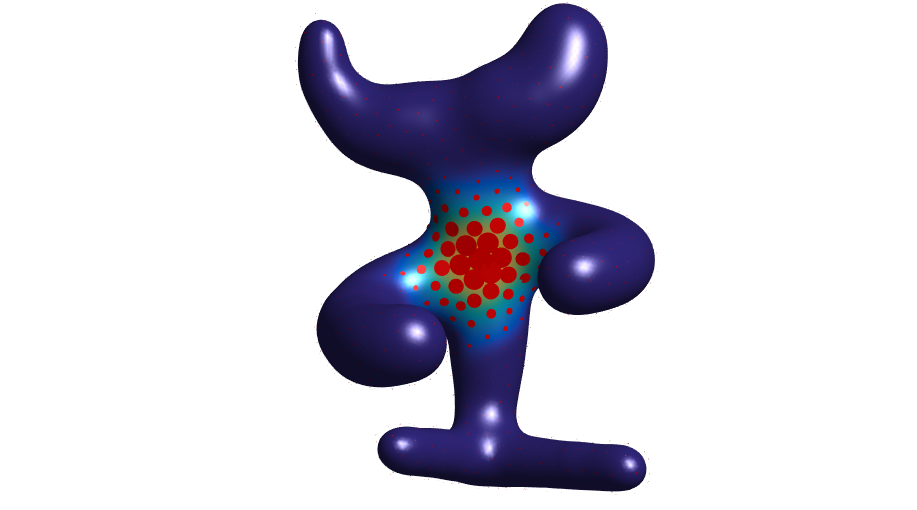}\\
$t=0$ & $t=1/8$ & $t=1/4$ & $t=3/8$ & $t=1/2$ & $t=5/8$ & $t=3/4$ & $t=7/8$ & $t=1$ 
\end{tabular}
\caption{
Interpolations between two input fields of positive semidefinite matrices (displayed at times $t \in \{0,1\}$ using ellipses) on some domain (here, a 2-D planar square and a surface mesh), using the method detailed in~\citet{2016-peyre-qot}. 
Unlike linear interpolation schemes, this OT-like method transports the ``mass'' of the tensors (size of the ellipses) as well as their anisotropy and orientation.
} \label{fig-tensors}
\end{figure}

\section{Gromov--Wasserstein Distances}

For some applications such as shape matching, an important weakness of optimal transport distances lies in the fact that they are not invariant to important families of invariances, such as rescaling, translation or rotations. Although some nonconvex variants of OT to handle such global transformations have been proposed~\citep{cohen1999earth,pele2013tangent} and recently applied to problems such as cross-lingual word embeddings alignments~\citep{grave2018unsupervised,alvarez2018towards,grave2018unsupervised}, these methods require specifying first a subset of invariances, possibly between different metric spaces, to be relevant. We describe in this section a more general and very natural extension of OT that can deal with measures defined on different spaces without requiring the definition of a family of invariances.

\subsection{Hausdorff Distance}
\label{sec-hausdorff}

The Hausdorff distance between two sets $A,B \subset \Z$ for some metric $\dist_\Z$ is 
\eq{
	\Hh_\Zz(A,B) \eqdef \max\Big(
  		\sup_{a \in A} \inf_{b \in B} d_{\Zz}(a,b), 
  		\sup_{b \in B} \inf_{a \in A} d_{\Zz}(a,b)
	\Big), 
}
see Figure~\ref{fig-hausdorff}.
This defines a distance between compact sets $\Kk(\Z)$ of $\Z$, and if $\Z$ is compact, then $(\Kk(\Z),\Hh_\Zz)$ is itself compact; see~\citep{burago2001course}.

\begin{figure}[h!]
\centering
\includegraphics[width=.26\linewidth]{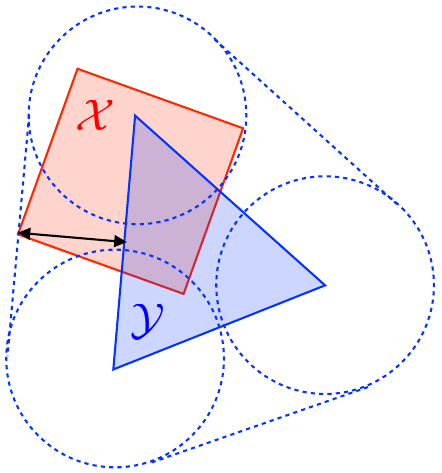}
\qquad
\includegraphics[width=.26\linewidth]{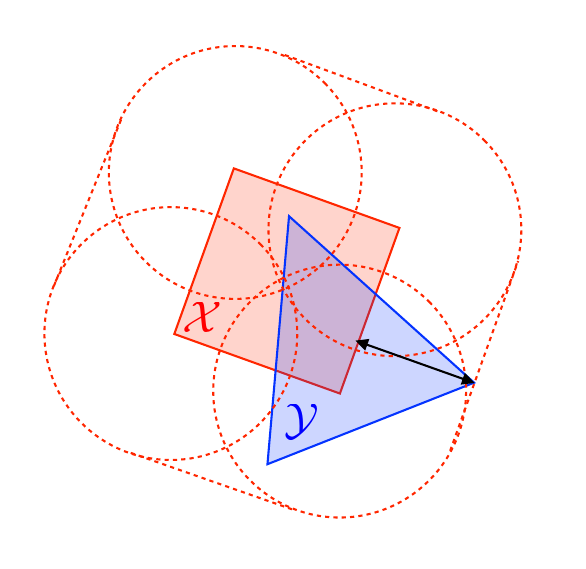}
\caption{\label{fig-hausdorff}
Computation of the Hausdorff distance in $\RR^2$. 
}
\end{figure}

Following~\citet{memoli-2011}, one remarks that this distance between sets $(A,B)$ can be defined similarly to the Wasserstein distance between measures (which should be somehow understood as ``weighted'' sets). One replaces the measures couplings~\eqref{eq-coupling-generic} by sets couplings
\eq{
	\Rr(A,B) \eqdef \enscond{ R \in \X \times \Y }{
	\begin{array}{l}
		\foralls a \in A, \exists b \in B, (a,b) \in R \\
		\foralls b \in B, \exists a \in A, (a,b) \in R
	\end{array}
	}.
}
With respect to Kantorovich problem~\eqref{eq-mk-generic}, one should replace integration (since one does not have access to measures) by maximization, and one has
\eql{\label{eq-hausf-couplings}
	\Hh_\Zz(A,B) = \uinf{ R \in \Rr(A,B) } \sup_{(a,b) \in R} d(a,b). 
}
Note that the support of a measure coupling $\pi \in \Couplings(\al,\be)$ is a set coupling between the supports, \ie $\Supp(\pi) \in \Rr( \Supp(\al),\Supp(\be) )$.
The Hausdorff distance is thus connected to the $\infty$-Wasserstein distance (see Remark~\ref{rem-p-inf}) and one has $\Hh(A,B) \leq \Wass_\infty(\al,\be)$ for any measure $(\al,\be)$ whose supports are $(A,B)$.

\subsection{Gromov--Hausdorff distance}

The Gromov--Hausdorff (GH) distance~\citep{gromov-2001} (see also~\citep{edwards1975structure}) is a way to measure the distance between two metric spaces $(\X,d_\X), (\Y,d_\Y)$ by quantifying how far they are from being isometric to each other, see Figure~\ref{fig-gh}. It is defined as the minimum Hausdorff distance between every possible isometric embedding of the two spaces in a third one,
\eq{
	\Gg\Hh(d_\Xx,d_\Yy) \eqdef 
	\inf_{\Zz,f,g} \enscond{
		\Hh_\Zz(f(\Xx), g(\Yy))
	}{
		\begin{array}{l}
		f : \Xx \overset{\text{isom}}{\longrightarrow} \Zz\\
		g : \Yy \overset{\text{isom}}{\longrightarrow} \Zz
		\end{array}
	}.
}
Here, the constraint is that $f$ must be an isometric embedding, meaning that $d_\Zz(f(x),f(x'))=d_\Xx(x,x')$ for any $(x,x') \in \X^2$ (similarly for $g$). One can show that $\Gg\Hh$ defines a distance between compact metric spaces up to isometries, so that in particular $\Gg\Hh(d_\Xx,d_\Yy)=0$ if and only if there exists an isometry $h : \Xx \rightarrow \Yy$, \ie $h$ is bijective and $d_\Yy(h(x),h(x'))=d_\Xx(x,x')$ for any $(x,x') \in \Xx^2$. 

\begin{figure}[h!]
\centering
\includegraphics[width=.45\linewidth]{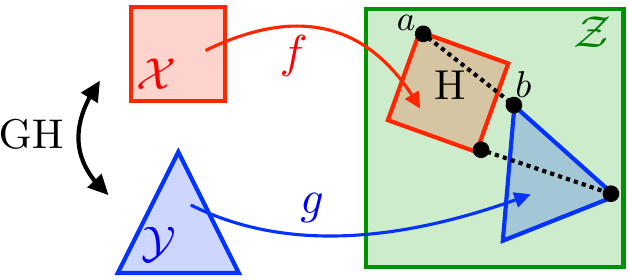}
\caption{\label{fig-gh}
The GH approach to compare two metric spaces.
}
\end{figure}

Similarly to~\eqref{eq-hausf-couplings} and as explained in~\citep{memoli-2011}, it is possible to rewrite equivalently the GH distance using couplings as follows:
\eq{
	\Gg\Hh(d_\Xx,d_\Yy) = \frac{1}{2} \uinf{ R \in \Rr(\X,\Y) } \sup_{ ((x,y),(x',y')) \in R^2 }
	|d_\Xx(x,x')-d_\Xx(y,y')|. 
}
For discrete spaces $\Xx = (x_i)_{i=1}^n, \Yy = (y_j)_{j=1}^m$ represented using a distance matrix $\distD = ( d_{\Xx}(x_i,x_{i'}) )_{i,i'} \in \RR^{n \times n}$,  $\distD' = ( d_{\Yy}(y_j,y_{j'}) )_{j,j'} \in \RR^{m \times m}$, one can rewrite this optimization using binary matrices $\VectMode{R} \in \{0,1\}^{n \times m}$ indicating the support of the set couplings $R$ as follows:
\eql{\label{eq-gh-coupling}
	\text{GH}(\distD,\distD') = \frac{1}{2} \uinf{ \VectMode{R} \ones >0,\VectMode{R}^\top \ones>0  } \max_{ (i,i',j,j') }
	\VectMode{R}_{i,j} \VectMode{R}_{j,j'} |\distD_{i,i'}-\distD_{j,j'}'|. 
}
The initial motivation of the GH distance is to define and study limits of metric spaces, as illustrated in Figure~\ref{fig-gh-limit}, and we refer to~\citep{burago2001course} for details. 
There is an explicit description of the geodesics for the GH distance~\citep{chowdhury2016constructing}, which is very similar to the one in Gromov--Wasserstein spaces, detailed in Remark~\ref{rem-geod-gw}. 

\begin{figure}[h!]
\centering
\includegraphics[width=1\linewidth]{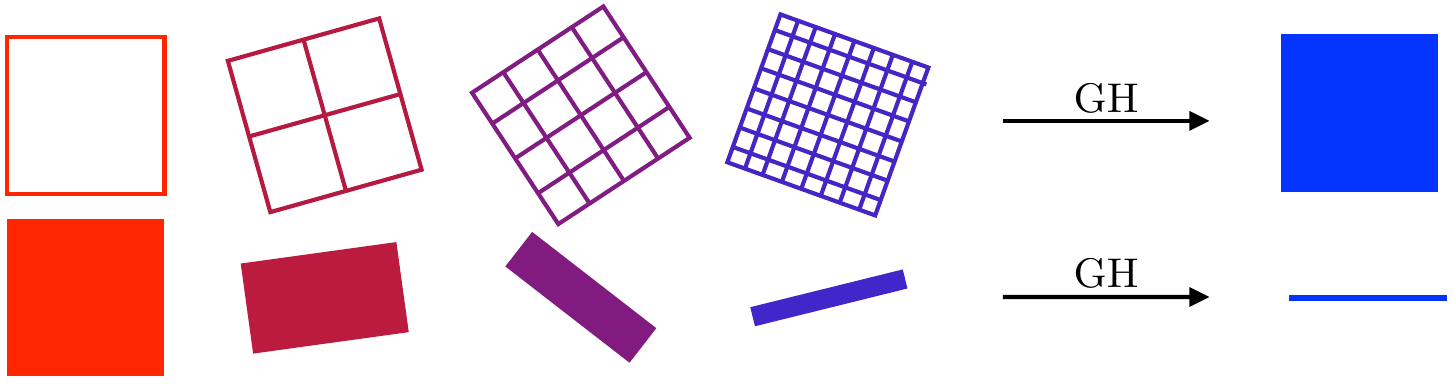}
\caption{\label{fig-gh-limit}
GH limit of sequences of metric spaces.  
}
\end{figure}

The underlying optimization problem~\eqref{eq-gh-coupling} is highly nonconvex, and computing the global minimum is untractable. It has been approached numerically using approximation schemes and has found applications in vision and graphics for shape matching~\citep{memoli2005theoretical,bronstein2006generalized}.

It is often desirable to ``smooth'' the definition of the Hausdorff distance by replacing the maximization by an integration. This in turn necessitates the introduction of measures, and it is one of the motivations for the definition of the GW distance in the next section.
 
\subsection{Gromov--Wasserstein Distance}

Optimal transport needs a ground cost $\C$ to compare histograms $(\a,\b)$ and thus cannot be used if the bins of those histograms are not defined on the same underlying space, or if one cannot preregister these spaces to define a ground cost between any pair of bins in the first and second histograms, respectively. 
To address this limitation, one can instead only assume a weaker assumption, namely that two matrices $\distD \in \RR^{n \times n}$ and $\distD' \in \RR^{m \times m}$ quantify similarity relationships between the points on which the histograms are defined. A typical scenario is when these matrices are (power of) distance matrices.
The GW problem reads
\eql{\label{eq-gw-def}
	\GWD( (\a,\distD), (\b,\distD') )^2 \eqdef \umin{ \P \in \CouplingsD(\a,\b) } 
		\Ee_{\distD,\distD'}(\P)
}
\eq{
	\qwhereq
	\Ee_{\distD,\distD'}(\P) \eqdef 
		\sum_{i,j,i',j'} |\distD_{i,i'} - \distD'_{j,j'}|^2 \P_{i,j}\P_{i',j'},  
}
see Figure~\ref{fig-gw}.
This problem is similar to the GH problem~\eqref{eq-gh-coupling} when replacing maximization by a sum and set couplings by measure couplings.
This is a nonconvex problem, which can be recast as a quadratic assignment problem~\citep{loiola-2007} and is in full generality NP-hard to solve for arbitrary inputs. 
It is in fact equivalent to a graph matching problem~\citep{lyzinski-2015} for a particular cost.\todoK{detail QAP/graph matching}

\begin{figure}[h!]
\centering
\includegraphics[width=.7\linewidth]{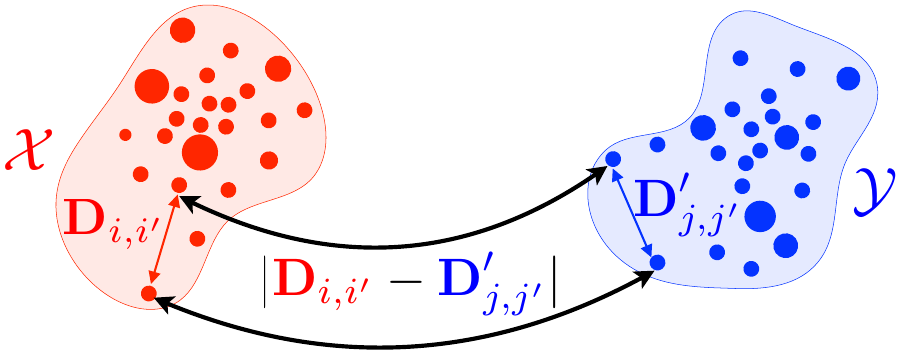}
\caption{\label{fig-gw}
The GW approach to comparing two metric measure spaces.  
}
\end{figure}

One can show that $\GWD$ satisfies the triangular inequality, and in fact it defines a distance between metric spaces equipped with a probability distribution, here assumed to be discrete in definition~\eqref{eq-gw-def}, up to isometries preserving the measures.
This distance was introduced and studied in detail by~\citet{memoli-2011}. An in-depth mathematical exposition (in particular, its geodesic structure and gradient flows) is given in~\citep{SturmGW}. See also~\citep{schmitzer2013modelling} for applications in computer vision.
This distance is also tightly connected with the GH distance~\citep{gromov-2001} between metric spaces, which have been used for shape matching~\citep{memoli-2007,bronstein-2010}.

\begin{rem2}{Gromov--Wasserstein distance}
	The general setting corresponds to computing couplings between metric measure spaces $(\X,\dist_\X,\al_\X)$
	and $(\Y,\dist_\Y,\al_\Y)$, where $(\dist_\X,\dist_\Y)$ are distances, while $\al_\X$ and $\al_\Y$ are measures on their respective spaces.
	One defines 
	\begin{equation}
		\label{eq-gw-generic}
		\begin{split}
		&\GW( (\al_\X,\dist_\X), (\al_\Y,\dist_\Y) )^2 \eqdef \\
		&\umin{ \pi \in \Couplings(\al_\X,\al_\Y) } 
		\int_{\X^2 \times \Y^2}
		| \dist_\X(x,x')-\dist_\Y(y,y') |^2
		\d\pi(x,y)\d\pi(x',y').
		\end{split}
	\end{equation}
	$\GW$ defines a distance between metric measure spaces up to isometries, where one says that $(\X,\al_\X,\dist_\X)$ and $(\Y,\al_\Y,\dist_\Y) $ are isometric if there exists a bijection $\phi : \X \rightarrow \Y$ such that $\phi_{\sharp}\al_\X=\al_\Y$ and $\dist_\Y(\phi(x),\phi(x'))=\dist_\X(x,x')$.
\end{rem2}

\begin{rem2}{Gromov--Wasserstein geodesics}\label{rem-geod-gw}
The space of metric spaces (up to isometries) endowed with this $\GW$ distance~\eqref{eq-gw-generic} has a geodesic structure. \citet{SturmGW} shows that the geodesic between  $(\X_0,\dist_{\X_0},\al_0)$ and $(\X_1,\dist_{\X_1},\al_1)$ can be chosen to be 
$t \in [0,1] \mapsto (\X_0 \times \X_1,\dist_t,\pi^\star),$ where $\pi^\star$ is a solution of~\eqref{eq-gw-generic} and for all $((x_0,x_1), (x_0',x_1')) \in (\X_0 \times \X_1)^2$, 
\eq{
	\dist_t((x_0,x_1), (x_0',x_1')) \eqdef
	(1-t)\dist_{\X_0}(x_0,x_0') + t\dist_{\X_1}(x_1,x_1').
}
This formula allows one to define and analyze gradient flows which minimize functionals involving metric spaces; see~\citet{SturmGW}. It is, however, difficult to handle numerically, because it involves computations over the product space $\X_0 \times \X_1$. 
A heuristic approach is used in~\citep{peyre2016gromov} to define geodesics and barycenters of metric measure spaces while imposing the cardinality of the involved spaces and making use of the entropic smoothing~\eqref{eq-gw-entropy} detailed below.
\end{rem2}

\subsection{Entropic Regularization}

To approximate the computation of $\GWD$, and to help convergence of minimization schemes to better minima, one can consider the entropic regularized variant
\eql{\label{eq-gw-entropy}
	\umin{ \P \in \CouplingsD(\a,\b) } 
		\Ee_{\distD,\distD'}(\P) - \varepsilon \HD(\P).
}
As proposed initially in~\citep{gold-1996,rangarajan-1999}, and later revisited in~\citep{2016-solomon-gw} for applications in graphics, one can use iteratively Sinkhorn's algorithm to progressively compute a stationary point of~\eqref{eq-gw-entropy}. 
Indeed, successive linearizations of the objective function lead to consider the succession of updates
\eql{\label{eq-gw-sinkh}
	\itt{\P} \eqdef \umin{ \P \in \CouplingsD(\a,\b) } \dotp{\P}{\it{\C}} - \varepsilon\H(\P)
		\qwhereq
}
\eq{
		\it{\C} \eqdef \nabla \Ee_{\distD,\distD'}(\it{\P}) = -\distD \it{\P} \distD', 
}
which can be interpreted as a mirror-descent scheme~\citep{2016-solomon-gw}. Each update can thus be solved using Sinkhorn iterations~\eqref{eq-sinkhorn} with cost $\it{\C}$.
Figure~\ref{fig-gw-iter} displays the evolution of the algorithm. 
Figure~\ref{fig-gw-sinkhorn} illustrates the use of this entropic GW to compute soft maps between domains.

\newcommand{\FigGW}[1]{\includegraphics[width=.21\linewidth]{gromov-wasserstein/#1}}

\begin{figure}[h!]
\centering
\begin{tabular}{@{}c@{\hspace{2mm}}c@{\hspace{2mm}}c@{\hspace{2mm}}c@{}}
\imgBox{\FigGW{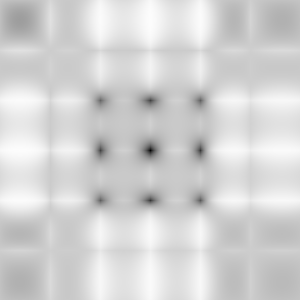}}&
\imgBox{\FigGW{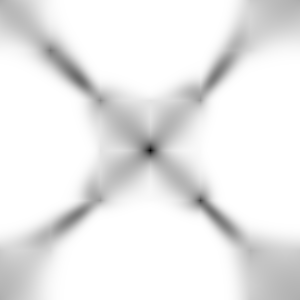}}&
\imgBox{\FigGW{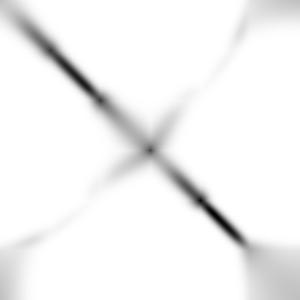}}&
\imgBox{\FigGW{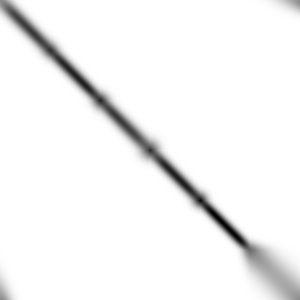}}\\
\FigGW{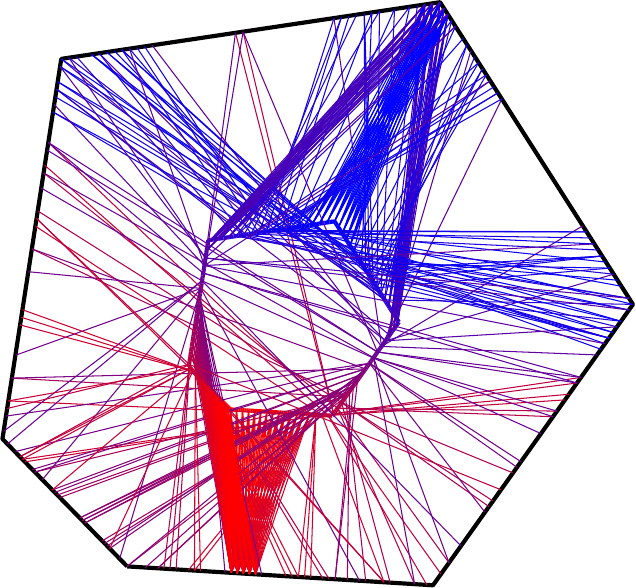}&\FigGW{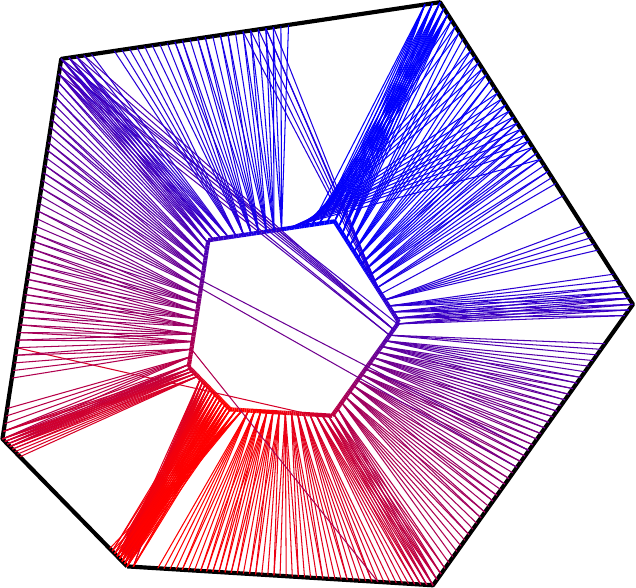}&\FigGW{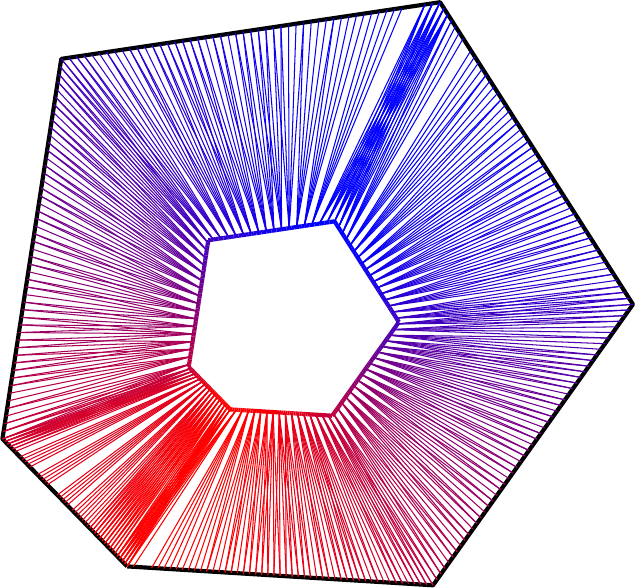}&\FigGW{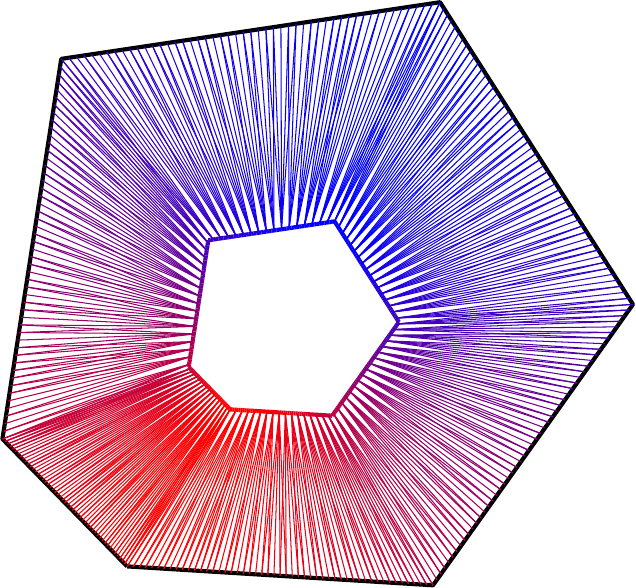}\\
$\ell=1$ & $\ell=2$ & $\ell=3$ & $\ell=4$ 
\end{tabular}
\caption{\label{fig-gw-iter}
Iterations of the entropic GW algorithm~\eqref{eq-gw-sinkh} between two shapes $(x_i)_i$ and $(y_j)_j$ in $\RR^2$, initialized with $\P^{(0)}=\a \otimes \b$. The distance matrices are $\distD_{i,i'} = \norm{x_i-x_{i'}}$ and $\distD_{j,j'}' = \norm{y_j-y_{j'}}$. 
Top row: coupling $\it{\P}$ displayed as a 2-D image.
Bottom row: matching induced by $\it{\P}$ (each point $x_i$ is connected to the three $y_j$ with the three largest values among $\{\it{\P}_{i,j}\}_j$). The shapes have the same size, but for display purposes, the inner shape $(x_i)_i$ has been reduced.
}
\end{figure}

\begin{figure}[h!]
\centering
\includegraphics[height=.27\linewidth]{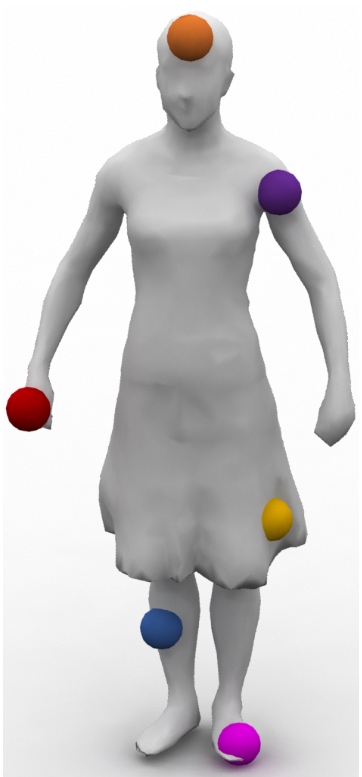}
\includegraphics[height=.27\linewidth]{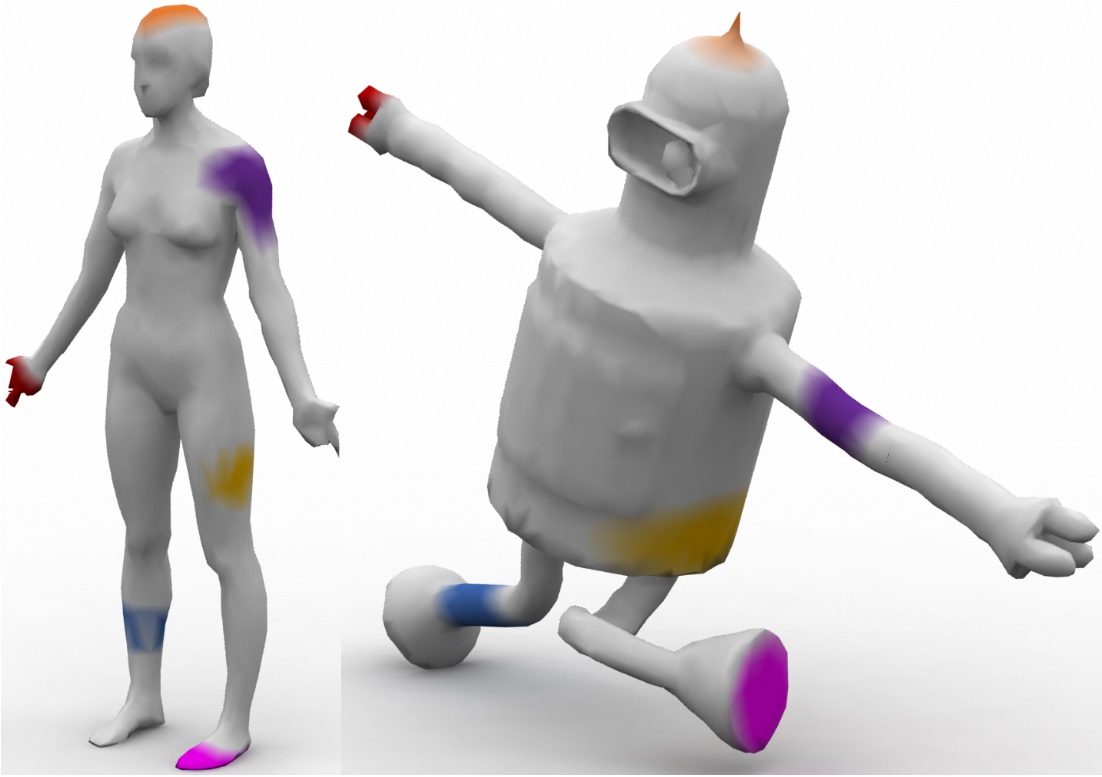}
\includegraphics[height=.27\linewidth]{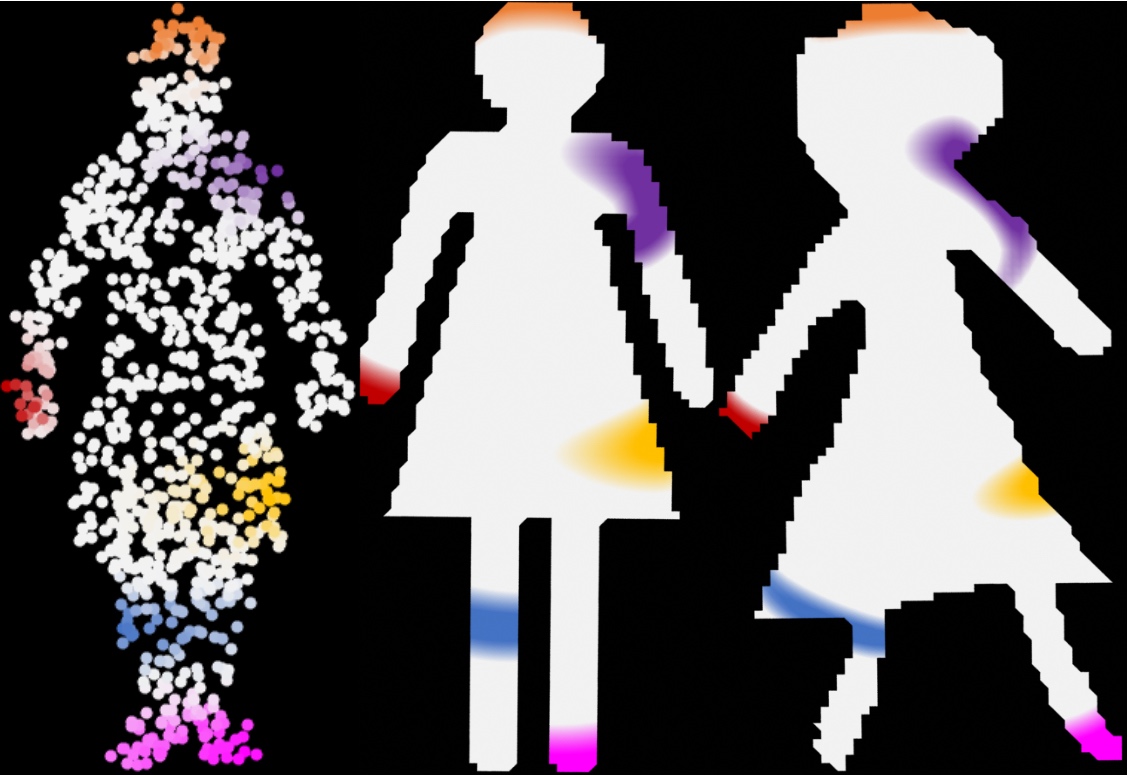}
\caption{\label{fig-gw-sinkhorn}
Example of fuzzy correspondences computed by solving GW problem~\eqref{eq-gw-entropy} with Sinkhorn iterations~\eqref{eq-gw-sinkh}. Extracted from~\citep{2016-solomon-gw}.
}
\end{figure}

\chapter*{Acknowledgements}
We would like to thank the many colleagues, collaborators and students who have helped us at various stages when preparing this survey. Some of their inputs have shaped this work, and we would like to thank in particular Jean-David Benamou, Yann Brenier, Guillaume Carlier, Vincent Duval and the entire MOKAPLAN team at Inria; Francis Bach, Espen Bernton, Mathieu Blondel, Nicolas Courty, R\'emi Flamary, Alexandre Gramfort, Young-Heon Kim, Daniel Matthes, Philippe Rigollet, Filippo Santambrogio, Justin Solomon, Jonathan Weed; as well as the feedback by our current and former students on these subjects, in particular Gwendoline de Bie, L\'ena\"ic Chizat, Aude Genevay, Hicham Janati, Th\'eo Lacombe, Boris Muzellec, Fran\c ois-Pierre Paty, Vivien Seguy.

\bibliographystyle{plainnat}
\bibliography{./biblio/all}

\begin{thebibliography}{440}
\providecommand{\natexlab}[1]{#1}
\providecommand{\url}[1]{\texttt{#1}}
\expandafter\ifx\csname urlstyle\endcsname\relax
  \providecommand{\doi}[1]{doi: #1}\else
  \providecommand{\doi}{doi: \begingroup \urlstyle{rm}\Url}\fi

\bibitem[Abadi et~al.(2016)Abadi, Agarwal, Barham, Brevdo, Chen, Citro,
  Corrado, Davis, Dean, Devin, et~al.]{abadi2016tensorflow}
Mart{\'\i}n Abadi, Ashish Agarwal, Paul Barham, Eugene Brevdo, Zhifeng Chen,
  Craig Citro, Greg~S Corrado, Andy Davis, Jeffrey Dean, Matthieu Devin, et~al.
\newblock Tensorflow: large-scale machine learning on heterogeneous distributed
  systems.
\newblock \emph{arXiv preprint arXiv:1603.04467}, 2016.

\bibitem[Abraham et~al.(2017)Abraham, Abraham, Bergounioux, and
  Carlier]{AbrahamRadon}
Isabelle Abraham, Romain Abraham, Ma{\i}tine Bergounioux, and Guillaume
  Carlier.
\newblock Tomographic reconstruction from a few views: a multi-marginal optimal
  transport approach.
\newblock \emph{Applied Mathematics \& Optimization}, 75\penalty0 (1):\penalty0
  55--73, 2017.

\bibitem[Adams and Zemel(2011)]{adams2011ranking}
Ryan~Prescott Adams and Richard~S Zemel.
\newblock Ranking via sinkhorn propagation.
\newblock \emph{arXiv preprint arXiv:1106.1925}, 2011.

\bibitem[Agueh and Bowles(2013)]{agueh2013one}
Martial Agueh and Malcolm Bowles.
\newblock One-dimensional numerical algorithms for gradient flows in the
  $p$-{Wasserstein} spaces.
\newblock \emph{Acta Applicandae Mathematicae}, 125\penalty0 (1):\penalty0
  121--134, 2013.

\bibitem[Agueh and Carlier(2011)]{Carlier_wasserstein_barycenter}
Martial Agueh and Guillaume Carlier.
\newblock Barycenters in the {W}asserstein space.
\newblock \emph{SIAM Journal on Mathematical Analysis}, 43\penalty0
  (2):\penalty0 904--924, 2011.

\bibitem[Agueh and Carlier(2017)]{agueh2017vers}
Martial Agueh and Guillaume Carlier.
\newblock Vers un th{\'e}or{\`e}me de la limite centrale dans l'espace de
  {Wasserstein}?
\newblock \emph{Comptes Rendus Mathematique}, 355\penalty0 (7):\penalty0
  812--818, 2017.

\bibitem[Al{-}Rfou et~al.(2016)Al{-}Rfou, Alain, Almahairi, Angerm{\"{u}}ller,
  Bahdanau, and et~al.]{theano2016}
Rami Al{-}Rfou, Guillaume Alain, Amjad Almahairi, Christof Angerm{\"{u}}ller,
  Dzmitry Bahdanau, and Nicolas~Ballas et~al.
\newblock Theano: {A} python framework for fast computation of mathematical
  expressions.
\newblock \emph{CoRR}, abs/1605.02688, 2016.

\bibitem[Ali and Silvey(1966)]{ali1966general}
Syed~Mumtaz Ali and Samuel~D Silvey.
\newblock A general class of coefficients of divergence of one distribution
  from another.
\newblock \emph{Journal of the Royal Statistical Society. Series B
  (Methodological)}, 28\penalty0 (1):\penalty0 131--142, 1966.

\bibitem[Allen-Zhu et~al.(2017)Allen-Zhu, Li, Oliveira, and
  Wigderson]{allen2017much}
Zeyuan Allen-Zhu, Yuanzhi Li, Rafael Oliveira, and Avi Wigderson.
\newblock Much faster algorithms for matrix scaling.
\newblock \emph{arXiv preprint arXiv:1704.02315}, 2017.

\bibitem[Altschuler et~al.(2017)Altschuler, Weed, and
  Rigollet]{altschuler2017near}
Jason Altschuler, Jonathan Weed, and Philippe Rigollet.
\newblock Near-linear time approximation algorithms for optimal transport via
  {Sinkhorn} iteration.
\newblock \emph{arXiv preprint arXiv:1705.09634}, 2017.

\bibitem[{\'A}lvarez-Esteban et~al.(2016){\'A}lvarez-Esteban, del Barrio,
  Cuesta-Albertos, and Matr{\'a}n]{alvarez2016fixed}
Pedro~C {\'A}lvarez-Esteban, E~del Barrio, JA~Cuesta-Albertos, and
  C~Matr{\'a}n.
\newblock A fixed-point approach to barycenters in {Wasserstein} space.
\newblock \emph{Journal of Mathematical Analysis and Applications},
  441\penalty0 (2):\penalty0 744--762, 2016.

\bibitem[Alvarez-Melis et~al.(2019)Alvarez-Melis, Jegelka, and
  Jaakkola]{alvarez2018towards}
David Alvarez-Melis, Stefanie Jegelka, and Tommi~S Jaakkola.
\newblock Towards optimal transport with global invariances.
\newblock 2019.

\bibitem[Amari et~al.(2018)Amari, Karakida, and Oizumi]{amari2017information}
Shun-ichi Amari, Ryo Karakida, and Masafumi Oizumi.
\newblock Information geometry connecting {Wasserstein} distance and
  {Kullback-Leibler} divergence via the entropy-relaxed transportation problem.
\newblock \emph{Information Geometry}, \penalty0 (1):\penalty0 13--37, 2018.

\bibitem[Ambrosio et~al.(2006)Ambrosio, Gigli, and
  Savar{\'e}]{ambrosio2006gradient}
L.~Ambrosio, N.~Gigli, and G.~Savar{\'e}.
\newblock \emph{Gradient Flows in Metric Spaces and in the Space of Probability
  Measures}.
\newblock Springer, 2006.

\bibitem[Ambrosio and Figalli(2009)]{AmbrosioFigalliEuler}
Luigi Ambrosio and Alessio Figalli.
\newblock Geodesics in the space of measure-preserving maps and plans.
\newblock \emph{Archive for Rational Mechanics and Analysis}, 194\penalty0
  (2):\penalty0 421--462, 2009.

\bibitem[Anderes et~al.(2016)Anderes, Borgwardt, and
  Miller]{anderes2016discrete}
Ethan Anderes, Steffen Borgwardt, and Jacob Miller.
\newblock Discrete {Wasserstein} barycenters: optimal transport for discrete
  data.
\newblock \emph{Mathematical Methods of Operations Research}, 84\penalty0
  (2):\penalty0 389--409, 2016.

\bibitem[Andoni et~al.(2008)Andoni, Indyk, and Krauthgamer]{andoni2008earth}
Alexandr Andoni, Piotr Indyk, and Robert Krauthgamer.
\newblock Earth mover distance over high-dimensional spaces.
\newblock In \emph{Proceedings of the nineteenth annual ACM-SIAM Symposium on
  Discrete Algorithms}, pages 343--352. Society for Industrial and Applied
  Mathematics, 2008.

\bibitem[Andoni et~al.(2018)Andoni, Naor, and Neiman]{andoni2015snowflake}
Alexandr Andoni, Assaf Naor, and Ofer Neiman.
\newblock Snowflake universality of {Wasserstein} spaces.
\newblock \emph{Annales scientifiques de l'\'Ecole normale sup\'erieure},
  51:\penalty0 657--700, 2018.

\bibitem[Arjovsky et~al.(2017)Arjovsky, Chintala, and Bottou]{WassersteinGAN}
Martin Arjovsky, Soumith Chintala, and L{\'e}on Bottou.
\newblock {W}asserstein generative adversarial networks.
\newblock \emph{Proceedings of the 34th International Conference on Machine
  Learning}, 70:\penalty0 214--223, 2017.

\bibitem[Aurenhammer(1987)]{aurenhammer1987power}
Franz Aurenhammer.
\newblock Power diagrams: properties, algorithms and applications.
\newblock \emph{SIAM Journal on Computing}, 16\penalty0 (1):\penalty0 78--96,
  1987.

\bibitem[Aurenhammer et~al.(1998)Aurenhammer, Hoffmann, and
  Aronov]{AurenhammerHA98}
Franz Aurenhammer, Friedrich Hoffmann, and Boris Aronov.
\newblock Minkowski-type theorems and least-squares clustering.
\newblock \emph{Algorithmica}, 20\penalty0 (1):\penalty0 61--76, 1998.

\bibitem[Averbuch et~al.(2001)Averbuch, Coifman, Donoho, Israeli, and
  Walden]{FastSlantStack}
Amir Averbuch, Ronald Coifman, David Donoho, Moshe Israeli, and Johan Walden.
\newblock Fast slant stack: a notion of radon transform for data in a cartesian
  grid which is rapidly computible, algebraically exact, geometrically faithful
  and invertible.
\newblock \emph{Tech. Rep., Stanford University}, 2001.

\bibitem[Bach(2010)]{bach2010self}
Francis Bach.
\newblock Self-concordant analysis for logistic regression.
\newblock \emph{Electronic Journal of Statistics}, 4:\penalty0 384--414, 2010.

\bibitem[Bach(2014)]{bach2014adaptivity}
Francis~R Bach.
\newblock Adaptivity of averaged stochastic gradient descent to local strong
  convexity for logistic regression.
\newblock \emph{Journal of Machine Learning Research}, 15\penalty0
  (1):\penalty0 595--627, 2014.

\bibitem[Bacharach(1965)]{bacharach1965estimating}
Michael Bacharach.
\newblock Estimating nonnegative matrices from marginal data.
\newblock \emph{International Economic Review}, 6\penalty0 (3):\penalty0
  294--310, 1965.

\bibitem[Barvinok(2002)]{barvinok2002course}
Alexander Barvinok.
\newblock \emph{A Course in Convexity}.
\newblock Graduate Studies in Mathematics. American Mathematical Society, 2002.
\newblock ISBN 9780821829684.

\bibitem[Bassetti et~al.(2006)Bassetti, Bodini, and
  Regazzini]{bassetti2006minimum}
Federico Bassetti, Antonella Bodini, and Eugenio Regazzini.
\newblock On minimum kantorovich distance estimators.
\newblock \emph{Statistics \& Probability Letters}, 76\penalty0 (12):\penalty0
  1298--1302, 2006.

\bibitem[Bauschke and Combettes(2011)]{BauschkeCombettes11}
Heinz~H Bauschke and Patrick~L Combettes.
\newblock \emph{Convex analysis and monotone operator theory in {H}ilbert
  spaces}.
\newblock Springer-Verlag, New York, 2011.

\bibitem[Bauschke and Lewis(2000)]{bauschke-lewis}
Heinz~H Bauschke and Adrian~S Lewis.
\newblock Dykstra's algorithm with {B}regman projections: a convergence proof.
\newblock \emph{Optimization}, 48\penalty0 (4):\penalty0 409--427, 2000.

\bibitem[Beck and Teboulle(2003)]{beck2003mirror}
Amir Beck and Marc Teboulle.
\newblock Mirror descent and nonlinear projected subgradient methods for convex
  optimization.
\newblock \emph{Operations Research Letters}, 31\penalty0 (3):\penalty0
  167--175, 2003.

\bibitem[Beckmann(1952)]{Beckmann52}
Martin Beckmann.
\newblock A continuous model of transportation.
\newblock \emph{Econometrica}, 20:\penalty0 643--660, 1952.

\bibitem[Beiglb{\"o}ck et~al.(2013)Beiglb{\"o}ck, Henry-Labord{\`e}re, and
  Penkner]{beiglbock2013model}
Mathias Beiglb{\"o}ck, Pierre Henry-Labord{\`e}re, and Friedrich Penkner.
\newblock Model-independent bounds for option prices: a mass transport
  approach.
\newblock \emph{Finance and Stochastics}, 17\penalty0 (3):\penalty0 477--501,
  2013.

\bibitem[Beirlant et~al.(1997)Beirlant, Dudewicz, Gyorfi, and Van~der
  Meulen]{beirlant1997nonparametric}
Jan Beirlant, Edward~J Dudewicz, Laszlo Gyorfi, and Edward~C Van~der Meulen.
\newblock Nonparametric entropy estimation: an overview.
\newblock \emph{International Journal of Mathematical and Statistical
  Sciences}, 6\penalty0 (1):\penalty0 17--39, 1997.

\bibitem[Benamou(2003)]{benamou2003numerical}
Jean-David Benamou.
\newblock Numerical resolution of an ``unbalanced'' mass transport problem.
\newblock \emph{ESAIM: Mathematical Modelling and Numerical Analysis},
  37\penalty0 (05):\penalty0 851--868, 2003.

\bibitem[Benamou and Brenier(2000)]{benamou2000computational}
Jean-David Benamou and Yann Brenier.
\newblock A computational fluid mechanics solution to the {M}onge-{K}antorovich
  mass transfer problem.
\newblock \emph{Numerische Mathematik}, 84\penalty0 (3):\penalty0 375--393,
  2000.

\bibitem[Benamou and Carlier(2015)]{benamou2015augmented}
Jean-David Benamou and Guillaume Carlier.
\newblock Augmented lagrangian methods for transport optimization, mean field
  games and degenerate elliptic equations.
\newblock \emph{Journal of Optimization Theory and Applications}, 167\penalty0
  (1):\penalty0 1--26, 2015.

\bibitem[Benamou et~al.(2014)Benamou, Froese, and
  Oberman]{benamou2014numerical}
Jean-David Benamou, Brittany~D Froese, and Adam~M Oberman.
\newblock Numerical solution of the optimal transportation problem using the
  {Monge--Ampere} equation.
\newblock \emph{Journal of Computational Physics}, 260:\penalty0 107--126,
  2014.

\bibitem[Benamou et~al.(2015)Benamou, Carlier, Cuturi, Nenna, and
  Peyr{\'e}]{2015-benamou-cisc}
Jean-David Benamou, Guillaume Carlier, Marco Cuturi, Luca Nenna, and Gabriel
  Peyr{\'e}.
\newblock Iterative {Bregman} projections for regularized transportation
  problems.
\newblock \emph{SIAM Journal on Scientific Computing}, 37\penalty0
  (2):\penalty0 A1111--A1138, 2015.

\bibitem[Benamou et~al.(2016{\natexlab{a}})Benamou, Carlier, M{\'e}rigot, and
  Oudet]{JDB-JKO}
Jean-David Benamou, Guillaume Carlier, Quentin M{\'e}rigot, and Edouard Oudet.
\newblock Discretization of functionals involving the {Monge--Amp{\`e}re}
  operator.
\newblock \emph{Numerische Mathematik}, 134\penalty0 (3):\penalty0 611--636,
  2016{\natexlab{a}}.

\bibitem[Benamou et~al.(2016{\natexlab{b}})Benamou, Collino, and
  Mirebeau]{benamou2016monotone}
Jean-David Benamou, Francis Collino, and Jean-Marie Mirebeau.
\newblock Monotone and consistent discretization of the {Monge-Ampere}
  operator.
\newblock \emph{Mathematics of Computation}, 85\penalty0 (302):\penalty0
  2743--2775, 2016{\natexlab{b}}.

\bibitem[Berg et~al.(1984)Berg, Christensen, and Ressel]{berg84harmonic}
Christian Berg, Jens Peter~Reus Christensen, and Paul Ressel.
\newblock \emph{Harmonic Analysis on Semigroups}.
\newblock Number 100 in Graduate Texts in Mathematics. Springer Verlag, 1984.

\bibitem[Berlinet and Thomas-Agnan(2003)]{berlinet03reproducing}
Alain Berlinet and Christine Thomas-Agnan.
\newblock \emph{Reproducing Kernel Hilbert Spaces in Probability and
  Statistics}.
\newblock Kluwer Academic Publishers, 2003.

\bibitem[Bernton(2018)]{pmlr-v75-bernton18a}
Espen Bernton.
\newblock {L}angevin {M}onte {C}arlo and {JKO} splitting.
\newblock In S\'ebastien Bubeck, Vianney Perchet, and Philippe Rigollet,
  editors, \emph{Proceedings of the 31st Conference On Learning Theory},
  volume~75 of \emph{Proceedings of Machine Learning Research}, pages
  1777--1798. PMLR, 2018.

\bibitem[Bernton et~al.(2017)Bernton, Jacob, Gerber, and
  Robert]{bernton2017inference}
Espen Bernton, Pierre~E Jacob, Mathieu Gerber, and Christian~P Robert.
\newblock Inference in generative models using the {Wasserstein} distance.
\newblock \emph{arXiv preprint arXiv:1701.05146}, 2017.

\bibitem[Bertsekas(1981)]{bertsekas1981new}
Dimitri~P Bertsekas.
\newblock A new algorithm for the assignment problem.
\newblock \emph{Mathematical Programming}, 21\penalty0 (1):\penalty0 152--171,
  1981.

\bibitem[Bertsekas(1992)]{bertsekas1992auction}
Dimitri~P Bertsekas.
\newblock Auction algorithms for network flow problems: a tutorial
  introduction.
\newblock \emph{Computational Optimization and Applications}, 1\penalty0
  (1):\penalty0 7--66, 1992.

\bibitem[Bertsekas(1998)]{bertsekas1998network}
Dimitri~P Bertsekas.
\newblock \emph{Network Optimization: Continuous and Discrete Models}.
\newblock Athena Scientific, 1998.

\bibitem[Bertsekas and Eckstein(1988)]{bertsekas1988dual}
Dimitri~P Bertsekas and Jonathan Eckstein.
\newblock Dual coordinate step methods for linear network flow problems.
\newblock \emph{Mathematical Programming}, 42\penalty0 (1):\penalty0 203--243,
  1988.

\bibitem[Bertsimas and Tsitsiklis(1997)]{bertsimas1997introduction}
Dimitris Bertsimas and John~N Tsitsiklis.
\newblock \emph{Introduction to Linear Optimization}.
\newblock Athena Scientific, 1997.

\bibitem[Bhatia et~al.(2018)Bhatia, Jain, and Lim]{bhatia2018bures}
Rajendra Bhatia, Tanvi Jain, and Yongdo Lim.
\newblock On the bures-wasserstein distance between positive definite matrices.
\newblock \emph{Expositiones Mathematicae, to appear}, 2018.

\bibitem[Bigot and Klein(2012{\natexlab{a}})]{BigotBarycenter}
J{\'e}r{\'e}mie Bigot and Thierry Klein.
\newblock Consistent estimation of a population barycenter in the {W}asserstein
  space.
\newblock \emph{arXiv Preprint arXiv:\-1212.\-2562}, 2012{\natexlab{a}}.

\bibitem[Bigot and Klein(2012{\natexlab{b}})]{bigot2012characterization}
J{\'e}r{\'e}mie Bigot and Thierry Klein.
\newblock Characterization of barycenters in the {Wasserstein} space by
  averaging optimal transport maps.
\newblock \emph{arXiv preprint arXiv:1212.2562}, 2012{\natexlab{b}}.

\bibitem[Bigot et~al.(2017{\natexlab{a}})Bigot, Cazelles, and
  Papadakis]{bigot2017central}
J{\'e}r{\'e}mie Bigot, Elsa Cazelles, and Nicolas Papadakis.
\newblock Central limit theorems for sinkhorn divergence between probability
  distributions on finite spaces and statistical applications.
\newblock \emph{arXiv preprint arXiv:1711.08947}, 2017{\natexlab{a}}.

\bibitem[Bigot et~al.(2017{\natexlab{b}})Bigot, Gouet, Klein, and
  L{\'o}pez]{bigot2017geodesic}
J{\'e}r{\'e}mie Bigot, Ra{\'u}l Gouet, Thierry Klein, and Alfredo L{\'o}pez.
\newblock Geodesic {PCA} in the {Wasserstein} space by convex pca.
\newblock \emph{Annales de l'Institut Henri Poincar\'e, Probabilit\'es et
  Statistiques}, 53\penalty0 (1):\penalty0 1--26, 2017{\natexlab{b}}.

\bibitem[Birkhoff(1946)]{birkhoff}
Garrett Birkhoff.
\newblock Tres observaciones sobre el algebra lineal.
\newblock \emph{Universidad Nacional de Tucum{\'a}n Revista Series A},
  5:\penalty0 147--151, 1946.

\bibitem[Birkhoff(1957)]{birkhoff1957extensions}
Garrett Birkhoff.
\newblock Extensions of jentzsch's theorem.
\newblock \emph{Transactions of the American Mathematical Society}, 85\penalty0
  (1):\penalty0 219--227, 1957.

\bibitem[Blanchet and Carlier(2015)]{blanchet2012optimal}
Adrien Blanchet and Guillaume Carlier.
\newblock Optimal transport and {Cournot-Nash} equilibria.
\newblock \emph{Mathematics of Operations Research}, 41\penalty0 (1):\penalty0
  125--145, 2015.

\bibitem[Blanchet et~al.(2008)Blanchet, Calvez, and
  Carrillo]{blanchet2008convergence}
Adrien Blanchet, Vincent Calvez, and Jos{\'e}~A Carrillo.
\newblock Convergence of the mass-transport steepest descent scheme for the
  subcritical {Patlak-Keller-Segel} model.
\newblock \emph{SIAM Journal on Numerical Analysis}, 46\penalty0 (2):\penalty0
  691--721, 2008.

\bibitem[Boissard(2011)]{boissard2011simple}
Emmanuel Boissard.
\newblock Simple bounds for the convergence of empirical and occupation
  measures in {1-Wasserstein} distance.
\newblock \emph{Electronic Journal of Probability}, 16:\penalty0 2296--2333,
  2011.

\bibitem[Boissard et~al.(2015)Boissard, Le~Gouic, and
  Loubes]{boissard2015distribution}
Emmanuel Boissard, Thibaut Le~Gouic, and Jean-Michel Loubes.
\newblock Distribution's template estimate with {Wasserstein} metrics.
\newblock \emph{Bernoulli}, 21\penalty0 (2):\penalty0 740--759, 2015.

\bibitem[Bolley et~al.(2007)Bolley, Guillin, and
  Villani]{bolley2007quantitative}
Fran{\c{c}}ois Bolley, Arnaud Guillin, and C{\'e}dric Villani.
\newblock Quantitative concentration inequalities for empirical measures on
  non-compact spaces.
\newblock \emph{Probability Theory and Related Fields}, 137\penalty0
  (3):\penalty0 541--593, 2007.

\bibitem[Bonneel et~al.(2011)Bonneel, {Van De Panne}, Paris, and
  Heidrich]{Bonneel-displacement}
Nicolas Bonneel, Michiel {Van De Panne}, Sylvain Paris, and Wolfgang Heidrich.
\newblock Displacement interpolation using lagrangian mass transport.
\newblock \emph{ACM Transactions on Graphics}, 30\penalty0 (6):\penalty0 158,
  2011.

\bibitem[Bonneel et~al.(2015)Bonneel, Rabin, Peyr{\'e}, and
  Pfister]{2013-Bonneel-barycenter}
Nicolas Bonneel, Julien Rabin, Gabriel Peyr{\'e}, and Hanspeter Pfister.
\newblock Sliced and {Radon} {Wasserstein} barycenters of measures.
\newblock \emph{Journal of Mathematical Imaging and Vision}, 51\penalty0
  (1):\penalty0 22--45, 2015.

\bibitem[Bonneel et~al.(2016)Bonneel, Peyr{\'e}, and
  Cuturi]{2016-bonneel-barycoord}
Nicolas Bonneel, Gabriel Peyr{\'e}, and Marco Cuturi.
\newblock Wasserstein barycentric coordinates: histogram regression using
  optimal transport.
\newblock \emph{ACM Transactions on Graphics}, 35\penalty0 (4):\penalty0
  71:1--71:10, 2016.

\bibitem[Borchardt and Jocobi(1865)]{borchardt1865investigando}
CW~Borchardt and CGJ Jocobi.
\newblock De investigando ordine systematis aequationum differentialium
  vulgarium cujuscunque.
\newblock \emph{Journal f{\"u}r die reine und angewandte Mathematik},
  64:\penalty0 297--320, 1865.

\bibitem[Borg and Groenen(2005)]{borg2005modern}
Ingwer Borg and Patrick~JF Groenen.
\newblock \emph{Modern Multidimensional Scaling: Theory and Applications}.
\newblock Springer Science \& Business Media, 2005.

\bibitem[Botsch et~al.(2010)Botsch, Kobbelt, Pauly, Alliez, and
  L{\'e}vy]{botsch-2010}
Mario Botsch, Leif Kobbelt, Mark Pauly, Pierre Alliez, and Bruno L{\'e}vy.
\newblock \emph{Polygon mesh processing}.
\newblock Taylor \& Francis, 2010.

\bibitem[Bousquet et~al.(2017)Bousquet, Gelly, Tolstikhin, Simon-Gabriel, and
  Schoelkopf]{Bousquet2017}
Olivier Bousquet, Sylvain Gelly, Ilya Tolstikhin, Carl-Johann Simon-Gabriel,
  and Bernhard Schoelkopf.
\newblock From optimal transport to generative modeling: the {VEGAN} cookbook.
\newblock \emph{arXiv preprint arXiv:1705.07642}, 2017.

\bibitem[Boyd et~al.(2011)Boyd, Parikh, Chu, Peleato, and Eckstein]{BoydADMM}
Stephen Boyd, Neal Parikh, Eric Chu, Borja Peleato, and Jonathan Eckstein.
\newblock Distributed optimization and statistical learning via the alternating
  direction method of multipliers.
\newblock \emph{Foundations and Trends in Machine Learning}, 3\penalty0
  (1):\penalty0 1--122, January 2011.

\bibitem[Bregman(1967)]{bregman1967relaxation}
Lev~M Bregman.
\newblock The relaxation method of finding the common point of convex sets and
  its application to the solution of problems in convex programming.
\newblock \emph{USSR Computational Mathematics and Mathematical Physics},
  7\penalty0 (3):\penalty0 200--217, 1967.

\bibitem[Brenier(1987)]{MR923203}
Yann Brenier.
\newblock {D\'ecomposition polaire et r\'earrangement monotone des champs de
  vecteurs}.
\newblock \emph{C. R. Acad. Sci. Paris S\'er. I Math.}, 305\penalty0
  (19):\penalty0 805--808, 1987.

\bibitem[Brenier(1990)]{BrenierEulerAMS}
Yann Brenier.
\newblock The least action principle and the related concept of generalized
  flows for incompressible perfect fluids.
\newblock \emph{Journal of the AMS}, 2:\penalty0 225--255, 1990.

\bibitem[Brenier(1991)]{Brenier91}
Yann Brenier.
\newblock Polar factorization and monotone rearrangement of vector-valued
  functions.
\newblock \emph{Communications on Pure and Applied Mathematics}, 44\penalty0
  (4):\penalty0 375--417, 1991.

\bibitem[Brenier(1993)]{BrenierEulerARMA}
Yann Brenier.
\newblock The dual least action problem for an ideal, incompressible fluid.
\newblock \emph{Archive for Rational Mechanics and Analysis}, 122\penalty0
  (4):\penalty0 323--351, 1993.

\bibitem[Brenier(1999)]{BrenierEulerCPAM}
Yann Brenier.
\newblock Minimal geodesics on groups of volume-preserving maps and generalized
  solutions of the {E}uler equations.
\newblock \emph{Communications on Pure and Applied Mathematics}, 52\penalty0
  (4):\penalty0 411--452, 1999.

\bibitem[Brenier(2008)]{BrenierGeneralized}
Yann Brenier.
\newblock Generalized solutions and hydrostatic approximation of the {E}uler
  equations.
\newblock \emph{Physica D. Nonlinear Phenomena}, 237\penalty0 (14-17):\penalty0
  1982--1988, 2008.

\bibitem[Bronstein et~al.(2006)Bronstein, Bronstein, and
  Kimmel]{bronstein2006generalized}
Alexander~M Bronstein, Michael~M Bronstein, and Ron Kimmel.
\newblock Generalized multidimensional scaling: a framework for
  isometry-invariant partial surface matching.
\newblock \emph{Proceedings of the National Academy of Sciences}, 103\penalty0
  (5):\penalty0 1168--1172, 2006.

\bibitem[Bronstein et~al.(2010)Bronstein, Bronstein, Kimmel, Mahmoudi, and
  Sapiro]{bronstein-2010}
Alexander~M Bronstein, Michael~M Bronstein, Ron Kimmel, Mona Mahmoudi, and
  Guillermo Sapiro.
\newblock A {G}romov-{H}ausdorff framework with diffusion geometry for
  topologically-robust non-rigid shape matching.
\newblock \emph{International Journal on Computer Vision}, 89\penalty0
  (2-3):\penalty0 266--286, 2010.

\bibitem[Brualdi(2006)]{brualdi2006combinatorial}
Richard~A Brualdi.
\newblock \emph{Combinatorial Matrix Classes}, volume 108.
\newblock Cambridge University Press, 2006.

\bibitem[Burago et~al.(2001)Burago, Burago, and Ivanov]{burago2001course}
Dmitri Burago, Yuri Burago, and Sergei Ivanov.
\newblock \emph{A Course in Metric Geometry}, volume~33.
\newblock American Mathematical Society Providence, RI, 2001.

\bibitem[Bures(1969)]{bures1969extension}
Donald Bures.
\newblock An extension of {Kakutani's} theorem on infinite product measures to
  the tensor product of semifinite $w^*$-algebras.
\newblock \emph{Transactions of the American Mathematical Society},
  135:\penalty0 199--212, 1969.

\bibitem[Burger et~al.(2010)Burger, Carrillo de~la Plata, and
  Wolfram]{burger2010mixed}
Martin Burger, Jos{\'e}~Antonio Carrillo de~la Plata, and Marie-Therese
  Wolfram.
\newblock A mixed finite element method for nonlinear diffusion equations.
\newblock \emph{Kinetic and Related Models}, 3\penalty0 (1):\penalty0 59--83,
  2010.

\bibitem[Burger et~al.(2012)Burger, Franek, and Sch{\"o}nlieb]{Burger-JKO}
Martin Burger, Marzena Franek, and Carola-Bibiane Sch{\"o}nlieb.
\newblock Regularised regression and density estimation based on optimal
  transport.
\newblock \emph{Applied Mathematics Research Express}, 2:\penalty0 209--253,
  2012.

\bibitem[Buttazzo et~al.(2012)Buttazzo, De~Pascale, and Gori-Giorgi]{BuDePGor}
Giuseppe Buttazzo, Luigi De~Pascale, and Paola Gori-Giorgi.
\newblock Optimal-transport formulation of electronic density-functional
  theory.
\newblock \emph{Physical Review A}, 85\penalty0 (6):\penalty0 062502, 2012.

\bibitem[Caffarelli(2003)]{caffarelli2003monge}
Luis Caffarelli.
\newblock The {Monge-Ampere} equation and optimal transportation, an elementary
  review.
\newblock \emph{Lecture Notes in Mathematics, Springer-Verlag}, pages 1--10,
  2003.

\bibitem[Caffarelli et~al.(2002)Caffarelli, Feldman, and
  McCann]{caffarelli2002constructing}
Luis Caffarelli, Mikhail Feldman, and Robert McCann.
\newblock Constructing optimal maps for {Monge's} transport problem as a limit
  of strictly convex costs.
\newblock \emph{Journal of the American Mathematical Society}, 15\penalty0
  (1):\penalty0 1--26, 2002.

\bibitem[Caffarelli and McCann(2010)]{CaffarelliMcCannPartial}
Luis~A Caffarelli and Robert~J McCann.
\newblock Free boundaries in optimal transport and {M}onge-{A}mp\`ere obstacle
  problems.
\newblock \emph{Annals of Mathematics}, 171\penalty0 (2):\penalty0 673--730,
  2010.

\bibitem[Caffarelli et~al.(1999)Caffarelli, Kochengin, and
  Oliker]{caffarelli1999problem}
Luis~A Caffarelli, Sergey~A Kochengin, and Vladimir~I Oliker.
\newblock Problem of reflector design with given far-field scattering data.
\newblock In \emph{Monge Amp{\`e}re equation: applications to geometry and
  optimization}, volume 226, page~13, 1999.

\bibitem[Canas and Rosasco(2012)]{CanasRosasco}
Guillermo Canas and Lorenzo Rosasco.
\newblock Learning probability measures with respect to optimal transport
  metrics.
\newblock In F.~Pereira, C.~J.~C. Burges, L.~Bottou, and K.~Q. Weinberger,
  editors, \emph{Advances in Neural Information Processing Systems 25}, pages
  2492--2500. 2012.

\bibitem[Carlen and Maas(2014)]{Carlen2014}
Eric~A Carlen and Jan Maas.
\newblock An analog of the 2-{W}asserstein metric in non-commutative
  probability under which the fermionic {F}okker--{P}lanck equation is gradient
  flow for the entropy.
\newblock \emph{Communications in Mathematical Physics}, 331\penalty0
  (3):\penalty0 887--926, 2014.

\bibitem[Carlier and Ekeland(2010)]{carlierekelandmatching}
Guillaume Carlier and Ivar Ekeland.
\newblock Matching for teams.
\newblock \emph{Economic Theory}, 42\penalty0 (2):\penalty0 397--418, 2010.

\bibitem[Carlier and Poon(2019)]{carlier2017total}
Guillaume Carlier and Clarice Poon.
\newblock On the total variation {Wasserstein} gradient flow and the {TV-JKO}
  scheme.
\newblock \emph{to appear in ESAIM: COCV}, 2019.

\bibitem[Carlier et~al.(2008)Carlier, Jimenez, and
  Santambrogio]{CarlierSantambrogioOTCongestion2008}
Guillaume Carlier, Chlo{\'e} Jimenez, and Filippo Santambrogio.
\newblock Optimal transportation with traffic congestion and {Wardrop}
  equilibria.
\newblock \emph{SIAM Journal on Control and Optimization}, 47\penalty0
  (3):\penalty0 1330--1350, 2008.

\bibitem[Carlier et~al.(2010)Carlier, Galichon, and
  Santambrogio]{carlier2010knothe}
Guillaume Carlier, Alfred Galichon, and Filippo Santambrogio.
\newblock From knothe's transport to {Brenier}'s map and a continuation method
  for optimal transport.
\newblock \emph{SIAM Journal on Mathematical Analysis}, 41\penalty0
  (6):\penalty0 2554--2576, 2010.

\bibitem[Carlier et~al.(2015)Carlier, Oberman, and
  Oudet]{Carlier-NumericsBarycenters}
Guillaume Carlier, Adam Oberman, and Edouard Oudet.
\newblock Numerical methods for matching for teams and {Wasserstein}
  barycenters.
\newblock \emph{ESAIM: Mathematical Modelling and Numerical Analysis},
  49\penalty0 (6):\penalty0 1621--1642, 2015.

\bibitem[Carlier et~al.(2016)Carlier, Chernozhukov, and
  Galichon]{carlier2016vector}
Guillaume Carlier, Victor Chernozhukov, and Alfred Galichon.
\newblock Vector quantile regression beyond correct specification.
\newblock \emph{arXiv preprint arXiv:\-1610.\-06833}, 2016.

\bibitem[Carlier et~al.(2017)Carlier, Duval, Peyr{\'e}, and
  Schmitzer]{2017-carlier-SIMA}
Guillaume Carlier, Vincent Duval, Gabriel Peyr{\'e}, and Bernhard Schmitzer.
\newblock Convergence of entropic schemes for optimal transport and gradient
  flows.
\newblock \emph{SIAM Journal on Mathematical Analysis}, 49\penalty0
  (2):\penalty0 1385--1418, 2017.

\bibitem[Carrillo and Moll(2009)]{carrillo2009numerical}
Jos{\'e}~A Carrillo and J~Salvador Moll.
\newblock Numerical simulation of diffusive and aggregation phenomena in
  nonlinear continuity equations by evolving diffeomorphisms.
\newblock \emph{SIAM Journal on Scientific Computing}, 31\penalty0
  (6):\penalty0 4305--4329, 2009.

\bibitem[Carrillo et~al.(2015)Carrillo, Chertock, and
  Huang]{CarrilloFiniteVolume}
Jos{\'e}~A Carrillo, Alina Chertock, and Yanghong Huang.
\newblock A finite-volume method for nonlinear nonlocal equations with a
  gradient flow structure.
\newblock \emph{Communications in Computational Physics}, 17:\penalty0
  233--258, 1 2015.

\bibitem[Censor and Reich(1998)]{CensorReich-Dykstra}
Yair Censor and Simeon Reich.
\newblock The {Dykstra} algorithm with {Bregman} projections.
\newblock \emph{Communications in Applied Analysis}, 2:\penalty0 407--419,
  1998.

\bibitem[Censor and Zenios(1992)]{censor1992proximal}
Yair Censor and Stavros~Andrea Zenios.
\newblock Proximal minimization algorithm with $d$-functions.
\newblock \emph{Journal of Optimization Theory and Applications}, 73\penalty0
  (3):\penalty0 451--464, 1992.

\bibitem[Champion et~al.(2008)Champion, De~Pascale, and
  Juutinen]{champion2008wasserstein}
Thierry Champion, Luigi De~Pascale, and Petri Juutinen.
\newblock The $\infty$-wasserstein distance: local solutions and existence of
  optimal transport maps.
\newblock \emph{SIAM Journal on Mathematical Analysis}, 40\penalty0
  (1):\penalty0 1--20, 2008.

\bibitem[Chan(1996)]{chan1996optimal}
Timothy~M Chan.
\newblock Optimal output-sensitive convex hull algorithms in two and three
  dimensions.
\newblock \emph{Discrete \& Computational Geometry}, 16\penalty0 (4):\penalty0
  361--368, 1996.

\bibitem[Chen et~al.(2016{\natexlab{a}})Chen, Georgiou, and
  Pavon]{chen2016relation}
Yongxin Chen, Tryphon~T Georgiou, and Michele Pavon.
\newblock On the relation between optimal transport and {Schr{\"o}dinger}
  bridges: A stochastic control viewpoint.
\newblock \emph{Journal of Optimization Theory and Applications}, 169\penalty0
  (2):\penalty0 671--691, 2016{\natexlab{a}}.

\bibitem[Chen et~al.(2016{\natexlab{b}})Chen, Georgiou, and
  Tannenbaum]{Chen2016}
Yongxin Chen, Tryphon~T Georgiou, and Allen Tannenbaum.
\newblock Matrix optimal mass transport: a quantum mechanical approach.
\newblock \emph{arXiv preprint arXiv:1610.03041}, 2016{\natexlab{b}}.

\bibitem[Chen et~al.(2017)Chen, Gangbo, Georgiou, and Tannenbaum]{ChenGangbo17}
Yongxin Chen, Wilfrid Gangbo, Tryphon~T Georgiou, and Allen Tannenbaum.
\newblock On the matrix {M}onge-{K}antorovich problem.
\newblock \emph{arXiv preprint arXiv:\-1701.\-02826}, 2017.

\bibitem[Chizat et~al.(2018{\natexlab{a}})Chizat, Peyr{\'e}, Schmitzer, and
  Vialard]{2015-chizat-unbalanced}
Lenaic Chizat, Gabriel Peyr{\'e}, Bernhard Schmitzer, and Fran{\c{c}}ois-Xavier
  Vialard.
\newblock Unbalanced optimal transport: geometry and {Kantorovich} formulation.
\newblock \emph{Journal of Functional Analysis}, 274\penalty0 (11):\penalty0
  3090--3123, 2018{\natexlab{a}}.

\bibitem[Chizat et~al.(2018{\natexlab{b}})Chizat, Peyr{\'e}, Schmitzer, and
  Vialard]{2016-chizat-sinkhorn}
Lenaic Chizat, Gabriel Peyr{\'e}, Bernhard Schmitzer, and Fran{\c{c}}ois-Xavier
  Vialard.
\newblock Scaling algorithms for unbalanced transport problems.
\newblock \emph{Mathematics of Computation}, 87:\penalty0 2563--2609,
  2018{\natexlab{b}}.

\bibitem[Chizat et~al.(2018{\natexlab{c}})Chizat, Peyr{\'e}, Schmitzer, and
  Vialard]{2017-chizat-focm}
Lenaic Chizat, Gabriel Peyr{\'e}, Bernhard Schmitzer, and Fran{\c{c}}ois-Xavier
  Vialard.
\newblock An interpolating distance between optimal transport and {Fisher--Rao}
  metrics.
\newblock \emph{Foundations of Computational Mathematics}, 18\penalty0
  (1):\penalty0 1--44, 2018{\natexlab{c}}.

\bibitem[Chow et~al.(2012)Chow, Huang, Li, and Zhou]{ChowHuangLiZhou2012}
Shui-Nee Chow, Wen Huang, Yao Li, and Haomin Zhou.
\newblock {Fokker-Planck} equations for a free energy functional or {Markov}
  process on a graph.
\newblock \emph{Archive for Rational Mechanics and Analysis}, 203\penalty0
  (3):\penalty0 969--1008, 2012.

\bibitem[Chow et~al.(2017{\natexlab{a}})Chow, Li, and Zhou]{chow2017discrete}
Shui-Nee Chow, Wuchen Li, and Haomin Zhou.
\newblock A discrete {Schrodinger} equation via optimal transport on graphs.
\newblock \emph{arXiv preprint arXiv:1705.07583}, 2017{\natexlab{a}}.

\bibitem[Chow et~al.(2017{\natexlab{b}})Chow, Li, and Zhou]{chow2017entropy}
Shui-Nee Chow, Wuchen Li, and Haomin Zhou.
\newblock Entropy dissipation of {Fokker-Planck} equations on graphs.
\newblock \emph{arXiv preprint arXiv:1701.04841}, 2017{\natexlab{b}}.

\bibitem[Chowdhury and M{\'e}moli(2016)]{chowdhury2016constructing}
Samir Chowdhury and Facundo M{\'e}moli.
\newblock Constructing geodesics on the space of compact metric spaces.
\newblock \emph{arXiv preprint arXiv:1603.02385}, 2016.

\bibitem[Chui and Rangarajan(2000)]{chui2000new}
Haili Chui and Anand Rangarajan.
\newblock A new algorithm for non-rigid point matching.
\newblock In \emph{Computer Vision and Pattern Recognition, 2000. Proceedings.
  IEEE Conference on}, volume~2, pages 44--51. IEEE, 2000.

\bibitem[Cisz{\'a}r(1967)]{ciszar1967information}
Imre Cisz{\'a}r.
\newblock Information-type measures of difference of probability distributions
  and indirect observations.
\newblock \emph{Studia Scientiarum Mathematicarum Hungarica}, 2:\penalty0
  299--318, 1967.

\bibitem[Cohen et~al.(2017)Cohen, Madry, Tsipras, and Vladu]{cohen2017matrix}
Michael~B Cohen, Aleksander Madry, Dimitris Tsipras, and Adrian Vladu.
\newblock Matrix scaling and balancing via box constrained {Newton's} method
  and interior point methods.
\newblock \emph{arXiv preprint arXiv:1704.02310}, 2017.

\bibitem[Cohen and Guibas(1999)]{cohen1999earth}
Scott Cohen and Leonidas Guibas.
\newblock The earth mover's distance under transformation sets.
\newblock In \emph{Proceedings of the Seventh IEEE International Conference on
  Computer vision}, volume~2, pages 1076--1083. IEEE, 1999.

\bibitem[Combettes and Pesquet(2007)]{Combettes2007}
Patrick~L Combettes and Jean-Christophe Pesquet.
\newblock A {D}ouglas-{R}achford splitting approach to nonsmooth convex
  variational signal recovery.
\newblock \emph{IEEE Journal of Selected Topics in Signal Processing},
  1\penalty0 (4):\penalty0 564 --574, 2007.

\bibitem[Cominetti and San~Mart{\'i}n(1994)]{CominettiAsympt}
Roberto Cominetti and Jaime San~Mart{\'i}n.
\newblock Asymptotic analysis of the exponential penalty trajectory in linear
  programming.
\newblock \emph{Mathematical Programming}, 67\penalty0 (1-3):\penalty0
  169--187, 1994.

\bibitem[Condat(2015)]{condat2015fast}
Laurent Condat.
\newblock Fast projection onto the simplex and the $\ell_1$ ball.
\newblock \emph{Math. Programming, Ser. A}, pages 1--11, 2015.

\bibitem[Costa et~al.(2015)Costa, Santos, and Strapasson]{costa2015fisher}
Sueli~IR Costa, Sandra~A Santos, and Jo{\~a}o~E Strapasson.
\newblock Fisher information distance: a geometrical reading.
\newblock \emph{Discrete Applied Mathematics}, 197:\penalty0 59--69, 2015.

\bibitem[Cotar et~al.(2013)Cotar, Friesecke, and Kl{\"u}ppelberg]{CotarDFT}
Codina Cotar, Gero Friesecke, and Claudia Kl{\"u}ppelberg.
\newblock Density functional theory and optimal transportation with {C}oulomb
  cost.
\newblock \emph{Communications on Pure and Applied Mathematics}, 66\penalty0
  (4):\penalty0 548--599, 2013.

\bibitem[Courty et~al.(2016)Courty, Flamary, Tuia, and
  Corpetti]{courty2016optimal}
Nicolas Courty, R{\'e}mi Flamary, Devis Tuia, and Thomas Corpetti.
\newblock Optimal transport for data fusion in remote sensing.
\newblock In \emph{2016 IEEE International Geoscience and Remote Sensing
  Symposium}, pages 3571--3574. IEEE, 2016.

\bibitem[Courty et~al.(2017{\natexlab{a}})Courty, Flamary, Habrard, and
  Rakotomamonjy]{courty2017joint}
Nicolas Courty, R\'{e}mi Flamary, Amaury Habrard, and Alain Rakotomamonjy.
\newblock Joint distribution optimal transportation for domain adaptation.
\newblock In I.~Guyon, U.~V. Luxburg, S.~Bengio, H.~Wallach, R.~Fergus,
  S.~Vishwanathan, and R.~Garnett, editors, \emph{Advances in Neural
  Information Processing Systems 30}, pages 3730--3739. 2017{\natexlab{a}}.

\bibitem[Courty et~al.(2017{\natexlab{b}})Courty, Flamary, Tuia, and
  Rakotomamonjy]{courty2017optimal}
Nicolas Courty, R{\'e}mi Flamary, Devis Tuia, and Alain Rakotomamonjy.
\newblock Optimal transport for domain adaptation.
\newblock \emph{IEEE Transactions on Pattern Analysis and Machine
  Intelligence}, 39\penalty0 (9):\penalty0 1853--1865, 2017{\natexlab{b}}.

\bibitem[Crane et~al.(2013)Crane, Weischedel, and Wardetzky]{Crane2013}
Keenan Crane, Clarisse Weischedel, and Max Wardetzky.
\newblock Geodesics in heat: a new approach to computing distance based on heat
  flow.
\newblock \emph{ACM Transaction on Graphics}, 32\penalty0 (5):\penalty0
  152:1--152:11, October 2013.

\bibitem[Cuesta and Matran(1989)]{cuesta1989}
Juan~Antonio Cuesta and Carlos Matran.
\newblock Notes on the wasserstein metric in hilbert spaces.
\newblock \emph{The Annals of Probability}, 17\penalty0 (3):\penalty0
  1264--1276, 07 1989.

\bibitem[Cuturi(2012)]{cuturi2012positivity}
Marco Cuturi.
\newblock Positivity and transportation.
\newblock \emph{arXiv preprint 1209.2655}, 2012.

\bibitem[Cuturi(2013)]{CuturiSinkhorn}
Marco Cuturi.
\newblock Sinkhorn distances: lightspeed computation of optimal transport.
\newblock In \emph{Advances in Neural Information Processing Systems 26}, pages
  2292--2300, 2013.

\bibitem[Cuturi and Avis(2014)]{CuturiGroundMetric2014}
Marco Cuturi and David Avis.
\newblock Ground metric learning.
\newblock \emph{Journal of Machine Learning Research}, 15:\penalty0 533--564,
  2014.

\bibitem[Cuturi and Doucet(2014)]{CuturiBarycenter}
Marco Cuturi and Arnaud Doucet.
\newblock Fast computation of {Wasserstein} barycenters.
\newblock In \emph{Proceedings of ICML}, volume~32, pages 685--693, 2014.

\bibitem[Cuturi and Fukumizu(2007)]{CuturiNIPS2006}
Marco Cuturi and Kenji Fukumizu.
\newblock Kernels on structured objects through nested histograms.
\newblock In P.~B. Sch\"{o}lkopf, J.~C. Platt, and T.~Hoffman, editors,
  \emph{Advances in Neural Information Processing Systems 19}, pages 329--336.
  MIT Press, 2007.

\bibitem[Cuturi and Peyr{\'e}(2016)]{2016-Cuturi-siims}
Marco Cuturi and Gabriel Peyr{\'e}.
\newblock A smoothed dual approach for variational {Wasserstein} problems.
\newblock \emph{SIAM Journal on Imaging Sciences}, 9\penalty0 (1):\penalty0
  320--343, 2016.

\bibitem[Cuturi and Peyr{\'e}(2018)]{cuturi2018semidual}
Marco Cuturi and Gabriel Peyr{\'e}.
\newblock Semidual regularized optimal transport.
\newblock \emph{SIAM Review}, 60\penalty0 (4):\penalty0 941--965, 2018.

\bibitem[Dalalyan(2017)]{pmlr-v65-dalalyan17a}
Arnak Dalalyan.
\newblock Further and stronger analogy between sampling and optimization:
  Langevin monte carlo and gradient descent.
\newblock In \emph{Proceedings of the 2017 Conference on Learning Theory},
  volume~65 of \emph{Proceedings of Machine Learning Research}, pages 678--689.
  PMLR, 2017.

\bibitem[Dalalyan and Karagulyan(2017)]{dalalyan2017user}
Arnak~S Dalalyan and Avetik~G Karagulyan.
\newblock User-friendly guarantees for the {Langevin Monte Carlo} with
  inaccurate gradient.
\newblock \emph{arXiv preprint arXiv:1710.00095}, 2017.

\bibitem[Dantzig(1949)]{dantzig49econometrica}
George~B. Dantzig.
\newblock Programming of interdependent activities: {II} mathematical model.
\newblock \emph{Econometrica}, 17\penalty0 (3/4):\penalty0 200--211, 1949.

\bibitem[Dantzig(1951)]{Dantzig51}
George~B Dantzig.
\newblock Application of the simplex method to a transportation problem.
\newblock \emph{Activity Analysis of Production and Allocation}, 13:\penalty0
  359--373, 1951.

\bibitem[Dantzig(1983)]{Dantzig1983}
George~B. Dantzig.
\newblock \emph{Reminiscences Aabout the origins of linear programming}, pages
  78--86.
\newblock Springer, 1983.

\bibitem[Dantzig(1991)]{dantzig1991}
George~B. Dantzig.
\newblock Linear programming.
\newblock In J.~K. Lenstra, A.~H. G.~Rinnooy Kan, and A.~Schrijver, editors,
  \emph{History of mathematical programming: a collection of personal
  reminiscences}, pages 257--282. Elsevier Science Publishers, 1991.

\bibitem[Dattorro(2017)]{dattorro2010convex}
Jon Dattorro.
\newblock \emph{Convex Optimization \& Euclidean Distance Geometry}.
\newblock Meboo Publishing, 2017.

\bibitem[De~Goes et~al.(2012)De~Goes, Breeden, Ostromoukhov, and
  Desbrun]{de2012blue}
Fernando De~Goes, Katherine Breeden, Victor Ostromoukhov, and Mathieu Desbrun.
\newblock Blue noise through optimal transport.
\newblock \emph{ACM Transactions on Graphics}, 31\penalty0 (6):\penalty0 171,
  2012.

\bibitem[de~Goes et~al.(2015)de~Goes, Wallez, Huang, Pavlov, and
  Desbrun]{deGoes2015}
Fernando de~Goes, Corentin Wallez, Jin Huang, Dmitry Pavlov, and Mathieu
  Desbrun.
\newblock Power particles: an incompressible fluid solver based on power
  diagrams.
\newblock \emph{ACM Transaction Graphics}, 34\penalty0 (4):\penalty0
  50:1--50:11, July 2015.

\bibitem[del Barrio et~al.(2016)del Barrio, Cuesta-Albertos, Matr{\'a}n, and
  Mayo-{\'I}scar]{del2016robust}
Eustasio del Barrio, JA~Cuesta-Albertos, C~Matr{\'a}n, and A~Mayo-{\'I}scar.
\newblock Robust clustering tools based on optimal transportation.
\newblock \emph{arXiv preprint arXiv:1607.01179}, 2016.

\bibitem[Delon(2004)]{delon2004midway}
Julie Delon.
\newblock Midway image equalization.
\newblock \emph{Journal of Mathematical Imaging and Vision}, 21\penalty0
  (2):\penalty0 119--134, 2004.

\bibitem[Delon et~al.(2010)Delon, Salomon, and Sobolevski]{delon-circle}
Julie Delon, Julien Salomon, and Andrei Sobolevski.
\newblock Fast transport optimization for {M}onge costs on the circle.
\newblock \emph{SIAM Journal on Applied Mathematics}, 70\penalty0 (7):\penalty0
  2239--2258, 2010.

\bibitem[Delon et~al.(2012)Delon, Salomon, and Sobolevski]{delon-concave}
Julie Delon, Julien Salomon, and Andrei Sobolevski.
\newblock Local matching indicators for transport problems with concave costs.
\newblock \emph{SIAM Journal on Discrete Mathematics}, 26\penalty0
  (2):\penalty0 801--827, 2012.

\bibitem[Deming and Stephan(1940)]{DemingStephanIPFP}
Edwards Deming and Frederick~F Stephan.
\newblock On a least squares adjustment of a sampled frequency table when the
  expected marginal totals are known.
\newblock \emph{Annals of Mathematical Statistics}, 11\penalty0 (4):\penalty0
  427--444, 1940.

\bibitem[Dereich et~al.(2013)Dereich, Scheutzow, and
  Schottstedt]{dereich2013constructive}
Steffen Dereich, Michael Scheutzow, and Reik Schottstedt.
\newblock Constructive quantization: Approximation by empirical measures.
\newblock In \emph{Annales de l'Institut Henri Poincar{\'e}, Probabilit{\'e}s
  et Statistiques}, volume~49, pages 1183--1203, 2013.

\bibitem[Deriche(1993)]{deriche1993recursively}
Rachid Deriche.
\newblock \emph{Recursively implementating the {Gaussian} and its derivatives}.
\newblock PhD thesis, INRIA, 1993.

\bibitem[Dessein et~al.(2017)Dessein, Papadakis, and
  Deledalle]{dessein2017parameter}
Arnaud Dessein, Nicolas Papadakis, and Charles-Alban Deledalle.
\newblock Parameter estimation in finite mixture models by regularized optimal
  transport: a unified framework for hard and soft clustering.
\newblock \emph{arXiv preprint arXiv:1711.04366}, 2017.

\bibitem[Dessein et~al.(2018)Dessein, Papadakis, and
  Rouas]{dessein2016regularized}
Arnaud Dessein, Nicolas Papadakis, and Jean-Luc Rouas.
\newblock Regularized optimal transport and the rot mover's distance.
\newblock \emph{Journal of Machine Learning Research}, 19\penalty0
  (15):\penalty0 1--53, 2018.

\bibitem[Di~Marino and Chizat(2017)]{marino2017tumor}
Simone Di~Marino and Lenaic Chizat.
\newblock A tumor growth model of {Hele-Shaw} type as a gradient flow.
\newblock \emph{Arxiv}, 2017.

\bibitem[Do~Ba et~al.(2011)Do~Ba, Nguyen, Nguyen, and
  Rubinfeld]{do2011sublinear}
Khanh Do~Ba, Huy~L Nguyen, Huy~N Nguyen, and Ronitt Rubinfeld.
\newblock Sublinear time algorithms for earth mover's distance.
\newblock \emph{Theory of Computing Systems}, 48\penalty0 (2):\penalty0
  428--442, 2011.

\bibitem[Dolbeault et~al.(2009)Dolbeault, Nazaret, and
  Savar{\'e}]{dolbeault2009new}
Jean Dolbeault, Bruno Nazaret, and Giuseppe Savar{\'e}.
\newblock A new class of transport distances between measures.
\newblock \emph{Calculus of Variations and Partial Differential Equations},
  34\penalty0 (2):\penalty0 193--231, 2009.

\bibitem[Dolinsky and Soner(2014)]{dolinsky2014martingale}
Yan Dolinsky and H~Mete Soner.
\newblock Martingale optimal transport and robust hedging in continuous time.
\newblock \emph{Probability Theory and Related Fields}, 160\penalty0
  (1-2):\penalty0 391--427, 2014.

\bibitem[Dudley(1969)]{dudley1969speed}
Richard~M. Dudley.
\newblock The speed of mean {Glivenko-Cantelli} convergence.
\newblock \emph{Annals of Mathematical Statistics}, 40\penalty0 (1):\penalty0
  40--50, 1969.

\bibitem[Dupuy et~al.(2016)Dupuy, Galichon, and Sun]{dupuy2016estimating}
Arnaud Dupuy, Alfred Galichon, and Yifei Sun.
\newblock Estimating matching affinity matrix under low-rank constraints.
\newblock \emph{Arxiv:1612.09585}, 2016.

\bibitem[Dvurechenskii et~al.(2018)Dvurechenskii, Dvinskikh, Gasnikov, Uribe,
  and Nedich]{NIPS2018_8274}
Pavel Dvurechenskii, Darina Dvinskikh, Alexander Gasnikov, Cesar Uribe, and
  Angelia Nedich.
\newblock Decentralize and randomize: Faster algorithm for wasserstein
  barycenters.
\newblock In S.~Bengio, H.~Wallach, H.~Larochelle, K.~Grauman, N.~Cesa-Bianchi,
  and R.~Garnett, editors, \emph{Advances in Neural Information Processing
  Systems 31}, pages 10783--10793. 2018.

\bibitem[Dvurechensky et~al.(2018)Dvurechensky, Gasnikov, and
  Kroshnin]{pmlr-v80-dvurechensky18a}
Pavel Dvurechensky, Alexander Gasnikov, and Alexey Kroshnin.
\newblock Computational optimal transport: Complexity by accelerated gradient
  descent is better than by sinkhorn's algorithm.
\newblock In Jennifer Dy and Andreas Krause, editors, \emph{Proceedings of the
  35th International Conference on Machine Learning}, volume~80 of
  \emph{Proceedings of Machine Learning Research}, pages 1367--1376. PMLR,
  2018.

\bibitem[Dykstra(1983)]{Dykstra83}
Richard~L Dykstra.
\newblock An algorithm for restricted least squares regression.
\newblock \emph{Journal American Statistical Association}, 78\penalty0
  (384):\penalty0 839--842, 1983.

\bibitem[Dykstra(1985)]{Dykstra85}
Richard~L Dykstra.
\newblock An iterative procedure for obtaining {$I$}-projections onto the
  intersection of convex sets.
\newblock \emph{Annals of Probability}, 13\penalty0 (3):\penalty0 975--984,
  1985.

\bibitem[Eckstein and Bertsekas(1992)]{Eckstein1992}
Jonathan Eckstein and Dimitri~P Bertsekas.
\newblock On the {D}ouglas-{R}achford splitting method and the proximal point
  algorithm for maximal monotone operators.
\newblock \emph{Mathematical Programming}, 55:\penalty0 293--318, 1992.

\bibitem[Edwards(1975)]{edwards1975structure}
David~A Edwards.
\newblock The structure of superspace.
\newblock In \emph{Studies in topology}, pages 121--133. Elsevier, 1975.

\bibitem[El~Moselhy and Marzouk(2012)]{el2012bayesian}
Tarek~A El~Moselhy and Youssef~M Marzouk.
\newblock Bayesian inference with optimal maps.
\newblock \emph{Journal of Computational Physics}, 231\penalty0 (23):\penalty0
  7815--7850, 2012.

\bibitem[Endres and Schindelin(2003)]{endres2003new}
Dominik~Maria Endres and Johannes~E Schindelin.
\newblock A new metric for probability distributions.
\newblock \emph{IEEE Transactions on Information theory}, 49\penalty0
  (7):\penalty0 1858--1860, 2003.

\bibitem[Erbar(2010)]{ErbarHeatManifold}
Matthias Erbar.
\newblock The heat equation on manifolds as a gradient flow in the
  {Wasserstein} space.
\newblock \emph{Annales de l'Institut Henri Poincar\'e, Probabilit\'es et
  Statistiques}, 46\penalty0 (1):\penalty0 1--23, 2010.

\bibitem[Erbar and Maas(2014)]{ErbarDCDS}
Matthias Erbar and Jan Maas.
\newblock Gradient flow structures for discrete porous medium equations.
\newblock \emph{Discrete and Continuous Dynamical Systems}, 34\penalty0
  (4):\penalty0 1355--1374, 2014.

\bibitem[Erlander(1980)]{erlander1980optimal}
Sven Erlander.
\newblock \emph{Optimal Spatial Interaction and the Gravity Model}, volume 173.
\newblock Springer-Verlag, 1980.

\bibitem[Erlander and Stewart(1990)]{erlander1990gravity}
Sven Erlander and Neil~F Stewart.
\newblock \emph{The Gravity Model in Transportation Analysis: Theory and
  Extensions}.
\newblock 1990.

\bibitem[Esfahani and Kuhn(2018)]{esfahani2018data}
Peyman~Mohajerin Esfahani and Daniel Kuhn.
\newblock Data-driven distributionally robust optimization using the
  wasserstein metric: Performance guarantees and tractable reformulations.
\newblock \emph{Mathematical Programming}, 171\penalty0 (1-2):\penalty0
  115--166, 2018.

\bibitem[Essid and Solomon(2017)]{essid2017quadratically}
Montacer Essid and Justin Solomon.
\newblock Quadratically-regularized optimal transport on graphs.
\newblock \emph{arXiv preprint arXiv:1704.08200}, 2017.

\bibitem[Evans and Gangbo(1999)]{evans1999differential}
Lawrence~C. Evans and Wilfrid Gangbo.
\newblock \emph{Differential Equations Methods for the Monge-Kantorovich Mass
  Transfer Problem}, volume 653.
\newblock American Mathematical Society, 1999.

\bibitem[Feldman and McCann(2002)]{FeldmanMacCann02}
Mikhail Feldman and Robert McCann.
\newblock Monge's transport problem on a {Riemannian} manifold.
\newblock \emph{Transaction AMS}, 354\penalty0 (4):\penalty0 1667--1697, 2002.

\bibitem[Feydy et~al.(2017)Feydy, Charlier, Vialard, and
  Peyr{\'e}]{2017-feydy-miccai}
Jean Feydy, Benjamin Charlier, Francois-Xavier Vialard, and Gabriel Peyr{\'e}.
\newblock Optimal transport for diffeomorphic registration.
\newblock In \emph{Proceedings of MICCAI'17}, pages 291--299. Springer, 2017.

\bibitem[Feydy et~al.(2019)Feydy, S{\'e}journ{\'e}, Vialard, Amari, Trouv{\'e},
  and Peyr{\'e}]{feydy2018interpolating}
Jean Feydy, Thibault S{\'e}journ{\'e}, Fran{\c{c}}ois-Xavier Vialard, Shun-Ichi
  Amari, Alain Trouv{\'e}, and Gabriel Peyr{\'e}.
\newblock Interpolating between optimal transport and mmd using sinkhorn
  divergences.
\newblock In \emph{Proceedings of the 22th International Conference on
  Artificial Intelligence and Statistics}, 2019.

\bibitem[Figalli(2010)]{FigalliPartial}
Alessio Figalli.
\newblock The optimal partial transport problem.
\newblock \emph{Archive for Rational Mechanics and Analysis}, 195\penalty0
  (2):\penalty0 533--560, 2010.

\bibitem[Flamary et~al.(2016)Flamary, F{\'e}votte, Courty, and
  Emiya]{flamary2016optimal}
R{\'e}mi Flamary, C{\'e}dric F{\'e}votte, Nicolas Courty, and Valentin Emiya.
\newblock Optimal spectral transportation with application to music
  transcription.
\newblock In \emph{Advances in Neural Information Processing Systems}, pages
  703--711, 2016.

\bibitem[Ford and Fulkerson(1962)]{ford1962flows}
Lester~Randolph Ford and Delbert~Ray Fulkerson.
\newblock \emph{Flows in Networks}.
\newblock Princeton University Press, 1962.

\bibitem[Forrester and Kieburg(2016)]{forrester2016relating}
Peter~J Forrester and Mario Kieburg.
\newblock Relating the {Bures} measure to the {Cauchy} two-matrix model.
\newblock \emph{Communications in Mathematical Physics}, 342\penalty0
  (1):\penalty0 151--187, 2016.

\bibitem[Fournier and Guillin(2015)]{fournier2015rate}
Nicolas Fournier and Arnaud Guillin.
\newblock On the rate of convergence in {Wasserstein} distance of the empirical
  measure.
\newblock \emph{Probability Theory and Related Fields}, 162\penalty0
  (3-4):\penalty0 707--738, 2015.

\bibitem[Franklin and Lorenz(1989)]{franklin1989scaling}
Joel Franklin and Jens Lorenz.
\newblock On the scaling of multidimensional matrices.
\newblock \emph{Linear Algebra and its Applications}, 114:\penalty0 717--735,
  1989.

\bibitem[Frisch et~al.(2002)Frisch, Matarrese, Mohayaee, and
  Sobolevski]{FrischNaturee}
Uriel Frisch, Sabino Matarrese, Roya Mohayaee, and Andrei Sobolevski.
\newblock A reconstruction of the initial conditions of the universe by optimal
  mass transportation.
\newblock \emph{Nature}, 417\penalty0 (6886):\penalty0 260--262, 2002.

\bibitem[Froese and Oberman(2011)]{froese2011convergent}
Brittany~D Froese and Adam~M Oberman.
\newblock Convergent finite difference solvers for viscosity solutions of the
  elliptic monge--amp{\`e}re equation in dimensions two and higher.
\newblock \emph{SIAM Journal on Numerical Analysis}, 49\penalty0 (4):\penalty0
  1692--1714, 2011.

\bibitem[Frogner et~al.(2015)Frogner, Zhang, Mobahi, Araya, and
  Poggio]{FrognerNIPS}
Charlie Frogner, Chiyuan Zhang, Hossein Mobahi, Mauricio Araya, and Tomaso~A
  Poggio.
\newblock Learning with a {Wasserstein} loss.
\newblock In \emph{Advances in Neural Information Processing Systems}, pages
  2053--2061, 2015.

\bibitem[Gabay and Mercier(1976)]{GabayMercier}
Daniel Gabay and Bertrand Mercier.
\newblock A dual algorithm for the solution of nonlinear variational problems
  via finite element approximation.
\newblock \emph{Computers \& Mathematics with Applications}, 2\penalty0
  (1):\penalty0 17--40, 1976.

\bibitem[Galichon(2016)]{galichon2016optimal}
Alfred Galichon.
\newblock \emph{Optimal Transport Methods in Economics}.
\newblock Princeton University Press, 2016.

\bibitem[Galichon and Salani{\'e}(2009)]{Galichon-Entropic}
Alfred Galichon and Bernard Salani{\'e}.
\newblock Matching with trade-offs: revealed preferences over competing
  characteristics.
\newblock Technical report, Preprint SSRN-1487307, 2009.

\bibitem[Galichon et~al.(2014)Galichon, Henry-Labord{\`e}re, and
  Touzi]{GalichonMartingale}
Alfred Galichon, Pierre Henry-Labord{\`e}re, and Nizar Touzi.
\newblock A stochastic control approach to no-arbitrage bounds given marginals,
  with an application to lookback options.
\newblock \emph{Annals of Applied Probability}, 24\penalty0 (1):\penalty0
  312--336, 2014.

\bibitem[Gallou{\"e}t and M{\'e}rigot(2017)]{gallouet2017lagrangian}
Thomas~O Gallou{\"e}t and Quentin M{\'e}rigot.
\newblock A lagrangian scheme {\`a} la brenier for the incompressible euler
  equations.
\newblock \emph{Foundations of Computational Mathematics}, 18:\penalty0 1--31,
  2017.

\bibitem[Gallou{\"e}t and Monsaingeon(2017)]{gallouet2017jko}
Thomas~O Gallou{\"e}t and Leonard Monsaingeon.
\newblock A {JKO} splitting scheme for {Kantorovich--Fisher--Rao} gradient
  flows.
\newblock \emph{SIAM Journal on Mathematical Analysis}, 49\penalty0
  (2):\penalty0 1100--1130, 2017.

\bibitem[Gangbo and McCann(1996)]{gangbo1996geometry}
Wilfrid Gangbo and Robert~J McCann.
\newblock The geometry of optimal transportation.
\newblock \emph{Acta Mathematica}, 177\penalty0 (2):\penalty0 113--161, 1996.

\bibitem[Gangbo and Swiech(1998)]{GangboSciech}
Wilfrid Gangbo and Andrzej Swiech.
\newblock Optimal maps for the multidimensional {Monge-Kantorovich} problem.
\newblock \emph{Communications on Pure and Applied Mathematics}, 51\penalty0
  (1):\penalty0 23--45, 1998.

\bibitem[GAO et~al.(2018)GAO, Xie, Xie, and Xu]{NIPS2018_8015}
RUI GAO, Liyan Xie, Yao Xie, and Huan Xu.
\newblock Robust hypothesis testing using wasserstein uncertainty sets.
\newblock In S.~Bengio, H.~Wallach, H.~Larochelle, K.~Grauman, N.~Cesa-Bianchi,
  and R.~Garnett, editors, \emph{Advances in Neural Information Processing
  Systems 31}, pages 7913--7923. 2018.

\bibitem[Gelbrich(1990)]{gelbrich1990formula}
Matthias Gelbrich.
\newblock On a formula for the $l^2$ wasserstein metric between measures on
  euclidean and hilbert spaces.
\newblock \emph{Mathematische Nachrichten}, 147\penalty0 (1):\penalty0
  185--203, 1990.

\bibitem[Genevay et~al.(2016)Genevay, Cuturi, Peyr{\'e}, and
  Bach]{genevay2016stochastic}
Aude Genevay, Marco Cuturi, Gabriel Peyr{\'e}, and Francis Bach.
\newblock Stochastic optimization for large-scale optimal transport.
\newblock In \emph{Advances in Neural Information Processing Systems}, pages
  3440--3448, 2016.

\bibitem[Genevay et~al.(2017)Genevay, Peyr{\'e}, and
  Cuturi]{2017-Genevay-gan-vae}
Aude Genevay, Gabriel Peyr{\'e}, and Marco Cuturi.
\newblock {GAN} and {VAE} from an optimal transport point of view.
\newblock \penalty0 (arXiv preprint arXiv:1706.01807), 2017.

\bibitem[Genevay et~al.(2018)Genevay, Peyr{\'e}, and
  Cuturi]{2017-Genevay-AutoDiff}
Aude Genevay, Gabriel Peyr{\'e}, and Marco Cuturi.
\newblock Learning generative models with {Sinkhorn} divergences.
\newblock In \emph{Proceedings of the 21st International Conference on
  Artificial Intelligence and Statistics}, pages 1608--1617, 2018.

\bibitem[Genevay et~al.(2019)Genevay, Chizat, Bach, Cuturi, and
  Peyr{\'e}]{genevay2018sample}
Aude Genevay, L{\'e}naic Chizat, Francis Bach, Marco Cuturi, and Gabriel
  Peyr{\'e}.
\newblock Sample complexity of sinkhorn divergences.
\newblock In \emph{Proceedings of the 22th International Conference on
  Artificial Intelligence and Statistics}, 2019.

\bibitem[Gentil et~al.(2015)Gentil, L{\'e}onard, and Ripani]{gentil2015analogy}
Ivan Gentil, Christian L{\'e}onard, and Luigia Ripani.
\newblock About the analogy between optimal transport and minimal entropy.
\newblock \emph{arXiv preprint arXiv:1510.08230}, 2015.

\bibitem[George and Liu(1989)]{george1989evolution}
Alan George and Joseph~WH Liu.
\newblock The evolution of the minimum degree ordering algorithm.
\newblock \emph{SIAM Review}, 31\penalty0 (1):\penalty0 1--19, 1989.

\bibitem[Georgiou and Pavon(2015)]{georgiou2015positive}
Tryphon~T Georgiou and Michele Pavon.
\newblock Positive contraction mappings for classical and quantum
  {S}chr{\"o}dinger systems.
\newblock \emph{Journal of Mathematical Physics}, 56\penalty0 (3):\penalty0
  033301, 2015.

\bibitem[Getreuer(2013)]{getreuer2013survey}
Pascal Getreuer.
\newblock A survey of {Gaussian} convolution algorithms.
\newblock \emph{Image Processing On Line}, 2013:\penalty0 286--310, 2013.

\bibitem[Gianazza et~al.(2009)Gianazza, Savar{\'e}, and Toscani]{GianazzaARMA}
Ugo Gianazza, Giuseppe Savar{\'e}, and Giuseppe Toscani.
\newblock The {Wasserstein} gradient flow of the {Fisher} information and the
  quantum drift-diffusion equation.
\newblock \emph{Archive for Rational Mechanics and Analysis}, 194\penalty0
  (1):\penalty0 133--220, 2009.

\bibitem[Gibbs and Su(2002)]{gibbs2002choosing}
Alison~L Gibbs and Francis~Edward Su.
\newblock On choosing and bounding probability metrics.
\newblock \emph{International Statistical Review}, 70\penalty0 (3):\penalty0
  419--435, 2002.

\bibitem[Glaunes et~al.(2004)Glaunes, Trouv{\'e}, and
  Younes]{glaunes2004diffeomorphic}
Joan Glaunes, Alain Trouv{\'e}, and Laurent Younes.
\newblock Diffeomorphic matching of distributions: a new approach for
  unlabelled point-sets and sub-manifolds matching.
\newblock In \emph{Proceedings of the 2004 IEEE Computer Society Conference on
  Computer Vision and Pattern Recognition}, volume~2, 2004.

\bibitem[Glorot et~al.(2011)Glorot, Bordes, and Bengio]{glorot2011domain}
Xavier Glorot, Antoine Bordes, and Yoshua Bengio.
\newblock Domain adaptation for large-scale sentiment classification: A deep
  learning approach.
\newblock In \emph{Proceedings of the 28th International Conference on Machine
  Learning}, pages 513--520, 2011.

\bibitem[Glowinski and Marroco(1975)]{GlowinskiMarroco}
Roland Glowinski and A.~Marroco.
\newblock Sur l'approximation, par \'el\'ements finis d'ordre un, et la
  r\'esolution, par p\'enalisation-dualit\'e d'une classe de probl\`emes de
  {Dirichlet} non lin\'eaires.
\newblock \emph{ESAIM: Mathematical Modelling and Numerical Analysis},
  9\penalty0 (R2):\penalty0 41--76, 1975.

\bibitem[Gold and Rangarajan(1996)]{gold-1996}
Steven Gold and Anand Rangarajan.
\newblock A graduated assignment algorithm for graph matching.
\newblock \emph{IEEE Transactions on Pattern Analysis and Machine
  Intelligence}, 18\penalty0 (4):\penalty0 377--388, April 1996.

\bibitem[Gold et~al.(1998)Gold, Rangarajan, Lu, Pappu, and
  Mjolsness]{gold1998new}
Steven Gold, Anand Rangarajan, Chien-Ping Lu, Suguna Pappu, and Eric Mjolsness.
\newblock New algorithms for 2d and 3d point matching: pose estimation and
  correspondence.
\newblock \emph{Pattern Recognition}, 31\penalty0 (8):\penalty0 1019--1031,
  1998.

\bibitem[G{\'o}mez et~al.(2003)G{\'o}mez, G{\'o}mez-Villegas, and
  Mar{\'\i}n]{gomez2003survey}
Eusebio G{\'o}mez, Miguel~A G{\'o}mez-Villegas, and J~Miguel Mar{\'\i}n.
\newblock A survey on continuous elliptical vector distributions.
\newblock \emph{Rev. Mat. Complut}, 16:\penalty0 345--361, 2003.

\bibitem[Gori-Giorgi et~al.(2009)Gori-Giorgi, Seidl, and Vignale]{GorSeiVig}
Paola Gori-Giorgi, Michael Seidl, and Giovanni Vignale.
\newblock Density-functional theory for strongly interacting electrons.
\newblock \emph{Physical Review Letters}, 103\penalty0 (16):\penalty0 166402,
  2009.

\bibitem[Gramfort et~al.(2015)Gramfort, Peyr{\'{e}}, and Cuturi]{GramfortPC15}
Alexandre Gramfort, Gabriel Peyr{\'{e}}, and Marco Cuturi.
\newblock Fast optimal transport averaging of neuroimaging data.
\newblock In \emph{Information Processing in Medical Imaging - 24th
  International Conference, {IPMI} 2015}, pages 261--272, 2015.

\bibitem[Grauman and Darrell(2005)]{grauman2005pyramid}
Kristen Grauman and Trevor Darrell.
\newblock The pyramid match kernel: discriminative classification with sets of
  image features.
\newblock In \emph{Tenth IEEE International Conference on Computer Vision},
  volume~2, pages 1458--1465. IEEE, 2005.

\bibitem[Grave et~al.(2019)Grave, Joulin, and Berthet]{grave2018unsupervised}
Edouard Grave, Armand Joulin, and Quentin Berthet.
\newblock Unsupervised alignment of embeddings with wasserstein procrustes.
\newblock In \emph{Proceedings of the 22th International Conference on
  Artificial Intelligence and Statistics}, 2019.

\bibitem[Gretton et~al.(2007)Gretton, Borgwardt, Rasch, Sch{\"o}lkopf, and
  Smola]{gretton2007kernel}
Arthur Gretton, Karsten~M Borgwardt, Malte Rasch, Bernhard Sch{\"o}lkopf, and
  Alex~J Smola.
\newblock A kernel method for the two-sample-problem.
\newblock In \emph{Advances in Neural Information Processing Systems}, pages
  513--520, 2007.

\bibitem[Gretton et~al.(2012)Gretton, Borgwardt, Rasch, Sch{\"o}lkopf, and
  Smola]{gretton2012kernel}
Arthur Gretton, Karsten~M Borgwardt, Malte~J Rasch, Bernhard Sch{\"o}lkopf, and
  Alexander Smola.
\newblock A kernel two-sample test.
\newblock \emph{Journal of Machine Learning Research}, 13\penalty0
  (Mar):\penalty0 723--773, 2012.

\bibitem[Griewank(1992)]{griewank1992achieving}
Andreas Griewank.
\newblock Achieving logarithmic growth of temporal and spatial complexity in
  reverse automatic differentiation.
\newblock \emph{Optimization Methods and Software}, 1\penalty0 (1):\penalty0
  35--54, 1992.

\bibitem[Griewank and Walther(2008)]{griewank2008evaluating}
Andreas Griewank and Andrea Walther.
\newblock \emph{Evaluating Derivatives: Principles and Techniques of
  Algorithmic Differentiation}.
\newblock SIAM, 2008.

\bibitem[Gromov(2001)]{gromov-2001}
Mikhail Gromov.
\newblock \emph{Metric Structures for {Riemannian} and Non-Rie\-man\-nian
  Spaces}.
\newblock Progress in Mathematics. Birkh{\"a}user, 2001.

\bibitem[Guo and Obloj(2017)]{guo2017computational}
Gaoyue Guo and Jan Obloj.
\newblock Computational methods for martingale optimal transport problems.
\newblock \emph{arXiv preprint arXiv:1710.07911}, 2017.

\bibitem[Gurvits(2004)]{gurvits2004classical}
Leonid Gurvits.
\newblock Classical complexity and quantum entanglement.
\newblock \emph{Journal of Computer and System Sciences}, 69\penalty0
  (3):\penalty0 448--484, 2004.

\bibitem[Guti{\'e}rrez(2016)]{gutierrez2016monge}
Cristian~E Guti{\'e}rrez.
\newblock \emph{The {Monge-Ampere} Equation}.
\newblock Springer, 2016.

\bibitem[Gutierrez et~al.(2017)Gutierrez, Rabin, Galerne, and
  Hurtut]{gutierrez2017optimal}
Jorge Gutierrez, Julien Rabin, Bruno Galerne, and Thomas Hurtut.
\newblock Optimal patch assignment for statistically constrained texture
  synthesis.
\newblock In \emph{International Conference on Scale Space and Variational
  Methods in Computer Vision}, pages 172--183. Springer, 2017.

\bibitem[Hadjidimos(2000)]{hadjidimos2000successive}
A~Hadjidimos.
\newblock Successive overrelaxation ({SOR}) and related methods.
\newblock \emph{Journal of Computational and Applied Mathematics}, 123\penalty0
  (1):\penalty0 177--199, 2000.

\bibitem[Haker et~al.(2004)Haker, Zhu, Tannenbaum, and
  Angenent]{haker2004optimal}
Steven Haker, Lei Zhu, Allen Tannenbaum, and Sigurd Angenent.
\newblock Optimal mass transport for registration and warping.
\newblock \emph{International Journal of Computer Vision}, 60\penalty0
  (3):\penalty0 225--240, 2004.

\bibitem[Hanin(1992)]{hanin1992kantorovich}
Leonid~G Hanin.
\newblock {K}antorovich-{R}ubinstein norm and its application in the theory of
  {L}ipschitz spaces.
\newblock \emph{Proceedings of the American Mathematical Society}, 115\penalty0
  (2):\penalty0 345--352, 1992.

\bibitem[Hashimoto et~al.(2016)Hashimoto, Gifford, and
  Jaakkola]{hashimoto2016learning}
Tatsunori Hashimoto, David Gifford, and Tommi Jaakkola.
\newblock Learning population-level diffusions with generative {RNNs}.
\newblock In \emph{International Conference on Machine Learning}, pages
  2417--2426, 2016.

\bibitem[Hilbert(1895)]{hilbert1895gerade}
David Hilbert.
\newblock {\"U}ber die gerade linie als k{\"u}rzeste verbindung zweier punkte.
\newblock \emph{Mathematische Annalen}, 46\penalty0 (1):\penalty0 91--96, 1895.

\bibitem[Hitchcock(1941)]{Hitchcock41}
Frank~L Hitchcock.
\newblock The distribution of a product from several sources to numerous
  localities.
\newblock \emph{Studies in Applied Mathematics}, 20\penalty0 (1-4):\penalty0
  224--230, 1941.

\bibitem[Ho et~al.(2017)Ho, Nguyen, Yurochkin, Bui, Huynh, and
  Phung]{ho2017multilevel}
Nhat Ho, XuanLong Nguyen, Mikhail Yurochkin, Hung~Hai Bui, Viet Huynh, and Dinh
  Phung.
\newblock Multilevel clustering via wasserstein means.
\newblock In \emph{International Conference on Machine Learning}, pages
  1501--1509, 2017.

\bibitem[Hofmann et~al.(2008)Hofmann, Sch{\"o}lkopf, and Smola]{Hofmann2008}
Thomas Hofmann, Bernhard Sch{\"o}lkopf, and Alexander~J Smola.
\newblock Kernel methods in machine learning.
\newblock \emph{Annals of Statistics}, 36\penalty0 (3):\penalty0 1171--1220,
  2008.

\bibitem[Hosmer~Jr et~al.(2013)Hosmer~Jr, Lemeshow, and
  Sturdivant]{hosmer2013applied}
David~W Hosmer~Jr, Stanley Lemeshow, and Rodney~X Sturdivant.
\newblock \emph{Applied Logistic Regression}, volume 398.
\newblock John Wiley \& Sons, 2013.

\bibitem[Huang et~al.(2016)Huang, Guo, Kusner, Sun, Sha, and
  Weinberger]{huang2016supervised}
Gao Huang, Chuan Guo, Matt~J Kusner, Yu~Sun, Fei Sha, and Kilian~Q Weinberger.
\newblock Supervised word mover's distance.
\newblock In \emph{Advances in Neural Information Processing Systems}, pages
  4862--4870, 2016.

\bibitem[Idel(2016)]{ReviewSinkhorn}
Martin Idel.
\newblock A review of matrix scaling and {Sinkhorn}'s normal form for matrices
  and positive maps.
\newblock \emph{arXiv preprint arXiv:1609.06349}, 2016.

\bibitem[Indyk and Price(2011)]{indyk2011k}
Piotr Indyk and Eric Price.
\newblock K-median clustering, model-based compressive sensing, and sparse
  recovery for earth mover distance.
\newblock In \emph{Proceedings of the forty-third annual ACM Symposium on
  Theory of Computing}, pages 627--636. ACM, 2011.

\bibitem[Indyk and Thaper(2003)]{indyk}
Piotr Indyk and Nitin Thaper.
\newblock Fast image retrieval via embeddings.
\newblock In \emph{3rd International Workshop on Statistical and Computational
  Theories of Vision}, 2003.

\bibitem[Jiang et~al.(2012)Jiang, Ning, and Georgiou]{JiangSpectral}
Xianhua Jiang, Lipeng Ning, and Tryphon~T Georgiou.
\newblock Distances and {R}iemannian metrics for multivariate spectral
  densities.
\newblock \emph{IEEE Transactions on Automatic Control}, 57\penalty0
  (7):\penalty0 1723--1735, 2012.

\bibitem[Johnson and Lindenstrauss(1984)]{Johnson84}
William~B Johnson and Joram Lindenstrauss.
\newblock Extensions of {L}ipschitz mappings into a {H}ilbert space.
\newblock In \emph{Conference in Modern Analysis and Probability ({N}ew
  {H}aven, {C}onn., 1982)}, volume~26 of \emph{Contemporary Mathematics}, pages
  189--206. American Mathematical Society, 1984.

\bibitem[Jordan et~al.(1998)Jordan, Kinderlehrer, and
  Otto]{jordan1998variational}
Richard Jordan, David Kinderlehrer, and Felix Otto.
\newblock The variational formulation of the {Fokker}-{Planck} equation.
\newblock \emph{SIAM Journal on Mathematical Analysis}, 29\penalty0
  (1):\penalty0 1--17, 1998.

\bibitem[Kantorovich(1942)]{Kantorovich42}
Leonid Kantorovich.
\newblock On the transfer of masses (in russian).
\newblock \emph{Doklady Akademii Nauk}, 37\penalty0 (2):\penalty0 227--229,
  1942.

\bibitem[Kantorovich and Rubinstein(1958)]{kantorovich1958space}
LV~Kantorovich and G.S. Rubinstein.
\newblock On a space of totally additive functions.
\newblock \emph{Vestn Leningrad Universitet}, 13:\penalty0 52--59, 1958.

\bibitem[Karcher(2014)]{karcher2014riemannian}
Hermann Karcher.
\newblock Riemannian center of mass and so called {Karcher} mean.
\newblock \emph{arXiv preprint arXiv:1407.2087}, 2014.

\bibitem[Karlsson and Ringh(2016)]{karlsson2016generalized}
Johan Karlsson and Axel Ringh.
\newblock Generalized {Sinkhorn} iterations for regularizing inverse problems
  using optimal mass transport.
\newblock \emph{arXiv preprint arXiv:1612.02273}, 2016.

\bibitem[Kim et~al.(2013)Kim, Ma, Mesa, and Coleman]{kim2013efficient}
Sanggyun Kim, Rui Ma, Diego Mesa, and Todd~P Coleman.
\newblock Efficient bayesian inference methods via convex optimization and
  optimal transport.
\newblock In \emph{IEEE International Symposium on Information Theory}, pages
  2259--2263. IEEE, 2013.

\bibitem[Kinderlehrer and Walkington(1999)]{kinderlehrer1999approximation}
David Kinderlehrer and Noel~J Walkington.
\newblock Approximation of parabolic equations using the {Wasserstein} metric.
\newblock \emph{ESAIM: Mathematical Modelling and Numerical Analysis},
  33\penalty0 (04):\penalty0 837--852, 1999.

\bibitem[Kitagawa et~al.(2016)Kitagawa, M{\'e}rigot, and
  Thibert]{kitagawa2016newton}
Jun Kitagawa, Quentin M{\'e}rigot, and Boris Thibert.
\newblock A {Newton} algorithm for semi-discrete optimal transport.
\newblock \emph{arXiv preprint arXiv:1603.05579}, 2016.

\bibitem[Knight(2008)]{knight2008sinkhorn}
Philip~A Knight.
\newblock The {Sinkhorn--Knopp} algorithm: convergence and applications.
\newblock \emph{SIAM Journal on Matrix Analysis and Applications}, 30\penalty0
  (1):\penalty0 261--275, 2008.

\bibitem[Knight and Ruiz(2013)]{knight2013fast}
Philip~A Knight and Daniel Ruiz.
\newblock A fast algorithm for matrix balancing.
\newblock \emph{IMA Journal of Numerical Analysis}, 33\penalty0 (3):\penalty0
  1029--1047, 2013.

\bibitem[Knight et~al.(2014)Knight, Ruiz, and U{\c{c}}ar]{knight2014symmetry}
Philip~A Knight, Daniel Ruiz, and Bora U{\c{c}}ar.
\newblock A symmetry preserving algorithm for matrix scaling.
\newblock \emph{SIAM Journal on Matrix Analysis and Applications}, 35\penalty0
  (3):\penalty0 931--955, 2014.

\bibitem[Knott and Smith(1984)]{knott1984optimal}
Martin Knott and Cyril~S Smith.
\newblock On the optimal mapping of distributions.
\newblock \emph{Journal of Optimization Theory and Applications}, 43\penalty0
  (1):\penalty0 39--49, 1984.

\bibitem[Knott and Smith(1994)]{KnottSmith}
Martin Knott and Cyril~S Smith.
\newblock On a generalization of cyclic monotonicity and distances among random
  vectors.
\newblock \emph{Linear Algebra and Its Applications}, 199:\penalty0 363--371,
  1994.

\bibitem[Kolouri et~al.(2016)Kolouri, Zou, and Rohde]{kolouri2016sliced}
Soheil Kolouri, Yang Zou, and Gustavo~K Rohde.
\newblock Sliced {Wasserstein} kernels for probability distributions.
\newblock In \emph{Proceedings of the IEEE Conference on Computer Vision and
  Pattern Recognition}, pages 5258--5267, 2016.

\bibitem[Kolouri et~al.(2017)Kolouri, Park, Thorpe, Slepcev, and
  Rohde]{kolouri2017optimal}
Soheil Kolouri, Se~Rim Park, Matthew Thorpe, Dejan Slepcev, and Gustavo~K
  Rohde.
\newblock Optimal mass transport: signal processing and machine-learning
  applications.
\newblock \emph{IEEE Signal Processing Magazine}, 34\penalty0 (4):\penalty0
  43--59, 2017.

\bibitem[Kondratyev et~al.(2016)Kondratyev, Monsaingeon, and
  Vofnikov]{kondratyev2015}
Stanislav Kondratyev, L{\'e}onard Monsaingeon, and Dmitry Vofnikov.
\newblock A new optimal transport distance on the space of finite {Radon}
  measures.
\newblock \emph{Advances in Differential Equations}, 21\penalty0
  (11/12):\penalty0 1117--1164, 2016.

\bibitem[Koopmans(1949)]{koopmans1949optimum}
Tjalling~C Koopmans.
\newblock Optimum utilization of the transportation system.
\newblock \emph{Econometrica: Journal of the Econometric Society}, pages
  136--146, 1949.

\bibitem[Korman and McCann(2015)]{km1}
Jonathan Korman and Robert McCann.
\newblock Optimal transportation with capacity constraints.
\newblock \emph{Transactions of the American Mathematical Society},
  367\penalty0 (3):\penalty0 1501--1521, 2015.

\bibitem[Korte and Vygen(2012)]{korte2012combinatorial}
Bernhard Korte and Jens Vygen.
\newblock \emph{Combinatorial Optimization}.
\newblock Springer, 2012.

\bibitem[Kosowsky and Yuille(1994)]{kosowsky1994invisible}
JJ~Kosowsky and Alan~L Yuille.
\newblock The invisible hand algorithm: Solving the assignment problem with
  statistical physics.
\newblock \emph{Neural networks}, 7\penalty0 (3):\penalty0 477--490, 1994.

\bibitem[Kruithof(1937)]{kruithof}
J.~Kruithof.
\newblock Telefoonverkeersrekening.
\newblock \emph{De Ingenieur}, 52:\penalty0 E15--E25, 1937.

\bibitem[Kuhn(1955)]{Kuhn1955}
Harold~W. Kuhn.
\newblock The hungarian method for the assignment problem.
\newblock \emph{Naval Research Logistics Quarterly}, 2:\penalty0 83--97, 1955.

\bibitem[Kulis(2012)]{MAL-019}
Brian Kulis.
\newblock Metric learning: a survey.
\newblock \emph{Foundations and Trends in Machine Learning}, 5\penalty0
  (4):\penalty0 287--364, 2012.

\bibitem[Kusner et~al.(2015)Kusner, Sun, Kolkin, and
  Weinberger]{kusner2015word}
Matt Kusner, Yu~Sun, Nicholas Kolkin, and Kilian Weinberger.
\newblock From word embeddings to document distances.
\newblock In \emph{International Conference on Machine Learning}, pages
  957--966, 2015.

\bibitem[Lacombe et~al.(2018)Lacombe, Cuturi, and Oudot]{NIPS2018_8184}
Theo Lacombe, Marco Cuturi, and Steve Oudot.
\newblock Large scale computation of means and clusters for persistence
  diagrams using optimal transport.
\newblock \emph{Advances in Neural Information Processing Systems 31}, pages
  9792--9802, 2018.

\bibitem[Lai and Zhao(2017)]{lai2014multi}
Rongjie Lai and Hongkai Zhao.
\newblock Multiscale nonrigid point cloud registration using rotation-invariant
  sliced-wasserstein distance via laplace--beltrami eigenmap.
\newblock \emph{SIAM Journal on Imaging Sciences}, 10\penalty0 (2):\penalty0
  449--483, 2017.

\bibitem[Lavenant(2017)]{Lavenant2017}
Hugo Lavenant.
\newblock Harmonic mappings valued in the {Wasserstein} space.
\newblock \emph{Preprint cvgmt 3649}, 2017.

\bibitem[Lazebnik et~al.(2006)Lazebnik, Schmid, and Ponce]{lazebnik2006beyond}
Svetlana Lazebnik, Cordelia Schmid, and Jean Ponce.
\newblock Beyond bags of features: Spatial pyramid matching for recognizing
  natural scene categories.
\newblock In \emph{IEEE Computer Society Conference on Computer Vision and
  Pattern Recognition}, volume~2, pages 2169--2178. IEEE, 2006.

\bibitem[Le~Gouic and Loubes(2016)]{leGouic2016existence}
Thibaut Le~Gouic and Jean-Michel Loubes.
\newblock Existence and consistency of {Wasserstein} barycenters.
\newblock \emph{Probability Theory and Related Fields}, 168:\penalty0 901--917,
  2016.

\bibitem[Lee and Seung(1999)]{lee1999learning}
Daniel~D Lee and H~Sebastian Seung.
\newblock Learning the parts of objects by non-negative matrix factorization.
\newblock \emph{Nature}, 401\penalty0 (6755):\penalty0 788--791, 1999.

\bibitem[Lee and Raginsky(2018)]{NIPS2018_7534}
Jaeho Lee and Maxim Raginsky.
\newblock Minimax statistical learning with wasserstein distances.
\newblock In S.~Bengio, H.~Wallach, H.~Larochelle, K.~Grauman, N.~Cesa-Bianchi,
  and R.~Garnett, editors, \emph{Advances in Neural Information Processing
  Systems 31}, pages 2692--2701. 2018.

\bibitem[Leeb and Coifman(2016)]{leeb2016holder}
William Leeb and Ronald Coifman.
\newblock H{\"o}lder--{Lipschitz} norms and their duals on spaces with
  semigroups, with applications to earth mover's distance.
\newblock \emph{Journal of Fourier Analysis and Applications}, 22\penalty0
  (4):\penalty0 910--953, 2016.

\bibitem[Lellmann et~al.(2014)Lellmann, Lorenz, Sch\"onlieb, and
  Valkonen]{lellmann2014imaging}
Jan Lellmann, Dirk~A Lorenz, Carola Sch\"onlieb, and Tuomo Valkonen.
\newblock Imaging with {Kantorovich--Rubinstein} discrepancy.
\newblock \emph{SIAM Journal on Imaging Sciences}, 7\penalty0 (4):\penalty0
  2833--2859, 2014.

\bibitem[Lemmens and Nussbaum(2012)]{lemmens2012nonlinear}
Bas Lemmens and Roger Nussbaum.
\newblock \emph{Nonlinear Perron-Frobenius Theory}, volume 189.
\newblock Cambridge University Press, 2012.

\bibitem[L{\'e}onard(2012)]{leonard2012schrodinger}
Christian L{\'e}onard.
\newblock From the {Schr{\"o}dinger} problem to the {Monge--Kantorovich}
  problem.
\newblock \emph{Journal of Functional Analysis}, 262\penalty0 (4):\penalty0
  1879--1920, 2012.

\bibitem[L{\'e}onard(2014)]{LeonardSchroedinger}
Christian L{\'e}onard.
\newblock A survey of the {Schr\"odinger} problem and some of its connections
  with optimal transport.
\newblock \emph{Discrete Continuous Dynamical Systems Series A}, 34\penalty0
  (4):\penalty0 1533--1574, 2014.

\bibitem[L{\'e}vy(2015)]{levy2015numerical}
Bruno L{\'e}vy.
\newblock A numerical algorithm for $l^2$ semi-discrete optimal transport in
  3d.
\newblock \emph{ESAIM: Mathematical Modelling and Numerical Analysis},
  49\penalty0 (6):\penalty0 1693--1715, 2015.

\bibitem[L{\'e}vy and Schwindt(2018)]{Levy2017review}
Bruno L{\'e}vy and Erica~L Schwindt.
\newblock Notions of optimal transport theory and how to implement them on a
  computer.
\newblock \emph{Computers \& Graphics}, 72:\penalty0 135--148, 2018.

\bibitem[Li et~al.(2013)Li, Wang, and Zhang]{li2013novel}
Peihua Li, Qilong Wang, and Lei Zhang.
\newblock A novel earth mover's distance methodology for image matching with
  {Gaussian} mixture models.
\newblock In \emph{Proceedings of the IEEE International Conference on Computer
  Vision}, pages 1689--1696, 2013.

\bibitem[Li et~al.(2018{\natexlab{a}})Li, Ryu, Osher, Yin, and
  Gangbo]{LiRyuOsherYinGangbo2017_parallel}
Wuchen Li, Ernest~K. Ryu, Stanley Osher, Wotao Yin, and Wilfrid Gangbo.
\newblock A parallel method for {Earth Mover}'s distance.
\newblock \emph{Journal of Scientific Computing}, 75\penalty0 (1):\penalty0
  182--197, 2018{\natexlab{a}}.

\bibitem[Li et~al.(2018{\natexlab{b}})Li, Li, and Cao]{liimage}
Yupeng Li, Wuchen Li, and Guo Cao.
\newblock Image segmentation via $l^1$ monge-kantorovich problem.
\newblock \emph{CAM report 17-73}, 2018{\natexlab{b}}.

\bibitem[Liero et~al.(2016)Liero, Mielke, and
  Savar{\'e}]{LieroMielkeSavareShort}
Matthias Liero, Alexander Mielke, and Giuseppe Savar{\'e}.
\newblock Optimal transport in competition with reaction: the
  {Hellinger--Kantorovich} distance and geodesic curves.
\newblock \emph{SIAM Journal on Mathematical Analysis}, 48\penalty0
  (4):\penalty0 2869--2911, 2016.

\bibitem[Liero et~al.(2018)Liero, Mielke, and
  Savar{\'e}]{LieroMielkeSavareLong}
Matthias Liero, Alexander Mielke, and Giuseppe Savar{\'e}.
\newblock Optimal entropy-transport problems and a new hellinger--kantorovich
  distance between positive measures.
\newblock \emph{Inventiones Mathematicae}, 211\penalty0 (3):\penalty0
  969--1117, 2018.

\bibitem[Ling and Okada(2006)]{ling2006diffusion}
Haibin Ling and Kazunori Okada.
\newblock Diffusion distance for histogram comparison.
\newblock In \emph{IEEE Computer Society Conference on Computer Vision and
  Pattern Recognition}, volume~1, pages 246--253. IEEE, 2006.

\bibitem[Ling and Okada(2007)]{TreeEMD2007}
Haibin Ling and Kazunori Okada.
\newblock An efficient earth mover's distance algorithm for robust histogram
  comparison.
\newblock \emph{IEEE Transactions on Pattern Analysis and Machine
  Intelligence}, 29\penalty0 (5):\penalty0 840--853, 2007.

\bibitem[Linial et~al.(1998)Linial, Samorodnitsky, and
  Wigderson]{linial1998deterministic}
Nathan Linial, Alex Samorodnitsky, and Avi Wigderson.
\newblock A deterministic strongly polynomial algorithm for matrix scaling and
  approximate permanents.
\newblock In \emph{Proceedings of the Thirtieth Annual ACM Symposium on Theory
  of Computing}, pages 644--652. ACM, 1998.

\bibitem[Lions and Mercier(1979)]{Lions-Mercier-DR}
Pierre-Louis Lions and Bertrand Mercier.
\newblock Splitting algorithms for the sum of two nonlinear operators.
\newblock \emph{SIAM Journal on Numerical Analysis}, 16:\penalty0 964--979,
  1979.

\bibitem[Loftsgaarden and Quesenberry(1965)]{loftsgaarden1965nonparametric}
Don~O Loftsgaarden and Charles~P Quesenberry.
\newblock A nonparametric estimate of a multivariate density function.
\newblock \emph{Annals of Mathematical Statistics}, 36\penalty0 (3):\penalty0
  1049--1051, 1965.

\bibitem[Loiola et~al.(2007)Loiola, de~Abreu, Boaventura-Netto, Hahn, and
  Querido]{loiola-2007}
Eliane~Maria Loiola, Nair Maria~Maia de~Abreu, Paulo~Oswaldo Boaventura-Netto,
  Peter Hahn, and Tania Querido.
\newblock A survey for the quadratic assignment problem.
\newblock \emph{European Journal Operational Research}, 176\penalty0
  (2):\penalty0 657--690, 2007.

\bibitem[Lyzinski et~al.(2016)Lyzinski, Fishkind, Fiori, Vogelstein, Priebe,
  and Sapiro]{lyzinski-2015}
Vince Lyzinski, Donniell~E Fishkind, Marcelo Fiori, Joshua~T Vogelstein,
  Carey~E Priebe, and Guillermo Sapiro.
\newblock Graph matching: relax at your own risk.
\newblock \emph{IEEE Transactions on Pattern Analysis and Machine
  Intelligence}, 38\penalty0 (1):\penalty0 60--73, 2016.

\bibitem[Maas(2011)]{Maas2011}
Jan Maas.
\newblock Gradient flows of the entropy for finite {Markov} chains.
\newblock \emph{Journal of Functional Analysis}, 261\penalty0 (8):\penalty0
  2250--2292, 2011.

\bibitem[Maas et~al.(2015)Maas, Rumpf, Sch{\"o}nlieb, and
  Simon]{maas2015generalized}
Jan Maas, Martin Rumpf, Carola Sch{\"o}nlieb, and Stefan Simon.
\newblock A generalized model for optimal transport of images including
  dissipation and density modulation.
\newblock \emph{ESAIM: Mathematical Modelling and Numerical Analysis},
  49\penalty0 (6):\penalty0 1745--1769, 2015.

\bibitem[Maas et~al.(2016)Maas, Rumpf, and Simon]{maas2016generalized}
Jan Maas, Martin Rumpf, and Stefan Simon.
\newblock Generalized optimal transport with singular sources.
\newblock \emph{arXiv preprint arXiv:1607.01186}, 2016.

\bibitem[Makihara and Yagi(2010)]{makihara2010earth}
Yasushi Makihara and Yasushi Yagi.
\newblock Earth mover's morphing: Topology-free shape morphing using
  cluster-based {EMD} flows.
\newblock In \emph{Asian Conference on Computer Vision}, pages 202--215.
  Springer, 2010.

\bibitem[Malag{\`o} et~al.(2018)Malag{\`o}, Montrucchio, and
  Pistone]{malago2018wasserstein}
Luigi Malag{\`o}, Luigi Montrucchio, and Giovanni Pistone.
\newblock Wasserstein riemannian geometry of positive-definite matrices.
\newblock \emph{arXiv preprint arXiv:1801.09269}, 2018.

\bibitem[Mallasto and Feragen(2017)]{NIPS2017_7149}
Anton Mallasto and Aasa Feragen.
\newblock Learning from uncertain curves: The 2-wasserstein metric for gaussian
  processes.
\newblock In I.~Guyon, U.~V. Luxburg, S.~Bengio, H.~Wallach, R.~Fergus,
  S.~Vishwanathan, and R.~Garnett, editors, \emph{Advances in Neural
  Information Processing Systems 30}, pages 5660--5670. 2017.

\bibitem[Mallat(2008)]{mallat2008wavelet}
Stephane Mallat.
\newblock \emph{A Wavelet Tour of Signal Processing: the Sparse Way}.
\newblock Academic press, 2008.

\bibitem[Mathon et~al.(2014)Mathon, Cayre, Bas, and Macq]{mathon2014optimal}
Benjamin Mathon, Francois Cayre, Patrick Bas, and Benoit Macq.
\newblock Optimal transport for secure spread-spectrum watermarking of still
  images.
\newblock \emph{IEEE Transactions on Image Processing}, 23\penalty0
  (4):\penalty0 1694--1705, 2014.

\bibitem[Matthes and Osberger(2014)]{Matthes1D}
Daniel Matthes and Horst Osberger.
\newblock Convergence of a variational {Lagrangian} scheme for a nonlinear
  drift diffusion equation.
\newblock \emph{ESAIM: Mathematical Modelling and Numerical Analysis},
  48\penalty0 (3):\penalty0 697--726, 2014.

\bibitem[Matthes and Osberger(2017)]{matthes2017convergent}
Daniel Matthes and Horst Osberger.
\newblock A convergent lagrangian discretization for a nonlinear fourth-order
  equation.
\newblock \emph{Foundations of Computational Mathematics}, 17\penalty0
  (1):\penalty0 73--126, 2017.

\bibitem[Maury and Preux(2017)]{maury201713}
Bertrand Maury and Anthony Preux.
\newblock Pressureless {Euler} equations with maximal density constraint: a
  time-splitting scheme.
\newblock \emph{Topological Optimization and Optimal Transport: In the Applied
  Sciences}, 17:\penalty0 333, 2017.

\bibitem[Maury et~al.(2010)Maury, Roudneff-Chupin, and
  Santambrogio]{maury2010macroscopic}
Bertrand Maury, Aude Roudneff-Chupin, and Filippo Santambrogio.
\newblock A macroscopic crowd motion model of gradient flow type.
\newblock \emph{Mathematical Models and Methods in Applied Sciences},
  20\penalty0 (10):\penalty0 1787--1821, 2010.

\bibitem[McCann(1997)]{mccann1997convexity}
Robert~J McCann.
\newblock A convexity principle for interacting gases.
\newblock \emph{Advances in Mathematics}, 128\penalty0 (1):\penalty0 153--179,
  1997.

\bibitem[M{\'e}moli(2007)]{memoli-2007}
Facundo M{\'e}moli.
\newblock On the use of {G}romov--{H}ausdorff distances for shape comparison.
\newblock In \emph{Symposium on Point Based Graphics}, pages 81--90. 2007.

\bibitem[M{\'e}moli(2011)]{memoli-2011}
Facundo M{\'e}moli.
\newblock Gromov--{W}asserstein distances and the metric approach to object
  matching.
\newblock \emph{Foundations of Computational Mathematics}, 11\penalty0
  (4):\penalty0 417--487, 2011.

\bibitem[M{\'e}moli and Sapiro(2005)]{memoli2005theoretical}
Facundo M{\'e}moli and Guillermo Sapiro.
\newblock A theoretical and computational framework for isometry invariant
  recognition of point cloud data.
\newblock \emph{Foundations of Computational Mathematics}, 5\penalty0
  (3):\penalty0 313--347, 2005.

\bibitem[M{\'e}rigot(2011)]{Merigot11}
Quentin M{\'e}rigot.
\newblock A multiscale approach to optimal transport.
\newblock \emph{Computer Graphics Forum}, 30\penalty0 (5):\penalty0 1583--1592,
  2011.

\bibitem[M{\'e}tivier et~al.(2016)M{\'e}tivier, Brossier, Merigot, Oudet, and
  Virieux]{metivier2016optimal}
Ludovic M{\'e}tivier, Romain Brossier, Quentin Merigot, Edouard Oudet, and Jean
  Virieux.
\newblock An optimal transport approach for seismic tomography: Application to
  {3D} full waveform inversion.
\newblock \emph{Inverse Problems}, 32\penalty0 (11):\penalty0 115008, 2016.

\bibitem[Meyron et~al.(2018)Meyron, M{\'e}rigot, and Thibert]{merigot2017light}
Jocelyn Meyron, Quentin M{\'e}rigot, and Boris Thibert.
\newblock Light in power: a general and parameter-free algorithm for caustic
  design.
\newblock In \emph{SIGGRAPH Asia 2018 Technical Papers}, page 224. ACM, 2018.

\bibitem[Mielke(2013)]{MielkeCVPDE}
Alexander Mielke.
\newblock Geodesic convexity of the relative entropy in reversible {Markov}
  chains.
\newblock \emph{Calculus of Variations and Partial Differential Equations},
  48\penalty0 (1-2):\penalty0 1--31, 2013.

\bibitem[Mikolov et~al.(2013)Mikolov, Chen, Corrado, and
  Dean]{mikolov2013efficient}
Tomas Mikolov, Kai Chen, Greg Corrado, and Jeffrey Dean.
\newblock Efficient estimation of word representations in vector space.
\newblock \emph{arXiv preprint arXiv:1301.3781}, 2013.

\bibitem[Mirebeau(2015)]{mirebeau2015discretization}
Jean-Marie Mirebeau.
\newblock Discretization of the {3D} {Monge-Ampere} operator, between wide
  stencils and power diagrams.
\newblock \emph{ESAIM: Mathematical Modelling and Numerical Analysis},
  49\penalty0 (5):\penalty0 1511--1523, 2015.

\bibitem[Monge(1781)]{Monge1781}
Gaspard Monge.
\newblock M{\'e}moire sur la th{\'e}orie des d{\'e}blais et des remblais.
\newblock \emph{Histoire de l'Acad{\'e}mie Royale des Sciences}, pages
  666--704, 1781.

\bibitem[Montavon et~al.(2016)Montavon, M\"{u}ller, and Cuturi]{CuturiBoltzman}
Gr\'{e}goire Montavon, Klaus-Robert M\"{u}ller, and Marco Cuturi.
\newblock Wasserstein training of restricted {Boltzmann} machines.
\newblock In D.~D. Lee, M.~Sugiyama, U.~V. Luxburg, I.~Guyon, and R.~Garnett,
  editors, \emph{Advances in Neural Information Processing Systems 29}, pages
  3718--3726. 2016.

\bibitem[Moon and Hero(2014)]{BerishaHero}
Kevin Moon and Alfred Hero.
\newblock Multivariate $f$-divergence estimation with confidence.
\newblock In \emph{Advances in Neural Information Processing Systems}, pages
  2420--2428, 2014.

\bibitem[Museyko et~al.(2009)Museyko, Stiglmayr, Klamroth, and
  Leugering]{museyko2009application}
Oleg Museyko, Michael Stiglmayr, Kathrin Klamroth, and G{\"u}nter Leugering.
\newblock On the application of the {Monge--Kantorovich} problem to image
  registration.
\newblock \emph{SIAM Journal on Imaging Sciences}, 2\penalty0 (4):\penalty0
  1068--1097, 2009.

\bibitem[Muzellec and Cuturi(2018)]{NIPS2018_8226}
Boris Muzellec and Marco Cuturi.
\newblock Generalizing point embeddings using the wasserstein space of
  elliptical distributions.
\newblock In S.~Bengio, H.~Wallach, H.~Larochelle, K.~Grauman, N.~Cesa-Bianchi,
  and R.~Garnett, editors, \emph{Advances in Neural Information Processing
  Systems 31}, pages 10258--10269. 2018.

\bibitem[Muzellec et~al.(2017)Muzellec, Nock, Patrini, and
  Nielsen]{muzellec2017tsallis}
Boris Muzellec, Richard Nock, Giorgio Patrini, and Frank Nielsen.
\newblock Tsallis regularized optimal transport and ecological inference.
\newblock In \emph{AAAI}, pages 2387--2393, 2017.

\bibitem[Myronenko and Song(2010)]{myronenko2010point}
Andriy Myronenko and Xubo Song.
\newblock Point set registration: coherent point drift.
\newblock \emph{IEEE Transactions on Pattern Analysis and Machine
  Intelligence}, 32\penalty0 (12):\penalty0 2262--2275, 2010.

\bibitem[Naor and Schechtman(2007)]{naor-2005}
Assaf Naor and Gideon Schechtman.
\newblock Planar earthmover is not in l$_{\mbox{1}}$.
\newblock \emph{SIAM Journal on Computing}, 37\penalty0 (3):\penalty0 804--826,
  2007.

\bibitem[Neidinger(2010)]{Neidinger10}
Richard~D Neidinger.
\newblock Introduction to automatic differentiation and {Matlab}
  object-oriented programming.
\newblock \emph{SIAM Review}, 52\penalty0 (3):\penalty0 545--563, 2010.

\bibitem[Nemirovski and Rothblum(1999)]{nemirovski1999complexity}
Arkadi Nemirovski and Uriel Rothblum.
\newblock On complexity of matrix scaling.
\newblock \emph{Linear Algebra and its Applications}, 302:\penalty0 435--460,
  1999.

\bibitem[Nesterov and Nemirovskii(1994)]{nesterov1994interior}
Yurii Nesterov and Arkadii Nemirovskii.
\newblock \emph{Interior-point polynomial algorithms in convex programming},
  volume~13.
\newblock SIAM, 1994.

\bibitem[Ni et~al.(2009)Ni, Bresson, Chan, and Esedoglu]{ni2009local}
Kangyu Ni, Xavier Bresson, Tony Chan, and Selim Esedoglu.
\newblock Local histogram based segmentation using the {Wasserstein} distance.
\newblock \emph{International Journal of Computer Vision}, 84\penalty0
  (1):\penalty0 97--111, 2009.

\bibitem[Ning and Georgiou(2014)]{Ning2014metrics}
Lipeng Ning and Tryphon~T Georgiou.
\newblock Metrics for matrix-valued measures via test functions.
\newblock In \emph{53rd IEEE Conference on Decision and Control}, pages
  2642--2647. IEEE, 2014.

\bibitem[Ning et~al.(2015)Ning, Georgiou, and Tannenbaum]{ning2015matrix}
Lipeng Ning, Tryphon~T Georgiou, and Allen Tannenbaum.
\newblock On matrix-valued {M}onge--{K}antorovich optimal mass transport.
\newblock \emph{IEEE Transactions on Automatic Control}, 60\penalty0
  (2):\penalty0 373--382, 2015.

\bibitem[Nocedal and Wright(1999)]{nocedal}
Jorge Nocedal and Stephen~J Wright.
\newblock \emph{Numerical Optimization}.
\newblock Springer-Verlag, 1999.

\bibitem[Oberman and Ruan(2015)]{oberman2015efficient}
Adam~M Oberman and Yuanlong Ruan.
\newblock An efficient linear programming method for optimal transportation.
\newblock \emph{arXiv preprint arXiv:1509.03668}, 2015.

\bibitem[Oliker and Prussner(1989)]{oliker1989numerical}
Vladimir Oliker and Laird~D Prussner.
\newblock On the numerical solution of the equation $\frac{\partial^2
  z}{\partial x^2} \frac{\partial^2 z}{\partial y^2}-\left( \frac{\partial^2
  z}{\partial x\partial y} \right)^2=f$ and its discretizations, {I}.
\newblock \emph{Numerische Mathematik}, 54\penalty0 (3):\penalty0 271--293,
  1989.

\bibitem[Oliva and Torralba(2001)]{oliva2001modeling}
Aude Oliva and Antonio Torralba.
\newblock Modeling the shape of the scene: a holistic representation of the
  spatial envelope.
\newblock \emph{International Journal of Computer Vision}, 42\penalty0
  (3):\penalty0 145--175, 2001.

\bibitem[Oliver(2014)]{oliver2014minimization}
Dean~S Oliver.
\newblock Minimization for conditional simulation: Relationship to optimal
  transport.
\newblock \emph{Journal of Computational Physics}, 265:\penalty0 1--15, 2014.

\bibitem[Orlin(1997)]{Orlin1997}
James~B. Orlin.
\newblock A polynomial time primal network simplex algorithm for minimum cost
  flows.
\newblock \emph{Mathematical Programming}, 78\penalty0 (2):\penalty0 109--129,
  1997.

\bibitem[{\"O}sterreicher and Vajda(2003)]{osterreicher2003new}
Ferdinand {\"O}sterreicher and Igor Vajda.
\newblock A new class of metric divergences on probability spaces and its
  applicability in statistics.
\newblock \emph{Annals of the Institute of Statistical Mathematics},
  55\penalty0 (3):\penalty0 639--653, 2003.

\bibitem[Otto(2001)]{otto2001geometry}
Felix Otto.
\newblock The geometry of dissipative evolution equations: the porous medium
  equation.
\newblock \emph{Communications in Partial Differential Equations}, 26\penalty0
  (1-2):\penalty0 101--174, 2001.

\bibitem[Owen(2001)]{owen2001empirical}
Art~B Owen.
\newblock \emph{Empirical Likelihood}.
\newblock Wiley Online Library, 2001.

\bibitem[Pan and Yang(2010)]{pan2010survey}
Sinno~Jialin Pan and Qiang Yang.
\newblock A survey on transfer learning.
\newblock \emph{IEEE Transactions on knowledge and data engineering},
  22\penalty0 (10):\penalty0 1345--1359, 2010.

\bibitem[Panaretos and Zemel(2016)]{panaretos2016amplitude}
Victor~M Panaretos and Yoav Zemel.
\newblock Amplitude and phase variation of point processes.
\newblock \emph{Annals of Statistics}, 44\penalty0 (2):\penalty0 771--812,
  2016.

\bibitem[Papadakis et~al.(2014)Papadakis, Peyr{\'e}, and Oudet]{FPapPeyOud13}
Nicolas Papadakis, Gabriel Peyr{\'e}, and Edouard Oudet.
\newblock Optimal transport with proximal splitting.
\newblock \emph{SIAM Journal on Imaging Sciences}, 7\penalty0 (1):\penalty0
  212--238, 2014.

\bibitem[Pass(2012)]{PassMultiMarginalStructure}
Brendan Pass.
\newblock On the local structure of optimal measures in the multi-marginal
  optimal transportation problem.
\newblock \emph{Calculus of Variations and Partial Differential Equations},
  43\penalty0 (3-4):\penalty0 529--536, 2012.

\bibitem[Pass(2015)]{PassMultiReview}
Brendan Pass.
\newblock Multi-marginal optimal transport: theory and applications.
\newblock \emph{ESAIM: Mathematical Modelling and Numerical Analysis},
  49\penalty0 (6):\penalty0 1771--1790, 2015.

\bibitem[Pele and Taskar(2013)]{pele2013tangent}
Ofir Pele and Ben Taskar.
\newblock The tangent earth mover's distance.
\newblock In \emph{Geometric Science of Information}, pages 397--404. Springer,
  2013.

\bibitem[Pele and Werman(2008)]{pele2008linear}
Ofir Pele and Michael Werman.
\newblock A linear time histogram metric for improved sift matching.
\newblock \emph{Computer Vision--ECCV 2008}, pages 495--508, 2008.

\bibitem[Pele and Werman(2009)]{Pele-iccv2009}
Ofir Pele and Michael Werman.
\newblock Fast and robust earth mover's distances.
\newblock In \emph{IEEE 12th International Conference on Computer Vision},
  pages 460--467, 2009.

\bibitem[Perthame et~al.(2014)Perthame, Quir{\'o}s, and
  V{\'a}zquez]{PerthameTumor}
Beno{\^\i}t Perthame, Fernando Quir{\'o}s, and Juan~Luis V{\'a}zquez.
\newblock The {Hele-Shaw} asymptotics for mechanical models of tumor growth.
\newblock \emph{Archive for Rational Mechanics and Analysis}, 212\penalty0
  (1):\penalty0 93--127, 2014.

\bibitem[Peyr{\'e}(2015)]{2015-Peyre-siims}
Gabriel Peyr{\'e}.
\newblock Entropic approximation of {Wasserstein} gradient flows.
\newblock \emph{SIAM Journal on Imaging Sciences}, 8\penalty0 (4):\penalty0
  2323--2351, 2015.

\bibitem[Peyr{\'e} et~al.(2012)Peyr{\'e}, Fadili, and
  Rabin]{peyre2012wasserstein}
Gabriel Peyr{\'e}, Jalal Fadili, and Julien Rabin.
\newblock Wasserstein active contours.
\newblock In \emph{19th IEEE International Conference on Image Processing},
  pages 2541--2544. IEEE, 2012.

\bibitem[Peyr{\'e} et~al.(2016)Peyr{\'e}, Cuturi, and Solomon]{peyre2016gromov}
Gabriel Peyr{\'e}, Marco Cuturi, and Justin Solomon.
\newblock Gromov-{Wasserstein} averaging of kernel and distance matrices.
\newblock In \emph{International Conference on Machine Learning}, pages
  2664--2672, 2016.

\bibitem[Peyr{\'e} et~al.(2017)Peyr{\'e}, Chizat, Vialard, and
  Solomon]{2016-peyre-qot}
Gabriel Peyr{\'e}, Lenaic Chizat, Francois-Xavier Vialard, and Justin Solomon.
\newblock Quantum entropic regularization of matrix-valued optimal transport.
\newblock \emph{to appear in European Journal of Applied Mathematics}, 2017.

\bibitem[Peyre(2011)]{peyre2011comparison}
R{\'e}mi Peyre.
\newblock Comparison between $w_2$ distance and $h^{-1}$ norm, and localisation
  of {Wasserstein} distance.
\newblock \emph{arXiv preprint arXiv:1104.4631}, 2011.

\bibitem[Piccoli and Rossi(2014)]{piccoli2014generalized}
Benedetto Piccoli and Francesco Rossi.
\newblock Generalized {W}asserstein distance and its application to transport
  equations with source.
\newblock \emph{Archive for Rational Mechanics and Analysis}, 211\penalty0
  (1):\penalty0 335--358, 2014.

\bibitem[Piti{\'e} et~al.(2007)Piti{\'e}, Kokaram, and
  Dahyot]{pitie2007automated}
Fran{\c{c}}ois Piti{\'e}, Anil~C Kokaram, and Rozenn Dahyot.
\newblock Automated colour grading using colour distribution transfer.
\newblock \emph{Computer Vision and Image Understanding}, 107\penalty0
  (1):\penalty0 123--137, 2007.

\bibitem[Pytorch(2017)]{pytorch}
Pytorch.
\newblock Pytorch library.
\newblock \emph{http://pytorch.org/}, 2017.

\bibitem[Rabin and Papadakis(2015)]{RabinPapadakisSSVM}
Julien Rabin and Nicolas Papadakis.
\newblock Convex color image segmentation with optimal transport distances.
\newblock In \emph{Proceedings of SSVM'15}, pages 256--269, 2015.

\bibitem[Rabin et~al.(2011)Rabin, Peyr{\'e}, Delon, and Bernot]{rabin-ssvm-11}
Julien Rabin, Gabriel Peyr{\'e}, Julie Delon, and Marc Bernot.
\newblock Wasserstein barycenter and its application to texture mixing.
\newblock In \emph{International Conference on Scale Space and Variational
  Methods in Computer Vision}, pages 435--446. Springer, 2011.

\bibitem[Rachev and R{\"u}schendorf(1998{\natexlab{a}})]{rachev1998mass}
Svetlozar~T Rachev and Ludger R{\"u}schendorf.
\newblock \emph{Mass Transportation Problems: Volume I: Theory}.
\newblock Springer Science \& Business Media, 1998{\natexlab{a}}.

\bibitem[Rachev and R{\"u}schendorf(1998{\natexlab{b}})]{rachev1998mass2}
Svetlozar~T Rachev and Ludger R{\"u}schendorf.
\newblock \emph{Mass Transportation Problems: Volume II: Applications}.
\newblock Springer Science \& Business Media, 1998{\natexlab{b}}.

\bibitem[Rall(1981)]{rall1981automatic}
Louis~B Rall.
\newblock \emph{Automatic Differentiation: Techniques and Applications}.
\newblock Springer, 1981.

\bibitem[Ramdas et~al.(2017)Ramdas, Trillos, and Cuturi]{ramdas2017wasserstein}
Aaditya Ramdas, Nicol{\'a}s~Garc{\'\i}a Trillos, and Marco Cuturi.
\newblock On {Wasserstein} two-sample testing and related families of
  nonparametric tests.
\newblock \emph{Entropy}, 19\penalty0 (2):\penalty0 47, 2017.

\bibitem[Rangarajan et~al.(1999)Rangarajan, Yuille, Gold, and
  Mjolsness]{rangarajan-1999}
Anand Rangarajan, Alan~L Yuille, Steven Gold, and Eric Mjolsness.
\newblock Convergence properties of the softassign quadratic assignment
  algorithm.
\newblock \emph{Neural Computation}, 11\penalty0 (6):\penalty0 1455--1474,
  August 1999.

\bibitem[Reich(2013)]{reich2013nonparametric}
Sebastian Reich.
\newblock A nonparametric ensemble transform method for bayesian inference.
\newblock \emph{SIAM Journal on Scientific Computing}, 35\penalty0
  (4):\penalty0 A2013--A2024, 2013.

\bibitem[Rockafellar(1976)]{rockafellar1976monotone}
R~Tyrrell Rockafellar.
\newblock Monotone operators and the proximal point algorithm.
\newblock \emph{SIAM Journal on Control and Optimization}, 14\penalty0
  (5):\penalty0 877--898, 1976.

\bibitem[Rolet et~al.(2016)Rolet, Cuturi, and Peyr{\'e}]{pmlr-v51-rolet16}
Antoine Rolet, Marco Cuturi, and Gabriel Peyr{\'e}.
\newblock Fast dictionary learning with a smoothed {Wasserstein} loss.
\newblock In \emph{Proceedings of the 19th International Conference on
  Artificial Intelligence and Statistics}, volume~51 of \emph{Proceedings of
  Machine Learning Research}, pages 630--638, 2016.

\bibitem[Rubner et~al.(2000)Rubner, Tomasi, and Guibas]{RubTomGui00}
Yossi Rubner, Carlo Tomasi, and Leonidas~J Guibas.
\newblock The earth mover's distance as a metric for image retrieval.
\newblock \emph{International Journal of Computer Vision}, 40\penalty0
  (2):\penalty0 99--121, 2000.

\bibitem[Ruschendorf(1995)]{Ruschendorf95}
Ludger Ruschendorf.
\newblock Convergence of the iterative proportional fitting procedure.
\newblock \emph{Annals of Statistics}, 23\penalty0 (4):\penalty0 1160--1174,
  1995.

\bibitem[R{\"u}schendorf and Rachev(1990)]{ruschendorf1990characterization}
L{\"u}dger R{\"u}schendorf and Svetlozar~T Rachev.
\newblock A characterization of random variables with minimum l2-distance.
\newblock \emph{Journal of Multivariate Analysis}, 32\penalty0 (1):\penalty0
  48--54, 1990.

\bibitem[R{\"u}schendorf and Thomsen(1998)]{RuschendorfThomsen}
Ludger R{\"u}schendorf and Wolfgang Thomsen.
\newblock Closedness of sum spaces and the generalized {Schrodinger} problem.
\newblock \emph{Theory of Probability and its Applications}, 42\penalty0
  (3):\penalty0 483--494, 1998.

\bibitem[R{\"u}schendorf and Uckelmann(2002)]{RuschendorfUckelmann}
Ludger R{\"u}schendorf and Ludger Uckelmann.
\newblock On the $n$-coupling problem.
\newblock \emph{Journal of Multivariate Analysis}, 81\penalty0 (2):\penalty0
  242--258, 2002.

\bibitem[Ryu et~al.(2017{\natexlab{a}})Ryu, Chen, Li, and Osher]{Ryu2017b}
Ernest~K. Ryu, Yongxin Chen, Wuchen Li, and Stanley Osher.
\newblock Vector and matrix optimal mass transport: theory, algorithm, and
  applications.
\newblock \emph{SIAM Journal on Scientific Compututing}, 40\penalty0
  (5):\penalty0 A3675--A3698, 2017{\natexlab{a}}.

\bibitem[Ryu et~al.(2017{\natexlab{b}})Ryu, Li, Yin, and Osher]{Ryu2017a}
Ernest~K. Ryu, Wuchen Li, Penghang Yin, and Stanley Osher.
\newblock Unbalanced and partial $l^1$ {Monge-Kantorovich} problem: a scalable
  parallel first-order method.
\newblock \emph{Journal of Scientific Computing}, 75\penalty0 (3):\penalty0
  1596--1613, 2017{\natexlab{b}}.

\bibitem[Salimans et~al.(2018)Salimans, Zhang, Radford, and
  Metaxas]{salimans2018improving}
Tim Salimans, Han Zhang, Alec Radford, and Dimitris Metaxas.
\newblock Improving {GAN}s using optimal transport.
\newblock In \emph{International Conference on Learning Representations}, 2018.

\bibitem[Samelson et~al.(1957)]{samelson1957perron}
Hans Samelson et~al.
\newblock On the perron-frobenius theorem.
\newblock \emph{Michigan Mathematical Journal}, 4\penalty0 (1):\penalty0
  57--59, 1957.

\bibitem[Sandler and Lindenbaum(2011)]{Sandler09}
Roman Sandler and Michael Lindenbaum.
\newblock Nonnegative matrix factorization with earth mover's distance metric
  for image analysis.
\newblock \emph{IEEE Transactions on Pattern Analysis and Machine
  Intelligence}, 33\penalty0 (8):\penalty0 1590--1602, 2011.

\bibitem[Santambrogio(2015)]{SantambrogioBook}
Filippo Santambrogio.
\newblock \emph{Optimal transport for applied mathematicians}.
\newblock Birkhauser, 2015.

\bibitem[Santambrogio(2017)]{santambrogio2017euclidean}
Filippo Santambrogio.
\newblock $\{$Euclidean, metric, and Wasserstein$\}$ gradient flows: an
  overview.
\newblock \emph{Bulletin of Mathematical Sciences}, 7\penalty0 (1):\penalty0
  87--154, 2017.

\bibitem[Santambrogio(2018)]{SantambrogioCrowdReview}
Filippo Santambrogio.
\newblock Crowd motion and population dynamics under density constraints.
\newblock \emph{GMT preprint 3728}, 2018.

\bibitem[Saumier et~al.(2015)Saumier, Khouider, and Agueh]{saumier2015optimal}
Louis-Philippe Saumier, Boualem Khouider, and Martial Agueh.
\newblock Optimal transport for particle image velocimetry.
\newblock \emph{Communications in Mathematical Sciences}, 13\penalty0
  (1):\penalty0 269--296, 2015.

\bibitem[Schiebinger et~al.(2017)Schiebinger, Shu, Tabaka, Cleary, Subramanian,
  Solomon, Liu, Lin, Berube, Lee, et~al.]{schiebinger2017reconstruction}
Geoffrey Schiebinger, Jian Shu, Marcin Tabaka, Brian Cleary, Vidya Subramanian,
  Aryeh Solomon, Siyan Liu, Stacie Lin, Peter Berube, Lia Lee, et~al.
\newblock Reconstruction of developmental landscapes by optimal-transport
  analysis of single-cell gene expression sheds light on cellular
  reprogramming.
\newblock \emph{bioRxiv}, page 191056, 2017.

\bibitem[Schmitzer(2016{\natexlab{a}})]{schmitzer2016sparse}
Bernhard Schmitzer.
\newblock A sparse multiscale algorithm for dense optimal transport.
\newblock \emph{Journal of Mathematical Imaging and Vision}, 56\penalty0
  (2):\penalty0 238--259, 2016{\natexlab{a}}.

\bibitem[Schmitzer(2016{\natexlab{b}})]{schmitzer2016stabilized}
Bernhard Schmitzer.
\newblock Stabilized sparse scaling algorithms for entropy regularized
  transport problems.
\newblock \emph{arXiv preprint arXiv:1610.06519}, 2016{\natexlab{b}}.

\bibitem[Schmitzer and Schn{\"o}rr(2013{\natexlab{a}})]{schmitzer2013modelling}
Bernhard Schmitzer and Christoph Schn{\"o}rr.
\newblock Modelling convex shape priors and matching based on the
  {Gromov}-{Wasserstein} distance.
\newblock \emph{Journal of Mathematical Imaging and Vision}, 46\penalty0
  (1):\penalty0 143--159, 2013{\natexlab{a}}.

\bibitem[Schmitzer and Schn{\"o}rr(2013{\natexlab{b}})]{schmitzer2013object}
Bernhard Schmitzer and Christoph Schn{\"o}rr.
\newblock Object segmentation by shape matching with {Wasserstein} modes.
\newblock In \emph{International Workshop on Energy Minimization Methods in
  Computer Vision and Pattern Recognition}, pages 123--136. Springer,
  2013{\natexlab{b}}.

\bibitem[Schmitzer and Wirth(2017)]{schmitzer2017framework}
Bernhard Schmitzer and Benedikt Wirth.
\newblock A framework for {Wasserstein}-1-type metrics.
\newblock \emph{arXiv preprint arXiv:1701.01945}, 2017.

\bibitem[Schoenberg(1938)]{schoenberg38}
Isaac~J Schoenberg.
\newblock Metric spaces and positive definite functions.
\newblock \emph{Transactions of the American Mathematical Society},
  38:\penalty0 522--356, 1938.

\bibitem[Sch{\"o}lkopf and Smola(2002)]{scholkopf2002learning}
Bernhard Sch{\"o}lkopf and Alexander~J Smola.
\newblock \emph{Learning with Kernels: Support Vector Machines, Regularization,
  Optimization, and Beyond}.
\newblock MIT Press, 2002.

\bibitem[Schr\"odinger(1931)]{Schroedinger31}
Erwin Schr\"odinger.
\newblock {\"U}ber die {U}mkehrung der {N}aturgesetze.
\newblock \emph{Sitzungsberichte Preuss. Akad. Wiss. Berlin. Phys. Math.},
  144:\penalty0 144--153, 1931.

\bibitem[Seguy and Cuturi(2015)]{SeguyCuturi}
Vivien Seguy and Marco Cuturi.
\newblock Principal geodesic analysis for probability measures under the
  optimal transport metric.
\newblock In \emph{Advances in Neural Information Processing Systems 28}, pages
  3294--3302. 2015.

\bibitem[Seguy et~al.(2018)Seguy, Damodaran, Flamary, Courty, Rolet, and
  Blondel]{seguy2018large}
Vivien Seguy, Bharath~Bhushan Damodaran, R{\'e}mi Flamary, Nicolas Courty,
  Antoine Rolet, and Mathieu Blondel.
\newblock Large-scale optimal transport and mapping estimation.
\newblock In \emph{Proceedings of ICLR 2018}, 2018.

\bibitem[Shafieezadeh~Abadeh et~al.(2015)Shafieezadeh~Abadeh,
  Mohajerin~Esfahani, and Kuhn]{NIPS2015_5745}
Soroosh Shafieezadeh~Abadeh, Peyman~Mohajerin Mohajerin~Esfahani, and Daniel
  Kuhn.
\newblock Distributionally robust logistic regression.
\newblock In C.~Cortes, N.~D. Lawrence, D.~D. Lee, M.~Sugiyama, and R.~Garnett,
  editors, \emph{Advances in Neural Information Processing Systems 28}, pages
  1576--1584. 2015.

\bibitem[Shafieezadeh~Abadeh et~al.(2018)Shafieezadeh~Abadeh, Nguyen, Kuhn, and
  Mohajerin~Esfahani]{NIPS2018_8067}
Soroosh Shafieezadeh~Abadeh, Viet~Anh Nguyen, Daniel Kuhn, and Peyman~Mohajerin
  Mohajerin~Esfahani.
\newblock Wasserstein distributionally robust kalman filtering.
\newblock In S.~Bengio, H.~Wallach, H.~Larochelle, K.~Grauman, N.~Cesa-Bianchi,
  and R.~Garnett, editors, \emph{Advances in Neural Information Processing
  Systems 31}, pages 8483--8492. 2018.

\bibitem[Shirdhonkar and Jacobs(2008)]{shirdhonkar2008approximate}
Sameer Shirdhonkar and David~W Jacobs.
\newblock Approximate earth mover's distance in linear time.
\newblock In \emph{IEEE Conference on Computer Vision and Pattern Recognition},
  pages 1--8. IEEE, 2008.

\bibitem[Silverman(1986)]{silverman1986density}
Bernard~W Silverman.
\newblock \emph{Density Estimation for Statistics and Data Analysis},
  volume~26.
\newblock CRC press, 1986.

\bibitem[Sinkhorn(1964)]{Sinkhorn64}
Richard Sinkhorn.
\newblock A relationship between arbitrary positive matrices and doubly
  stochastic matrices.
\newblock \emph{Annals of Mathematical Statististics}, 35:\penalty0 876--879,
  1964.

\bibitem[Slomp et~al.(2011)Slomp, Mikamo, Raytchev, Tamaki, and
  Kaneda]{slomp2011gpu}
Marcos Slomp, Michihiro Mikamo, Bisser Raytchev, Toru Tamaki, and Kazufumi
  Kaneda.
\newblock Gpu-based softassign for maximizing image utilization in
  photomosaics.
\newblock \emph{International Journal of Networking and Computing}, 1\penalty0
  (2):\penalty0 211--229, 2011.

\bibitem[Solomon et~al.(2013)Solomon, Guibas, and
  Butscher]{solomon2013dirichlet}
Justin Solomon, Leonidas Guibas, and Adrian Butscher.
\newblock Dirichlet energy for analysis and synthesis of soft maps.
\newblock In \emph{Computer Graphics Forum}, volume~32, pages 197--206. Wiley
  Online Library, 2013.

\bibitem[Solomon et~al.(2014{\natexlab{a}})Solomon, Rustamov, Guibas, and
  Butscher]{SolomonEMDSurfaces2014}
Justin Solomon, Raif Rustamov, Leonidas Guibas, and Adrian Butscher.
\newblock Earth mover's distances on discrete surfaces.
\newblock \emph{Transaction on Graphics}, 33\penalty0 (4), 2014{\natexlab{a}}.

\bibitem[Solomon et~al.(2014{\natexlab{b}})Solomon, Rustamov, Leonidas, and
  Butscher]{Solomon-ICML}
Justin Solomon, Raif Rustamov, Guibas Leonidas, and Adrian Butscher.
\newblock Wasserstein propagation for semi-supervised learning.
\newblock In \emph{Proceedings of the 31st International Conference on Machine
  Learning}, pages 306--314, 2014{\natexlab{b}}.

\bibitem[Solomon et~al.(2015)Solomon, De~Goes, Peyr{\'e}, Cuturi, Butscher,
  Nguyen, Du, and Guibas]{2015-solomon-siggraph}
Justin Solomon, Fernando De~Goes, Gabriel Peyr{\'e}, Marco Cuturi, Adrian
  Butscher, Andy Nguyen, Tao Du, and Leonidas Guibas.
\newblock Convolutional {Wasserstein} distances: efficient optimal
  transportation on geometric domains.
\newblock \emph{ACM Transactions on Graphics}, 34\penalty0 (4):\penalty0
  66:1--66:11, 2015.

\bibitem[Solomon et~al.(2016{\natexlab{a}})Solomon, Peyr{\'e}, Kim, and
  Sra]{2016-solomon-gw}
Justin Solomon, Gabriel Peyr{\'e}, Vladimir~G Kim, and Suvrit Sra.
\newblock Entropic metric alignment for correspondence problems.
\newblock \emph{ACM Transactions on Graphics}, 35\penalty0 (4):\penalty0
  72:1--72:13, 2016{\natexlab{a}}.

\bibitem[Solomon et~al.(2016{\natexlab{b}})Solomon, Rustamov, Guibas, and
  Butscher]{solomon2016continuous}
Justin Solomon, Raif Rustamov, Leonidas Guibas, and Adrian Butscher.
\newblock Con\-tinu\-ous-flow graph transportation distances.
\newblock \emph{arXiv preprint arXiv:\-1603.\-06927}, 2016{\natexlab{b}}.

\bibitem[Sommerfeld and Munk(2018)]{sommerfeld2018inference}
Max Sommerfeld and Axel Munk.
\newblock Inference for empirical wasserstein distances on finite spaces.
\newblock \emph{Journal of the Royal Statistical Society: Series B (Statistical
  Methodology)}, 80\penalty0 (1):\penalty0 219--238, 2018.

\bibitem[Sriperumbudur et~al.(2009)Sriperumbudur, Fukumizu, Gretton,
  Sch{\"o}lkopf, and Lanckriet]{sriperumbudur2009integral}
Bharath~K Sriperumbudur, Kenji Fukumizu, Arthur Gretton, Bernhard
  Sch{\"o}lkopf, and Gert~RG Lanckriet.
\newblock On integral probability metrics,$\phi$-divergences and binary
  classification.
\newblock \emph{arXiv preprint arXiv:0901.2698}, 2009.

\bibitem[Sriperumbudur et~al.(2012)Sriperumbudur, Fukumizu, Gretton,
  Sch{\"o}lkopf, and Lanckriet]{sriperumbudur2012empirical}
Bharath~K Sriperumbudur, Kenji Fukumizu, Arthur Gretton, Bernhard
  Sch{\"o}lkopf, and Gert~RG Lanckriet.
\newblock On the empirical estimation of integral probability metrics.
\newblock \emph{Electronic Journal of Statistics}, 6:\penalty0 1550--1599,
  2012.

\bibitem[Srivastava et~al.(2015{\natexlab{a}})Srivastava, Cevher, Dinh, and
  Dunson]{srivastava2015scalable}
Sanvesh Srivastava, Volkan Cevher, Quoc Dinh, and David Dunson.
\newblock {WASP: Scalable Bayes via barycenters of subset posteriors}.
\newblock In Guy Lebanon and S.~V.~N. Vishwanathan, editors, \emph{Proceedings
  of the Eighteenth International Conference on Artificial Intelligence and
  Statistics}, volume~38 of \emph{Proceedings of Machine Learning Research},
  pages 912--920, San Diego, California, USA, 2015{\natexlab{a}}. PMLR.
\newblock URL \url{http://proceedings.mlr.press/v38/srivastava15.html}.

\bibitem[Srivastava et~al.(2015{\natexlab{b}})Srivastava, Cevher, Dinh, and
  Dunson]{srivastava2015wasp}
Sanvesh Srivastava, Volkan Cevher, Quoc Dinh, and David Dunson.
\newblock {WASP}: scalable bayes via barycenters of subset posteriors.
\newblock In \emph{Artificial Intelligence and Statistics}, pages 912--920,
  2015{\natexlab{b}}.

\bibitem[Staib et~al.(2017{\natexlab{a}})Staib, Claici, Solomon, and
  Jegelka]{NIPS2017_6858}
Matthew Staib, Sebastian Claici, Justin~M Solomon, and Stefanie Jegelka.
\newblock Parallel streaming wasserstein barycenters.
\newblock In I.~Guyon, U.~V. Luxburg, S.~Bengio, H.~Wallach, R.~Fergus,
  S.~Vishwanathan, and R.~Garnett, editors, \emph{Advances in Neural
  Information Processing Systems 30}, pages 2647--2658. 2017{\natexlab{a}}.

\bibitem[Staib et~al.(2017{\natexlab{b}})Staib, Claici, Solomon, and
  Jegelka]{staib2017parallel}
Matthew Staib, Sebastian Claici, Justin~M Solomon, and Stefanie Jegelka.
\newblock Parallel streaming wasserstein barycenters.
\newblock In I.~Guyon, U.~V. Luxburg, S.~Bengio, H.~Wallach, R.~Fergus,
  S.~Vishwanathan, and R.~Garnett, editors, \emph{Advances in Neural
  Information Processing Systems 30}, pages 2647--2658. 2017{\natexlab{b}}.

\bibitem[Stougie(2002)]{stougie2002polynomial}
Leen Stougie.
\newblock A polynomial bound on the diameter of the transportation polytope.
\newblock Technical report, TU/e, Technische Universiteit Eindhoven, Department
  of Mathematics and Computing Science, 2002.

\bibitem[Sturm(2012)]{SturmGW}
Karl-Theodor Sturm.
\newblock The space of spaces: curvature bounds and gradient flows on the space
  of metric measure spaces.
\newblock Preprint 1208.0434, arXiv, 2012.

\bibitem[Su et~al.(2015)Su, Wang, Shi, Zeng, Sun, Luo, and Gu]{su2015optimal}
Zhengyu Su, Yalin Wang, Rui Shi, Wei Zeng, Jian Sun, Feng Luo, and Xianfeng Gu.
\newblock Optimal mass transport for shape matching and comparison.
\newblock \emph{IEEE Transactions on Pattern Analysis and Machine
  Intelligence}, 37\penalty0 (11):\penalty0 2246--2259, 2015.

\bibitem[Sudakov(1979)]{sudakov1979geometric}
Vladimir~N Sudakov.
\newblock \emph{Geometric Problems in the Theory of Infinite-dimensional
  Probability Distributions}.
\newblock Number 141. American Mathematical Society, 1979.

\bibitem[Sugiyama et~al.(2017)Sugiyama, Nakahara, and
  Tsuda]{sugiyama2017tensor}
Mahito Sugiyama, Hiroyuki Nakahara, and Koji Tsuda.
\newblock Tensor balancing on statistical manifold.
\newblock \emph{arXiv preprint arXiv:1702.08142}, 2017.

\bibitem[Sulman et~al.(2011)Sulman, Williams, and Russell]{sulman2011efficient}
Mohamed~M Sulman, JF~Williams, and Robert~D Russell.
\newblock An efficient approach for the numerical solution of the
  monge--amp{\`e}re equation.
\newblock \emph{Applied Numerical Mathematics}, 61\penalty0 (3):\penalty0
  298--307, 2011.

\bibitem[Swoboda and Schn\"orr(2013)]{SchnorSegmentation}
Paul Swoboda and Christoph Schn\"orr.
\newblock Convex variational image restoration with histogram priors.
\newblock \emph{SIAM Journal on Imaging Sciences}, 6\penalty0 (3):\penalty0
  1719--1735, 2013.

\bibitem[Sz{\'e}kely and Rizzo(2004)]{szekely2004testing}
G{\'a}bor~J Sz{\'e}kely and Maria~L Rizzo.
\newblock Testing for equal distributions in high dimension.
\newblock \emph{InterStat}, 5\penalty0 (16.10), 2004.

\bibitem[Takatsu(2011)]{takatsu2011wasserstein}
Asuka Takatsu.
\newblock Wasserstein geometry of {Gaussian} measures.
\newblock \emph{Osaka Journal of Mathematics}, 48\penalty0 (4):\penalty0
  1005--1026, 2011.

\bibitem[Tan and Touzi(2013)]{TanTouzi}
Xiaolu Tan and Nizar Touzi.
\newblock Optimal transportation under controlled stochastic dynamics.
\newblock \emph{Annals of Probability}, 41\penalty0 (5):\penalty0 3201--3240,
  2013.

\bibitem[Tarjan(1997)]{Tarjan1997}
Robert~E. Tarjan.
\newblock Dynamic trees as search trees via euler tours, applied to the network
  simplex algorithm.
\newblock \emph{Mathematical Programming}, 78\penalty0 (2):\penalty0 169--177,
  1997.

\bibitem[Tartavel et~al.(2016)Tartavel, Peyr{\'e}, and
  Gousseau]{2016-tartavel-siims}
Guillaume Tartavel, Gabriel Peyr{\'e}, and Yann Gousseau.
\newblock Wasserstein loss for image synthesis and restoration.
\newblock \emph{SIAM Journal on Imaging Sciences}, 9\penalty0 (4):\penalty0
  1726--1755, 2016.

\bibitem[Thorpe et~al.(2017)Thorpe, Park, Kolouri, Rohde, and
  Slep{\v{c}}ev]{thorpe2017transportation}
Matthew Thorpe, Serim Park, Soheil Kolouri, Gustavo~K Rohde, and Dejan
  Slep{\v{c}}ev.
\newblock A transportation $l^p$ distance for signal analysis.
\newblock \emph{Journal of Mathematical Imaging and Vision}, 59\penalty0
  (2):\penalty0 187--210, 2017.

\bibitem[Tolsto{\i}(1930)]{tolstoi1930methods}
AN~Tolsto{\i}.
\newblock Metody nakhozhdeniya naimen'shego summovogo kilome-trazha pri
  planirovanii perevozok v prostranstve (russian; methods of finding the
  minimal total kilometrage in cargo transportation planning in space).
\newblock \emph{TransPress of the National Commissariat of Transportation},
  pages 23--55, 1930.

\bibitem[Tolsto{\i}(1939)]{tolstoi1939metody}
AN~Tolsto{\i}.
\newblock Metody ustraneniya neratsional'nykh perevozok priplanirovanii
  [russian; methods of removing irrational transportation in planning].
\newblock \emph{Sotsialisticheski{\i} Transport}, 9:\penalty0 28--51, 1939.

\bibitem[Trouv{\'e} and Younes(2005)]{Metamorphosis2005}
Alain Trouv{\'e} and Laurent Younes.
\newblock Metamorphoses through {Lie} group action.
\newblock \emph{Foundations of Computational Mathematics}, 5\penalty0
  (2):\penalty0 173--198, 2005.

\bibitem[Trudinger and Wang(2001)]{trudinger2001monge}
Neil~S Trudinger and Xu-Jia Wang.
\newblock On the monge mass transfer problem.
\newblock \emph{Calculus of Variations and Partial Differential Equations},
  13\penalty0 (1):\penalty0 19--31, 2001.

\bibitem[Vaillant and Glaun{\`e}s(2005)]{vaillant2005surface}
Marc Vaillant and Joan Glaun{\`e}s.
\newblock Surface matching via currents.
\newblock In \emph{Information Processing in Medical Imaging}, pages 1--5.
  Springer, 2005.

\bibitem[Varadhan(1967)]{varadhan-1967}
Sathamangalam R~Srinivasa Varadhan.
\newblock On the behavior of the fundamental solution of the heat equation with
  variable coefficients.
\newblock \emph{Communications on Pure and Applied Mathematics}, 20\penalty0
  (2):\penalty0 431--455, 1967.

\bibitem[Villani(2003)]{Villani03}
Cedric Villani.
\newblock \emph{Topics in Optimal Transportation}.
\newblock Graduate Studies in Mathematics Series. American Mathematical
  Society, 2003.
\newblock ISBN 9780821833124.

\bibitem[Villani(2009)]{Villani09}
Cedric Villani.
\newblock \emph{Optimal Transport: Old and New}, volume 338.
\newblock Springer Verlag, 2009.

\bibitem[Vogt and Lellmann(2018)]{vogt2017measure}
Thomas Vogt and Jan Lellmann.
\newblock Measure-valued variational models with applications to
  diffusion-weighted imaging.
\newblock \emph{Journal of Mathematical Imaging and Vision}, 60\penalty0
  (9):\penalty0 1482--1502, 2018.

\bibitem[Wang and Guibas(2012)]{WangECCV12OLD}
Fan Wang and Leonidas~J Guibas.
\newblock Supervised earth mover's distance learning and its computer vision
  applications.
\newblock \emph{ECCV2012}, pages 442--455, 2012.

\bibitem[Wang et~al.(2011)Wang, Ozolek, Slepcev, Lee, Chen, and
  Rohde]{wang2011optimal}
Wei Wang, John~A. Ozolek, Dejan Slepcev, Ann~B. Lee, Cheng Chen, and Gustavo~K.
  Rohde.
\newblock An optimal transportation approach for nuclear structure-based
  pathology.
\newblock \emph{IEEE Transactions on Medical Imaging}, 30\penalty0
  (3):\penalty0 621--631, 2011.

\bibitem[Wang et~al.(2013)Wang, Slep{\v{c}}ev, Basu, Ozolek, and
  Rohde]{wang2013linear}
Wei Wang, Dejan Slep{\v{c}}ev, Saurav Basu, John~A Ozolek, and Gustavo~K Rohde.
\newblock A linear optimal transportation framework for quantifying and
  visualizing variations in sets of images.
\newblock \emph{International Journal of Computer Vision}, 101\penalty0
  (2):\penalty0 254--269, 2013.

\bibitem[Weed and Bach(2017)]{weed2017sharp}
Jonathan Weed and Francis Bach.
\newblock Sharp asymptotic and finite-sample rates of convergence of empirical
  measures in {Wasserstein} distance.
\newblock \emph{arXiv preprint arXiv:1707.00087}, 2017.

\bibitem[Weinberger and Saul(2009)]{weinberger2009distance}
Kilian~Q Weinberger and Lawrence~K Saul.
\newblock Distance metric learning for large margin nearest neighbor
  classification.
\newblock \emph{Journal of Machine Learning Research}, 10:\penalty0 207--244,
  2009.

\bibitem[Westdickenberg and Wilkening(2010)]{Westdickenberg2010}
Michael Westdickenberg and Jon Wilkening.
\newblock Variational particle schemes for the porous medium equation and for
  the system of isentropic {Euler} equations.
\newblock \emph{ESAIM: Mathematical Modelling and Numerical Analysis},
  44\penalty0 (1):\penalty0 133--166, 2010.

\bibitem[Wilson(1969)]{wilson1969use}
Alan~Geoffrey Wilson.
\newblock The use of entropy maximizing models, in the theory of trip
  distribution, mode split and route split.
\newblock \emph{Journal of Transport Economics and Policy}, pages 108--126,
  1969.

\bibitem[Xia et~al.(2014)Xia, Ferradans, Peyr{\'e}, and Aujol]{2014-xia-siims}
Gui-Song Xia, Sira Ferradans, Gabriel Peyr{\'e}, and Jean-Fran{\c{c}}ois Aujol.
\newblock Synthesizing and mixing stationary {Gaussian} texture models.
\newblock \emph{SIAM Journal on Imaging Sciences}, 7\penalty0 (1):\penalty0
  476--508, 2014.

\bibitem[Yule(1912)]{yule1912methods}
G~Udny Yule.
\newblock On the methods of measuring association between two attributes.
\newblock \emph{Journal of the Royal Statistical Society}, 75\penalty0
  (6):\penalty0 579--652, 1912.

\bibitem[Zemel and Panaretos(2018)]{zemel2017fr}
Yoav Zemel and Victor~M Panaretos.
\newblock Fr\'echet means and procrustes analysis in wasserstein space.
\newblock \emph{to appear in Bernoulli}, 2018.

\bibitem[Zen et~al.(2014)Zen, Ricci, and Sebe]{ZenICPR14}
Gloria Zen, Elisa Ricci, and Nicu Sebe.
\newblock Simultaneous ground metric learning and matrix factorization with
  earth mover's distance.
\newblock \emph{Proceedings of ICPR'14}, pages 3690--3695, 2014.

\bibitem[Zhu et~al.(2007)Zhu, Yang, Haker, and Tannenbaum]{zhu2007image}
Lei Zhu, Yan Yang, Steven Haker, and Allen Tannenbaum.
\newblock An image morphing technique based on optimal mass preserving mapping.
\newblock \emph{IEEE Transactions on Image Processing}, 16\penalty0
  (6):\penalty0 1481--1495, 2007.

\bibitem[Zinsl and Matthes(2015)]{MatthesPositive}
Jonathan Zinsl and Daniel Matthes.
\newblock Transport distances and geodesic convexity for systems of degenerate
  diffusion equations.
\newblock \emph{Calculus of Variations and Partial Differential Equations},
  54\penalty0 (4):\penalty0 3397--3438, 2015.

\end{thebibliography}

\end{document}